\RequirePackage[l2tabu,orthodox]{nag}
\documentclass
[letterpaper,11pt,]
{article}

\usepackage{etex}
\usepackage{verbatim}
\usepackage{xspace,enumerate}
\usepackage[dvipsnames]{xcolor}
\usepackage[T1]{fontenc}
\usepackage[full]{textcomp}
\usepackage[american]{babel}
\usepackage{mathtools}
\usepackage{amsthm}
\usepackage[
letterpaper,
top=1in,
bottom=1in,
left=1in,
right=1in]{geometry}
\usepackage{newpxtext} 
\usepackage{textcomp} 
\usepackage[varg,bigdelims]{newpxmath}
\usepackage[scr=rsfso]{mathalfa}
\usepackage{bm} 
\let\mathbb\varmathbb
\usepackage{microtype}
\usepackage[pagebackref,colorlinks=true,urlcolor=blue,linkcolor=blue,citecolor=OliveGreen]{hyperref}
\usepackage[capitalise,nameinlink]{cleveref}
\crefname{lemma}{Lemma}{Lemmas}
\crefname{fact}{Fact}{Facts}
\crefname{theorem}{Theorem}{Theorems}
\crefname{corollary}{Corollary}{Corollaries}
\crefname{claim}{Claim}{Claims}
\crefname{example}{Example}{Examples}
\crefname{algorithm}{Algorithm}{Algorithms}
\crefname{problem}{Problem}{Problems}
\crefname{definition}{Definition}{Definitions}
\crefname{exercise}{Exercise}{Exercises}
\usepackage{amsthm}

\newtheorem{theorem}{Theorem}[section]
\newtheorem*{theorem*}{Theorem}
\newtheorem{lemma}[theorem]{Lemma}
\newtheorem*{lemma*}{Lemma}
\newtheorem{fact}[theorem]{Fact}
\newtheorem*{fact*}{Fact}

\newtheorem*{proposition*}{Proposition}
\newtheorem{corollary}[theorem]{Corollary}
\newtheorem*{corollary*}{Corollary}

\newtheorem*{hypothesis*}{Hypothesis}

\newtheorem*{conjecture*}{Conjecture}
\theoremstyle{definition}
\newtheorem{definition}[theorem]{Definition}
\newtheorem*{definition*}{Definition}

\newtheorem*{construction*}{Construction}

\newtheorem*{example*}{Example}

\newtheorem*{question*}{Question}
\newtheorem{algorithm}[theorem]{Algorithm}
\newtheorem*{algorithm*}{Algorithm}

\newtheorem*{assumption*}{Assumption}
\newtheorem{problem}[theorem]{Problem}
\newtheorem*{problem*}{Problem}

\newtheorem*{openquestion*}{Open Question}
\theoremstyle{remark}

\newtheorem*{claim*}{Claim}
\newtheorem{remark}[theorem]{Remark}
\newtheorem*{remark*}{Remark}

\newtheorem*{observation*}{Observation}
\usepackage{paralist}
\let\originalleft\left
\let\originalright\right
\renewcommand{\left}{\mathopen{}\mathclose\bgroup\originalleft}
\renewcommand{\right}{\aftergroup\egroup\originalright}
\usepackage{turnstile}
\usepackage{mdframed}
\usepackage{tikz}
\usetikzlibrary{positioning}
\usepackage{caption}
\DeclareCaptionType{Algorithm}
\usepackage{newfloat}
\usepackage{array}
\usepackage{subfig}
\usepackage{bbm}
\usepackage{xparse}
\usepackage{amsthm} 
\makeatletter
\let\latexparagraph\paragraph
\RenewDocumentCommand{\paragraph}{som}{%
  \IfBooleanTF{#1}
    {\latexparagraph*{#3}}
    {\IfNoValueTF{#2}
       {\latexparagraph{\maybe@addperiod{#3}}}
       {\latexparagraph[#2]{\maybe@addperiod{#3}}}%
  }%
}
\newcommand{\maybe@addperiod}[1]{%
  #1\@addpunct{.}%
}
\makeatother

\newcommand{\Authornote}[2]{}
\newcommand{\Authornotecolored}[3]{}
\newcommand{\Authorcomment}[2]{}
\newcommand{\Authorfnote}[2]{}

\newcommand{\Tnote}{\Authornote{T}}

\usepackage{boxedminipage}
\newcommand{\paren}[1]{(#1)}
\newcommand{\Paren}[1]{\left(#1\right)}

\newcommand{\brac}[1]{[#1]}
\newcommand{\Brac}[1]{\left[#1\right]}

\newcommand{\Bigbrac}[1]{\Big[#1\Big]}

\newcommand{\abs}[1]{\lvert#1\rvert}
\newcommand{\Abs}[1]{\left\lvert#1\right\rvert}

\newcommand{\card}[1]{\lvert#1\rvert}
\newcommand{\Card}[1]{\left\lvert#1\right\rvert}

\newcommand{\set}[1]{\{#1\}}
\newcommand{\Set}[1]{\left\{#1\right\}}

\newcommand{\norm}[1]{\lVert#1\rVert}
\newcommand{\Norm}[1]{\left\lVert#1\right\rVert}

\newcommand{\Snorm}[1]{\Norm{#1}^2}

\newcommand{\Normo}[1]{\Norm{#1}_1}

\newcommand{\Normi}[1]{\Norm{#1}_\infty}

\newcommand{\iprod}[1]{\langle#1\rangle}
\newcommand{\Iprod}[1]{\left\langle#1\right\rangle}

\newcommand{\Esymb}{\mathbb{E}}

\DeclareMathOperator*{\E}{\Esymb}

\newcommand{\given}{\mathrel{}\middle\vert\mathrel{}}
\newcommand{\suchthat}{\;\middle\vert\;}

\newcommand{\sge}{\succeq}

\renewcommand{\ij}{{ij}}

\newcommand{\sbits}{\{\pm1\}}

\newcommand{\super}[2]{#1^{\paren{#2}}}

\newcommand{\seteq}{\mathrel{\mathop:}=}

\newcommand\bdot\bullet

\DeclareMathOperator{\Tr}{Tr}

\DeclareMathOperator{\opt}{opt}

\newcommand{\Erdos}{Erd\H{o}s\xspace}
\newcommand{\Renyi}{R\'enyi\xspace}

\newcommand{\Hoelder}{H\"{o}lder\xspace}
\newcommand{\Holder}{\Hoelder}

\newcommand{\N}{\mathbb N}
\newcommand{\R}{\mathbb R}

\newcommand{\cA}{\mathcal A}
\newcommand{\cB}{\mathcal B}

\newcommand{\cD}{\mathcal D}

\newcommand{\cF}{\mathcal F}
\newcommand{\cG}{\mathcal G}

\newcommand{\cL}{\mathcal L}
\newcommand{\cM}{\mathcal M}

\newcommand{\cS}{\mathcal S}
\newcommand{\cT}{\mathcal T}

\newcommand{\cW}{\mathcal W}

\newcommand{\bbP}{\mathbb P}

\renewcommand{\leq}{\leqslant}
\renewcommand{\le}{\leqslant}
\renewcommand{\geq}{\geqslant}
\renewcommand{\ge}{\geqslant}
\let\epsilon=\varepsilon
\numberwithin{equation}{section}
\newcommand\MYcurrentlabel{xxx}
\newcommand{\MYstore}[2]{%
  \global\expandafter \def \csname MYMEMORY #1 \endcsname{#2}%
}
\newcommand{\MYload}[1]{%
  \csname MYMEMORY #1 \endcsname%
}
\newcommand{\MYnewlabel}[1]{%
  \renewcommand\MYcurrentlabel{#1}%
  \MYoldlabel{#1}%
}
\newcommand{\MYdummylabel}[1]{}
\newcommand{\torestate}[1]{%
  \let\MYoldlabel\label%
  \let\label\MYnewlabel%
  #1%
  \MYstore{\MYcurrentlabel}{#1}%
  \let\label\MYoldlabel%
}
\newcommand{\restatetheorem}[1]{%
  \let\MYoldlabel\label
  \let\label\MYdummylabel
  \begin{theorem*}[Restatement of \cref{#1}]
    \MYload{#1}
  \end{theorem*}
  \let\label\MYoldlabel
}
\newcommand{\restatelemma}[1]{%
  \let\MYoldlabel\label
  \let\label\MYdummylabel
  \begin{lemma*}[Restatement of \cref{#1}]
    \MYload{#1}
  \end{lemma*}
  \let\label\MYoldlabel
}
\newcommand{\restateprop}[1]{%
  \let\MYoldlabel\label
  \let\label\MYdummylabel
  \begin{proposition*}[Restatement of \cref{#1}]
    \MYload{#1}
  \end{proposition*}
  \let\label\MYoldlabel
}
\newcommand{\restatefact}[1]{%
  \let\MYoldlabel\label
  \let\label\MYdummylabel
  \begin{fact*}[Restatement of \cref{#1}]
    \MYload{#1}
  \end{fact*}
  \let\label\MYoldlabel
}
\newcommand{\restate}[1]{%
  \let\MYoldlabel\label
  \let\label\MYdummylabel
  \MYload{#1}
  \let\label\MYoldlabel
}

\newcommand{\e}{\epsilon}
\newcommand{\eps}{\epsilon}

\allowdisplaybreaks
\newcommand*{\Id}{\mathrm{Id}}

\newcommand*{\Normf}[1]{\Norm{#1}_{\mathrm{F}}}

\newenvironment{algorithmbox}{\begin{mdframed}[nobreak=true]
\begin{algorithm}}{\end{algorithm}\end{mdframed}}
\newcommand{\bsaw}[2]{\text{BSAW}_{#1,#2}}
\newcommand{\nbsaw}[2]{\text{NBSAW}_{#1,#2}}
\newcommand{\normn}[1]{\norm{#1}_\textnormal{nuc}}
\newcommand{\Normn}[1]{\Norm{#1}_\textnormal{nuc}}

\newcommand{\ind}[1]{\mathbf{1}_{\Brac{#1}}}

\providecommand{\Yij}{\mathbf{Y}_{ij}}

\newcommand{\saw}[2]{\text{SAW}^{#2}_{#1}}

\newcommand{\sbm}{\SBM_n(d,\e)}
\newcommand{\SBM}{\mathsf{SBM}}

\newcommand{\Q}{Q^{(s)}}

\renewcommand{\Normo}[1]{\Norm{#1}_{\text{sum}}}

\renewcommand{\Normi}[1]{\Norm{#1}_{\max}}

\newcommand*{\transpose}[1]{{#1}{}^{\mkern-1.5mu\mathsf{T}}}
\newcommand*{\dyad}[1]{#1#1{}^{\mkern-1.5mu\mathsf{T}}}

\title{
  Robust recovery for stochastic block models\thanks{This project has received funding from the European Research Council (ERC) under the European Union’s Horizon 2020 research and innovation programme (grant agreement No 815464)
}
}

\author{
  Jingqiu Ding\thanks{ETH Z\"urich.
}
  \and
  Tommaso d'Orsi\thanks{ETH Z\"urich.
}
  \and
  Rajai Nasser\thanks{ETH Z\"urich.
}
  \and
  David Steurer\thanks{ETH Z\"urich.
}
}

\begin{document}

\pagestyle{empty}


\maketitle
\thispagestyle{empty} 


\begin{abstract}

We develop an efficient algorithm for weak recovery in a robust version of the stochastic block model.
The algorithm matches the statistical guarantees of the best known algorithms for the vanilla version of the stochastic block model.
In this sense, our results show that there is no price of robustness in the stochastic block model.

Our work is heavily inspired by recent work of Banks, Mohanty, and Raghavendra (SODA 2021) that provided an efficient algorithm for the corresponding distinguishing problem.

Our algorithm and its analysis significantly depart from previous ones for robust recovery.
A key challenge is the peculiar optimization landscape underlying our algorithm:
The planted partition may be far from optimal in the sense that completely unrelated solutions could achieve the same objective value.
This phenomenon is related to the push-out effect at the BBP phase transition for PCA.
To the best of our knowledge, our algorithm is the first to achieve robust recovery in the presence of such a push-out effect in a non-asymptotic setting.

Our algorithm is an instantiation of a framework based on convex optimization (related to but distinct from sum-of-squares), which may be useful for other robust matrix estimation problems.

A by-product of our analysis is a general technique that boosts the probability of success (over the randomness of the input) of an arbitrary robust weak-recovery algorithm from constant (or slowly vanishing) probability to exponentially high probability.

\end{abstract}

\clearpage


\microtypesetup{protrusion=false}
\tableofcontents{}
\microtypesetup{protrusion=true}

\clearpage

\pagestyle{plain}
\setcounter{page}{1}


\section{Introduction}\label{sec:introduction}

The stochastic block model is an extensively studied statistical model for community detection in graphs.
(For a broad overview, see the excellent survey \cite{DBLP:journals/corr/Abbe17}.)
In its most basic form, the stochastic block model describes the following joint distribution \((\mathbf{x},\mathbf{G})\sim \SBM_n(d,\e)\) between a vector \(x\) of \(n\) binary\footnote{
  More general versions of the stochastic block model allow for more than two labels, non-uniform probabilities for the labels, and general edge probabilities depending on the label assignment.
  However, many of the algorithmic phenomena of the general version can in their essence already be observed for the basic version that we consider in this work.
} labels and an \(n\)-vertex graph \(\mathbf{G}\):
\begin{itemize}
\item draw a vector~\(\mathbf{x}\in \sbits^n\) uniformly at random,
\item for every pair of distinct vertices~\(i,j\in [n]\), independently create an edge~\(\{i,j\}\) in the graph~\(\mathbf{G}\) with probability~\((1+\frac{\e}{2}\cdot \mathbf{x}_i\cdot \mathbf{x}_j)\cdot \frac d n\).
\end{itemize}
Note that for distinct vertices \(i,j\in[n]\), the edge \(\{i,j\}\) is present in \(\mathbf{G}\) with probability \((1+\frac{\eps}{2})\cdot \frac d n \) if the vertices have the same label \(\mathbf{x}_i=\mathbf{x}_j\) and with probability \((1-\frac{\eps}{2})\cdot \frac d n\) if the vertices have different labels \(\mathbf{x}_i\neq \mathbf{x}_j\).

Given a graph \(\mathbf{G}\) sampled according to this model, the goal is to recover the (unknown) underlying vector of labels as well as possible.

\paragraph{Weak recovery}

We say that an algorithm achieves \emph{(weak) recovery} for the stochastic block model \(\{\SBM_n(d,\e)\}_{n\in \N}\) if the correlation of the algorithm's output \(\hat x(\mathbf{G})\in \sbits^n\) and the underlying vector \(\mathbf{x}\) of labels is bounded away from zero as \(n\) grows,\footnote{We remark that other definitions of weak recovery require that the algorithm achieves constant correlation with probability tending to \(1\) as \(n\) grows.
It turns out that in the robust setting considered in this paper it is always possible to boost the success probability from \(\Omega(1)\) to \(1-o(1)\). See \cref{sec:techniques}.
}
\begin{equation}
  \label{eq:weak-recovery}
  \E_{(\mathbf{x},\mathbf{G})\sim \SBM_n(d,\e)}\Bigbrac{\tfrac 1n\abs{\iprod{x,\hat x(\mathbf{G})}}} \ge \Omega_{\e,d}(1)\,.
\end{equation}
(Here, \(\Omega_{\e,d}(1)\) hides a positive number depending on \(\e\) and \(d\) but independent of \(n\)).
A series of seminal works \cite{Mossel_2014,DBLP:journals/corr/MosselNS13a,DBLP:conf/stoc/Massoulie14} showed that weak recovery is possible (also computationally efficiently) if and only if \(d>4/\e^2\) confirming a conjecture \cite{Decelle_2011} from statistical physics (this threshold is commonly referred to as the Kesten-Stigum threshold).\footnote{For $k$ communities the threshold is $d>k^2/\e^2$. }

\paragraph{Robustness}

A fascinating issue that arises in the context of the works on weak recovery is \emph{robustness}:
The algorithms used to show that weak recovery is possible when \(d>4/\e^2\) are fragile in the sense that adversarially modifying a vanishing fraction of edges could fool the algorithm into outputting labels completely unrelated to the true labels.
The reason is that these algorithms are based on particular kinds of random walks (self-avoiding or non-backtracking) that can be affected disproportionally by adding small cliques or other dense subgraphs.

Indeed, other kinds of algorithms (based on semidefinite programming) have stronger guarantees in robust settings \cite{Feige_2001, guedon14:_commun_groth}.
However, these kinds of algorithms are only known to work for \(d>C\cdot 4/\e^2\) for an absolute constant \(C>1\) (even in the non-robust setting).\footnote{
For high-degree graphs, the best-known bound for the basic semidefinite programming relaxation to achieve weak-recovery is of the form \(\frac{\e^2}{4} d > 1+o_d(1)\) for an unspecified function \(o_d(1)\) tending to 0 as \(d\) grows \cite{montanari15:_semid_progr_spars_random_graph}.}

Taken together these results raise the question whether, at the threshold \(d>4/\e^2\), it is inherently impossible to achieve weak recovery robustly.
Indeed, certain strong adversarial models (``monotone adversaries''\footnote{A \emph{monotone adversary} is allowed to change an arbitrary number of edges as long as each change increases the likelihood of the planted labeling}) were shown to change the threshold and require a bound of the form \(d > C\cdot 4/\e^2\) for \(C>1\) to ensure that weak recovery remains possible \cite{DBLP:conf/stoc/MoitraPW16}.
However, for a natural class of adversaries that are allowed to alter any vanishing fraction of edges, it remained open whether the threshold for weak recovery changes, or whether weak recovery robust against this class of adversaries is possible.
In the asymptotic setting \(d\to \infty\), this kind of robustness was achieved using the basic semidefinite programming relaxation for the likelihood maximization problem \cite{montanari15:_semid_progr_spars_random_graph}.
However, non-rigorous statistical-physics calculations \cite{Javanmard_2016} suggest that the same algorithm cannot achieve the threshold for constant degree parameter \(d\).

Groundbreaking recent work \cite{banks19:_local_statis_semid_progr_commun_detec} developed a polynomial-time algorithm for the corresponding distinguishing problem:
Given a graph drawn from \(\SBM_n(d,\e)\) with \(d>4/\e^2\) and an  \Erdos-\Renyi  random graph with the same expected number of edges, the algorithm can distinguish between the two graphs and robustly so, i.e., even after altering a small constant fraction of edges.
Similarly to previous algorithms achieving the threshold \(d>4/\e^2\), the algorithm takes into account certain kinds of random walks (specifically non-backtracking ones).
A crucial difference is that the algorithm considers only walks of constant length (for fixed \(d\) and \(\e\)) as opposed to walks of logarithmic length like previous algorithms.
In order to leverage the more limited information provided by shorter walks, the algorithm in \cite{banks19:_local_statis_semid_progr_commun_detec} needs to employ heavier convex optimization techniques (specifically sum-of-squares).
The upside is that robustness follows almost directly: The altered edges can affect only a small constant fraction of the (constant-length) walks considered by the algorithm.
Correspondingly, the effect on the optimal value for the optimization is small.

\subsection{Result}

We say that an algorithm that given a graph \(\mathbf{G}\) outputs an estimate \(\hat x(\mathbf{G})\) for the community labels of \(\mathbf{G}\) achieves \(\rho\)-robust weak recovery for \(\set{\SBM_{n}(d,\e)}_{n\in\N}\) if
\begin{equation}
  \label{eq:robust-weak-recovery}
  \E_{(\mathbf{x},\mathbf{G})\sim \SBM_n(d,\e)}\min_{G^\circ\in N_\rho(\mathbf{G})} \Bigbrac{\tfrac 1n \abs{\iprod{\mathbf{x},\hat x(G^\circ)}}} \ge \Omega_{d,\e}(1)\ ,
\end{equation}
where \(N_{\rho}(\mathbf{G})\) is the set of graphs \(G^\circ\) that can be obtained from \(\mathbf{G}\) by changing at most a \(\rho\)-fraction of its edges\footnote{That is, each $G^\circ$ can be obtained from $\mathbf{G}$ through a sequence of $\rho\cdot \Card{E(\mathbf{G})}$ edits, each consisting of an addition or deletion.} (so that \(\card{E(\mathbf{G})\mathop{\triangle} E(G^\circ)}\le \rho\cdot\paren{ \card{E(\mathbf{G})} + \card{E(G^\circ)}} \)).

The following theorem shows that this notion of robustness does not alter the statistical threshold and that robust polynomial-time algorithms exist that work all the way up to this threshold.

\begin{theorem}\label{thm:main}
	For every \(\e,d\) with \(d>4/\e^2\), there exists \(\rho>0\) such that \(\rho\)-robust weak recovery for \(\set{\SBM_{n}(d,\e)}_{n\in\N}\) is possible.
  	Moreover, the underlying algorithm runs in polynomial time.
\end{theorem}

We present a formal statement in \cref{thm:sbm-recovery-technical-vector}.

\subsection{Organization}
\Tnote{To update if the organization changes}
The rest of the paper is organized as follows.
In \cref{sec:techniques} we introduce the main ideas and techniques developed to prove \cref{thm:main}.
Preliminary notions are discussed in \cref{sec:preliminaries}. In \cref{sec:scale-free-recovery-algorithm} we present our general framework for robust algorithms, we then apply it to the stochastic block model in \cref{sec:sbm-recovery}.
In \cref{sec:bounds-moments-Q} we prove the probabilistic results needed for the algorithm to succeed.
We present most of the technical probabilistic and combinatorial details  through \cref{sec:appendix-lower-bound-non-centered}, \cref{sec:appendix-upper-bound-centered}, \cref{sec:technical-lemmas-trace-bounds} and \cref{sec:tools-bsaws}.

\section{Techniques}\label{sec:techniques}

Before explaining our new ideas, we briefly discuss related prior approaches for weak-recovery in stochastic block models.

\paragraph{Basic semidefinite programming approach and robust recovery away from the Kesten-Stigum threshold}

The following approach based on semidefinite programming is known to have strong robustness properties when the degree parameter \(d\) exceeds the KS threshold by a large enough constant factor \cite{guedon14:_commun_groth}.\footnote{%
	As discussed earlier, in the asymptotic regime \(d\to \infty\) better guarantees for this approach are known \cite{montanari15:_semid_progr_spars_random_graph}.
	However, these results have no bearing on constant degree parameters.}

Let \((\mathbf{x}, \mathbf{G}) \sim \SBM_n(d,\e)\) with \(\e>0\)\footnote{
	We remark that it also makes sense to consider stochastic block models with negative bias parameter \(\e\) (sometimes called the anti-ferromagnetic case).
} 
and let \(\mathbf{Y}\) be its centered adjacency matrix, so that \(\mathbf{Y}_{ij}=1-d/n\) if \(ij\in E(\mathbf{G})\), \(\mathbf{Y}_{ii}=0\), and \(\mathbf{Y}_{ij}=-d/n\) otherwise.
This matrix satisfies \(\E \mathbf{Y} = 0\) and, up to a scaling factor, its conditional expectation agrees with \(\dyad{\mathbf{x}}\) on all off-diagonal entries,
\begin{equation}
	\label{eq:unbiased-adjacency-matrix}
	\bar{\mathbf{Y}} \seteq \E \brac{\mathbf{Y} \mid \mathbf{x}} = \tfrac{\e d}{2n} \cdot (\dyad {\mathbf{x}} - \Id_n)\,.
\end{equation}
This property gives us hope that the maximizer \(\hat {\mathbf{x}}\) of \(\iprod{\mathbf{Y}, \dyad{x}}\) over all \(x\in\set{\pm 1}^n\) is correlated with the underlying labeling \(\mathbf{x}\).\footnote{This optimization problem is closely related to the likelihood maximization problem for the stochastic block model.
}
Concretely, the optimal value of this optimization problem is at least the value achieved by the planted labeling \(\mathbf{x}\), which is \(o(n)\)-close to its expectation \(\iprod{\bar{\mathbf{Y}},\dyad{\mathbf{x}}}=\tfrac {\e d }{2} \cdot (n-1)\).
Hence, the maximizer \(\hat{\mathbf{x}}\) satisfies the inequality
\[
\tfrac{\e d }{2} \cdot n - o(n)
\le \iprod{\mathbf{Y}, \dyad{\hat{\mathbf{x}}}}
= \tfrac{\e d }{2}  \cdot  n \cdot \Bigl (\tfrac 1 {n^2}\iprod{\mathbf{x},\hat{\mathbf{x}}}^2 - \tfrac 1 n \Bigr)
+ \iprod{\mathbf{Y}-\bar{\mathbf{Y}}, \dyad{\hat{\mathbf{x}}}}
\,.
\]
This inequality allows us to conclude that \(\hat{\mathbf{x}}\) achieves the desired correlation \(\tfrac 1 {n^2}\iprod{\mathbf{x},\hat{\mathbf{x}}}^2\ge \Omega(1)\) as long as we have an upper bound on \(\iprod{\mathbf{Y}-\bar{\mathbf{Y}}, \dyad{\hat{\mathbf{x}}}}\) smaller than \(\tfrac {\e d }{2} \cdot n\) by a constant factor.
This approach is pursued by \cite{guedon14:_commun_groth}, who proved that with high probability, for some constant factor \(C> 1\),
\begin{equation}
	\max_{x\in\set{\pm 1}^n}\abs{\iprod{\mathbf{Y}-\bar{\mathbf{Y}}, \dyad{x}}}
	\le C \cdot \sqrt {d\,}\cdot n
	\,.
	\label{eq:cut-norm}
\end{equation}
It follows that this estimator achieves weak recovery if \(\tfrac{\e \sqrt d}{2 C}\) exceeds \(1\) by a constant.
(Since \(C>1\), this bound fails to approach the Kesten-Stigum threshold.\footnote{We remark that an exact analysis of the maximum value of \(\abs{\iprod{\mathbf{Y}-\bar{\mathbf{Y}}, \dyad{x}}}\le C \sqrt {d\,}\cdot n\) over all \(x\in\set{\pm 1}^n\) is very challenging.
	This maximum value is related to the coefficient of the second-order term of the maximum cut in an \Erdos--\Renyi graph with degree parameter \(d\).
	This coefficient has been analyzed only in the asymptotic regime \(d\to \infty\) 
	\cite{DBLP:journals/corr/DemboMS15}.
	Even in this simplified setting, the coefficient corresponds to \(C>1\).
})

So far, the discussed approach is not computationally efficient (the underlying optimization problem is NP-hard).
To remedy this issue, we consider the following semidefinite programming relaxation of the problem,
\begin{equation}
	\label{eq:basic-sdp}  
	\text{maximize }\iprod{\mathbf{Y}, X}
	\text{ subject to } X\succeq 0\,,~\forall i\,.~X_{ii}=1\,.
\end{equation}
We refer to this relaxation as \emph{basic SDP}.\footnote{We remark that the famous Goemans--Williamson approximation algorithm for the max-cut problem \cite{DBLP:journals/jacm/GoemansW95} uses essentially the same basic SDP relaxation.}
To show that its optimal solution \(\hat {\mathbf{X}}\) achieves constant correlation \(\tfrac 1{n^2} \iprod{\hat{\mathbf{X}},\dyad{\mathbf{x}}}\ge \Omega(1)\), we can imitate the previous analysis (as done in \cite{guedon14:_commun_groth}).
The main difference is that we need to upper bound the deviation term \(\abs{\iprod{\mathbf{Y}-\bar{\mathbf{Y}},X}}\) uniformly over all feasible solutions \(X\) to the basic SDP (instead of just over all cut matrices \(\dyad x\) for \(x\in\set{\pm 1}^n\)).
Here, Grothendieck's inequality \Tnote{Cite Grothendieck} \cite{DBLP:journals/corr/abs-1108-2464, DBLP:conf/stoc/AlonN04} turns out to imply a bound that is at most a constant factor worse than the bound we had before (for cut matrices).

A key benefit of this kind of analysis (observed in early works on semirandom graph problems \cite{Feige_2001}; see also \cite{DBLP:conf/stoc/MoitraPW16}) is that it directly implies strong robustness guarantees.
The reason is that we used only one property of the estimator \(\hat{\mathbf{X}}\) (or \(\hat {\mathbf{x}}\)):
it is a feasible solution to our optimization problem with objective value at least as high as the planted solution \(\dyad{\mathbf{x}}\) (up to potentially a small fudge factor).
In other words, if \(\e \sqrt d/2 \) is a large enough constant for the above analysis of the basic SDP to succeed, then there exists some \(\rho>0\) (independent of \(n\)) such that with high probability (over \(\mathbf{Y}\)), every solution \(X\) to the basic SDP with objective value
\begin{equation}
	\label{eq:objective-value-at-least-planted}
	\iprod{\mathbf{Y},X}\ge \iprod{\mathbf{Y},\dyad{\mathbf{x}}} - \rho\cdot n
\end{equation}
achieves constant correlation \(\iprod{X, \dyad{\mathbf{x}}}\ge \Omega(1)\cdot n^2\).

\emph{Why does this property imply robustness?}
Suppose that \(Y'\) is the centered adjacency matrix of some corrupted versions \(G'\) of \(\mathbf{G}\) (according to some adversarial model).
Let \(\hat X'\) be the optimal solution to the basic SDP with \(\mathbf{Y}\) replaced by \(Y'\) in the objective function.
Since the planted solution \(\dyad{\mathbf{x}}\) cannot have a higher objective value than \(\hat X'\), it holds \(\iprod{Y',\hat X'} \ge \iprod{Y',\dyad{\mathbf{x}}}\).
Hence, it suffices to verify that for the adversarial model of interest, the  inequality \(\iprod{Y',X} \ge \iprod{Y',\dyad{\mathbf{x}}}\) for a feasible solution \(X\) to the basic SDP implies the previous inequality \cref{eq:objective-value-at-least-planted}.
For monotone adversaries, this implication holds (even without the fudge term \(\rho \cdot n\)) because for every monotone edge alteration,\footnote{
	A monotone edge alteration consists of either adding an edge between vertices with the same planted label or deleting an edge between vertices with different planted labels.
} the objective function increases for the planted solution by at least as much as for \(X\).
Even in the non-monotone case, every edge alteration (insertion or deletion) can change the objective function by at most \(2\) for the planted solution or for \(X\).
Hence, if we allow up to \(\rho \cdot n / 4\) edge alterations in our adversarial model, the desired implication holds.

\paragraph{Fragile recovery up to the Kesten--Stigum threshold using convex optimization}

In the non-robust setting, several algorithms are known to achieve weak-recovery all the way up to the Kesten-Stigum threshold \cite{DBLP:journals/corr/MosselNS13a,DBLP:conf/stoc/Massoulie14, DBLP:conf/focs/BordenaveLM15, DBLP:conf/nips/AbbeS16, DBLP:conf/focs/HopkinsS17}. 
The analyses of all these algorithms involve statistics of certain kinds of walks in graphs (e.g., self-avoiding, non-backtracking, or shortest walks).
Among these algorithms, our work is most closely related to the algorithm in \cite{DBLP:conf/focs/HopkinsS17}, which combines walk-statistics and convex-optimization techniques similar to those discussed earlier in the context of robustness.
For a parameter \(s\in\N\), this algorithm considers a random \(n\)-by-\(n\) matrix \(\super{\mathbf{Q}}{s}\) such that each off-diagonal entry is obtained by evaluating a (deterministic) multivariate degree-\(s\) polynomial at the centered adjacency matrix \(\mathbf{Y}\) of \(\mathbf{G}\).
Concretely, up to a scaling factor depending only on \(n,d,\e,s\),
\[
\super{\mathbf{Q}}{s}_{ij}
\propto \sum_{W\in \saw{ij}{s}} \mathbf{Y}_W\,
\text{ where }\mathbf{Y}_W \seteq \prod_{uv\in W} \mathbf{Y} _{uv}
\,.
\]
Here, \(\saw{ij}{s}\) consists of all length-$s$ self-avoiding walks between \(i\) and \(j\) in the complete graph on \(n\) vertices.
Since the entries of \(\mathbf{Y}\) are independent when conditioned on \(\mathbf{x}\), by our earlier observation \cref{eq:unbiased-adjacency-matrix} we have for all \(W\in \saw{ij}{s}\) with \(i\neq j\),
\[
\E \brac{\mathbf{Y}_ W \mid \mathbf{x}}=\prod_{uv\in W} \E\brac{\mathbf{Y}_{uv} \mid \mathbf{x}}=\paren{\tfrac {\e d}{2n}}^s \cdot \mathbf{x}_i \mathbf{x}_j
\,.
\]
Hence, if we choose the diagonal entries of \(\super{\mathbf{Q}} s\) to be \(0\) (like for \(\mathbf{Y}\)) and the aforementioned scaling factor to be  \(\paren{\tfrac{2n}{\e d}}^s/\card{\saw{ij}s}\), then \(\super{\mathbf{Q}} s\) is an unbiased estimator similar to \(\mathbf{Y}\),
\[
\E\brac{\super{\mathbf{Q}}s\mid \mathbf{x}} = \dyad{\mathbf{x}}-\Id_n
\,.
\]
Note that \(\super{\mathbf{Q}}1\) is up to a scaling factor equal to \(\mathbf{Y}\).

\emph{What can we gain from \(\super{\mathbf{Q}}s\) for \(s>1\)?}
Here, it is instructive to compare the variance of entries of the matrices (conditioned on \(\mathbf{x}\)).
In \(\super{\mathbf{Q}}1\) each entry has conditional variance about \(\paren{\tfrac{2n}{\e d}}^2\cdot \tfrac d n=\tfrac {4} {d\e^2} \cdot n\), substantially larger than the magnitude of its conditional expectation, which is \(1\).
Now suppose we are barely above the Kesten--Stigum threshold so that \(\tfrac{\e^2d}{4}=1+\delta\) for some \(\delta>0\) independent of \(n\).
Then, as we increase \(s\), the conditional variance decreases\footnote{The proof that the variance decreases argues that the unbiased estimators summed in an entry of \(\super{\mathbf{Q}}s\) behave similarly to pairwise independent estimators \cite{DBLP:journals/corr/MosselNS13a,DBLP:conf/focs/HopkinsS17}.} (roughly like \(\paren{\tfrac{4 \e^2}{d}}^{s}\cdot n=n/(1+\delta)^s\)) while the conditional expectation stays the same.
Indeed, it turns out that we can choose \(s\le O_{\delta}(\log n)\) such that the conditional variance of an entry of \(\super{\mathbf{Q}}s\) is bounded by a constant \(O_\delta(1)\) and thus has the same order as its conditional expectation.

This property implies the following inequality for a small enough \(\delta'>\Omega_\delta(1)\),
\begin{equation}
	\label{eq:frobenius-correlation}
	\E \norm{\super{\mathbf{Q}} s - \dyad{\mathbf{x}}}_F^2
	\le (1-\delta')\cdot \norm{\super{\mathbf{Q}} s}_F^2
	\,.
\end{equation}
Based on this notion of correlation, the algorithm\footnote{
	The description of the algorithm in \cite{DBLP:conf/focs/HopkinsS17} is slightly different from our description in that part of the objective function we stated appears as a constraint in \cite{DBLP:conf/focs/HopkinsS17}.
	For an appropriate choice of \(\lambda\) both versions are equivalent.
} in \cite{DBLP:conf/focs/HopkinsS17} considers the following (tractable) convex optimization problem for a regularization parameter \(\lambda>0\) (together with a simple rounding algorithm),
\begin{equation}
	\label{eq:hs17-relaxation}
	\text{maximize }
	\iprod{\super{\mathbf{Q}} s, X} - \lambda\cdot \norm{X}_F^2
	\text{ subject to }
	X\succeq 0\,,~\forall i\,.~X_{ii}=1\,.
\end{equation}
We remark that this algorithm is an instance of an emerging meta-algorithm for estimation problems that is based on sum-of-squares and specified in terms of low-degree polynomials in two kinds of variables, ``instance variables'' for the input and ``solution variables'' for the desired output \cite{DBLP:conf/focs/HopkinsKPRSS17,DBLP:conf/focs/HopkinsS17,RSS18}.

\emph{Could this algorithm be robust?}
A key property of the basic SDP \cref{eq:basic-sdp} for its robustness analysis is that every edge alteration can change the objective value (for a particular feasible solution) by at most \(O_{\e,d}(1/n)\) times the original objective value of the planted solution.
In contrast, a single edge alteration can change the objective value\footnote{
	For the purposes of this argument, we can think of \(\super {\mathbf{Q}} s_{ij}\) as \(n\cdot (2/\e)^s\) times the average of \(\mathbf{Y}_W\) over all self-avoiding walks \(W\) between \(i\) and \(j\) in \(\mathbf{G}\) (as opposed to walks in the complete graph).
} of a feasible solution in \cref{eq:hs17-relaxation} by as much as about \(n\cdot (2/\e)^s\).
Since the original objective value of the planted solution is close to its expectation \(\E \iprod{\super{\mathbf{Q}}{s},\dyad{\mathbf{x}}} = n\cdot(n-1)\), we can afford this sensitivity of the objective function only for constant \(s\le O_{\e,d}(1)\).
Unfortunately, the correlation \cref{eq:frobenius-correlation} required for the analysis of \cref{eq:hs17-relaxation} can only be achieved for \(s\) logarithmic in \(n\) (this is related to the well-known fact that constant-degree graphs have at least logarithmic mixing time).

\paragraph{Robust distinguishing up to the Kesten--Stigum threshold}

Our discussion so far suggests a natural starting point for designing a robust algorithm that works up to the Kesten--Stigum threshold: matrices defined in terms of constant-length walks, e.g., the matrix \(\super{\mathbf{Q}}s\) for constant \(s\le O_{\e,d}(1)\).

Indeed,  recent work \cite{banks19:_local_statis_semid_progr_commun_detec} takes this approach in order to robustly solve a related distinguishing problem.
The goal is to distinguish between an \Erdos--\Renyi random graph \(\mathbf{G}_0\sim \mathcal G(n,\tfrac d n)\) and the stochastic block model \(\mathbf{G}_1\sim \SBM_{n}(d,\e)\) up to the Kesten--Stigum threshold \(d>4/\e^2\).
Recall that \(\super{\mathbf{Q}}s\) corresponds to the polynomial \(\super Q s(Y)\seteq \tfrac{(2n)^s}{(\e d)^s\card{\saw{ij}{s}}}\sum_{W\in\saw{ij}s} Y_W\).
Let \(\mathbf{Y}_0,\mathbf{Y}_1\) be the respective centered adjacency matrices of \(\mathbf{G}_0,\mathbf{G}_1\).
In order to distinguish \(\mathbf{G}_0\) and \(\mathbf{G}_1\), we seek an (efficiently computable) matrix norm that with high probability is much smaller for \(\super Q s(\mathbf{Y}_0)\) than for \(\super Q s(\mathbf{Y}_1)\).
Both matrix norms discussed so far (the Frobenius norm in \cref{eq:frobenius-correlation} and the cut norm in \cref{eq:cut-norm}) appear to be poor choices:
The Frobenius norm of \(\super Q s(\cdot)\) cannot distinguish between \(\mathbf{G}_0\) and \(\mathbf{G}_1\) for constant \(s\).
A tight analysis of the cut-norm of \(\super Q s(\cdot)\) appears to be challenging due to the high amount of dependencies between the entries of the matrix.
A natural alternative (actually related to the cut norm) is the spectral norm combined with some truncation step.\footnote{
	The truncation step refers to removing vertices with unusually large degree.
	This truncation step is necessary because with high probability (non-regular) constant-degree random graphs have a small number of vertices with logarithmic degree that skew the spectral norm of the centered adjacency matrix in an undesirable way.
	This issue persists also for the matrices \(\super Q s(\cdot)\) when \(s\) is constant.
}
Concretely, the analysis in \cite{banks19:_local_statis_semid_progr_commun_detec} boils down\footnote{
	We remark that the presentation of the algorithm and the choice of the function \(\super{\bar Q} s(\cdot)\) in \cite{banks19:_local_statis_semid_progr_commun_detec} is slightly different from what's described here.
	For example, they use non-backtracking walks instead of self-avoiding walks.
	However, the same proof strategy works for both versions.
} to showing that with high probability the spectral norm satisfies the inequality,
\begin{equation}
	\label{eq:spectral-norm-distinguishing}
	\Norm{\super{\bar Q}s(\mathbf{Y}_0)} \leq o( \Norm{\super{\bar Q} s(\mathbf{Y}_1)})\,.
\end{equation}
Here, \(\super{\bar Q}s(Y)\) is a matrix-valued function obtained by composing the polynomial \(\super{Q}s(Y)\) together with a truncation step (we omit the details here).
In order to prove this inequality, the authors upper-bound (the expectation of) the left-hand side using the trace-method for a parameter \(t\) logarithmic in \(n\),
\[
\E \Norm{\super{\bar Q} s(\mathbf{Y}_0)}
\le \Paren{\E \Norm{\super{\bar Q} s(\mathbf{Y}_0)}^t}^{1/t}
\le \Paren{\Tr \E \super{\bar Q} s(\mathbf{Y}_0)^t}^{1/t}\,,
\]
and then lower-bound (the expectation of) the right-hand side using the planted labeling \(\mathbf{x}\) as a test vector,
\[
\E \Norm{\super{\bar Q} s(\mathbf{Y}_1)}
\ge \tfrac 1n \E \iprod{\mathbf{x},\super{\bar Q} s(\mathbf{Y}_1)\, \mathbf{x}}
\approx  \tfrac 1n \E \iprod{\mathbf{x},\super{ Q} s(\mathbf{Y}_1)\, \mathbf{x}} = n-1\,.
\]
In order to make their distinguishing algorithm robust, the authors consider the following related semidefinite program,
\begin{equation}
	\label{eq:q-sdp}
	\text{maximize}~\iprod{\super{\bar Q}s(Y), X}~
	\text{subject to}~X\succeq 0\,,~\forall i\,.~X_{ii}=1
	\,.
\end{equation}
The optimal value for \(Y=\mathbf{Y}_0\) is upper bounded by \(n\cdot \Norm{\super{\bar Q} s(\mathbf{Y}_0)}\).
At the same time, the optimal value for \(Y=\mathbf{Y}_1\) is lower bounded by the value of planted solution \(X=\dyad{\mathbf{x}}\), which is close to \(n\cdot (n-1)\).
This algorithm is robust because the objective function in \cref{eq:q-sdp} has low sensitivity to edge alterations for constant \(s\), every edge alteration changes the objective value by at most \(n\cdot (2/\e)^s\), which is a \(O_{\e,s}(1/n)\) fraction of the objective value of the planted solution (also see our discussion of the sensitivity of \cref{eq:hs17-relaxation}).

\paragraph{The push out effect and the challenges for weak-recovery using constant-length walks}

\textit{Can we use the matrix \(\super{\bar Q}s (\mathbf{Y})\) (with constant \(s\)) also for weak-recovery?}
Ignoring the issue of robustness for now, the fact that, for the stochastic block model, the spectral norm of \(\super{\bar Q}s (\mathbf{Y})\) is substantially larger  than for the corresponding \Erdos--\Renyi random graph suggests that the top eigenvector of this matrix carries useful information about the planted labeling.
Its top eigenvector can be characterized in terms of the optimal solution to the following semidefinite program,
\begin{equation}
	\label{eq:q-sdp-non-flat}
	\text{maximize}~\iprod{\super{\bar Q}s(\mathbf{Y}), Z}~
	\text{subject to}~Z\succeq 0\,, \Tr Z =1
	\,.
\end{equation}
In order to prove that the optimal solution is correlated with the planted labeling, we could try to mimic the analysis of the basic SDP.
To this end, we would have to prove that the spectral norm \(\norm{\super{\bar Q}s(\mathbf{Y})-\dyad{\mathbf{x}}}\) is smaller than the objective value of the planted solution \(\tfrac 1 n \dyad{\mathbf{x}}\).
However, this turns out to be false for \(d\) close to the Kesten--Stigum threshold, i.e.
\(	\iprod{\super{\bar Q}s(\mathbf{Y}), \tfrac 1 n \dyad{\mathbf{x}}}	<\norm{\super{\bar Q}s(\mathbf{Y})-\dyad{\mathbf{x}}}\).
Another ramification of the phenomenon is that for constant \(s\), the optimal value of \cref{eq:q-sdp-non-flat} is substantially larger than the objective value of the planted solution with high probability. 

Nevertheless, it is still possible to weakly recover the hidden communities. 
To understand how one could overcome this challenge, suppose that we could show the following inequality for the spectral norm (similar to the inequality \cref{eq:frobenius-correlation} for the Frobenius norm, for which however this bound does not hold),
\begin{align}\label{eq:techniques-push-out-spectral}
	\norm{\super{\bar Q}s(\mathbf{Y})-\mathbf{\dyad{x}}}\leq (1-\delta')\norm{\super{\bar Q}s(\mathbf{Y})}\,,
\end{align} 
for some small enough $\delta'> \Omega_\delta(1)$.
Then, even though $\dyad{\mathbf{x}}$ would still be far from being the optimal solution to the program, the communities vector would partially align with the  leading eigenvector of $\super{\bar Q}s(\mathbf{Y})$ and thus we would still be able to  weakly recover it.
(This phenomenon is closely related to the push-out effect in the context of principal component analysis).

In order to actually prove \cref{eq:techniques-push-out-spectral}, one  needs to  upper bound the spectral norm of $\super{\bar Q}s(\mathbf{Y})-\mathbf{\dyad{x}}$ and lower bound the spectral norm of $\super{\bar Q}s(\mathbf{Y})$. As in \cite{banks19:_local_statis_semid_progr_commun_detec}, a natural approach is that of using the trace method. Notice however that by our discussion above, compared to \cite{banks19:_local_statis_semid_progr_commun_detec}, in order to achieve weak recovery the required bounds  need to be significantly sharper. Furthermore, an additional technical difference is that one needs to carry out the trace method in the planted distribution $\sbm$ (and after an additional truncation step!). Indeed we are only aware of few instances where the trace method is carried out on the planted distribution  \cite{DBLP:conf/focs/BordenaveLM15, DBLP:conf/innovations/GulikersLM17}.

For several technical reasons we are only able to show 
\begin{align}\label{eq:techniques-push-out-schatten}
	\norm{\super{\bar Q}s(\mathbf{Y})-\mathbf{\dyad{x}}}_t\leq (1-\delta')\norm{\super{\bar Q}s(\mathbf{Y})}_t\,,
\end{align} 
for some Schatten norm $\Omega(\log n)\leq t < \log n$.
While for $t\gtrsim \log n$, the Schatten norm is within a constant factor of the spectral norm, in the delicate context of \cref{eq:techniques-push-out-schatten} this difference is not negligible as it might result in losing the correlation with the communities vector.
A similar reasoning as the one outlined for the spectral norm carries over to the context of Schatten norm. Using the dual definition of Schatten norm we consider 
the following optimization problem:
\begin{equation}
	\label{eq:q-sdp-non-flat-schatten}
	\text{maximize}~\iprod{\super{\bar Q}s(\mathbf{Y})^2, Z}~
	\text{subject to}~Z\succeq 0\,, \Tr Z^{\frac{t}{t-1}} =1
	\,,
\end{equation}
(Note that we squared the matrix polynomial here since otherwise its trace may not contain meaningful information).
The bound \cref{eq:techniques-push-out-schatten}  implies that we can weakly recover the planted solution from the optimal solution to \cref{eq:q-sdp-non-flat-schatten}.

\paragraph{Achieving robustness}

It is by now clear that, to enable robustness,  a key property is low sensitivity of the objective function to edge alterations.
For \cref{eq:q-sdp-non-flat-schatten}, however, the sensitivity of the objective function may be high because the program allows solutions that are spiky (in the sense that there could be some $i,j \in [n]$ such that $\Abs{Z_\ij}\geq \frac{\omega(1)}{n^2}\underset{i,j \in [n]}{\sum}\Abs{Z_\ij}$).
In the context of the basic SDP, we could fix the sensitivity of the objective function by adding a constraint on the diagonal entries (and hence, by semidefiniteness, on all entries).
The reason we could add  this constraint was because the communities matrix $\dyad{\mathbf{x}}$ --which satisfied the constraint-- was close to the optimal solution.
However, as we approach the KS threshold this is no longer true and thus it remains unclear whether this approach could work. Indeed, the recovery analysis outlined above did not compare arbitrary solutions to the communities matrix $\dyad{\mathbf x}$, but to an a-priori \textit{unknown} optimal solution!

\textit{How can we fix the sensitivity?}
While, with high probability, $\super{\bar Q}s(\mathbf{Y})$ is not positive semidefinite, its second power  $\super{\bar Q}s(\mathbf{Y})^2$ clearly is.  A crucial consequence of this property is  that the optimal solution to \cref{eq:q-sdp-non-flat-schatten} actually has a nice analytical expression: $\super{\bar Q}s(\mathbf{Y})^{2(t-1)}$. 
Due to the truncation, $\super{\bar Q}s(\mathbf{Y})$ is approximately flat and thus, using calculations similar to the trace method, we are able to  show that   $\super{\bar Q}s(\mathbf{Y})^{2(t-1)}$ is also approximately flat.
This implies that we can control the sensitivity of the objective function by replacing the program in \cref{eq:q-sdp-non-flat-schatten} with
\begin{equation}
	\label{eq:q-sdp-flat-schatten}
	\text{maximize}~\iprod{\super{\bar Q}s(\mathbf{Y})^2, Z}~
	\text{subject to}~Z\succeq 0\,, \Tr Z^{\frac{t}{t+1}} =1\,, \forall i, \, Z_{ii}\leq \tfrac{O(1)}{n}
	\,,
\end{equation}
which immediately yields a robust algorithm.

\paragraph{From the Schatten norm to graph counting} 
Proving the bound \cref{eq:techniques-push-out-schatten}  turns out to be the main technical challenge of the paper.
We outline here the main ideas. 
For simplicity we limit the current discussion to showing that the inequality holds in expectation (see \cref{sec:bounds-moments-Q}). By definition 
\begin{align}
	\label{eq:informal-trace-method}
	\underset{(\mathbf{x},\mathbf{G})\sim \SBM_n(d,\e)}{\E}\Brac{\Norm{\super{\bar Q}s(\mathbf{Y})}_t^t}&=\underset{(\mathbf{x},\mathbf{G})\sim \SBM_n(d,\e)}{\E} \Brac{\Tr\Paren{\super{\bar Q}s({\mathbf{Y}})}^t}\,,\\
	\label{eq:informal-trace-method-centered}
	\underset{(\mathbf{x},\mathbf{G})\sim \SBM_n(d,\e)}{\E}\Brac{\Norm{\super{\bar Q}s({\mathbf{Y}})-\dyad{\mathbf{x}}}_t^t}&=\underset{(\mathbf{x},\mathbf{G})\sim \SBM_n(d,\e)}{\E} \Brac{\Tr\Paren{\super{\bar Q}s({\mathbf{Y}})-\dyad{\mathbf{x}}}^t}
\end{align} 
Our approach to obtain  \cref{eq:techniques-push-out-schatten} will be to reduce \textit{both} trace computations to a graph counting problem. 
Consider \cref{eq:informal-trace-method}, each element  corresponds to a walk of the following form:
\begin{definition*}[Block self-avoiding walk]
	A closed walk of length $st$ is a $(s,t)$-\textit{block self-avoiding walk} if it can be split into $t$ self-avoiding walks of length $s$.
\end{definition*}
For a  $(s,t)$-block self avoiding walk $H$, we call  its $t$ self-avoiding walks $\Set{W_1,\ldots,W_t}$, the generating walks of $H$.\footnote{It is possible that there are multiple choices for the set $\Set{W_1,\ldots,W_t}$, at the granularity of this discussion we may assume such set to be given.} 
With this definition, \cref{eq:informal-trace-method} amount to computing the expected number of copies in the instance graph $\mathbf{G}$ of any $(s,t)$-block self-avoiding walk.

To see how to carry out this computation, and for simplicity,  consider the non-truncated graph $\mathbf{G}$ with its centered adjacency matrix $\mathbf{Y}$. While this simplification will make it impossible to get a good bound, it will be useful to build some intuition.
For a $(s,t)$-block self avoiding walk $H$, let $\mathbf{Y}_H=\underset{i\in [t]}{\prod}\mathbf{Y}_{W_i}$ where $W_1,\ldots,W_t$ are the self-avoiding walks generating $H$.
It is easy to see that 
\begin{align}
	\underset{(\mathbf{x},\mathbf{G})\sim \SBM_n(d,\e)}{\E} \Brac{\mathbf{Y}_H\given \mathbf{x}} &=\underset{(\mathbf{x},\mathbf{G})\sim \SBM_n(d,\e)}{\E} \Brac{\underset{ij\in E(H)}{\prod} \Yij^{m(ij)}\given  \mathbf{x}}\nonumber\\
	&=\underset{ij\in E(H)}{\prod}\underset{(\mathbf{x},\mathbf{G})\sim \SBM_n(d,\e)}{\E} \Brac{ \Yij^{m(ij)}\given  \mathbf{x}}\nonumber\\
	&\approx \Brac{\underset{\ij \in E_1(H)}{\prod}\Paren{\frac{\eps\cdot d}{2n}\mathbf{x}_i\mathbf{x}_j}}\cdot \Brac{\underset{\ij \in E_{\geq 2}(H)}{\prod}\Paren{1+\frac{\eps}{2}\mathbf{x}_i\mathbf{x}_j}\cdot \frac{d}{n}}\nonumber\\
	&= \Paren{\frac{d}{n}}^{\Card{E(H)}}\cdot \Brac{\underset{\ij \in E_1(H)}{\prod}\Paren{\frac{\eps}{2}\mathbf{x}_i\mathbf{x}_j}}\cdot \Brac{\underset{\ij \in E_{\geq 2}(H)}{\prod}\Paren{1+\frac{\eps}{2}\mathbf{x}_i\mathbf{x}_j}}\label{eq:informal-trace-expectation-sketch}\,,
\end{align}
where $m(ij)$ is the number of times the edge $ij$ appears in the walk $H$, $E_1(H)$ is the set of edges with multiplicity $m(ij)=1$ and $E_{\geq 2}(H)=E(H)\setminus E_1(H)$. 
The main burden is then to count for each $m_1,m_2\geq 0$ how many $(s,t)$-block self-avoiding walks we may have with $m_1$ edges with multiplicity one and $m_2$ edges of multiplicity at least $2$.

\paragraph{The effect of centering} To understand why $\dyad{\mathbf{x}}$  correlates with $\super{\bar Q}s({\mathbf{Y}})$, we need to observe what is the effect of subtracting the communities from $\super{\bar Q}s({\mathbf{Y}})$. Consider  \cref{eq:informal-trace-method-centered} and again,
for simplicity, let us simply use the non-truncated adjacency matrix and the corresponding polynomial $\Q(\mathbf{Y})$.

For any $(s,t)$-block self-avoiding walk $H$, let $W_{v_1v_2},\ldots,W_{v_tv_1}$ be its $t$ self-avoiding walks so that for each $W_{v_iv_{i+1}}$ the vertices $v_i,v_{i+1}$ correspond to its endpoint.\footnote{To simplify the notation we write $v_{t+1}$ for $v_1$.}
For every block self-avoiding walk $H$, there is a polynomial in  \cref{eq:informal-trace-method-centered} of the form
$\underset{i\in [t]}{\prod}\Paren{\mathbf{Y}_{W_{v_iv_{i+1}}}-\Paren{\frac{\eps\cdot d}{2n}}^s\cdot \mathbf{x}_{i}\cdot \mathbf{x}_{i+1}}$. The catch is that if one of the walks $W_{v_1v_2},\ldots,W_{v_tv_1}$ --for example $W_{v_tv_1}$-- has all edges with multiplicity one in $H$, then by \cref{eq:unbiased-adjacency-matrix}
\begin{align*}
	0=&	\underset{(\mathbf{x},\mathbf{G})\sim \SBM_n(d,\e)}{\E}  \Brac{\underset{i\in [t-1]}{\prod}\Paren{\mathbf{Y}_{W_{v_iv_{i+1}}}-\Paren{\frac{\eps\cdot d}{2n}}^s\cdot \mathbf{x}_{i}\cdot \mathbf{x}_{i+1}}}\cdot \underset{(\mathbf{x},\mathbf{G})\sim \SBM_n(d,\e)}{\E} \Brac{\mathbf{Y}_{W_{v_tv_1}}-\Paren{\frac{\eps\cdot d}{2n}}^s\cdot \mathbf{x}_{v_t}\mathbf{x}_{v_1}}\\
	=&\underset{(\mathbf{x},\mathbf{G})\sim \SBM_n(d,\e)}{\E} \Brac{\underset{i\in [t]}{\prod}\Paren{\mathbf{Y}_{W_{v_iv_{i+1}}}-\Paren{\frac{\eps\cdot d}{2n}}^s\cdot \mathbf{x}_{i}\cdot \mathbf{x}_{i+1}}}\,.
\end{align*}
So, all $(s,t)$-block self-avoiding walks with at least one of the generating walks having all edges traversed exactly \textit{once} in $H$ have expectation $0$ and do not contribute to $\underset{(\mathbf{x},\mathbf{G})\sim \SBM_n(d,\e)}{\E} \Brac{\Tr\Paren{\Q(\mathbf{Y})-\mathbf{x}\transpose{\mathbf{x}}}^t}$.  This effect will approximately carry over the truncated graph $\bar{\mathbf{G}}$ (with centered adjacency matrix $\bar{\mathbf{Y}})$.
Since block self-avoiding walks as the one described above turn out to have a significant contribution to  \cref{eq:informal-trace-method}, this observation will allow us to show
\begin{align}
	\label{eq:informal-separation-expectation}
	\underset{(x,\mathbf{G})\sim \SBM_n(d,\e)}{\E}\Brac{\Norm{\super{\bar Q}s({\mathbf{Y}})-\mathbf{x}\transpose{\mathbf{x}}}_t^t}\leq \Paren{1+\delta}^{-\Omega(t)}\cdot 	\underset{(\mathbf{x},\mathbf{G})\sim \SBM_n(d,\e)}{\E}\Brac{\Norm{\super{\bar Q}s({\mathbf{Y}})}_t^t}\,.
\end{align}

\paragraph{Boosting the probability of success}
The approach described so far yields an algorithm that weakly recovers the communities vector  with constant probability. The main reason behind this shortcoming is that we only showed the required structural inequalities hold with constant probability.\footnote{In fact, it is possible to prove that the structural inequalities hold with high probability without invoking the boosting argument. However, the needed calculations are more complicated. In any case, the boosting argument has the advantage of achieving exponentially high probability.} It turns out, however, that we can exploit robustness to argue that, indeed, we can weakly recover $\mathbf{x}$ with high probability.

This boosting argument relies on a concentration of measure inequality known as the \emph{blowing-up lemma} (see e.g., \cite{AhlswedeEtAlBlowingUp, MartonBlowingUp1, MartonBlowingUp2, ConcentrationOfMeasureRaginskySason}), which is widely used in information theory to prove strong converse results.\footnote{A strong converse result in information theory typically has the following form: If the rate of a code is above the capacity of a channel, then the probability of error of the code goes to 1 as the blocklength becomes large.} Roughly speaking, the blowing-up lemma states that if $\mathbf{z}_1,\ldots,\mathbf{z}_N$ are $N$ independent random variables taking values in a finite set, and $\mathcal{E}_N$ is an event on $\mathbf{z}_1,\ldots,\mathbf{z}_N$ whose probability does not decay exponentially in $N$, i.e., $\displaystyle\lim_{N\to\infty}\frac{1}{N}\log\frac{1}{\mathbb{P}[\mathcal{E}_N]}=0$, then there exists $\ell_N=o(N)$ such that, if we "inflate" $\mathcal{E}_N$ by adding all the strings $(z_1,\ldots,z_N)$ that are at a Hamming distance of at most $\ell_N=o(N)$ from $\mathcal{E}_N$, then the probability of the obtained event becomes $1-o(1)$.

In our context of interest, this means that we do not need to directly show that our graph $\mathbf{G}\sim \sbm$ satisfies the desired structural inequalities (such as  \cref{eq:techniques-push-out-schatten}). It suffices that \textit{some} graph $\tilde{G}\in N_{o(1)}(\mathbf{G})$ satisfies them. Since our algorithm is robust to a constant fraction $\rho$ of adversarial changes, and since $N_{\rho - o(1)}(\mathbf{G})\subseteq N_{\rho}(\tilde{G})$, we can see that by paying a negligible price in the robustness and correlation guarantees, we can use the blowing-up lemma to boost the probability of success of the algorithm to $1-o(1)$. 

It is worth noting that the boosting argument can give us exponentially high probability and not only $1-o(1)$. Furthermore, the argument is not unique to the stochastic block model and can be applied to a wide range of estimation problems.

\section{Preliminaries}\label{sec:preliminaries}

We present here some elementary definitions and results that we will use in the following sections. 

\subsection{General definitions and notation}\label{sec:definitions-notation}

For sets $S,S'$ we denote by $S\times S'$ their Cartesian product.
We denote graphs by $G=(V,E)$ where $V$ is the set of vertices and $E$ the set of edges.
For a multigraph $H=(V,M)$, we denote by $V(H)$ the set of vertices in $H$, by $M(H)$ the multi-set of edges in $H$ and by $E(H)$ the set of \textit{distinct} edges. For $e \in E(H)$, we let $m_H(e)$ be the multiplicity of edge $e$ in $H$. For $v\in V(H)$ we denote by $d^H(v)$ the number of \textit{distinct} edges incident to $v$ in $H$. Given multi-sets $S_1,S_2$ we denote their union by $S_1\oplus S_2$, similarly for multi-graphs $H_1, H_2$ we write $H_1 \oplus H_2$ for the multi-graph $H$ with vertex set $V(H_1)\cup V(H_2)$ and edges $M(H_1)\oplus M(H_2).$
For a set of vertices (or edges) $S$, we denote by $H(S)$ the subgraph of $H$ induced by the set $S$. For a graph $G\subseteq H$ we denote by $H(G)$ the multigraph induced by $V(G)$.  For an edge $e$ and a multigraph $H$, we  write $H+e$ for the multigraph obtained adding $e$ to $H$.
We denote by $K_n$ be the complete graph on $n$ vertices.

\begin{remark}\label{rem:notation}
	All graphs we consider in the paper will have vertex sets that are subsets of $[n]$.
\end{remark}

\begin{definition}[Walk]
	A walk $W$ over $K_n$ is a sequence of vertices and edges $v_1,e_1,v_2,\ldots,v_s$ such that for each  $i \in [s]\subset 
	[n]$ the endpoints of $e_i$ are $v_i$ and $v_{i+1}$. We say $W$ is a \textit{self-avoiding walk} if no vertex is visited twice.
\end{definition}

For a walk $W$ over $K_n$, we denote by $M(W)$ the sequence of edges in $W$. We write $E(W)$ for the set of \textit{distinct} edges in $W$ and $V(W)$ for the set of \textit{distinct} vertices. For simplicity, we also use $W$ to refer to the multigraph with vertex set $V(W)$ and edges $M(W)$.

\begin{definition}[Eulerian multi-graph]
	We say that a multigraph $H$ is Eulerian if there exists a closed walk $W$ in $H$ such that for any $e \in E(H)$, $e$ appears in $W$ exactly $m_H(e)$ times.
\end{definition}

\begin{definition}[Cut]
	For a multi-graph $H$ with vertex set $V$, a cut $(B, V\setminus B)$ is a partition of the vertices into two disjoint subsets. We denote with $(B, V\setminus B)$ the edges in the cut and by $H(B, V\setminus B)$ the multigraphs spanned by those edges in $H$ .
	A $q$-multi-way cut is a partition $(B_1,\ldots,B_, V\setminus (B_1 \cup \ldots \cup B_q))$ of the vertices in $H$ in $q$ disjoint subsets.
\end{definition}

\begin{definition}[Isomorphic graphs]
	Two graphs $G=(V,E)$ and $G'=(V',E')$ are isomorphic if there exists a bijective mapping $\phi:V\rightarrow V'$ such that for any $u,v \in V$, $\Set{u,v}\in E$ if and only if $\Set{\phi(u), \phi(v)}\in E'$.
\end{definition}

For a graph $G$ with $V(G)\subseteq [n]$ and vertices $i,j \in V(G)$, we let $\saw{ij}{s}(G)$ be the set of self-avoiding walks between $i$ and $j$ in $G$ of length $s$.
We write $K_n$ for the complete graph on $n$ vertices and for $i,j \in [n], s\geq 1 $, we use $\saw{ij}{s}$ to denote the set of self-avoiding walks between $i$ and $j$ in $K_n$. 
We let $\saw{i}{s}(G):=\underset{j \in [n]}{\bigcup} \saw{ij}{s}(G)$.

\subsection{Norms}

\begin{definition}[Spectral norm]
		For a matrix $M\in \R^{n\times n}$ we denote with $\Norm{M}$ its spectral norm.
\end{definition}

\begin{definition}[$L_1$ norm]
	For a matrix $M\in \R^{n\times n}$ we denote with $\Normo{M}$ its $L_1$ norm:
	\begin{align*}
		\Normo{M}= \underset{i,j \in [n]}{\sum} \Abs{M_\ij}\,.
	\end{align*}
\end{definition}

\begin{definition}[Infinity norm]
	For a matrix $M\in \R^{n\times n}$ we denote with $\Normi{M}$ its infinity norm:
	\begin{align*}
		\Normi{M}= \underset{i,j \in [n]}{\max} \Abs{M_\ij}\,.
	\end{align*}
\end{definition}

\begin{definition}[Nuclear norm]
	For a matrix $M\in \R^{n\times n}$ we denote with $\Normn{M}$ its nuclear norm:
	\begin{align*}
		\Normn{M}= \underset{i \in [n]}{\sum} \Abs{\lambda_i(M)}\,.
	\end{align*}
	where $\lambda_i(M)$ is the $i$-th eigenvalue of matrix $M$.
\end{definition}

\begin{definition}[Schatten norm]
	For a matrix $M\in \R^{n\times n}$ and $t\geq 1$ we denote by $\Norm{M}_t$ its $t$-Schatten norm:
	\begin{align*}
		\Norm{M}_t = \Paren{\Tr\Paren{M^t}}^{1/t}\,.
	\end{align*}
	Notice that for $t=\infty$ we recover the spectral norm of $M$.
\end{definition}

\subsection{Polynomial statistics}\label{sec:expectation-of-polynomials}
We formally present here the polynomials used in the subsequent sections.
Let $Y$ be the $n$-by-$n$ adjacency matrix of a graph $G$ with vertex set $[n]$, centered with respect to the \Erdos-\Renyi distribution with parameter $d$. 
That is for $a,b \in [n]$,
\begin{align}\label{eq:definition-centered-adjacency-matrix}
	Y_{ab}(G)=\begin{cases}
		1-\frac{d}{n}&\text{ if } ab \in E(G),\\
		-\frac{d}{n}&\text{ if } ab \notin E(G),\\
		0 &\text{if } a=b\,.
	\end{cases}
\end{align}
When the context is clear we write $Y$ instead of $Y(G)$.
For simplicity, for a multigraph $H$ with vertex set $V(H)\subseteq [n]$, we write $Y_H$ for the polynomial
\begin{align*}
	Y_H:=\underset{ab \in E(H)}{\prod} Y^{m(ab)}_{ab}\,.
\end{align*}

For fixed parameters $n,d,s,\epsilon$ and  for $i,j \in [n]$, we define the polynomial 
\begin{align}\label{eq:definition-polynomial-q}
	Q_{ij}^{(s)}(Y) = 
	\begin{cases}
		\frac{1}{\card{\saw{ij}{s}}}\Paren{\frac{2n}{\eps \cdot d}}^{s}\underset{H\in \saw{ij}{s}}{\sum} Y_H & \text{if $i\neq j$\,,}\\
		0 &\text{otherwise.}
	\end{cases}
\end{align}

$Q^{(s)}(Y):\R^{n\times n}\rightarrow \R^{n\times n}$ is then the matrix valued polynomial with entry $ij$ equal to $Q_{ij}^{(s)}(Y)$.
Notice that $\Card{\saw{ij}{s}}=(n-2)^{\underline{s-1}}$ where $(n-2)^{\underline{s-1}}$ is the falling factorial $(n-2)\cdot (n-3)\cdots(n-2-s+1)$.

As already mentioned, the polynomial \cref{eq:definition-polynomial-q} is an unbiased estimators of the vector of communities $\mathbf{x}$ with respect to the distribution $\sbm$. 
The following facts formally show this and other basic properties of the polynomials \cref{eq:definition-centered-adjacency-matrix} and \cref{eq:definition-polynomial-q}.

\begin{fact}\label{fact:sbm-expectation-polynomial-2}
	Let $ \mathbf{G}\sim \sbm\paren{ \mathbf{Y} | \mathbf{x}}$ and let $ \mathbf{Y}$ be its centered adjacency matrix. Let $F$ be  a simple graph. Then 
	\begin{align*}
		\E_{\sbm} \Brac{\mathbf{Y}_F \given \mathbf{x}} &= \Paren{\frac{d\cdot\eps}{2n}}^{\Abs{E(F)}}\underset{ab \in E(F)}{\prod}\mathbf{x}_{a}\mathbf{x}_b\,.
	\end{align*}
	\begin{proof}
		\begin{align*}
			\E \Brac{\mathbf{Y}_F \given \mathbf{x}} &= \underset{ab \in E(F)}{\prod} \E \Brac{\mathbf{Y}_{ab}\given \mathbf{x}}\\
			&=  \underset{ab \in E(F)}{\prod}\Brac{\Paren{1+\frac{\eps}{2} \mathbf{x}_a\mathbf{x}_b}\frac{d}{n}\Paren{1-\frac{d}{n}} - \Paren{1-(1+\frac{\eps}{2} \mathbf{x}_a\mathbf{x}_b)\frac{d}{n}}\frac{d}{n}} \\
			&= \underset{ab \in E(F)}{\prod} \frac{d}{n}\Brac{\Paren{1-\frac{d}{n}} -\Paren{1-\frac{d}{n}}+ \frac{\eps}{2} \mathbf{x}_a\mathbf{x}_b \Paren{1-\frac{d}{n}+\frac{d}{n}}}\\
			&= \underset{ab \in E(F)}{\prod} \frac{d\cdot\eps}{2n}\mathbf{x}_a\mathbf{x}_b\,.
		\end{align*}
	\end{proof}
\end{fact}

Notice that if $F$ is a self-avoiding walk of length-$s$ between vertices $i,j \in [n]$ then $\E\Brac{\mathbf{Y}_F\given \mathbf{x}} = \Paren{\frac{d\cdot \eps}{2n}}^{s} \mathbf{x}_i\mathbf{x}_j$.
The expectation of small powers of $\mathbf{Y}_{ab}$ is also easy to approximate.
We write $f(n)= (1\pm o(1))g(n)$ if $(1-o(1))\cdot g(n)\leq f(n)\leq (1+o(1))\cdot g(n)$. 

\begin{fact}\label{fact:sbm-expectation-powers-polynomial-2}
	Let $a,b \in [n]$ and let $2\leq q \lesssim \log n$ be some integer. Then
	\begin{align*}
			\E_{\sbm} \Brac{\mathbf{Y}_{ab}^q \given \mathbf{x}} &=  \Paren{1\pm o\Paren{\frac{1}{n^{0.99}}}} \Paren{1+\frac{\eps}{2}\mathbf{x}_a\mathbf{x}_b}\frac{d}{n}\,.
	\end{align*}
	\begin{proof}
		By choice of $q$ and since $d$ is a constant,
		\begin{align*}
			\E \Brac{\mathbf{Y}_{ab}^q \given \mathbf{x}} &=  \Paren{1+\frac{\eps}{2} \mathbf{x}_a\mathbf{x}_b}\frac{d}{n}\Paren{1-\frac{d}{n}}^q - \Paren{1-(1+\frac{\eps}{2} \mathbf{x}_a\mathbf{x}_b)\frac{d}{n}}\Paren{\frac{d}{n}}^q \\
			&= \Paren{1\pm o\Paren{\frac{1}{n^{0.99}}}} \Paren{1+\frac{\eps}{2}\mathbf{x}_a\mathbf{x}_b}\frac{d}{n}\,.
		\end{align*}
	\end{proof}
\end{fact}

\section{Robust recovery meta-algorithm}
\label{sec:scale-free-recovery-algorithm}

\newcommand{\optref}[1]{\opt_{\text{#1}}}
In this section we consider a more general problem than the stochastic block model with constant average degree.
We show how to robustly recover the hidden structure even though it may have a low objective value in our optimization problem. We apply this algorithm to the stochastic block model in \cref{sec:sbm-recovery}.

\begin{problem}[Small-Correlation Robust Signal Recovery]\label{problem:scale-free-robust-recovery}
	Let $v\in \set{\pm 1/\sqrt{n}}^n$ be a unit flat vector and  $t=\frac{1}{C^*}\cdot \log n$ for some constant $C^*>0$. For an \textit{unknown} matrix $Q \in \R^{n\times n}$ --weakly correlated with the \textit{unknown} signal $v$-- and an \textit{observed} matrix $\tilde{Q}\in \R^{n\times n}$ such that the tuple $(Q,\tilde{Q},v)$  satisfies the conditions \cref{eq:properties-of-Q} with parameters $\alpha,\beta,\gamma,\delta^*,\zeta, C^*\geq 0$. The goal is to \textit{return} a unit vector $\hat{v}$ such that $\iprod{\hat{v}, v}^2\geq \Omega(1)$.
\end{problem}

The set of conditions is:
\begin{equation}\label{eq:properties-of-Q}\tag{\(\cA_{\alpha,\beta,\gamma,\delta^*,\zeta, C^*}\)}
	\cA_{\alpha,\beta,\delta^*,\gamma,\zeta, C^*}\colon
	\left \{
	\begin{aligned}
		&\text{correlation:}& \Tr Q^{2t}=1&\\
		&&\Tr \Paren{Q-\alpha\cdot v \transpose{v}}^{2t}\leq \Paren{1-\delta^*}^{2t}&\\
		\\
		&\text{sensitivity:}& \underset{i \in [n]}{\sum}\Paren{Q^{2t-2}}^2_{ii}\leq \frac{\gamma}{n}\cdot \Paren{\Tr Q^{2t-2}}^2&\\
		&&\max_{i \in [n]}\underset{j \in [n]}{\sum}\Abs{Q_{ij}}\leq \zeta&\\
		&\text{perturbations:}&\frac{1}{n}\Normo{Q-\tilde{Q}}\leq\beta&\\
		&&\max_{i \in [n]}\underset{j \in [n]}{\sum}\Abs{\tilde{Q}_{ij}}\leq \zeta&\\
	\end{aligned}
	\right \}
\end{equation} 

The correlation conditions enforce the matrix $Q$ to contain non-trivial information about $v$. Here $\delta^*$ is the parameter quantifying  the correlation between  $Q$ and the planted solution.\footnote{For the stochastic block model, $\delta^*$ will be a polynomial in $\delta$ for $d\geq (1+\delta)\frac{4}{\eps^2}$.}
The sensitivity conditions (with parameters $\gamma,\zeta$) are necessary so that perturbations cannot completely hide the signal. They roughly correspond to the flatness property discussed in \cref{sec:techniques}. Finally, the perturbations inequalities dictate how far the original matrix $Q$ is from its corrupted observation $\tilde{Q}$.

To solve \cref{problem:scale-free-robust-recovery}, we will use the following algorithm. 
\begin{algorithmbox}[Scale-free robust recovery]\label{alg:scale-free-recovery}
	\mbox{}\\
	\textbf{Settings:}  Let $(Q,\tilde{Q},v)$ be an instance of \cref{problem:scale-free-robust-recovery} satisfying \cref{eq:properties-of-Q} with parameters $\alpha, \beta,\gamma, \delta^*, \zeta\geq 0, C^*>0$.
	
	\noindent 
	\textbf{Input:} $\tilde{Q}\in \R^{n\times n}, \gamma\geq0, C^*>0 $.
	\begin{enumerate}
		\item Let $t^*=(1-1/t)^{-1}$. Compute a matrix $Z\in \R^{n\times n}$ solving the program \begin{equation}\label{eq:program-scale-free-main}\tag{P.1}
			\begin{array}{ll@{}ll}
				&\text{maximize }  & \iprod{Z, \tilde{Q}^2} \qquad 
				\text{subject to}& \left\{
				\begin{aligned}
					\Tr Z^{t^*} \leq 1&\\
					Z \sge 0&\\
					\forall i \in [n], \quad Z_{ii}\leq \frac{\gamma^2}{n}&
				\end{aligned}
				\right\}
			\end{array}
		\end{equation} 
		\item Return $M:=Z\tilde{Q}+\tilde{Q}Z$.
	\end{enumerate}
\end{algorithmbox}

Notice that the constraints of program~\ref{eq:program-scale-free-main} are convex and can be checked in polynomial time, hence the whole algorithm runs in polynomial time.

\cref{alg:scale-free-recovery}, will allow us to approximately recover the unknown vector $v$ with constant probability. Concretely, we will use it to prove the central theorem below.
\begin{theorem}\label{thm:scale-free-main}
	Let $Q,\tilde{Q}\in \R^{n\times n}$, $v\in \Set{\pm 1/\sqrt{n}}^n$ form an instance of \cref{problem:scale-free-robust-recovery}. Suppose that $(Q,\tilde{Q},v)$ satisfies \cref{eq:properties-of-Q} with parameters $\alpha,\beta, \gamma, \delta^*,\zeta\geq0, C^*>0$ and that
	\begin{align}
		\label{eq:scale-free-approximability}
		&\text{approximability: }\qquad  &e^{2C^*}\cdot \frac{\zeta^2}{\sqrt{\gamma}}\leq \frac{\delta^*}{4}\\
		\label{eq:scale-free-robustness}
		&\text{robustness: }\qquad &8\zeta\cdot \gamma^2\cdot \beta\leq \frac{\delta^*}{4}\,.
	\end{align}
	Then, given $\tilde{Q}, \gamma$ there exists a polynomial time algorithm (\cref{alg:scale-free-recovery}) that computes a matrix $M$ with $\normn{M}\leq O(1)$ such that
	\begin{align*}
		\iprod{M, \dyad v}\geq \Paren{\frac{\delta^*}{2\alpha}-\frac{\gamma^2\cdot \beta}{n}}\cdot  \Normn{M}\,.
	\end{align*}
\end{theorem}

Before proving the theorem, let's try to understand its conditions and meaning.
First, notice that if we had access to $Q$ then simply computing the matrix maximizing the $t$-Schatten norm of $Q^2$ would result in a matrix $\delta$-correlated with $\dyad{v}$.\footnote{We need to use $Q^2$ since the matrix $Q$ may not be positive semidefinite. See \cref{sec:techniques}.}  Unfortunately we only have access to an \textit{arbitrarily perturbed} version of $Q$, thus the main difficulty is to use this perturbed matrix $\tilde{Q}$ as a proxy for $Q$ and still obtain a matrix close to $\dyad{v}$.
In \cref{eq:scale-free-approximability}, the exponential term is a residue of the use of schatten norms.  One side, large $\gamma$ allows for stronger approximability inequalities, on the other hand the smaller $\gamma$ is the more resilient the algorithm is to adversarial perturbations (as seen in \cref{eq:scale-free-robustness}). On a similar note, small values $\zeta$ implies the data has low sensitivity and thus one can obtain better guarantees.

The strategy of \cref{alg:scale-free-recovery} to solve \cref{problem:scale-free-robust-recovery} will be  to obtain a matrix $Z$ with small entries that is close to the matrix maximizing the $t$-Schatten norm of  $\tilde{Q}^2$. By the sensitivity and perturbation constraints,  this will also be close to the matrix maximizing the $t$-Schatten norm of  $Q^2$.
The  approximability condition \cref{eq:scale-free-approximability} ensures that the constraints of \ref{eq:program-scale-free-main} on the entries of feasible solutions are not too strong and so that $\tilde{Q}^{2t}$ is close to an optimal solution.
On the other hand, the robustness condition ensures that $\tilde{Q}$ and $Q$ are not too far from each other and so the optimal solution to \ref{eq:program-scale-free-main} non-trivially correlates with ${Q}^{2}$ and thus also with $\dyad{v}$.

Weak recovery of $v$ then immediately follows as a corollary of the rounding below.
\begin{lemma}[Rounding lemma]\label{rounding_lemma}
    Let $\hat{\delta} >0$. Given a matrix $M\in\mathbb{R}^{n\times n}$ and a unit vector $v\in\mathbb{R}^n$ achieving correlation in the following sense
    \begin{align*}
        \iprod{M,vv^\top}\geq  \Normn{M}\cdot\delta^*\,,
    \end{align*}
    a uniformly at random choice $\hat{\mathbf{v}}$ among the top $\frac{2}{\delta^*}$ unit norm eigenvectors of $M$ can achieve $(\delta^*)^{O(1)}$-correlation  with $v$: $$\E\iprod{ \hat{\mathbf{v}},v}^2\geq\frac{(\delta^*)^{3}}{8}\,.$$
\end{lemma}
\begin{proof}
    By performing a spectral decomposition,
    \begin{align*}
       \frac{M}{\Normn{M}}=\sum_{i=[n]}\xi_i\cdot z_iz_i^\top \,,
    \end{align*}
    where $z_i$ are orthonormal eigenvectors of $M$ and $\xi_\ell\geq \xi_{\ell+1}$
	 for any $\ell \in [n]$ are scalars. By assumption 
	 $\delta^*=\sum_{i=[n]}\xi_i \iprod{z_i,v}^2$.
	 Since $\frac{M}{\Normn{M}}$ has unit nuclear norm, we have $\sum_{i=[n]} \Abs{\xi_i}=1$.
	 Thus for $k=\frac{2}{\delta^*}$, when $i>k$, $\xi_i\leq \frac{\delta^*}{2}$. 
	 It follows that $\sum_{i\in [k]} \xi_i\cdot  \iprod{z_i,v}^2\geq \delta^*-\frac{\delta^*}{2}=\frac{\delta^*}{2}$ 
	 and so there must be $\mathbf{i}\in [k]$ such that 
	 $\iprod{z_{\mathbf{i}},v}^2\geq \xi_i\cdot \iprod{z_i,v}^2 \geq \frac{(\delta^*)^2}{4}$.
    
    Hence, if we choose $\mathbf{i}\in [k]$ uniformly at random, then with probability at least $1/k$, we get a vector such that $ \iprod{z_\mathbf{i},v}^2\geq \frac{(\delta^*)^2}{4}$. As $k=\frac{2}{\delta^*}$,  we obtain the claim.     
\end{proof}

The rest of the section contains a proof of \cref{thm:scale-free-main}. We split it in two lemmas. The first lower bounds the optimal value of program \ref{eq:program-scale-free-main}, the second uses such lower bound to show that  any nearly optimal solution is non-trivially correlated with $\dyad{v}$. Together they will imply the theorem. 

\begin{lemma}[Lower bound for the Optimum]\label{lem:scale-free-optimum-lower-bound}
	Let $(Q,\tilde{Q},v)$ be an instance of \cref{problem:scale-free-robust-recovery} satisfying \ref{eq:properties-of-Q} with parameters $\alpha, \beta,\gamma, \delta^*, \zeta\geq 0, C^*>0$ and consider the program \ref{eq:program-scale-free-main}. Then
	\begin{align*}
		\optref{\ref{eq:program-scale-free-main}}\geq 1-e^{2C^*}\cdot\frac{\zeta^2}{\sqrt{\gamma}}-4\zeta\cdot \gamma^2\cdot \beta\,.
	\end{align*}
\end{lemma}

\begin{lemma}[Correlation of nearly-optimal solutions] \label{lem:scale-free-correlation}
	Let $(Q,\tilde{Q},v)$ be an instance of \cref{problem:scale-free-robust-recovery} satisfying \cref{eq:properties-of-Q} with parameters $\alpha, \beta,\gamma, \delta^*, \zeta\geq 0, C^*>0$ and consider the program \ref{eq:program-scale-free-main}. 
	Suppose that
	\begin{align*}
		e^{2C^*}\cdot\frac{\zeta^2}{\sqrt{\gamma}}+8\zeta\cdot \gamma^2\cdot \beta\leq \frac{\delta^*}{2}\,.
	\end{align*}
	Then any feasible solution $Z\in \R^{n\times n}$ such that $\iprod{Z, \tilde{Q}^2}\geq 1-e^{2C^*}\cdot\frac{\gamma^2}{\sqrt{\zeta}}-4\zeta\cdot \gamma^2\cdot \beta\,,$ satisfies
	\begin{align*}
		\iprod{Z\tilde{Q}+\tilde{Q}Z,v \transpose{v}}\geq \frac{\delta^*}{2\alpha}-\frac{\gamma^2\cdot \beta}{n}\,.
	\end{align*}
\end{lemma}

We are now ready to prove \cref{thm:scale-free-main}.
\begin{proof}[Proof of \cref{thm:scale-free-main}]
	By hypothesis, combining   \cref{lem:scale-free-optimum-lower-bound} and \cref{lem:scale-free-correlation} we get
	\begin{align*}
		\iprod{M, \dyad{v}}\geq \frac{\delta^*}{2\alpha}-\frac{\gamma^2\cdot \beta}{n}\,.
	\end{align*}
	It remains to bound $\normn{M}$. By Holder's inequality with respect to the spectral norm and the nuclear norm, we have $\normn{Z\tilde{Q}}\leq \Normn{Z}\norm{\tilde{Q}}$. Since $Z$ is a feasible solution to program \ref{eq:program-scale-free-main}, we have $\Norm{Z}_{t*}\leq 1$ for $t^{*}=\Paren{1-\frac{1}{t}}^{-1}$. Again by Holder's inequality,  $\Normn{Z}\leq \Norm{Z}_{t*}\Norm{\text{Id}_n}_{t}$ and for  $t=\frac{\log n}{C^*}$, $\Norm{\text{Id}_n}_{t}=O(1)$. Thus $\Normn{Z}=O(1)$. By the triangle inequality of trace norm, we have $\Normn{M}\leq \Normn{ZQ}+\Normn{QZ}\leq O(1)$. The result follows.
\end{proof}

The rest of the section will be devoted to  and. We prove \cref{lem:scale-free-optimum-lower-bound} in \cref{sec:scale-free-optimum-lower-bound} and  \cref{lem:scale-free-correlation} in \cref{sec:scale-free-correlation}. The central tool to both results will the following intermediate lemma, which, informally speaking, states that for any feasible solution $Z$ to program \ref{eq:program-scale-free-main} we have $\iprod{Z,Q^2}\approx \iprod{Z, \tilde{Q}^2}$.

\begin{lemma}\label{lem:scale-free-distance-between-matrices}
	Let $(Q,\tilde{Q},v)$ be an instance of \cref{problem:scale-free-robust-recovery} satisfying \cref{eq:properties-of-Q} with parameters $\alpha, \beta,\gamma, \delta^*, \zeta\geq 0, C^* >0$.
	Let $Z\in \R^{n\times n}$ be a feasible solution to the program \ref{eq:program-scale-free-main}. Then
	\begin{align*}
		\Abs{\iprod{Z, Q^2-\tilde{Q}^2}}\leq 4 \zeta\cdot \gamma^2\cdot \beta\,.
	\end{align*}
	\begin{proof}
		We may rewrite
		\begin{align*}
			Q^2-\tilde{Q}^2 =& \frac{1}{2}\Paren{Q-\tilde{Q}}\Paren{Q+\tilde{Q}} +\frac{1}{2}\Paren{Q+\tilde{Q}}\Paren{Q-\tilde{Q}}\,.
		\end{align*}
		So let $Z':=Z\cdot \Paren{\tilde{Q}+Q}$ and $Z'':=\Paren{\tilde{Q}+Q}Z$, then we have
		\begin{align*}
			\Abs{\iprod{Z, Q^2-\tilde{Q}^2}} &= \Abs{\iprod{Z', Q-\tilde{Q}} + \iprod{Z'', Q-\tilde{Q}}}\\
			&\leq \Paren{\Normi{Z'}+\Normi{Z''}} \Normo{ Q-\tilde{Q}}\\
			&\leq \Paren{\Normi{Z'}+\Normi{Z''}}\cdot \beta n\,.
		\end{align*}
		Since by the perturbation conditions the matrix $Q-\tilde{Q}$ has bounded $\ell_1$ norm, it suffices to upper bound $\Normi{Z'}$ and $\Normi{Z''}$.
		So notice that for any $i,j \in [n]$
		\begin{align*}
			\Abs{Z'_\ij} &= \Abs{\iprod{Z_i, \Paren{Q+\tilde{Q}}_j}}\\
			&\leq \frac{\gamma^2}{n} \underset{\ell \in [n]}{\sum} \Abs{Q_{j\ell}-\tilde{Q}_{j\ell}}\\
			&=  2\zeta\frac{\gamma^2}{n}\,.
		\end{align*}
		Applying the same reasoning to $\Abs{Z''_\ij}$, the result follows.
	\end{proof}
\end{lemma}

\subsection{Lower bound for the optimum}\label{sec:scale-free-optimum-lower-bound}

Here we prove \cref{lem:scale-free-optimum-lower-bound}. Throughout the section we assume $(Q,\tilde{Q},v)$ to be an instance of \cref{problem:scale-free-robust-recovery} satisfying \cref{eq:properties-of-Q} with parameters $\alpha, \beta,\gamma, \delta^*, \zeta\geq 0, C^*>0$. We also define $t^*=(1-1/t)^{-1}$, where $t=\frac{\log n}{C^*}$.
We will build our solution starting from a feasible solution to a simpler program than \ref{eq:program-scale-free-main} and then modifying it to satisfy the missing constraints.
To this end consider the semidefinite program:
\begin{equation}\label{eq:program-large-entries}\tag{P.2}
	\begin{array}{ll@{}ll}
		&\text{maximize }  & \iprod{Z, Q^2} \qquad 
		\text{subject to}& \left\{
		\begin{aligned}
			\Tr Z^{t^*} \leq 1&\\
			Z \sge 0&
		\end{aligned}
		\right\}
	\end{array}
\end{equation} 

It is easy to compute its optimum.
\begin{lemma}\label{lem:scale-free-optmium-large-entries}
	Consider the program \ref{eq:program-large-entries}. Then $\opt_{\text{\ref{eq:program-large-entries}}}=1$. 
	\begin{proof}
		Applying \Holder inequality
		\begin{align*}
			\iprod{Z,Q^2}\leq \Norm{Z}_{t^*}\cdot \Norm{Q^2}_{t}\leq 1\,.
		\end{align*}
		To see that  $\opt_{\text{\ref{eq:program-large-entries}}}\geq 1$ instead let $Z=Q^{2t-2}$. 
		Here $Z$ is a feasible solution because $\Tr Z^{t^*}=\Tr Q^{(2t-2)\cdot t^*} = \Tr Q^{2t}=1$ and it is clearly positive semidefinite.
		We also have
		\begin{align*}
			\iprod{Z,Q^2} = \Tr Q^{2t} = 1\,.
		\end{align*}
	\end{proof}
\end{lemma}

We can use \cref{lem:scale-free-optmium-large-entries} to show that, for $Q=\tilde{Q}$ (that is,  $\beta=0$), we can obtain a nearly solution to program  \ref{eq:program-scale-free-main} from a (specific) optimal solution to program \ref{eq:program-large-entries}. Solutions of program \ref{eq:program-large-entries} may have large entries and so are not feasible solutions to program  \ref{eq:program-scale-free-main} in general. Thus, the plan is to project the optimal solution $Q^{2t-2}$ for program \ref{eq:program-large-entries} to the space spanned by its columns  with bounded entries. Since entries of $Q$ are also bounded, this projection will not decrease significantly the correlation between our new matrix and $Q$. 
Consider the program:

\begin{equation}\label{eq:program-small-entries}\tag{P.3}
	\begin{array}{ll@{}ll}
		&\text{maximize }  & \iprod{Z, Q^2} \qquad 
		\text{subject to}& \left\{
		\begin{aligned}
			\Tr Z^{t^*} \leq 1&\\
			Z \sge 0&\\
			\forall i \in [n], \quad Z_{ii}\leq \frac{\gamma^2}{n}&
		\end{aligned}
		\right\}
	\end{array}
\end{equation} 

Let $Z^*=Q^{2t-2}$ and let $S$ be the set of large diagonal entries of $Z^*$, that is $S:=\Set{i \suchthat Z^*_{ii}>\frac{\gamma^2}{n}}$.
Denote by $\Pi_S$ the projector into the subspace spanned by the columns of $Z^*$ with index in $S$.
We first observe a simple fact concerning $Z^*, \Pi_S$.

\begin{fact}\label{thm:projectionReduceTrace}
	Consider the program \ref{eq:program-small-entries} and let
	\begin{align*}
		Z:= \Paren{\Id-\Pi_S} Z^* \Paren{\Id-\Pi_S}\,.
	\end{align*}
    For any integer $t>0$ we have
    \begin{equation*}
        \Tr \Paren{Z^{*t}} \geq \Tr \Paren{Z^t}
    \end{equation*}
\end{fact}
\begin{proof}
	We denote the dimension of space spanned by $S$ as $k$. Then
	 we denote $i$-th smallest eigenvalue of $Z^*$ as $\lambda_i(Z^{*})$ and the $i$-th smallest eigenvalue of the matrix $Z$ as $\lambda_i(Z)$.
	 We note that $\Tr Z^{t}=\sum_{i=1}^n \lambda_i(Z)^t$, and $\Tr Z^{*t}=\sum_{i=1}^n \lambda_i(Z^*)^t$. Furthermore, since $Z^*$ is positive semidefinite, $Z$ is also positive semidefinite and the smallest $k$ eigenvalues of $Z$ are given by $0$. Thus, it's sufficient to prove that $\lambda_{k+i}(Z^*)\geq \lambda_{k+i}(Z)$ for any $1\leq i\leq n-k$. 
	 
	 To prove this, we denote the space spanned by the top $n-k-i+1$ eigenvectors of $Z$ as $U_i$. Applying  Courant-Fischer min-max  \cref{thm:courant-min-max-theorem}, we have
	 \begin{align*}
		 \lambda_{k+i}(Z^*) & \geq \min_x \{\mathbb{R}_{Z^*}(x)\mid{x\in U_i \text{ and } x\neq 0}\}\\
		  & = \min_x\{\mathbb{R}_{Z}(x)\mid x\in U_i \text{ and } x\neq 0\}\\
		  & =\lambda_{i+k}(Z)
	 \end{align*}
	 This completes the proof.
\end{proof}

Using this fact, we prove that $Z$ is close to an optimal solution for $\ref{eq:program-small-entries}$.
\begin{lemma}\label{lem:scale-free-optimum-lower-bound-matrix-Q}
	Consider the program \ref{eq:program-small-entries} and let
	\begin{align*}
		Z:= \Paren{\Id-\Pi_S} Z^* \Paren{\Id-\Pi_S}\,.
	\end{align*}
	Then
	\begin{align*}
		\iprod{Z, Q^2}\geq 1 - e^{2C^*}\cdot\frac{\zeta^2}{\sqrt{\gamma}}\,.
	\end{align*}
	\begin{proof}
		By construction $Z$ is the projection of $Z^*$ into the space orthogonal to the subspace spanned by  columns of $Z^*$ with index in $S$.
		The matrix $Z$ is clearly positive semidefinite. Moreover, as shown in \cref{thm:projectionReduceTrace}, since it is a projection fo $Z^*$ to a subspace, the eigenvalues of $Z^*$ dominates the eigenvalues of $Z$ and thus $\Tr Z^{t^*}\leq \Tr Z^{*t^*}$. It follows that $Z$ is a feasible solution to \cref{eq:program-small-entries}.
		
		It remains to lower bound its objective value.
		Now by \cref{lem:scale-free-optmium-large-entries},
		\begin{align*}
			\iprod{Z, Q^2} \geq 1 - \Abs{\iprod{Z-Z^*, Q^2}}\,,
		\end{align*}
		hence it suffices to bound $\Abs{\iprod{Z-Z^*, Q^2}}$. Then
		\begin{align*}
			\Abs{\iprod{Z-Z^*,Q^2}}\leq& \underset{i \in S}{\sum}\underset{j \in [n]}{\sum}\Abs{Z^*_\ij}\cdot \Abs{\Paren{Q^2}_\ij}\\
			\leq & \underset{i \in S}{\sum}\underset{j \in [n]}{\sum}\sqrt{Z^*_{ii}\cdot Z^*_{jj}}\cdot \Abs{\Paren{Q^2}_\ij}\\
			\leq &  \Paren{\underset{i \in S}{\sum} Z^*_{ii}}^{1/2}\cdot \Paren{\underset{i \in S}{\sum} \Paren{\underset{j \in [n]}{\sum}\sqrt{Z^*_{jj}}\Abs{\Paren{Q^2}_{\ij}}}^2}^{1/2}\\
			\leq & \Paren{\underset{i \in S}{\sum} Z^*_{ii}}^{1/2}\cdot \Paren{\underset{i \in [n]}{\sum} \Paren{\underset{j \in [n]}{\sum}\sqrt{Z^*_{jj}}\Abs{\Paren{Q^2}_{\ij}}}^2}^{1/2}\\
			\leq & \Paren{\underset{i \in S}{\sum} Z^*_{ii}}^{1/2}\cdot\Norm{\abs{Q^2}}\cdot \Paren{\Tr Z^*}^{1/2}\,.
		\end{align*}
		By the second sensitivity condition on $Q$, $\Norm{\Abs{Q}}\leq \zeta$ and thus $\Norm{\abs{Q^2}}\leq \zeta^2$. 
		By the first sensitivity condition,
		\begin{align*}
			\underset{i \in S}{\sum} Z^*_{ii} \leq \frac{n}{\gamma^2} \underset{i \in [n]}{\sum} \Paren{Z^*_{ii}}^2\leq \frac{n}{\gamma^2}\cdot \frac{\gamma}{n}\Paren{\underset{i \in [n]}{\sum} Z^*_{ii}}^2\leq \frac{1}{\gamma} \Paren{\Tr Z^*}^2\,.
		\end{align*}
		By \cref{lem:relationship-between-norms} and choice of $t=\frac{1}{C^*}\cdot \log n$
		\begin{align*}
			\Tr Z^* \leq e^{C^*}\cdot \Tr Z^{*t^*}\,.
		\end{align*}
		It follows that
		\begin{align*}
			\Abs{\iprod{Z-Z^*,Q^2}}\leq e^{\frac{3}{2}C^*}\cdot \frac{\zeta^2}{\sqrt{\gamma}}\,,
		\end{align*}
		concluding the proof.
	\end{proof}
\end{lemma}

Using \cref{lem:scale-free-optimum-lower-bound-matrix-Q}, we can prove \cref{lem:scale-free-optimum-lower-bound}. The idea is that since by  perturbation conditions $Q$ and $\tilde{Q}$ are close in $L_1$-norm, the product $\iprod{Z,\tilde{Q}^2}$ will not be too far from $\iprod{Z,Q^2}$.

\begin{proof}[Proof of \cref{lem:scale-free-optimum-lower-bound}]
	Any feasible solution for the program \ref{eq:program-small-entries} is a feasible solution for \ref{eq:program-scale-free-main}. Choosing $Z\in \R^{n\times n}$ as constructed in \cref{lem:scale-free-optimum-lower-bound-matrix-Q}
	we get
	\begin{align*}
		\iprod{Z, \tilde{Q}^2}&\geq \iprod{Z, Q^2}-\iprod{Z, \tilde{Q}^2-Q^2}\\
		&\geq 1-e^{2C^*}\cdot\frac{\zeta^2}{\sqrt{\gamma}}-4\zeta\cdot \gamma^2\cdot \beta\,,
	\end{align*}
	where we applied both \cref{lem:scale-free-distance-between-matrices} and \cref{lem:scale-free-optimum-lower-bound-matrix-Q}.
\end{proof}

\subsection{Correlation of nearly-optimal solutions}\label{sec:scale-free-correlation}
We prove here \cref{lem:scale-free-correlation}. Again, we  assume $(Q,\tilde{Q},v)$ to be an instance of \cref{problem:scale-free-robust-recovery} satisfying \cref{eq:properties-of-Q} with parameters $\alpha, \beta,\gamma, \delta^*, \zeta\geq 0, C^*>0$ and we let $t^*=(1-1/t)^{-1}$, where $t=\frac{\log n}{C^*}$.
We start  by showing that any nearly optimal solution to program \ref{eq:program-small-entries} is non trivially correlated with the vector $v$ in the following sense.

\begin{lemma}\label{lem:scale-free-correlation-nearly-optimal-solution-Q}
	Consider the program \ref{eq:program-small-entries}. Any feasible solution $Z\in \R^{n\times n}$ such that $\iprod{Z,Q^2}\geq 1-\frac{\delta^*}{2}$ satisfies 
	\begin{align*}
		\iprod{ZQ+QZ, v \transpose{v}}\geq \frac{\delta^*}{2\alpha}\,.
	\end{align*}
	\begin{proof}
		Let $Z$ be a feasible solution such that $\iprod{Z,Q^2}\geq 1-\frac{\delta^*}{2}$. Then
		\begin{align*}
			1-\frac{\delta^*}{2}&\leq \iprod{Z, Q^2}\\
			&= \iprod{Z, \Brac{\Paren{Q-\alpha \cdot v\transpose{v}} +  \alpha \cdot v\transpose{v} }^2}\\
			&= \iprod{Z, \Paren{Q-\alpha \cdot v\transpose{v}}} + \alpha \iprod{Z, \Paren{Q-\alpha \cdot v\transpose{v}}v\transpose{v} +v\transpose{v}\Paren{Q-\alpha \cdot v\transpose{v}}} + \alpha^2\iprod{Z, v\transpose{v}}\,.
		\end{align*}
		Now since $Q$ satisfies the correlation conditions in \cref{eq:properties-of-Q}
		\begin{align*}
			1-\frac{\delta^*}{2} &\leq 1-\delta^* + \alpha \iprod{Z, \Paren{Q-\alpha \cdot v\transpose{v}}v\transpose{v} +v\transpose{v}\Paren{Q-\alpha \cdot v\transpose{v}}} + \alpha^2\iprod{Z, v\transpose{v}}\\
			&\leq 1-\delta^* +\alpha\iprod{ZQ+QZ, v \transpose{v}} -\alpha^2\iprod{Z,v\transpose{v}}\\
			&\leq 1-\delta^* +\alpha\iprod{ZQ+QZ, v \transpose{v}}\,.
		\end{align*}
		Rearranging the result follows.
	\end{proof}
\end{lemma}

While \cref{lem:scale-free-correlation-nearly-optimal-solution-Q} shows how $Z$ contains non-trivial information about $v$, from an algorithmic point of view the result is of little use since we do not have access to $Q$.
However, as $Q$ and $\tilde{Q}$ are close, \textit{also} the matrix $Z\tilde{Q}+\tilde{Q}Z$ will be correlated with $\dyad v$.

\begin{lemma}\label{lem:scale-free-correlation-nearly-optimal-solution-corrupted-Q}
	Let $Z\in \R^{n\times n}$ be a feasible solution to program \ref{eq:program-small-entries} such that
	\begin{align*}
		\iprod{ZQ+QZ, v \transpose{v}}\geq \frac{\delta^*}{2\alpha}\,.
	\end{align*}
	Then 
	\begin{align*}
		\iprod{Z\tilde{Q}+\tilde{Q}Z, v\transpose{v}} \geq \frac{\delta^*}{2\alpha} -2\frac{\gamma^2}{n}\beta\,.
	\end{align*}
	\begin{proof}
		We may rewrite the product as,
		\begin{align*}
			\iprod{Z\tilde{Q}+\tilde{Q}Z, v\transpose{v}} &= \iprod{Z\Paren{\tilde{Q}-Q+Q}+\Paren{\tilde{Q}-Q+Q}Z, v \transpose{v}}\\
			&= \iprod{Z\Paren{\tilde{Q}-Q}+\Paren{\tilde{Q}-Q}Z, v \transpose{v}} + \iprod{ZQ+QZ, v \transpose{v}}\,.
		\end{align*}
		Since $v$ is a flat unit vector and using the conditions on $Q,\tilde{Q}$,
		\begin{align*}
			\iprod{Z\Paren{\tilde{Q}-Q}, v \transpose{v}} &\leq \frac{1}{n}\underset{i,j \in [n]}{\sum}\Iprod{Z_i,\Paren{ Q-\tilde{Q}}_j}\\
			&\leq \frac{\gamma^2}{n^2}\Normo{Q-\tilde{Q}}\\
			&\leq \frac{\gamma^2}{n}\beta\,.
		\end{align*}
		An analogous computation for $\iprod{\Paren{\tilde{Q}-Q}Z, v \transpose{v}} $ concludes the proof.
	\end{proof}
\end{lemma}

We can now prove \cref{lem:scale-free-correlation}.

\begin{proof}[Proof of \cref{lem:scale-free-correlation}]
	Let $Z\in \R^{n\times n}$ be a feasible solution to program \ref{eq:program-scale-free-main} satisfying
	\begin{align*}
		\iprod{Z, \tilde{Q}^2}\geq 1-e^{2C^*}\cdot \frac{\zeta^2}{\sqrt{\gamma}}-4\zeta\cdot \gamma^2\cdot \beta\,.
	\end{align*}
	By \cref{lem:scale-free-distance-between-matrices} we have
	\begin{align*}
		\iprod{Z, Q^2}\geq 1-e^{2C^*}\cdot \frac{\zeta^2}{\sqrt{\gamma}}-8\zeta\cdot \gamma^2\cdot \beta\,.
	\end{align*}
	By assumptions on the parameters, and applying \cref{lem:scale-free-correlation-nearly-optimal-solution-Q} and \cref{lem:scale-free-correlation-nearly-optimal-solution-corrupted-Q} the result follows.
\end{proof}

\section{Robust Recovery for stochastic block model}\label{sec:sbm-recovery}
\newcommand{\degbound}{\Delta}
\newcommand{\valueofs}{10^{10}\Paren{1+\frac{1}{\delta}}\Paren{\max \Set{\log d,\log \log \frac{2}{\eps},  1,  \log\frac{1}{\delta}}}^2}

In this section we show how to use \cref{alg:scale-free-recovery} to obtain a robust algorithm for the stochastic block model and prove \cref{thm:main}.
First recall our settings. 

\begin{problem}\label{problem:sbm_technical}
	Let $\delta, \rho > 0$.
	For a pair $(\mathbf{x},\mathbf{G})\sim\sbm$ where $d=\Paren{1+\delta}\frac{4}{\eps^2}$  we are given $\delta$  and a graph $G^\circ$ 
	obtained from $\mathbf{G}$ applying at most $\rho\cdot n$ arbitrary edge edits.\footnote{Hence we assume $\rho\cdot n$ to be an integer} The goal is then to return a \textit{unit} vector $\hat{\mathbf{x}}\in \R^n$ that is $\delta^{O(1)}$-correlated  with $\mathbf{x}$ in the sense that $\iprod{\hat{\mathbf{x}}, \mathbf{x}}^2\geq \delta^{O(1)}\Snorm{\mathbf{x}}$.
\end{problem}

\begin{remark}[Learning the distribution parameters $\delta,d,\eps$]
	In principle the parameter $d, \epsilon$ of the distribution may not be known.
	However, we can easily learn good bounds on $d,\eps$. A constant upper bound on $d$ can be obtained from $G^\circ$ with a simple counting argument. Since $d\cdot \eps^2> 1$ (as otherwise the problem is impossible to solve), we can lower bound $\eps$ by $1/\sqrt{d}$.
	These constant approximations are precise enough for our purposes. For this reason, we will simply assume that $d, \eps$ are known.
\end{remark}

We state here the main result of the section.  \cref{thm:main}  follows as a corollary.

\begin{theorem}\label{thm:sbm-recovery-technical}
	Let $(\mathbf{x},\mathbf{G})\sim \sbm$ where $d\geq (1+\delta)\frac{4}{\eps^2}$ for some constant $\delta> 0$. 
	Suppose $G^\circ$ is an arbitrary graph that differs from $\mathbf{G}$ in at most $\rho\cdot n$ edges for 
	\begin{align*}
		\rho \leq \Paren{\frac{1}{\delta}\cdot \log \frac{1}{\eps}}^{-O(1/\delta)}\,.
	\end{align*}
	Then, there exists a $n^{\text{poly}(1/\delta,\log d)}$-time algorithm (\cref{alg:sbm-robust-recovery}) that, given $G^\circ$, $\delta$,  computes a $n$-by-$n$ matrix $\mathbf{M}$ 
	such that
	\begin{align*}
		\iprod{\mathbf{M}, \dyad{\mathbf{x}}}&\geq \Snorm{\mathbf{x}}\cdot \Normn{\mathbf{M}}\cdot \delta^{O(1)}\,,
	\end{align*}
	with probability $1-o(1)$.
\end{theorem}

We split the proof of \cref{thm:sbm-recovery-technical} in two steps. First, we apply the results of section \cref{sec:scale-free-recovery-algorithm}.  That is, we show there are matrix polynomials $\mathbf{Q},\tilde{Q}$, computable respectively from the adjacency matrices $\mathbf{Y}, Y^\circ$ of  $\mathbf{G}$ and $G^\circ$ in polynomial time, such that with \textit{constant} probability $(\mathbf{Q},\tilde{Q},\mathbf{x})$ satisfies \ref{eq:properties-of-Q} for some meaningful set of parameters. Second, in \cref{sec:boosting} we show how to turn this into a high probability statement. 

Notice that as an immediate consequence of \cref{thm:sbm-recovery-technical}, applying the rounding  \cref{rounding_lemma} one gets the following corollary, which is a stronger version of the main theorem.

\begin{corollary}[Formal version of the main theorem]\label{thm:sbm-recovery-technical-vector}
	For $n$ large enough, let $(\mathbf{x},\mathbf{G})\sim \sbm$ where $d\geq (1+\delta)\frac{4}{\eps^2}$ for some constant $\delta> 0$. 
	Suppose $G^\circ$ is an arbitrary graph that differs from $\mathbf{G}$ in at most $\rho\cdot n$ edges for $\rho$ as in \cref{thm:sbm-recovery-technical}. 
	Then, there exists a polynomial-time algorithm  that, given $G^\circ$, $\delta$,  computes a  $n$-dimensional unit vector $\hat{\mathbf{x}}$ such that
	\begin{align*}
		\E_{\hat{\mathbf{x}}}
		\iprod{\hat{\mathbf{x}}, \mathbf{x}}^2 &\geq \Snorm{\mathbf{x}}\cdot\delta^{O(1)}\,.
	\end{align*}
	with high probability over $(\mathbf{x}, \mathbf{G})$.
\end{corollary}

\subsection{Applying the meta-algorithm to the stochastic block model}\label{sec:meta-algorithm-to-sbm}
In this section we prove the following weaker version of \cref{thm:sbm-recovery-technical}.

\begin{theorem}\label{thm:sbm-recovery-technical-constant-probability}
	The result in \cref{thm:sbm-recovery-technical} holds with probability at least $0.99$.
\end{theorem}

For a graph $G$ on $n$ vertices, recall the definition of the centered adjacency matrix $Y(G)\in \R^{n\times n}$, where for $i,j \in [n]$
\begin{align*}
	Y_{ij}(G):=\begin{cases}
		1-\frac{d}{n} &\text{if }ij \in E(G)\\
		-\frac{d}{n}&\text{ if } ij \notin E(G),\\
		0 &\text{if } i=j\,.
	\end{cases}
\end{align*}
When the context is clear we simply write $Y$.
As already discussed in \cref{sec:techniques}  and as we will more extensively explain in \cref{sec:bounds-moments-Q} and \cref{sec:appendix-lower-bound-non-centered}, we consider the following truncated version of the adjacency matrix.

\begin{definition}[$\Delta$-truncated adjacency matrix]\label{def:truncated_adjacency_matrix}
	Let $G$ be a graph over $[n]$ and let $\Delta\geq 0$ be an integer. We define $\overline{Y}(G)\in \Set{-\frac{d}{n},0,1-\frac{d}{n}}^{n\times n}$ to be the matrix with entries 
	\begin{align*}
			\overline{Y}_{ij}(G):=\begin{cases}
			1-\frac{d}{n} &\text{if }ij \in E(G) \text{ and }d^G(i)\leq \Delta,\, d^G(j)\leq \Delta \\
			-\frac{d}{n} &\text{if }ij \notin E(G) \text{ and }d^G(i)\leq \Delta,\, d^G(j)\leq \Delta \\
			0 &\text{otherwise.}
		\end{cases}
	\end{align*}
	for any $i,j \in [n]$. We we will denote the graph obtained from $G$ by removing vertices of degree larger than $\Delta$ by $\overline{G}$. 
\end{definition}
The matrices of interest will be the block self-avoiding walk matrix polynomials as defined in  \cref{eq:definition-polynomial-q}. We restate the definition here. For a fixed integer $s>0$, and a centered adjacency matrix $Y$ of a graph $G$, let $\Q(Y):\R^{n\times n}\rightarrow\R^{n\times n}$ be the matrix polynomial with entries:
\begin{align}
	Q_{ij}^{(s)}(Y(G))   = 
	\begin{cases}
		\frac{1}{\card{\saw{ij}{s}}}\Paren{\frac{2n}{\eps \cdot d}}^{s}\underset{H\in \saw{ij}{s}}{\sum} Y_H(G) & \text{if $i\neq j$\,,}\\
		0 &\text{otherwise.}
		\end{cases}
\end{align}
For $\mathbf{G},G^\circ$ as in \cref{thm:sbm-recovery-technical-constant-probability}, we simplify $Y(\mathbf{G})$ to $\mathbf{Y}$ and $Y(G^\circ)$ to $Y^\circ$. Similarly we  define $\overline{\mathbf{Y}}$ and $\overline{Y^\circ}$ to be the $\Delta$-truncated centered adjacency matrices of $\mathbf{G}$ and $G^\circ$.

The choice of the truncation threshold $\Delta>0$ depends on the proof of the  \cref{eq:techniques-push-out-schatten}.
As explained in \cref{sec:techniques},  $\Delta$ has to be a constant if we hope to solve \cref{problem:sbm_technical} for constant $\rho \in \brac{0,1}$. Indeed the smaller the  choice of $\Delta$ the larger the constant fraction $\rho$ of corruptions that our algorithm can tolerate. On the other hand, the threshold cannot be too small as otherwise we may loose too much information. 
With these mild conditions in mind, for technical reasons, we define 
\begin{align}
	\label{eq:Delta-form}
	\Delta=\left\lceil\max\left\{128e^4d^4,40Asd, 2\log(2As)+ 12A\tau s^2\cdot \log 2 + 8A^2\tau^2s^2 \left(\log\frac{6}{\epsilon}\right)^2\right\}\right\rceil\,,
\end{align}
where  $A\geq 1000s^3$ and 
\begin{align}
	\tau&=As\log\frac{6}{\epsilon}\,.\label{eq:deg1-threshold}
\end{align}
We remark that this is not a delicate choice in the sense that the results we will prove hold for many larger values of $\Delta$, as long as the fraction of corruptions $\rho$ is a small constant.
However, it is important to observe that $\Delta$ is polynomial in $(d, s)$, this will turn out to be especially useful in computing moments of $\Q( \overline{\mathbf{Y}})$ and $\Q( \overline{Y^\circ})$.
Finally, notice that $\Delta$ mildly depends on $\epsilon$, this suggests that the fraction of corruption that our algorithm can tolerate decreases as $\log \frac{1}{\eps}$ increases. We defer a more detailed discussion to \cref{sec:appendix-lower-bound-non-centered}.

The algorithm we  use is the following.
\begin{algorithmbox}[Robust Recovery for SBM]\label{alg:sbm-robust-recovery}
	\mbox{}\\
	\textbf{Input:} An instance $(G^\circ,d,\epsilon)$ of \cref{problem:sbm_technical}.
	\begin{enumerate}[1.]
		\item Fix  $\Delta,A,\tau$ as in \cref{eq:Delta-form} and $t=\log n/400$. 
		
		Let $s\geq \valueofs$.
		
		Compute $\Q(\overline{Y^\circ})$ where $\overline{Y^\circ}$ is the $\Delta$-truncation of $Y^\circ$.
		\item Run \cref{alg:scale-free-recovery} on the rescaled matrix  $\Q(\overline{Y^\circ})/\Paren{  \E\Tr\Paren{\Q(\overline{\mathbf{Y}})}^{2t}}^{1/2t}$ with $C^*=400$ and a large enough universal constant $\gamma>0$.
	\end{enumerate}
\end{algorithmbox}

The proof that \cref{alg:sbm-robust-recovery} solves \cref{problem:sbm_technical} essentially amounts to showing that the tuple
$$\Paren{\Q(\overline{\mathbf{Y}})/\Paren{\Tr\Paren{\Q(\overline{\mathbf{Y}})}^{2t}}^{1/2t},\Q(\overline{Y^\circ})/\Paren{\E\Tr\Paren{\Q(\overline{\mathbf{Y}})}^{2t}}^{1/2t}, \mathbf{x}/\norm{\mathbf{x}}}$$
satisfies \ref{eq:properties-of-Q} for a meaningful range of parameters, with sufficiently large probability over the realization of $\mathbf{G}, \mathbf{x}$, \textit{for all} admissible $G^\circ$.
In particular, the main additional ingredients needed to prove \cref{thm:sbm-recovery-technical-constant-probability} consists of bounds on the moments of $\Q(\overline{\mathbf{Y}})$ and $\Q(\overline{\mathbf{Y}})-\dyad{\mathbf{x}}$.  
We remark that, since in practice we don't have access to $\Paren{\Tr\Paren{\Q(\overline{\mathbf{Y}})}^{2t}}^{1/2t}$, we can only scale down $\Q(\overline{Y^\circ})$ by the expectation of $\Paren{\Tr\Paren{\Q(\overline{\mathbf{Y}})}^{2t}}^{1/2t}$. It turns out however that this random variable concentrates around its expectation, so the error introduced with this rescaling will be negligible. 

For simplicity let us write
\begin{align*}
	\mathbf{Q} &:= \Q(\overline{\mathbf{Y}})/\Paren{\Tr\Paren{\Q(\overline{\mathbf{Y}})}^{2t}}^{1/2t}\\
	\tilde{Q} &:=\Q(\overline{Y^\circ})/\Paren{\E\Tr\Paren{\Q(\overline{\mathbf{Y}})}^{2t}}^{1/2t}\,,\\
	\overline{\mathbf{x}} &:= \mathbf{x}/\norm{\mathbf{x}}\,.
\end{align*}
The central step in the proof of  \cref{thm:sbm-recovery-technical-constant-probability} is the statement below.
\begin{lemma}\label{lem:sbm-recovery-appropriate-parameters}
	Consider the settings of \cref{thm:sbm-recovery-technical-constant-probability}
	Then the tuple $(\mathbf{Q},	\tilde{Q},\overline{\mathbf{x}} )$  satisfy  \cref{eq:properties-of-Q} with parameters 
	\begin{enumerate}
		\item $\alpha\leq 2\,,$
		\item $\delta^*=\delta^{O(1)}\,,$
		\item $\beta = 2\rho\cdot \degbound^{s+1}\,,$
		\item $C^*=400$
		\item $\zeta =10\degbound^{s+1}\cdot e^{2C^*}\,,$
	\end{enumerate}
	and a constant $\gamma>0$ depending only on $s, C^*, \log (1/\epsilon)$, with probability at least $0.99$.
\end{lemma}

We  prove \cref{lem:sbm-recovery-appropriate-parameters} in \cref{sec:sbm-satisfies-constraints} and directly use it here.

\begin{proof}[Proof of \cref{thm:sbm-recovery-technical-constant-probability}]
	In order to apply \cref{thm:scale-free-main} we need to show that
	\begin{align*}
		e^{2C^*}\frac{\zeta^2}{\sqrt{\gamma}}&\leq \frac{\delta^*}{4}\,,\\
		8\zeta\cdot \gamma^2\cdot \beta&\leq \frac{\delta^*}{4}\,.
	\end{align*}
	By \cref{lem:sbm-recovery-appropriate-parameters} and choice of $\Delta, s,A,\tau$ for $\rho \leq \frac{\delta^*}{4}\Paren{100 \cdot \gamma^6 \cdot \Delta^{2s+2}\cdot e^{2C^*}}^{-1}$ a direct application of \cref{thm:scale-free-main} shows that with probability at least $0.99$ the algorithm finds a $n$-times-$n$ matrix $\mathbf{M}$ satisfying
	\begin{align*}
		\iprod{\mathbf{M}, \dyad{\mathbf{x}}}\geq \Normf{\dyad{\mathbf{x}}}\cdot \Normn{\mathbf{M}}\cdot \delta^{O(1)}\,.
	\end{align*}

	For the running time dependence, we note that the entries in $Q^{(s)}$ are degree
	$s$ polynomials in $n$ variables, thus we can evaluate the matrix $Q^{(s)}$
	in time $n^{O(s)}$. Then the convex programming, matrix powering and rounding 
	can all be solved in
	$\text{poly}(n)$ time as we take
	\begin{align*}
		s \geqslant 10^{10}\left(1+\frac{1}{\delta}\right)\left(\max \left\{\log d, \log \log \frac{2}{\varepsilon}, 1, \log \frac{1}{\delta}\right\}\right)^{2}\,.
	\end{align*}
	Therefore the running time of the algorithm can be bounded by 
	$n^{\text{poly}(1/\delta,\log d)}$
\end{proof}

\subsubsection{The stochastic block model satisfies correlation, sensitivity and perturbations constraints}\label{sec:sbm-satisfies-constraints}
In this section we prove \cref{lem:sbm-recovery-appropriate-parameters}.
Our central tool will be the following concentration inequalities, which we prove in \cref{sec:bounds-moments-Q}.

\begin{lemma}\label{lem:every-little-thing-is-gonna-be-alright}
	Consider the settings of \cref{thm:sbm-recovery-technical-constant-probability}. Let $\Delta, A, \tau$ be as defined in \cref{eq:Delta-form} and let $t=\frac{\log n}{400}$. Then
	\begin{align}
		\Brac{\Tr\Paren{\Q(\overline{\mathbf{Y}})}^{2t}}^{1/2t}= & \Paren{1\pm o(1)} \E \Brac{\Tr\Paren{\Q(\overline{\mathbf{Y}})}^{2t}}^{1/2t}\,,\label{eq:event-bound-trace}\\
		\Brac{\Tr\Paren{\Q(\overline{\mathbf{Y}})-\dyad{\mathbf{x}}}^{2t}}^{1/2t}\leq & (1+\delta)^{-1/10} \E \Brac{\Tr\Paren{\Q(\overline{\mathbf{Y}})}^{2t}}^{1/2t}\,,\label{eq:event-bound-centered-trace}\\
		\underset{i \in [n]}{\sum}\Paren{\Paren{\Q(\overline{\mathbf{Y}})}^{2t-2}}^2_{ii}\leq &\frac{\gamma}{n}\cdot \Paren{\Tr \Paren{\Q(\overline{\mathbf{Y}})}^{2t-2}}^2\label{eq:event-bound-entries}\,,
	\end{align}
	for a universal constant $\gamma>0$, with probability at least $0.99$.
	Moreover
	\begin{align}
		\frac{n}{2}\leq \E \Brac{\Tr\Paren{\Q(\overline{\mathbf{Y}})}^{2t}}^{1/2t}&\,. \label{eq:event-bound-expectation-trace} 
	\end{align}
\end{lemma}

\begin{remark}
	The careful reader may have noticed that so far, we never  explicitly compute the quantity $\Paren{\E\Tr\Paren{\Q(\overline{\mathbf{Y}})}^{2t}}^{1/2t}$. In practice, for each value of $n$, we need to compute this expectation to sufficiently close precision only \textit{once} and not at every run of the algorithm. This can be done efficiently and accurately (with high probability) by sampling several graphs from $\sbm$, compute the corresponding $2t$-Schatten norm and take the average. 
\end{remark}

By scaling $\Tr\mathbf{Q}^{2t}=1$. We show now that $\mathbf{Q}$ satisfy the second correlation constraint in \ref{eq:properties-of-Q} .

\begin{lemma}\label{lem:sbm-correlation-constraint}
	Consider the settings of \cref{thm:sbm-recovery-technical-constant-probability}. Suppose \cref{eq:event-bound-trace} and \cref{eq:event-bound-centered-trace} are satisfied.
	Then  for some $0<\alpha\leq 2$ and $\delta^*=\delta^{O(1)}$,
	\begin{align*}
		\Tr\Paren{\mathbf{Q} -\alpha\cdot \dyad{\overline{\mathbf{x}}}}^{2t}\leq \Paren{1-\delta^*}^{-2t}\,. 
	\end{align*}
	\begin{proof}
		By \cref{eq:event-bound-trace} and \cref{eq:event-bound-centered-trace} 
		\begin{align*}
			\Brac{\Tr\Paren{\Q(\overline{\mathbf{Y}})-\dyad{\mathbf{x}}}^{2t}}^{1/2t} \leq 	\Paren{1+\delta}^{-1/11}\Brac{\Tr\Paren{\Q((\overline{\mathbf{Y}})}^{2t}}^{1/2t}\,.
		\end{align*}
		By \cref{eq:event-bound-expectation-trace} $\Brac{\Tr\Paren{\Q(\overline{\mathbf{Y}})}^{2t}}^{1/2t}\geq n/2$, so
		rescaling by $1/\Brac{\Tr\Paren{\Q(\overline{\mathbf{Y}})}^{2t}}^{1/2t}$, the polynomial satisfy the correlation constraints in \ref{eq:properties-of-Q} with some $0< \alpha \leq 2$ and $\delta^*=\delta^{O(1)}\,.$
	\end{proof}
\end{lemma}

Next we show how by construction both $\mathbf{Q}$ and $\tilde{Q}$ have columns bounded in $\ell_1$-norm. 

\begin{lemma}\label{lem:sbm-truncation-bound-l1-norm}
	Consider the settings of \cref{thm:sbm-recovery-technical-constant-probability}. Suppose \cref{eq:event-bound-trace} and \cref{eq:event-bound-centered-trace} are satisfied. Then
	\begin{align*}
		\max_{i \in [n]}\underset{j \in [n]}{\sum}\Abs{\mathbf{Q}_\ij}&\leq 2\degbound^{s+1} \\
		\max_{i \in [n]}\underset{j \in [n]}{\sum}\Abs{\tilde{Q}_\ij}&\leq 2 \degbound^{s+1}\,.
	\end{align*}
	\begin{proof}
		It suffices to show a bound for any graph with no  vertex with degree larger than $\Delta$. 
		So let $\overline{G}$ be such a graph. 
		Let $\overline{Y}$ be its centered adjacency matrix (notice that its adjacency matrix is also its $\Delta$-truncated adjacency matrix, so we may use both definition interchangeably). 
		For $u,v \in V(\overline{G})$ and $\ell \in [s]$, define  the set 
		$\cW_{u,\ell}(\overline{G}):=\Set{W\in \saw{u}{s} \suchthat \Card{E(W)\setminus E(\overline{G})} = \ell}$.
		That is $\cW_{u,\ell}(\overline{G})$ contains the set of self-avoiding walks over $K_n$ starting from $u$ which  have exactly $\ell$ edges not in $E(\overline{G})$.
		Notice that for any $W\in \cW_{u,\ell}(\overline{G})$, $\overline{Y}_W\leq \Paren{\frac{d}{n}}^{\ell}$, and
		$\Q_{uv}(\overline{Y})\leq n \underset{W\in \saw{uv}{s}}{\sum} \overline{Y}_W$.
		Recall that we denote the set of vertices at distance $s$ from $u\in V(\overline{G})$ by $N^{s}_{\overline{G}}(u)$.
		Now, for any $u\in [n]$, by assumption $\Card{N^s_{\overline{G}}(u)}\leq \degbound^s$.
		So
		\begin{align*}
			\underset{v \in [n]}{\sum}\Abs{\Q_{uv}(\overline{Y})} &\leq n\cdot \underset{\ell \in [s-1] }{\sum}\quad \underset{W \in \cW_{u,\ell}(H)}{\sum}\Abs{\overline{Y}_W}\\
			&\leq n\cdot \underset{\ell \in [s-1]}{\sum}\quad \underset{W \in \cW_{u,\ell}(H)}{\sum} \Paren{\frac{d}{n}}^{\ell}\,.
		\end{align*}
		For any $\ell \in [s]$, there are at most $\degbound^{s-\ell}\cdot n^{\ell}$ self-avoiding walks in $\cW_{u,\ell}(\overline{G})$, thus
		\begin{align*}
			\underset{v \in [n]}{\sum}\Abs{\Q_{uv}(\overline{Y})} &\leq n\cdot s\cdot  \degbound^{s}\,.
		\end{align*}
		Since \cref{eq:event-bound-trace} and \cref{eq:event-bound-centered-trace} are verified by assumption, applying the analysis above to $\overline{\mathbf{Y}}$, $\overline{Y^\circ}$ and  scaling down the inequalities the result follows.
	\end{proof}
\end{lemma}

\cref{eq:event-bound-entries} immediately implies the remaining sensitivity constraint on $\mathbf{Q}$. Thus we only need to show that $\mathbf{Q}$ and $\tilde{Q}$ are close in a $\ell_1$-norm sense.
To do that, we need to first argue how many different edges $\overline{\mathbf{G}}$ and $\overline{G^\circ}$ may have.

\begin{fact}\label{fact:sbm-recovery-sequence-of-edits}
Let $\Delta\geq 0$ be an integer.
Let $G, G^\circ$ be graphs on $n$ vertices that differs in at most $\rho\cdot n$ edges, for some $\rho\geq 0$.
Let $\overline{G}$ and $\overline{G^\circ}$ be respectively the graphs obtained from $G$ and $G^\circ$ by removing all vertices with degree larger than $\Delta$.
Then $\overline{G}$ and $\overline{G^\circ}$ differs by at most $2\rho\cdot n\cdot \degbound$ edges.
\begin{proof}
	$\overline{G}$ and $G^\circ$ differs by $\rho\cdot n$ edges. Each such edge changes the degree of at most $2$ vertices, thus after the truncation $\overline{G}$ and $\overline{G^\circ}$ differs by at most $2\rho\cdot n\cdot \Delta$ edges.
\end{proof}
\end{fact}

Next we show that if $G$ and $G^\circ$ are two arbitrary graphs (with bounded maximum degree) which differ by at most one edge, than the matrices $\Q(\overline{Y}(G))$ and $\Q(\overline{Y}(G^\circ))$ are close.

\begin{lemma}\label{lem:sbm-recovery-distance-after-one-edge-edit}
	Let $\Delta\geq 0$ be an integer.
	Let $\overline{G}$ be a graph on $n$ vertices with maximum degree at most $\Delta$.
	 Let $\overline{G^\circ}=\overline{G}+uv$ for some $u,v\in V(\overline{G})$ such that $uv\notin E(\overline{G})$ and $d^{\overline{G^\circ}}(u),d^{\overline{G^\circ}}(v)\leq \Delta$.
	Denote respectively by $\overline{Y},\overline{Y^\circ}$ the centered adjacency matrices of $\overline{G}$ and $\overline{G^\circ}$. 
	Then 
	\begin{align*}
		\Normo{\Q(\overline{Y})-\Q(\overline{Y^\circ})}\leq n\cdot \degbound^{s}\,.
	\end{align*}
	\begin{proof}
		It suffices to consider length-$s$ self-avoiding walks traversing $uv$.
		We reuse the notation introduced in \cref{lem:sbm-truncation-bound-l1-norm}, thus for $a,b \in V(\overline{G^\circ})$ let $\cW_{a,b,\ell}(\overline{G^\circ}):=\Set{W\in \saw{ab}{s}\suchthat\Card{E(W)\setminus E(\overline{G^\circ})}=\ell}$.
		Furthermore, for $\ell \leq s$ 
		consider the set 
		\begin{align*}
			\cW_{a,b,\ell}^{uv}(\overline{G^\circ}):=\Set{W \in \cW_{a,b,\ell}(\overline{G^\circ})\suchthat  uv \in E(W)}\,.
		\end{align*}
		In other words, we look at the subsets of self-avoiding walks in  $\cW_{a,b,\ell}$ containing the edge $uv$.
		Then,
		\begin{align*}
			\Normo{\Q(\overline{Y})-\Q(\overline{Y^\circ})}&\leq n\cdot \underset{0\leq \ell \leq s-1}{\sum}\quad \underset{a,b\in [n]}{\sum} \quad \underset{W\in\cW_{a,b,\ell}^{uv}(\overline{G^\circ}) }{\sum} \Abs{\overline{Y}_W-\overline{Y^\circ}_W}\\
			&\leq n\cdot \underset{0\leq \ell \leq s-1}{\sum}\quad\underset{a,b\in [n]}{\sum} \quad \underset{W\in\cW_{a,b,\ell}^{uv}(\overline{G^\circ}) }{\sum} \Abs{\overline{Y}_W}+\Abs{\overline{Y^\circ}_W}\\
			&\leq n\cdot \underset{0\leq \ell \leq s-1}{\sum}\quad\underset{a,b\in [n]}{\sum} \quad \underset{W\in\cW_{a,b,\ell}^{uv}(\overline{G^\circ}) }{\sum} \Paren{\frac{d}{n}}^{\ell+1}+\Paren{\frac{d}{n}}^{\ell}\\
			&= \Paren{1+o(1)}n\cdot \underset{0\leq \ell \leq s-1}{\sum}\quad\underset{a,b\in [n]}{\sum} \quad \underset{W\in\cW_{a,b,\ell}^{uv}(\overline{G^\circ}) }{\sum}\Paren{\frac{d}{n}}^{\ell}\,.
		\end{align*}
		For any $0\leq \ell \leq s-1$ we have $\underset{a,b \in [n]}{\bigcup}\cW_{a,b,\ell}^{uv}(\overline{G^\circ})\leq s\cdot \degbound^{s-\ell-1}\cdot n^{\ell}$.
		It follows that 
		\begin{align*}
			\Normo{\Q(\overline{Y})-\Q(\overline{Y^\circ})}&\leq n\cdot  \degbound^{s}\,.
		\end{align*}
	\end{proof}
\end{lemma}

Now we can we show that $\mathbf{Q}$ and $\tilde{Q}$ satisfy the first perturbation constraint.
\begin{lemma}\label{lem:sbm-perturbation-constraint}
	Consider the settings of \cref{thm:sbm-recovery-technical-constant-probability}. Suppose \cref{eq:event-bound-trace} holds. Then
	\begin{align*}
		\frac{1}{n}\Normo{\mathbf{Q}-\tilde{Q}}\leq 2\Delta^{s+1}\cdot \rho\,.
	\end{align*}
	\begin{proof}
		By \cref{fact:sbm-recovery-sequence-of-edits} there is a sequence $\Set{\overline{G}_i}$ of $2\rho\cdot \Delta\cdot n$ graphs with maximum degree $\Delta$ such that  $\overline{G}_{2\rho\cdot \Delta\cdot n}=\overline{G^\circ}$,   $\overline{G}_1=\overline{\mathbf{G}}$ and for any $i\in [\rho\cdot \Delta\cdot n-1]$, $\overline{G}_i$ and $\overline{G}_{i+1}$ differ by at most one edge. For each $\overline{G}_i$ let $Y_i$ be its centered adjacency matrix.
		Then
		\begin{align*}
			\Normo{\Q(\mathbf{Y}_1)-\Q(\mathbf{Y}_{2\rho\cdot \Delta\cdot n})} &\leq \underset{i \in [2\rho\cdot \Delta\cdot n-1]}{\sum} \Normo{\Q(Y_i)-\Q(Y_{i+1})}\\
			&\leq 2\rho\cdot   n^2\cdot \Delta^{s+1}\,,
		\end{align*}
		where we used \cref{lem:sbm-recovery-distance-after-one-edge-edit} in the last step.
		Since 
		\[\Normo{\Q(\overline{Y}(\mathbf{G}))-\Q(\overline{Y}(G^\circ))}=\Normo{ \Q(Y(\overline{\mathbf{G}})) -\Q(Y(\overline{G^\circ}))}\,,\]
		 rescaling the lemma follows.
	\end{proof}
\end{lemma}

\paragraph{Putting things together}  We are ready to prove \cref{lem:sbm-recovery-appropriate-parameters}.

\begin{proof}[Proof of \cref{lem:sbm-recovery-appropriate-parameters}.]
	We condition our analysis on the event that \cref{eq:event-bound-trace}, \cref{eq:event-bound-centered-trace} and \cref{eq:event-bound-entries} are verified, which by \cref{lem:every-little-thing-is-gonna-be-alright} happen with probability $0.99$.
	Consider the tuple $\Paren{\mathbf{Q}, \tilde{Q}, \overline{\mathbf{x}}}$
	By \cref{lem:sbm-correlation-constraint} the correlation constraints are satisfied for $\delta^*=\delta^{O(1)}$ and some $0<\alpha\leq 2$. By \cref{lem:sbm-truncation-bound-l1-norm} it holds that
	\begin{align*}
		\max_{i \in [n]}\underset{j \in [n]}{\sum}\Abs{\mathbf{Q}_i}&\leq \zeta  \\
		\max_{i \in [n]}\underset{j \in [n]}{\sum}\Abs{\tilde{Q}_i}&\leq \zeta\,,
	\end{align*}
	for any $\zeta \geq 2\degbound^{s+1}$.
	By \cref{lem:sbm-perturbation-constraint}, the tuple satisfies the remaining perturbation constraint for $\beta=2\Delta^{s+1}\cdot \rho$. 
	The result follows.
\end{proof}

\subsection{Boosting the proability of success}\label{sec:boosting}

In this section we conclude the proof of \cref{thm:sbm-recovery-technical} showing how to increase the probability of success of \cref{alg:sbm-robust-recovery}.

\newcommand{\boundonkappa}{\delta^{O(1)}}
\newcommand{\boundonrho}{\Paren{\frac{1}{\delta}\cdot \log \frac{1}{\eps}}^{-O(1/\delta)}}
\newcommand{\boundons}{\valueofs}
\newcommand{\expressionofp}{0.99}

\begin{definition}
\label{def:def_success_event}
Let $M$ be a function that takes as input a graph $G^{\circ}$ having $[n]$ as the set of vertices, and produces an $n\times n$ matrix $M(G^{\circ})$ as output. Let $(\mathbf{G},\mathbf{x})\sim \sbm$, and define $\mathcal{E}_{n,d,\epsilon,M,\rho,\kappa}^{\text{succ}}$ to be the event that the function $M$ succeeds, up to robustness $\rho$, in providing an output that $\kappa$-correlates with the community labels.

More precisely, if $(\mathbf{G},\mathbf{x})\sim \sbm$, then $\mathcal{E}_{n,d,\epsilon,M,\rho,\kappa}^{\text{succ}}$ is the event that for every graph $G^{\circ}$ that differs from $\mathbf{G}$ by at most $\rho n$ edges, we have
$$\iprod{M(G^{\circ}), \mathbf{x}\mathbf{x}^T}\geq \kappa\cdot\Norm{\mathbf{x}}^2\cdot\Normn{M(G^{\circ})}.$$
\end{definition}

\cref{thm:sbm-recovery-technical-constant-probability} implies that if $M$ is \cref{alg:sbm-robust-recovery} with 
\[
s\geq \boundons\,,
\] then if $d\geq (1+\delta)\frac{4}{\eps^2}$ for some constant $\delta> 0$, and if $\rho\leq  \boundonrho$ and $\kappa\leq  \boundonkappa$, then for $n$ large enough, we have
\begin{align*}
\mathbb{P}\big[\mathcal{E}_{n,d,\epsilon,M,\rho,\kappa}^{\text{succ}}\big]\geq \expressionofp.
\end{align*}

In this section we will prove that, at the expense of paying a negligible price in the robustness and correlation guarantees, we can boost the success probability and make it converge to 1 at a rate that is roughly exponential in $n^{\frac{1}{3}}$. In fact, we will prove a very general boosting result:

\begin{theorem}
\label{thm:thm-boosting}
Let $\epsilon,d$ be such that $d\geq (1+\delta)\frac{4}{\eps^2}$ for some constant $\delta> 0$. Let $M$ be an arbitrary algorithm that takes as input a graph with $[n]$ as the set of vertices, and produces an $n\times n$ matrix as output. For every $0<\rho'<\rho$ and every $0<\kappa'<\kappa$, if
$$\mathbb{P}\big[\mathcal{E}_{n,d,\epsilon,M,\rho,\kappa}^{\text{succ}}\big]\geq \Omega\Paren{e^{-n^{\frac{1}{4}}}},$$
then for $n$ is large enough, we have
$$\mathbb{P}\big[\mathcal{E}_{n,d,\epsilon,M,\rho',\kappa'}^{\text{succ}}\big]\geq 1-e^{-n^{\frac{1}{4}}}.$$
The theorem remains correct if we replace the exponent $\frac{1}{4}$ with any constant $\xi<\frac{1}{3}$.
\end{theorem}

It is worth noting that this boosting argument is not unique to the stochastic block model, and similar results hold in a wide range of estimation problems.

Combining \cref{thm:thm-boosting} and \cref{thm:sbm-recovery-technical-constant-probability},
 we obtain \cref{thm:sbm-recovery-technical}.

In order to prove \cref{thm:thm-boosting}, we will use a concentration of measure inequality known as the \emph{blowing-up lemma}, and which is widely used in information theory to prove strong converse results. In order to describe this lemma, we need the following definition:

\begin{definition}
Let $\mathcal{Y}$ be an arbitrary finite set, and let $y=(y_1,\ldots,y_N)\in\mathcal{Y}^N$ and $y'=(y_1',\ldots,y_N')\in\mathcal{Y}^N$. The \emph{Hamming distance} between $y$ and $y'$ is defined as
$$D_H(y,y')=\big|\big\{i\in[N]:\;y_i\neq y_i'\big\}\big|.$$

If $\mathcal{E}$ is a subset of $\mathcal{Y}^N$, we define the \emph{$\ell$-blowup} of $\mathcal{E}$ as:
$$\Gamma^\ell (\mathcal{E})=\big\{y'\in\mathcal{Y}^n:\;\exists y\in \mathcal{E}, D_H(y,y')\leq\ell\big\}.$$

In other words, the $\ell$-blowup of $\mathcal{E}$ is the set of elements of $\mathcal{Y}^n$ that are at a Hamming distance of at most $\ell$ from $\mathcal{E}$.
\end{definition}

Roughly speaking, the blowing-up lemma states that if we have $n$ independent random variables $\mathbf{y}_1,\ldots,\mathbf{y}_N$ taking values in a finite set $\mathcal{Y}$, and if $\mathcal{E}\subseteq\mathcal{Y}^n$ is an event whose probability is not too small, then if we "inflate" $\mathcal{E}$ a little by adding the elements of $\mathcal{Y}^N$ that are close to $\mathcal{E}$ in Hamming distance, then the probability of the event becomes $1-o(1)$.

The blowing-up lemma was first introduced in \cite{AhlswedeEtAlBlowingUp}, and there are several versions of it (e.g., \cite{MartonBlowingUp1} and \cite{MartonBlowingUp2}). We will use the following version that was proved by Marton in \cite{MartonBlowingUp2} (see also Lemma 3.6.2 in \cite{ConcentrationOfMeasureRaginskySason}):

\begin{lemma}
\label{lem:lem_blowing_up}
[Blowing up Lemma \cite{MartonBlowingUp2}] Let $\mathbf{y}_1,\ldots,\mathbf{y}_N$ be $N$ random variables taking values in the same finite set $\mathcal{Y}$ that can be of arbitrary size. If $\mathbf{y}_1,\ldots,\mathbf{y}_N$ are independent (but not necessarily identically distributed), then for every $\mathcal{E}\subseteq\mathcal{Y}^N$ and every $\ell> \sqrt{\frac{N}{2}\log \Paren{\frac{1}{\mathbb{P} [ \mathcal{E}]}}}$, we have
\begin{align*}
\mathbb{P}[\Gamma^\ell(\mathcal{E})]&\geq 1-\exp\left[-\frac{2}{N}\left(\ell-\sqrt{\frac{N}{2}\log \Paren{\frac{1}{\mathbb{P} [ \mathcal{E}]}}}\right)^2\right]\\
&= 1-\exp\left[-2N\left(\frac{\ell}{N}-\sqrt{\frac{1}{2N}\log \Paren{\frac{1}{\mathbb{P} [ \mathcal{E}]}}}\right)^2\right],
\end{align*}
where $\mathbb{P}[\Gamma^\ell(\mathcal{E})]=\mathbb{P}\big[(\mathbf{y}_1,\ldots,\mathbf{y}_N)\in\Gamma^\ell(\mathcal{E})\big]$ and $\mathbb{P}[ \mathcal{E}]=\mathbb{P} \big[(\mathbf{y}_1,\ldots,\mathbf{y}_N)\in \mathcal{E}\big]$.
\end{lemma}

\begin{remark}
\label{rem:rem_blowing_up}
As can be easily seen from the lemma, if $(\mathbf{y}_N)_{N\geq 1}$ is a sequence of independent (but not necessarily identically distributed) random variables taking values in $\mathcal{Y}$, and if $(\mathcal{E}_N)_{N\geq 1}$ is a sequence of events such that $\mathcal{E}_N\subseteq\mathcal{Y}^N$ and  $\mathbb{P}[\mathcal{E}_N]$ is not exponentially small in the sense that $\displaystyle\lim_{N\rightarrow\infty}\frac{1}{N}\log\mathbb{P}[\mathcal{E}_N]=0$, then:
\begin{itemize}
\item We can find a sequence of integers $(\ell_N)_{N\geq 1}$ such that $\displaystyle\lim_{N\rightarrow\infty}\frac{\ell_N}{N}=0$ and $\displaystyle\lim_{N\rightarrow\infty}\mathbb{P}[\Gamma^{\ell_N}(\mathcal{E}_N)]=1$.
\item We need $\ell_N=\omega(\sqrt{N})$ in order for the lemma to guarantee that $\displaystyle\lim_{N\rightarrow\infty}\mathbb{P}[\Gamma^{\ell_N}(\mathcal{E}_N)]=1$.
\end{itemize}
\end{remark}

In the following, we will show how we can apply the blowing-up lemma to boost the success probability of a weak-recovery algorithm and make it converge to 1 at a rate that is exponential in $n^{\frac{1}{4}}$. We will do this in two steps. First, we show how to boost the conditional probability of success given the community labels $\mathbf{x}$, and then we show how to boost the success probability unconditionally.

\subsubsection{Boosting the conditional probability of success given the community labels}

\label{subsec:subsec_boost_cond_prob}

Since our algorithm is robust against a linear number of adversarial edge changes, namely $\rho n$ changes, we can benefit from the blowing-up lemma to boost the success probability to $1-e^{-n^{\frac{1}{4}}}$ by paying a negligible price in the robustness of the algorithm. However, we need to be careful how we apply the blowing-up lemma, because in $(\mathbf{G},\mathbf{x})\sim\sbm$, we have $N=\frac{n(n-1)}{2}=\Omega(n^2)$ (conditionally\footnote{The edges are conditionally independent given $\mathbf{x}$.}) independent random edges that may or may not be present in the graph, and our algorithm is only robust against up to $\rho n=\Theta(\sqrt{N})$ adversarial changes, whereas the naive and straightforward application of the blowing-up lemma needs a robustness of at least $\omega\left(\sqrt{N}\right)$ in order to guarantee the convergence of probability to 1.

In order to successfully apply the blowing-up lemma, we will faithfully reorganize the randomness of $(\mathbf{G},\mathbf{x})\sim\sbm$ in $n$ (conditionally) independent random variables. If the new representation also faithfully captures closeness in the sense that representations that are at Hamming distance $\ell$ induce graphs that differ by at most $\ell\cdot o(\sqrt{n})$ edges, then our algorithm is robust against up to $\frac{\rho n}{o(\sqrt{n})}=\omega(\sqrt{n})$ adversarial changes in the random variables, and this will allow us to successfully apply the blowing-up lemma.

\begin{lemma}
\label{lem:lem-boosting-conditional}
Let $\epsilon,d,M,\rho$ and $\kappa$ be as in \cref{thm:thm-boosting}. For every $\rho'<\rho$, if
$$\mathbb{P}\big[\mathcal{E}_{n,d,\epsilon,M,\rho,\kappa}^{\text{succ}}\big|\mathbf{x}\big]\geq \Omega\Paren{e^{-n^{\frac{1}{4}}}},$$
then for $n$ large enough, we have
$$\mathbb{P}\big[\mathcal{E}_{n,d,\epsilon,M,\rho',\kappa}^{\text{succ}}\big|\mathbf{x}\big]\geq 1-\frac{e^{-n^{\frac{1}{4}}}}{2}.$$
\end{lemma}
\begin{proof}
Roughly speaking, our plan is to define $n$ random variables $\mathbf{z}_1,\ldots,\mathbf{z}_{n}$ taking values in a set $\mathcal{Z}$ such that:
\begin{itemize}
\item[(a)] There is a mapping $G$ from $\mathcal{Z}^{n}$ to the set of graphs having $[n]$ as the set of vertices, such that $\mathbf{G}=G(\mathbf{z}_1,\ldots,\mathbf{z}_{n})$.
\item[(b)] For every $z,z'\in\mathcal{Z}^{n}$, if $D_H(z,z')\leq \ell$, then $G(z)$ differs from $G(z')$ by at most $2\ell\cdot n^{\frac{3}{10}}$ edges.
\item[(c)] $\mathbf{z}_1,\ldots,\mathbf{z}_{n}$ are conditionally independent given $\mathbf{x}$.
\end{itemize}
Property (a) means that $\mathbf{z}_1,\ldots,\mathbf{z}_{n}$ form a faithful representation of $\mathbf{G}$. Property (b) means that there the Hamming distance in $\mathcal{Z}^{n}$ is a good estimate for the number of edge changes between the induced graphs. Properties (b) and (c) will allow us to successfully apply the blowing-up lemma.

For every $x\in\{-1,+1\}^n$, define the set
$$\mathcal{Z}_{n,d,\epsilon,M,\rho,\kappa}^{\text{succ}}(x)=\Big\{z\in\mathcal{Z}^n:\; (G(z),x)\in \mathcal{E}_{n,d,\epsilon,M,\rho,\kappa}^{\text{succ}}\Big\}.$$
Clearly, 
$$\mathbb{P}\big[\mathbf{z}\in \mathcal{Z}_{n,d,\epsilon,M,\rho,\kappa}^{\text{succ}}(\mathbf{x})\big|\mathbf{x}\big]=\mathbb{P}[\mathcal{E}_{n,d,\epsilon,M,\rho,\kappa}^{\text{succ}}|\mathbf{x}].$$

If we have $\mathbb{P}[\mathcal{E}_{n,d,\epsilon,M,\rho,\kappa}^{\text{succ}}|\mathbf{x}]\geq \Omega\Paren{e^{-n^{\frac{1}{4}}}}$, then if we take $\ell=\lceil n^{\frac{2}{3}}\rceil$, \cref{lem:lem_blowing_up} implies that
\begin{align*}
\mathbb{P}\left[\mathbf{z}\in\Gamma^{\ell}\left(\mathcal{Z}_{n,d,\epsilon,M,\rho,\kappa}^{\text{succ}}(\mathbf{x})\right)\middle|\mathbf{x}\right]&\geq  1-\exp\left[-2n\left(\frac{n^{\frac{2}{3}}}{n}-\sqrt{\frac{1}{2n}\log \Paren{\frac{1}{\Omega\Paren{e^{-n^{\frac{1}{4}}}}}}}\right)^2\right]\\
&=1-\exp\left[-2n\left(n^{-\frac{1}{3}}-\frac{1}{\sqrt{2}} n^{-\frac{3}{8}}\sqrt{1\pm O\Paren{n^{-\frac{1}{4}}}}\right)^2\right]\geq 1-\exp\left[-n^{\frac{1}{3}}\right],
\end{align*}
where the last inequality is true for $n$ large enough.

Now assume that $(\mathbf{G},\mathbf{x})$ and $\mathbf{z}=(\mathbf{z}_1,\ldots,\mathbf{z}_{n})$ are such that $\mathbf{z}\in\Gamma^{\ell}\left(\mathcal{Z}_{n,d,\epsilon,M,\rho,\kappa}^{\text{succ}}(\mathbf{x})\right)$ and $\mathbf{G}=G(\mathbf{z})$. Let $\tilde{z}=(\tilde{z}_1,\ldots,\tilde{z}_{n})\in \mathcal{Z}_{n,d,\epsilon,M,\rho,\kappa}^{\text{succ}}(\mathbf{x})$ be such that
$$D_H(\mathbf{z},\tilde{z})\leq \ell.$$
From Property (b) we know that $G(\tilde{z})$ differs from $\mathbf{G}=G(\mathbf{z})$ by at most $2\ell\cdot n^{\frac{3}{10}}\leq 2\lceil n^{\frac{2}{3}}\rceil \cdot n^{\frac{3}{10}}=o(n)$ edges. On the other hand, since $G(\tilde{z})\in \mathcal{Z}_{n,d,\epsilon,M,\rho,\kappa}^{\text{succ}}(\mathbf{x})\,$, we have $$(G(\tilde{z}),\mathbf{x})\in \mathcal{E}_{n,d,\epsilon,M,\rho,\kappa}^{\text{succ}}.$$

Let $\rho'<\rho$ and let $G^{\circ}$ be an arbitrary graph with $V(G^{\circ})=[n]$, and which differs from $\mathbf{G}=G(\mathbf{z})$ by at most $\rho' n$ edges. Since $G(\tilde{z})$ differs from $G(\mathbf{z})$ by at most $2\ell n^{\frac{3}{10}}=o(n)$ edges, if $n$ is large enough, the graph $G^{\circ}$ differs from $G(\tilde{z})$ by at most $\rho' n+o(n)\leq \rho n$ edges. Now since $(G(\tilde{z}),\mathbf{x})\in \mathcal{E}_{n,d,\epsilon,M,\rho,\kappa}^{\text{succ}}$ and since $G^{\circ}$ differs from $G(\tilde{z})$ by at most $\rho n$ edges, we have
$$\iprod{M(G^{\circ}), \mathbf{x}\mathbf{x}^T}\geq \kappa\cdot\Snorm{\mathbf{x}}\cdot\Normn{M(G^{\circ})},$$
which implies that $(\mathbf{G},\mathbf{x})=(G(\mathbf{z}),\mathbf{x})\in \mathcal{E}_{n,d,\epsilon,M,\rho',\kappa}^{\text{succ}}$ and $\mathbf{z}\in \mathcal{Z}_{n,d,\epsilon,M,\rho',\kappa}^{\text{succ}}(\mathbf{x})$. Therefore,
$$\Gamma^{\ell}\left(\mathcal{Z}_{n,d,\epsilon,M,\rho,\kappa}^{\text{succ}}(\mathbf{x})\right)\subseteq \mathcal{Z}_{n,d,\epsilon,M,\rho',\kappa}^{\text{succ}}(\mathbf{x}).$$
Now since $\mathbb{P}\big[\mathbf{z}\in \mathcal{Z}_{n,d,\epsilon,M,\rho',\kappa}^{\text{succ}}(\mathbf{x})\big|\mathbf{x}\big]=\mathbb{P}[\mathcal{E}_{n,d,\epsilon,M,\rho',\kappa}^{\text{succ}}|\mathbf{x}]$, we conclude that
\begin{equation}
\label{eq:eq_Prob_success_conditional_on_Y_first_approach}
\mathbb{P}\left[\mathcal{E}_{n,d,\epsilon,M,\rho',\kappa}^{\text{succ}}\middle|\mathbf{x}\right]\geq  1-\exp\left[-n^{\frac{1}{3}}\right].
\end{equation}

In the following, we will show how we can define the random variables $\mathbf{z}_1,\ldots,\mathbf{z}_{n}$ so that Properties (a - c) are (almost) satisfied.

We partition the set of edges $\big\{uv:\; u,v\in [n]\big\}$ into $n$ subsets $B_1,\ldots,B_{n}$ such that each set $B_i$ has size $\left\lceil \frac{n+1}{2}\right\rceil$ or $\left\lfloor \frac{n+1}{2}\right\rfloor$. For each $1\leq i\leq n$, we fix an ordering $e^{(i)}_1,\ldots,e^{(i)}_{|B_i|}$ of the edges in $B_i$, and then define
$$\mathbf{z}_i=\left\{j\in\left\{1,\ldots,|B_i|\right\} :\; e_j^{(i)}\in \mathbf{G}\right\}.$$
Note that $\mathbf{z}_i$ is in one-to-one correspondence with $\{e\in B_i :\; e\in \mathbf{G}\}$, and so  $\mathbf{z}=(\mathbf{z}_1,\ldots,\mathbf{z}_{n})$ is in one-to-one correspondence with $\mathbf{G}$. Furthermore, $\mathbf{z}_1,\ldots,\mathbf{z}_{n}$ take values in the power set of $\left\{1,\ldots,\left\lceil\frac{n+1}{2}\right\rceil\right\}$.

It is easy to see that $\mathbf{z}_1,\ldots,\mathbf{z}_{n}$ satisfy properties (a) and (c) above, but unfortunately they do not satisfy Property (b): It is possible for a change in just one random variable $\mathbf{z}_i$ to cause $\Omega(n)$ edge changes in the graph.

In order to make the above approach work, we need one more ingredient: If we look closely, we find that the main reason why Property (b) does not hold is because it is possible for a graph to contain too many edges coming from one set $B_i$. But since we are working in the sparse regime of stochastic block models, the probability of this happening is negligible. We can leverage this phenomenon in order to make the above approach work.

More precisely, given $\mathbf{x}$, the size $|\mathbf{z}_i|$ of the set $\mathbf{z}_i$ is the sum of $|B_i|=\Theta(n)$ Bernoulli random variables:
$$|\mathbf{z}_i|=\sum_{uv\in B_i}\mathbbm{1}_{uv\in\mathbf{G}}.$$
Now for every $uv\in B_i$, we have
$$\mathbb{P}[uv\in\mathbf{G}|\mathbf{x}]=\left(1+\frac{\epsilon \mathbf{x}_u\mathbf{x}_v}{2}\right)\frac{d}{n}\leq\frac{2d}{n}.$$
It follows from the Chernoff bound that
\begin{equation}
\label{eq:eq_chernoff_eq_edges_1}
\begin{aligned}
\mathbb{P}\big[|\mathbf{z}_i|>n^{\frac{3}{10}}\big|\mathbf{x}\big]\leq \exp\left(-|B_i|\cdot D_{KL}\left(\frac{n^{\frac{3}{10}}}{|B_i|}\;\Big|\Big|\;\frac{2d}{n}\right)\right),
\end{aligned}
\end{equation}
where $D_{KL}(p||q)=p\log\frac{p}q + (1-p)\log\frac{1-p}{1-q}$ is the Kullback-Leibler divergence between  $\text{Bernoulli}(p)$ and $\text{Bernoulli}(q)$. Now notice that
\begin{equation}
\label{eq:eq_chernoff_eq_edges_2}
\begin{aligned}
D_{KL}\left(\frac{n^{\frac{3}{10}}}{|B_i|}\;\Big|\Big|\;\frac{2d}{n}\right)&=  \frac{n^{\frac{3}{10}}}{|B_i|}\log \frac{\frac{n^{\frac{3}{10}}}{|B_i|}}{\frac{2d}{n}} + \left(1-\frac{n^{\frac{3}{10}}}{|B_i|}\right)\log \frac{1-\frac{n^{\frac{3}{10}}}{|B_i|}}{1-\frac{2d}{n}}\\
&=  \frac{n^{\frac{3}{10}}}{\Theta(n)}\log \frac{\frac{n^{\frac{3}{10}}}{\Theta(n)}}{\frac{2d}{n}} + \left(1-\frac{n^{\frac{3}{10}}}{\Theta(n)}\right)\log \frac{1-\frac{n^{\frac{3}{10}}}{\Theta(n)}}{1-\frac{2d}{n}}\\
&=  \Theta\Paren{n^{-\frac{7}{10}}}\log \Theta\Paren{n^{\frac{3}{10}}} + \left(1-\Theta\Paren{n^{-\frac{7}{10}}}\right)\log \Paren{1-\Theta\Paren{n^{-\frac{7}{10}}}}\\
&=  \Theta\Paren{n^{-\frac{7}{10}}}\log \Theta\Paren{n^{\frac{3}{10}}} -\Theta\Paren{n^{-\frac{7}{10}}} \\
&\geq  \Omega\Paren{n^{-\frac{7}{10}}}.
\end{aligned}
\end{equation}

Therefore, the probability that there is at least one $i\in [n]$ such that $|\mathbf{z}_i|\geq n^{\frac{3}{10}}$ can be upper bounded as

\begin{equation}
\label{eq:eq_chernoff_eq_edges_3}
\begin{aligned}
\mathbb{P}\big[\big\{\exists i\in[n]:\; |\mathbf{z}_i|>n^{\frac{3}{10}}\big\}\big|\mathbf{x}\big]\leq n \exp\left(-\Theta(n)\cdot\Omega\Paren{n^{-\frac{7}{10}}}\right)= n\exp\Paren{-\Omega\Paren{n^{\frac{3}{10}}}}\leq o\Paren{e^{-n^{\frac{1}{4}}}},
\end{aligned}
\end{equation}
where the last inequality is true for $n$ large enough.

Let $\mathcal{E}_B$ be the event that the graph $\mathbf{G}$ contains at most $n^{\frac{3}{10}}$ edges from each set $B_i$. The above shows that
$$\mathbb{P}\left[\mathcal{E}_B\middle| \mathbf{x}\right]\geq 1- o\Paren{e^{-n^{\frac{1}{4}}}}.$$
Now if we have $\mathbb{P}[\mathcal{E}_{n,d,\epsilon,M,\rho,\kappa}^{\text{succ}}|\mathbf{x}]\geq \Omega\Paren{e^{-n^{\frac{1}{4}}}}$, then
\begin{align*}
\mathbb{P}\left[\mathcal{E}_{n,d,\epsilon,M,\rho,\kappa}^{\text{succ}}\middle| \mathbf{x},\mathcal{E}_B\right]&\geq \mathbb{P}\left[\mathcal{E}_{n,d,\epsilon,M,\rho,\kappa}^{\text{succ}}\cap \mathcal{E}_B\middle| \mathbf{x}\right]\geq \mathbb{P}\left[\mathcal{E}_{n,d,\epsilon,M,\rho,\kappa}^{\text{succ}}\middle| \mathbf{x}\right] - \mathbb{P}\left[\mathcal{E}_B^c\middle| \mathbf{x}\right]\\
&\geq \Omega\Paren{e^{-n^{\frac{1}{4}}}}-o\Paren{e^{-n^{\frac{1}{4}}}}= \Omega\Paren{e^{-n^{\frac{1}{4}}}}.
\end{align*}

If we condition on $\mathcal{E}_B$, then with probability 1, the random variables $\mathbf{z}_1,\ldots,\mathbf{z}_{n}$ take values in the set
$$\mathcal{Z}_{(n^{1/3})}=\left\{S\subseteq \left\{1,\ldots,\left\lceil\frac{n+1}{2}\right\rceil\right\}:\; |S|\leq n^{\frac{1}{3}}\right\}.$$
Now define a mapping $G$ from $\mathcal{Z}^{n}_{(n^{1/3})}$ to the set of graphs having $[n]$ as the set of vertices, as follows: If $(z_1,\ldots,z_{n})\in\mathcal{Z}^{n}_{(n^{1/3})}$, then for every $u,v\in[n]$, the edge $uv$ is present in $G(z_1,\ldots,z_{n})$ if and only if there exists $1\leq i\leq n$ and $1\leq j\leq \lceil\frac{n+1}{2}\rceil$ such that $uv\in B_i$, $uv=e_j^{(i)}$, and $j\in z_i$. It is easy to see that we have:
\begin{itemize}
\item[(a')] If $\mathcal{E}_B$ occurs, then $\mathbf{G}=G(\mathbf{z}_1,\ldots,\mathbf{z}_m)$.
\item[(b')] For every $z,z'\in\mathcal{Z}_{(n^{1/3})}^{n}$, if $D_H(z,z')\leq \ell$, then $G(z)$ differs from $G(z')$ by at most $2\ell n^{\frac{1}{3}}$ edges.
\item[(c')] Given $\mathbf{x}$ and $\mathcal{E}_B$, the random variables $\mathbf{z}_1,\ldots,\mathbf{z}_{n}$ are conditionally independent.
\end{itemize}

If we repeat the proof of \cref{eq:eq_Prob_success_conditional_on_Y_first_approach} verbatim but instead of only conditioning on $\mathbf{x}$, we condition on $\mathbf{x}$ and $\mathcal{E}_B$, we get
$$\mathbb{P}\left[\mathcal{E}_{n,d,\epsilon,M,\rho',\kappa}^{\text{succ}}\middle|\mathbf{x},\mathcal{E}_B\right]\geq 1-e^{-n^{\frac{1}{3}}}.$$

Now since $\displaystyle\mathbb{P}\left[\mathcal{E}_B\middle| \mathbf{x}\right]\geq 1-o\Paren{e^{-n^{\frac{1}{4}}}}$, we conclude that

$$\mathbb{P}\left[\mathcal{E}_{n,d,\epsilon,M,\rho',\kappa}^{\text{succ}}\middle|\mathbf{x}\right]\geq \mathbb{P}\left[\mathcal{E}_{n,d,\epsilon,M,\rho',\kappa}^{\text{succ}}\middle|\mathcal{E}_B,\mathbf{x}\right]\cdot\mathbb{P}\left[\mathcal{E}_B\middle| \mathbf{x}\right] \geq 1-e^{-n^{\frac{1}{3}}}-o\Paren{e^{-n^{\frac{1}{4}}}}\geq 1 - \frac{e^{-n^{\frac{1}{4}}}}{2},$$

where the last inequality is true for $n$ large enough.
\end{proof}

\subsubsection{Boosting the probability of success unconditionally}

So far we managed to boost the conditional probability of success given $\mathbf{x}$, but we would like to boost the success probability unconditionally. The blowing-up lemma will help us once again in achieving that. First, we will show that if there is a lower bound $\mathbb{P}\big[\mathcal{E}_{n,d,\epsilon,M,\rho,\kappa}^{\text{succ}}\big]\geq p$, then we can derive a lower bound on the probability that $\mathbf{x}$ satisfies $\mathbb{P}\big[\mathcal{E}_{n,d,\epsilon,M,\rho,\kappa}^{\text{succ}}\big|\mathbf{x}\big]\geq \frac{p}{2}$.

\begin{lemma}
\label{lem:lem_Bound_prob_cond_prob_succ_x}
Let $\epsilon,d,M,\rho$ and $\kappa$ be as in \cref{thm:thm-boosting}. If $p>0$ is such that
$$\mathbb{P}\big[\mathcal{E}_{n,d,\epsilon,M,\rho,\kappa}^{\text{succ}}\big]\geq p,$$
then
\begin{align*}
\mathbb{P}\left[\mathbb{P}\big[\mathcal{E}_{n,d,\epsilon,M,\rho,\kappa}^{\text{succ}}\big|\mathbf{x}\big]> \frac{p}{2}\right]>\frac{p}{2}.
\end{align*}
\end{lemma}
\begin{proof}
Since
\begin{align*}
\mathbb{E}\left[\mathbb{P}\big[\mathcal{E}_{n,d,\epsilon,M,\rho,\kappa}^{\text{succ}}\big|\mathbf{x}\big]\right]&=\mathbb{P}\big[\mathcal{E}_{n,d,\epsilon,M,\rho,\kappa}^{\text{succ}}\big]\geq p,
\end{align*}
we have
\begin{align*}
\mathbb{E}\left[1-\mathbb{P}\big[\mathcal{E}_{n,d,\epsilon,M,\rho,\kappa}^{\text{succ}}\big|\mathbf{x}\big]\right]\leq 1-p.
\end{align*}
Now by Markov inequality, we have
\begin{align*}
\mathbb{P}\left[1-\mathbb{P}\big[\mathcal{E}_{n,d,\epsilon,M,\rho,\kappa}^{\text{succ}}\big|\mathbf{x}\big]\geq 1-\frac{p}{2}\right]\leq \frac{1-p}{1-\frac{p}{2}}=\frac{2-2p}{2-p}.
\end{align*}
Therefore,
\begin{align*}
\mathbb{P}\left[\mathbb{P}\big[\mathcal{E}_{n,d,\epsilon,M,\rho,\kappa}^{\text{succ}}\big|\mathbf{x}\big]> \frac{p}{2}\right]
&\geq 1-\frac{2-2p}{2-p}=\frac{p}{2-p}>\frac{p}{2}.
\end{align*}
\end{proof}

Now we are ready to prove \cref{thm:thm-boosting}

\begin{proof}[Proof of \cref{thm:thm-boosting}]
If $\displaystyle \mathbb{P}\big[\mathcal{E}_{n,d,\epsilon,M,\rho,\kappa}^{\text{succ}}\big]\geq \Omega\Paren{e^{-n^{\frac{1}{4}}}}$ then there exist $C>0$ such that 
$\displaystyle \mathbb{P}\big[\mathcal{E}_{n,d,\epsilon,M,\rho,\kappa}^{\text{succ}}\big]\geq Ce^{-n^{\frac{1}{4}}}$ for $n$ large enough. \cref{lem:lem_Bound_prob_cond_prob_succ_x} now implies that
\begin{align*}
\mathbb{P}\left[\mathbb{P}\big[\mathcal{E}_{n,d,\epsilon,M,\rho,\kappa}^{\text{succ}}\big|\mathbf{x}\big]> \frac{Ce^{-n^{\frac{1}{4}}}}{2}\right] > \frac{Ce^{-n^{\frac{1}{4}}}}{2}=\Omega\Paren{e^{-n^{\frac{1}{4}}}}.
\end{align*}

Now let $\rho''$ be such that $\rho'<\rho''<\rho$. We know from \cref{lem:lem-boosting-conditional} we know that if $\displaystyle\mathbb{P}\big[\mathcal{E}_{n,d,\epsilon,M,\rho,\kappa}^{\text{succ}}\big|\mathbf{x}\big]> \frac{Ce^{-n^{\frac{1}{4}}}}{2}$ and $n$ is large enough, then
$$\mathbb{P}\left[\mathcal{E}_{n,d,\epsilon,M,\rho'',\kappa}^{\text{succ}}\middle|\mathbf{x}\right]\geq 1-\frac{e^{-n^{\frac{1}{4}}}}{2}.$$

We conclude that

\begin{align*}
\mathbb{P}\left[\mathbb{P}\big[\mathcal{E}_{n,d,\epsilon,M,\rho'',\kappa}^{\text{succ}}\big|\mathbf{x}\big]\geq 1-\frac{e^{-n^{\frac{1}{4}}}}{2}\right] \geq \Omega\Paren{e^{-n^{\frac{1}{4}}}}.
\end{align*}

Now define $$\mathcal{X}_{n,d,\epsilon,M,\rho'',\kappa}^{\text{succ}}=\left\{x\in\{-1,+1\}^n:\; \mathbb{P}\big[(\mathbf{G},\mathbf{x})\in\mathcal{E}_{n,d,\epsilon,M,\rho'',\kappa}^{\text{succ}}\big|\mathbf{x}=x\big]\geq 1-\frac{e^{-n^{\frac{1}{4}}}}{2}\right\}.$$ 

We have
\begin{align*}
\mathbb{P}\left[\mathbf{x}\in \mathcal{X}_{n,d,\epsilon,M,\rho'',\kappa}^{\text{succ}}\right]=\mathbb{P}\left[\mathbb{P}\big[\mathcal{E}_{n,d,\epsilon,M,\rho'',\kappa}^{\text{succ}}\big|\mathbf{x}\big]\geq 1-\frac{e^{-n^{\frac{1}{4}}}}{2}\right] \geq \Omega\Paren{e^{-n^{\frac{1}{4}}}}.
\end{align*}

By the blowing-up lemma, if we take $\ell=\lceil n^{\frac{2}{3}}\rceil $, then
\begin{equation}
\label{eq:eq_succ_prob_x_blowup_cond}
\begin{aligned}
\mathbb{P}\left[\mathbf{x}\in \Gamma^{\ell}\left(\mathcal{X}_{n,d,\epsilon,M,\rho'',\kappa}^{\text{succ}}\right)\right] &\geq  1-\exp\left[-2n\left(\frac{n^{\frac{2}{3}}}{n}-\sqrt{\frac{1}{2n}\log \Paren{\frac{1}{\Omega\Paren{e^{-n^{\frac{1}{4}}}}}}}\right)^2\right]\\
&=1-\exp\left[-2n\left(n^{-\frac{1}{3}}-\frac{1}{\sqrt{2}} n^{-\frac{3}{8}}\sqrt{1\pm O\Paren{n^{-\frac{1}{4}}}}\right)^2\right]\geq 1-\exp\left[-n^{\frac{1}{3}}\right],
\end{aligned}
\end{equation}
where the last inequality is true for $n$ large enough.

Let $\tilde{\mathbf{x}}\in \mathcal{X}_{n,d,\epsilon,M,\rho'',\kappa}^{\text{succ}}$ be such that the Hamming distance $D_H(\mathbf{x},\tilde{\mathbf{x}})$ is minimal. We break the ties using the lexicographic order of $\{-1,+1\}^n$. Define $$V_{\mathbf{x}}=\{v\in [n]:\; \mathbf{x}_v\neq \tilde{\mathbf{x}}_v\}.$$ Clearly, $|V_{\mathbf{x}}|=D_H(\mathbf{x},\tilde{\mathbf{x}})$. We will generate a random graph $\tilde{\mathbf{G}}$ with $V(\tilde{\mathbf{G}})=[n]$ as follows:
\begin{itemize}
\item For every $u,v\in [n]\setminus V_{\mathbf{x}}$, we make the edge $uv$ present in $\tilde{\mathbf{G}}$ if and only if it is present in $\mathbf{G}$.
\item For an edge $uv$ such that $u\in V_{\mathbf{x}}$ or $v\in V_{\mathbf{x}}$, we randomly decide its presence in $\tilde{\mathbf{G}}$ independently of $(\mathbf{G},\mathbf{x})$ and in such a way that the conditional probability of $\tilde{\mathbf{G}}$ given $\tilde{\mathbf{x}}$ is consistent with the distribution of the stochastic block model. More precisely, we generate $$\big\{uv\in\tilde{\mathbf{G}}:\; u\in V_{\mathbf{x}}\text{ or }v\in V_{\mathbf{x}}\big\}$$
independently of $(\mathbf{G},\mathbf{x})$, and with the following conditional distribution: For every two sets of edges $$\mathsf{E}\subseteq\big\{uv:\; u,v\in[n],\; u\in V_{\mathbf{x}}\text{ or }v\in V_{\mathbf{x}}\big\}\quad\text{and}\quad \mathsf{E}'\subseteq \big\{uv:\; u,v\in [n]\setminus V_{\mathbf{x}}\big\},$$
we have
\begin{align*}
\mathbb{P}\Big[&\big\{uv\in\tilde{\mathbf{G}}:\; u\in V_{\mathbf{x}}\text{ or }v\in V_{\mathbf{x}}\big\}=\mathsf{E}\;\Big|\;\tilde{\mathbf{x}},\big\{uv\in\tilde{\mathbf{G}}:\; u,v\in [n]\setminus V_{\mathbf{x}}\big\}=\mathsf{E}'\Big]\\
&=\mathbb{P}\Big[\big\{uv\in\mathbf{G}:\; u\in V_{\mathbf{x}}\text{ or }v\in V_{\mathbf{x}}\big\}=\mathsf{E}\;\Big|\;\mathbf{x}=\tilde{\mathbf{x}},\big\{uv\in \mathbf{G}:\; u,v\in [n]\setminus V_{\mathbf{x}}\big\}=\mathsf{E}'\Big].
\end{align*}
\end{itemize}

It is easy to see that the conditional distribution of $\tilde{\mathbf{G}}$ given $\tilde{\mathbf{x}}$ is the same as the conditional distribution of $\mathbf{G}$ given $\mathbf{x}=\tilde{\mathbf{x}}$. Now since $\tilde{\mathbf{x}}\in\mathcal{X}_{n,d,\epsilon,M,\rho'',\kappa}^{\text{succ}}\,$ and since the conditional distribution of $\tilde{\mathbf{G}}$ given $\tilde{\mathbf{x}}$ is that of $\sbm$, it follows from the definition of $\mathcal{X}_{n,d,\epsilon,M,\rho'',\kappa}^{\text{succ}}$ that with probability at least $\displaystyle 1-\frac{e^{-n^{\frac{1}{4}}}}{2}$, we have $(\tilde{\mathbf{G}},\tilde{\mathbf{x}})\in\mathcal{E}_{n,d,\epsilon,M,\rho'',\kappa}^{\text{succ}}\,$, which means that for every graph $G^{\circ}$ that differs from $\tilde{\mathbf{G}}$ by at most $\rho'' n$ edges, we have
$$\iprod{M(G^{\circ}), \tilde{\mathbf{x}}\tilde{\mathbf{x}}^T}\geq {\kappa}\cdot\Snorm{\tilde{\mathbf{x}}}\cdot\Normn{M(G^{\circ})}.$$

Let $\mathcal{E}$ be the event that $\mathbf{x}\in \Gamma^{\ell}\left(\mathcal{X}_{n,d,\epsilon,M,\rho'',\kappa}^{\text{succ}}\right)$ and that for every graph $G^{\circ}$ that differs from $\tilde{\mathbf{G}}$ by at most $\rho'' n$ edges, we have
$$\iprod{M(G^{\circ}), \tilde{\mathbf{x}}\tilde{\mathbf{x}}^T}\geq {\kappa}\cdot\Snorm{\tilde{\mathbf{x}}}\cdot\Normn{M(G^{\circ})}.$$

It follows from \cref{eq:eq_succ_prob_x_blowup_cond} and from the above discussion that
\begin{equation}
\label{eq:eq_succ_prob_Y_prime}
\mathbb{P}\left[\mathcal{E}\right]\geq 1- e^{-n^{\frac{1}{3}}}-\frac{e^{-n^{\frac{1}{4}}}}{2}.
\end{equation}

Let $\mathcal{V}$ be the event that the degree in $\mathbf{G}$ of every vertex $v\in[n]$ is at most $n^{\frac{3}{10}}$. Similarly, let $\tilde{\mathcal{V}}$ be the event that the degree in $\tilde{\mathbf{G}}$ of every vertex $v\in[n]$ is at most $n^{\frac{3}{10}}$. Since the degree of every vertex is the sum of $n$ independent Bernoulli random variables each having a success probability of at most $\frac{2d}{n}$, then by following calculations that are very similar to \cref{eq:eq_chernoff_eq_edges_1}, \cref{eq:eq_chernoff_eq_edges_2} and \cref{eq:eq_chernoff_eq_edges_3}, we can deduce that
\begin{equation}
\label{eq:eq_succ_For_Every_V}
\mathbb{P}[\mathcal{V}]\geq 1 -o\Paren{e^{-n^{\frac{1}{4}}}},
\end{equation}
and
\begin{equation}
\label{eq:eq_succ_For_Every_V_prime}
\mathbb{P}[\tilde{\mathcal{V}}]\geq 1-o\Paren{e^{-n^{\frac{1}{4}}}}.
\end{equation}

Now assume that $\mathcal{E} \cap \mathcal{V}\cap \tilde{\mathcal{V}}$ occurs. Observe the following:
\begin{itemize}
\item Since $\mathcal{E}$ occurs, we have $\mathbf{x}\in \Gamma^{\ell}\left(\mathcal{X}_{n,d,\epsilon,M,\rho'',\kappa}^{\text{succ}}\right)$, which means that $|V_{\mathbf{x}}|=D_H(\mathbf{x},\tilde{\mathbf{x}})\leq \ell$.
\item From the definition of $\tilde{\mathbf{G}}$, we can see that the graphs $\tilde{\mathbf{G}}$ and $\mathbf{G}$ can differ only in edges that are incident to vertices in $V_{\mathbf{x}}$.
\item Since $\mathcal{V}$ occurs, $\mathbf{G}$ contains at most $n^{\frac{3}{10}}\cdot |V_{\mathbf{x}}|\leq  n^{\frac{3}{10}}\cdot \lceil n^{\frac{2}3}\rceil=o(n)$ edges that are incident to vertices in $V_{\mathbf{x}}$.
\item Since $\tilde{\mathcal{V}}$ occurs, $\tilde{\mathbf{G}}$ contains at most $n^{\frac{3}{10}}\cdot |V_{\mathbf{x}}|\leq  n^{\frac{3}{10}}\cdot \lceil n^{\frac{2}3}\rceil=o(n)$ edges that are incident to vertices in $V_{\mathbf{x}}$.
\end{itemize}
Therefore, if $\mathcal{E} \cap \mathcal{V}\cap \tilde{\mathcal{V}}$ occurs, then $\tilde{\mathbf{G}}$ differs from $\mathbf{G}$ by at most $o(n)+o(n)=o(n)$ edges.

Now let $G^{\circ}$ be an arbitrary graph with $V(G^{\circ})=[n]$ such that $G^{\circ}$ differs from $\mathbf{G}$ by at most $\rho'n$ edges. Since $\rho'<\rho''$, then if $n$ is large enough, the graph $G^{\circ}$ differs from $\tilde{\mathbf{G}}$ by at most $\rho'n+o(n)\leq \rho'' n$ edges. Now since $\mathcal{E}$ occurs and since $G^{\circ}$ differs from $\tilde{\mathbf{G}}$ by at most $\rho'' n$ edges, we have
$$\iprod{M(G^{\circ}), \tilde{\mathbf{x}}\tilde{\mathbf{x}}^T}\geq {\kappa}\cdot\Snorm{\tilde{\mathbf{x}}}\cdot\Normn{M(G^{\circ})}.$$

Now notice that
\begin{align*}
\iprod{M(G^{\circ}), \mathbf{x}\mathbf{x}^T}&= \iprod{M(G^{\circ}), \tilde{\mathbf{x}}\tilde{\mathbf{x}}^T} + \iprod{M(G^{\circ}), \mathbf{x}\mathbf{x}^T-\tilde{\mathbf{x}}\tilde{\mathbf{x}}^T}\\
&\geq  {\kappa}\cdot\Snorm{\tilde{\mathbf{x}}}\cdot\Normn{M(G^{\circ})} - |\iprod{M(G^{\circ}), \mathbf{x}\mathbf{x}^T-\tilde{\mathbf{x}}\tilde{\mathbf{x}}^T}|\\
&\geq  {\kappa}\cdot\Snorm{\mathbf{x}}\cdot\Normn{M(G^{\circ})}-\Norm{\tilde{\mathbf{x}}\tilde{\mathbf{x}}^T-\mathbf{x}\mathbf{x}^T}\cdot\Normn{M(G^{\circ})}.
\end{align*}

On the other hand, since $\mathbf{x}$ differs from $\tilde{\mathbf{x}}$ only on $V_{\mathbf{x}}$, and since $|V_{\mathbf{x}}|\leq \ell = \lceil n^{\frac23}\rceil$, we have
\begin{align*}
\Norm{\tilde{\mathbf{x}}\tilde{\mathbf{x}}^T-\mathbf{x}\mathbf{x}^T}
&\leq  \Norm{\tilde{\mathbf{x}}\tilde{\mathbf{x}}^T-\tilde{\mathbf{x}}\mathbf{x}^T} + \Norm{\tilde{\mathbf{x}}\mathbf{x}^T-\mathbf{x}\mathbf{x}^T} =  \Norm{\tilde{\mathbf{x}}(\tilde{\mathbf{x}}^T-\mathbf{x}^T)} + \Norm{(\tilde{\mathbf{x}}-\mathbf{x})\mathbf{x}^T}\\
&\leq  \Norm{\tilde{\mathbf{x}}}\cdot\Norm{\tilde{\mathbf{x}}^T-\mathbf{x}^T} + \Norm{\tilde{\mathbf{x}}-\mathbf{x}}\cdot\Norm{\mathbf{x}^T} = \sqrt{n}\cdot 2\sqrt{|V_{\mathbf{x}}|} + 2\sqrt{|V_{\mathbf{x}}|}\cdot \sqrt{n}\\
&\leq 4\sqrt{n}\cdot \sqrt{n^{\frac{3}{10}}}=o(n)=o\big((\sqrt{n})^2\big)=o(\Snorm{\mathbf{x}}).
\end{align*}

Now since $\kappa'<\kappa$, we get that for $n$ is large enough, we have

\begin{align*}
\iprod{M(G^{\circ}), \mathbf{x}\mathbf{x}^T} &\geq  (\kappa-o(1))\cdot\Snorm{\mathbf{x}}\cdot\Normn{M(G^{\circ})}\\
&\geq \kappa'\cdot\Snorm{\mathbf{x}}\cdot\Normn{M(G^{\circ})}\,.
\end{align*}
This implies that if $\mathcal{E} \cap \mathcal{V}\cap \tilde{\mathcal{V}}$ occurs and $n$ is large enough, then $\mathcal{E}_{n,d,\epsilon,M,\rho',\kappa'}^{\text{succ}}$ occurs as well. By combining this with \cref{eq:eq_succ_prob_Y_prime}, \cref{eq:eq_succ_For_Every_V}, and \cref{eq:eq_succ_For_Every_V_prime}, we conclude that for $n$ large enough, we have
\begin{align*}
\mathbb{P}\big[\mathcal{E}_{n,d,\epsilon,M,\rho',\kappa'}^{\text{succ}}\big]&\geq 1- e^{-n^{\frac{1}{3}}}-\frac{e^{-n^{\frac{1}{4}}}}{2}-o\Paren{e^{-n^{\frac{1}{4}}}}-o\Paren{e^{-n^{\frac{1}{4}}}}\\
&\geq 1-e^{-n^{\frac{1}{4}}}.
\end{align*}
\end{proof}

\section{Trace bounds for stochastic block models}\label{sec:bounds-moments-Q}
\providecommand{\Ho}{H_{(1)}}
\providecommand{\Ht}{H_{(2)}}
\providecommand{\Vo}{V_{(1)}}
\providecommand{\Vt}{V_{(2)}}
\providecommand{\Vot}{V_{(1,2)}}
\providecommand{\Mo}{M_{(1)}}
\providecommand{\Mt}{M_{(2)}}
\providecommand{\nmultig}[2]{\text{NMULTIG}_{#1,#2}}

In this section  we prove \cref{lem:every-little-thing-is-gonna-be-alright}.
The lemma will be a direct consequence of the  following theorems: a separation in Schatten norm between $\Q(\overline{\mathbf{Y}})$ and $\Q(\overline{\mathbf{Y}})-\dyad{\mathbf{x}}$ and some concentration results on the diagonal entries of $\Q(\overline{\mathbf{Y}})$.

\begin{theorem}
	\label{thm:sbm-schatten-norm}
	Let  $(\mathbf{x}, \mathbf{G})\sim \sbm$. Let $\Delta, A, \tau$ be as defined in \cref{eq:Delta-form}. 
	Suppose $d\geq (1+\delta)\frac{4}{\eps^2}$ for some $\delta > 0$.
	Let $s$ be an integer satisfying
	\begin{align*}
		s\geq \valueofs\,,
	\end{align*}
	and let $t \in \Brac{\frac{1}{400}\log n\,, \frac{1}{100}\log n}$. Then  for $n$ large enough
	\begin{align}
		\label{eq:sbm-Schatten-norm-lower-bound}
		\E \Brac{\Tr\Paren{\Q(\overline{\mathbf{Y}})}^t}\geq  \frac{n^t}{10}\cdot n^{-\frac{2}{50A}}\cdot \Paren{1-(1+\delta)^{-\sqrt{s}\cdot t}-n^{1/10}}\,.
	\end{align}
	Furthermore
	\begin{align}
		\label{eq:sbm-Schatten-norm-upper-bound}
		\E \Brac{\Tr\Paren{\Q(\overline{\mathbf{Y}})-\dyad{\mathbf{x}}}^t}\leq (1+\delta)^{-t/5}\cdot \E \Brac{\Tr\Paren{\Q(\overline{\mathbf{Y}})}^t}\,.
	\end{align}
\end{theorem}

Let's spend a moment discussing \cref{thm:sbm-schatten-norm}.
Consider \cref{eq:sbm-Schatten-norm-lower-bound}, the term containing $(1+\delta)^{\sqrt{s}\cdot t}$ corresponds to random noise that is uncorrelated with the underlying partition $x$, and it can be made arbitrarily small by increasing the length $s$ of the self-avoiding walks. The term $n^{-\frac{1}{100A}}$ is an approximation factor that appears for technical reasons\footnote{This approximation factor appears when trying to bound the dependencies that are caused by the truncation.}. Hhowever  by choice of $A$ it will be negligible. While we limit the constant chosen for $t$ in a  certain range, this appears to be quite flexible as long as $t< \log n$.

More interestingly,  \cref{eq:sbm-Schatten-norm-upper-bound} shows  the push-out effect of the matrix $\Q(\overline{\mathbf{Y}})$: even though $\dyad{\mathbf{x}}$ is not the matrix maximizing the $t$-Schatten norm of $\Q(\overline{\mathbf{Y}})$, it is $\delta^{O(1)}$-correlated with it.

\begin{theorem}\label{thm:sbm-bound-second-moment-large-t}
	Consider the settings of \cref{thm:sbm-schatten-norm}.
	Then for any $u,v\in [n]\,, u\neq v$
	\begin{align*}
		\E \Brac{\Paren{\Q(\overline{\mathbf{Y}})^t}_{uu}\Paren{\Q(\overline{\mathbf{Y}})^t}_{vv}}\leq (1+o(1))	\E \Brac{\Paren{\Q(\overline{\mathbf{Y}})^t}_{uu}}^2\,.
	\end{align*}
	Moreover,  for any $u \in [n]$
	\begin{align*}
		\E \Brac{\Paren{\Paren{\Q(\overline{\mathbf{Y}})^t}_{uu}}^2}\leq C\cdot \E \Brac{\Paren{\Q(\overline{\mathbf{Y}})^t}_{uu}}^2\,,
	\end{align*}
	where $C>1$ is a universal constant.
\end{theorem}

The inequalities in \cref{thm:sbm-bound-second-moment-large-t} shows that both the trace itself and the diagonal entries of $\Paren{\Q(\overline{\mathbf{Y}})}^t$ have small variance and thus concentrate around their expectations.
Together with  \cref{thm:sbm-schatten-norm}, we can use  \cref{thm:sbm-bound-second-moment-large-t} to prove \cref{lem:every-little-thing-is-gonna-be-alright}. 
\begin{proof}[Proof of \cref{lem:every-little-thing-is-gonna-be-alright}]
	\cref{eq:event-bound-expectation-trace} follows directly by \cref{thm:sbm-schatten-norm}.
	We can bound the variance of the trace applying \cref{thm:sbm-bound-second-moment-large-t} 
	\begin{align*}
			\E \Brac{\Tr\Paren{\Q(\overline{\mathbf{Y}})}^{2t}} =& \underset{i,j \in [n]}{\sum} \E \Brac{\Paren{\Q(\overline{\mathbf{Y}})^t}_{ii}\Paren{\Q(\overline{\mathbf{Y}})^t}_{jj}}\\
			=& (1+o(1)) \underset{i,j \in [n]}{\sum} \E \Brac{\Paren{\Q(\overline{\mathbf{Y}})^t}_{ii}\Paren{\Q(\overline{\mathbf{Y}})^t}_{jj}}\\
			=& (1+o(1)) \underset{i,j \in [n]\,, i\neq j}{\sum} \E \Brac{\Paren{\Q(\overline{\mathbf{Y}})^t}_{ii}\Paren{\Q(\overline{\mathbf{Y}})^t}_{jj}}\\
			=& (1+o(1)) \E \Brac{\Tr\Paren{\Q(\overline{\mathbf{Y}})}^{t}}^2 \,,
	\end{align*}
	where in the third step we used the fact that there is only a linear number of terms with $i=j$.
	Thus \cref{eq:event-bound-trace} holds with high probability through an application of Chebyshev's inequality.
	Then by Markov's inequality
	\begin{align*}
		\bbP  \Set{\Tr\Paren{\Q(\overline{\mathbf{Y}})-\dyad{\mathbf{x}}}^t\geq \Paren{1+\delta}^{-t/10}\Tr\Paren{\Q(\overline{\mathbf{Y}})}^{t}}\leq o(1)\,.
	\end{align*}
	It remains to prove \cref{eq:event-bound-entries}.
	By \cref{thm:sbm-bound-second-moment-large-t}, there is a universal constant $C>0$ such that
	\begin{align*}
		\E 	\Brac{\underset{i \in [n]}{\sum}\Paren{\Paren{\Q(\overline{\mathbf{Y}})}^{2t-2}}^2_{ii}} &\leq  C\cdot 	\E 	\Brac{\underset{i \in [n]}{\sum}\Paren{\Paren{\Q(\overline{\mathbf{Y}})}^{2t-2}}_{ii}}^2\\
		& = C\cdot 	\E 	\Brac{ \Tr \Paren{\Q(\overline{\mathbf{Y}})}^{2t-2}}^2\\
		& = (1+o(1))\cdot C \cdot \Brac{ \Tr \Paren{\Q(\overline{\mathbf{Y}})}^{2t-2}}^2\,,
	\end{align*}
	where in the last step we used concentration of the trace.
	By Markov's inequality, setting $\gamma = O(C)$ we obtain the desired result.
\end{proof}

\paragraph{Organization of the section}
The rest of the section will contain the proofs of  \cref{thm:sbm-schatten-norm} and  \cref{thm:sbm-bound-second-moment-large-t}. To prove the theorems
we will reduce the study of $\Tr\Paren{Q^{(s)}(\overline{\mathbf{Y}})}^t$ and related quantities to the combinatorial problem of counting multi-graphs. 

We introduce preliminary facts and a bird-eye view of the in \cref{sec:trace-preliminaries}. 
We prove \cref{thm:sbm-schatten-norm}  in \cref{sec:lowerbound-Schatten-norm} and  \cref{sec:upperbound-schatten-norm-centered}. Finally, we obtain \cref{thm:sbm-bound-second-moment-large-t} in \cref{sec:concentration-bsaws}. For simplicity of the exposition, the following sections will  be essentially oblivious to the technical  challenges arising when studying moments of truncated graphs.
We defer a high level discussion of the truncation effect to sections \cref{sec:appendix-lower-bound-non-centered} and \cref{sec:appendix-upper-bound-centered}. We present the technical arguments in \cref{sec:appendix-lower-bound-non-centered},  \cref{sec:appendix-upper-bound-centered}, \cref{sec:technical-lemmas-trace-bounds} and \cref{sec:tools-bsaws}.
In the forthcoming sections, we always assume the settings of \cref{thm:sbm-schatten-norm} to hold.

\subsection{Preliminary discussion}\label{sec:trace-preliminaries}

Let $(\mathbf{x}, \mathbf{G})\sim \sbm$ and let $\mathbf{Y}$ be the adjacency matrix of $\mathbf{G}$. Recall the matrix polynomial $\Q(\mathbf{Y})$ defined in \cref{eq:definition-polynomial-q}: 
\begin{align*}
	Q_{ij}^{(s)}(\mathbf{Y}) = 
	\begin{cases}
		\frac{1}{\card{\saw{ij}{s}}}\Paren{\frac{2n}{\eps \cdot d}}^{s}\underset{H\in \saw{ij}{s}}{\sum} Y_H & \text{if $i\neq j$\,,}\\
		0 &\text{otherwise.}
	\end{cases}\,,
\end{align*}
for all $i,j \in [n]$. By \cref{fact:sbm-expectation-polynomial-2}, $\Q(\mathbf{Y})$ is an unbiased estimator of $\dyad{\mathbf{x}}$ (up to the diagonal entries) but, as already discussed, we will need to work with a truncated version of the graph.
For a graph $G$, we will denote by $\overline{G}$ the graph obtained from $G$ by deleting all the vertices which have degree more than $\Delta$ in $G$. Similarly, we will denote by $\overline{Y}$ its truncated adjacency matrix as defined in \cref{def:truncated_adjacency_matrix}.

\subsubsection{From trace computations to graph counting}
Next, we show how to reduce the trace computation to a multigraph counting problem.  Each term in the sum $\Tr\Paren{Q^{(s)}(\overline{Y})}^t$ is a concatenation of self-avoiding walks and hence correspond  to a multigraph over some subset of vertices $[n]$. We make this idea precise below.

\begin{definition}[Block SAW]\label{def:block-saw}
	Let $\bsaw{s}{t}$ be the set of  multi-graphs  obtained as follows.
	Pick $i_1,\ldots,i_{t}\in [n]$ distinct vertices in $K_n$ and set $i_{t+1}:=i_1$. For each $q\in [t]$, pick $W_q\in \saw{i_qi_{q+1}}{s}(K_n)$. Then 
	\[
	\underset{q\in [t]}{\bigoplus} W_q\in \bsaw{s}{t}	\,.
	\]
	Let $H= \underset{q\in [t]}{\bigoplus}W_q$.
	We call $I(H)= \Set{i_1,\ldots,i_t}$ the set of \textit{pivot} vertices of $H$ and $\cW(H)=\Set{W_1,\ldots,W_t}$ the generating self-avoiding walks of $H$. We denote by $M(\cW(H))$ the sequence of edges obtained concatenating the sequences $M(W_1),\ldots, M(W_t)$.
	We call $\bsaw{s}{t}$ the set of $(s,t)$-block self-avoiding walks. At times, we will also use $M(\cW(H))$ to denote the set of edges in the sequence $M(\cW(H))$. 
\end{definition}

We can now expand $\Tr\Paren{Q^{(s)}(Y)}$.

\begin{lemma}\label{lem:expectation-trace-expansion}
	Consider the settings of  \cref{thm:sbm-schatten-norm}. Let $G$ be a graph over $[n]$ with centered adjacency matrix $Y$. Then
	\begin{align}
		\label{eq:sbm-trace-expansion}
		\Tr\Paren{Q^{(s)}( Y)^t} =  \Paren{1\pm o(1)}\cdot n^t \Paren{\frac{2}{\eps \cdot d}}^{st}\cdot \underset{H \in \bsaw{s}{t}}{\sum} Y_H\,.
	\end{align}
	\begin{proof}
		By definition of trace, 
		\begin{align*}
			\Tr \Paren{Q^{(s)}( Y)^t} =& \underset{i_1,\ldots,i_t \in [n]}{\sum} Q_{i_1i_2}^s( Y)\cdots Q_{i_{t-1}i_t}^s( Y)\cdot Q_{i_ti_1}^s( Y)\\
			= & \frac{1}{\Paren{(n-2)^{\underline{s-1}}}^t}\cdot \Paren{\frac{2n}{\eps\cdot d}}^{st}\cdot \underset{i_1,\ldots,i_t \in [n]}{\sum} \Brac{ \underset{\ell \in [t]}{\prod} \Paren{\underset{W \in \saw{i_\ell i_{\ell+1}}{s}}{\sum}Y_W}}\\
			= & \frac{1}{\Paren{(n-2)^{\underline{s-1}}}^t}\cdot \Paren{\frac{2n}{\eps\cdot d}}^{st}\cdot \underset{H \in \bsaw{s}{t}}{\sum} Y_H\\
			= &   \Paren{1\pm o(1)}\cdot n^t \Paren{\frac{2}{\eps \cdot d}}^{st}\cdot \underset{H \in \bsaw{s}{t}}{\sum} Y_H\,.
		\end{align*}	
	\end{proof}
\end{lemma}

\cref{def:block-saw} captures all the elements in $\Tr\Paren{Q^{(s)}(Y)}$. Furthermore, a similar expansion can be carried out for the centered trace $\Tr\Paren{Q^{(s)}(\overline{Y})-\dyad{x}}^t$.
For simplicity of the notation, in the next expressions, for a given set of $t$ vertices $\Set{i_1,\ldots,i_t}$ we denote $i_1$ also by $i_{t+1}$.

\begin{fact}\label{fact:sbm-trace-expansion-centered}
	Consider the settings of  \cref{thm:sbm-schatten-norm}. Let $M\in \R^{n\times n}$. Then for any graph $G$ with centered adjacency matrix $Y$
	\begin{align}
		\label{eq:sbm-trace-expansion-centered}
		\Tr \Paren{Q^{(s)}(Y)-M}^t =& \Paren{1\pm o(1)}n^t\cdot \Paren{\frac{2}{\eps \cdot d}}^{st}\nonumber\\
		\cdot & \underset{i_1,\ldots,i_t\in [n]}{\sum}\underset{\ell \in [t]}{\prod} \Brac{\underset{W\in \saw{i_{\ell}i_{\ell+1}}{s}}{\sum} \Paren{Y_W-\Paren{\frac{\eps\cdot d}{2n}}^s\cdot \frac{1}{\Card{\saw{i_\ell i_{\ell+1}}{s}}}\cdot M_{i_\ell i_{\ell+1}}}}\,.
	\end{align}
	\begin{proof}
		Simply expanding the trace
		\begin{align*}
			\Tr \Paren{Q^{(s)}(Y)-M}^t =& \underset{i_1,\ldots,i_t\in [n]}{\sum}\underset{\ell \in [t]}{\prod} \Paren{Q^{(s)}_{i_\ell i_{\ell+1}}(Y)-M_{i_\ell i_{\ell+1}}}\\
			=& \underset{i_1,\ldots,i_t\in [n]}{\sum}\underset{\ell \in [t]}{\prod} \Paren{\frac{1}{\card{\saw{i_{\ell}i_{\ell+1}}{s}}}\underset{W\in \saw{i_{\ell}i_{\ell+1}}{s}}{\sum}\Paren{\frac{2n}{\eps\cdot d}}^s Y_W-M_{i_\ell i_{\ell+1}}}\\
			=& \Paren{\frac{2}{\eps \cdot d}}^{st} \underset{i_1,\ldots,i_t\in [n]}{\sum}\underset{\ell \in [t]}{\prod} \Paren{\frac{n^s}{\card{\saw{i_{\ell}i_{\ell+1}}{s}}}\underset{W\in \saw{i_{\ell}i_{\ell+1}}{s}}{\sum} Y_W-\Paren{\frac{\eps\cdot d}{2}}^s\cdot M_{i_\ell i_{\ell+1}}}\\
			=&\Paren{1\pm o(1)}n^t\cdot \Paren{\frac{2}{\eps \cdot d}}^{st}\nonumber \\
			\cdot& \underset{i_1,\ldots,i_t\in [n]}{\sum}\underset{\ell \in [t]}{\prod} \Brac{\underset{W\in \saw{i_{\ell}i_{\ell+1}}{s}}{\sum} \Paren{Y_W-\Paren{\frac{\eps\cdot d}{2n}}^s\cdot \frac{1}{\card{\saw{i_{\ell}i_{\ell+1}}{s}}}\cdot M_{i_\ell i_{\ell+1}}}}\,.
		\end{align*}
	\end{proof}
\end{fact}

To easily refer to the elements in the sum of \cref{fact:sbm-trace-expansion-centered}, we additionally use the following notation. For a given walk $W\in \saw{ij}{s}$ and $(\mathbf{x}, \mathbf{G})\sim \sbm$ we define the polynomial 
\begin{align}\label{eq:centered-polynomial}
	\hat{\mathbf{Y}}_{W}:= \overline{\mathbf{Y}}_W-\Paren{\frac{\eps\cdot d}{2n}}^s\cdot \mathbf{x}_i \mathbf{x}_j\,.
\end{align}
At times, for a set of self-avoiding walks $\cW$ and a multi-graph $H=\underset{W\in \cW}{\bigoplus} W$, we will use the notation
\begin{align}\label{eq:product-centered-polynomials}
	\hat{Y}_H= \underset{W\in \cW}{\prod}\hat{Y}_{W} \,.
\end{align}

\subsubsection{Proofs strategy}\label{sec:proof-strategy}
Here we  outline our proof strategy.

\paragraph{Proof strategy for \cref{thm:sbm-schatten-norm}}
We show the theorem in two steps, first we prove a lower bound $\Norm{\Q(\overline{\mathbf{Y}})}_t\geq C$ (for some meaningful quantity $C$), second an upper bound of the form $\Norm{\Q(\overline{\mathbf{Y}})-\dyad{\mathbf{x}}}_t\leq (1+\delta)^{-\Omega(1)}\cdot C$. The two together immediately imply the theorem.
Our strategy to lower bound $\Norm{\Q(\overline{\mathbf{Y}})}_t$ will be the following:
\begin{enumerate}
	\item For any $x\in\Set{\pm 1}^n$. Prove an upper bound $\overline{U}_H(x)$ for $\big|\mathbb{E}\big[\overline{\mathbf{Y}}_H\big|x\big]\big|$ for every $H\in\bsaw{s}{t}$.
	\item Find a large enough class $\nbsaw{s}{t}\subset \bsaw{s}{t}$ of nice and well-behaved block-self-avoiding-walks whose structure allows us to prove a lower bound $\overline{L}_H$ for $\mathbb{E}\big[\overline{\mathbf{Y}}_H\big]$ for every $H\in \nbsaw{s}{t}$.
	\item Show that the contribution of $\nbsaw{s}{t}^c=\bsaw{s}{t}\setminus\nbsaw{s}{t}$ is negligible with respect to that of $\bsaw{s}{t}$. More precisely, we will show that
	\begin{align}
		\label{eq:step-3}
		\sum_{H\in\nbsaw{s}{t}^c}\mathbb{E}[\overline{U}_H(\mathbf{x})]=o\left(\sum_{H\in\nbsaw{s}{t}}\overline{L}_H\right)\,.
	\end{align}
	This will imply that $$ (1\pm o(1))\cdot n^t\left(\frac{2}{\epsilon\cdot d}\right)^{st}\cdot\sum_{H\in\nbsaw{s}{t}}\overline{L}_H$$
	is a good lower bound for $\mathbb{E}\left[\Tr\left(Q^{(s)}(\overline{\mathbf{Y}})\right)^t\right]$.
\end{enumerate}

Next, to upper bound  $\Norm{\Q(\overline{\mathbf{Y}})-\dyad{\mathbf{x}}}_t$ we will need the following two additional ingredients:
\begin{enumerate}
  \setcounter{enumi}{3}
	\item Show that the class of multigraphs $\nbsaw{s}{t}$ is strongly correlated to $\mathbf{x}$ in the sense that for many $H\in \nbsaw{s}{t}$ 
	\begin{align*}
		\E \Brac{\hat{\mathbf{Y}}_{H}} \leq \Paren{1+\delta}^{-t}\cdot  \overline{L}_H\,.
	\end{align*}
	To get an intuition of why this should be true, notice that for any self-avoiding walk $W\in \saw{ij}{s}$ and for any $(\mathbf{x},\mathbf{G})\sim \sbm$, we have $	\E \Brac{\mathbf{Y}_W-\Paren{\frac{\eps\cdot d}{2n}}^s\cdot \mathbf{x}_i\mathbf{x}_j\given \mathbf{x}} =0$.
	\item The class of multigraphs $\nbsaw{s}{t}^c$ is poorly correlated with $\mathbf{x}$ in the sense that for any $H\in \nbsaw{s}{t}^c$ 
	\begin{align*}
		\E \Brac{\hat{\mathbf{Y}}_H\given \mathbf{x}} \approx \overline{U}_H(\mathbf{x})\,.
	\end{align*}
	Together with step $3$ these will imply a good upper bound on $\mathbb{E}\left[\Tr\Paren{Q^{(s)}(\overline{\mathbf{Y}}) -\dyad{\mathbf{x}}}^t\right]$.
\end{enumerate}

\paragraph{Proof strategy for \cref{thm:sbm-bound-second-moment-large-t}} 
For $u,v\in [n]$, let $\bsaw{s}{t,u}\subseteq\bsaw{s}{t}$ the set of block self-avoiding walks having $u$ as pivot. To provide concentration we will use a similar approach to the one outlined above.
\begin{enumerate}
	\item Find a nice set of multigraphs $\nmultig{s,t}{u,v}\subseteq\bsaw{s}{t,u}\times \bsaw{s}{t,v}$ such that
	\begin{align*}
		\underset{H\in \nmultig{s}{t,u,v}^c}{\sum} \E \Brac{\overline{U}_H(\mathbf{x})}\leq o(1)\underset{H\in \nmultig{s}{t,u,v}}{\sum} \E \Brac{\overline{\mathbf{Y}}_H}\,.
	\end{align*}
	\item Show that for such nice multigraphs
	\begin{align*}
			\underset{H\in \nmultig{s}{t,u,v}}{\sum} \E \Brac{\overline{\mathbf{Y}}_H}\leq& (1+o(1))\underset{H\in \bsaw{s}{t,u}}{\sum}\E\Brac{Y_H}^2\,,\\
			\underset{H\in \nmultig{s}{t,u,u}}{\sum} \E \Brac{\overline{\mathbf{Y}}_H}\leq& (1+o(1))\underset{H\in \bsaw{s}{t,u}}{\sum}\E\Brac{Y_H}^2
	\end{align*}
	for some universal constant $C>1$.
\end{enumerate}
Combining the two we will obtain the theorem.

\subsubsection{Additional notation} \label{sec:additional-notation}
We introduce some additional definitions which will be helpful in our discussion of block self-avoiding walks. We will introduce additional notation when needed. We suggest the  impatient reader to skip the section and come back here when needed.

For simplicity at times we write $\mathbf{X}$ for $\dyad{\mathbf{x}}$.
For a multigraph $H$ with vertex set $V(H)\subseteq [n]$ and a vertex $v\in V(H)$ we write $d^H_1(v)$ to denote the number of edges in $H$ of multiplicity $1$ incident to $v$. Similarly, we write $d^H_{\geq 2}(v)$ to denote the number of distinct edges in $H$ of multiplicity at least $2$ incident to $v$. Then $d^H(v)=d_1^H(v)+d_{\geq 2}^H(v)$, notice that $d^H(v)$ is the number of \textit{distinct} edges incident to $v$.

\begin{definition}[Underlying graph]\label{def:underlying-graph}
	Let $H=(V,M)$ be a multigraph. Let $G=(V,E)$ be the graph with vertex set $V(G)=V(H)$ and edge set $E(G)$ such that $\Set{i,j}\in E(G)$ if there exists an edge $e\in M(H)$ with endpoints $i,j$. We call $G=(V,E)$ the underlying graph of $H$ and denote it with $G(H)$.
\end{definition}

\begin{definition}
	For integers $v\leq m$ we denote by $\cT(m,v)$ the set of non-isomorphic trees 
	(picking one arbitrary representative per class) on $m$ vertices with $v$ leaves.
\end{definition}

\begin{definition}[Extension Set] \label{def:extension-set}
	Let $G$ be a graph on $m$ vertices and $r$ edges. 
	For any $q \in \Set{r, m(m+1)/2}$, let $\cG(G, q)$ be the set containing a representative graph from each isomorphic class of graphs obtained from $G$ by adding \textit{exactly} $q-r$ edges.
	We call $\cG(G,q)$ the $q$-extension set of $G$. Notice that, trivially $\cG(G,r)=\Set{G}\,.$
\end{definition}

\begin{figure}[!ht]
	\centering
	\begin{tikzpicture}[roundnode/.style={draw,shape=circle,minimum size=1mm}]
		\node[roundnode]      (v1)        {};
		\node[roundnode]      (v3)       [right= of v1] {};
		\node[roundnode]      (v4)       [below =of v3] {};
		\node[roundnode]      (v5)       [right = of v4] {};
		\node[roundnode]      (v6)       [above  = of v5] {};
		\node[roundnode]      (v7)       [ right =of v6] {};
		\node[roundnode]      (v8)		 [below right = of v6] {};
		\node[roundnode]      (v9)		 [left = of v4] {};
		\node[roundnode]      (v10)		 [below left = of v4] {};
		\node[roundnode]      (v11)		 [below = of v5] {};
		
		\node[roundnode]      (u1)        [right=4cm of v7] {};
		\node[roundnode]      (u3)       [right= of u1] {};
		\node[roundnode]      (u4)       [below =of u3] {};
		\node[roundnode]      (u5)       [right = of u4] {};
		\node[roundnode]      (u6)       [above  = of u5] {};
		\node[roundnode]      (u7)       [ right =of u6] {};
		\node[roundnode]      (u8)		 [below right = of u6] {};
		\node[roundnode]      (u9)		 [left = of u4] {};
		\node[roundnode]      (u10)		 [below left = of u4] {};
		\node[roundnode]      (u11)		 [below = of u5] {};
		\draw[] (v1) -- (v3);
		\draw[] (v3) -- (v4);
		\draw[] (v3) -- (v5);
		\draw[] (v10) -- (v4);
		\draw[] (v9) -- (v4);
		\draw[] (v11) -- (v5);
		\draw[] (v5) -- (v6);
		\draw[] (v6) -- (v7);
		\draw[] (v6) -- (v8);
		
		\draw[] (u1) -- (u3);
		\draw[] (u3) -- (u4);
		\draw[] (u3) -- (u5);
		\draw[] (u10) -- (u4);
		\draw[] (u9) -- (u4);
		\draw[] (u11) -- (u5);
		\draw[] (u5) -- (u6);
		\draw[] (u6) -- (u7);
		\draw[] (u6) -- (u8);
		\draw[] (u4) -- (u11) [dashed];
	\end{tikzpicture}
	\caption{Example of tree $T$ on $9$ edges and a graph $G\in \cG(T,10)$.}
\end{figure}
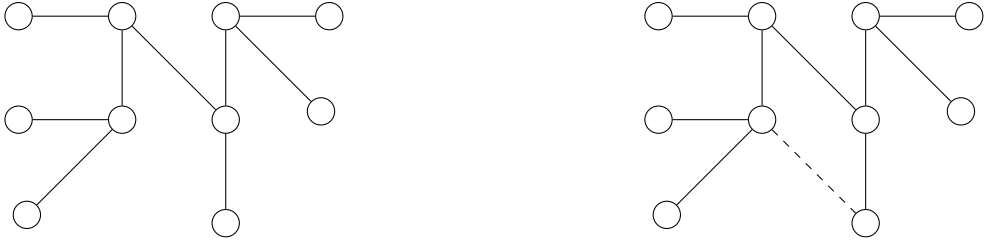

\subsection{Lower bound for non-centered Schatten norm}\label{sec:lowerbound-Schatten-norm}

We prove here the following theorem, which implies the first half of \cref{thm:sbm-schatten-norm}.

\begin{theorem}\label{thm:truncation-schatten-norm-lower-bound}
	Consider the settings of \cref{thm:sbm-schatten-norm}. Then
	\begin{align*}
		\underset{H \in \bsaw{s}{t}}{\sum}\E \Brac{\overline{\mathbf{Y}}_H} \geq  \frac{1}{10n^{\frac{2}{50A}}}\cdot \Paren{\frac{\eps\cdot d}{2}}^{st}\cdot \Paren{ 1-n^{-\frac{1}{10}}-(1+\delta)^{-t\sqrt{s}}}\,.
	\end{align*}
\end{theorem}

Our proof will roughly consist of the first three steps of the strategy outlined in \cref{sec:proof-strategy}.  Specifically, we show step one in \cref{sec:upperbound-any-multigraph}. In section \cref{sec:upperbound-negligible-bsaw}  we show that $\underset{H \in \nbsaw{s}{t}^c}{\sum} \E U_H(\mathbf{x})$ is small, hence preparing the ground for an inequality of the form \cref{eq:step-3}. Finally, in \cref{sec:bounding-non-centered-Schatten-norm} we will obtain a lower bound for nice block self-avoiding walks. Taken together, these results will imply \cref{thm:truncation-schatten-norm-lower-bound}. 

\subsubsection{An upper bound for every multigraph}\label{sec:upperbound-any-multigraph}
In this section we show an upper bound on the expectation of $\mathbf{Y}_H$ for any multigraph $H$. In order to do this we need to introduce some definitions. Let $H$ be a \textit{multigraph} (hence possibly not a block self-avoiding walks) with vertex set $V(H)\subseteq [n]$.

\begin{definition}
	\label{def:deg1-classification}
	We classify the vertices $v\in V(H)$ according to there degree-$1$ as follows:
	\begin{itemize}
		\item If $d^H_{1}(v)\leq \tau$, we say that $v$ is $1$-small in $H$. We denote the set of $1$-small vertices in $H$ as $\mathcal{S}_1(H)$.
		\item If $d^H_{1}(v)>\tau$, we say that $v$ is $1$-large in $H$. We denote the set of $1$-large vertices in $H$ as $\mathcal{L}_1(H)$.
	\end{itemize}
\end{definition}

\begin{definition}
	\label{def:deg2-classification}
	We classify the vertices $v\in V(H)$ according to their degree-$\geq 2$ as follows:
	\begin{itemize}
		\item If $d^H_{\geq {2}}(v)\leq \frac{\Delta}{2}$, we say that $v$ is $(\geq 2)$-small in $H$. We denote the set of $(\geq 2)$-small vertices in $H$ as $\mathcal{S}_{\geq 2}(H)$.
		\item If $\frac{\Delta}{2}<d^H_{\geq{2}}(v)\leq \Delta$, we say that $v$ is $(\geq 2)$-intermediate in $H$. We denote the set of $(\geq 2)$-intermediate vertices in $H$ as $\mathcal{I}_{\geq 2}(H)$.
		\item If $d^H_{\geq 2}(v)>\Delta$, we say that $v$ is $(\geq 2)$-large in $H$. We denote the set of $(\geq 2)$-large vertices in $H$ as $\mathcal{L}_{\geq 2}(H)$.
	\end{itemize}
\end{definition}

\begin{definition}
	\label{def:E1-E2-annoying}
	We denote the set of edges of multiplicity 1 as $E_1(H)$, and denote the set of edges of multiplicity at least 2 as $E_{\geq 2}(H)$.
	An edge of multiplicity 1 is said to be \emph{annoying} if one of its end vertices is in $\mathcal{L}_1(H)$. We denote the set of annoying edges of multiplicity 1 as $E_{1}^a(H)$.
	We partition $E_{\geq 2}(H)$ into two sets:
	$$E_{\geq 2}^a(H)=\{uv\in E_{\geq 2}(H): u\notin \mathcal{L}_{\geq 2}(H)\text{ and }v\notin \mathcal{L}_{\geq 2}(H)\},$$
	and
	$$E_{\geq 2}^b(H)=\{uv\in E_{\geq 2}(H): u\in \mathcal{L}_{\geq 2}(H)\text{ or }v\in \mathcal{L}_{\geq 2}(H)\}.$$
\end{definition}

\begin{definition}
	\label{def:upper-bound-UH}
	For every \textit{multigraph} $H$, we define the quantity
	\begin{align*}
	\overline{U}_H(x)=&2n^{\frac{1}{50A}}\frac{1}{\resizebox{0.055\textwidth}{!}{$\displaystyle\prod_{v\in\mathcal{L}_{\geq 2}(H)}$} n^{\frac{1}{4}\left(d^H_{\geq 2}(v)-\Delta\right)}}\left(\frac{\epsilon d}{2n}\right)^{|E_1(H)|}\left(\frac{2d}{n}\right)^{|E_{\geq 2}^b(H)|}\\
	&\cdot \underset{v\in V(H)}{\prod}\left(\frac{6}{\epsilon}\right)^{\max\Set{2d_1^H(v)-\tau,0}}\prod_{uv\in E_{\geq 2}^a(H)}\left[\left(1+\frac{\epsilon x_ux_v}{2}\right)\frac{d}{n}+\frac{3d^2}{n\sqrt{n}} \right]\,.
	\end{align*}
\end{definition}

We can now present  a general upper bound for the expectation of block self-avoiding walks.

\begin{lemma}
	\label{lem:sbm-upper-bound-any-bsaw}
	Consider the settings of \cref{thm:sbm-schatten-norm}.
	For every multigraph $H$ with $V(H)\subseteq [n]\,, \Card{V(H)}\leq s\log n$, let $\overline{U}_H(x)$ be as in  \Cref{def:upper-bound-UH}. Then
	$$\big|\mathbb{E}\big[\overline{\mathbf{Y}}_H|x\big]\big|\leq \overline{U}_H(x),$$
	for $n$ large enough.
\end{lemma}

We prove \cref{lem:sbm-upper-bound-any-bsaw} in \cref{sec:appendix-lower-bound-non-centered}. 
Notice that, for block self-avoiding walks, the expression
$$\left(\frac{\epsilon d}{2n}\right)^{|E_1(H)|}\prod_{uv\in E_{\geq 2}(H)}\left[\left(1+\frac{\epsilon x_{u}x_v}{2}\right)\frac{d}{n}+\frac{d^2}{n^2}\right]$$
is roughly the upper bound that we can get for $|\mathbb{E}[\mathbf{Y}_H|x]|$ in the non-truncated case (see \cref{sec:preliminaries}). Therefore, truncation has the effect of:
\begin{itemize}
	\item An amplification by a factor $n^{\frac{1}{50A}}$.
	\item An amplification  by a factor of $\Paren{\frac{6}{\epsilon}}^{2d_1^H(v)-\tau}$ for every vertex in $\cL_1(H)$.
	\item For every $uv\in E_{\geq 2}^b(H)$, $\left[\left(1+\frac{\epsilon x_ux_v}{2}\right)\frac{d}{n}+\frac{d^2}{n^2}\right]$ is replaced by $\frac{2d}{n}$.
	\item For every $uv\in E_{\geq 2}^a(H)$, $\left[\left(1+\frac{\epsilon x_ux_v}{2}\right)\frac{d}{n}+\frac{d^2}{n^2}\right]$ is replaced by $\left[\left(1+\frac{\epsilon x_ux_v}{2}\right)\frac{d}{n}+\frac{3d^2}{n\sqrt{n}}\right]$.
	\item A reduction by a factor of $n^{\frac{1}{4}(d^H_{\geq 2}(v)-\Delta)}$ for every $(\geq 2)$-large vertex $v$ in $H$.
\end{itemize}

\subsubsection{Walks that are not nice have negligible contributions}\label{sec:upperbound-negligible-bsaw}
With the tools developed in \cref{sec:upperbound-any-multigraph} we can now approach \cref{thm:truncation-schatten-norm-lower-bound}. Remember from \cref{sec:proof-strategy} we want to show that only a specific subset of block self-avoiding walks have large contribution to the expectation \cref{eq:sbm-trace-expansion}. The next lemma formalizes this idea.

\begin{lemma}\label{lem:sbm-upperbound-negligible-walks}
	Consider the settings of \cref{thm:sbm-schatten-norm}. Let $\nbsaw{s}{t}$ be the set of block self-avoiding walks $H$ with the following structure:
	\begin{itemize}
		\item Every edge $e\in E(H)$ satisfies $m_H(e)\leq 2\,.$
		\item Every vertex $v\in V(H)$ satisfies $d_{1}^H(v)\in \Set{0,2}\,, d_{\geq 2}^H(v)\leq \Delta$.
		\item $E_1(H)$ is a non-empty cycle.
		\item $E_1(H)$ is a cycle on at least $t/\sqrt{A}$ edges.
		\item The edges of multiplicity 2 form a forest, i.e., $E_{\geq 2}(H)$ is a forest.
		\item Each connected component of the forest of edges of multiplicity $2$ is connected to $E_1(H)$ through a single vertex.
	\end{itemize}
	Then for $n$ large enough 
	\begin{align*}
		\underset{H \in\nbsaw{s}{t}^c}{\sum} \E \overline{U}_H(\mathbf{x})\leq \Paren{n^{-\frac{1}{6}} +\Paren{1+\delta}^{-t\sqrt{s}}}\underset{H \in \nbsaw{s}{t}}{\sum}\E \overline{U}_{H}(\mathbf{x})\,.
	\end{align*}
\end{lemma}
Block self-avoiding walks in $\nbsaw{s}{t}$ are said to be \textit{nice}. Block self-avoiding walks in $\nbsaw{s}{t}^c$ are said to be \textit{negligible}.

\paragraph{Bounding negligible block self-avoiding walks with few vertices} 
We start our proof of \cref{lem:sbm-upperbound-negligible-walks} with an observation.
There are many block self-avoiding walks that "\textit{clearly}" have a negligible contribution in the expectation of \cref{eq:sbm-trace-expansion}.
To get some intuition, consider the following example. Let $\cS$ (we will use this notation \textit{only} for this specific example) be the set of all block self-avoiding walks in $\bsaw{s}{t}$ which have all vertices with degree $2$ and all edges with multiplicity $1$. That is, each walk in $\cS$ is a cycle. Let $\cS'$ be instead the subset of block self-avoiding walks in which all but one vertex have degree $2$, one vertex has degree $4$ and all edges have multiplicity $1$.
Now it is immediate to see that for any $H\in \cS$ and $H'\in \cS'$ we have $\E\Brac{\overline{U}_H(\mathbf{x})}=\E \Brac{\overline{U}_{H'}(\mathbf{x})}$ but
\begin{align*}
	\underset{H' \in \cS'}{\sum} \E \Brac{\overline{U}_{H'}(\mathbf{x})} =(1\pm o(1))\cdot n^{-1}\cdot  \underset{H \in \cS}{\sum} \E \Brac{\overline{U}_{H}(\mathbf{x})}\,,
\end{align*}
where the inequality follows simply because we have $\Card{\cS}\approx n\cdot \Card{\cS'}$. 
In \cref{lem:remove-overlapping-vertices} we  formalize this and similar observations.
We introduce first additional tools.

\begin{fact}\label{fact:sbm-upperbound-removing-cycles}
	Consider the settings of \Cref{thm:sbm-schatten-norm} and let $\Psi\geq 0$.
	Let $H\in \bsaw{s}{t}$ be a multigraph on at most $O(t)$ vertices and let $H^*$ be an induced sub-multigraph of $H$ satisfying:
	\begin{enumerate}
		\item the maximum $(\geq 2)$-degree in $H^*$ is $\Psi\geq 0\,,$
		\item all the edges in the cut $H(V(H), V(H)\setminus V(H^*))$ have multiplicity one in $H\,.$
	\end{enumerate}
	We denote $\ell,q\geq 0$ as the number of multiplicity-$1$ edges in $H^*$ and $H(V(H), V(H)\setminus V(H^*))$
	respectively. Let $Z$ be a set of vertices in $V(H^*)$ such that 
	$H(V(H^*)\setminus Z)$ has no multiplicity-$2$ cycles. Then
	\begin{align*}
		\E \overline{U}_H(\mathbf{x}) \leq & \frac{1}{4}n^{-1/25A}\Paren{\frac{6}{\epsilon}}^{2\ell+2q}\E \overline{U}_{H\Paren{V,V\setminus V(H^*)}}(\mathbf{x})\cdot \E \overline{U}_{H\Paren{V\setminus V(H^*)}}(\mathbf{x})
		 \\
		&\cdot \Paren{1+\frac{\eps}{2}}^{\Card{Z}\cdot \Psi}
		\cdot \Paren{\frac{2d}{n}}^{\Card{E^b_{\geq 2}(H^*)}}\cdot \underset{v\in V(H^*)}{\prod}\left(\frac{6}{\epsilon}\right)^{\max\Set{2d_1^H(v)-\tau,0}}\\
		&\cdot\frac{1}{\resizebox{0.055\textwidth}{!}{$\displaystyle\prod_{v\in\mathcal{L}_{\geq 2}(H^*)}$} n^{\frac{1}{4}\left(d^{H^*}_{\geq 2}(v)-\Delta\right)}}
		\left(\frac{\epsilon d}{2n}\right)^{|E_1(H^*)|}\prod_{e\in E_{\geq 2}^a(H^*)}\left[\frac{d}{n}+\frac{3d^2}{n\sqrt{n}} \right]\,,
	\end{align*}
	and
	\begin{align*}
		\E \overline{U}_H(\mathbf{x}) \geq & \frac{1}{4}n^{-1/25A}\E \overline{U}_{H\Paren{V,V\setminus V(H^*)}}(\mathbf{x})\cdot \E \overline{U}_{H\Paren{V\setminus V(H^*)}}(\mathbf{x}) \\
		&\cdot \Paren{\frac{2d}{n}}^{\Card{E^b_{\geq 2}(H^*)}}\cdot \underset{v\in V(H^*)}{\prod}\left(\frac{6}{\epsilon}\right)^{\max\Set{2d_1^H(v)-\tau,0}}\\
		&\cdot\frac{1}{\resizebox{0.055\textwidth}{!}{$\displaystyle\prod_{v\in\mathcal{L}_{\geq 2}(H^*)}$} 
		n^{\frac{1}{4}\left(d^{H^*}_{\geq 2}(v)-\Delta\right)}}\left(\frac{\epsilon d}{2n}\right)^{|E_1(H^*)|}\prod_{e\in E_{\geq 2}^a(H^*)}\left[\frac{d}{n}+\frac{3d^2}{n\sqrt{n}} \right]\,,
	\end{align*}
\end{fact}

We prove \cref{fact:sbm-upperbound-removing-cycles}  in \cref{sec:splitting-expectations}. 

We now temporarily limit our analysis to block self-avoiding walks with bounded maximum degree-$(\geq 2)$. As observed in \cref{sec:upperbound-any-multigraph}, walks with large degree-$(\geq 2)$ will have small contribution to \cref{eq:sbm-trace-expansion} and will be easily bounded. Let
\begin{align*}
	\bsaw{s}{t,d_{\geq 2}^H\leq \Delta}&:= \Set{H\in \bsaw{s}{t}\suchthat\max_{v\in v(H)}d_{\geq 2}^H(v)\leq \Delta}\,,
\end{align*}
we require another definition.

\begin{definition}\label{def:overlapping-vertices}
	Let $H\in \bsaw{s}{t, d_{\geq 2}^H\leq \Delta}$ and let $v\in V(H)$. We denote by $q_H(v)$ the number of connected components of the line graph  with vertex set $E_H(v)$ and such that there is an edge between $e,e'\in E_H(v)$ if and only if $e,e'$ appear in the sequence of edges $M(\cW(H))$ consecutively. Let $q_H:=\underset{v\in V(H)}{\sum}(q_H(v)-1)\,.$
	Now, for $q\geq 0$ we define $\cD_{q,s,t}\subseteq\bsaw{s}{t, d_{\geq 2}^H\leq \Delta}$ to be the subset of block self-avoiding walks  $H$ with $q_H=q\,.$
	We also write $\cD_{q,s,t}(H)\subseteq V(H)$ to be the set of vertices in $H$ with $q_H(v)\geq 2\,.$
	Finally we write
	\begin{align*}
		\cD_{\geq 1,s,t} = \underset{q\geq 1}{\bigcup} \cD_{q,s,t}\,.
	\end{align*}
	When the context is clear we write $ \cD_{q}$ instead of $\cD_{q,s,t}$.
\end{definition}

We can now prove that block self-avoiding walks in $\cD_{\geq 1}$ have negligible contribution to the expectation of \cref{eq:sbm-trace-expansion}.

\begin{lemma}\label{lem:remove-overlapping-vertices}
	Consider the settings of \cref{thm:sbm-schatten-norm}.
	Then for $n$ large enough 
	\begin{align*}
		\underset{H \in \cD_{\geq 1}}{\sum} \E\overline{U}_H(\mathbf{x})\leq \frac{1}{n^{2/3}} \underset{H \in \bsaw{s}{t, d_{\geq 2}^H\leq \Delta}\setminus \cD_{\geq 1}}{\sum} \E\overline{U}_H(\mathbf{x})\,.
	\end{align*}
	\begin{proof}
		Fix $q\geq 1$ and consider the following procedure  to obtain a block self-avoiding walk in $\bsaw{s}{t,d_{\geq 2}^H\leq \Delta}\setminus\cD_{q}$ from a block self-avoiding walk $H\in \cD_{q}$. 
		Let $M(\cW)$ be the sequence of edges obtained concatenating the generating self-avoiding walks of $H$ so that the subsequence $\Set{e^{(\ell-1)\cdot s+1},\ldots, e^{(\ell-1)\cdot s+s}}$ corresponds to the  $\ell$-th generating self-avoiding walk of $H$ (for simplicity we let $i-1=st$ for $i=1$ and analogously we let $i+1=1$ for $i=st$).
		Let $v\in \cD_q(H)$. Let $F_{H,v}$ be the line graph with vertex set $E_H(v)$ and edges as described in \cref{def:overlapping-vertices}.
		Let $E_H(v)^1$ be an arbitrary connected component of $F_{H,v}$ and let $u$ be a vertex not in $H$.  
		We construct the block self-avoiding walk $H'\in \cD_{q-1}$  with $V(H')=V(H)\cup\Set{u}$ applying the following operation on $H$: 
		\begin{itemize}
			\item Consider the sequence of edges $M(\cW(H))$, we replace  every edge $v w\in M(\cW(H))$ (and $wv\in M(\cW(H))$) such that $v w\in E^1_H(v)$ with the edge $uw$ (resp. $wu$).
		\end{itemize}
		Clearly, $H'\in \bsaw{s}{t,d_{\geq 2}^H\leq \Delta}\setminus\cD_{q}$ and   $\Card{V(H')}-\Card{V(H)} = 1$. Furthermore $\Card{E_1^a(H')}\geq \Card{E^a_1(H)}-\tau-2$. Thus 
		\begin{align*}
			\underset{v\in V(H)}{\prod}\left(\frac{6}{\epsilon}\right)^{\max\Set{2d_1^H(v)-\tau,0}}\leq& \Paren{\frac{6}{\eps}}^{\tau+2}\cdot \underset{v\in V(H')}{\prod}\left(\frac{6}{\epsilon}\right)^{\max\Set{2d_1^{H'}(v)-\tau,0}}\,.
		\end{align*}
		By  \cref{fact:sbm-upperbound-removing-cycles} it follows that
		\begin{align}
			\label{eq:decrease-overlapping-vertices}
			\frac{1}{n^{\Card{V(H')}-\Card{V(H)}}}\cdot \frac{\E\Brac{ \overline{U}_H(\mathbf{x})}}{\E \Brac{\overline{U}_{H'}(\mathbf{x})}}\leq \frac{1}{n}\cdot \Paren{\frac{6}{\eps}}^{2\tau}\Paren{1+\frac{\eps}{2}}^{2\Delta}\leq \frac{1}{n}\cdot \Paren{\frac{12}{\eps}}^{3\Delta}\,.
		\end{align}
		To obtain a multi-graph not in $\cD_{\geq 1}$ we repeatedly apply the operation above until $\cD_q(H)$ is empty. Notice that $(st)^{O(q)}$  applications suffice. 
		It remains to show that the contribution to the expectation of \cref{eq:sbm-trace-expansion} of block self-avoiding walks in $\cD_{\geq 1}$ is negligible. For this,  observe that at each step there are at most $(st)^2$ block self-avoiding walks that can produce the same multigraph $H'$.
		So using \cref{eq:decrease-overlapping-vertices}, we get for any $q\geq 1$
		\begin{align*}
			\underset{H\in \cD_{q}}{\sum}\E \overline{U}_H(\mathbf{x}) \leq \frac{(st)^2}{n^{0.99}}\underset{H' \in \cD_{q-1}}{\sum}\E \overline{U}_H(\mathbf{x})\,.
		\end{align*}
		The result follows since the maximum degree in any block self-avoiding walk is $2t$.
	\end{proof}
\end{lemma}

\paragraph{Bounding block self-avoiding walks from their shape and edges multiplicities}
Next we develop a general bound on the contribution of block self-avoiding walks based  on the shape of their underlying graph and the multiplicity of each edge. 
Together with  \cref{lem:remove-overlapping-vertices} this will be enough to obtain \cref{lem:sbm-upperbound-negligible-walks}.
We will need the following definitions. 
\begin{definition}
	For a collection of disjoint connected graphs on at least two vertices $\cB=\Set{B_1,\ldots,B_z}$, we define the set $\cM_{s,t}(B_1,\ldots,B_z)$ to be the subset of $\bsaw{s}{t}$ satisfying the following: if $H\in \cM_{s,t}(B_1,\ldots,B_z)$, then for any $B\in \cB$
	\begin{enumerate}[(i)]
		\item $B\subseteq G(H)$, we denote with $B'$ a (arbitrary) copy of $B$ in $H$ and by $H(B)$ the multigraph induced by $V(B')$,
		\item $\forall e\in E(H(B))$ that is also an edge in $B$, $m_H(e)\geq 2$,
		\item there exists a cut $H(V(B^{\prime}),V(H)\setminus V(B^{\prime}))$ in $H$ such that each edge in the cut has multiplicity $1$ in $H$. 
	\end{enumerate}
	and furthermore
	\begin{enumerate}[(iv)]
		\item the copies $B_1^{\prime},B_2^{\prime},\ldots,B_z^{\prime}$ are disjoint
		\item[(v)] every edge in $H\Paren{V(H)\setminus \Paren{\underset{j \in [z]}{\bigcup}V(B_j^{\prime})}}$ has multiplicity $1$.
	\end{enumerate}
	With a slight abuse of notation, we will simply write $H(B)$ instead of $H(B')$.
\end{definition}

\begin{definition}
	\label{def:generic-bsaw}
	Let $\cB=\Set{B_1,\ldots,B_z}$ be a collection of disjoint connected graphs, let $\Set{\ell_i,q_i,p_i,h_i}_{i=1}^z$ be a sequence of tuples of integers such that for all $i\in[z]$ $\ell_i,q_i,p_i,h_i\geq 0$. Further we denote $\mathcal{F}_i=\{\ell_i,q_i,p_i,h_i\}$.
	We write $\cM_{s,t,\Set{\mathcal{F}_i,\Psi_i}_{i=1}^z}(\cB)$ for the subset of $\cM_{s,t}(\cB)$ such that for any $i\in [z]$:
	\begin{enumerate}[(i)]
		\item the size of the cut $H(V(H)\setminus V(B_i), V(B_i))$ is $\ell_i$
		\item the number of edges in $H(B_i)$ of multiplicity one is $q_i$
		\item the number of edges $e$ in $H(B_i)$ with $m_H(e)=2$ is $h_i$ 
		\item the maximum degree-$(\geq 2)$ in $H(B_i)$ is $\Psi_i$.
		\item the edges with multiplicity larger than $2$  in $H(B_i)$  satisfy
		\begin{equation*}
			\sum_{\substack{e\in H(B_i)\\ m_H(e)\geq 3}} m_H(e)= p_i\,,
		\end{equation*}
	\end{enumerate}
\end{definition}

When the context is clear we simply write $\cM(\cB)$ and $\cM_{\Set{\mathcal{F}_i,\Psi_i}_{i=1}^z}(\cB)\subseteq \cM(B)$. For $\cB=\Set{B}$ we simply write $\cM(B)$. 
For $\cB=\emptyset$, the set $\cM(\cB)$ corresponds to the set of block self-avoiding walks in $\bsaw{s}{t}$ where all edges have multiplicity $1$. The next lemma studies the contribution of block self-avoiding walks in $\cM(\cB)$ for all $\cB$.

\begin{lemma}\label{lem:sbm-trace-main-technical-bound}
	Consider the settings of \cref{thm:sbm-schatten-norm}. Let $z\geq 1$ and $m_1,\ldots,m_z\geq 1$ be integers.
	Then for $n$ large enough,
	\begin{align*}
		\underset{\substack{\text{for }i \in [z]:\\q_i,\Psi_i,\ell_i,v_i\geq 0\\
		T_i\in \cT(m_i,v_i)\\
		r_i\geq m_i-1\\
		B_i\in \cG(T_i,r_i)\\
		h_i,p_i\geq 0}}{\sum}& \quad \underset{H\in \cM_{\Set{\cF_i,\Psi_i}_{i=1}^z}(\Set{B_1,\ldots,B_z})}{\sum} \E \overline{U}_H(\mathbf{x})\\
		&\leq \underset{i \in [z]}{\prod}\Brac{ (st)^{14}\cdot \Paren{\frac{200\Delta+s^{10}+2^{\frac{\log n}{t}}}{(1+\delta)^{s/4}}}^{2m_i/s}} \cdot \underset{H \in \cM(\emptyset)}{\sum}\E \overline{U}_H(\mathbf{x})\,.
	\end{align*}
	Furthermore, restricting the sum over $r_i> m_i-1$, or over $p_i\geq 1$, or $q_i\geq 1$, or $\Psi_i >\Delta$ or $\ell_i >  2$, for some $i\in [z]$ the inequality holds with an additional $n^{-\frac{1}{5}\Paren{p_i+q_i+\ind{\Psi_i>\Delta}(\Psi_i-\Delta)  +\ind{\ell_i>2}\Paren{\ell_i-2} }   }$ factor.
\end{lemma}

\cref{lem:sbm-trace-main-technical-bound} formalizes the following idea: for all possible collections of $z$ (possibly isomorphic) graphs $B_1,\ldots,B_z$ 
respectively on $m_1,\ldots,m_z$ vertices, the contribution to the expectation of \cref{eq:sbm-trace-expansion} of block self-avoiding walks in $\cM(\cB)$ can be upper bounded by $\underset{H\in \cM(\emptyset)}{\sum}\E \overline{U}_H(\mathbf{x})$ times a scalar which depends on the order and the shape of the graphs. Indeed the sum on the left-hand side captures all possible choices of graphs $B_1,\ldots,B_z$ and sets $\cM_{\Set{\cF_i,\Psi_i}_{i=1}^z}(\Set{B_1,\ldots,B_z})$. Moreover the lemma implies that the contribution of block self-avoiding walks in $\cM(\cB)$ is negligible if any of the graphs in $\cB$ contains a cycle, and that among the block self-avoiding walks in $\cM(\cB)$, most of the mass is concentrated in a very specific subset of block self-avoiding walks. This last observation will be extremely useful in simplifying our analysis and prove \cref{lem:sbm-upperbound-negligible-walks}.
It can be observed how for certain block self-avoiding walks with few edges of multiplicity at least $2$ this bound appears very rough, however, we can bound the contribution of these walks to the expectation of \cref{eq:sbm-trace-expansion} using \cref{lem:remove-overlapping-vertices}.

Concerning the parameters,  $2(2h_j+p_j+q_j)/s+\ell_j$ is an upper bound on the maximum number of pivots of $H$ that can be in $H(B_j)$. The parameter $v_j$ (the number of degree $1$ vertices in $B_j$) has a loose correspondence with the number of vertices $u$ with $d_{\geq 2}^{H(B_j)}(u)=1$ in $H(B)$, in the sense that $v_j-2r_j\leq \Card{\Set{u \in H(B_j)\suchthat d_{\geq 2}^{H(B_j)}(u)=1 }}\leq v$. 
Finally, recall from section \cref{sec:additional-notation} that with $\cT(m,v)$ we denote the set of non-isomorphic (picking one arbitrary representative per class) trees on $m$ vertices and $v$ leaves. For a given tree $T$ we let $\cG(T,r)$ be the set of non-isomorphic graphs obtained from $T$ adding $r- \Card{E(T)}$ edges.

To prove \cref{lem:sbm-trace-main-technical-bound} we need some intermediate results. First, we need to count how many block self-avoiding walks are in $\cM_{\Set{\cF_i,\Psi_i}_{i=1}^{z}}(\cB)$ (for some choice of the parameters). 
\begin{lemma}\label{lem:sbm-encoding-of-multigraph-with-B}
	Consider the settings of \cref{thm:sbm-schatten-norm}.
	Let $\cB=\Set{B_1,\ldots,B_z}$ be a collection of disjoint connected graphs each with respectively $m_1,\ldots,m_z$ vertices. 
	Let $\Set{\cF_i}_{i=1}^z$ be a sequence of tuples of integers  as in \cref{def:generic-bsaw}.
	Let $f_{s,t},g_{s,t}$ be the functions
	\begin{align*}
		f_{s,t}(m, m', \mathcal{F},\Psi) =& \Paren{\Psi}^{2h/s+10(q+\ell+p+1)+2h-2(m'-1)}\cdot (st)^{5\ell+5q+8p+4h+4-4(m'-1)}\,,\\
		g_{s,t}(m', \mathcal{F}) =& n^{-p-\ell/2-q-2h+m'}\,.
	\end{align*}
	Let $m=\underset{j \in [z]}{\sum}m_j$. 
	Then there are at most $n^{st}\cdot \underset{i \in [z]}{\prod}f_{s,t}(m,m_i,\mathcal{F}_i,\Psi_i)\cdot g_{s,t}(m_i,\mathcal{F}_i)$ block self-avoiding walks in the set $\cM_{s,t,\Set{\cF_i,\Psi_i}_{i=1}^{z}}(\cB)$.
\end{lemma}

We prove \cref{lem:sbm-encoding-of-multigraph-with-B} in \cref{sec:counting-bsaw} and directly use it here. Second we introduce bounds to split the expectation of $\overline{U}_H(\mathbf{x})$ into the expectation of its components.

\begin{fact}\label{fact:splitting-upper-bound}
	Consider the settings of \cref{thm:sbm-bound-second-moment-large-t}.
	Let $\cB=\Set{B_1,\ldots,B_z}$ be a collection of disjoint connected graphs on at least $2$ vertices. 
	Then for any $H \in \cM_{\Set{\cF_i,\Psi_i}_{i=1}^{z}}(\cB)$ and $i \in [z]$
	\begin{align*}
			\E \overline{U}_H(\mathbf{x}) &\leq   \frac{1}{4}n^{-1/25A}\Paren{\frac{6}{\eps}}^{2\ell_i+2q_i}\E \overline{U}_{H\Paren{V,V\setminus B_i}}(\mathbf{x})\cdot \E \overline{U}_{H\Paren{V\setminus B_i}}(\mathbf{x})\cdot \E\overline{U}_{H(B_i)}(\mathbf{x})\,,\\
			\E \overline{U}_H(\mathbf{x}) &\geq  \frac{1}{4}n^{-1/25A}\E \overline{U}_{H\Paren{V,V\setminus B_i}}(\mathbf{x})\cdot \E \overline{U}_{H\Paren{V\setminus B_i}}(\mathbf{x})\cdot \E\overline{U}_{H(B_i)}(\mathbf{x})\,.
	\end{align*}
\end{fact}

We are now ready to prove \cref{lem:sbm-trace-main-technical-bound}.

\begin{proof}[Proof of \cref{lem:sbm-trace-main-technical-bound}]
	Our strategy will be the following: fix some graphs $B_1,\ldots,B_{z-1}$ and some tuples $\Set{\cF_i,\Psi_i}_{i=1}^{z-1}$. Then we will show that the contribution of block self-avoiding walks in $\cM_{\Set{\cF_i,\Psi_i}_{i=1}^z}\Paren{\Set{B_1,\ldots,B_{z}}}$ for all possible choice of the graph $B_{z}$ and the parameters $\cF_z,\Psi_z$ is upper bounded by some function of $\underset{H \in \cM_{\Set{\cF_i,\Psi_i}_{i=1}^{z-1}}\Paren{\Set{B_1,\ldots,B_{z-1}}}}{\sum}\E \overline{U}_H(\mathbf{x})$.
	This function will be easy to upper bound for the set $\cM(\emptyset)$, thus by reiterating the analysis for $j=z,\ldots,j=1$ we will obtain the desired bound.
	
	Now, for simplicity, for any $H\in \cM(\cB)$ let us write $H^*=H(V\setminus(B_1\cup\ldots\cup B_z))$. Define $\sum_{j\in [z]}m_j=:m$. Recall that for any $m_z,v_z$ by \cref{fact:bound-unlabeled-trees} we have $\Card{\cT(m_z,v_z)} \leq 2v_z\cdot \Paren{8e\cdot m_z/v_z}^{2v_z}$ and that for any graph in the extension set of some tree in $T\in \cT(m_z,v)$ we have at most $m_z^2$ possible choices for each additional edge. Combining  these bounds with \cref{lem:sbm-encoding-of-multigraph-with-B} we will be able to compute the number of block self-avoiding walks at hand.
	
	We can start carrying out the computation.  Fix some graphs $B_1,\ldots,B_{z-1}$ and some tuples $\Set{\cF_i,\Psi_i}_{i=1}^{z-1}$.
	By \cref{fact:splitting-upper-bound}, for any multigraph in $\cM_{\Set{\cF_i,\Psi_i}_{i=1}^z}\Paren{\Set{B_1,\ldots,B_{z}}}$ for some $B_z, \cF_z,\Psi_z$:
	\begin{align}
		\E \overline{U}_H(\mathbf{x}) \leq   \underset{j\in [z]}{\prod}&\Brac{\Paren{\frac{6}{\eps}}^{2\ell_j+2q_j} \cdot (2n^{1/50A})^{-5/2}\cdot  \E\overline{U}_{H(B_j)}(\mathbf{x})\cdot \E \overline{U}_{H(V(H^*),B_j)} (\mathbf{x})	\cdot \underset{k\in [z], k\neq j}{\prod} \sqrt{\E \overline{U}_{H(B_j,B_k)}(\mathbf{x})}}\nonumber\\
		&\cdot 	\E \overline{U}_{H^*}(\mathbf{x})\,,\label{eq:split-expectation-MB}
	\end{align}
	where we used the squared root to avoid counting for the edges in the cut $H(B_j,B_k)$ twice and the factor $(2n^{1/50A})^{-5/2}$ appears due to the fact that we use multiple times  upper bounds of the form $\overline{U}_H(\mathbf{x})$.
	For fixed $\Set{\cF_i,\Psi_i}_{i=1}^{z-1}$ and fixed $m_z>0$ we can thus focus on bounding,
	\begin{align}
		&\underset{\substack{q_z,\ell_z\geq 0\\ \Psi_z\geq 0}}{\sum}\quad \underset{\substack{v_z\leq 2(2h_z+p_z+q_z)/s+\ell_z\\T\in \cT(m_z,v_z)}}{\sum}\quad \underset{\substack{r\geq m_z-1\\B_z \in \cG(T,r)\,:\\ \max \deg (B_z)=\Psi_z}}{\sum}\quad
		\underset{\substack{h_z,p_z\geq 0\\ H\in \cM_{\Set{\cF_i,\Psi_i}_{i=1}^{z}}(\Set{B_1,\ldots,B_{z}})}}{\sum} \nonumber \\
		&\E\overline{U}_{H^*}(\mathbf{x})\cdot \E \overline{U}_{H(B_z)}(\mathbf{x})\cdot \E \overline{U}_{H(V(H^*), B_z)}(\mathbf{x})\cdot \underset{k\in [z-1]}{\prod} \sqrt{\E \overline{U}_{H(B_z,B_k)}(\mathbf{x})}\cdot \Paren{\frac{6}{\eps}}^{2\ell_z+2q_z}\,.\label{eq:sum-many-B}
	\end{align}
	For each $i \in [z]$, we let $\ell_i=\ell_i'+\ell_i''$ where $\ell_i'$ corresponds to the number of edges in $H(V(H^*), B_i)$ and $\ell_i''$ corresponds to the number of remaining edges in the cut.
	Note that $h_z+\Card{\Set{e\in H(B_z)\,:\; m_H(e\geq 3)}} = r \leq h_z+p_z$. 
	So,  for fixed $\Psi_z,q_z,h_z,p_z,\ell_z,r,v_z$ it holds by \cref{fact:sbm-upperbound-removing-cycles}
	\begin{align}
		&\E\overline{U}_{H^*}(\mathbf{x})\cdot \E \overline{U}_{H(B_z)}(\mathbf{x})\cdot \E \overline{U}_{H(V(H^*), B_z)(\mathbf{x})}\cdot \underset{k\in [z-1]}{\prod} \sqrt{\E \overline{U}_{H(B_j,B_k)}(\mathbf{x})}\cdot \Paren{\frac{6}{\eps}}^{2\ell_z+2q_z}\nonumber\\
		&\leq \Brac{\underset{u\in V(H^*)}{\prod}\Paren{\frac{6}{\eps}}^{\max\Set{2d_1^{H^*}(u)-\tau,0}}} \cdot \Paren{\frac{\eps d}{2n}}^{\Card{E(H^*)}}\cdot \Paren{\frac{6}{\eps}}^{2\ell_z+2q_z}\cdot \Paren{\frac{\eps d}{2n}}^{q_z+\ell_z'+\ell_z''/2}\cdot \Brac{\Paren{1+\frac{1}{\sqrt{n}}}\frac{d}{n}}^r\nonumber\\
		&\quad\cdot  2^{2(r-m_z-1)\cdot \Psi_z}\cdot \underset{u\in \cL_{\geq 2}(H(B_z))}{\prod}\Brac{\Paren{\frac{2d}{n}}^{\frac{1}{4}\left(d^H_{\geq 2}(u)-\Delta\right)}}\,.\label{eq:fixed-parameters-bounding-split-expectation}
	\end{align}
	Notice that changing $\ell_z,q_z$ then the multigraph $H^*$ changes and so does $\E \overline{U}_{H^*}(\mathbf{x})$.
	Moreover by \cref{lem:sbm-encoding-of-multigraph-with-B},  since $\Psi_z\leq 2t$ it follows that the contribution to  \cref{eq:sum-many-B} of block self-avoiding walks with $\Psi_z\geq \Delta$ will be at least a factor $n^{\frac{1}{5}}$ smaller.
	
	By \cref{lem:remove-overlapping-vertices}, we may assume there are no annoying edges in $H^*$. Thus we may upper bound  \cref{eq:sum-many-B} by
	\begin{align}
		n^{st}\cdot &\Brac{\underset{j \in [z-1]}{\prod} f_{s,t}\Paren{m,m_j,\cF_j,\Psi_j}\cdot g_{s,t}(m_j,\cF_j) } \nonumber\\
		\cdot & \underset{\substack{p_z,h_z,\ell_z, q_z\geq 0\\ r\geq m_z-1\\ \Psi_z\leq \Delta\\ v_z\leq 2(2h_z+p_z+q_z)/s+\ell_z}}{\sum}\Paren{8e\cdot m_z/v_z}^{2v_z} \cdot f_{s,t}(m,m_z,\cF_z,\Psi_z)\cdot g_{s,t}(m_z,\cF_z)\\
		\cdot &
		\Paren{\frac{\eps d}{2n}}^{st-\underset{i \in [z]}{\sum} (2h_i+p_i+q_i+\ell_i'+\ell_i''/2)}\cdot \Paren{\frac{6}{\eps}}^{2\ell_z+2q_z}\cdot \Paren{\frac{\eps d}{2n}}^{q_z+\ell_z'+\ell_z''/2}\cdot \Brac{\Paren{1+\frac{1}{\sqrt{n}}}\frac{d}{n}}^r\nonumber\\
		\cdot &\Paren{2m_z}^{2(r-m_z-1)}\cdot 2^{2(r-m_z-1)\cdot \Psi_z}
		\,.\label{eq:sum-over-parameters}
	\end{align}
	By \cref{eq:fixed-parameters-bounding-split-expectation}  and \cref{lem:sbm-encoding-of-multigraph-with-B} it is easy to see that 
	 \cref{eq:sum-over-parameters} is a geometric sum. So define the quantity
	\begin{align*}
		M_{\ell'_z,\ell_z''}:=
		(st)^4\cdot &\Delta^{2m_z/s+10}\cdot (100e\cdot s)^{2(m_z/s+\ell_z)}\cdot \Paren{\frac{\eps\cdot d}{2n}}^{\ell'_z+\ell_z''/2}\cdot \Paren{\frac{d}{n}}^{m_z-1}\cdot\Paren{\frac{6}{\eps}}^{2\ell_z}\\
		&\cdot f_{s,t}(m, m_z, \Set{\ell_z,0,0,m_z-1},\Psi_z)\cdot g_{s,t}( m_z,\Set{\ell_z,0,0,m_z-1})\,.
	\end{align*}
	If, for fixed $\ell_z= \ell'_z+\ell''>0,  h_z$ we restrict the sum \cref{eq:sum-over-parameters} to the case $p_z\geq 1$ we get the upper bound
	\begin{align*}
		\leq n^{st}\cdot &\Brac{\underset{j \in [z-1]}{\prod} f_{s,t}\Paren{m,m_j,\cF_j,\Psi_j}\cdot g_{s,t}(m_j,\cF_j)} \\
		\cdot & n^{-1/3}\underset{\substack{\ell_z\geq 0\\ r\geq m_z-1\\ \Psi_z\leq \Delta\\ v_z\leq 2(2h_z+q_z)/s+\ell_z}}{\sum}M_{\ell_z}\cdot 
		 \Paren{\frac{\eps d}{2n}}^{st-\underset{i \in [z]}{\sum} (2h_i+p_i+q_i+\ell_i'+\ell_i''/2)}\cdot \Paren{\frac{6}{\eps}}^{2\ell_z+2q_z}\cdot \Paren{\frac{\eps d}{2n}}^{q_z+\ell'_z+\ell_z''/2}\\ 
		\cdot &\Brac{\Paren{1+\frac{1}{\sqrt{n}}}\frac{d}{n}}^r	\cdot \Paren{2m_z}^{2(r-m_z-1)}\cdot 2^{2(r-m_z-1)\cdot \Psi_z}\,.
	\end{align*}
	Similarly, restricting the sum  to the case $r>m_z-1$ (since for any additional edge in $B$ we then have at most $m_z^2$ possible choices) or $q_z\geq 1$ we also have $n^{-\frac{1}{3}}\cdot M_{\ell_z}$.  It remains to consider \cref{eq:sum-over-parameters} for $q_z=0\,,r=m_z-1\,,p_z=0$ for any $\ell_z$. 
	To do this we study first 
	the behavior of $M_{\ell_z}$ as $\ell_z$ grows.
	We distinguish two cases depending on whether $\ell_z=0$.
	
	If $\ell_z=0$, it means that for any of the graphs considered we had $H(B)=H$ and $\cB=\Set{B}$. Thus $2m_z +p_z= st$ and we get
	\begin{align*}
		M_0 \leq 
		& (st)\cdot (100e\cdot s)^{4t}\cdot \Paren{\frac{d}{n}}^{st/2}\\
		&\cdot f_{s,t}(st/2+1, st/2+1, \Set{0,0,0,st/2},\Delta)\cdot g_{s,t}(st/2+1, \Set{0,0,0,st/2})\\
		&\leq \Paren{100\Delta+s^{10}}^{2t}\cdot (st)^{6}\cdot n\cdot d^{st/2}\,,
	\end{align*}
	which yields a ratio between \cref{eq:sum-over-parameters} and $\underset{H\in \cM(\emptyset)}{\sum}\E \overline{U}_H(\mathbf{x})$ of 	
	\begin{align*}
		\Paren{\frac{\eps^2}{4d}}^{st/2}\cdot \Paren{200\Delta+s^{10}+C}^{2t} = \Paren{\frac{200\Delta+s^{10}+C}{(1+\delta)^{s/4}}}^{2t}\,,
	\end{align*}
	where $C=2^{2\frac{\log n}{t}}$.
	Conversely, suppose $\ell_z\geq 2$ (it must be even given that the graph is Eulerian). Then 
	\begin{align*}
		M_{\ell_z}\leq  (\Delta)^{2 m_z/2 +10\ell_z}\cdot &(st)^{6+5\ell_z}\cdot \Paren{\frac{d}{n}}^{m_z-1}\cdot \Paren{\frac{d}{n}\cdot \Paren{\frac{6}{\eps}}^4}^{\ell_z/2}\\
		\cdot   & f_{s,t}(m, m_z, \Set{\ell_z,0,0,m_z-1},\Delta)g_{s,t}(m_z, \Set{\ell_z,0,0,m_z-1})\,.
	\end{align*}
	For $\ell_z= 2$	
	the ratio with 
	\begin{align*}
		&n^{st} \cdot 	\E \overline{U}_{H^*}(\mathbf{x}) \\
		&\quad\underset{j\in [z-1]}{\prod}\Brac{f_{s,t}(m,m_j,\cF_j,\Psi_j) g_{s,t}(m_j,\cF_j)\Paren{\frac{6}{\eps}}^{2\ell_j+2q_j}  \E\overline{U}_{H(B_j)}(\mathbf{x})\E \overline{U}_{H(H^*,B_j)} (\mathbf{x}) \underset{k\in [z], k\neq j}{\prod} \sqrt{\E \overline{U}_{H(B_j,B_k)}(\mathbf{x})}}		
	\end{align*}
	can be bounded by
	\begin{align*}
		(st)^{14}\Paren{\frac{200\Delta^{10}}{(1+\delta)^{s/4}}}^{2m_z/s}\,.
	\end{align*}
	We get an additional $n^{-\frac{\ell_z-2}{5}}$ factor if $\ell_z > 2$. 
	We can now reiterate the analysis on $\cM_{\Set{\cF_i,\Psi_i}_{i=1}^j}(\Set{B_1,\ldots,B_{j}})$ for $j=z-1,\ldots,1$. The result follows.
\end{proof}

\paragraph{Putting things together} We are now ready to prove \cref{lem:sbm-upperbound-negligible-walks}. 

\begin{proof}[Proof of \cref{lem:sbm-upperbound-negligible-walks}]
	We argue that if a block self-avoiding walk $H$ is not in $\nbsaw{s}{t}$, then it satisfies one of the following:
	\begin{itemize}
		\item $H\in \cD_{\geq 1}$,
		\item $H\in \cM_{\Set{\cF_i,\Psi_i}_{i=1}^z}(\cB)$ for any non-empty $\cB$ and tuples $\Set{\cF_i}_{i=1}^z$ with $p_i\geq 1$ or $q_i\geq 1$ or $\ell_i\geq 3$ or $\Psi_i>\Delta$ for some $i\in [z]$,
		\item $H\in \cM_{\Set{\cF_i,\Psi_i}_{i=1}^z}(\cB)$ for any non-empty $\cB$ containing some $B_i$ with a cycle,
		\item all edges in $H$ have multiplicity at least $2$.
	\end{itemize}
	Suppose this claim holds. The contribution of walks satisfying one of the first three bullets is negligible by  \cref{lem:remove-overlapping-vertices} and \cref{lem:sbm-trace-main-technical-bound}. Thus we only need to consider walks with all edges with multiplicity $2$. By \cref{lem:sbm-trace-main-technical-bound}, it suffices to show that 
	\begin{align}
		\label{eq:sbm-s-large-enough}
		2\Paren{\frac{300\Delta+s^{10}+2^{\frac{\log n}{t}}}{(1+\delta)^{s/4}}}\leq \Paren{1+\delta}^{-10\sqrt{s}}\,,
	\end{align} 
	which, for $t\in\Brac{\frac{\log n}{400}\,,\frac{\log n}{100}}$, may be rewritten as 
	\begin{align*}
		s\geq 
		\frac{10^8}{\delta}\Paren{\max \Set{\log 300\Delta,\log s,  1}}^2\,.
	\end{align*}
	By assumption on $s$, \cref{eq:sbm-s-large-enough} is satisfied and thus the desired inequality follows.
	
	It remains to verify our claim. 
	So consider some $H$ in $\nbsaw{s}{t}^c$. Suppose that there is some vertex $v\in V(H)$ with $d_{\geq 2}^H(v)=0$ and $d_1^H(v)\geq 4$, by construction then $H\in \cD_{\geq 1}$.
	Suppose now $H\in \cM(\cB)$ for some collection of graphs $\cB=\Set{B_1,\ldots,B_z}$.
	We may assume that $H\in \cM_{\Set{\cF_i,\Psi_i}_{i=1}^z}(\cB)$ where for all $i\in [z]$, $\ell_z=2\,,q_z=p_z=0\,,, h_z=m_z-1, \Psi_z\leq \Delta$ as otherwise $H$ satisfies one of the bullet points listed above. Moreover, for any $i\in [z]$ the graph $B_i$ must be a tree.
	Now, the claim would follow if we can show that for any $i \in[z]$ $H(B_i)$ is connected to $H(V\setminus B_i)$ by a single vertex.
	Suppose this is not true, since by assumption $\ell_i=2$, the vertex-cut between $H(B_i)$ and $H(V\setminus B_i)$ must have cardinality $2$. We obtain a contradiction since this implies that $q_i+p_i\geq 1$ as $H$ is a closed walk.
\end{proof}

\subsubsection{Bounding the non-centered Schatten norm}\label{sec:bounding-non-centered-Schatten-norm}
The machinery of \cref{sec:upperbound-any-multigraph} allowed us to upper bound the expectation of any block self-avoiding walk under the truncated distribution. Then in \cref{sec:upperbound-negligible-bsaw} we used such machinery to show that certain block self-avoiding walk have small contribution to the expectation of \cref{eq:sbm-trace-expansion}. In this section we prove \cref{thm:truncation-schatten-norm-lower-bound}.

Our strategy will be the following. For most nice block self-avoiding walks we will be able to lower bound the expectation. THe remaining ones will be few and have negligible contribution to the expectation of \cref{eq:sbm-trace-expansion}.
Combining these observations with \cref{lem:sbm-upperbound-negligible-walks} will yield ref{thm:truncation-schatten-norm-lower-bound}.

\paragraph{Bounds on nice block self-avoiding walks} 
We start proving additional bounds on the expectation of nice block self-avoiding walks.
We further divide the set of nice block self-avoiding walks into sets.

\begin{definition}
	For $\gamma\geq 0$ define the set
	\begin{align*}
		\nbsaw{s}{t,d^H\leq \gamma}:=\Set{H\in \nbsaw{s}{t}\suchthat \forall v \in V(H)\,, d^H(v)\leq \gamma}\,.
	\end{align*}
	For $m,z\geq 1$ let $\nbsaw{s}{t,m,z}$ be the subset of block self-avoiding walks $H\in\nbsaw{s}{t}$ such that the graph obtained from $G(H)$ by removing all edges of multiplicity $1$ in $H$ is a forest on $m$ vertices and $z$ components. 
	Notice that
	if $H\in \nbsaw{s}{t,m,z}$ then $\Card{E_{\geq 2}(H)}=m-z$.
	We also define
	\begin{align*}
		\nbsaw{s}{t,d^H\leq \gamma,m,z} :=	\nbsaw{s}{t,d^H\leq \gamma}\cap 	\nbsaw{s}{t,m,z}\\
		\,.
	\end{align*}
	Observe that for $m=z=0$ we have
	\begin{align*}
		\nbsaw{s}{t,0,0}=\nbsaw{s}{t}\cap \cM_{s,t}(\emptyset)\,.
	\end{align*}
\end{definition}

For many nice block self-avoiding walks, we can lower bound the expectation.

\begin{lemma}
	\label{lem:sbm-lower-bound-nice-bsaws}
	Consider the settings of \cref{thm:sbm-schatten-norm}. Let $m,z $ be integers such that $st-2m-2z\geq t/A$. 
	If $n$ is large enough, then for every 	$H\in \nbsaw{s}{t,d^H\leq \Delta,m,z}$
	\begin{align*}
		\mathbb{E}\big[\overline{\mathbf{Y}}_H\big]\geq \overline{L}_H:=\frac{1}{3n^{\frac{1}{100 A}}}\cdot \left(\frac{\epsilon d}{2n}\right)^{|E_1(H)|}\cdot\left(\frac{d}{n}\right)^{|E_{\geq 2}(H)|}\,.
	\end{align*}
\end{lemma}
We present the proof of \cref{lem:sbm-lower-bound-nice-bsaws}  in \cref{sec:appendix-bounds-nice-multigraphs}, and directly use the result.
For the remaining nice block self-avoiding walks, the next two results upper bound their expectation.

\begin{lemma}
	\label{lem:sbm-upper-bound-nice-bsaws-large-degree}
	Consider the settings of \cref{thm:sbm-schatten-norm}. 
	If $n$ is large enough, then for any $H\in \nbsaw{s}{t}\setminus\nbsaw{s}{t,d^H\leq \Delta}$
	\begin{align*}
		\Abs{	\mathbb{E}\big[\overline{\mathbf{Y}}_H\big]}\leq \frac{4}{n^{1/12}}\cdot \Paren{\frac{\eps d}{2n}}^{\Card{E_1(H)}}\cdot \Paren{\frac{d}{n}}^{\Card{E_{\geq 2}(H)}}\,.
	\end{align*}
\end{lemma}
We prove \cref{lem:sbm-upper-bound-nice-bsaws-large-degree} in \cref{sec:appendix-bounds-nice-multigraphs}.

\begin{fact}
	\label{lem:sbm-upper-bound-nice-bsaws-few-edges}
	Consider the settings of \cref{thm:sbm-schatten-norm}. Let $m,z\geq0$ be integers.
	If $n$ is large enough, then for any $H\in \nbsaw{s}{t, m,z}$ and any $H'\in \nbsaw{s}{t,0,0}$
	\begin{align*}
		\Abs{\mathbb{E}\big[\overline{\mathbf{Y}}_H\big]}\leq 6n^{3/100A}\cdot \Brac{(1-o(1))\Paren{1+\delta}}^{-m+z} \cdot n^{m-z}\cdot \overline{L}_{H'}\,.
	\end{align*}
	\begin{proof}
		Notice that 
		\begin{align*}
			\frac{\E \overline{U}_{H}(\mathbf{x})}{\E \overline{U}_{H'}(\mathbf{x})}\leq \Paren{\frac{d}{n}}^{z-m}\cdot \Paren{\frac{(1-o(1))\eps}{2}}^{2z-2m} = \Brac{(1-o(1))(1+\delta)}^{-m+z}\cdot n^{m-z}\,.
		\end{align*}
		Thus the result follows applying \cref{lem:sbm-lower-bound-nice-bsaws} and the definition of $U_H(\mathbf{x})$.
	\end{proof}
\end{fact}

\paragraph{Counting nice block self-avoiding walks} Next we count the number of nice block self-avoiding walks with maximum degree larger than $\Delta$ or few edges of multiplicity $1$.

\begin{lemma}\label{lem:count-nice-walks-large-degree}
	Consider the settings of \cref{thm:sbm-schatten-norm}.  
	Define the set 
	\begin{equation*}
		\begin{split}
			\nbsaw{s}{t,m,z, \ell_1,\ell_2} := \left\{H \in \nbsaw{s}{t,m,z}\suchthat  \Card{\Set{v\in v(H)\suchthat d^H(v)=\Delta+1}}= \ell_1  \right.   \,,\\
			\left. \Card{\Set{v\in v(H)\suchthat d^H(v)=\Delta+2}}= \ell_2 \right\}
		\end{split}
	\end{equation*}
	Then for $n$ large enough, there exists $\ell\geq \ell_1+2\ell_2$ such that
	\begin{align*}
		\Card{\nbsaw{s}{t,m,z,\ell_1,\ell_2}}\leq (1+o(1))\cdot 2^t\cdot n^{-\ell}\Card{\nbsaw{s}{t,d^H\leq \Delta, m',z'}}\,.
	\end{align*}
	for some $m',z'\geq 0$  such that $m-z=m'-z'+\ell$.
\end{lemma}
We  prove \cref{lem:count-nice-walks-large-degree} in \cref{sec:counting-bsaw}.  

\begin{lemma}
	\label{lem:count-nice-walks-few-edges-multiplicity-1}
	Consider the settings of \cref{thm:sbm-schatten-norm}. 
	Let $m,z\geq 0$ be integers such that $st-2m-2z\geq t/10\sqrt{A}$. Then
	\begin{align*}
		\Card{\nbsaw{s}{t,m,z}}\leq 2^{10t}s^t\cdot n^{-m+z}\cdot \Card{\nbsaw{s}{t,0,0}}\,.
	\end{align*}
\end{lemma}
We also prove \cref{lem:count-nice-walks-few-edges-multiplicity-1} in \cref{sec:counting-bsaw}. 

\paragraph{Putting things together}
We can now bound the contribution to the expectation of \cref{eq:sbm-trace-expansion} of nice block self-avoiding walks, and hence of all block self-avoiding walks of length $st$.

\begin{lemma}\label{lem:sbm-contribution-nice-walks}
	Consider the settings of \cref{thm:sbm-schatten-norm}. Then for $n$ large enough
	\begin{align*}
		\underset{H \in \nbsaw{s}{t}}{\sum} \E \Brac{\overline{\mathbf{Y}}_H}>\frac{1}{9n^{\frac{2}{50A}}}\cdot \Paren{ \frac{\eps \cdot d}{2}}^{st}\,. 
	\end{align*}
	\begin{proof}
		By \cref{lem:count-nice-walks-large-degree} and \cref{lem:sbm-upper-bound-nice-bsaws-large-degree}, the contribution of nice block self-avoiding walks with maximum degree larger than $\Delta$ can be bounded by
		\begin{align*}
			\underset{m,z\geq 1}{\sum}\quad\underset{H \in \nbsaw{s}{t,m,z}\setminus\nbsaw{s}{t, d^H\leq \Delta}}{\sum} \E \Brac{\overline{\mathbf{Y}}_H}&\leq 	\underset{m,z\geq 1}{\sum}\quad \underset{H \in \nbsaw{s}{t,m,z}\setminus\nbsaw{s}{t, d^H\leq \Delta}}{\sum} \E \Brac{\overline{U}_H(\mathbf{x})}\\
			&\leq 2^t\cdot n^{-1/12}\cdot \underset{m,z\geq 1}{\sum} \quad\underset{H \in \nbsaw{s}{t,d^H\leq \Delta,m,z}}{\sum} \E \Brac{\overline{U}_H(\mathbf{x})}\\
			&\leq n^{-1/13}\cdot \underset{m,z\geq 1}{\sum} \quad\underset{H \in \nbsaw{s}{t,d^H\leq \Delta,m,z}}{\sum} \E \Brac{\overline{U}_H(\mathbf{x})}\,.
		\end{align*}
		By \cref{lem:count-nice-walks-few-edges-multiplicity-1} and \cref{lem:sbm-upper-bound-nice-bsaws-few-edges}, the contribution of nice block self-avoiding walks with at most $t/\sqrt{A}$ edges of multiplicity $1$ can be bounded by
		\begin{align*}
			&\underset{\substack{m,z\geq 1 s.t.\\ st-2m+2z\leq t/\sqrt{A}}}{\sum}\quad \underset{H \in \nbsaw{s}{t, d^H\leq \Delta, m,z}}{\sum} \E \Brac{\overline{\mathbf{Y}}_H}\\
			\quad&\leq \underset{\substack{m,z\geq 1 s.t.\\ st-2m+2z\leq t/\sqrt{A}}}{\sum}\quad \underset{H \in \nbsaw{s}{t, d^H\leq \Delta, m,z}}{\sum}  \E \Brac{\overline{U}_H(\mathbf{x})}\\
			\quad&\leq (6st)\cdot n^{3/100A}\cdot (2^{10}s)^t\cdot \Brac{\paren{1-o(1)}\paren{1+\delta}}^{-st/2}\underset{H \in \nbsaw{s}{t, d^H\leq \Delta, 0,0}}{\sum}  \overline{L}_H\\
			\quad&\leq \paren{1+\delta}^{-st/3}\underset{H \in \nbsaw{s}{t, d^H\leq \Delta, 0,0}}{\sum}  \overline{L}_H\,.
		\end{align*}
		Combining the two bounds, it follows by \cref{lem:sbm-lower-bound-nice-bsaws}
		\begin{align*}
			\underset{H \in \nbsaw{s}{t}}{\sum}\geq \underset{H \in \nbsaw{s}{t, d^H\leq \Delta,0,0}}{\sum}L_H\geq \frac{1}{9n^{\frac{2}{50A}}}\Paren{\frac{\eps \cdot d}{2}}^{st}\,.
		\end{align*}
	\end{proof}
\end{lemma}

\cref{thm:truncation-schatten-norm-lower-bound} now follows as a direct consequence.

\begin{proof}[Proof of \cref{thm:truncation-schatten-norm-lower-bound}]
	Combining  \cref{lem:sbm-upperbound-negligible-walks} and \cref{lem:sbm-contribution-nice-walks}, by assumption on $t,\Delta,s,d,\tau$ the result follows.
\end{proof}

\subsection{Upper bound on the centered Schatten norm}\label{sec:upperbound-schatten-norm-centered}
In this section we provide  an upper bound on $\Tr \Paren{\Q\Paren{\overline{\mathbf{Y}}}-\dyad{\mathbf{x}}}^{t}$ to obtain \cref{thm:sbm-schatten-norm}. We do it through the following result.

\begin{theorem}\label{thm:sbm-expectation-trace-centered}
	Consider the settings of \cref{thm:sbm-schatten-norm}. 
	Then
	\begin{align*}
	\E \Brac{\Tr\Paren{\Q(\overline{\mathbf{Y}})-\dyad{\mathbf{x}}}^t}\leq (1+\delta)^{-t/4}\cdot \E \Brac{\Tr\Paren{\Q(\overline{\mathbf{Y}})}^t}\,.
	\end{align*}
\end{theorem}

\cref{thm:sbm-schatten-norm} immediately follows combining \cref{thm:sbm-expectation-trace-centered} with \cref{thm:truncation-schatten-norm-lower-bound}.
The proof of \cref{thm:sbm-expectation-trace-centered}, amounts of step $4$ and $5$ of the scheme outlined in \cref{sec:proof-strategy}. 
The main observation we will need is that the subtraction of $\dyad{\mathbf{x}}$ has the effect of removing the contribution of many graphs and to leave the contribution of the others essentially unchanged.

\paragraph{Upper bound on negligible block self-avoiding walks}  
Here we bound from above the contribution to the expectation of \cref{eq:sbm-trace-expansion-centered} of negligible block self-avoiding walks. For $H\in \bsaw{s}{t}$, recall the definitions of $S_1(H)$ and $S_{\geq 2}(H)$ as in \cref{def:deg1-classification}, \cref{def:deg2-classification}.

\begin{definition}
	\label{def:sbm-upper-bound-centered-bsaw}
	Let $H\in \bsaw{s}{t}$ and let $\cW_{1}(H)\subseteq\cW(H)$ be the subset of generating self-avoiding walks $W$ of $H$ such that $V(W)\subseteq S_1(H)\cap S_{\geq 2}(H)$ and $E(W)\subseteq E_1(H)$.
	Then, for any $\mathbf{x}\sim \sbm$ define the quantity
	\begin{align*}
		\hat{U}_H(\mathbf{x}) = \frac{n^{\frac{1}{100A}}\cdot 2^{\Card{E_{1}^a(H)}}} {2^{As\cdot \Card{\cW_{1}(H)}}}\cdot \overline{U}_{H}(\mathbf{x})\,,
	\end{align*}
	where $A$ is as defined in \cref{eq:Delta-form}.
\end{definition}

Similarly to \cref{lem:sbm-upper-bound-any-bsaw}, for any $H \in \bsaw{s}{t}$ we can show an upper bound on $\hat{Y}_H$ for any $H\in \bsaw{s}{t}$
\begin{lemma}
	\label{lem:sbm-centered-upper-bound-any-bsaw}
	Consider the settings of \cref{thm:sbm-schatten-norm}. Let $H\in \bsaw{s}{t}$, for $n$ large enough and for any $\mathbf{x}\sim \sbm$
	\begin{align*}
		\Abs{\E \Brac{\hat{\mathbf{Y}}_H\given \mathbf{x}}} \leq \hat{U}_H(\mathbf{x})\,.
	\end{align*}
\end{lemma}

We defer  the proof of  \cref{lem:sbm-centered-upper-bound-any-bsaw}  to \cref{sec:appendix-upper-bound-centered}. Crucially, the lemma implies that $\E\hat{U}_H(\mathbf{x})\ll \E \overline{U}_H(\mathbf{x})$ for many  nice block self-avoiding walks. On the other hand for others: $\E\hat{U}_H(\mathbf{x})\approx \E \overline{U}_H(\mathbf{x})$. Together, these two bounds will allow us to obtain \cref{thm:sbm-expectation-trace-centered}.
First we show that the contribution to the expectation of \cref{eq:sbm-trace-expansion-centered} of negligible block self-avoiding walks is \textit{still} negligible. 

\begin{lemma}
	\label{lem:negligible-bsaw-still-negligible}
	Consider the settings of \cref{thm:sbm-schatten-norm}. 	
	Then for $n$ large enough
	\begin{align*}
		\underset{H \in \nbsaw{s}{t}^c}{\sum}\E \hat{U}_H(\mathbf{x})\leq
		\Paren{n^{-1/7}+\Paren{1+\delta}^{-t\sqrt{s}/2}}\cdot \underset{H \in \nbsaw{s}{t}}{\sum}\E \overline{U}_{H}(\mathbf{x})\,.
	\end{align*}
	\begin{proof}
	 By  \cref{lem:sbm-upperbound-negligible-walks},  \cref{def:sbm-upper-bound-centered-bsaw} and \cref{lem:sbm-centered-upper-bound-any-bsaw}, since 
	 $t\in \Brac{\frac{\log n}{400},\frac{\log n}{100}}$ and $A\geq 1000$ for any  $\Set{\cF_i, \Psi_i}_{i=1}^{z-1}$ the result follows. 	
	\end{proof}
\end{lemma}

So we only need to tackle nice walks. As in the proof of \cref{thm:truncation-schatten-norm-lower-bound} we get rid of a tiny fraction of nice block self-avoiding walks. 

\begin{lemma}\label{lem:sbm-contribution-bad-nice-walks-centered}
	Consider the settings of \cref{thm:sbm-schatten-norm}. Then for $n$ large enough
	\begin{align*}
		\underset{H \in \nbsaw{s}{t}}{\sum} \E \Brac{\hat{\mathbf{Y}}_H}\leq n^{-1/14}\underset{\substack{m,z\geq 0\\ st-2m+2z\geq t/\sqrt{A}}}{\sum}\quad \underset{H \in \nbsaw{s}{t, d^H\leq \Delta,m,z}}{\sum} \E \Brac{\overline{U}_H(\mathbf{x})}\,. 
	\end{align*}
	\begin{proof}
		By \cref{lem:count-nice-walks-large-degree} and \cref{lem:sbm-upper-bound-nice-bsaws-large-degree}, the contribution of nice block self-avoiding walks with maximum degree larger than $\Delta$ can be bounded by
		\begin{align*}
			\underset{m,z\geq 1}{\sum}\quad\underset{H \in \nbsaw{s}{t,m,z}\setminus\nbsaw{s}{t, d^H\leq \Delta}}{\sum} \E \Brac{\hat{\mathbf{Y}}_H}&\leq 	\underset{m,z\geq 1}{\sum}\quad \underset{H \in \nbsaw{s}{t,m,z}\setminus\nbsaw{s}{t, d^H\leq \Delta}}{\sum} \E \Brac{\hat{U}_H(\mathbf{x})}\\
			&\leq 2^t\cdot n^{-1/12}\cdot \underset{m,z\geq 1}{\sum} \quad\underset{H \in \nbsaw{s}{t,d^H\leq \Delta,m,z}}{\sum} \E \Brac{\hat{U}_H(\mathbf{x})}\,.
		\end{align*}
		By \cref{lem:count-nice-walks-few-edges-multiplicity-1} and \cref{lem:sbm-upper-bound-nice-bsaws-few-edges}, the contribution of nice block self-avoiding walks with at most $t/\sqrt{A}$ edges of multiplicity $1$ can be bounded by
		\begin{align*}
			&\underset{\substack{m,z\geq 1 s.t.\\ st-2m+2z\leq t/\sqrt{A}}}{\sum}\quad \underset{H \in \nbsaw{s}{t, d^H\leq \Delta, m,z}}{\sum} \E \Brac{\overline{\mathbf{Y}}_H}\\
			\quad&\leq \underset{\substack{m,z\geq 1 s.t.\\ st-2m+2z\leq t/\sqrt{A}}}{\sum}\quad \underset{H \in \nbsaw{s}{t, d^H\leq \Delta, m,z}}{\sum}  \E \Brac{\hat{U}_H(\mathbf{x})}\\
			\quad&\leq (6st)\cdot n^{1/100A}\cdot (2^{10}s)^t\cdot \Brac{\paren{1-o(1)}\paren{1+\delta}}^{-st/2}\underset{H \in \nbsaw{s}{t, d^H\leq \Delta, 0,0}}{\sum}  \E\Brac{\hat{U}_H(\mathbf{x})}\\
			\quad&\leq \paren{1+\delta}^{-st/3}\underset{H \in \nbsaw{s}{t, d^H\leq \Delta, 0,0}}{\sum}   \E\Brac{\hat{U}_H(\mathbf{x})}\,.
		\end{align*}
		Combining the two bounds with \cref{lem:sbm-centered-upper-bound-any-bsaw} the result follows.
	\end{proof}
\end{lemma}

\paragraph{Upper bound on the remaining self-avoiding walks}
Now we split the remaining walks in two sets. For one set, the contribution to the expectation of \cref{eq:sbm-trace-expansion-centered} will be roughly the same contribution to the expectation of \cref{eq:sbm-trace-expansion}. For the other, it will be considerably smaller. The catch is that in the expectation of \cref{eq:sbm-trace-expansion} the latter set has a significantly larger contribution.

\providecommand{\nnbsaw}[2]{\text{NNBSAW}_{#1,#2}}

\begin{definition}\label{def:nbsaw-negligible-when-centered}
	Denote by $\nnbsaw{s}{t}\subseteq \underset{\substack{m,z\geq 0\\ st-2m+2z\geq t/\sqrt{A}}}{\bigcup}\nbsaw{s}{t, d^H\leq \Delta, m,z}$ the set of nice self-avoiding walks $H$ such that 
	\begin{align*}
		\Card{\cW_{1}(H)}\geq \frac{t}{s}\,.
	\end{align*}
	Denote by $\nnbsaw{s}{t}^{c}$ the set $\Paren{\underset{\substack{m,z\geq 0\\ st-2m+2z\geq t/\sqrt{A}}}{\bigcup}\nbsaw{s}{t, d^H\leq \Delta, m,z}}\setminus\nnbsaw{s}{t}$. 
\end{definition}

Next we show how  the contribution of walks in $\nnbsaw{s}{t}, \nnbsaw{s}{t}^c$ changes from the expectation of \cref{eq:sbm-trace-expansion} to that of \cref{eq:sbm-trace-expansion-centered}. 

\begin{lemma}
	\label{lem:nice-walks-centered}
	Consider the settings of \cref{thm:sbm-schatten-norm}. Then for $n$ large enough
	\begin{align*}
		&\text{for every }H\in\nnbsaw{s}{t}\qquad &\E \hat{U}_H(\mathbf{x})\leq 2^{-st}\cdot \E \overline{U}_H(\mathbf{x})\,,\\
		&\text{for every }H\in\nbsaw{s}{t}\qquad &\E  \hat{U}_H(\mathbf{x}) \leq 3n^{\frac{1}{100A}}\cdot \E \overline{U}_H(\mathbf{x})\,.
	\end{align*} 
	\begin{proof}
		The first inequality follows by \cref{def:nbsaw-negligible-when-centered} and \cref{lem:sbm-centered-upper-bound-any-bsaw}.
		The second  by observing that for any $H \in \nbsaw{s}{t}$ we have  $E_{1}^a(H)=\emptyset$.
	\end{proof}
\end{lemma}

As an immediate corollary, \cref{lem:nice-walks-centered} implies  that block self-avoiding walks in $\nnbsaw{s}{t}$ have small expectation.

\begin{corollary}\label{cor:bound-contribution-nnbsaw}
	Consider the settings of \cref{thm:sbm-schatten-norm}. Then for $n$ large enough
	\begin{align*}
		\underset{H\in \nnbsaw{s}{t}}{\sum}\E \Brac{\hat{\mathbf{Y}}_H}\leq 2^{-st/2}\underset{H \in \nnbsaw{s}{t}}{\sum} L_H\,.
	\end{align*} 
\end{corollary}

It remains to bound the contribution to nice block self-avoiding walks in $\nnbsaw{s}{t}^c$.
We require an additional definition.
\begin{definition}
	Let $w,b,m,z\geq 0$ be integers. Define $\text{N}_{w,q,m,z}\subset\nnbsaw{s}{t}^c$ to be the set of nice block self-avoiding walks $H$ such that:
	\begin{itemize}
		\item the number of vertices $v$ with $d^H(v)\leq 1$ is $q$,
		\item the number of walks in $\mathcal{W}_1$ is $w$,
		\item $E_{\geq 2}(H)$ is a forest on $m$ edges and $z$ components.
	\end{itemize}
	Notice that by definition of $\nnbsaw{s}{t}^c$ the set $\text{N}_{q,w,m,z}$ is empty if $w\geq t/s$ or $m-z\leq t-2t/s$. Moreover, observe that $q<t$ since for any block self-avoiding walks all vertices with total degree $1$ must be pivots.
\end{definition}

By \cref{lem:nice-walks-centered}, we can upper bound the contribution of walks in $\text{N}_{w,q,m,z}$. To upper bound their number --and hence their total contribution-- we use the following two results.

\begin{lemma}\label{lem:nice-bsaw-upper-bound-m-z}
	Consider the settings of \cref{thm:sbm-schatten-norm}. Let $r\geq 0$ be an integer. For $n$ large enough
	\begin{align*}
		\underset{\substack{m,z\geq 0\\
		m-z=r}}{\sum}\Card{\nbsaw{s}{t,m,z}}\leq \Delta^{6t}n^{st-r}\,.
	\end{align*}
\end{lemma}

We prove \cref{lem:nice-bsaw-upper-bound-m-z} in \cref{sec:counting-bsaw}.

\begin{lemma}\label{lem:counting-special-nbsaw-centered}
	Consider the settings of \cref{thm:sbm-schatten-norm}. Let $m,z\geq0$ be integers such that $0\leq m-z< \frac{t\sqrt{s}}{2}$. Then for $n$ large enough
	\begin{align*}
		\Card{\text{N}_{w,q, m,z}}\leq 2^{ws}\cdot  \Paren{1+\delta}^{\frac{t}{10}}\cdot n^{-q}\cdot \Card{\nbsaw{s}{t, d^H\leq \Delta, m',z'}}
	\end{align*}
	for some $m',z'$ such that $m'-z'=m-z-q$.
	Moreover, it holds that $q>t-2t/\sqrt{s}$.
\end{lemma}

We also defer the proof of \cref{lem:counting-special-nbsaw-centered} to \cref{sec:counting-bsaw}.
We can finally bound the expectation of nice block self-avoiding walks in $N_{w,q,m,z}.$

\begin{lemma}\label{lem:bound-special-nbsaw-centered}
	Consider the settings of \cref{thm:sbm-schatten-norm}. Then for $n$ large enough
	\begin{align*}
		\underset{w,q,m,z\geq 0}{\sum}\quad \underset{H \in N_{w,q,m,z}}{\sum} \E \Brac{\hat{U}_{H}(\mathbf{x})}\leq \paren{1+\delta}^{-t/3}\underset{\substack{m,z\geq 0\\ st-2m+2z\geq t/\sqrt{A}}}{\sum}\quad \underset{H \in \nbsaw{s}{t, d^H\leq \Delta, m,z}}{\sum} \E \Brac{\overline{U}_H(\mathbf{x})}
	\end{align*}
	\begin{proof}
		First consider the case $m-z\geq t\sqrt{s}/2$. Since 
		\[
			\underset{w,q}{\bigcup}\text{N}_{w,q,m,z}\subseteq\nbsaw{s}{t,d^h\leq \Delta,m,z}\,,
		\]
		 by \cref{lem:nice-bsaw-upper-bound-m-z} and definition of $\hat{U}(\mathbf{x})\,, \overline{U}(\mathbf{x})$ we get
		\begin{align*}
			\underset{\substack{m,z\geq 0\\
					m-z\geq t/2\sqrt{s}}}{\sum}\quad \underset{H \in \nbsaw{s}{t, d^H\leq \Delta,m,z}}{\sum} \E \Brac{\mathbf{Y}_H} &\leq
			\underset{\substack{m-z\\m-z\geq t/2\sqrt{s}}}{\sum}\quad \underset{H \in \nbsaw{s}{t, d^H\leq \Delta,m,z}}{\sum} \E \Brac{\hat{U}_H(\mathbf{x})} \\
			&\leq \Paren{1+\delta}^{-t\sqrt{s}/2}\cdot \Paren{\Delta}^{6t}\cdot (1+o(1))^{st}\cdot \underset{H \in \nbsaw{s}{t,0,0}}{\sum}\E \Brac{\hat{U}_H(\mathbf{x})} \\
			&\leq (1+\delta)^{-t/2}\underset{H \in \nbsaw{s}{t,0,0}}{\sum}\E \Brac{\overline{U}_H(\mathbf{x})}\,.
		\end{align*}
		So consider the case  $m-z< t\sqrt{s}/2$. By \cref{lem:counting-special-nbsaw-centered} then $q\geq t-2t/\sqrt{s}$. For any $H\in N_{w,q,m,z}$ and any $H'\in \nbsaw{s}{t, d^H\leq \Delta,m',z'}$ with $m'-z'=m-z-q$ we have
		\begin{align*}
				\E \Brac{\hat{U}_H(\mathbf{x})} &\leq 3n^{1/100A}\cdot 2^{-Asw}\cdot \Brac{(1+o(1))(1+\delta)}^{-q} \cdot n^q\cdot \E\Brac{\overline{U}_{H'}(\mathbf{x})}\\
				&\leq \Paren{1+\delta}^{-t/2}\cdot n^q\cdot \E\Brac{\overline{U}_{H'}(\mathbf{x})}\,.
		\end{align*}
		By \cref{lem:counting-special-nbsaw-centered} it follows
		\begin{align*}
			\underset{H \in N_{w,q,m,z}}{\sum}\E \Brac{\hat{U}_H(\mathbf{x})} 
			&\leq (1+\delta)^{-t/2}\underset{H' \in \nbsaw{s}{t,m',z'}}{\sum} \E\Brac{\overline{U}_{H'}(\mathbf{x})}\,.
		\end{align*}
		Repeating the argument for each $w,q,m,z\geq 0$ such that $N_{w,q,m,z}$ is non-empty (and thus so is the corresponding $\nbsaw{s}{t,m',z'}$) we obtain the desired result.
	\end{proof}
\end{lemma}

We are now ready to prove \cref{thm:sbm-expectation-trace-centered}.

\begin{proof}[Proof of \cref{thm:sbm-expectation-trace-centered}]
	By \cref{lem:negligible-bsaw-still-negligible} and \cref{lem:sbm-contribution-bad-nice-walks-centered}
	\begin{align*}
		\underset{H \in \bsaw{s}{t}\setminus \Paren{\nnbsaw{s}{t}\cup\nnbsaw{s}{t}^c}}{\sum} \E \Brac{\hat{U}_H(\mathbf{x})} \leq \Paren{n^{-1/15}+(1+\delta)^{-t\sqrt{s}/2} }\underset{H\in \bsaw{s}{t}}{\sum} \E \Brac{\overline{U}_H(\mathbf{x})}\,.
	\end{align*}
	By \cref{cor:bound-contribution-nnbsaw} and \cref{lem:bound-special-nbsaw-centered}
	\begin{align*}
		\underset{H \in \nnbsaw{s}{t}\cup \nnbsaw{s}{t}^c}{\sum} \E \Brac{\hat{U}_H(\mathbf{x})}\leq \Paren{2^{-st}+(1+\delta)^{t/3}}\underset{H \in \bsaw{s}{t}}{\sum}\E \Brac{\overline{U}_H(\mathbf{x})}\,.
	\end{align*}
	Putting the two inequalities together, by \cref{lem:sbm-upper-bound-nice-bsaws-few-edges} we get
	\begin{align*}
		\E \Brac{\Tr\Paren{\Q(\overline{\mathbf{Y}})-\dyad{\mathbf{x}}}^t}\leq (1+\delta)^{-t/4}\cdot \E \Brac{\Tr\Paren{\Q(\overline{\mathbf{Y}})}^t}
	\end{align*}
	as desired.
\end{proof}

\subsection{Concentration of block self-avoiding walks}\label{sec:concentration-bsaws}
We prove here \cref{thm:sbm-bound-second-moment-large-t}. 
Our strategy will be similar to the one used for \cref{thm:sbm-schatten-norm}.
Concretely, we will show that for any $u,v \in [n]$ most of the mass of $\E \Brac{\Paren{\Q(\overline{\mathbf{Y}})^t}_{uu}\Paren{\Q(\overline{\mathbf{Y}})^t}_{vv}}$ comes from a specific subset of \textit{nice} block self-avoiding walks. 

\subsubsection{Multigraphs that are not nice have negligible contributions}\label{sec:multigraphs-negligible}
In this section,  for $u,v \in [n]$, we show which multigraphs  have negligible contribution in $\E \Brac{\Paren{\Q(\overline{\mathbf{Y}})^t}_{uu}\Paren{\Q(\overline{\mathbf{Y}})^t}_{vv}}$.
We need the following definitions.

\begin{definition}[Block self-avoiding walks with fixed pivot]
	For $u\in [n]$, let $\bsaw{s}{t,u}\subseteq\bsaw{s}{t}$ be the set of block self-avoiding walks with $u$ as pivot. We think of $u$ as the first (and last) vertex and we will refer to it as the \textit{first pivot}. Notice that the set $\bsaw{s}{t,u}$ corresponds to the set of block self-avoiding walks arising in $\E \Brac{\Paren{\Q(\overline{\mathbf{Y}})^t}_{uu}}$. For every  $H \in \bsaw{s}{t,u}$ we write $M(\cW(H))$ for its sequence of edges. We denote by $e^H_1,e^H_{st}$ respectively the first and last edges in the sequence $M(\cW(H))$ and write $M^{u}(H):=\Set{e^H_1,e^H_{st}}$. By construction both $e^H_1,e^H_{st}$ are incident to vertex $u$.
	Similarly, we define $\nbsaw{s}{t,u}=\bsaw{s}{t,u}\cap \nbsaw{s}{t}$.
\end{definition}

\begin{definition}[Decomposition block self-avoiding walks]
	For $u,v\in [n]$ and  a multigraph $H\in \bsaw{s}{t,u}\times \bsaw{s}{t,v}$ we denote by $\Ho\in \bsaw{s}{t,u}$ and $\Ht\in \bsaw{s}{t,v}$ the two block self-avoiding walks such that $H=\Ho\oplus \Ht$. We call $\Ho,\Ht$ the \textit{decomposition block self-avoiding walks} of $H$.
	We write $\Mo(H)$ for $M^u(\Ho)$ and $\Mt(H)$ for $M^v(\Ht)$. When the context is clear we simply write $\Mo, \Mt$. 
\end{definition}

Our central tool will be the following lemma, which resemble \cref{lem:sbm-upperbound-negligible-walks}.

\begin{lemma}\label{lem:sbm-upperbound-negligible-multigraphs}
	Consider the settings of \cref{thm:sbm-bound-second-moment-large-t}. Let $u,v \in [n]$ and let $\nmultig{s}{t,u,v}$ be the set of multigraphs $H$  in $\nbsaw{s}{t,u}\times \nbsaw{s}{t,v}$ with the following structure.
	If $u\neq v$:
	\begin{itemize}
		\item $V(\Ho)\cap V(\Ht)=\emptyset$. That is, the decomposition block self-avoiding walks of $H$ are disjoint.
	\end{itemize}
	If $u=v$ 
	\begin{itemize}
		\item The edges of multiplicity 2 form a forest, i.e., $E_{\geq 2}(H)$ is a forest.
		\item For each $v\in V(H)$, $d^H_{\geq 2}(v)\leq \Delta$.
		\item Each connected component $B$ of $E_{\geq 2}(H)$  not satisfying:
		\begin{align*}
			\Mo\cup \Mt&\subseteq E(H(V, V\setminus B)\oplus H(B))\,,\\
			u&\in V(B^{uu})\,,
		\end{align*}
		is connected to $E_1(H)$ through a single vertex, every edge in $M(H(B))$ satisfies $m_H(e)\leq 2\,.$ 
		\item If there is a connected component $B^{uu}$ of $E_{\geq 2}(H)$
		satisfying:
		\begin{align*}
			\Mo\cup \Mt&\subseteq E(H(V, V\setminus B)\oplus H(B))\,,\\
			u&\in V(B^{uu})\,,
		\end{align*}
		then $H(B^{uu})$ is connected to $E_1(H)$ by  $4$ edges and at most two vertices.
		\item  If $E_1(H)$  contains two connected components than these components are cycles. If $E_1(H)$ is connected then it is either a path, or a cycle, or a path connected to a cycle by one of its endpoints, or two cycles with a single vertex in common.
	\end{itemize}
	Then for $n$ large enough 
	\begin{align*}
		\underset{H \in\nmultig{s}{t,u,v}^c}{\sum} \E \overline{U}_H(\mathbf{x})\leq \Paren{n^{-\frac{1}{6}} +\Paren{1+\delta}^{-t\sqrt{s}}}\underset{H \in \nmultig{s}{t,u,v}}{\sum}\E \overline{U}_{H}(\mathbf{x})\,.
	\end{align*}
\end{lemma}
Multigraphs  in $\nmultig{s}{t,u,v}$ are said to be \textit{nice}. Multigraphs in $\nmultig{s}{t,u,v}^c$ are said to be \textit{negligible}.
It is important to observe that for $u=v$ the family of multigraphs we need to consider grows.
However, since $\nmultig{s}{t,u,u}\subseteq \nbsaw{s}{t,u}\times \nbsaw{s}{t,u}$ we can still ensures nice multigraphs satisfy several useful properties. For example, in every $H\in \nmultig{s}{t,u,u}$ edges have multiplicity at most $4$, moreover no vertex  $v\in V(H)$ has degree-$1$ larger than $4$.

\paragraph{Bounding negligible multigraphs with few vertices} 
The first step to prove \cref{lem:sbm-upperbound-negligible-multigraphs} is to obtain a result similar in spirit to \cref{lem:remove-overlapping-vertices}. 
For $u,v\in [n]$ define 
\begin{align*}
	\Paren{\bsaw{s}{t,u}\times \bsaw{s}{t,v}}_{d_{\geq 2}^H\leq \Delta}:=\Set{H\in \bsaw{s}{t,u}\times \bsaw{s}{t,v}\suchthat \max_{v\in V(H)}d_{\geq 2}^H(v)\leq \Delta}\,.
\end{align*}

\begin{definition}\label{def:overlapping-vertices-multigraphs}
	Let $u,v\in [n]$.
	Let $H\in 	\Paren{\bsaw{s}{t,u}\times \bsaw{s}{t,v}}_{d_{\geq 2}^H\leq \Delta}$ and let $w\in V(H)$. We denote by $q_H(w)$ the number of connected components of the line graph  with vertex set $E_H(w)$ and such that there is an edge between $e,e'\in E_H(w)$ if and only if $e,e'$ appear in the sequence of edges $M(\cW(H))$ consecutively. 
	We define  
	\begin{align*}
		q_H:=\Paren{q_H(v)+q_H(u)\ind{u\neq v}-2}+\underset{w\in V(H)\setminus\Set{u,v}}{\sum}(q_H(w)-1)\,.
	\end{align*}
	Now, for $q\geq 0$ we define $\cD_{q,s,t,u,v}\subseteq	\Paren{\bsaw{s}{t,u}\times \bsaw{s}{t,v}}_{d_{\geq 2}^H\leq \Delta}$ to be the subset of multigraphs $H$ with $q_H=q\,.$
	We also write $\cD_{q,s,t,u,v}(H)\subseteq V(H)$ to be the set of vertices $w$ in $H$ with
	\begin{align*}
		q_H(w)\geq 
		\begin{cases}
			3 &\text{if $u=v=w$}\\
			2 &\text{otherwise}\,.
		\end{cases}
	\end{align*} 
	Finally we write
	\begin{align*}
		\cD_{\geq 1,s,t,u,v} = \underset{q\geq 1}{\bigcup} \cD_{q,s,t,u,v}\,.
	\end{align*}
	When the context is clear we write $ \cD_{q,u,v}$ instead of $\cD_{q,s,t,u,v}$.
\end{definition}
We prove that multigraphs in $\cD_{\geq 1,u,v}$ have negligible contribution to the expectation of $\Paren{\Q(\overline{\mathbf{Y}})^t}_{uu}\Paren{\Q(\overline{\mathbf{Y}})^t}_{vv}$.

\begin{lemma}\label{lem:remove-overlapping-vertices-multigraphs}
	Consider the settings of \cref{thm:sbm-bound-second-moment-large-t}. Let $u,v\in[n]$.
	Then for $n$ large enough 
	\begin{align*}
		\underset{H \in \cD_{\geq 1,u,v}}{\sum} \E\overline{U}_H(\mathbf{x})\leq \frac{1}{n^{2/3}} \underset{H \in \Paren{\bsaw{s}{t,u}\times \bsaw{s}{t,v}}_{d_{\geq 2}^H\leq \Delta}\setminus \cD_{\geq 1,u,v}}{\sum} \E\overline{U}_H(\mathbf{x})\,.
	\end{align*}
	\begin{proof}
		Fix $q\geq 1$ and consider the following procedure  to obtain a multigraph in $\Paren{\bsaw{s}{t,u}\times \bsaw{s}{t,v}}_{d_{\geq 2}^H\leq \Delta}\setminus\cD_{q,u,v}$ from a multigraph $H\in \cD_{q,u,v}$. 
		Let $M^{(1)}(\cW), M^{(2)}(\cW)$ be respectively the sequence of edges obtained concatenating the generating self-avoiding walks of $\Ho$ and $\Ht$. We write  $\Set{e_{(1)}^{(\ell-1)\cdot s+1},\ldots, e_{(1)}^{(\ell-1)\cdot s+s}}$ for the subsequence corresponding to the  $\ell$-th generating self-avoiding walk of $\Ho$ (for simplicity we let $i-1=st$ for $i=1$ and analogously we let $i+1=1$ for $i=st$). We denote the generating self-avoiding walks of $\Ht$ similarly.
		Let $w\in \cD_{q,u,v}(H)$. Let $F_{H,w}$ be the line graph with vertex set $E_H(w)$ and edges as described in \cref{def:overlapping-vertices-multigraphs}.
		If $w\notin\Set{u,v}$, let $E_H(w)^1$ be an arbitrary connected component of $F_{H,w}$ and let $z$ be a vertex not in $H$.  Conversely if $w=u$  let $E_H(w)^1$ be an arbitrary connected component of $F_{H,w}$ such that $E_H(w)^1\cap \Mo(H)=\emptyset$, since $w\in \cD_{q,u,v}(H)$ there must exists such connected component (Notice that this also covers the case $u=v$). Analogously, if $w=v$  let $E_H(w)^1$ be an arbitrary connected component of $F_{H,w}$ such that $E_H(w)^1\cap \Mt(H)=\emptyset$, since $w\in \cD_{q,u,v}(H)$ 
		We construct the multigraph $H'\in \cD_{q-1,u,v}$  with $V(H')=V(H)\cup\Set{z}$ applying the following operation on $H$: 
		\begin{itemize}
			\item Consider the sequence of edges $M^{(1)}(\cW), M^{(2)}(\cW)$, we replace  every edge $w' w\in M^{(1)}(\cW)$ (and $ww'\in M^{(1)}(\cW)$) such that $w' w\in E^1_H(w)$ with the edge $w'z$ (resp. $zw'$). Similarly, we replace  every edge $w' w\in M^{(2)}(\cW)$ (and $ww'\in M^{(2)}(\cW)$) such that $w' w\in E^1_H(w)$ with the edge $w'z$ (resp. $zw'$).
		\end{itemize}
		Clearly, $H'\in \Paren{\bsaw{s}{t,u}\times \bsaw{s}{t,v}}_{d_{\geq 2}^H\leq \Delta}\setminus\cD_{q,u,v}$ and   $\Card{V(H')}-\Card{V(H)} = 1$. Furthermore $\Card{E_1^a(H')}\geq \Card{E^a_1(H)}-\tau-2$. Thus 
		\begin{align*}
			\underset{v\in V(H)}{\prod}\left(\frac{6}{\epsilon}\right)^{\max\Set{2d_1^H(v)-\tau,0}}\leq& \Paren{\frac{6}{\eps}}^{\tau+2}\cdot \underset{v\in V(H')}{\prod}\left(\frac{6}{\epsilon}\right)^{\max\Set{2d_1^{H'}(v)-\tau,0}}\,.
		\end{align*}
		By  \cref{fact:sbm-upperbound-removing-cycles} it follows that
		\begin{align}
			\label{eq:decrease-overlapping-vertices-multigraphs}
			\frac{1}{n^{\Card{V(H')}-\Card{V(H)}}}\cdot \frac{\E\Brac{ \overline{U}_H(\mathbf{x})}}{\E \Brac{\overline{U}_{H'}(\mathbf{x})}}\leq \frac{1}{n}\cdot \Paren{\frac{6}{\eps}}^{2\tau}\Paren{1+\frac{\eps}{2}}^{2\Delta}\leq \frac{1}{n}\cdot \Paren{\frac{12}{\eps}}^{3\Delta}\,.
		\end{align}
		To obtain a multi-graph not in $\cD_{\geq 1,u,v}$ we repeatedly apply the operation above until $\cD_{q,u,v}(H)$ is empty. Notice that $(st)^{O(q)}$  applications suffice.
		It remains to show that the contribution to the expectation of $\Paren{\Q(\overline{\mathbf{Y}})^t}_{uu}\Paren{\Q(\overline{\mathbf{Y}})^t}_{vv}$ of multigraphs in $\cD_{\geq 1,u,v}$is negligible. For this,  observe that at each step there are at most $(st)^4$ multigraphs in $\Paren{\bsaw{s}{t,u}\times \bsaw{s}{t,v}}_{d_{\geq 2}^H\leq \Delta}$ that can produce the same multigraph $H'$.
		So using \cref{eq:decrease-overlapping-vertices-multigraphs}, we get for any $q\geq 1$
		\begin{align*}
			\underset{H\in \cD_{q,u,v}}{\sum}\E \overline{U}_H(\mathbf{x}) \leq \frac{(st)^4}{n^{0.99}}\underset{H' \in \cD_{q-1,u,v}}{\sum}\E \overline{U}_H(\mathbf{x})\,.
		\end{align*}
		The result follows since the maximum degree in any multigraphs in $\Paren{\bsaw{s}{t,u}\times \bsaw{s}{t,v}}_{d_{\geq 2}^H\leq \Delta}$ is $4t$.
	\end{proof}
\end{lemma}

\paragraph{Bounding multigraphs from their shape and edges multiplicities} 
Next we extend \cref{lem:sbm-trace-main-technical-bound} to multigraphs in $\bsaw{s}{t,u}\times \bsaw{s}{t,v}$ for any $u,v \in [n]$. 
That is, we  compute  a general bound on multigraphs in $\bsaw{s}{t,u}\times \bsaw{s}{t,v}$ based on the shape of its underlying graph and the multiplicity of each edge. We introduce some needed definitions.

\begin{definition}
	Let $u,v\in [n]$. Let $\cB=\Set{B_1,\ldots,B_z}$ be a collections of disjoint connected graphs on at least two vertices. Let $B^{uv}$ be a connected graph on at least two vertices disjoint from any graph in $\cB$.
	We define $\cM_{s,t,u,v}(\cB, B^{uv})$ to be the subset of $\bsaw{s}{t,u}\times \bsaw{s}{t,v}$ satisfying the following. 
	For $H\in \cM_{s,t,u,v}(\cB,B^{uv})$, let $\Ho,\Ht$ be the decomposition block self-avoiding walks of $H$.
	For any $B\in \cB\cup\Set{B^{uv}}$:
	\begin{itemize}
		\item $B\subseteq G(H)$, we denote with $B'$ a (arbitrary) copy of $B$ in $H$ and by $H(B)$ the multigraph induced by $V(B')$ (With a slight abuse of notation we will simply write $V(B)$ for $V(B')$),
		\item $\forall e \in E(H(B))$ that is also an edge in $V$, $m_H(e)\geq 2$,
		\item there exists a cut $H(V(B'), V(H)\setminus V(B'))$ in $H$ such that each edge in the cut has multiplicity $1$ in $H$.
	\end{itemize}
	Furthermore, 
	\begin{itemize}
		\item every edge in $H\Paren{V(H)\setminus \Paren{\underset{B\in \cB\cup\cB^{uv}}{\bigcup}V(B')}}$ has multiplicity $1$,
		\item it holds  that $u, v \in V(B^{uv})$ and 
		\begin{align*}
			\Mo(H)\cup \Mt(H)\subseteq H(B^{uv})\cup H(B^{uv}, V\setminus B^{uv})\,.
		\end{align*}
	\end{itemize}
	That is, $B^{uv}$ contains  the first pivots $u,v$ of the decomposition block self-avoiding walks $\Ho\in \bsaw{s}{t,u}$ and  $\Ht\in \bsaw{s}{t,u}$ of $H$.
	If no such graph $B^{uv}$ exists we simply write $\cM_{s,t,u,v}(\cB,\emptyset)$. When the context is clear we drop the subscripts $s,t$.
\end{definition}

\begin{definition}
	\label{def:product-bsaws-shaped}
	Let $u,v\in [n]$.
	Let $\cB=\Set{B_1,\ldots,B_z}$  be  a collections of disjoint connected graphs and let $B^{uv}$ be a graph disjoint from any graph in $\cB$. Let $\Set{\ell_{i},q_i,p_i,h_i}_{i=1}^{z+1}$  be a sequence of tuples of integers such that for all $i\in [z]$, $\ell_i,q_i,p_i,h_i\geq 0$. Let $\Set{\Psi_i}_{i=1}^{z+1}$ be a sequence of positive integer. Further we denote $\mathcal{F}_i=\{\ell_i,q_i,p_i,h_i\}$. 
	We write $\cM_{s,t,u,v\Set{\mathcal{F}_i,\Psi_i}_{i=1}^{z+1}}(\cB, B^{uv})$ for the subset of $\cM_{s,t,u,v}(\cB,B^{uv})$ such that for any $i\in [z+1]$ 
	\begin{enumerate}[(i)]
		\item the size of the cut $H(V(H)\setminus V(B_i), V(B_i))$ is $\ell_i$,
		\item the number of edges in $H(B_i)$ of multiplicity one is $q_i$,
		\item the number of edges $e$ in $H(B_i)$ with $m_H(e)=2$ is $h_i$, 
		\item the maximum degree-$(\geq 2)$ in $H(B_i)$ is $\Psi_i$.
		\item the edges with multiplicity larger than $2$ in $H(B_i)$ satisfy
		\begin{equation*}
			\sum_{\substack{e\in H(B_i)\\ m_H(e)\geq 3}} m_H(e)= p_i\,.
		\end{equation*}
	\end{enumerate}
\end{definition}

Now we study the contribution of block self-avoiding walks in $\cM_{u,v}(\cB,B^{uv})$ for all $\cB, B^{uv}$.

\begin{lemma}\label{lem:sbm-second-moment-main-technical-bound}
Consider the settings of \cref{thm:sbm-bound-second-moment-large-t}. Let $z\geq 1$ and $m_1,\ldots,m_{z},m_{z+1}$ be nonnegative integers. 
Let $u,v\in [n]$. 
Then for $n$ large enough,
\begin{align*}
	\underset{\substack{\text{for }i \in [z+1]:\\q_i,\Psi_i,\ell_i,v_i\geq 0\\
			T_i\in \cT(m_i,v_i)\\
			r_i\geq m_i-1\\
			B_i\in \cG(T_i,r_i)\\
			h_i,p_i\geq 0}}{\sum}& \quad \underset{H\in \cM_{u,v,\Set{\cF_i,\Psi_i}_{i=1}^{z+1}}(\Set{B_1,\ldots,B_z}, B_{z+1})}{\sum} \E \overline{U}_H(\mathbf{x})\\
	&\leq \underset{i \in [z+1]}{\prod}\Brac{ (st)^{30}\cdot \Paren{\frac{200\Delta+s^{10}+4^{\frac{\log n}{t}}}{(1+\delta)^{s/8}}}^{4m_i/s}}
	\cdot \underset{H \in \cM(\emptyset, \emptyset)}{\sum}\E \overline{U}_H(\mathbf{x})\,.
\end{align*}
Furthermore:
\begin{itemize}
	\item restricting  the sum over $r_i> m_i-1$, or over $p_i\geq 1$, or $q_i\geq 1$, or $\Psi_i >\Delta$ or $\ell_i >  2$, for some $i\in [z]$ or $i\in [z+1]$ if $u\neq v$,  the inequality holds with an additional $n^{-\frac{1}{5}\Paren{p_i+q_i+\ind{\Psi_i>\Delta}(\Psi_i-\Delta)  +\ind{\ell_i>2}\Paren{\ell_i-2} }   }$ factor. 
	\item if $u= v$, restricting the sum over $r_{z+1}> m_{z+1}-1$, or over multigraphs $H$ with $m_{\Ho}(e)\geq 3 $  or $m_{\Ht}(e)\geq 3$ for some $e\in H$, or $q_{z+1}\geq 1$, or $\Psi_{z+1} >\Delta$ or $\ell_{z+1} >4$ the inequality holds with an additional $n^{-\frac{1}{5}\Paren{p_{z+1}/5+q_{z+1}+\ind{\Psi_{z+1}>\Delta}(\Psi_{z+1}-\Delta)  +\ind{\ell_{z+1}>2}\Paren{\ell_{z+1}-2} }   }$ factor. 
\end{itemize}
\end{lemma}

To prove \cref{lem:sbm-second-moment-main-technical-bound} we need two intermediate steps. First, an adaptation of \cref{lem:sbm-encoding-of-multigraph-with-B} to the sets in \cref{def:product-bsaws-shaped}. Second, a result along the lines of \cref{fact:splitting-upper-bound}.

\begin{lemma}
	\label{lem:sbm-counting-product-bsaws}
	Consider the settings of \cref{thm:sbm-bound-second-moment-large-t}. Let $u,v \in [n]$.
	Let $\cB=\Set{B_1,\ldots,B_z}$ be collections of disjoint connected graphs each with respectively $m_1,\ldots,m_{z}\geq 2$ vertices. Let $B^{uv}$ be a connected graph disjoint from any graph in $\cB$ and with $m_{z+1}\geq 2$ vertices.
	Let $\Set{\cF_k}_{i=1}^{z+1}$ be a sequence of tuples of integers  as in \cref{def:product-bsaws-shaped}.
	Let $f^*_{s,t},g^*_{s,t}$ be the functions 
	\begin{align*}
		f^*_{s,t}(m, m', \mathcal{F},\Psi) =& \Paren{\Psi}^{2h/s+10(q+\ell+p+1)+2h-2(m'-1)}\cdot (st)^{5\ell+5q+8p+4h+4-4(m'-1)}\,,\\
		g^*_{s,t}(m', \mathcal{F}) =& n^{-p-\ell/2-q-2h+m'}\,.
	\end{align*}
	Let $m=\underset{j \in [z+1]}{\sum}m_j$. 
	Then there are at most 
	\begin{align*}
		2n^{2st-2+\ind{u=v}\ind{B_{uv}\neq\emptyset}}\cdot &\underset{1\leq k\leq z'}{\prod}f^*_{s,t}(m,m_k,\mathcal{F}_i,\Psi_i)\cdot g^*_{s,t}(m_i,\mathcal{F}_i,h_i)
	\end{align*} 
	block self-avoiding walk pairs in the set $\cM_{u,v\Set{\cF_i,\Psi_i}_{i=1}^{z+z'}}(\cB)$.
\end{lemma}

We show \cref{lem:sbm-counting-product-bsaws} in \cref{sec:counting-bsaw}. We also extend \cref{fact:splitting-upper-bound} to mutligraphs in $\bsaw{s}{t,u}\times \bsaw{s}{t,v}$. 

\begin{fact}\label{fact:splitting-upper-bound-second-moment}
	Consider the settings of \cref{thm:sbm-bound-second-moment-large-t}.
	Let $u,v\in [n]$ , let $\cB=\Set{B_1,\ldots,B_z}$ be a collection of disjoint connected graphs on at least $2$ vertices and 
	let $B^{uv}$ be a (possibly empty) graph disjoint from any graph in $\cB$.
	Then  for any $H \in \cM_{\Set{\cF_i,\Psi_i}_{i=1}^{z+1}}(\cB, B^{uv})$ and $i \in [z+1]$
	\begin{align*}
		\E \overline{U}_H(\mathbf{x}) &\leq   \frac{1}{4}n^{-1/25A}\Paren{\frac{6}{\eps}}^{2\ell_i+2q_i}\E \overline{U}_{H\Paren{V,V\setminus B_i}}(\mathbf{x})\cdot \E \overline{U}_{H\Paren{V\setminus B_i}}(\mathbf{x})\cdot \E\overline{U}_{H(B_i)}(\mathbf{x})\,,\\
		\E \overline{U}_H(\mathbf{x}) &\geq  \frac{1}{4}n^{-1/25A}\E \overline{U}_{H\Paren{V,V\setminus B_i}}(\mathbf{x})\cdot \E \overline{U}_{H\Paren{V\setminus B_i}}(\mathbf{x})\cdot \E\overline{U}_{H(B_i)}(\mathbf{x})\,.
	\end{align*}
\end{fact}

We obtain \cref{fact:splitting-upper-bound} in \cref{sec:splitting-expectations} and directly apply it here. Next we  prove \cref{lem:sbm-second-moment-main-technical-bound}.

\begin{proof}[Proof of \cref{lem:sbm-second-moment-main-technical-bound}]
	Our argument closely resembles that of \cref{lem:sbm-trace-main-technical-bound}.
	Consider first the case $u\neq v$.  For any non-empty $B_{z+1}$ the same proof as in \cref{lem:sbm-trace-main-technical-bound}, combined with  \cref{lem:sbm-counting-product-bsaws} implies 
	\begin{align*}
		\underset{\substack{\text{for }i \in [z]:\\q_i,\Psi_i,\ell_i,v_i\geq 0\\
				T_i\in \cT(m_i,v_i)\\
				r_i\geq m_i-1\\
				B_i\in \cG(T_i,r_i)\\
				h_i,p_i\geq 0}}{\sum}& \quad \underset{\substack{H\in \cM_{u,v,\Set{\cF_i,\Psi_i}_{i=1}^z} (\Set{B_1,\ldots,B_z}, B_{z+1})}}{\sum} \E \overline{U}_H(\mathbf{x})\\
		&\leq \Paren{n^{-1/5}+(1+\delta)^{-\sqrt{s} t}}
		 \cdot\underset{\substack{\text{for }i \in [z]:\\q_i,\Psi_i,\ell_i,v_i\geq 0\\
		 		T_i\in \cT(m_i,v_i)\\
		 		r_i\geq m_i-1\\
		 		B_i\in \cG(T_i,r_i)\\
		 		h_i,p_i\geq 0}}{\sum}& \quad \underset{\substack{H\in \cM_{u,v,\Set{\cF_i,\Psi_i}_{i=1}^z} (\Set{B_1,\ldots,B_z}, \emptyset)}}{\sum}\E \overline{U}_H(\mathbf{x})\,,
	\end{align*}
	as it must be that either $\ell_{z+1}\geq 4$ or $p_{z+1}\geq 1$. Here we used the definition of $s,\Delta,t$.
	The rest of the proof then continues as in \cref{lem:sbm-trace-main-technical-bound} so we omit it.
	Conversely, consider the case $u=v$.
	Again as in \cref{lem:sbm-trace-main-technical-bound}, applying \cref{lem:sbm-counting-product-bsaws} we get for any $B_{z+1}$
	\begin{align*}
		\underset{\substack{\text{for }i \in [z]:\\q_i,\Psi_i,\ell_i,v_i\geq 0\\
				T_i\in \cT(m_i,v_i)\\
				r_i\geq m_i-1\\
				B_i\in \cG(T_i,r_i)\\
				h_i,p_i\geq 0}}{\sum}& \quad \underset{\substack{H\in \cM_{u,v,\Set{\cF_i,\Psi_i}_{i=1}^z} (\Set{B_1,\ldots,B_z}, B_{z+1})}}{\sum} \E \overline{U}_H(\mathbf{x})\\
		&\leq \underset{i \in [z]}{\prod}\Brac{ (st)^{30} \Paren{\frac{200\Delta+s^{10}+2^{\frac{\log n}{t}}}{(1+\delta)^{s/8}}}^{4m_i/s}}\\
		&\quad \cdot
		\underset{\substack{q_{z+1},\Psi_{z+1},\ell_{z+1},v_{z+1}\geq 0\\
				T_{z+1}i\in \cT(m_{z+1}i,v_{z+1})\\
				r_{z+1}\geq m_{z+1}-1\\
				B_{z+1}i\in \cG(T_{z+1}i,r_{z+1})\\
				h_{z+1}i,p_{z+1}i\geq 0}}{\sum}\quad \underset{\substack{H\in \cM_{u,v,\Set{\cF_{z+1},\Psi_{z+1}}} {(\emptyset, B_{z+1})}}}{\sum}\E \overline{U}_H(\mathbf{x})\,,
	\end{align*}
	where if we restrict the sum over $r_i>m_i-1$ or $p_i\geq1$ or $q_i\geq 1$ or $\Psi_i>\Delta$ or $\ell_i >2$ for some $i \in [z]$ the inequality holds with an additional $n^{-1/5}$ factor. 
	For simplicity of the notation let $i=z+1$. It remains to study
	\begin{align*}
		\underset{\substack{q_i,\Psi_i,\ell_i,v_i\geq 0\\
				T_i\in \cT(m_i,v_i)\\
				r_i\geq m_i-1\\
				B_i\in \cG(T_i,r_i)\\
				h_i,p_i\geq 0}}{\sum}& \quad \underset{H\in \cM_{u,v,\Set{\cF_i,\Psi_i}}(\emptyset, B_{i})}{\sum} \E \overline{U}_H(\mathbf{x})\,.
	\end{align*}
	For any $\Set{\cF_i,\Psi_i}, B_i, r_i, v_i$ and any $H\in \cM_{u,u,\Set{\cF_i,\Psi_i}}(\emptyset, B_{i})$ by \cref{fact:splitting-upper-bound-second-moment} and \cref{fact:sbm-upperbound-removing-cycles}
	\begin{align}
		\E \overline{U}_H(\mathbf{x}) &\leq   \Paren{\frac{6}{\eps}}^{2\ell_i+2q_i}\cdot (2n^{1/50A})^{-2} \E \overline{U}_{H(B_i)}(\mathbf{x})\cdot \E \overline{U}_{H(V\setminus B_i)}(\mathbf{x})\cdot \E \overline{U}_{H(V(H\setminus B_i), B_i)}(\mathbf{x})\nonumber\\
		&\leq 2n^{1/50A}\cdot \Brac{\underset{u \in V(H\setminus B_i)}{\prod}\Paren{\frac{6}{\eps}}^{\max \Set{2d_1^{H(V\setminus B_i)}-\tau,0}}}\cdot \Paren{\frac{6}{\eps}}^{2\ell_i+2q_i}\cdot \Paren{\frac{\eps d}{2n}}^{\Card{E(H(V\setminus B_i))}}\cdot\Paren{\frac{\eps d}{2n}}^{q_i+\ell_i}\nonumber
		\\
		&\cdot \Brac{\Paren{1+\frac{1}{\sqrt{n}}\frac{d}{n}}}^{r_i}\cdot2^{2(r_i-m_i-1)\cdot \Psi_i}\cdot \underset{u\in \cL_{\geq 2}(H(B_i))}{\prod}\Brac{\Paren{\frac{2d}{n}}^{\frac{1}{4}\left(d^H_{\geq 2}(u)-\Delta\right)}}\,.\label{eq:fixed-parameters-bounding-split-expectation-second-moment}
	\end{align}	
	Thus we can use the bound
	\begin{align}
		&\underset{\substack{q_i,\Psi_i,\ell_i,v_i\geq 0\\
				T_i\in \cT(m_i,v_i)\\
				r_i\geq m_i-1\\
				B_i\in \cG(T_i,r_i)\\
				h_i,p_i\geq 0}}{\sum} \quad \underset{H\in \cM_{u,v,\Set{\cF_i,\Psi_i}}(\emptyset, B_{z+1})}{\sum} \E \overline{U}_H(\mathbf{x}) \nonumber\\
		&\quad\leq  2n^{1/50A}
		\underset{\substack{q_i,\Psi_i,\ell_i,v_i\geq 0\\
				T_i\in \cT(m_i,v_i)\\
				r_i\geq m_i-1\\
				B_i\in \cG(T_i,r_i)\\
				h_i,p_i\geq 0}}{\sum} \Brac{\underset{u \in V(H\setminus B_i)}{\prod}\Paren{\frac{6}{\eps}}^{\max \Set{2d_1^{H(V\setminus B_i)}-\tau,0}}}\cdot \Paren{\frac{6}{\eps}}^{2\ell_i+2q_i}\cdot \Paren{\frac{\eps d}{2n}}^{\Card{E(H(V\setminus B_i))}}\cdot\Paren{\frac{\eps d}{2n}}^{q_i+\ell_i} \nonumber
		\\
		&\quad \cdot \Brac{\Paren{1+\frac{1}{\sqrt{n}}\frac{d}{n}}}^{r_i}\cdot2^{2(r_i-m_i-1)\cdot \Psi_i}\cdot \underset{u\in \cL_{\geq 2}(H(B_i))}{\prod}\Brac{\Paren{\frac{2d}{n}}^{\frac{1}{4}\left(d^H_{\geq 2}(u)-\Delta\right)}}\,.\label{eq:sum-many-B-second-moment}
	\end{align}
	By \cref{lem:sbm-counting-product-bsaws},  since $\Psi_i\leq 4t$ it follows that the contribution to  \cref{eq:sum-many-B-second-moment} of multigraphs walks with $\Psi_i\geq \Delta$ will be at least a factor $n^{\frac{1}{5}}$ smaller than the others. 
	Thus by \cref{lem:remove-overlapping-vertices-multigraphs}, \cref{lem:sbm-counting-product-bsaws} and \cref{fact:bound-unlabeled-trees} we may upper bound \cref{eq:sum-many-B-second-moment} by 
	\begin{align}
		2n^{1/50A}\cdot 2n^{2st-2}\cdot &\underset{\substack{p_i,h_i,\ell_i\geq 0\\ r_i\geq m_i-1\\ \Psi_i\leq \Delta\\ v_i\leq 2(2h_i+p_i+q_i)/s+\ell_i}}{\sum} n^{\ind{m_i\geq 2}}\cdot f_{s,t}(m_i,m_i,\cF_i,\Psi_i)\cdot g_{s,t}(m_i,\cF_i)\nonumber \\
		\cdot & \Paren{8e\cdot m_i/v_i}^{2v_i} \cdot \Paren{2m_i}^{2(r_i-m_i-1)}\cdot 2^{2(r-m_i-1)\cdot \Psi_i}\nonumber \\
		\cdot &
		\Paren{\frac{6}{\eps}}^{2\ell_i+2q_i}\cdot \Paren{\frac{\eps d}{2n}}^{2st-2h_i-p_i-\ell_i-q_i}\cdot\Paren{\frac{\eps d}{2n}}^{q_i+\ell_i} \nonumber\\
		\cdot &  \Brac{\Paren{1+\frac{1}{\sqrt{n}}}\frac{d}{n}}^{r_i}\,.\label{eq:sum-over-parameters-second-moment}
	\end{align}
	As for \cref{lem:sbm-trace-main-technical-bound},  it is easy to see that  \cref{eq:sum-over-parameters-second-moment} is a geometric sum which can be upper bounded by
	\begin{align}
		2n^{1/50A}\cdot 2\cdot \Paren{1+\frac{4}{n^{1/5}}}\cdot n^{2st-2}\cdot &\underset{\substack{h_i,p_i, \geq 0, 0\leq \ell_i\leq 4\\ \Psi_i\leq \Delta\\ v_i\leq 2(2h_i)/s+\ell_i}}{\sum} n^{\ind{m_i\geq 2}}\nonumber \\
		& \cdot f_{s,t}(m_i, m_i, \Set{\ell_i,0,p_i,m_i-1},\Psi_i)\cdot g_{s,t}( m_i,\Set{\ell_i,0,p_i,m_i-1})\nonumber \\
		\cdot & \Paren{8e\cdot m_i/v_i}^{2v_i} \nonumber \\
		\cdot &
		 \Paren{\frac{\eps d}{2n}}^{2st-2h_i-p_i} \nonumber\\
		\cdot &  \Brac{\Paren{1+\frac{1}{\sqrt{n}}}\frac{d}{n}}^{m_i-1}\,.\label{eq:convenient-upper-bound-sum}
	\end{align}
	That is,  if we restrict \cref{eq:sum-over-parameters-second-moment} to $r_i> m_i-1$, or $q_i>0$ or $\Psi_i>\Delta$ or $\ell_i\geq 5$ the contribution drops by a $n^{-1/5}$ factor. 
	So we need only to consider the settings $q_i=0, r_i=m_i-1, \Psi_i\le\Delta, \ell_i\leq 4$.
	If $\ell_i< 4$ then since each multigraph considered is a product of two block self-avoiding walks it must be that $2h_i+p_i \geq st+1$ and thus we obtain a ratio with 
	\begin{align}\label{eq:techincal-bound-multigraph-ratio-benchmark-uu}
		\underset{H\in \cM_{u,u}(\emptyset, \emptyset)}{\sum} \E \overline{U}_H(\mathbf{x})
	\end{align}
	of at most 
	\begin{align*}
		\Paren{\frac{\eps^2}{4d}}^{st/2}\cdot \Paren{200\Delta+s^{10}+C}^{2t} = \Paren{\frac{200\Delta+s^{10}+C^2}{(1+\delta)^{s/4}}}^{2t}\,,
	\end{align*}
	where $C=2^{2\frac{\log n}{t}}$.
	It remains to consider the case $\ell_i=4$. Notice that for any  multigraph $H$ in the sum, this means no edge can have multiplicity larger than $4$ and no edge can have multiplicity larger than $2$ in $\Ho$ and $\Ht$.
	Thus we obtain a ratio with \cref{eq:techincal-bound-multigraph-ratio-benchmark-uu} of at most
	\begin{align*}
		(st)^{14}\Paren{\frac{200\Delta^{10}}{(1+\delta)^{s/4}}}^{m_i/s}\,.
	\end{align*}
	Putting things together the proof follows.
\end{proof}

\paragraph{Putting things together} We are now ready to prove \cref{lem:sbm-upperbound-negligible-multigraphs}. 

\begin{proof}[Proof of \cref{lem:sbm-upperbound-negligible-multigraphs}]
	We argue that if a multigraph $H\in \Paren{\bsaw{s}{t,u}\times \bsaw{s}{t,v}}_{d_{\geq 2}^H\leq \Delta}$ is not in $\nmultig{s}{t,u,v}$, then it satisfies one of the following:
	\begin{itemize}
		\item $H\in \cD_{\geq 1,u,v}$.
		\item if $u\neq v$ either $E_1(\Ho)=\emptyset$ or $E_1(\Ht)=\emptyset$.
		\item if $u=v$ and $E_1(H)=0$.
		\item $H\in \cM_{u,v\Set{\cF_i,\Psi_i}_{i=1}^{z+1}}(\cB, B^{uv})$ for any non-empty $\cB\cup \Set{B^{uv}}$ containing some $B_i$ with a cycle, for $i\in [z+1]$.
		\item  $H\in \cM_{u,v\Set{\cF_i,\Psi_i}_{i=1}^{z+1}}(\cB, B^{uv})$  for tuples $\Set{\cF_i}_{i=1}^{z}$ such that for some $i\in [z]$, $p_i\geq 1$ or $q_i\geq 1$ or $\Psi_i>\Delta$  or $\ell_i\geq 3$.
		\item if $u\neq v$, $H\in \cM_{u,v\Set{\cF_i,\Psi_i}_{i=1}^{z+1}}(\cB, B^{uv})$ for any non-empty $\Set{B^{uv}}$.
		\item if $u=v$ and $H\in \cM_{u,v\Set{\cF_i,\Psi_i}_{i=1}^{z+1}}(\cB, B^{uv})$ for any non-empty $\Set{B^{uv}}$, then $q_{z+1}\geq 1$ or $\ell_{z+1}\neq 4$ or there exists $e\in H(B^{uv})$ such that $m_{\Ho}(e)\geq 3$ or $m_{\Ht}(e)\geq 3$.
	\end{itemize}
	Suppose this claim holds. 
	As for \cref{lem:sbm-upperbound-negligible-walks},  the inequality in \cref{lem:sbm-upperbound-negligible-multigraphs} immediately follows if
	\begin{align}
		\label{eq:sbm-s-large-enough-multigraphs}
		2\Paren{\frac{300\Delta+s^{10}+2^{\frac{\log n}{t}}}{(1+\delta)^{s/4}}}\leq \Paren{1+\delta}^{-\sqrt{s}}\,,
	\end{align} 
	which, for $t\in\Brac{\frac{\log n}{400}\,,\frac{\log n}{100}}$, may be rewritten as 
	\begin{align*}
		s\geq 
		\frac{10^8}{\delta}\Paren{\max \Set{\log 300\Delta,\log s,  1}}^2\,.
	\end{align*}
	Since by assumption \cref{eq:sbm-s-large-enough-multigraphs} is satisfied, applying \cref{lem:sbm-second-moment-main-technical-bound}, \cref{lem:remove-overlapping-vertices-multigraphs} and observing that the elements in $\underset{H \in\nbsaw{s}{t}^c}{\sum} \E \overline{U}_H(\mathbf{x})$ form a geometric sum we obtain the desired inequality.
	

	It remains to verify our claim. 
	Suppose first $u\neq v$ and consider some $H$ in $\nmultig{s}{t,u,v}^c$. If $\Vo\cap \Vt=\emptyset$ the claim follows as in  \cref{lem:sbm-upperbound-negligible-walks} and the fact that $\Ho,\Ht$ are nice block self-avoiding walks. Otherwise either 
	$E(\Ho)\cap E(\Ht)\neq \emptyset$ or $E(\Ho)\cap E(\Ht)= \emptyset$. In the former case $H\in \cM_{u,v}(\cB, B^{uv})\cup \cM_{uv}(\emptyset, B^{uv})\cup \cM_{uv}(\cB, \emptyset)$ for some $\cB\cup B^{uv}$ with $\Card{(V\setminus B, V)}\geq 4$. In the latter $H\in \cD_{\geq 1,u,v}$. 
	
	So let $u=v$, if $\Vo\cap \Vt=\Set{u}$ then again the claim follows as in  \cref{lem:sbm-upperbound-negligible-walks}. If $\Set{u}\subset \Vo\cap \Vt$ then
	$H\in \cM_{u,u\Set{\cF_i,\Psi_i}_{i=1}^{z+1}}(\cB, B^{uu})\cup \cM_{u,u\Set{\cF_i,\Psi_i}}(\emptyset, B^{uu})$ for some tuples of parameters $\Set{\cF_i,\Psi_i}_{i=1}^{z+1}$ and non-empty $B^{uu}$.
	We may assume that $\cB$ is a collection of trees such that for each $B_i\in \cB$, $\ell_i=2,p_i=q_i=0$ and $\Psi_i\leq \Delta$ as otherwise the claim holds.
	Similarly we may assume $B^{uu}$ is a tree with $\ell_{z+1}= 4, q_{z+1}=0$ and $\Psi_{z+1}\leq \Delta$.
	If $H(B^{uu})$ is connected to $E_1(H)$ by more than two vertices then $\ell_{z+1}\geq 5$ or $q_{z+1}\geq 1$  or there exists $e\in E(B^{uu})$ such that $\max\Set{m_{\Ho}(e)m m_{\Ht}}\geq 3$ contradicting our assumption. The remaining cases are trivial.
\end{proof}

\subsubsection{Bounding the variance of non-negligible multigraphs}\label{sec:bounding-variance-multigraphs}
By \cref{lem:sbm-upperbound-negligible-multigraphs}, to obtain \cref{thm:sbm-bound-second-moment-large-t} it remains to bound the variance of nice multigraphs. As for nice block self-avoiding walks, we split nice multigraphs in different sets.

\begin{definition}
For $u,v\in [n]$, define the set $\nmultig{s}{t,u,v,m,z}$ to be the set of nice multigraphs in which $E_{\geq 2}(H)$ is a forest on $m$ vertices and $z$ components. Further define 
\begin{equation*}
	\begin{split}
		\nmultig{s}{t,u,v}^* :=& \left\{ H  \in \nmultig{s}{t,u,v}\suchthat \Ho \in \nbsaw{s}{t,m_1,z_1} \right. \\
		&\left. \qquad \Ht\in \nbsaw{s}{t,m_2,z_2}\right. \\
		&\left. \qquad	st-2m_1-2z_1\geq t/\sqrt{A}\,, st-2m_2-2z_2\geq t/\sqrt{A} \vphantom{\nbsaw{s}{t,m,z}} \phantom{i^2}\right\}\,,\\
		\nmultig{s}{t,u,v, d^H\leq \Delta} :=& \Set{H  \in \nmultig{s}{t,u,v} \suchthat \max_{v\in V(H)} d^H(v)\leq \Delta }\,.
	\end{split}
\end{equation*}
That is, $\nmultig{s}{t,u,v}^*$ denotes the set of nice multigraphs in which both decomposition block self-avoiding walks have at least $t/\sqrt{A}$ edges with multiplicity $1$. 
We also define
\begin{align*}
	\nmultig{s}{t,u,v, d^H\leq \Delta}^* &:= \nmultig{s}{t,u,v d^H\leq \Delta}\cap 	\nmultig{s}{t,u,v}^*\,,\\
	\nmultig{s}{t,u,v, d^H>\Delta}^* &:= \nmultig{s}{t,u,v}^*\setminus \nmultig{s}{t,u,v d^H\leq \Delta}\,,\\
	\nmultig{s}{t,u,v, d^H\leq \Delta,m,z} &:= \nmultig{s}{t,u,v d^H\leq \Delta}\cap 	\nmultig{s}{t,u,v,m,z}\,.
\end{align*}
\end{definition}

\paragraph{Variance of the products of walks with different first pivots} We consider first the settings $u, v\in [n], u\neq v$.
We provide an upper bound on the expectation of nice multigraphs in $\nmultig{s,t}{u,u}^*$.

\begin{lemma}\label{lem:variance-nice-multigraphs-different-pivot}
	Consider the settings of \cref{thm:sbm-bound-second-moment-large-t}. Let $u,v \in [n]$ and	let $m\geq 0$ be an integer. 
	Then for any $H\in \nmultig{s,t}{u,v}^*$ with decomposition block self-avoiding walks $\Ho,\Ht$
	\begin{align*}
		\E \Brac{\mathbf{Y}_H}\leq (1+o(1))\E \Brac{\mathbf{Y}_{\Ho}} \cdot E \Brac{\mathbf{Y}_{\Ht}}\,.
	\end{align*}
	Moreover for any $H\in \nmultig{s}{t,u,v,d^H>\Delta}^*$
	\begin{align*}
		\Abs{\E \Brac{\mathbf{Y}_H}}\leq \frac{4}{n^{1/13}}\E \overline{U}_{\Ho}(\mathbf{x})\E \overline{U}_{\Ht}(\mathbf{x})\,.
	\end{align*}
\end{lemma}

We prove \cref{lem:variance-nice-multigraphs-different-pivot} in \cref{sec:bounds-products-bsaws}. 
The first inequality  of \cref{thm:sbm-bound-second-moment-large-t} follows directly.
\begin{lemma}
	Consider the settings of \cref{thm:sbm-bound-second-moment-large-t}. Let $u,v \in [n]\,, u\neq v$. Then  
	\begin{align*}
		\E \Brac{\Paren{\Q(\overline{\mathbf{Y}})^t}_{uu}\Paren{\Q(\overline{\mathbf{Y}})^t}_{vv}}\leq (1+o(1))	\E \Brac{\Paren{\Q(\overline{\mathbf{Y}})^t}_{uu}}^2\,.
	\end{align*}
	\begin{proof}
		By  \cref{lem:sbm-upperbound-negligible-multigraphs} 
		\begin{align*}
			\E \Brac{\Paren{\Q(\overline{\mathbf{Y}})^t}_{uu}\Paren{\Q(\overline{\mathbf{Y}})^t}_{vv}}
			&\leq \underset{H \in \nmultig{s,t}{u,v}}{\sum} \Brac{\E \Brac{\mathbf{Y}_H}+ \Paren{n^{-1/6}+(1+\delta)^{t\sqrt{s}}}U_H(\mathbf{x}) }\\
			&\leq \underset{H \in \nmultig{s,t}{u,v, d^H\leq \Delta}}{\sum} \Brac{\E \Brac{\mathbf{Y}_H}+ \Paren{n^{-1/6}+(1+\delta)^{t\sqrt{s}}}U_H(\mathbf{x}) } \\
			&\quad +  \underset{H \in \nmultig{s,t}{u,v, d^H>\Delta}}{\sum}\E \Brac{U_H(\mathbf{x}) }\cdot \Paren{1+n^{-1/6}+(1+\delta)^{t\sqrt{s}}}\,.
		\end{align*}
		By \cref{lem:sbm-upper-bound-nice-bsaws-large-degree}, \cref{lem:count-nice-walks-large-degree} and \cref{lem:variance-nice-multigraphs-different-pivot}
		\begin{align*}
			\underset{H \in \nmultig{s,t}{u,v, d^H>\Delta}}{\sum}\E \Brac{U_H(\mathbf{x}) }
			&\leq \frac{1}{n^{1/13}}\underset{H \in \nmultig{s,t}{u,v, d^H>\Delta}}{\sum}\E \Brac{U_{\Ho}(\mathbf{x}) }\E \Brac{U_{\Ht}(\mathbf{x}) }\\
			&\leq \frac{2^{2t}}{n^{1/13}}\underset{m,z\geq 1}{\sum}\quad \underset{H \in \nmultig{s,t}{u,v, d^H>\Delta,m,z}}{\sum}\E \Brac{U_{\Ho}(\mathbf{x}) }\E \Brac{U_{\Ht}(\mathbf{x}) }\\
			&\leq \frac{1}{n^{1/14}} \underset{m,z\geq 0}{\sum}\quad \underset{H \in \nmultig{s,t}{u,v, d^H\leq \Delta}}{\sum}\E \Brac{U_{\Ho}(\mathbf{x}) }\E \Brac{U_{\Ht}(\mathbf{x}) }\,.
		\end{align*}
		On the other hand  by \cref{lem:sbm-lower-bound-nice-bsaws}, \cref{lem:sbm-upper-bound-nice-bsaws-few-edges},  \cref{lem:count-nice-walks-few-edges-multiplicity-1} and \cref{lem:variance-nice-multigraphs-different-pivot}
		\begin{align*}
			 &\underset{H \in \nmultig{s,t}{u,v, d^H\leq \Delta}\setminus \nmultig{s}{t,u,v,d^H\leq \Delta}^*}{\sum}\E \Brac{\overline{U}_H(\mathbf{x})}\\
			 &\leq \underset{H \in \nmultig{s,t}{u,v, d^H\leq \Delta}\setminus \nmultig{s,t}{u,v, d^H\leq \Delta}^*}{\sum}\E \Brac{U_{\Ho}(\mathbf{x}) }\E \Brac{U_{\Ht}(\mathbf{x}) }\\
			 &\leq \Brac{(st)\cdot 2^{10t}\cdot s^t\cdot (1+\delta)^{-st/2}}^2 \underset{H \in \nmultig{s,t}{u,v, d^H\leq \Delta}^*}{\sum} \E \Brac{U_{\Ho}(\mathbf{x}) }\E \Brac{U_{\Ht}(\mathbf{x}) }\\
			 &\leq o(1)\underset{H \in \nmultig{s,t}{u,v, d^H\leq \Delta}^*}{\sum} L_{\Ho}L_{\Ht}\\
			 &\leq o(1)\underset{H \in \nmultig{s,t}{u,v, d^H\leq \Delta}^*}{\sum}\E \Brac{\overline{\mathbf{Y}}_{\Ho}}\E \Brac{\overline{\mathbf{Y}}_{\Ht}}\,.
		\end{align*}
		All in all we get
		\begin{align*}
			&\E \Brac{\Paren{\Q(\overline{\mathbf{Y}})^t}_{uu}\Paren{\Q(\overline{\mathbf{Y}})^t}_{vv}}\leq (1+o(1))\underset{H \in \nmultig{s,t}{u,v,d^H\leq \Delta}^*}{\sum} \E \Brac{\mathbf{Y}_{\Ho}}\E \Brac{\mathbf{Y}_{\Ht}}\\
			&\leq \Paren{1+o(1)}\underset{H \in \nmultig{s,t}{u,v, d^H\leq \Delta}^*}{\sum} \E \Brac{\mathbf{Y}_{\Ho}}\E \Brac{\mathbf{Y}_{\Ht}}\\
			&\leq  \Paren{1+o(1)}\underset{H \in \bsaw{s,t}{u}\times \bsaw{s,t}{v}}{\sum} \E \Brac{\mathbf{Y}_{\Ho}}\E \Brac{\mathbf{Y}_{\Ht}}
		\end{align*}
		where the last step follows by the fact that by \cref{lem:sbm-upperbound-negligible-walks} and  \cref{lem:sbm-contribution-nice-walks} 
		\begin{align*}
			\underset{\substack{m,z\geq 0\\ st-2m-2z\geq t/\sqrt{A}}}{\sum}\underset{H \in \nbsaw{s}{t,u, d^H\leq \Delta, m,z}}{\sum} \E  \Brac{\mathbf{Y}_{H}}\geq (1-o(1))	\underset{H \in \bsaw{s}{t,u}}{\sum} \E  \Brac{\mathbf{Y}_{H}}\,.
		\end{align*}
		 as for \cref{thm:truncation-schatten-norm-lower-bound}.
	\end{proof}
\end{lemma}

\paragraph{Concentration of diagonal entries}  It remains to prove the second inequality of \cref{thm:sbm-bound-second-moment-large-t}. Let $u\in [n]$.
Consider the sets
\begin{align*}
	\nmultig{s,t}{u,u,\ell}&:= \Set{H \in \nmultig{s}{t,u,u}\suchthat \Card{E(\Ho)\cap E(\Ht)}=\ell}\\
	\nmultig{s,t}{u,u}^{**}&:= \Set{H\in \nmultig{s}{t,u,u}^*\suchthat \Card{E_{1}(H)}\geq t/A}	\,.
\end{align*}

\begin{lemma}\label{lem:variance-nice-multigraphs-same-pivot}
	Consider the settings of \cref{thm:sbm-bound-second-moment-large-t}. Let $u,v \in [n]$ and	let $m\geq 0$ be an integer.  
	Then for any $H\in \nmultig{s,t}{u,u,\ell}^{**}$ with decomposition block self-avoiding walks $\Ho,\Ht$.
	If $\Card{E_1(\Ho)\cap E_{\geq 2}(\Ht)}+\Card{E_1(\Ht)\cap E_{\geq 2}(\Ho)}\leq \frac{\log n}{12\log \Paren{\frac{2}{\eps}}}$  and  $\max_{v\in V(H)}d^H(v)\leq \Delta$
	\begin{align*}
		\E \Brac{\mathbf{Y}_H}\leq \Paren{1+\Paren{1+\delta}^{-s}}\cdot  \Paren{1+\frac{\delta}{10}}^{-\ell}\cdot\E \Brac{\mathbf{Y}_{\Ho}}  E \Brac{\mathbf{Y}_{\Ht}}\,.
	\end{align*}
	If $\Card{E_1(\Ho)\cap E_{\geq 2}(\Ht)}+\Card{E_1(\Ht)\cap E_{\geq 2}(\Ho)}> \frac{\log n}{12\log \Paren{\frac{2}{\eps}}}$  or  $\max_{v\in V(H)}d^H(v)> \Delta$
	\begin{align*}
		\Abs{\E \Brac{\mathbf{Y}_H}}\leq \frac{1}{n^{1/13}}\E \Brac{\overline{U}_{\Ho}(\mathbf{x})} \E \Brac{\overline{U}_{\Ht}(\mathbf{x})}\,.
	\end{align*}
\end{lemma}
We prove \cref{lem:variance-nice-multigraphs-same-pivot} in \cref{sec:bounds-products-bsaws}. 
We introduce an additional counting argument.

\begin{lemma}
	\label{lem:second-moment-counting-nice-multigraphs}
	Consider the settings of \cref{thm:sbm-bound-second-moment-large-t}. Let $u\in [n]$ and 
	Then for $\ell\neq st$
	\begin{align*}
		\Card{\nmultig{s,t}{u,u, \ell}}\leq \frac{\ell^{10}}{n^\ell}	\Card{\nmultig{s,t}{u,u, 0}}\,.
	\end{align*}
	For $\ell=st$
	\begin{align*}
		\Card{\nmultig{s,t}{u,u, \ell}}\leq \frac{4}{n^{\ell-1}}	\Card{\nmultig{s,t}{u,u, 0}}\,.
	\end{align*}	
\end{lemma}

We are now ready to prove the inequality.

\begin{lemma}
	Consider the settings of \cref{thm:sbm-bound-second-moment-large-t}. Let $u \in [n]$. Then  
	\begin{align*}
		\E \Brac{\Paren{\Q(\overline{\mathbf{Y}})^t}_{uu}\Paren{\Q(\overline{\mathbf{Y}})^t}_{uu}}\leq C\cdot 	\E \Brac{\Paren{\Q(\overline{\mathbf{Y}})^t}_{uu}}^2\,,
	\end{align*}
	where $C$ is a universal constant.
	\begin{proof}
		By  \cref{lem:sbm-upperbound-negligible-multigraphs} 
		\begin{align*}
			\E \Brac{\Paren{\Q(\overline{\mathbf{Y}})^t}_{uu}\Paren{\Q(\overline{\mathbf{Y}})^t}_{uu}}
			&\leq \underset{H \in \bsaw{s}{t,u}\times \bsaw{s}{t,u}}{\sum} \E \Brac{\mathbf{Y}_H}\\
			&\leq \underset{H \in \nmultig{s,t}{u,u}}{\sum} \Brac{\E \Brac{\mathbf{Y}_H}+ \Paren{n^{-1/6}+(1+\delta)^{t\sqrt{s}}}U_H(\mathbf{x}) }\,.
			%
		\end{align*}
		For any $\ell\leq st$, we split the set $\nmultig{s,t}{u,u}$ in
		\begin{align*}
			S_{1,\ell}&:=\nmultig{s,t}{u,u, \ell}^{**}\setminus \nmultig{s,t}{u,u,  d^H\leq \Delta}\\
			S_{2,\ell}&:=\nmultig{s,t}{u,u,\ell}\setminus \nmultig{s,t,}{u,u, \ell}^*\\
			S_{3,\ell}&:=\nmultig{s}{t,u,u, \ell}^*\setminus \nmultig{s}{t,u,u}^{**} \\
			S_{4,\ell}&:= \nmultig{s}{t,u,u,\ell, d^H\leq \Delta}^{**}\,.
		\end{align*}
		Notice that $S_{1,\ell}\cup S_{2,\ell}\cup S_{3,\ell}\cup S_{4,\ell}=\nmultig{s}{t,u,u}$. We bound each term separately.
		Consider the $S_{1,\ell}$  by \cref{lem:sbm-lower-bound-nice-bsaws}, \cref{lem:second-moment-counting-nice-multigraphs} and \cref{lem:variance-nice-multigraphs-same-pivot}
		\begin{align*}
			\underset{\ell \leq st}{\sum}\quad\underset{H\in S_{1,\ell}}{\sum}\E \Brac{\overline{U}_H(\mathbf{x}) }
			&\leq 
			\frac{1}{n^{1/13}}\underset{\ell \leq st}{\sum}\frac{\Paren{1+\delta}^{-\ell}}{n^{1/50A}}n^{\ell}\underset{H\in S_{1,\ell}}{\sum}\E \Brac{U_{\Ho}(\mathbf{x}) }\E \Brac{U_{\Ht}(\mathbf{x}) }\\
			&\leq \frac{1}{n^{1/14}}\underset{H\in S_{1,0}}{\sum }L_{\Ho}L_{\Ht}\\
			&\leq \frac{1}{n^{1/14}}\underset{H \in S_{1,0}}{\sum}\E \Brac{\mathbf{Y}_{\Ho} }\E \Brac{\mathbf{Y}_{\Ht} }\,.
		\end{align*}
		Next we consider  $S_{2,\ell}$. Notice that if $H\in S_{2,\ell}$ then $E_{1}(H)\leq 2t/{\sqrt{A}}$
		by \cref{lem:count-nice-walks-few-edges-multiplicity-1} and \cref{lem:second-moment-counting-nice-multigraphs}
		\begin{align*}
			\underset{\ell \leq st}{\sum}\quad \underset{H \in S_{2,\ell}}{\sum}\E \Brac{\overline{U}_H(\mathbf{x}) }
			&\leq \underset{\ell \leq st}{\sum}\frac{\Paren{1+\delta}^{-\ell}}{n^{1/50A}} n^{\ell}\underset{H \in S_{2,\ell}}{\sum}\E \Brac{\overline{U}_{\Ho}(\mathbf{x}) }\E \Brac{\overline{U}_{\Ht}(\mathbf{x}) }\\
			&\leq \frac{O(1)}{n^{1/50A}}\underset{H \in S_{2,0}}{\sum}\E \Brac{\overline{U}_{\Ho}(\mathbf{x}) }\E \Brac{\overline{U}_{\Ht}(\mathbf{x}) }\\
			&\leq \frac{O(1)}{n^{1/50A}}\cdot (1+\delta)^{-t\sqrt{s}/2}\underset{H\in \nmultig{s}{t,u,u,0,0}}{\sum}	\E \Brac{\overline{U}_{\Ho}(\mathbf{x}) }\E \Brac{\overline{U}_{\Ht}(\mathbf{x}) }\\	
			&\leq o(1)\underset{H \in \nmultig{s}{t,u,u,0,0}}{\sum} L_{\Ho}L_{\Ht}\\
			&\leq o(1)\underset{H \in \nmultig{s}{t,u,u,0,0}}{\sum} \E \Brac{\mathbf{Y}_{\Ho}}\E \Brac{\mathbf{Y}_{\Ht}}\,.
		\end{align*}
		Notice now that $S_{3,\ell}$ is empty unless $\ell \geq 9t/\sqrt{A}$. Moreover by choice of $A$, it holds $\Paren{1+\delta}^{-9t/\sqrt{A}}n^{10/A}\leq \Paren{1+\delta}^{-t/\sqrt{A}}$.
		It follows that by \cref{lem:second-moment-counting-nice-multigraphs} and  \cref{lem:sbm-lower-bound-nice-bsaws}
		\begin{align*}
			\underset{\ell \leq st}{\sum}\quad \underset{H \in S_{3,\ell}}{\sum}\E \Brac{\overline{U}_H(\mathbf{x}) }
			&\leq \underset{\ell \leq st}{\sum}\frac{(1+\delta)^{-\ell}}{n^{1/50A}}n^{\ell} \underset{H \in S_{3,\ell}}{\sum}\E \Brac{\overline{U}_{\Ho}(\mathbf{x}) }\E \Brac{\overline{U}_{\Ht}(\mathbf{x}) }\\
			&\leq o(1)\underset{H \in \nmultig{s}{t,u,u,0,0}}{\sum} \E \Brac{\mathbf{Y}_{\Ho}}\E \Brac{\mathbf{Y}_{\Ht}}\,.
		\end{align*}
		It remains to bound the contribution of nice multigraphs in $S_{4,\ell}$. We first consider the case $\ell<st$. By \cref{lem:second-moment-counting-nice-multigraphs}
		\begin{align*}
			\underset{\ell < st}{\sum}\quad \underset{H \in S_{4,\ell}}{\sum}\E \Brac{\mathbf{Y}_{H} } &\leq  \Paren{1+(1+\delta)^{-s}}\underset{\ell < st}{\sum}\quad \underset{H \in S_{4,\ell}}{\sum}\Paren{1+\frac{\delta}{10}}^{-\ell}n^{\ell}\E \Brac{\mathbf{Y}_{\Ho} } \E \Brac{\mathbf{Y}_{\Ht} } \\
			&\leq C \underset{H \in S_{4,0}}{\sum}\E \Brac{\mathbf{Y}_{\Ho} } \E \Brac{\mathbf{Y}_{\Ht} }
		\end{align*}
		where $C>1$ is a universal constant.
		For $\ell= st$, similarly we get
		\begin{align*}
			\underset{H \in S_{4,st}}{\sum}\E \Brac{\mathbf{Y}_{H} }  &\leq (1+(1+\delta)^{-s})\underset{H \in S_{4,st}}{\sum} \Paren{1+\frac{\delta}{10}}^{-st}n^{st}\E \Brac{\mathbf{Y}_{\Ho} } \E \Brac{\mathbf{Y}_{\Ht} }\\
			&\leq  n\cdot (st)^{10}\cdot \Paren{1+\frac{\delta}{10}}^{-st}\underset{H\in S_{4,0}}{\sum} \E \Brac{\mathbf{Y}_{\Ho} } \E \Brac{\mathbf{Y}_{\Ht} }\\
			&\leq o(1)\underset{H\in S_{4,0}}{\sum} \E \Brac{\mathbf{Y}_{\Ho} } \E \Brac{\mathbf{Y}_{\Ht} }\,.
		\end{align*}
		Putting things together
		\begin{align*}
				\E \Brac{\Paren{\Q(\overline{\mathbf{Y}})^t}_{uu}\Paren{\Q(\overline{\mathbf{Y}})^t}_{uu}}&\leq C'\cdot \E \Brac{\Paren{\Q(\overline{\mathbf{Y}})^t}_{uu}}^2
		\end{align*}
		for a universal constant $C'>1$.
	\end{proof}
\end{lemma}


\phantomsection
\addcontentsline{toc}{section}{References}
\bibliographystyle{amsalpha}
\bibliography{bib/mathreview,bib/dblp,bib/custom,bib/scholar}

\appendix


\section*{Organization of appendices}
The appendices are organized as follows.
In \cref{sec:appendix-lower-bound-non-centered} we introduce our techniques for computing the truncated expectation of block self-avoiding walks. and prove \cref{lem:sbm-upper-bound-any-bsaw}, \cref{lem:sbm-lower-bound-nice-bsaws}, \cref{lem:sbm-upper-bound-nice-bsaws-large-degree}, \cref{lem:variance-nice-multigraphs-different-pivot} and \cref{lem:variance-nice-multigraphs-same-pivot}.

In \cref{sec:appendix-upper-bound-centered} we study the expectation of polynomials arising in the computation of the centered Schatten norm and prove \cref{lem:sbm-centered-upper-bound-any-bsaw}.
Most of the technical details of \cref{sec:appendix-lower-bound-non-centered} and \cref{sec:appendix-upper-bound-centered}  can be found in \cref{sec:technical-lemmas-trace-bounds}.
In \cref{sec:splitting-expectations} we prove \Cref{fact:sbm-upperbound-removing-cycles}, \Cref{fact:splitting-upper-bound} and  \cref{fact:splitting-upper-bound-second-moment}.
The counting arguments required in \cref{sec:bounds-moments-Q} are proved in \cref{sec:counting-bsaw}. Specifically, we prove here \cref{lem:sbm-encoding-of-multigraph-with-B},  \cref{lem:count-nice-walks-large-degree}, \cref{lem:count-nice-walks-few-edges-multiplicity-1}, \cref{lem:sbm-counting-product-bsaws}, \cref{lem:second-moment-counting-nice-multigraphs}, \cref{lem:nice-bsaw-upper-bound-m-z} and \cref{lem:counting-special-nbsaw-centered}.
Finally, \cref{sec:appendix-additional-tools} contains additional basic tools that are used throughout the paper.

To help the reader navigate the sections we provide here a table where the various proofs of each lemma can be found.
\begin{center}
	\begin{tabular}{ |m{20em}|m{20em}| } 
		\hline
		\multicolumn{2}{|c|}{Statement A \quad is implied by  \quad Statement  B}\\
		\hline
		\hline
		\cref{lem:sbm-upper-bound-any-bsaw}& \cref{lem:lem_UH} \\ \hline
		\cref{fact:sbm-upperbound-removing-cycles} & \cref{fact:sbm-upperbound-removing-cycles-appendix}\\ \hline
		\Cref{fact:splitting-upper-bound} & \Cref{fact:splitting-upper-bound-appendix}\\ \hline
		\cref{lem:sbm-encoding-of-multigraph-with-B} & \cref{lem:sbm-encoding-of-H-B-equivalent} \\ \hline
		\cref{lem:sbm-lower-bound-nice-bsaws}  & \cref{lem:lem_upper_bound_nice_Wc}\\ \hline
		\cref{lem:sbm-upper-bound-nice-bsaws-large-degree} &  \cref{lem:lem_upper_bound_nice_Wc} \\ \hline
		\cref{lem:count-nice-walks-large-degree} & \cref{lem:count-nice-walks-large-degree-appendix} \\ \hline
		\cref{lem:count-nice-walks-few-edges-multiplicity-1} & \cref{lem:count-nice-walks-few-edges-multiplicity-1-appendix} \\ \hline
		\cref{lem:sbm-centered-upper-bound-any-bsaw} &  \cref{lem:lem_upper_bound_UhatH} \\ \hline
		\cref{lem:nice-bsaw-upper-bound-m-z} & \cref{lem:nice-bsaw-upper-bound} \\ \hline
		\cref{lem:counting-special-nbsaw-centered} & \cref{lem:counting-special-nbsaw-centered-appendix}  \\ \hline
		\cref{lem:sbm-counting-product-bsaws} & \cref{lem:sbm-counting-product-bsaws-appendix} \\ \hline
		\cref{fact:splitting-upper-bound-second-moment} & \cref{fact:splitting-upper-bound-second-moment-appendix}\\ \hline
		\cref{lem:variance-nice-multigraphs-different-pivot} &  \cref{lem:lem_E_Y_H_Product_nbsaw_disjoint} \\ \hline
		\cref{lem:variance-nice-multigraphs-same-pivot} &  \cref{lem:lem_E_Y_H_Product_nbsaw_sharing} \\ \hline
		 \cref{lem:second-moment-counting-nice-multigraphs}  &  \cref{lem:second-moment-counting-nice-multigraphs-appendix} \\ \hline
	\end{tabular}
\end{center}

\section{Bounds for the non-centered matrix}\label{sec:appendix-lower-bound-non-centered}
\renewcommand{\nbsaw}[2]{{\text{NBSAW}_{#1,#2}^*}}
\subsection{Useful notation}
\label{subsec:subsec_notation_truncation}
Let $\Delta$ be the constant truncation threshold (which we will specify later), and let $\overline{\mathbf{G}}$ be the graph obtained from $\mathbf{G}\sim\SBM_n(d,\e)$ by deleting every vertex that has a degree strictly greater than $\Delta$ in $\mathbf{G}$.

For every $v\in\mathbf{G}$, let $\mathcal{E}_v$ be the event that $v$ is truncated, i.e., $$\mathcal{E}_v:=\big\{v\notin\overline{\mathbf{G}}\big\}=\{d_{\mathbf{G}}(v)>\Delta\}.$$

For every $V\subseteq [n]=V(\mathbf{G})$, we say that $V$ is truncated if there exists at least one vertex in $V$ that is truncated. Let $\mathcal{E}_V$ be the event that $V$ is truncated, i.e., $$\displaystyle\mathcal{E}_V:=\bigcup_{v\in V}\mathcal{E}_v=\big\{V\subsetneq V\big(\overline{\mathbf{G}}\big)\big\}.$$ Similarly, for every multigraph $H$ with $V(H)\subseteq [n]$, we say that $H$ is truncated if there exists at least one vertex in $H$ that is truncated. Let $\mathcal{E}_H$ be the event that $H$ is truncated, i.e.,  $$\displaystyle\mathcal{E}_H:=\mathcal{E}_{V(H)}=\big\{V(H)\subsetneq V\big(\overline{\mathbf{G}}\big)\big\}.$$

Recall that for every $i,j\in [n]$, we have 
$$\mathbf{Y}_{ij}=\begin{cases}1-\frac{d}{n}\quad&\text{if }ij\in\mathbf{G},\\-\frac{d}{n}\quad&\text{otherwise}.\end{cases}$$

We define $\overline{\mathbf{Y}}_{ij}\in\left\{-\frac{d}{n},0,1-\frac{d}{n}\right\}$ as follows:
\begin{align*}
\overline{\mathbf{Y}}_{ij}&=\mathbf{Y}_{ij}\cdot \mathbbm{1}_{\mathcal{E}_i^c}\cdot \mathbbm{1}_{\mathcal{E}_j^c}=\mathbf{Y}_{ij}\cdot \mathbbm{1}_{\mathcal{E}_i^c\cap \mathcal{E}_j^c}=\mathbf{Y}_{ij}\cdot \mathbbm{1}_{(\mathcal{E}_i\cup\mathcal{E}_j)^c}.
\end{align*}

For every multigraph $H$, we write $m_H(ij)$ to denote the multiplicity in $H$ of an edge $ij\in E(H)$. Define

\begin{align*}
\overline{\mathbf{Y}}_H&=\prod_{ij\in E(H)}\overline{\mathbf{Y}}_{ij}^{m_H(ij)}=\prod_{ij\in E(H)}\left(\mathbf{Y}_{ij}^{m_H(ij)}\cdot\mathbbm{1}_{(\mathcal{E}_i\cup\mathcal{E}_j)^c}\right)\\
&=\left(\prod_{ij\in E(H)}\mathbf{Y}_{ij}^{m_H(ij)}\right)\cdot\mathbbm{1}_{\bigcap_{ij\in E(H)}(\mathcal{E}_i\cup\mathcal{E}_j)^c}=\mathbf{Y}_H\cdot\mathbbm{1}_{\left(\bigcup_{ij\in E(H)}(\mathcal{E}_i\cup\mathcal{E}_j)\right)^c}\\
&=\mathbf{Y}_H\cdot\mathbbm{1}_{\mathcal{E}_{V(H)}^c}=\mathbf{Y}_H\cdot\mathbbm{1}_{\mathcal{E}_{H}^c}=\begin{cases}
\mathbf{Y}_H\quad&\text{if }H\text{ is truncated,}\\
0\quad&\text{otherwise}.
\end{cases}
\end{align*}

\subsection{An upper bound for every multigraph}

\label{subsec:subsec_upper_bound_H}

In this section, we derive an upper bound $\overline{U}_H(\mathbf{x})$ on $\big|\mathbb{E}\big[\overline{\mathbf{Y}}_H\big|\mathbf{x}\big]\big|$ for every multigraph $H$ with at most $st$ vertices and at most $st$ multi-edges, where $t=K\cdot\log n$.

\subsubsection{Informal discussion and useful definitions}
\label{subsubsec:InformalDiscussion}
Before delving into the details of the upper bound $\overline{U}_H(\mathbf{x})$ and the lower bound $\overline{L}_H$ that we mentioned in \Cref{sec:proof-strategy}, let us first explain a few key observations that provide some intuition, and which clarifies our proof strategy. This section provides a road map for the proof of the bounds, and introduces notations and concepts that will be useful later.

Let $H$ be an arbitrary multigraph. Observe that in the non-truncated case, the computation of $\mathbb{E}[\mathbf{Y}_H|\mathbf{x}]$ is easy because the edges are conditionally mutually independent given $\mathbf{x}$. More precisely,
$$\mathbb{E}[\mathbf{Y}_H|\mathbf{x}]=\prod_{ij\in E(H)}\mathbb{E}\big[\mathbf{Y}_{ij}^{m_H(ij)}\big|\mathbf{x}\big].$$

Unfortunately, this is not the case for the truncated case: Since the presence of one edge can lead to the truncation of an incident edge, the random variables $\big(\overline{\mathbf{Y}}_{ij}\big)_{ij\in E(H)}$ are not conditionally mutually independent given $\mathbf{x}$. The dependencies between $\big(\overline{\mathbf{Y}}_{ij}\big)_{ij\in E(H)}$ makes the exact computation of $\mathbb{E}\big[\overline{\mathbf{Y}}_H\big|\mathbf{x}\big]$ very complicated.

Another level of complication is that, unlike the non-truncated case, $\big(\overline{\mathbf{Y}}_{ij}\big)_{ij\in E(H)}$ are not conditionally independent of $\mathbf{G}-G(H)$ given $\mathbf{x}$. More precisely, $\big(\overline{\mathbf{Y}}_{ij}\big)_{ij\in E(H)}$ depend on $\mathbf{G}-G(H)$ through $\big(d_{\mathbf{G}-G(H)}(v)\big)_{v\in V(H)}$, where
$$d_{\mathbf{G}-G(H)}(v)=\big|\big\{e\in\mathbf{G}:\;e\notin E(H)\text{ and }e\text{ is incident to }v\big\}\big|.$$
Furthermore, $\big(d_{\mathbf{G}-G(H)}(v)\big)_{v\in V(H)}$ are not conditionally mutually independent given $\mathbf{x}$ because of edges in $\big\{ij: i,j\in V(H)\text{ and }ij\notin E(H)\big\}$. We call such edges \emph{$H$-cross-edges}. Fortunately, if $V(H)$ contains at most $st$ vertices, we have very few $H$-cross-edges and their effect will be negligible. More precisely, with high probability, no $H$-cross-edge will be present in $\mathbf{G}$, and if we condition on this event and on $\mathbf{x}$, the random variables $\big(d_{\mathbf{G}-G(H)}(v)\big)_{v\in V(H)}$ become conditionally mutually independent. This will be made precise later.

Another phenomenon that we should be aware of, is the effect of truncation on $\mathbb{E}\big[\overline{\mathbf{Y}}_H\big|\mathbf{x},\big(d_{\mathbf{G}-G(H)}(v))_{v\in V(H)}\big]$. In order to explain this, we will divide the edges of $G(H)$ into two categories:
\begin{itemize}
\item Edges of multiplicity 1 in $H$.
\item Edges of multiplicity at least 2 in $H$.
\end{itemize}

Before discussing the effect of truncation on these two categories of edges, let us quickly mention a remark regarding some tempting approaches to upper and lower bound $\mathbb{E}\big[\overline{\mathbf{Y}}_H\big|\mathbf{x}\big]$, and which we did not find very successful.

\begin{remark}
Since $\overline{\mathbf{Y}}_H=\mathbf{Y}_H\cdot \mathbbm{1}_{\mathcal{E}_H^c}$, we have $\big|\overline{\mathbf{Y}}_H\big|=|\mathbf{Y}_H|\cdot \mathbbm{1}_{\mathcal{E}_H^c}\leq |\mathbf{Y}_H|$. Therefore,
$$\big|\mathbb{E}\big[\overline{\mathbf{Y}}_H\big|\mathbf{x}\big]\big|\leq \mathbb{E}\big[|\overline{\mathbf{Y}}_H|\big|\mathbf{x}\big]\leq \mathbb{E}\big[|\mathbf{Y}_H|\big|\mathbf{x}\big].$$
We might hope to get a good upper bound for $\big|\mathbb{E}\big[\overline{\mathbf{Y}}_H\big|\mathbf{x}\big]\big|$ using the above inequality since $\displaystyle \mathbb{E}\big[|\mathbf{Y}_H|\big|\mathbf{x}\big]= \prod_{ij\in E(H)}\mathbb{E}\big[|\mathbf{Y}_{ij}^{m_H(ij)}|\big|\mathbf{x}\big]$ is easy to compute. However, $\mathbb{E}\big[|\mathbf{Y}_H|\big|\mathbf{x}\big]$ can be too large compared to $\mathbb{E}\big[\mathbf{Y}_H\big|\mathbf{x}\big]$. This is mainly because of edges of multiplicity 1 in $H$: If $ij\in E(H)$ is of multiplicity 1 in $H$, then
$$|\mathbb{E}[\mathbf{Y}_{ij}|\mathbf{x}]|=\frac{\epsilon d}{2n},$$
while
$$\mathbb{E}\big[|\mathbf{Y}_{ij}|\big|\mathbf{x}\big]=(1+o(1))\left(2+\frac{\epsilon \mathbf{x}_i\mathbf{x}_j}{2}\right)\frac{d}{n}.$$

Combining this observation with the fact that $E(H)$ can contain $\Theta(\log n)$ edges of multiplicity 1, it becomes evident that in general, $\mathbb{E}\big[|\mathbf{Y}_H|\big|\mathbf{x}\big]$ is not a good upper bound on $\big|\mathbb{E}\big[\overline{\mathbf{Y}}_H\big|\mathbf{x}\big]\big|$.

Other in-the-same-spirit approaches attempt to separately analyze the cases when $\overline{\mathbf{Y}}_H$ is positive or negative, i.e., write $\overline{\mathbf{Y}}_H=\overline{\mathbf{Y}}_H^+-\overline{\mathbf{Y}}_H^-$, where $\overline{\mathbf{Y}}_H^+=\max\{0,\overline{\mathbf{Y}}_H\}$ and $\overline{\mathbf{Y}}_H^-=\max\{0,-\overline{\mathbf{Y}}_H\}$. One might hope to obtain good bounds on $\mathbb{E}\big[\overline{\mathbf{Y}}_H\big|\mathbf{x}\big]=\mathbb{E}\big[\overline{\mathbf{Y}}_H^+\big|\mathbf{x}\big]-\mathbb{E}\big[\overline{\mathbf{Y}}_H^-\big|\mathbf{x}\big]$ by obtaining good upper and lower bounds on $\mathbb{E}\big[\overline{\mathbf{Y}}_H^+\big|\mathbf{x}\big]$ and $\mathbb{E}\big[\overline{\mathbf{Y}}_H^-\big|\mathbf{x}\big]$, respectively. However, it turns out that $\mathbb{E}\big[\overline{\mathbf{Y}}_H^+\big|\mathbf{x}\big]$ and $\mathbb{E}\big[\overline{\mathbf{Y}}_H^-\big|\mathbf{x}\big]$ can be too large compared to $\mathbb{E}\big[\overline{\mathbf{Y}}_H\big|\mathbf{x}\big]$. This is essentially for the same reason why $\mathbb{E}\big[|\mathbf{Y}_H|\big|\mathbf{x}\big]$ can be too large compared to $\big|\mathbb{E}\big[\mathbf{Y}_H\big|\mathbf{x}\big]\big|$. Therefore, in order for this approach to succeed in deriving good bounds for $\mathbb{E}\big[\overline{\mathbf{Y}}_H\big|\mathbf{x}\big]$, the bounds on $\mathbb{E}\big[\overline{\mathbf{Y}}_H^+\big|\mathbf{x}\big]$ and $\mathbb{E}\big[\overline{\mathbf{Y}}_H^-\big|\mathbf{x}\big]$ need to be extremely tight, and this is very hard to obtain for a general $H$.
\end{remark}

\vspace*{3mm}

\noindent {\bf Effect of truncation on edges of multiplicity 1}

\vspace*{2mm}

For edges of multiplicity 1, we observe that truncation can make the absolute value of $\mathbb{E}\big[\overline{\mathbf{Y}}_H\big|\mathbf{x},\big(d_{\mathbf{G}-G(H)}(v)\big)_{v\in V(H)}\big]$ grow too much. To illustrate this point, take the example where $H$ is a cycle with $st$ edges of multiplicity 1. For the non-truncated case, since $\big(d_{\mathbf{G}-G(H)}(v)\big)_{v\in V(H)}$ is conditionally independent from $\mathbf{Y}_H$ given $\mathbf{x}$, if we condition on the event that $d_{\mathbf{G}-G(H)}(v)=\Delta$ for every $v\in V(H)$, we get
$$\mathbb{E}\big[\mathbf{Y}_H\big|\mathbf{x},\big\{d_{\mathbf{G}-G(H)}(v)=\Delta,\forall v\in V(H)\big\}\big]=\prod_{uv\in E(H)}\frac{\epsilon d\mathbf{x}_u\mathbf{x}_v}{2n}=\left(\frac{\epsilon d}{2n}\right)^{st}.$$
On the other hand, for the truncated case, if we condition on the event that $d_{\mathbf{G}-G(H)}(v)=\Delta$ for every $v\in V(H)$, then $\overline{\mathbf{Y}}_H$ is non-zero if and only if all the edges in $E(H)$ are not present in $\mathbf{G}$, hence
\begin{align*}
\mathbb{E}\big[\overline{\mathbf{Y}}_H\big|\mathbf{x},\big\{d_{\mathbf{G}-G(H)}(v)=\Delta,\forall v\in V(H)\big\}\big]&=\prod_{ij\in E(H)}\left(1-\left(1+\frac{\epsilon \mathbf{x}_i\mathbf{x}_j}{2}\right)\frac{d}{n}\right)\left(-\frac{d}{n}\right)\\
&=(1\pm o(1))\left(-\frac{d}{n}\right)^{st}.
\end{align*}

Therefore, the absolute value was multiplied by a factor of approximately $\left(\frac{2}{\epsilon}\right)^{st}$, which can be too large. Fortunately, this can be mitigated by choosing $\Delta$ to be large enough so that problematic events such as $\big\{d_{\mathbf{G}-G(H)}(v)=\Delta,\forall v\in V(H)\big\}$ will have a very small probability.

Roughly speaking, the contribution of an edge of multiplicity 1 in the non-truncated case is $|\mathbb{E}[\mathbf{Y}_{ij}|\mathbf{x}]|=\frac{\epsilon d}{2n}$, and its contribution in the truncated case can be as large as $\frac{d}{n}$. Since we are trying to make our upper bound as tight as possible --- i.e., as low as possible --- we would like to find situations where edges of multiplicity 1 behave similarly to the non-truncated case, because their contribution will be relatively small.

This brings us to the following observation: Assume that $ij\in E(H)$ is an edge of multiplicity 1 such that $d_{\mathbf{G}-G(H)}(i)+d_{G(H)}(i)\leq \Delta$ and $d_{\mathbf{G}-G(H)}(j)+d_{G(H)}(j)\leq \Delta$, where $d_{G(H)}(i)$ denotes the degree of $i$ in $H$ without counting multiplicities, i.e., $d_{G(H)}(i)$ is the degree of $i$ in the underlying graph $G(H)$. In this case, we have
$$d_{\mathbf{G}}(i)=d_{\mathbf{G}-G(H)}(i)+d_{\mathbf{G}\cap G(H)}(i)\leq d_{\mathbf{G}-G(H)}(i)+d_{G(H)}(i)\leq \Delta.$$
Similarly, we have $d_{\mathbf{G}}(j)\leq \Delta$. Therefore, no matter which edges of $E(H)$ are present in $\mathbf{G}$ and which are absent, we are sure that neither $i$ nor $j$ will be deleted, and so $\overline{\mathbf{Y}}_{ij}=\mathbf{Y}_{ij}$. Furthermore, it is easy to see that given $\mathbf{x}$, and given that $d_{\mathbf{G}-G(H)}(i)+d_{G(H)}(i)\leq \Delta$ and $d_{\mathbf{G}-G(H)}(j)+d_{G(H)}(j)\leq \Delta$, the random variable $\overline{\mathbf{Y}}_{ij}=\mathbf{Y}_{ij}$ is conditionally independent from $\overline{\mathbf{Y}}_{H-ij}$. This motivates the following definition:
\begin{definition}
\label{def:def_safe_vertices}
Let $H$ be a multigraph. For every vertex $v\in V(H)$, if $d_{\mathbf{G}-G(H)}(v)+d_{G(H)}(v)\leq \Delta$, we say that $v$ is $(\mathbf{G},H)$-\emph{safe}. We say that $v$ is $(\mathbf{G},H)$-\emph{unsafe} if it is not $(\mathbf{G},H)$-\emph{safe}.

We say that a subset $S$ of $V(H)$ is \emph{completely} $(\mathbf{G},H)$-\emph{safe} if all the vertices in it are $(\mathbf{G},H)$-\emph{safe}. Similarly, we say that it is \emph{completely} $(\mathbf{G},H)$-\emph{unsafe} if all the vertices in it are $(\mathbf{G},H)$-\emph{unsafe}.

An edge $ij\in E(H)$ is said to be $(\mathbf{G},H)$-\emph{safe} if both $i$ and $j$ are safe.

If $\mathbf{G}$ and $H$ are clear from the context, we drop $(\mathbf{G},H)$ and simply write safe, completely safe, unsafe, and completely unsafe.
\end{definition}

In summary, we have the following observations:
\begin{itemize}
\item If a vertex $v\in V(H)$ is safe, we are sure that it will not be deleted. On the other hand, if $v$ is unsafe, it may or may not be deleted.
\item If $ij\in E(H)$ is safe, then $\overline{\mathbf{Y}}_{ij}=\mathbf{Y}_{ij}$ and $\overline{\mathbf{Y}}_{ij}$ is conditionally independent of $\overline{\mathbf{Y}}_{H-ij}$ given $\mathbf{x}$. On the other hand, if $ij$ is unsafe, then $\overline{\mathbf{Y}}_{ij}$ may or may not be equal to $\mathbf{Y}_{ij}$.
\end{itemize}

If $ij\in E(H)$ is an edge of multiplicity 1 in $H$, then we have
\begin{align*}
\mathbb{E}\big[\overline{\mathbf{Y}}_H\big|\mathbf{x},\big\{ij\text{ is safe}\big\}\big]&=\mathbb{E}\big[\overline{\mathbf{Y}}_{ij}\big|\mathbf{x},\big\{ij\text{ is safe}\big\}\big]\cdot\mathbb{E}\big[\overline{\mathbf{Y}}_{H-ij}\big|\mathbf{x},\big\{ij\text{ is safe}\big\}\big]\\
&=\mathbb{E}\big[\mathbf{Y}_{ij}\big|\mathbf{x},\big\{ij\text{ is safe}\big\}\big]\cdot\mathbb{E}\big[\overline{\mathbf{Y}}_{H-ij}\big|\mathbf{x},\big\{ij\text{ is safe}\big\}\big]\\
&=\mathbb{E}[\mathbf{Y}_{ij}|\mathbf{x}]\cdot \mathbb{E}\big[\overline{\mathbf{Y}}_{H-ij}\big|\mathbf{x},\big\{ij\text{ is safe}\big\}\big]\\
&=\frac{\epsilon \mathbf{x}_i\mathbf{x}_j d}{2n}\cdot \mathbb{E}\big[\overline{\mathbf{Y}}_{H-ij}\big|\mathbf{x},\big\{ij\text{ is safe}\big\}\big].
\end{align*}

As we can see from the above discussion, edges of multiplicity 1 that are safe behave similarly to the non-truncated case. In order to benefit from this phenomenon as much as possible, we would like the probability that an edge $ij$ is unsafe to be small. Equivalently, we would like the probability that $i$ (respectively $j$) is unsafe to be small. By analyzing the event $\big\{d_{\mathbf{G}-G(H)}(i)+d_{G(H)}(i)> \Delta\big\}$, we notice the following:
\begin{itemize}
\item The distribution of the random variable $d_{\mathbf{G}-G(H)}(i)$ is $\text{Binomial}\left(n-O(st),\frac{d}{n}\right)$,\footnote{Note that here we are not conditioning on $\mathbf{x}$.} so we can approximate it by a  $\text{Poisson}(d)$ distribution.
\item The probability of the event $\big\{d_{\mathbf{G}-G(H)}(i)+d_{G(H)}(i)> \Delta\big\}$ is lower bounded by the probability that $d_{\mathbf{G}-G(H)}(i)> \Delta$. Now since $\Delta$ must be a constant\footnote{Recall that the main motivation behind truncation is to get rid of the effect of large degree vertices, and so $\Delta$ must be a constant and cannot diverge with $n$. Nevertheless, we can tune $\Delta$ and make it as large as we wish, as long as it stays polynomial in $s$ and $d$, and polylogarithmic in $\epsilon$.} that should not scale with $n$, it is easy to see that the probability that $i$ is unsafe is lower bounded by a strictly positive constant\footnote{This is because the probability that $\text{Poisson}(d)>\Delta$ is constant.} and cannot be made go to zero as $n$ goes to infinity. Nevertheless, the larger the value of $\Delta-d_{G(H)}(i)$ is, the smaller is the probability that $i$ is unsafe. In particular, if $\frac{\Delta}{2}$ is very large with respect to $d$ and $d_{G(H)}(i)\leq\frac{\Delta}{2}$, then the probability that $i$ is unsafe is at most the probability that $d_{\mathbf{G}-G(H)}(i)> \frac{\Delta}{2}$, which can be approximated by the probability that a $\text{Poisson}(d)$ random variable is greater than $\frac{\Delta}2$, which in turn is small as long as $\frac{\Delta}{2}$ is large with respect to $d$.
\end{itemize}
We conclude that vertices satisfying $d_{G(H)}(i)\leq\frac{\Delta}{2}$ are "well-behaved", and edges whose end-vertices are both of this type are also "well-behaved" in the sense that with large probability they will be safe and will behave similarly to the non-truncated case.

For edges $ij\in E(H)$ of multiplicity 1 in $H$ that are not well-behaved, i.e., $d_{G(H)}(i)>\frac{\Delta}{2}$ or $d_{G(H)}(j)>\frac{\Delta}{2}$, we will use a very loose\footnote{Essentially, we will upper bound $\big|\overline{\mathbf{Y}}_{ij}\big|$ by $|\mathbf{Y}_{ij}|$.} upper-bound to estimate the contribution of $\overline{\mathbf{Y}}_{ij}$. Nevertheless, we will mitigate the effect of such edges depending on what caused them to be not well-behaved. In principal, a vertex $v\in V(H)$ can have a large degree in $H$, i.e., $d_{G(H)}(v)>\frac{\Delta}{2}$, either because it has many edges of multiplicity 1 in $H$ which are incident to it, or because it has many edges of multiplicity at least 2 in $H$ which are incident to it. We will treat each case differently. This motivates the following definition:

\begin{definition}
Let $H$ be an arbitrary multigraph. For every $v\in V(H)$, we define the following:
\begin{itemize}
\item The 1-degree of $v$ in $H$, denoted $d^H_{1}(v)$ is the number of edges in $E(H)$ that are incident to $v$, and which have multiplicity 1 in $H$.
\item An edge in $E(H)$ is said to be of multiplicity $\geq\hspace*{-1.2mm}2$ in $H$ if it has multiplicity at least 2 in $H$. The $(\geq\hspace*{-1.2mm}2)$-degree of $v$ in $H$, denoted $d^H_{\geq 2}(v)$ is the number of edges in $E(H)$ that are incident to $v$, and which have multiplicity $\geq\hspace*{-1.2mm}2$ in $H$.
\end{itemize}
Clearly, $d_{G(H)}(v)=d^H_{1}(v)+d^H_{\geq 2}(v)$ is the degree of $v$ in the underlying graph $G(H)$, i.e., the degree of $v$ in $H$ without counting multiplicities.
\end{definition}

Roughly speaking, we will mitigate the effects of "not-well-behaved" edges using the following ideas:
\begin{itemize}
\item We will show that there are very few block self-avoiding-walks that have vertices with large 1-degree, and this will counter-act the loose upper bounds that we use for the edges of multiplicity 1 that are incident to such vertices.
\item Assuming that $H$ is a multigraph that does not contain any vertex with large 1-degree, we will show that there are very few vertices with large $(\geq\hspace*{-1.2mm}2)$-degree. Therefore, the number of edges of multiplicity 1  that are incident to such vertices is not large, and so the aggregate effect of the loose upper bounds that we use for such edges is not too severe.
\end{itemize}

\vspace*{3mm}

\noindent {\bf Effect of truncation on edges of multiplicity $\geq\hspace*{-1.2mm}2$}

\vspace*{2mm}

Compared to edges of multiplicity 1 in $H$, edges of multiplicity $\geq\hspace*{-1.2mm}2$ in $H$ are much easier to treat: Unlike edges of multiplicity 1 in $H$, for an edge $ij$ of multiplicity $m_H(ij)\geq 2$, the conditional expectations $\mathbb{E}\big[\mathbf{Y}_{ij}^{m_H(ij)}\big|\mathbf{x}\big]$ and $\mathbb{E}\big[|\mathbf{Y}_{ij}^{m_H(ij)}|\big|\mathbf{x}\big]$ are approximately equal. Therefore, for many cases of interest, $\big|\mathbf{Y}_{ij}^{m_H(ij)}\big|$ is a good upper bound on $\overline{\mathbf{Y}}_{ij}^{m_H(ij)}$.

If $m_H(ij)\geq 2$, we have the following:
\begin{itemize}
\item If $ij\in \mathbf{G}$, then
$$\big|\mathbf{Y}_{ij}^{m_H(ij)}\cdot \mathbb{P}[ij\in\mathbf{G}]\big|=\left(1-\frac{d}{n}\right)^{m_H(ij)}\cdot\left(1+\frac{\epsilon \mathbf{x}_i\mathbf{x}_j}{2}\right)\frac{d}{n}=\Omega\left(\frac{1}{n}\right).$$
\item If $ij\notin \mathbf{G}$, then
$$\big|\mathbf{Y}_{ij}^{m_H(ij)}\cdot\mathbb{P}[ij\notin\mathbf{G}]\big|=\left|\left(-\frac{d}{n}\right)^{m_H(ij)}\cdot\left(1-\left(1+\frac{\epsilon \mathbf{x}_i\mathbf{x}_j}{2}\right)\frac{d}{n}\right)\right|=O\left(\frac{1}{n^{m_H(ij)}}\right).$$
\end{itemize}
As we can see, the contribution of the case when $ij\in\mathbf{G}$ dominates the contribution of the case when $ij\notin\mathbf{G}$. Therefore, when we want to compute $\mathbb{E}\big[\mathbf{Y}_{ij}^{m_H(ij)}\big|\mathbf{x}\big]$, or $\mathbb{E}\big[|\mathbf{Y}_{ij}^{m_H(ij)}|\big|\mathbf{x}\big]$, we can discard the case for which $ij\notin\mathbf{G}$. This is essentially the main reason why $\mathbb{E}\big[\mathbf{Y}_{ij}^{m_H(ij)}\big|\mathbf{x}\big]$ and $\mathbb{E}\big[|\mathbf{Y}_{ij}^{m_H(ij)}|\big|\mathbf{x}\big]$ are approximately equal.

A similar phenomenon occurs for the truncated case, except that the contribution of the case for which $ij\in\mathbf{G}$ might be multiplied by zero due to truncation. Therefore, roughly speaking, when we want to compute $\mathbb{E}\big[\overline{\mathbf{Y}}_{ij}^{m_H(ij)}\big|\mathbf{x}\big]$, it is sufficient to consider the cases where the number of edges of multiplicity $\geq\hspace*{-1.2mm}2$ in $H$ that are present in $\mathbf{G}$ is as large as possible.\footnote{To be more precise, we should give priority to edges with larger multiplicities. For example, if we have a multigraph where the multiplicities are 2, 3 and 4, we first try to include as many edges of multiplicity 4 as possible without causing truncation, then, if there is still room to include edges of multiplicity 3, we include as many of them as possible before turning to edges of multiplicity 2.}

Now notice that if there is a vertex $v$ with $d^H_{\geq 2}(v)>\Delta$, it is impossible to include all the edges of multiplicity $\geq\hspace*{-1.2mm}2$ in $H$ that are incident to $v$ without causing truncation. Therefore, roughly speaking, $\big|\mathbb{E}\big[\overline{\mathbf{Y}}_H\big|\mathbf{x}\big]\big|$ will be at least $O\left(\frac{1}{n}\right)$ smaller compared to $|\mathbb{E}[\mathbf{Y}_H|\mathbf{x}]|$. This will allow us to show that the aggregate contribution of all multigraphs having at least one vertex $v$ satisfying $d^H_{\geq 2}(v)>\Delta$ is negligible.\footnote{This is the main reason why the algorithm that is based on $Q^{(s)}\big(\overline{\mathbf{Y}}\big)$ works whereas the one based on $Q^{(s)}(\mathbf{Y})$ does not work: When we write $\mathbb{E}[\Tr((Q^{(s)}(\mathbf{Y})-\mathbf{x}\transpose{\mathbf{x}})^t)]$ as the sum of contributions of block self-avoiding-walks, we find that there are too many block self-avoiding-walks where $d^H_{\geq 2}(v)>\Delta$ for some vertex $v\in V(H)$. This makes the value of $\|Q^{(s)}(\mathbf{Y})-\mathbf{x}\transpose{\mathbf{x}}\|_t$ too large. On the other hand, for the truncated case, even though we have too many such block self-avoiding-walks, the contribution of each one of them is very small.}

\subsubsection{The truncation threshold}

We choose the truncation threshold to be of the form
\begin{equation}
\label{eq:eq_Delta_form}
\Delta=\max\left\{128e^4d^4,40Asd, 2\log(2As)+ 12A\tau s^2\cdot \log 2 + 8A^2\tau^2s^2 \left(\log\frac{6}{\epsilon}\right)^2\right\},
\end{equation}
where 
\begin{equation}
\label{eq:eq_deg1_Threshold}
\tau=As \log\frac{6}{\epsilon},
\end{equation}
and $A> \max\{1,100K,\frac{100}{K}\}$ is some constant to be chosen later. This form of $\Delta$ was carefully chosen to allow for a proof of all the phenomena that were mentioned in \Cref{subsubsec:InformalDiscussion}. Before starting to prove the upper bound $\overline{U}_H(\mathbf{x})$ on $\big|\mathbb{E}\big[\overline{\mathbf{Y}}_H\big|\mathbf{x}\big]\big|$, we need to introduce a few definitions.

\begin{definition}
\label{def:def_deg1_classification}
Let $H$ be a multigraph. We classify the vertices $v\in V(H)$ according to their degree-$1$ as follows:
\begin{itemize}
\item If $d^H_{1}(v)\leq \tau$, we say that $v$ is $1$-\emph{small} in $H$. We denote the set of $1$-small vertices in $H$ as $\mathcal{S}_1(H)$.
\item If $d^H_{1}(v)>\tau$, we say that $v$ is $1$-\emph{large} in $H$. We denote the set of $1$-large vertices in $H$ as $\mathcal{L}_1(H)$.
\end{itemize}
\end{definition}

\begin{definition}
\label{def:def_deg2_classification}
Let $H$ be a multigraph. We classify the vertices $v\in V(H)$ according to their $(\geq\hspace*{-1.2mm}2)$-degree as follows:
\begin{itemize}
\item If $d^H_{\geq 2}(v)\leq \frac{\Delta}{4}$, we say that $v$ is $(\geq\hspace*{-1.2mm}2)$-\emph{small} in $H$. We denote the set of $(\geq\hspace*{-1.2mm}2)$-small vertices in $H$ as $\mathcal{S}_{\geq 2}(H)$.
\item If $\frac{\Delta}{4}<d^H_{\geq 2}(v)\leq \Delta$, we say that $v$ is $(\geq\hspace*{-1.2mm}2)$-\emph{intermediate} in $H$. We denote the set of $(\geq\hspace*{-1.2mm}2)$-intermediate vertices in $H$ as $\mathcal{I}_{\geq2}(H)$.
\item If $d^H_{\geq 2}(v)>\Delta$, we say that $v$ is $(\geq\hspace*{-1.2mm}2)$-\emph{large} in $H$. We denote the set of $(\geq\hspace*{-1.2mm}2)$-large vertices in $H$ as $\mathcal{L}_{\geq2}(H)$.
\end{itemize}
\end{definition}

\begin{definition}
\label{def:def_E1_E2_annoying}
Let $H$ be a multigraph. We denote the set of edges in $G(H)$ of multiplicity 1 in $H$ as $E_1(H)$, and denote the set of edges in $G(H)$ of multiplicity $\geq\hspace*{-1.2mm}2$ in $H$ as $E_{\geq 2}(H)$.

An edge in $E_1(H)$ is said to be \emph{annoying} if it is incident to at least one vertex in $\mathcal{L}_1(H)$. We denote the set of annoying edges as $E_1^a(H)$.

We partition $E_{\geq 2}(H)$ into two sets:
$$E_{\geq2}^a(H)=\big\{uv\in E_{\geq 2}(H): u\notin \mathcal{L}_{\geq2}(H)\text{ and }v\notin \mathcal{L}_{\geq2}(H)\big\},$$
and
$$E_{\geq2}^b(H)=\big\{uv\in E_{\geq 2}(H): u\in \mathcal{L}_{\geq2}(H)\text{ or }v\in \mathcal{L}_{\geq2}(H)\big\}.$$
\end{definition}

\begin{definition}
\label{def:def_upper_bound_UH}
For every multigraph $H$ with at most $st=sK\log n$ vertices and at most $st$ multi-edges, we define the quantity
$$\overline{U}_H(\mathbf{x})=\resizebox{0.87\textwidth}{!}{$\displaystyle n^{\frac{2K}{A}}\left(\frac{6}{\epsilon}\right)^{|E_1^a(H)|}\frac{1}{\resizebox{0.07\textwidth}{!}{$\displaystyle\prod_{v\in\mathcal{L}_{\geq2}(H)}$} n^{\frac{1}{4}\left(d^H_{\geq 2}(v)-\Delta\right)}}\left(\frac{\epsilon d}{2n}\right)^{|E_1(H)|}\left(\frac{d}{n}\right)^{|\E_{\geq2}(H)|}\prod_{uv\in E_{\geq2}^a(H)}\left[1+\frac{\epsilon \mathbf{x}_u\mathbf{x}_v}{2}+\frac{3d}{\sqrt{n}}\right].$}$$
\end{definition}

In the rest of \Cref{subsec:subsec_upper_bound_H}, we show that $\overline{U}_H(\mathbf{x})$ is an upper bound on $\big|\mathbb{E}\big[\overline{\mathbf{Y}}_H\big|\mathbf{x}\big]\big|$ for $n$ large enough, assuming that $H$ is a multigraph with at most $st=sK\log n$ vertices and at most $st$ multi-edges.


Note that with more refined calculations, it is possible to get a better upper bound. In any case, $\overline{U}_H(\mathbf{x})$ is good enough for our purposes.

\subsubsection{Analyzing edges of multiplicity 1}

We start by proving an upper bound on the probability that a subset of $\mathcal{S}_1(H)\cap\mathcal{S}_{\geq 2}(H)$ is completely unsafe. We first need to introduce some notation:
\begin{definition}
\label{def:def_in_out_degree}
An edge $ij$ satisfying $i,j\in V(H)$ and $ij\notin E(H)$ is called an $H$-\emph{cross-edge}.

For every $v\in V(H)$, define the following:
\begin{align*}
d_{\mathbf{G}-G(H)}(v)&=\big|\big\{e\in\mathbf{G}:\;e\text{ is incident to }v\text{ and }e\notin E(H)\big\}\big|\\
&=\big|\big\{uv\in\mathbf{G}:\;uv\notin E(H)\big\}\big|,
\end{align*}
$$d_{\mathbf{G}-G(H)}^i(v)=\big|\big\{uv\in\mathbf{G}:\;uv\notin E(H)\text{ and }u\in V(H)\big\}\big|,$$
and
$$d_{\mathbf{G}-G(H)}^o(v)=\big|\big\{uv\in\mathbf{G}:\;uv\notin E(H)\text{ and }u\notin V(H)\big\}\big|.$$
Clearly, $d_{\mathbf{G}-G(H)}(v)=d_{\mathbf{G}-G(H)}^i(v)+d_{\mathbf{G}-G(H)}^o(v)$.
\end{definition}

$d_{\mathbf{G}-G(H)}^o(v)$ can be thought of as the "$V(H)$-outside degree" of $v$ in $\mathbf{G}-G(H)$, i.e., the number of edges in $\mathbf{G}-G(H)$ that go from $v$ to the outside of $V(H)$. On the other hand, $d_{\mathbf{G}-G(H)}^i(v)$ can be thought of as the "$V(H)$-inside degree" of $v$ in $\mathbf{G}-G(H)$, i.e., the number of edges in $\mathbf{G}-G(H)$ that are incident to $v$ and are inside $V(H)$.

If a vertex $v\in V(H)$ is unsafe, then either $d_{\mathbf{G}-G(H)}^i(v)$ is large or $d_{\mathbf{G}-G(H)}^o(v)$ is large. It turns out that with high probability, we have\footnote{This is essentially because we have at most $s^2t^2=o(n)$ $H$-cross-edges, and the probability of any particular one of them being present in $\mathbf{G}$ is $O\left(\frac{1}{n}\right)$.} $d_{\mathbf{G}-G(H)}^i(v)=0$. Therefore, the probability that $v$ is unsafe is dominated by the probability that $d_{\mathbf{G}-G(H)}^o(v)> \Delta- d_{G(H)}(v)$.

Notice that if $v\in \mathcal{S}_1(H)\cap\mathcal{S}_{\geq 2}(H)$, then
$$d_{G(H)}(v)=d_1^H(v)+d_{\geq 2}^H(v)\leq \tau + \frac{\Delta}{4}\leq \frac{\Delta}{4} + \frac{\Delta}{4}= \frac{\Delta}{2}.$$

The following lemma derives an upper bound on the probability that the outside degree is larger than $\frac{\Delta}{2}$.

\begin{lemma}
\label{lem:lem_bound_prob_large_outside_deg}
Let $H$ be a multigraph such that $|V(H)|\leq st$, where $t=K\log n$. For every $v\in\mathcal{S}_1(H)\cap\mathcal{S}_{\geq 2}(H)$, we have
$$\mathbb{P}\left[d_{\mathbf{G}-G(H)}^o(v)>\frac{\Delta}{4}\middle| \mathbf{x}\right]\leq \frac{\eta}{2},$$
where\footnote{In the calculations for $\mathbb{E}\left[\Tr\left(Q^{(s)}\big(\overline{\mathbf{Y}}\big)^t\right)\right]$, we will need $\eta\leq\frac{1}{As}\left(\frac{\epsilon}{6}\right)^{\tau}$. On the other hand, in the calculations for $\mathbb{E}\left[\Tr\left(\big(Q^{(s)}\big(\overline{\mathbf{Y}}\big)-\mathbf{x}\transpose{\mathbf{x}}\big)^t\right)\right]$, we will need $\eta\leq \frac{1}{s\cdot 2^{6A\tau s^2}}\left(\frac{\epsilon}{6}\right)^{4s^2\tau^2}$.}
\begin{equation}
\label{eq:eq_def_eta}
\eta:=\frac{1}{As\cdot 2^{6A\tau s^2}}\left(\frac{\epsilon}{6}\right)^{4s^2\tau^2}\leq\frac{1}{As}\left(\frac{\epsilon}{6}\right)^{\tau}.
\end{equation}
\end{lemma}
\begin{proof}
The proof is based on simple calculations based on the probability distribution of the sum of Bernoulli random variables. The detailed proof can be found in \Cref{subsubsec:subsubsec_lem_bound_prob_large_outside_deg}.
\end{proof}

In the following, we will show that $d_{\mathbf{G}-G(H)}^i(v)=0$ with high probability. The following definition will be useful.

\begin{definition}
\label{def:def_crossing}
Let $H$ be a multigraph. We say that a vertex $v\in V(H)$ is \emph{$H$-cross-free in $\mathbf{G}$} if there is no $H$-cross-edge that is present in $\mathbf{G}$, and which is incident to $v$. In other words, $v$ is $H$-cross-free in $\mathbf{G}$ if $d_{\mathbf{G}-G(H)}^i(v)=0$. We say that $v\in V(H)$ is \emph{$H$-crossing in $\mathbf{G}$} if it is not $H$-cross-free in $\mathbf{G}$.

A subset of $V(H)$ is said to be \emph{completely $H$-cross-free in $\mathbf{G}$} if all the vertices in it are $H$-cross-free in $\mathbf{G}$. We say that it is \emph{completely $H$-crossing in $\mathbf{G}$} if all the vertices in it are $H$-crossing in $\mathbf{G}$.

If $\mathbf{G}$ and $H$ are clear from the context, we drop $\mathbf{G}$ and $H$ and simply write cross-free, crossing, completely cross-free, and completely crossing.
\end{definition}

The following lemma shows that the probability that a nonempty subset of $V(H)$ is completely crossing is small. Furthermore, the larger the set, the smaller is the probability. This essentially means that we have $d_{\mathbf{G}-G(H)}^i(v)=0$ with high probability.

\begin{lemma}
\label{lem:lem_bound_prob_completely_crossing}
Let $H$ be a multigraph such that $|V(H)|\leq st$, where $t=K\cdot\log n$. If $n$ is large enough, then for every $S\subseteq V(H)$, the conditional probability given $\mathbf{x}$ that $S$ is completely $H$-crossing in $\mathbf{G}$ can be upper bounded by:
$$\mathbb{P}\left[\left\{S\text{ is completely }H\text{-crossing in }\mathbf{G}\right\}\middle|\mathbf{x}\right]\leq \left(\frac{2d s^2 t^2}{n}\right)^{|S|/2}.$$
\end{lemma}
\begin{proof}
The proof is based on a simple application of the union bound. The detailed proof can be found in \Cref{subsubsec:subsubsec_lem_bound_prob_large_outside_deg}.
\end{proof}

By combining \Cref{lem:lem_bound_prob_large_outside_deg} and \Cref{lem:lem_bound_prob_completely_crossing} and using the fact that an unsafe vertex $v\in V(H)$ is either crossing or satisfies $d_{\mathbf{G}-G(H)}^o(v)>\frac{\Delta}{2}$, we can show the following upper bound on the probability that a subset of $\mathcal{S}_1(H)\cap \mathcal{S}_{\geq 2}(H)$ is completely unsafe:

\begin{lemma}
\label{lem:lem_bound_prob_completely_unsafe}
Recall the definition of safe vertices in \Cref{def:def_safe_vertices}, and let $H$ be a multigraph such that $|V(H)|\leq st$, where $t=K\cdot\log n$. If $n$ is large enough, then for every $V\subseteq \mathcal{S}_1(H)\cap\mathcal{S}_{\geq 2}(H)$, the conditional probability that $V$ is completely unsafe given $\mathbf{x}$ can be upper bounded by:
$$\mathbb{P}\left[\left\{V\text{ is completely }(\mathbf{G},H)\text{-unsafe}\right\}\middle|\mathbf{x}\right]\leq \eta^{|V|},$$
where $\eta$ is as in \cref{eq:eq_def_eta}.
\end{lemma}
\begin{proof}
The detailed proof can be found in \Cref{subsubsec:subsubsec_lem_bound_prob_large_outside_deg}.
\end{proof}

Since $\eta$ is very small for large $A$, \Cref{lem:lem_bound_prob_completely_unsafe} implies that it is unlikely for unsafe edges to occur. On the other hand, the following lemma implies that safe edges behave similarly to the truncated case. These two observations are what ultimately makes it possible to derive a good upper bound on $\big|\mathbb{E}\big[\overline{\mathbf{Y}}_H\big|\mathbf{x}\big]\big|$.

\begin{lemma}
\label{lem:lem_expectation_safe_edge}
Let $H$ be a multigraph and let $uv\in E(H)$. Let $\mathcal{E}$ be an event satisfying:
\begin{itemize}
\item $\mathcal{E}$ implies that $uv$ is $(\mathbf{G},H)$-safe, i.e., $\forall (G,x)\in\mathcal{E}$, $u$ and $v$ are $(G,H)$-safe.
\item The event $\mathcal{E}$ depends only on $\mathbf{x}$ and $\mathbf{G}-uv$, i.e., $\mathcal{E}$ is $\sigma(\mathbf{x},\mathbf{G}-uv)$-measurable. In other words, if we condition on $\mathbf{x}$, then $\mathcal{E}$ depends only $(\mathbbm{1}_{u'v'\in\mathbf{G}})_{u',v'\in [n]:\; u'v'\neq uv}$. This implies that given $\mathbf{x}$, the event $\mathcal{E}$ is conditionally independent from $\mathbbm{1}_{\{uv\in\mathbf{G}\}}$.
\end{itemize}
Then,
\begin{equation}
\label{eq:eq_lem_expectation_safe_edge_1}
\forall (G,x)\in\mathcal{E},\;\overline{\mathbf{Y}}_{uv}(G)=\mathbf{Y}_{uv}(G),
\end{equation}
and
 \begin{equation}
\label{eq:eq_lem_expectation_safe_edge_2}
\mathbb{E}\big[\overline{\mathbf{Y}}_{uv}^{m_H(uv)}\big|\mathbf{x},\mathcal{E}\big]=\mathbb{E}\big[\mathbf{Y}_{uv}^{m_H(uv)}\big|\mathbf{x}\big].
\end{equation}
Furthermore, if $H'$ is a submultigraph of $H$ containing $uv$ with the same multiplicity as in $H$, then we have
 \begin{equation}
\label{eq:eq_lem_expectation_safe_edge_3}
\mathbb{E}\big[\overline{\mathbf{Y}}_{H'}\big|\mathbf{x},\mathcal{E}\big]=\mathbb{E}\big[\mathbf{Y}_{uv}^{m_H(uv)}\big|\mathbf{x}\big]\cdot \mathbb{E}\big[\overline{\mathbf{Y}}_{H'-uv}\big|\mathbf{x},\mathcal{E}\big].
\end{equation}
\end{lemma}
\begin{proof}
Equations \cref{eq:eq_lem_expectation_safe_edge_1} and \cref{eq:eq_lem_expectation_safe_edge_2} are trivial. In order to get Equation \cref{eq:eq_lem_expectation_safe_edge_3}, observe that we have
\begin{align*}
\mathbb{E}\big[\overline{\mathbf{Y}}_{H'}\big|\mathbf{x},\mathcal{E}\big]&=\mathbb{E}\big[\overline{\mathbf{Y}}_{uv}^{m_H(uv)}\cdot\overline{\mathbf{Y}}_{H'-e}\big|\mathbf{x},\mathcal{E}\big]\\
&=\mathbb{E}\big[\mathbf{Y}_{uv}^{m_H(uv)}\cdot \overline{\mathbf{Y}}_{H'-e}\big|\mathbf{x},\mathcal{E}\big]\\
&\stackrel{(\ast)}{=}\mathbb{E}\big[\mathbf{Y}_{uv}^{m_H(uv)}\big|\mathbf{x},\mathcal{E}\big]\cdot \mathbb{E}\big[\overline{\mathbf{Y}}_{H'-{uv}}\big|\mathbf{x},\mathcal{E}\big]\\
&\stackrel{(\dagger)}{=}\mathbb{E}\big[\mathbf{Y}_{uv}^{m_H(uv)}\big|\mathbf{x}\big]\cdot \mathbb{E}\big[\overline{\mathbf{Y}}_{H'-uv}\big|\mathbf{x},\mathcal{E}\big],
 \end{align*}
 where $(\ast)$ follows from the fact that $\mathcal{E}$ is $\sigma(\mathbf{x},\mathbf{G}-uv)$-measurable, which implies that given $\mathbf{x}$ and $\mathcal{E}$, we have that $\mathbf{Y}_{uv}=\mathbbm{1}_{\{uv\in\mathbf{G}\}}-\frac{d}{n}$ is conditionally independent from $\left(\overline{\mathbf{Y}}_{u'v'}\right)_{u'v'\in E(H'-uv)}$, and $(\dagger)$ follows from the fact that given $\mathbf{x}$, the event $\mathcal{E}$ is conditionally independent of $\mathbbm{1}_{\{uv\in\mathbf{G}\}}$.
\end{proof}

We will only use \Cref{lem:lem_expectation_safe_edge} to upper bound the contribution of multiplicity-1 edges. For edges of multiplicity $\geq\hspace*{-1.2mm}2$, we do not need the safeness mechanism. In fact, for edges of multiplicity $\geq\hspace*{-1.2mm}2$, we will ignore the edges in $\mathbf{G}-G(H)$ and the edges in $\mathbf{G}\cap E_1(H)$, and focus on the truncation events that are caused by having too many edges in $\mathbf{G}\cap E_{\geq 2}(H)$. This motivates the following definition:

\begin{definition}
\label{def:def_tilde_X_E_2}
For every multigraph $H$, define
$$\tilde{\mathbf{Y}}_{\geq 2}^H=\prod_{uv\in E_{\geq 2}(H)}\tilde{\mathbf{Y}}_{uv,E_{\geq 2}(H)}^{m_H(uv)},$$
where
$$\tilde{\mathbf{Y}}_{uv,E_{\geq 2}(H)}:=\mathbf{Y}_{uv}\cdot\mathbbm{1}_{\{d_{\mathbf{G}\cap E_{\geq 2}(H)}(u)\leq \Delta\}}\cdot \mathbbm{1}_{\{d_{\mathbf{G}\cap E_{\geq 2}(H)}(v)\leq \Delta\}},$$
and
$$d_{\mathbf{G}\cap E_{\geq 2}(H)}(v):=\big|\big\{w\in V(H):\; vw\in \mathbf{G}\cap E_{\geq 2}(H)\big\}\big|.$$
\end{definition}

\begin{lemma}
\label{lem:lem_UH_deg1}
For every multigraph $H$ with at most $st=sK\log n$ vertices and at most $st$ multi-edges, we have
$$\big|\mathbb{E}\big[\overline{\mathbf{Y}}_H\big|\mathbf{x}\big]\big|\leq n^{\frac{2K}{A}}\left(\frac{6}{\epsilon}\right)^{|E_1^a(H)|}\left(\frac{\epsilon d}{2n}\right)^{|E_1(H)|} \cdot\mathbb{E}\big[|\tilde{\mathbf{Y}}_{\geq 2}^H|\big|\mathbf{x}\big],$$
where $E_1^a(H)$ is as in \Cref{def:def_E1_E2_annoying}.
\end{lemma}
\begin{proof}
We only provide a proof sketch here. The detailed proof can be found in \Cref{subsubsec:subsubsec_edges_mult_1_upper}.

The main idea of the proof is based on partitioning $E_1(H)$ into three sets:
$$E_1^a(H)=\Big\{uv\in E_1(H):\; u\in \mathcal{L}_1(H)\text{ or }v\in \mathcal{L}_{1}(H)\Big\},$$
\begin{equation}
\label{eq:eq_E_1b_H}
E_1^b(H)=\Big\{uv\in E_1(H)\setminus E_1^a(H):\; u\in \mathcal{S}_{\geq 2}(H)\text{ and }v\in \mathcal{S}_{\geq 2}(H)\Big\},
\end{equation}
and
\begin{equation}
\label{eq:eq_E_1c_H}
E_1^d(H)=\Big\{uv\in E_1(H)\setminus E_1^a(H):\; u\notin \mathcal{S}_{\geq 2}(H)\text{ or }v\notin\mathcal{S}_{\geq 2}(H)\Big\}.
\end{equation}

The end-vertices of edges in $E_1^b(H)$ belong to $\mathcal{S}_1(H)\cap\mathcal{S}_{\geq 2}(H)$. \Cref{lem:lem_bound_prob_completely_unsafe} implies that it is unlikely for edges in $E_1^b(H)$ to be unsafe, and so \Cref{lem:lem_expectation_safe_edge} implies that they will likely behave similarly to the truncated case. As we will see in the detailed proof in \Cref{subsubsec:subsubsec_edges_mult_1_upper}, the total contribution of edges in $E_1^b(H)$ can be upper bounded by $\displaystyle n^{\frac{K}{A}}\left(\frac{\epsilon d}{2n}\right)^{|E_1^b(H)|}$. Note that the term $\displaystyle\left(\frac{\epsilon d}{2n}\right)^{|E_1^b(H)|}$ is the contribution of edges in $E_1^b(H)$ in the non-truncated case. The factor $n^{\frac{K}{A}}$ comes from the fact that the edges in $E_1^b(H)$ are not always safe: They are only likely to be so, and the small probability for the edges in $E_1^b(H)$ to be unsafe will ultimately cause a multiplication by a factor that can be upper bounded by $n^{\frac{K}{A}}$.

For the edges in $E_1^a(H)\cup E_1^d(H)$, we did not find an easy way to get a good upper bound on their contribution, so we used a potentially very loose upper bound. Roughly speaking, we used the fact that $\big|\overline{\mathbf{Y}}_{uv}\big|\leq |\mathbf{Y}_{uv}|$ for every $uv\in G(H)$ in order to upper bound the contribution of an edge $uv\in E_1^a(H)\cup E_1^d(H)$ by $\mathbb{E}\big[|\mathbf{Y}_{uv}|\big|\mathbf{x}\big]\leq \frac{3d}{n}$. Therefore, the total contribution of edges in $E_1^a(H)\cup E_1^d(H)$ can be upper bounded by
$$\left(\frac{3d}{n}\right)^{|E_1^a(H)|+| E_1^d(H)|}=\left(\frac{6}{\epsilon}\right)^{|E_1^a(H)|}\left(\frac{6}{\epsilon}\right)^{|E_1^d(H)|}\left(\frac{\epsilon d}{2n}\right)^{|E_1^a(H)|+| E_1^d(H)|}.$$
Now since the vertices in $V(H)\setminus\mathcal{S}_{\geq 2}(H)$ have degrees of at least $\frac{\Delta}{4}$, we cannot have too many vertices in $V(H)\setminus\mathcal{S}_{\geq 2}(H)$. Now since every edge in $E_1^d(H)$ must be incident to a vertex in $\mathcal{S}_1(H)\cap \big(V(H)\setminus\mathcal{S}_{\geq 2}(H)\big)$ and since every vertex in $\mathcal{S}_1(H)$ is incident to at most $\tau$ vertices, we can deduce that we cannot have too many edges in $E_1^d(H)$. This observation be used to show that $\displaystyle\left(\frac{6}{\epsilon}\right)^{|E_1^d(H)|}\leq n^{\frac{K}{A}}$.

For edges of multiplicity $\geq\hspace*{-1.2mm}2$, it is easy to see that $$\prod_{uv\in E_{\geq 2}(H)}\big|\overline{\mathbf{Y}}_{uv}^{m_H(uv)}\big|\leq \big|\tilde{\mathbf{Y}}_{\geq 2}^H\big|.$$ By combining all these observations together, we get
$$\big|\mathbb{E}\big[\overline{\mathbf{Y}}_H\big|\mathbf{x}\big]\big|\leq n^{\frac{2K}{A}}\left(\frac{6}{\epsilon}\right)^{|E_1^a(H)|}\left(\frac{\epsilon d}{2n}\right)^{|E_1(H)|} \cdot\mathbb{E}\big[|\tilde{\mathbf{Y}}_{\geq 2}^H|\big|\mathbf{x}\big].$$
\end{proof}

\subsubsection{Analyzing edges of multiplicity at least 2}

\begin{lemma}
\label{lem:lem_UH_deg2}
For every multigraph $H$ with at most $st=sK\log n$ vertices and at most $st$ multi-edges, and for $n$ large enough, we have
$$\mathbb{E}\big[|\tilde{\mathbf{Y}}_{\geq 2}^H|\big|\mathbf{x}\big]\leq \frac{1}{\resizebox{0.055\textwidth}{!}{$\displaystyle\prod_{v\in\mathcal{L}_{\geq2}(H)}$} n^{\frac{1}{4}\left(d^H_{\geq 2}(v)-\Delta\right)}}\left(\frac{d}{n}\right)^{|\E_{\geq2}(H)|}\prod_{uv\in E_{\geq2}^a(H)}\left[1+\frac{\epsilon \mathbf{x}_u\mathbf{x}_v}{2}+\frac{d}{n}\right],$$
where $\tilde{\mathbf{Y}}_{\geq 2}^H$ is as in \Cref{lem:lem_UH_deg1}, and $E_{\geq2}^a(H)$ and $E_{\geq2}^b(H)$ are as in \Cref{def:def_E1_E2_annoying}.
\end{lemma}
\begin{proof}
We only provide a proof sketch here. The detailed proof can be found in \Cref{subsubsec:subsubsec_edges_mult_2_upper}.

For an edge $uv\in E_{\geq2}^a(H)$, we use the fact that $\big|\tilde{\mathbf{Y}}_{uv,E_{\geq 2}(H)}\big|\leq |\mathbf{Y}_{uv}|$, which allows us to upper bound the contribution of $uv$ by $$\mathbb{E}\big[|\mathbf{Y}_{uv}|^{m_{H}(uv)}\big|\mathbf{x}\big] \leq\left(1+\frac{\epsilon \mathbf{x}_u\mathbf{x}_v}{2}\right)\frac{d}{n}+\frac{d^2}{n^2}.$$

For edges in $E_{\geq2}^b(H)$, notice the following:
\begin{itemize}
\item If an edge $uv\in E_{\geq2}^b(H)$ is present in $\mathbf{G}$, then its contribution to the expectation is at most:
\begin{equation}
\label{eq:eq_edge_mult_2_present}
\left(1-\frac{d}{n}\right)^{m_H(uv)}\left(1+\frac{\epsilon \mathbf{x}_u\mathbf{x}_v}{2}\right)\frac{d}{n}\leq \frac{2d}{n}.
\end{equation}
\item If an edge $uv\in E_{\geq2}^b(H)$ is not present in $\mathbf{G}$, then its contribution to the expectation is at most:
\begin{equation}
\label{eq:eq_edge_mult_2_absent}
\left(-\frac{d}{n}\right)^{m_H(uv)}\cdot\left[1-\left(1+\frac{\epsilon \mathbf{x}_u\mathbf{x}_v}{2}\right)\frac{d}{n}\right]=  O\left(\frac{1}{n^2}\right)=\frac{2d}{n}\cdot O\left(\frac{1}{n}\right).
\end{equation}
\end{itemize}
Now notice that every edge in $E_{\geq2}^b(H)$ is incident to some vertex in $\mathcal{L}_{\geq2}(H)$. On the other hand, every vertex $v\in\mathcal{L}_{\geq2}(H)$ is incident to $d^H_{\geq 2}(v)>\Delta$ edges in $E_{\geq2}^b(H)$. These edges cannot all be present in $\mathbf{G}$ without causing truncation. In fact, if more than $\Delta$ of these edges are present in $\mathbf{G}$, then $\tilde{\mathbf{Y}}_{\geq 2}^H=0$, so we can consider only  the cases where at most $\Delta$ of these edges are present in $\mathbf{G}$. This ultimately makes it possible to show that the total contribution of the edges in $E_{\geq2}^b(H)$ that are incident to a vertex $v\in\mathcal{L}_{\geq2}(H)$ is at most
\begin{equation}
\label{eq:eq_L_2_upper_bound_one_vertex}
\tilde{O}\left(\frac{1}{n}\right)^{d^H_{\geq 2}(v)-\Delta}\cdot\left(\frac{2d}{n}\right)^{d_{\geq 2}^H(v)}=\tilde{O}\left(\frac{1}{n}\right)^{d^H_{\geq 2}(v)-\Delta}\cdot\left(\frac{d}{n}\right)^{d_{\geq 2}^H(v)}.
\end{equation}
If there is no edge in $E_{\geq2}^b(H)$ which has both its end-vertices in $\mathcal{L}_{\geq2}(H)$, then we can multiply the upper bounds \cref{eq:eq_L_2_upper_bound_one_vertex} for every $v\in\mathcal{L}_{\geq2}(H)$, and deduce that the total contribution of edges in $E_{\geq2}^b(H)$ can be upper bounded by
\begin{align*}
\prod_{v\in\mathcal{L}_{\geq2}(H)}\left[\tilde{O}\left(\frac{1}{n}\right)^{d^H_{\geq 2}(v)-\Delta}\cdot\left(\frac{d}{n}\right)^{d_{\geq 2}^H(v)}\right]=\frac{1}{\resizebox{0.055\textwidth}{!}{$\displaystyle\prod_{v\in\mathcal{L}_{\geq2}(H)}$}\tilde{\Omega}\left(n^{d^H_{\geq 2}(v)-\Delta}\right)}\cdot\left(\frac{d}{n}\right)^{|E_{\geq2}^b(H)|}.
\end{align*}
However, since it is possible for an edge to have both its end-vertices in $\mathcal{L}_{\geq2}(H)$, we cannot just multiply the upper bounds \cref{eq:eq_L_2_upper_bound_one_vertex} for every $v\in\mathcal{L}_{\geq2}(H)$ because some edges would be counted twice. Instead, it is possible to show that the total contribution of edges in $E_{\geq2}^b(H)$ can be upper bounded by
\begin{align*}
\frac{1}{\resizebox{0.055\textwidth}{!}{$\displaystyle\prod_{v\in\mathcal{L}_{\geq2}(H)}$}\tilde{\Omega}\left(n^{\frac{1}{2}\left(d^H_{\geq 2}(v)-\Delta\right)}\right)}\cdot\left(\frac{d}{n}\right)^{|E_{\geq2}^b(H)|}\leq \frac{1}{\resizebox{0.055\textwidth}{!}{$\displaystyle\prod_{v\in\mathcal{L}_{\geq2}(H)}$}n^{\frac{1}{4}\left(d^H_{\geq 2}(v)-\Delta\right)}}\cdot\left(\frac{d}{n}\right)^{|E_{\geq2}^b(H)|}.
\end{align*}
Combining the contribution of edges in $E_{\geq2}^a(H)$ and edges in $E_{\geq2}^b(H)$, we get
$$\mathbb{E}\big[|\tilde{\mathbf{Y}}_{\geq 2}^H|\big|\mathbf{x}\big]\leq \frac{1}{\resizebox{0.055\textwidth}{!}{$\displaystyle\prod_{v\in\mathcal{L}_{\geq2}(H)}$} n^{\frac{1}{4}\left(d^H_{\geq 2}(v)-\Delta\right)}}\left(\frac{d}{n}\right)^{|\E_{\geq2}(H)|}\prod_{uv\in E_{\geq2}^a(H)}\left[1+\frac{\epsilon \mathbf{x}_u\mathbf{x}_v}{2}+\frac{d}{n}\right],$$
See \Cref{subsubsec:subsubsec_edges_mult_2_upper} for the details.
\end{proof}

\subsubsection{Proof of the upper bound for every multigraph}

Now we are ready to prove the upper bound $\overline{U}_H(\mathbf{x})$ on $\big|\mathbb{E}\big[\overline{\mathbf{Y}}_H|\mathbf{x}\big]\big|$:

\begin{lemma}
\label{lem:lem_UH}
For every multigraph $H$ with at most $st=sK\log n$ vertices and at most $st$ multi-edges, if $\overline{U}_H(\mathbf{x})$ is as in \Cref{def:def_upper_bound_UH} and $n$ is large enough, then
$$\big|\mathbb{E}\big[\overline{\mathbf{Y}}_H|\mathbf{x}\big]\big|\leq \overline{U}_H(\mathbf{x}).$$
\end{lemma}
\begin{proof}
This is a direct corollary of \Cref{lem:lem_UH_deg1} and \Cref{lem:lem_UH_deg2}, and the fact that\footnote{Note that if we put $\frac{d^2}{n^2}$ instead of $\frac{3d^2}{n\sqrt{n}}$ in $\overline{U}_H(\mathbf{x})$, we still get a valid upper bound on $\big|\mathbb{E}\big[\overline{\mathbf{Y}}_H|\mathbf{x}\big]\big|$. We used the term $\frac{3d^2}{n\sqrt{n}}$ because it will be convenient when we upper bound $\mathbb{E}\left[\Tr\left(\big(Q^{(s)}\big(\overline{\mathbf{Y}}\big)-\mathbf{x}\transpose{\mathbf{x}}\big)^t\right)\right]$.} $\frac{d}{n}\leq \frac{3d}{\sqrt{n}}$.
\end{proof}

\subsection{Bounds for nice multigraphs}\label{sec:appendix-bounds-nice-multigraphs}

In this section, we will prove a lower bound on $\mathbb{E}\big[\overline{\mathbf{Y}}_H\big]$ for multigraghs $H$ belonging to a nice family of block self-avoid-avoiding walks.

\begin{definition}
\label{def:def_nice_bsaw}
A block self-avoiding-walk $H$ is said to be \emph{nice} if it satisfies the following:
\begin{itemize}
\item $H$ contains at most one cycle, and it should be formed by edges of multiplicity 1 in $H$. In other words,
\begin{itemize}
\item $E_1(H)$ is either empty or a cycle.\footnote{Hence, if $E_1(H)\neq\varnothing$, then $d^H_{1}(v)=2$ for all $v\in V(E_1(H))$.}
\item The edges of multiplicity $\geq\hspace*{-1.2mm}2$ in $H$ form a forest, i.e., $E_{\geq 2}(H)$ is a forest.
\item There is no path in $E_{\geq 2}(H)$ between any two vertices $u,v\in V(E_1(H))$.
\end{itemize}
\item $\mathcal{L}_{\geq2}(H)=\varnothing$, i.e., $E_{\geq 2}(H)=\mathcal{S}_{\geq 2}(H)\cup \mathcal{I}_{\geq 2}(H)$ and so $d^H_{\geq 2}(v)\leq \Delta$ for all $v\in V(H)$.
\item $|E_1(H)|\geq \frac{t}{A}$.
\end{itemize}
We denote the set of nice $(s,t)$-block self-avoiding-walks as $\nbsaw{s}{t}$.
\end{definition}

\begin{remark}
	The careful reader may notice that the set $\nbsaw{s}{t}$ slightly differs from the definition in \cref{lem:sbm-upperbound-negligible-walks}, due to the additional constraint $|E_1(H)|\geq \frac{t}{A}$. The set of nice block self-avoiding walks with $|E_1(H)|< \frac{t}{A}$ however is negligible as shown in \cref{lem:count-nice-walks-large-degree}.
\end{remark}

In fact, we will provide tight bounds on $\mathbb{E}\big[\overline{\mathbf{Y}}_H\big]$ for a family of multigraphs that is larger than $\nbsaw{s}{t}$. In the following two definitions, we introduce the family of $(s,t)$-pleasant multigraphs.

\begin{definition}
\label{def:def_agreeable_multigraphs}
A multigraph $H$ is said to be \emph{agreeable} if satisfies one of the following three conditions:
\begin{itemize}
\item[(1)] The underlying graph of $G(H)$ is a cycle. In this case, we say that $H$ is \emph{type-1 agreeable}.
\item[(2)] There are two sets of edges $E'(H)\subset E(H)$, $E''(H)\subset E(H)$, and a vertex $u_H\in V(H)$ such that:
\begin{itemize}
\item $E(H)=E'(H)\cup E''(H)$.
\item $E'(H)$ and $E''(H)$ are cycles.
\item $V(E'(H))\cap V(E''(H))=\{u_H\}$.
\end{itemize}
In this case, we say that $H$ is \emph{type-2 agreeable}.
\item[(3)] There are two sets of edges $E'(H)\subset E(H)$ and $E''(H)\subset E(H)$ such that:
\begin{itemize}
\item $E(H)=E'(H)\cup E''(H)$.
\item $E'(H)$ and $E''(H)$ are cycles.
\item $E'(H)\cap E''(H)\neq\varnothing$ and $V(E'(H))\cap V(E''(H))=V\big(E'(H)\cap E''(H)\big)$.
\item $E'(H)\cap E''(H)$ is a simple path.
\item $E'(H)\cap E''(H)\subset E_{\geq 2}(H)$, i.e., all the edges in $E'(H)\cap E''(H)$ are of multiplicity at least 2 in $H$.
\end{itemize}
In this case, we say that $H$ is \emph{type-3 agreeable}.
\end{itemize}
\end{definition}

\begin{definition}
\label{def:def_pleasant_multigraphs}
A multigraph $H$ is said to be $(s,t)$-\emph{pleasant} if it satisfies the following conditions:
\begin{itemize}
\item $H$ contains at most $st$ vertices and at most $st$ multi-edges.
\item $\mathcal{L}_{\geq 2}(H)=\varnothing$.
\item There are $r_H$ sub-multigraphs $H^{(1)},\ldots,H^{(r_H)}$ of $H$ such that:
\begin{itemize}
\item For every $i\in [r_H]$, $H^{(i)}$ is an induced sub-multigraph of $H$, i.e., $H^{(i)} = H\big(V\big(H^{(i)}\big)\big)$.
\item $H^{(1)},\ldots,H^{(r_H)}$ are vertex-disjoint, i.e., $V\big(H^{(1)}\big),\ldots,V\big(H^{(r_H)}\big)$ are mutually disjoint.
\item $H^{(1)},\ldots,H^{(r_H)}$ are agreeable.
\item $\displaystyle E_1(H)= \bigcup_{i\in [r_H]}E_1\big(H^{(i)}\big)=E_1\big(H^{(\ast)}\big)$, where $\displaystyle H^{(\ast)}=\bigcup_{i\in [r_H]}H^{(i)}$.\footnote{This means that all edges in $H-H^{(\ast)}$ are of multiplicity at least 2.}
\item The only cycles in $H$ are those inside $\displaystyle H^{(\ast)}=\bigcup_{i\in [r_H]}H^{(i)}$.\footnote{This means that if we contract $H^{(1)},\ldots,H^{(r_H)}$ into $r_H$ vertices, $H$ becomes a forest.}
\end{itemize}
\item Every cycle in $H$ contains at least $\frac{t}{A}$ multiplicity-1 edges.
\end{itemize}
The sub-multigraphs $H^{(1)},\ldots,H^{(r_H)}$ are said to be the \emph{agreeable components} of $H$.

It is easy to see that every nice $(s,t)$-block self-avoiding-walk is $(s,t)$-pleasant.
\end{definition}

\subsubsection{Informal discussion and proof strategy}

In order to be able to show tights bound on $\mathbb{E}\big[\overline{\mathbf{Y}}_H\big]$ for an $(s,t)$-pleasant multigraph $H$, we need to be more precise in our calculations. The techniques that were developed in \Cref{subsec:subsec_upper_bound_H} will be useful, but we need more ideas in order to get precise calculations that allow for a proof of tight bounds. In the following few paragraphs, we will informally describe how we will prove the tight bounds on $\mathbb{E}\big[\overline{\mathbf{Y}}_H\big]$ for a pleasant multigraph $H$.

Roughly speaking, if we write the exact expression of $\mathbb{E}\big[\overline{\mathbf{Y}}_H\big|\mathbf{x}\big]$, we will get a very complicated polynomial $g(\mathbf{x})$ of $\mathbf{x}$. We can decompose the complicated polynomial $g(\mathbf{x})$ into positive terms and negative terms, and we might hope to get a lower bound on $\mathbb{E}\big[\overline{\mathbf{Y}}_H\big|\mathbf{x}\big]$ by showing that the negative terms are negligible. However, it does not seem that we can easily show that the negative monomials are negligible.\footnote{In fact, this might not even be possible. It might be the case that the negative terms are not negligible, but the total aggregate of the positive terms is more important than the total aggregate of the negative terms. For example, consider $n=5n-4n$: While the negative term $4n$ is not negligible with respect to the positive term $5n$, the positive term is more important. Something like this occurs in $\mathbb{E}\big[\overline{\mathbf{Y}}_H\big|\mathbf{x}\big]$: In order to see this, consider the case of a single edge of multiplicity 1, and suppose that $\epsilon$ is very small.}

Instead of proving a lower bound on $\mathbb{E}\big[\overline{\mathbf{Y}}_H\big|\mathbf{x}\big]=g(\mathbf{x})$, we will prove a lower bound on $\mathbb{E}\big[\overline{\mathbf{Y}}_H\big]=\mathbb{E}[g(\mathbf{x})]$. The main reason why we considered lower bounding $\mathbb{E}\big[\overline{\mathbf{Y}}_H\big]=\mathbb{E}[g(\mathbf{x})]$ instead of $\mathbb{E}\big[\overline{\mathbf{Y}}_H\big|\mathbf{x}\big]=g(\mathbf{x})$ is that $\mathbb{E}[\mathbf{x}_u\mathbf{x}_v]=0$ for every edge $uv$. This property implies  that the expectation of the vast majority of the monomials that appear in $g(\mathbf{x})$ is actually zero. In fact, it is possible to show that for any set $\mathsf{E}$ of edges, we have $\displaystyle\mathbb{E}\left[\prod_{uv\in \mathsf{E}} \mathbf{x}_u\mathbf{x}_v \right]\neq 0$ if and only if $\mathsf{E}$ is the disjoint union of cycles, in which case we have $\displaystyle\mathbb{E}\left[\prod_{uv\in \mathsf{E}} \mathbf{x}_u\mathbf{x}_v \right]=1$. So by computing $\mathbb{E}\big[\overline{\mathbf{Y}}_H\big]=\mathbb{E}[g(\mathbf{x})]$, we can get rid of the vast majority of the terms in $g(\mathbf{x})$. This is the main reason why analyzing $\mathbb{E}\big[\overline{\mathbf{Y}}_H\big]=\mathbb{E}[g(\mathbf{x})]$ is much simpler than analyzing $\mathbb{E}\big[\overline{\mathbf{Y}}_H\big|\mathbf{x}\big]=g(\mathbf{x})$.

The following lemma shows that the nice structure of pleasant multigraphs makes the behavior of the expectation of monomials in $(\mathbf{x}_u\mathbf{x}_v)_{uv\in E(H)}$ very simple:

\begin{lemma}
\label{lem:lem_Expectation_Y_Nice}
Let $H$ be an arbitrary multigraph and let $\mathsf{E}\subset E(H)$. We have:
\begin{itemize}
\item If $\mathsf{E}$ is an edge-disjoint union of cycles, then $\displaystyle\mathbb{E}\left[\prod_{uv\in \mathsf{E}} \mathbf{x}_u\mathbf{x}_v \right]=1$.
\item If $\mathsf{E}$ is not an edge-disjoint union of cycles, then $\displaystyle\mathbb{E}\left[\prod_{uv\in \mathsf{E}} \mathbf{x}_u\mathbf{x}_v \right]=0$.
\end{itemize}
\end{lemma}
\begin{proof}
We have:
\begin{align*}
\prod_{uv\in \mathsf{E}} \mathbf{x}_u\mathbf{x}_v&=\prod_{v\in V(\mathsf{E})} \mathbf{x}_v^{d_{\mathsf{E}}(v)},
\end{align*}
where $$d_{\mathsf{E}}(v)=\big|\big\{u\in V(H):\; uv\in\mathsf{E}\big\}\big|.$$

Since $(\mathbf{x}_v)_{v\in V(H)}$ are i.i.d. Rademacher random variables, it follows that
$$\mathbb{E}\left[\prod_{uv\in \mathsf{E}} \mathbf{x}_u\mathbf{x}_v\right]=
\begin{cases}
1\quad&\text{if }d_{\mathsf{E}}(v)\text{ is even for all }v\in V(\mathsf{E}),\\
0\quad&\text{if there exists at least one vertex }v\in V(\mathsf{E})\text{ such that }d_{\mathsf{E}}(v)\text{ is odd}.\\
\end{cases}
$$

Now notice the following:
\begin{itemize}
\item If $\mathsf{E}$ is an edge-disjoint union of cycles, then $d_{\mathsf{E}}(v)$ is even for every $v\in V(\mathsf{E})$, and so $\displaystyle\mathbb{E}\left[\prod_{uv\in \mathsf{E}} \mathbf{x}_u\mathbf{x}_v \right]=1$.
\item If $\mathsf{E}$ is not an edge-disjoint union of cycles, then there must exist one vertex $v\in V(\mathsf{E})$ such that $d_{\mathsf{E}}(v)$ is odd, and so $\displaystyle\mathbb{E}\left[\prod_{uv\in \mathsf{E}} \mathbf{x}_u\mathbf{x}_v \right]=0$.
\end{itemize}
\end{proof}

As can be seen from \Cref{lem:lem_Expectation_Y_Nice}, if $H$ is a pleasant multigraph, the behavior of the expectation of monomials in $(\mathbf{x}_u\mathbf{x}_v)_{uv\in E(H)}$ is very simple: Only the monomials corresponding to $\bigcup_{i\in I}C_i(H)$ for some $I\subseteq [z]$ have nonzero expectation, where $C_1(H),\ldots,C_z(H)$ are the cycles of $E_1(H)$. If we could write $\mathbb{E}\big[\overline{\mathbf{Y}}_H\big|\mathbf{x}\big]=g(\mathbf{x})$ as a polynomial that only depends on $(\mathbf{x}_u\mathbf{x}_v)_{uv\in G(H)}$, the expectation $\mathbb{E}\big[\overline{\mathbf{Y}}_H\big]=\mathbb{E}[g(\mathbf{x})]$ would be very simple. Unfortunately, this is not the case because of truncation: The presence or absence of edges outside $E(H)$ in $\mathbf{G}$ makes $g(\mathbf{x})$ also depend on $(\mathbf{x}_u\mathbf{x}_v)_{uv\notin E(H)}$. Therefore, even if we use the fact that for every set $\mathsf{E}$ of edges\footnote{Here $\mathsf{E}$ may or may not be  a subset of $E(H)$.}, $\displaystyle\mathbb{E}\left[\prod_{uv\in \mathsf{E}} \mathbf{x}_u\mathbf{x}_v \right]\neq 0$ if and only if $\mathsf{E}$ is a disjoint union of cycles, the expression of $\mathbb{E}[g(\mathbf{x})]$ is still too complicated to analyze. In fact, we have two complications:
\begin{itemize}
\item $H$-cross-edges\footnote{Recall that an $H$-cross-edge is an edge $uv$ such that $u,v\in V(H)$ and $uv\notin E(H)$.} makes it possible to have monomials in $g(\mathbf{x})$ corresponding to cycles in $V(H)$, which are different than $E_1(H)$.
\item Edges from $V(H)$ to $[n]\setminus V(H)$ makes it possible to have monomials in $g(\mathbf{x})$ corresponding to cycles that include vertices from $V(H)$ and vertices from $[n]\setminus V(H)$.
\end{itemize}

The first complication is not too severe because we have very few $H$-cross-edges, namely $O(s^2t^2)=\tilde{O}(1)$ $H$-cross-edges, and the probability of any particular edge being present in $\mathbf{G}$ is $O\left(\frac{1}{n}\right)$. Therefore, the event that at least one $H$-cross-edge is present in $\mathbf{G}$ has negligible probability, and we can assume that no $H$-cross-edge is present in $\mathbf{G}$. This does not completely solve the first complication, because the probability that an $H$-cross-edge $uv$ is not present in $\mathbf{G}$ is 
$$\mathbb{P}[uv\notin\mathbf{G}]=1-\left(1+\frac{\epsilon \mathbf{x}_u\mathbf{x}_v}{2}\right)\frac{d}{n},$$ which contains a term that depends on $\mathbf{x}_u\mathbf{x}_v$. Therefore, even if we focus on the event that no $H$-cross-edge is present in $\mathbf{G}$, we still have monomials that depend on $(\mathbf{x}_u\mathbf{x}_v)_{uv\text{ is an }H\text{-cross-edge}}$. Fortunately, the term in $\mathbb{P}[uv\notin\mathbf{G}]$ that depends on $\mathbf{x}_u\mathbf{x}_v$ is negligible with respect to the constant term that does not depend on $\mathbf{x}_u\mathbf{x}_v$. We can leverage this observation to show that the total contribution of monomials containing $H$-cross-edges is negligible.

The second complication is a bit trickier to overcome. We know that the truncation event depends only on $\big(d_{\mathbf{G}}(v)\big)_{v\in V(H)}$, and so depends only on $\big(d_{\mathbf{G}-G(H)}(v)\big)_{v\in V(H)}$ and $(\mathbf{Y}_{uv})_{uv\in E(H)}$. Only $\big(d_{\mathbf{G}-G(H)}(v)\big)_{v\in V(H)}$ is problematic for us because it will lead to terms that depend on $(\mathbf{x}_u\mathbf{x}_v)_{uv\notin E(H)}$. In order to analyze the conditional expectation of $\overline{\mathbf{Y}}_H$ given $\mathbf{x}$ and given that no $H$-cross-edge is present in $\mathbf{G}$, we will further condition over $\big(d_{\mathbf{G}-G(H)}(v)\big)_{v\in V(H)}$. If no $H$-cross-edge is present in $\mathbf{G}$, then for every $v\in V(H)$, we have
$$d_{\mathbf{G}-G(H)}(v)=d_{\mathbf{G}-G(H)}^o(v):=\big|\big\{uv\in\mathbf{G}:\; u\notin V(H)\big\}\big|.$$
The nice thing about $\big(d_{\mathbf{G}-G(H)}^o(v)\big)_{v\in V(H)}$ is that, unlike $\big(d_{\mathbf{G}-G(H)}(v)\big)_{v\in V(H)}$, the random variables $\big(d_{\mathbf{G}-G(H)}^o(v)\big)_{v\in V(H)}$ are conditionally mutually independent given $\mathbf{x}$.\footnote{Because of $H$-cross-edges, the random variables $\big(d_{\mathbf{G}-G(H)}(v)\big)_{v\in V(H)}$ are not conditionally mutually independent given $\mathbf{x}$.} Therefore, for every $(\mathsf{d}_v)_{v\in V(H)}\in\mathbb{N}^{V(H)}$, we have
$$\mathbb{P}\big[\big\{\forall v\in V(H), d_{\mathbf{G}-G(H)}^o(v)=\mathsf{d}_v\big\}\big|\mathbf{x}\big]=\prod_{v\in V(H)}\mathbb{P}\big[d_{\mathbf{G}-G(H)}^o(v)=\mathsf{d}_v\big|\mathbf{x}\big],$$
and for each $v\in V(H)$, we have
\begin{align*}
\mathbb{P}\big[d_{\mathbf{G}-G(H)}^o(v)=\mathsf{d}_v\big|\mathbf{x}\big]&=\sum_{\substack{U \subseteq[n]\setminus V(H):\\|S|=\mathsf{d}_v}}\left[\prod_{u\in U}\mathbb{P}[uv\in \mathbf{G}|\mathbf{x}]\right]\cdot\left[ \prod_{u\in [n]\setminus (U\cup V(H))}\mathbb{P}[uv\notin \mathbf{G}|\mathbf{x}]\right]\\
&=\resizebox{0.7\textwidth}{!}{$\displaystyle\sum_{\substack{U \subseteq[n]\setminus V(H):\\|S|=\mathsf{d}_v}}\left[\prod_{u\in U}\left(1+\frac{\epsilon \mathbf{x}_u\mathbf{x}_v}{2}\right)\frac{d}{n}\right]\cdot\left[ \prod_{u\in [n]\setminus (U\cup V(H))}\left[1-\left(1+\frac{\epsilon \mathbf{x}_u\mathbf{x}_v}{2}\right)\frac{d}{n}\right]\right]$}.
\end{align*}
The second complication that we mentioned comes from the fact that for every $v\in V(H)$, the probability $\mathbb{P}\big[d_{\mathbf{G}-G(H)}^o(v)=\mathsf{d}_v\big|\mathbf{x}\big]$ depends on $(\mathbf{x}_u\mathbf{x}_v)_{u\in [n]\setminus V(H)}$, as can be seen from the above equation. Now here comes the crucial observation that allows us to overcome this complication: If $\mathbf{x}$ is balanced on $[n]\setminus V(H)$ in the sense that $[n]\setminus V(H)$ contains an equal number of vertices from each community, then $\mathbb{P}\big[d_{\mathbf{G}-G(H)}^o(v)=\mathsf{d}_v\big|\mathbf{x}\big]$ will not depend on $(\mathbf{x}_u\mathbf{x}_v)_{u\in[n]\notin V(H)}$, and it will only be a function of  $\mathsf{d}_v$.

Unfortunately, $[n]\setminus V(H)$ is not likely to be balanced. Nevertheless, from Hoeffding inequality it can be easily seen that $[n]\setminus V(H)$ is very likely to be approximately balanced. More precisely, if $\alpha>0$ is a fixed (but small) constant, then with high probability we can find a subset $V_b(H,\mathbf{x})$ of $[n]\setminus V(H)$ that is exactly balanced, i.e., it contains an equal number of vertices from each community, and such that $$\big|[n]\setminus \big(V(H)\cup V_b(H,\mathbf{x})\big)\big|\leq n^{\frac{1}{2}+\alpha}.$$
Notice that there are at most $st\cdot n^{\frac{1}{2}+\alpha}=\tilde{O}\big(n^{\frac{1}{2}+\alpha}\big)=o(n)$ edges from $V(H)$ to $[n]\setminus \big(V(H)\cup V_b(H,\mathbf{x})\big)$. On the other hand, the probability that any particular one of these edges is present in $\mathbf{G}$ is at most $O\left(\frac{1}{n}\right)$, hence, with high probability, none of these edges will be present in $\mathbf{G}$. Therefore, we can treat these edges exactly as we treated the $H$-cross-edges, i.e., we can assume that none of them will be present in $\mathbf{G}$. Furthermore, for any particular edge $uv$ between $u\in [n]\setminus \big(V(H)\cup V_b(H,\mathbf{x})\big)$ and $v\in V(H)$, the term that contains $\mathbf{x}_u\mathbf{x}_v$ in $\mathbb{P}[uv\notin\mathbf{G}]$ is negligible with respect to the constant term that does not depend on $\mathbf{x}_u\mathbf{x}_v$. This means that we can completely discard the edges between $V(H)$ and $[n]\setminus \big(V(H)\cup V_b(H,\mathbf{x})\big)$ exactly as we did with $H$-cross-edges.

Now since $V_b(H,\mathbf{x})$ is exactly balanced, we can see that the remaining polynomial will contain monomials that depend only on $(\mathbf{x}_u\mathbf{x}_v)_{uv\in E(H)}$. By computing the expectation, many terms will disappear and we will get a simple expression that is very easy to analyze, and this will eventually yield tight bounds on $\mathbb{E}\big[\overline{\mathbf{Y}}_H\big]$.

In the following, we will turn the above informal discussion into a formal proof.

%
%
%
%

\subsubsection{The well-behaved event}

\label{subsubsec:subsubsec_well_behaved_even_V_b}

In the following, we assume that there is a fixed ordering\footnote{We can use the total order that is induced by the names of the vertices as integers between $1$ and $n$, e.g., $3\leq 5$ and $4\leq 17$.} of the vertices $\{1,\ldots,n\}$ of $\mathbf{G}$. This ordering will be used to define some useful concepts. We emphasize that this ordering can be arbitrary, but it should not dependent on the (random) SBM sample $(\mathbf{G},\mathbf{x})$.
\begin{definition}
Let $H$ be a multigraph with at most $st=sK\log n$ vertices. We say that $\mathbf{x}$ is \emph{approximately balanced on $[n]\setminus V(H)$} if there are at least $\left\lceil\frac{n}{2}- n^{\frac{3}{4}}\right\rceil$ vertices from $[n]\setminus V(H)$ in the first community, and at least $\left\lceil\frac{n}{2}- n^{\frac{3}{4}}\right\rceil$ vertices from $[n]\setminus V(H)$ in the second community, i.e.,
$$\big|\big\{v\in [n]\setminus V(H):\; \mathbf{x}_v=+1\big\}\big|\geq \left\lceil\frac{n}{2}- n^{\frac{3}{4}}\right\rceil,$$
and
$$\big|\big\{v\in [n]\setminus V(H):\; \mathbf{x}_v=-1\big\}\big|\geq \left\lceil\frac{n}{2}- n^{\frac{3}{4}}\right\rceil.$$

Now assume that $\mathbf{x}$ is approximately balanced on $[n]\setminus V(H)$. Let $V_{b}(H,\mathbf{x})$ be the set containing the first\footnote{Here the vertices are chosen according to the fixed ordering of $V(\mathbf{G})=[n]$.} $\left\lceil\frac{n}{2}- n^{\frac{3}{4}}\right\rceil$ vertices in $[n]\setminus V(H)$ of the first community and the first $\left\lceil\frac{n}{2}- n^{\frac{3}{4}}\right\rceil$ vertices in $[n]\setminus V(H)$ of the second community, i.e.,
\begin{equation}
\label{eq:eq_def_V_b}
V_b(H,\mathbf{x})=\Big\{v\in [n]\setminus V(H):\;\big|\big\{u\in [n]\setminus V(H):\; u\leq v,\; \mathbf{x}_u=\mathbf{x}_v\big\}\big|\leq \left\lceil\frac{n}{2}- n^{\frac{3}{4}}\right\rceil\Big\}.
\end{equation}

If $\mathbf{x}$ is not approximately balanced, we still define $V_b(H,\mathbf{x})$ as in \cref{eq:eq_def_V_b}, but now $V_b(H,\mathbf{x})$ contains at most $2\left\lceil \frac{n}{2}-n^{\frac{3}{4}}\right\rceil-1$ vertices, and $V_b(H,\mathbf{x})$ would not necessarily be exactly balanced.
\end{definition}

\begin{remark}
If $\mathbf{x}$ is approximately balanced on $[n]\setminus V(H)$ then:
\begin{itemize}
\item $|V_{b}(H,\mathbf{x})|=2\left\lceil \frac{n}{2}-n^{\frac{3}{4}}\right\rceil$.
\item $V_{b}(H,\mathbf{x})$ contains exactly $\left\lceil\frac{n}{2}- n^{\frac{3}{4}}\right\rceil$ vertices of the first community and $\left\lceil\frac{n}{2}- n^{\frac{3}{4}}\right\rceil$ vertices of the second community.
\item $\big|\big([n]\setminus V(H)\big)\setminus V_b(H,\mathbf{x})|\leq \big|[n]\setminus V_{b}(H,\mathbf{x})\big|\leq 2 n^{\frac{3}{4}}$.
\end{itemize}
\end{remark}

\begin{definition}
\label{def:def_well_behaved_H}
Let $H$ be a multigraph with at most $st=sK\log n$ vertices. We say that $(\mathbf{G},\mathbf{x})$ is $H$-\emph{well-behaved} if:
\begin{itemize}
\item[(1)] $\mathbf{x}$ is approximately balanced on $[n]\setminus V(H)$.
\item[(2)] In the sampled graph $\mathbf{G}$, no vertex in $V(H)$ is adjacent to any vertex in $\big([n]\setminus V(H)\big)\setminus V_b(H,\mathbf{x})$.
\item[(3)] There is no $H$-cross-edge that is present in $\mathbf{G}$.
\item[(4)] All edges in $E_{\geq 2}(H)$ are present\footnote{Note that it is possible to show tight bounds without adding Condition (4) to the definition of the well-behaved event. We only added this condition because it makes the proof of the bounds simpler and easier to describe.} in $\mathbf{G}$.
\end{itemize}
We denote the event that $(\mathbf{G},\mathbf{x})$ is $H$-well-behaved as $\mathcal{E}_{wb,H}$, i.e.,
$$\mathcal{E}_{wb,H}=\big\{(\mathbf{G},\mathbf{x})\text{ is $H$-well-behaved}\big\}.$$
\end{definition}

We will decompose $\mathbb{E}\big[\overline{\mathbf{Y}}_H]$ into two parts by conditioning on the event $\mathcal{E}_{wb,H}$:
$$\mathbb{E}\big[\overline{\mathbf{Y}}_H\big]=\mathbb{E}\big[\overline{\mathbf{Y}}_H\big|\mathcal{E}_{wb,H}\big]\cdot\mathbb{P}[\mathcal{E}_{wb,H}]+\mathbb{E}\big[\overline{\mathbf{Y}}_H\big|\mathcal{E}_{wb,H}^c\big]\cdot\mathbb{P}[\mathcal{E}_{wb,H}^c].$$
We will prove tight bounds on $\mathbb{E}\big[\overline{\mathbf{Y}}_H\big|\mathcal{E}_{wb,H}\big]\cdot\mathbb{P}[\mathcal{E}_{wb,H}]$ and an upper bound on $\left|\mathbb{E}\big[\overline{\mathbf{Y}}_H\big|\mathcal{E}_{wb,H}^c\big]\cdot\mathbb{P}[\mathcal{E}_{wb,H}^c]\right|$. These bounds will imply that the contribution of the not-well-behaved event is negligible with respect to that of the well-behaved event.

\subsubsection{Upper bound on the contribution of the not-well-behaved event}

The following lemma provides an upper bound on the total contribution of the not-well-behaved event in the expectation of $\overline{\mathbf{Y}}_H$.

\begin{lemma}
\label{lem:lem_upper_bound_nice_Wc}
Let $H$ be a multigraph with at most $st=sK\log n$ vertices and at most $st$ multi-edges. Assume that $\mathcal{L}_1(H)=\mathcal{L}_{\geq 2}(H)=\varnothing$, and that $E_{\geq 2}(H)$ forms a forest. If $A>\max\{100K,1\}$ and $n$ is large enough, then we have
$$\left|\mathbb{E}\big[\overline{\mathbf{Y}}_H\big|\mathcal{E}_{wb,H}^c\big]\cdot\mathbb{P}[\mathcal{E}_{wb,H}^c]\right|\leq \frac{2}{n^{\frac{1}{6}}}\cdot \left(\frac{\epsilon d}{2n}\right)^{|E_1(H)|}\cdot\left(\frac{d}{n}\right)^{|E_{\geq 2}(H)|}.$$
\end{lemma}
\begin{proof}
We only provide a proof sketch here. The detailed proof can be found in \Cref{subsubsec:subsubsec_upper_bounding_not_well_behaved}.

We start by further conditioning on more refined events corresponding to the reasons that caused $(\mathbf{G},\mathbf{x})$ to be not-well-behaved. Define the following events:

\begin{equation}
\label{eq:eq_def_E_wb_H_x}
\begin{aligned}
\mathcal{E}_{H,b} &=\big\{\text{Condition (1) of \Cref{def:def_well_behaved_H} is satisfied}\big\}\\
&=\big\{\mathbf{x} \text{ is approximately balanced on }[n]\setminus V(H)\big\},
\end{aligned}
\end{equation}

\begin{equation}
\label{eq:eq_def_E_wb_H_g}
\mathcal{E}_{H,g}=\big\{\text{Conditions (2) and (3) of \Cref{def:def_well_behaved_H} are satisfied}\big\},
\end{equation}
and
\begin{equation}
\label{eq:eq_def_E_wb_H_d}
\begin{aligned}
\mathcal{E}_{H,d}&=\big\{\text{Condition (4) of \Cref{def:def_well_behaved_H} is satisfied}\big\}\\
&=\big\{\text{All edges in $E_{\geq 2}(H)$ are present in $\mathbf{G}$}\big\}.
\end{aligned}
\end{equation}

Clearly,
$$\mathcal{E}_{wb,H}=\mathcal{E}_{H,b}\cap \mathcal{E}_{H,g}\cap \mathcal{E}_{H,d}.$$

We also define the following:

\begin{equation}
\label{eq:eq_def_E_wb_H_xg}
\mathcal{E}_{H,bg}=\mathcal{E}_{H,b}\cap\mathcal{E}_{H,g},
\end{equation}

\begin{equation}
\label{eq:eq_def_E_wb_H_xog}
\mathcal{E}_{H,b\overline{g}}=\mathcal{E}_{H,b}\cap\mathcal{E}_{H,g}^c,
\end{equation}
and
\begin{equation}
\label{eq:eq_def_E_wb_H_xgod}
\mathcal{E}_{H,bg\overline{d}}=\mathcal{E}_{H,bg}\cap\mathcal{E}_{H,d}^c.
\end{equation}

We will decompose $\mathcal{E}_{wb,H}^c$ into mutually exclusive events as follows:
$$\mathcal{E}_{wb,H}^c=\mathcal{E}_{H,b}^c\cup \mathcal{E}_{H,b\overline{g}}\cup \mathcal{E}_{H,bg\overline{d}}.$$
Therefore,
\begin{align*}
\mathbb{E}\big[\overline{\mathbf{Y}}_H\big|\mathcal{E}_{wb,H}^c\big]\cdot\mathbb{P}[\mathcal{E}_{wb,H}^c]=\mathbb{E}\big[\overline{\mathbf{Y}}_H\big|\mathcal{E}_{H,b}^c\big]\cdot\mathbb{P}[\mathcal{E}_{H,b}^c]
&+ \mathbb{E}\big[\overline{\mathbf{Y}}_H\big|\mathcal{E}_{H,b\overline{g}}\big]\cdot\mathbb{P}[\mathcal{E}_{H,b\overline{g}}]\\
&\;\;+ \mathbb{E}\big[\overline{\mathbf{Y}}_H\big|\mathcal{E}_{H,bg\overline{d}}\big]\cdot\mathbb{P}[\mathcal{E}_{H,bg\overline{d}}].
\end{align*}
We will separately upper bound the absolute value of each term in the right hand side of the above equation.

For $\mathbb{E}\big[\overline{\mathbf{Y}}_H\big|\mathcal{E}_{H,b}^c\big]\cdot\mathbb{P}[\mathcal{E}_{H,b}^c]$, we use Hoeffding's inequality to deduce that 
$$\mathbb{P}[\mathcal{E}_{H,b}^c]\leq 2e^{-\frac{9}{8}\sqrt{n}}.$$
By combining this with \Cref{lem:lem_UH}, it is possible to show that for $n$ large enough, we have
\begin{equation}
\label{eq:eq_lem_upper_bound_nice_Wc_y}
\left|\mathbb{E}\big[\overline{\mathbf{Y}}_H\big|\mathcal{E}_{H,b}^c\big]\cdot\mathbb{P}[\mathcal{E}_{H,b}^c]\right|\leq e^{-\sqrt{n}}\cdot\left(\frac{\epsilon d}{2n}\right)^{|E_1(H)|}\cdot\left(\frac{d}{n}\right)^{|E_{\geq 2}(H)|}.
\end{equation}

For $\mathbb{E}\big[\overline{\mathbf{Y}}_H\big|\mathcal{E}_{H,b\overline{g}}\big]\cdot\mathbb{P}[\mathcal{E}_{H,b\overline{g}}]$, we first write
$$\mathbb{E}\big[\overline{\mathbf{Y}}_H\big|\mathcal{E}_{H,b\overline{g}}\big]\cdot\mathbb{P}[\mathcal{E}_{H,b\overline{g}}]=\mathbb{E}\left[\mathbb{E}\big[\overline{\mathbf{Y}}_H\big|\mathbf{x},\mathcal{E}_{H,b\overline{g}}\big]\cdot\mathbb{P}[\mathbf{x}|\mathcal{E}_{H,b\overline{g}}]\right].$$
It is possible to show that $\displaystyle\left|\mathbb{E}\big[\overline{\mathbf{Y}}_H\big|\mathbf{x},\mathcal{E}_{H,b\overline{g}}\big]\right|$ can be upper bounded using the same techniques that allowed us to upper bound $\big|\mathbb{E}\big[\overline{\mathbf{Y}}_H\big|\mathbf{x}\big]\big|$ in \Cref{lem:lem_UH}. On the other hand, notice that there are very few $H$-cross-edges\footnote{There are at most $(st)^2=\tilde{O}(1)$ such edges.}, and very few edges between\footnote{There are at most $st\cdot 2n^{\frac{3}{4}}=\tilde{O}\big(n^{\frac{3}{4}}\big)$ such edges.} $V(H)$ and $\big([n]\setminus V(H)\big)\setminus V_b(H,\mathbf{x})$. Now since each particular edge is present in $\mathbf{G}$ with probability $O\left(\frac{1}{n}\right)$, it can be easily seen that $\mathbb{P}[\mathcal{E}_{H,b\overline{g}}]=\tilde{O}\big(n^{-\frac{1}{4}}\big)$. By carefully combining these facts together, we can show that
\begin{equation}
\label{eq:eq_lem_upper_bound_nice_Wc_g}
\begin{aligned}
\left|\mathbb{E}\big[\overline{\mathbf{Y}}_H\big|\mathcal{E}_{H,b\overline{g}}\big]\cdot\mathbb{P}[\mathcal{E}_{H,b\overline{g}}]\right|&\leq \tilde{O}\left(\frac{n^{\frac{2K}{A}}}{n^{\frac{1}{4}}}\right)\cdot\left(\frac{\epsilon d}{2n}\right)^{|E_1(H)|}\cdot\left(\frac{d}{n}\right)^{|E_{\geq 2}(H)|}\\
&\leq \frac{1}{n^{\frac{1}{6}}}\cdot\left(\frac{\epsilon d}{2n}\right)^{|E_1(H)|}\cdot\left(\frac{d}{n}\right)^{|E_{\geq 2}(H)|},
\end{aligned}
\end{equation}
where the last inequality is true for $n$ large enough.

For $\mathbb{E}\big[\overline{\mathbf{Y}}_H\big|\mathcal{E}_{H,bg\overline{d}}\big]\cdot\mathbb{P}[\mathcal{E}_{H,bg\overline{d}}]$, notice that $\mathcal{E}_{H,bg\overline{d}}$ implies that at least one edge of multiplicity $\geq\hspace*{-1.2mm}2$ is absent from $\mathbf{G}$. On the other hand, \cref{eq:eq_edge_mult_2_present} and \cref{eq:eq_edge_mult_2_absent} imply that the contribution of an edge $uv\in E_{\geq 2}(H)$ in case it is absent from $\mathbf{G}$ is at least $O\left(\frac{1}{n}\right)$ smaller than its contribution when it is present. By the same techniques that allowed us to upper bound $\big|\mathbb{E}\big[\overline{\mathbf{Y}}_H\big|\mathbf{x}\big]\big|$ in \Cref{lem:lem_UH}, together with the fact that we have at most $st=\tilde{O}(1)$ edges in $E_{\geq 2}(H)$, and the fact that at least one of these edges is absent from $\mathbf{G}$, it is possible to show that
\begin{equation}
\label{eq:eq_lem_upper_bound_nice_Wc_d}
\begin{aligned}
\left|\mathbb{E}\big[\overline{\mathbf{Y}}_H\big|\mathcal{E}_{H,bg\overline{d}}\big]\cdot\mathbb{P}[\mathcal{E}_{H,bg\overline{d}}]\right|&\leq \tilde{O}\left(\frac{n^{\frac{2K}{A}}}{n}\right)\cdot\left(\frac{\epsilon d}{2n}\right)^{|E_1(H)|}\cdot\left(\frac{d}{n}\right)^{|E_{\geq 2}(H)|}\\
&\leq \frac{1}{n^{\frac{1}{2}}}\cdot\left(\frac{\epsilon d}{2n}\right)^{|E_1(H)|}\cdot\left(\frac{d}{n}\right)^{|E_{\geq 2}(H)|},
\end{aligned}
\end{equation}
where the last inequality is true for $n$ large enough.

By combining \cref{eq:eq_lem_upper_bound_nice_Wc_y}, \cref{eq:eq_lem_upper_bound_nice_Wc_g} and \cref{eq:eq_lem_upper_bound_nice_Wc_d}, we get the lemma. See \Cref{subsubsec:subsubsec_upper_bounding_not_well_behaved} for the details.
\end{proof}

\subsubsection{Upper bound on the negligible part of the contribution of the well-behaved event}

Now in order to study $\mathbb{E}\big[\overline{\mathbf{Y}}_H\big|\mathcal{E}_{wb,H}\big]\cdot\mathbb{P}[\mathcal{E}_{wb,H}]$, we will divide this expression into the sum of two terms: one that is negligible, and one that is significant. In this section, we prove an upper bound on the negligible part. In the next section, we will prove tight bounds on the significant part.

If we denote the set of $H$-cross-edges as $E_c(H)$, then we have\footnote{Recall that $V_b(H,\mathbf{x})$ is always defined, i.e., it is defined even when $\mathbf{x}$ is not approximately balanced on $[n]\setminus V(H)$.}
\begin{equation*}
\begin{aligned}
\mathbb{P}[\mathcal{E}_{wb,H}|\mathbf{x}]&=\resizebox{0.8\textwidth}{!}{$\displaystyle\left[\prod_{e\in E_c(H)}\mathbb{P}[e\notin\mathbf{G}|\mathbf{x}]\right]\cdot\left[\prod_{\substack{v\in V(H),\\u\in \big([n]\setminus V(H)\big)\setminus V_b(H,\mathbf{x})}}\mathbb{P}[uv\notin\mathbf{G}|\mathbf{x}]\right]\cdot\left[\prod_{e\in E_{\geq 2}(H)}\mathbb{P}[e\in\mathbf{G}|\mathbf{x}]\right]\cdot\mathbbm{1}_{\mathcal{E}_{H,b}}(\mathbf{x})$}\\
&=\left[\prod_{e\in E_{\overline{b}c}(H,\mathbf{x})}\mathbb{P}[e\notin\mathbf{G}|\mathbf{x}]\right]\cdot\left[\prod_{e\in E_{\geq 2}(H)}\mathbb{P}[e\in\mathbf{G}|\mathbf{x}]\right]\cdot\mathbbm{1}_{\mathcal{E}_{H,b}}(\mathbf{x}),
\end{aligned}
\end{equation*}
where
\begin{equation}
\label{eq:eq_def_E_obc_H_H}
E_{\overline{b}c}(H,\mathbf{x})=E_c(H)\cup\big\{uv:\; u\in \big([n]\setminus V(H)\big)\setminus V_b(H,\mathbf{x}),v\in V(H)\big\}.
\end{equation}

Note that we could write $\mathbbm{1}_{\mathcal{E}_{H,b}}=\mathbbm{1}_{\mathcal{E}_{H,b}}(\mathbf{x})$ as a function of $\mathbf{x}$ because the event $\mathcal{E}_{H,b}$ depends only on $\mathbf{x}$, i.e., it is $\sigma(\mathbf{x})$-measurable. Hence,

\begin{align*}
\mathbb{P}[\mathcal{E}_{wb,H}|\mathbf{x}]&=\left[\prod_{uv\in E_{\overline{b}c}(H,\mathbf{x})}\left[1-\left(1+\frac{\epsilon \mathbf{x}_u\mathbf{x}_v}{2}\right)\frac{d}{n}\right]\right]\cdot\left[\prod_{uv\in E_{\geq 2}(H)}\left[\left(1+\frac{\epsilon \mathbf{x}_u\mathbf{x}_v}{2}\right)\frac{d}{n}\right]\right]\cdot\mathbbm{1}_{\mathcal{E}_{H,b}}(\mathbf{x})\\
&=\left[\sum_{S\subseteq E_{\overline{b}c}(H,\mathbf{x})}\prod_{uv\in S}\left[-\left(1+\frac{\epsilon \mathbf{x}_u\mathbf{x}_v}{2}\right)\frac{d}{n}\right]\right]\cdot\left[\prod_{uv\in E_{\geq 2}(H)}\left[\left(1+\frac{\epsilon \mathbf{x}_u\mathbf{x}_v}{2}\right)\frac{d}{n}\right]\right]\cdot\mathbbm{1}_{\mathcal{E}_{H,b}}(\mathbf{x})\\
&=P_{1,H}^{wb}(\mathbf{x})+P_{2,H}^{wb}(\mathbf{x}),
\end{align*}
where
\begin{equation}
\label{eq:eq_Prob_Well_Behaved_Y_1}
P_{1,H}^{wb}(\mathbf{x})=\left[\prod_{uv\in E_{\geq 2}(H)}\left[\left(1+\frac{\epsilon \mathbf{x}_u\mathbf{x}_v}{2}\right)\frac{d}{n}\right]\right]\cdot\mathbbm{1}_{\mathcal{E}_{H,b}}(\mathbf{x}),
\end{equation}
and
\begin{equation}
\label{eq:eq_Prob_Well_Behaved_Y_2}
P_{2,H}^{wb}(\mathbf{x})=\left[\sum_{\substack{S\subseteq E_{\overline{b}c}(H,\mathbf{x}):\\S\neq\varnothing}}\prod_{uv\in S}\left[-\left(1+\frac{\epsilon \mathbf{x}_u\mathbf{x}_v}{2}\right)\frac{d}{n}\right]\right]\cdot\left[\prod_{uv\in E_{\geq 2}(H)}\left[\left(1+\frac{\epsilon \mathbf{x}_u\mathbf{x}_v}{2}\right)\frac{d}{n}\right]\right]\cdot\mathbbm{1}_{\mathcal{E}_{H,b}}(\mathbf{x}).
\end{equation}

It is easy to see that $P_{1,H}^{wb}(\mathbf{x})$ contains the monomials in $\mathbb{P}[\mathcal{E}_{wb,H}|\mathbf{x}]$ that do not depend on $(\mathbf{x}_u\mathbf{x}_v)_{uv\in E_{\overline{b}c}(H,\mathbf{x})}$, whereas $P_{2,H}^{wb}(\mathbf{x})$ contains the monomials in $\mathbb{P}[\mathcal{E}_{wb,H}|\mathbf{x}]$ that depend on $(\mathbf{x}_u\mathbf{x}_v)_{uv\in E_{\overline{b}c}(H,\mathbf{x})}$. We will decompose $\mathbb{E}\big[\overline{\mathbf{Y}}_H\big|\mathcal{E}_{wb,H}\big]\cdot\mathbb{P}[\mathcal{E}_{wb,H}]$ as follows:
\begin{equation}
\label{eq:eq_decompose_X_H_W_H}
\begin{aligned}
\mathbb{E}\big[\overline{\mathbf{Y}}_H\big|\mathcal{E}_{wb,H}\big]\cdot\mathbb{P}[\mathcal{E}_{wb,H}]&= \mathbb{E}\left[\mathbb{E}\big[\overline{\mathbf{Y}}_H\big|\mathbf{x},\mathcal{E}_{wb,H}\big]\cdot\mathbb{P}[\mathcal{E}_{wb,H}|\mathbf{x}]\right]\\
&=\resizebox{0.6\textwidth}{!}{$\displaystyle \mathbb{E}\left[\mathbb{E}\big[\overline{\mathbf{Y}}_H\big|\mathbf{x},\mathcal{E}_{wb,H}\big]\cdot P_{1,H}^{wb}(\mathbf{x})\right]+ \mathbb{E}\left[\mathbb{E}\big[\overline{\mathbf{Y}}_H\big|\mathbf{x},\mathcal{E}_{wb,H}\big]\cdot P_{2,H}^{wb}(\mathbf{x})\right]$}.
\end{aligned}
\end{equation}

The following lemma provides an upper bound on $\left|\mathbb{E}\left[\mathbb{E}\big[\overline{\mathbf{Y}}_H\big|\mathbf{x},\mathcal{E}_{wb,H}\big]\cdot P_{2,H}^{wb}(\mathbf{x})\right]\right|$.

\begin{lemma}
\label{lem:lem_upper_bound_nice_W_2}
Let $H$ be a multigraph with at most $st=sK\log n$ vertices and at most $st$ multi-edges. Assume that $\mathcal{L}_1(H)=\mathcal{L}_{\geq 2}(H)=\varnothing$, and that $E_{\geq 2}(H)$ forms a forest. If $A>\max\{100K,1\}$ and $n$ is large enough, then we have
$$\mathbb{E}\left[\left|\mathbb{E}\big[\overline{\mathbf{Y}}_H\big|\mathbf{x},\mathcal{E}_{wb,H}\big]\cdot P_{2,H}^{wb}(\mathbf{x})\right|\right]\leq \frac{1}{n^{\frac{1}{6}}}\cdot\left(\frac{\epsilon d}{2n}\right)^{|E_1(H)|}\cdot\left(\frac{d}{n}\right)^{|E_{\geq 2}(H)|},$$
where $P_{2,H}^{wb}(\mathbf{x})$ is defined in \cref{eq:eq_Prob_Well_Behaved_Y_2}.
\end{lemma}
\begin{proof}
We only provide a proof sketch here. The detailed proof can be found in \Cref{subsubsec:subsubsec_upper_bounding_well_behaved_negligible}.

Recall that $P_{2,H}^{wb}(\mathbf{x})$ contains only the monomials in $\mathbb{P}[\mathcal{E}_{wb,H}|\mathbf{x}]$ that depend on $(\mathbf{x}_u\mathbf{x}_v)_{uv\in E_{\overline{b}c}(H,\mathbf{x})}$, where $E_{\overline{b}c}(H,\mathbf{x})$ is as in \cref{eq:eq_def_E_obc_H_H}. On the other hand, for every $uv\in E_{\overline{b}c}(H,\mathbf{x})$, the term in $\mathbb{P}[uv\notin\mathbf{G}|\mathbf{x}]$ that depends on $\mathbf{x}_u\mathbf{x}_v$ is equal to $O\left(\frac{1}{n}\right)$, whereas the term that does not depend on $\mathbf{x}_u\mathbf{x}_v$ is equal to $1-O\left(\frac1n\right)$. Now since there are at most $s^2t^2+2st\cdot n^{\frac{3}{4}}=\tilde{O}\left(n^{\frac{3}{4}}\right)$ edges in $E_{\overline{b}c}(H,\mathbf{x})$, after carrying out a few careful calculations, it is possible to show that
$$|P_{2,H}^{wb}(\mathbf{x})|\leq \tilde{O}\left(\frac{1}{n^{\frac{1}{4}}}\right)\cdot \prod_{uv\in E_{\geq 2}(H)}\left[\left(1+\frac{\epsilon \mathbf{x}_u\mathbf{x}_v}{2}\right)\frac{d}{n}\right].$$

By using the same techniques that allowed us to prove the upper bound in \Cref{lem:lem_UH}, we can now show that
\begin{align*}
\mathbb{E}\left[\left|\mathbb{E}\big[\overline{\mathbf{Y}}_H\big|\mathbf{x},\mathcal{E}_{wb,H}\big]\cdot P_{2,H}^{wb}(\mathbf{x})\right|\right]
&\leq \tilde{O}\left(\frac{n^{\frac{K}{A}}}{n^{\frac{1}{4}}}\right)\cdot\left(\frac{\epsilon d}{2n}\right)^{|E_1(H)|}\cdot\left(\frac{d}{n}\right)^{|E_{\geq 2}(H)|}\\
&\leq \frac{1}{n^{\frac{1}{6}}}\cdot\left(\frac{\epsilon d}{2n}\right)^{|E_1(H)|}\cdot\left(\frac{d}{n}\right)^{|E_{\geq 2}(H)|}.
\end{align*}

See \Cref{subsubsec:subsubsec_upper_bounding_well_behaved_negligible} for the details.
\end{proof}

\subsubsection{Tight bounds on the significant part of the contribution of the well-behaved event}

The following lemma provides tight bounds for $\mathbb{E}\left[\mathbb{E}\big[\overline{\mathbf{Y}}_H\big|\mathbf{x},\mathcal{E}_{wb,H}\big]\cdot P_{1,H}^{wb}(\mathbf{x})\right]$.

\begin{definition}
\label{def:def_E_geq2_triple_prime}
Let $H$ be an $(s,t)$-pleasant multigraph and let $H^{(1)},\ldots,H^{(r_H)}$ be the agreeable components of $H$. For every $i\in[r_H]$, we define $E_{\geq 2}'''(H^{(i)})$ as follows:
\begin{itemize}
\item If $H^{(i)}$ is an agreeable component of type 1 or type 2, we define $E_{\geq 2}'''(H^{(i)})=E_{\geq 2}\big(H^{(i)}\big)$.
\item If $H^{(i)}$ is an agreeable component of type 3, let $E'\big(H^{(i)}\big)$ and $E''\big(H^{(i)}\big)$ be as in \Cref{def:def_agreeable_multigraphs}, i.e., $E'\big(H^{(i)}\big)$ and $E''\big(H^{(i)}\big)$ are cycles and $E'\big(H^{(i)}\big)\cap E''\big(H^{(i)}\big)$ is a simple path of edges of multiplicity at least 2. We define $$E_{\geq 2}'''\big(H^{(i)}\big)=E_{\geq 2}\big(H^{(i)}\big)\setminus\Big(E'\big(H^{(i)}\big)\cap E''\big(H^{(i)}\big)\Big).$$
\end{itemize}
We also define $\displaystyle E_{\geq 2}'''(H)=\bigcup_{i\in[r_H]}E_{\geq 2}'''\big(H^{(i)}\big)$.
\end{definition}

\begin{lemma}
\label{lem:lem_main_lemma_pleasant_wb}
Let $H$ be an $(s,t)$-pleasant multigraph, where $t=K\log n$. If $A>\max\left\{1,100K,\frac{100}{K}\right\}$ and $n$ is large enough, then
\begin{align*}
\left(P_s^H -\frac{2}{\sqrt{n}}\right)\cdot&\left(\frac{\epsilon d}{2n}\right)^{|E_1(H)|+|E_{\geq 2}'''(H)|}\left(\frac{d}{n}\right)^{|E_{\geq 2}(H)|-|E_{\geq 2}'''(H)|}\\
&\leq \mathbb{E}\Big[\mathbb{E}\big[\overline{\mathbf{Y}}_H\big|\mathbf{x},\mathcal{E}_{wb,H}\big]\cdot P_{1,H}^{wb}(\mathbf{x})\Big]\\
&\leq \left(P_s^H+\frac{2}{\sqrt{n}}\right)\cdot\left(\frac{\epsilon d}{2n}\right)^{|E_1(H)|+|E_{\geq 2}'''(H)|}\left(\frac{d}{n}\right)^{|E_{\geq 2}(H)|-|E_{\geq 2}'''(H)|},
\end{align*}
where
$$P_s^H=\mathbb{P}\left[\big\{\forall v\in V(H),\;\mathbf{D}^H_v+ d_1^H(v)+d_{\geq 2}^H(v)\leq \Delta\big\}\middle| \mathcal{E}_{H,b}\right],$$
and $\mathbf{D}^H_v=|\{u\in V_b(H,\mathbf{x}):\; uv\in \mathbf{G}\}|$ is the number of edges from $v$ to $V_b(H,\mathbf{x})$ which are present in $\mathbf{G}$.
\end{lemma}
\begin{proof}
Roughly speaking, since $\mathbb{E}\big[\overline{\mathbf{Y}}_H\big|\mathbf{x},\mathcal{E}_{wb,H}\big]\cdot P_{1,H}^{wb}(\mathbf{x})$ contains monomials that depend only on $(\mathbf{x}_u\mathbf{x}_v)_{uv\in E(H)}$, \Cref{lem:lem_Expectation_Y_Nice} implies that by taking the expectation, we can get rid of most of the terms, and obtain a simple expression that can be easily bounded. The detailed proof can be found in \Cref{subsubsec:subsubsec_bounding_well_behaved_significant}.
\end{proof}


\subsubsection{Tight bounds for pleasant multigraphs}

The following two lemmas study the probability of safety $P_s^H$:

\begin{lemma}
\label{lem:lem_Prob_safety_pleasant_wb}
There exists a non-decreasing function $P_s:\mathbb{Z}\to [0,1]$ such that:
\begin{itemize}
\item $P_s(\ell)=0$ for every $\ell<0$.
\item $P_s(\ell)\geq e^{-4d}$ for every $\ell\geq0$.
\item $P_s(\ell)\geq 1-\frac{\eta}{2}$ for every $\ell\geq\frac{\Delta}{4}$, where $\eta$ is as in \cref{lem:lem_bound_prob_large_outside_deg}.
\item If $n$ is large enough, then for every $(s,t)$-pleasant multigraph $H$ with $t=K\log n$, we have
$$P_s^H=\prod_{v\in V(H)}P_s\big(\Delta-d_1^H(v)-d_{\geq 2}^H(v)\big),$$
where $P_s^H$ is as in \Cref{lem:lem_main_lemma_pleasant_wb}.
\end{itemize}
\end{lemma}
\begin{proof}
The proof can be found in \Cref{subsec:subsec_prob_safety_pleasant}.
\end{proof}

\begin{lemma}
\label{lem:lem_Prob_nice_safe_lower_bound}
Let $H$ be an $(s,t)$-pleasant multigraph, where $t=K\log n$. Let $P_s^H$ be as in \Cref{lem:lem_main_lemma_pleasant_wb}. Assume that $A>1$. We have the following:
\begin{itemize}
\item If $d_1^H(v)+d_{\geq 2}^H(v)\leq \Delta$ for every $v\in V(H)$, then for $n$ large enough, we have $P_s^H\geq \frac{1}{n^{\frac{2K}{A}}}$.
\item If there exists $v\in V(H)$ such that $d_1^H(v)+d_{\geq 2}^H(v)> \Delta$, then $P_s^H=0$.
\end{itemize}
\end{lemma}
\begin{proof}
The proof can be found in \Cref{subsec:subsec_prob_safety_pleasant}.
\end{proof}

Now we are ready to prove tight lower and upper bounds on $\mathbb{E}\big[\overline{\mathbf{Y}}_H\big]$ for every pleasant multigraph.

\begin{lemma}
\label{lem:lem_main_lemma_pleasant_tight_bounds}
Let $H$ be an $(s,t)$-pleasant multigraph, where $t=K\log n$. If $A>\max\left\{1,100K,\frac{100}{K}\right\}$ and $n$ is large enough, then
\begin{itemize}
\item If $|E_{\geq 2}'''(H)|\leq\frac{\log n}{12\log\left(\frac{2}{\epsilon}\right)}$ and $d_1^H(v)+d_{\geq 2}^H(v)\leq \Delta$ for every $v\in V(H)$, then for $n$ large enough, we have
\begin{equation}
\label{eq:eq_lem_main_lemma_pleasant_tight_bounds_1}
\begin{aligned}
P_s^H\left(1 -\frac{1}{n^{\frac{1}{16}}}\right) &\left(\frac{\epsilon d}{2n}\right)^{|E_1(H)|+|E_{\geq 2}'''(H)|}\left(\frac{d}{n}\right)^{|E_{\geq 2}(H)|-|E_{\geq 2}'''(H)|}\\
&\leq \mathbb{E}\big[\overline{\mathbf{Y}}_H\big]\leq P_s^H\left(1 +\frac{1}{n^{\frac{1}{16}}}\right) \left(\frac{\epsilon d}{2n}\right)^{|E_1(H)|+|E_{\geq 2}'''(H)|}\left(\frac{d}{n}\right)^{|E_{\geq 2}(H)|-|E_{\geq 2}'''(H)|}.
\end{aligned}
\end{equation}
\item If $|E_{\geq 2}'''(H)|>\frac{\log n}{12\log\left(\frac{2}{\epsilon}\right)}$ or there exists $v\in V(H)$ such that $d_1^H(v)+d_{\geq 2}^H(v)> \Delta$, then for $n$ large enough, we have
\end{itemize}
\begin{equation}
\label{eq:eq_lem_main_lemma_pleasant_tight_bounds_2}
\begin{aligned}
|\mathbb{E}\big[\overline{\mathbf{Y}}_H\big]|\leq \frac{4}{n^{\frac{1}{12}}}\left(\frac{\epsilon d}{2n}\right)^{|E_1(H)|}\left(\frac{d}{n}\right)^{|E_{\geq 2}(H)|},
\end{aligned}
\end{equation}
where $P_s^H$ is as in \Cref{lem:lem_main_lemma_pleasant_wb}.
\end{lemma}
\begin{proof}
We have:
\begin{align*}
\mathbb{E}\big[\overline{\mathbf{Y}}_H\big]&=\mathbb{E}\big[\overline{\mathbf{Y}}_H\big|\mathcal{E}_{wb,H}\big]\cdot\mathbb{P}[\mathcal{E}_{wb,H}]+\mathbb{E}\big[\overline{\mathbf{Y}}_H\big|\mathcal{E}_{wb,H}^c\big]\cdot \mathbb{P}[\mathcal{E}_{wb,H}^c]\\
&=\mathbb{E}\left[\mathbb{E}\big[\overline{\mathbf{Y}}_H\big|\mathbf{x},\mathcal{E}_{wb,H}\big]\cdot P_{1,H}^{wb}(\mathbf{x})\right]+\mathbb{E}\left[\mathbb{E}\big[\overline{\mathbf{Y}}_H\big|\mathbf{x},\mathcal{E}_{wb,H}\big]\cdot P_{2,H}^{wb}(\mathbf{x})\right]\\
&\quad\quad\quad\quad\quad\quad\quad\quad\quad\quad\quad\quad\quad\quad\quad\quad\quad\quad\quad\quad\quad+\mathbb{E}\big[\overline{\mathbf{Y}}_H\big|\mathcal{E}_{wb,H}^c\big]\cdot\mathbb{P}[\mathcal{E}_{wb,H}^c],
\end{align*}
where the last equality follows from \cref{eq:eq_decompose_X_H_W_H}. By combining this with \Cref{lem:lem_upper_bound_nice_Wc}, \Cref{lem:lem_upper_bound_nice_W_2} and \Cref{lem:lem_main_lemma_pleasant_wb}, we get
\begin{equation}
\label{eq:eq_lem_main_lemma_pleasant_tight_bounds_3}
\begin{aligned}
&\left(P_s^H -\frac{2}{\sqrt{n}}-\frac{3}{n^{\frac{1}{6}}}\cdot\left(\frac{2}{\epsilon}\right)^{|E_{\geq 2}'''(H)|}\right) \left(\frac{\epsilon d}{2n}\right)^{|E_1(H)|+|E_{\geq 2}'''(H)|}\left(\frac{d}{n}\right)^{|E_{\geq 2}(H)|-|E_{\geq 2}'''(H)|}\\
&\quad\quad\quad\quad\leq \mathbb{E}\big[\overline{\mathbf{Y}}_H\big]\leq \left(P_s^H +\frac{2}{\sqrt{n}}+\frac{3}{n^{\frac{1}{6}}}\cdot\left(\frac{2}{\epsilon}\right)^{|E_{\geq 2}'''(H)|}\right) \left(\frac{\epsilon d}{2n}\right)^{|E_1(H)|+|E_{\geq 2}'''(H)|}\left(\frac{d}{n}\right)^{|E_{\geq 2}(H)|-|E_{\geq 2}'''(H)|},
\end{aligned}
\end{equation}

We distinguish between two cases:
\begin{itemize}
\item[(1)] If $|E_{\geq 2}'''(H)|\leq\frac{\log n}{12\log\left(\frac{2}{\epsilon}\right)}$, we have
$$\frac{3}{n^{\frac{1}{6}}}\cdot\left(\frac{2}{\epsilon}\right)^{|E_{\geq 2}'''(H)|}\leq \frac{3}{n^{\frac{1}{6}}}\cdot n^{\frac{1}{12}}=\frac{3}{n^{\frac{1}{12}}}.$$
In this case, we get 
\begin{equation}
\label{eq:eq_lem_main_lemma_pleasant_tight_bounds_4}
\begin{aligned}
\left(P_s^H -\frac{4}{n^{\frac{1}{12}}}\right) &\left(\frac{\epsilon d}{2n}\right)^{|E_1(H)|+|E_{\geq 2}'''(H)|}\left(\frac{d}{n}\right)^{|E_{\geq 2}(H)|-|E_{\geq 2}'''(H)|}\\
&\leq \mathbb{E}\big[\overline{\mathbf{Y}}_H\big]\leq \left(P_s^H +\frac{4}{n^{\frac{1}{12}}}\right) \left(\frac{\epsilon d}{2n}\right)^{|E_1(H)|+|E_{\geq 2}'''(H)|}\left(\frac{d}{n}\right)^{|E_{\geq 2}(H)|-|E_{\geq 2}'''(H)|},
\end{aligned}
\end{equation}
We will further split case (1) into two subcases:
\begin{itemize}
\item[(i)] If $d_1^H(v)+d_{\geq 2}^H(v)\leq \Delta$ for every $v\in V(H)$, then \cref{lem:lem_Prob_nice_safe_lower_bound} implies that for $n$ large enough, we have $P_s^H\geq \frac{1}{n^{\frac{2K}{A}}}$. Since $A>100 K$, we have
$$\frac{4}{n^{\frac{1}{12}}P_s^H}\leq \frac{4 n^{\frac{2K}{A}}}{n^{\frac{1}{12}}}\leq\frac{4n^{\frac{1}{50}}}{n^{\frac{1}{12}}}\leq\frac{1}{n^{\frac{1}{16}}},$$
where the last inequality is true for $n$ large enough. By combining this with \cref{eq:eq_lem_main_lemma_pleasant_tight_bounds_4}, we get \cref{eq:eq_lem_main_lemma_pleasant_tight_bounds_1}.
\item[(ii)] If there exists $v\in V(H)$ such that $d_1^H(v)+d_{\geq 2}^H(v)> \Delta$, then \cref{lem:lem_Prob_nice_safe_lower_bound} implies that $P_s^H=0$. By combining this with \cref{eq:eq_lem_main_lemma_pleasant_tight_bounds_4} and using the fact that $\frac{\epsilon}{2}\leq 1$, we get \cref{eq:eq_lem_main_lemma_pleasant_tight_bounds_2}.
\end{itemize}
\item[(2)] If $|E_{\geq 2}'''(H)|>\frac{\log n}{12\log\left(\frac{2}{\epsilon}\right)}$, we have
$$\left(\frac{\epsilon}{2}\right)^{|E_{\geq 2}'''(H)|}=\frac{1}{\left(\frac{2}{\epsilon}\right)^{|E_{\geq 2}'''(H)|}}\leq \frac{1}{n^{\frac{1}{12}}},$$
hence, \cref{eq:eq_lem_main_lemma_pleasant_tight_bounds_3} implies that
\begin{align*}
\big|\mathbb{E}\big[\overline{\mathbf{Y}}_H\big]\big|&\leq \left(\left(P_s^H +\frac{2}{\sqrt{n}}\right)\left(\frac{\epsilon}{2}\right)^{|E_{\geq 2}'''(H)|}+\frac{3}{n^{\frac{1}{6}}}\right) \left(\frac{\epsilon d}{2n}\right)^{|E_1(H)|}\left(\frac{d}{n}\right)^{|E_{\geq 2}(H)|}\\
&\leq \left(\left(P_s^H +\frac{2}{\sqrt{n}}\right)\frac{1}{n^{\frac{1}{12}}}+\frac{3}{n^{\frac{1}{6}}}\right) \left(\frac{\epsilon d}{2n}\right)^{|E_1(H)|}\left(\frac{d}{n}\right)^{|E_{\geq 2}(H)|}\\
&\leq \frac{4}{n^{\frac{1}{12}}} \left(\frac{\epsilon d}{2n}\right)^{|E_1(H)|}\left(\frac{d}{n}\right)^{|E_{\geq 2}(H)|}.
\end{align*}
\end{itemize}
\end{proof}

The following lemma provides an upper bound on the expectation $\mathbb{E}\big[\overline{\mathbf{Y}}_H\big]$ for a significant $(s,t)$-pleasant multigraph $H$ in terms of the expectation $\mathbb{E}\big[\overline{\mathbf{Y}}_C\big]$ for a cycle $C$ with $st$ edges of multiplicity 1.

\begin{lemma}
\label{lem:lem_lower_bound_nice}
Assume that $A>\max\left\{1,100K,\frac{100}{K}\right\}$, and let $H\in \nbsaw{s}{t}$.
\begin{itemize}
\item If $d_1^H(v)+d_{\geq 2}^H(v)\leq \Delta$ for every $v\in V(H)$, then for $n$ large enough, we have
$$\mathbb{E}\big[\overline{\mathbf{Y}}_H\big]\geq \frac{1}{2n^{\frac{2K}{A}}}\cdot \left(\frac{\epsilon d}{2n}\right)^{|E_1(H)|}\cdot\left(\frac{d}{n}\right)^{|E_{\geq 2}(H)|}.$$
\item If there exists $v\in V(H)$ such that $d_1^H(v)+d_{\geq 2}^H(v)> \Delta $, then for $n$ large enough, we have
$$\big|\mathbb{E}\big[\overline{\mathbf{Y}}_H\big]\big|\leq \frac{4}{n^{\frac{1}{12}}}\cdot \left(\frac{\epsilon d}{2n}\right)^{|E_1(H)|}\cdot\left(\frac{d}{n}\right)^{|E_{\geq 2}(H)|}.$$
\end{itemize}
\end{lemma}
\begin{proof}
This is a direct corollary of \Cref{lem:lem_main_lemma_pleasant_tight_bounds}, \Cref{lem:lem_Prob_nice_safe_lower_bound}, and the fact that every $H\in\nbsaw{s}{t}$ is an $(s,t)$-pleasant multigraph that satisfies $E_{\geq 2}'''(H)=\varnothing$.
\end{proof}


\subsection{Bounds for products of block self-avoiding-walks}\label{sec:bounds-products-bsaws}
\begin{lemma}
\label{lem:lem_E_Y_H_Product_pleasant}
Let $H=\Ho\oplus \Ht$, where $\Ho$ and $\Ht$ are two $(s,t)$-pleasant multigraphs such that $V(\Ho)\cap V(\Ht)=\varnothing$, and $t=K\log n$. Assume that $|E_{\geq 2}'''(H)|\leq\frac{\log n}{12\log\left(\frac{2}{\epsilon}\right)}$ and $d_1^H(v)+d_{\geq 2}^H(v)\leq \Delta$ for every $v\in V(H)$. If $A>\max\left\{1,200K,\frac{100}{K}\right\}$, then for $n$ large enough, we have
$$\mathbb{E}\big[\overline{\mathbf{Y}}_H\big]\leq\left(1+\frac{4}{n^{\frac{1}{16}}}\right)\cdot\mathbb{E}\big[\overline{\mathbf{Y}}_{\Ho}\big]\cdot \mathbb{E}\big[\overline{\mathbf{Y}}_{\Ht}\big].$$
\end{lemma}
\begin{proof}
From \cref{lem:lem_Prob_safety_pleasant_wb}, we have
\begin{equation}
\label{eq:eq_lem_E_Y_H_Product_pleasant_1}
\begin{aligned}
P_s^H&=\prod_{v\in V(H)}P_s\big(\Delta-d_1^H(v)-d_{\geq 2}^H(v)\big)\\
&=\left(\prod_{v\in V(\Ho)}P_s\big(\Delta-d_1^H(v)-d_{\geq 2}^H(v)\big)\right)\cdot \left(\prod_{v\in V(\Ht)}P_s\big(\Delta-d_1^H(v)-d_{\geq 2}^H(v)\big)\right)\\
&=\left(\prod_{v\in V(\Ho)}P_s\big(\Delta-d_1^{\Ho}(v)-d_{\geq 2}^{\Ho}(v)\big)\right)\cdot \left(\prod_{v\in V(\Ht)}P_s\big(\Delta-d_1^{\Ht}(v)-d_{\geq 2}^{\Ht}(v)\big)\right)\\
&= P_s^{\Ho}\cdot P_s^{\Ht}.
\end{aligned}
\end{equation}

Since $\Ho$ and $\Ht$ are $(s,t)$-pleasant and $V(\Ho)\cap V(\Ht)=\varnothing$, the multigraph $H=\Ho\oplus \Ht$ is $(s,2t)$-pleasant, i.e., $H$ is $(s,2K\log n)$-pleasant. Now since $V(H_1)\cap V(H_2)=\varnothing$, $|E_{\geq 2}'''(H)|\leq \frac{\log n}{12\log\left(\frac{2}{\epsilon}\right)}$ and $d_1^H(v)+d_{\geq 2}^H(v)\leq \Delta$ for every $v\in V(H)$, we have:
\begin{itemize}
\item For every $v\in V(\Ho)$, we have $d_1^{\Ho}(v)+d_{\geq 2}^{\Ho}(v)=d_1^{H}(v)+d_{\geq 2}^{H}(v)\leq \Delta$.
\item For every $v\in V(\Ht)$, we have $d_1^{\Ht}(v)+d_{\geq 2}^{\Ht}(v)=d_1^{H}(v)+d_{\geq 2}^{H}(v)\leq \Delta$.
\item Since $|E_{\geq 2}'''(H)|=|E_{\geq 2}'''(\Ho)|+|E_{\geq 2}'''(\Ht)|$, we have $|E_{\geq 2}'''(\Ho)|\leq  \frac{\log n}{12\log\left(\frac{2}{\epsilon}\right)}$ and $|E_{\geq 2}'''(\Ht)|\leq  \frac{\log n}{12\log\left(\frac{2}{\epsilon}\right)}$. 
\end{itemize}

Now if we apply \Cref{lem:lem_main_lemma_pleasant_tight_bounds} to $H,\Ho$ and $\Ht$ and use the fact that $|E_1(H)|=|E_1(\Ho)|+|E_1(\Ht)|$, $|E_{\geq 2}(H)|=|E_{\geq 2}(\Ho)|+|E_{\geq 2}(\Ht)|$ and $|E_{\geq 2}'''(H)|=|E_{\geq 2}'''(\Ho)|+|E_{\geq 2}'''(\Ht)|$, we get
\begin{align*}
\frac{\mathbb{E}\big[\overline{\mathbf{Y}}_H\big]}{\mathbb{E}\big[\overline{\mathbf{Y}}_{\Ho}\big]\cdot \mathbb{E}\big[\overline{\mathbf{Y}}_{\Ht}\big]}&\leq \frac{P_s^H\left(1+\frac{1}{n^{\frac{1}{16}}}\right)}{P_s^{\Ho}\left(1 -\frac{1}{n^{\frac{1}{16}}}\right)P_s^{\Ht}\left(1 -\frac{1}{n^{\frac{1}{16}}}\right)}\stackrel{(\ast)}{=} \frac{\left(1+\frac{1}{n^{\frac{1}{16}}}\right)}{\left(1 -\frac{1}{n^{\frac{1}{16}}}\right)\left(1 -\frac{1}{n^{\frac{1}{16}}}\right)}\stackrel{(\dagger)}{\leq} \left(1+\frac{4}{n^{\frac{1}{16}}}\right),
\end{align*}
where $(\ast)$ follows from \Cref{eq:eq_lem_E_Y_H_Product_pleasant_1} and $(\dagger)$ is true for $n$ is large enough.
\end{proof}

\begin{lemma}
\label{lem:lem_E_Y_H_Product_nbsaw_disjoint}
Let $H=\Ho\oplus \Ht$, where $\Ho,\Ht\in\nbsaw{s}{t}$ are such that $V(\Ho)\cap V(\Ht)=\varnothing$, and $t=K\log n$. If $A>\max\left\{1,200K,\frac{100}{K}\right\}$, then
\begin{itemize}
\item If $d_1^H(v)+d_{\geq 2}^H(v)\leq \Delta$ for every $v\in V(H)$, then for $n$ large enough, we have
$$\mathbb{E}\big[\overline{\mathbf{Y}}_H\big]\leq\left(1+\frac{4}{n^{\frac{1}{16}}}\right)\cdot\mathbb{E}\big[\overline{\mathbf{Y}}_{\Ho}\big]\cdot \mathbb{E}\big[\overline{\mathbf{Y}}_{\Ht}\big].$$
\item If there exists $v\in V(H)$ such that $d_1^H(v)+d_{\geq 2}^H(v)> \Delta$, then for $n$ large enough, we have
$$\big|\mathbb{E}\big[\overline{\mathbf{Y}}_H\big]\big|\leq \frac{4}{n^{\frac{1}{12}}}\cdot \left(\frac{\epsilon d}{2n}\right)^{|E_1(H)|}\cdot\left(\frac{d}{n}\right)^{|E_{\geq 2}(H)|}.$$
\end{itemize}
\end{lemma}
\begin{proof}
Since $\Ho,\Ht\in\nbsaw{s}{t}$ and $V(\Ho)\cap V(\Ht)=\varnothing$, we have $|E_{\geq 2}'''(H)|=\varnothing$. The lemma now follows immediately from \Cref{lem:lem_main_lemma_pleasant_tight_bounds} and \cref{lem:lem_E_Y_H_Product_pleasant}.
\end{proof}

Now we will study products of nice block self-avoiding-walks sharing a vertex.

\begin{lemma}
\label{lem:lem_E_Y_H_Product_nbsaw_sharing}
Let $H=\Ho\oplus \Ht$, where $\Ho,\Ht\in\nbsaw{s}{t}$ are such that $V(\Ho)\cap V(\Ht)\neq\varnothing$, $H$ is an $(s,2t)$-pleasant multigraph, and $t=K\log n$. If $A>\max\left\{1,200K,\frac{100}{K}\right\}$, then
\begin{itemize}
\item If $|E_{\geq 2}'''(H)|\leq\frac{\log n}{12\log\left(\frac{2}{\epsilon}\right)}$ and $d_1^H(v)+d_{\geq 2}^H(v)\leq \Delta$ for every $v\in V(H)$, then for $n$ large enough, we have
\begin{align*}
\mathbb{E}\big[\overline{\mathbf{Y}}_H\big]\leq\left(1+\frac{\eta}{2}\right)\cdot \left(\frac{4n\cdot e^{16d/\Delta}}{\epsilon^2 d(1-\eta)}\right)^{\left|E(\Ho)\cap E(\Ht)\right|}\cdot \mathbb{E}\big[\overline{\mathbf{Y}}_{\Ho}\big]\cdot \mathbb{E}\big[\overline{\mathbf{Y}}_{\Ht}\big],
\end{align*}
where $\eta$ is as in \cref{lem:lem_bound_prob_large_outside_deg}.
\item If $|E_{\geq 2}'''(H)|>\frac{\log n}{12\log\left(\frac{2}{\epsilon}\right)}$ or there exists $v\in V(H)$ such that $d_1^H(v)+d_{\geq 2}^H(v)> \Delta$, then for $n$ large enough, we have
$$\big|\mathbb{E}\big[\overline{\mathbf{Y}}_H\big]\big|\leq \frac{4}{n^{\frac{1}{12}}}\cdot \left(\frac{\epsilon d}{2n}\right)^{|E_1(H)|}\cdot\left(\frac{d}{n}\right)^{|E_{\geq 2}(H)|}.$$
\end{itemize}
\end{lemma}
\begin{proof}
If $|E_{\geq 2}'''(H)|>\frac{\log n}{12\log\left(\frac{2}{\epsilon}\right)}$ or there exists $v\in V(E_1(H))$ such that $d_1^H(v)+d_{\geq 2}^H(v)> \Delta$, then \Cref{lem:lem_main_lemma_pleasant_tight_bounds} implies that
\begin{align*}
\mathbb{E}\big[\overline{\mathbf{Y}}_H\big]&\leq \frac{4}{n^{\frac{1}{12}}}\cdot\left(\frac{\epsilon d}{2n}\right)^{|E_1(H)|}\left(\frac{d}{n}\right)^{|E_{\geq 2}(H)|}.
\end{align*}

Now assume that $|E_{\geq 2}'''(H)|\leq \frac{\log n}{12\log\left(\frac{2}{\epsilon}\right)}$ and $d_1^H(v)+d_{\geq 2}^H(v)\leq \Delta$ for every $v\in V(H)$. If we apply \Cref{lem:lem_main_lemma_pleasant_tight_bounds} to $H,\Ho$ and $\Ht$, and using the fact that $E'''_{\geq2}(\Ho)=E'''_{\geq2}(\Ht)=\varnothing$, we get
\begin{align*}
\frac{\mathbb{E}\big[\overline{\mathbf{Y}}_H\big]}{\mathbb{E}\big[\overline{\mathbf{Y}}_{\Ho}\big]\cdot \mathbb{E}\big[\overline{\mathbf{Y}}_{\Ht}\big]}&\leq \frac{P_s^H\left(1+\frac{1}{n^{\frac{1}{16}}}\right)}{P_s^{\Ho}\left(1 -\frac{1}{n^{\frac{1}{16}}}\right)P_s^{\Ht}\left(1 -\frac{1}{n^{\frac{1}{16}}}\right)}\left(\frac{\epsilon d}{2n}\right)^{|E_1(H)|+|E_{\geq 2}'''(H)|-|E_1(\Ho)|-|E_1(\Ht)|}\\
&\quad\quad\quad\quad\quad\quad\quad\quad\quad\quad\quad\quad\quad\times\left(\frac{d}{n}\right)^{|E_{\geq 2}(H)|-|E_{\geq 2}'''(H)|-|E_{\geq 2}(\Ho)|-|E_{\geq 2}(\Ht)|}.
\end{align*}

Now since $\epsilon\leq 2$, if $n$ is large enough, we get
\begin{equation}
\label{eq:eq_lem_E_Y_H_Product_nbsaw_sharing_1}
\begin{aligned}
\resizebox{0.95\textwidth}{!}{$\displaystyle
\frac{\mathbb{E}\big[\overline{\mathbf{Y}}_H\big]}{\mathbb{E}\big[\overline{\mathbf{Y}}_{\Ho}\big]\cdot \mathbb{E}\big[\overline{\mathbf{Y}}_{\Ht}\big]}\leq \left(1+\frac{4}{n^{\frac{1}{16}}}\right)\cdot \frac{P_s^H}{P_s^{\Ho}P_s^{\Ht}}\cdot\left(\frac{\epsilon d}{2n}\right)^{|E_1(H)|-|E_1(\Ho)|-|E_1(\Ht)|}\left(\frac{d}{n}\right)^{|E_{\geq 2}(H)|-|E_{\geq 2}(\Ho)|-|E_{\geq 2}(\Ht)|}$}.
\end{aligned}
\end{equation}

Now since $H=\Ho\oplus \Ht$, we have
$$E_1(H)\subseteq E_1(\Ho)\cup E_1(\Ht)\quad \text{and}\quad E_{\geq 2}(\Ho)\cup E_{\geq 2}(\Ht)\subseteq E_{\geq 2}(H).$$
Furthermore,
\begin{align*}
&\left(E_1(\Ho)\cup E_1(\Ht)\right)\setminus E_1(H)\\
&\quad\quad\quad=\left(E_1(\Ho)\cap E_1(\Ht)\right)\cup \left(E_1(\Ho)\cap E_{\geq 2}(\Ht)\right)\cup \left(E_{\geq 2}(\Ho)\cap E_1(\Ht)\right),
\end{align*}
which implies that,
\begin{equation}
\label{eq:eq_lem_E_Y_H_Product_nbsaw_sharing_2}
\begin{aligned}
|E_1(\Ho)|&+|E_1(\Ht)|-|E_1(H)|\\
&=|E_1(\Ho) \cap E_1(\Ht)| + |E_1(\Ho)\cap E_1(\Ht)|-|E_1(H)| \\
&=2\left|E_1(\Ho)\cap E_1(\Ht)\right|+ \left|E_1(\Ho)\cap E_{\geq 2}(\Ht)\right|+\left|E_{\geq 2}(\Ho)\cap E_1(\Ht)\right|.
\end{aligned}
\end{equation}
On the other hand,
\begin{align*}
E_{\geq 2}(H)\setminus\left(E_{\geq 2}(\Ho)\cup E_{\geq 2}(\Ht)\right)\;=\;\left(E_1(\Ho)\cap E_1(\Ht)\right),
\end{align*}
which implies that
\begin{equation}
\label{eq:eq_lem_E_Y_H_Product_nbsaw_sharing_3}
\begin{aligned}
|E_{\geq 2}(H)|&-|E_{\geq 2}(\Ho)|-|E_{\geq 2}(\Ht)|\\
&= |E_{\geq 2}(H)|-|E_{\geq 2}(\Ho)\cup E_{\geq 2}(\Ht)|-|E_{\geq 2}(\Ho)\cap E_{\geq 2}(\Ht)|\\
&= |E_1(\Ho)\cap E_1(\Ht)|-|E_{\geq 2}(\Ho)\cap E_{\geq 2}(\Ht)|
\end{aligned}
\end{equation}

By combining \cref{eq:eq_lem_E_Y_H_Product_nbsaw_sharing_2} and \cref{eq:eq_lem_E_Y_H_Product_nbsaw_sharing_3}, we get

\begin{align*}
&\left(\frac{\epsilon d}{2n}\right)^{|E_1(H)|-|E_1(\Ho)|-|E_1(\Ht)|}\left(\frac{d}{n}\right)^{|E_{\geq 2}(H)|-|E_{\geq 2}(\Ho)|-|E_{\geq 2}(\Ht)|}\\
&\leq \resizebox{0.9\textwidth}{!}{$\displaystyle\left(\frac{\epsilon d}{2n}\right)^{-2\left|E_1(\Ho)\cap E_1(\Ht)\right|- \left|E_1(\Ho)\cap E_{\geq 2}(\Ht)\right|-\left|E_{\geq 2}(\Ho)\cap E_1(\Ht)\right|}\cdot\left(\frac{d}{n}\right)^{|E_1(\Ho)\cap E_1(\Ht)|-|E_{\geq 2}(\Ho)\cap E_{\geq 2}(\Ht)|}$}\\
&= \resizebox{0.9\textwidth}{!}{$\displaystyle\left(\frac{\epsilon^2 d}{4n}\right)^{-\left|E_1(\Ho)\cap E_1(\Ht)\right|}\cdot \left(\frac{\epsilon d}{2n}\right)^{- \left|E_1(\Ho)\cap E_{\geq 2}(\Ht)\right|-\left|E_{\geq 2}(\Ho)\cap E_1(\Ht)\right|}\cdot \left(\frac{d}{n}\right)^{-|E_{\geq 2}(\Ho)\cap E_{\geq 2}(\Ht)|}$}\\
&\stackrel{(\ast)}{\leq} \left(\frac{\epsilon^2 d}{4n}\right)^{-\left|E_1(\Ho)\cap E_1(\Ht)\right|- \left|E_1(\Ho)\cap E_{\geq 2}(\Ht)\right|-\left|E_{\geq 2}(\Ho)\cap E_1(\Ht)\right|-|E_{\geq 2}(\Ho)\cap E_{\geq 2}(\Ht)|}\\
&= \left(\frac{\epsilon^2 d}{4n}\right)^{-\left|E(\Ho)\cap E(\Ht)\right|}.
\end{align*}
where the last inequality follows from the fact that $\epsilon\leq 2$. By combining this with \cref{eq:eq_lem_E_Y_H_Product_nbsaw_sharing_1}, we get
\begin{equation}
\label{eq:eq_lem_E_Y_H_Product_nbsaw_sharing_4}
\frac{\mathbb{E}\big[\overline{\mathbf{Y}}_H\big]}{\mathbb{E}\big[\overline{\mathbf{Y}}_{\Ho}\big]\cdot \mathbb{E}\big[\overline{\mathbf{Y}}_{\Ht}\big]}\leq \left(1+\frac{4}{n^{\frac{1}{16}}}\right)\cdot \frac{P_s^H}{P_s^{\Ho}P_s^{\Ht}}\cdot\left(\frac{4n}{\epsilon^2 d}\right)^{\left|E(\Ho)\cap E(\Ht)\right|}.
\end{equation}

Now from \cref{lem:lem_Prob_safety_pleasant_wb}, we have
\begin{align*}
P_s^H&=\prod_{v\in V(H)}P_s\big(\Delta-d_1^H(v)-d_{\geq 2}^H(v)\big)\\
&=\left(\prod_{v\in V(\Ho)\setminus V(\Ht)}P_s\big(\Delta-d_1^H(v)-d_{\geq 2}^H(v)\big)\right)\cdot \left(\prod_{v\in V(\Ht)\setminus V(\Ho)}P_s\big(\Delta-d_1^H(v)-d_{\geq 2}^H(v)\big)\right)\\
&\quad\quad\quad\quad\quad\quad\quad\quad\quad\quad\quad\quad\quad\quad\quad\quad\quad\quad\times \left(\prod_{v\in V(\Ho)\cap V(\Ht)}P_s\big(\Delta-d_1^H(v)-d_{\geq 2}^H(v)\big)\right).
\end{align*}

Now for every $v\in V(H)$, we have
$$d_1^H(v)+d_{\geq 2}^H(v)=d_{G(H)}(v)\geq d_{G(\Ho)}(v)=d_1^{\Ho}(v)+d_{\geq 2}^{\Ho}(v).$$
Similarly, we gave
$$d_1^H(v)+d_{\geq 2}^H(v)\geq d_1^{\Ht}(v)+d_{\geq 2}^{\Ht}(v).$$

From \cref{lem:lem_Prob_safety_pleasant_wb}, we know that the function $P_s$ is non-decreasing. Therefore,

\begin{align*}
P_s^H&\leq\left(\prod_{v\in V(\Ho)\setminus V(\Ht)}P_s\big(\Delta-d_1^{\Ho}(v)-d_{\geq 2}^{\Ho}\big)\right)\cdot \left(\prod_{v\in V(\Ht)\setminus V(\Ho)}P_s\big(\Delta-d_1^{\Ht}(v)-d_{\geq 2}^{\Ht}(v)\big)\right)\\
&\quad\quad\quad\quad\quad\times \left(\prod_{v\in V(\Ho)\cap V(\Ht)}\min\left\{P_s\big(\Delta-d_1^{\Ho}(v)-d_{\geq 2}^{\Ho}\big),P_s\big(\Delta-d_1^{\Ht}(v)-d_{\geq 2}^{\Ht}(v)\big)\right\}\right)\\
&=\frac{\resizebox{0.7\textwidth}{!}{$\displaystyle\left(\prod_{v\in V(\Ho)}P_s\big(\Delta-d_1^{\Ho}(v)-d_{\geq 2}^{\Ho}\big)\right)\cdot \left(\prod_{v\in V(\Ht)}P_s\big(\Delta-d_1^{\Ht}(v)-d_{\geq 2}^{\Ht}(v)\big)\right)$}}{\resizebox{0.7\textwidth}{!}{$\displaystyle\prod_{v\in V(\Ho)\cap V(\Ht)}\max\left\{P_s\big(\Delta-d_1^{\Ho}(v)-d_{\geq 2}^{\Ho}\big),P_s\big(\Delta-d_1^{\Ht}(v)-d_{\geq 2}^{\Ht}(v)\big)\right\}$}}\\
&=\frac{P_s^{\Ho} P_s^{\Ht}}{\resizebox{0.7\textwidth}{!}{$\displaystyle\prod_{v\in V(\Ho)\cap V(\Ht)}\max\left\{P_s\big(\Delta-d_1^{\Ho}(v)-d_{\geq 2}^{\Ho}\big),P_s\big(\Delta-d_1^{\Ht}(v)-d_{\geq 2}^{\Ht}(v)\big)\right\}$}},
\end{align*}
where the last equality follows from applying \cref{lem:lem_Prob_safety_pleasant_wb} to $\Ho$ and $\Ht$. Therefore,

\begin{equation}
\label{eq:eq_lem_E_Y_H_Product_nbsaw_sharing_5}
\frac{P_s^H}{P_s^{\Ho}P_s^{\Ht}}\leq \prod_{v\in V(\Ho)\cap V(\Ht)}F_v,
\end{equation}
where
$$F_v=\min\left\{\frac{1}{P_s\big(\Delta-d_1^{\Ho}(v)-d_{\geq 2}^{\Ho}\big)},\frac{1}{P_s\big(\Delta-d_1^{\Ht}(v)-d_{\geq 2}^{\Ht}(v)\big)}\right\}.$$

Now for every vertex $v\in V(\Ho)\cap V(\Ht)$, we have:
\begin{itemize}
\item If $d_1^{\Ho}(v)+d_{\geq 2}^{\Ho}\big)\leq \frac{3\Delta}{4}$ or $d_1^{\Ht}(v)+d_{\geq 2}^{\Ht}\big)\leq \frac{3\Delta}{4}$, then \cref{lem:lem_Prob_safety_pleasant_wb} implies that
\begin{equation}
\label{eq:eq_lem_E_Y_H_Product_nbsaw_sharing_6}
F_v\leq\frac{1}{1-\frac{\eta}{2}}.
\end{equation}
\item If $d_1^{\Ho}(v)+d_{\geq 2}^{\Ho}\big)> \frac{3\Delta}{4}$ and $d_1^{\Ht}(v)+d_{\geq 2}^{\Ht}\big)> \frac{3\Delta}{4}$, then \cref{lem:lem_Prob_safety_pleasant_wb} implies that
\begin{equation}
\label{eq:eq_lem_E_Y_H_Product_nbsaw_sharing_7}
F_v\leq e^{4d}.
\end{equation}
Now for every set $\mathsf{E}$ of edges, let $d_{\mathsf{E}}(v)$ be the number of edges in $\mathsf{E}$ which are incident to $v$. We have:
\begin{align*}
d_{E(\Ho)\cup E(\Ht)}(v)=d_{G(H)}(v) = d_1^{H}(v)+d_{\geq 2}^{H}(v)\leq \Delta,
\end{align*}
\begin{align*}
d_{E(\Ho)}(v)=d_{G(\Ho)}(v) = d_1^{\Ho}(v)+d_{\geq 2}^{\Ho}(v)> \frac{3\Delta}{4},
\end{align*}
and
\begin{align*}
d_{E(\Ht)}(v)=d_{G(\Ht)}(v) = d_1^{\Ht}(v)+d_{\geq 2}^{\Ht}(v)> \frac{3\Delta}{4}.
\end{align*}
Therefore, if $d_1^{\Ho}(v)+d_{\geq 2}^{\Ho}\big)> \frac{3\Delta}{4}$ and $d_1^{\Ht}(v)+d_{\geq 2}^{\Ht}\big)> \frac{3\Delta}{4}$, we have
\begin{equation}
\label{eq:eq_lem_E_Y_H_Product_nbsaw_sharing_8}
d_{E(\Ho)\cap E(\Ht)}(v)=d_{E(\Ho)}(v)+d_{E(\Ht)}(v)-d_{E(\Ho)\cup E(\Ht)}(v)\geq\frac{\Delta}{2}.
\end{equation}
\end{itemize}
Now define the sets
$$V_{\Ho,\Ht}^{d=0}=\Big\{v\in V(\Ho)\cap V(\Ht):\; d_{E(\Ho)\cap E(\Ht)}(v)=0\Big\}$$
and
$$V_{\Ho,\Ht}^{d>0}=\Big\{v\in V(\Ho)\cap V(\Ht):\; d_{E(\Ho)\cap E(\Ht)}(v)>0\Big\}.$$

Clearly, $\left\{V_{\Ho,\Ht}^{d=0},V_{\Ho,\Ht}^{d>0}\right\}$ is a partition of $V(\Ho)\cap V(\Ht)$. Therefore, from \cref{eq:eq_lem_E_Y_H_Product_nbsaw_sharing_5} and  \cref{eq:eq_lem_E_Y_H_Product_nbsaw_sharing_6}, we get

\begin{equation}
\label{eq:eq_lem_E_Y_H_Product_nbsaw_sharing_9}
\begin{aligned}
\frac{P_s^H}{P_s^{\Ho}P_s^{\Ht}}&\leq \Paren{\prod_{v\in V_{\Ho,\Ht}^{d=0}}F_v}\cdot \Paren{\prod_{v\in V_{\Ho,\Ht}^{d>0}}F_v}\leq \frac{1}{\left(1-\frac{\eta}{2}\right)^{\big|V_{\Ho,\Ht}^{d=0}\big|}}\cdot \prod_{v\in V_{\Ho,\Ht}^{d>0}}F_v\\
&=\frac{1}{\left(1-\frac{\eta}{2}\right)^{\big|V_{\Ho,\Ht}^{d=0}\big|}}\cdot \prod_{v\in V_{\Ho,\Ht}^{d>0}}\tilde{F}^{d_{E(\Ho)\cap E(\Ht)}(v)},
\end{aligned}
\end{equation}
where
$$\tilde{F}=\max_{v\in V_{\Ho,\Ht}^{d>0}}F_v^{1/d_{E(\Ho)\cap E(\Ht)}(v)}.$$
Now from \cref{eq:eq_lem_E_Y_H_Product_nbsaw_sharing_6}, \cref{eq:eq_lem_E_Y_H_Product_nbsaw_sharing_7} and \cref{eq:eq_lem_E_Y_H_Product_nbsaw_sharing_8}, we get
\begin{align*}
\tilde{F}\leq \max\left\{\frac{1}{1-\frac{\eta}{2}},e^{\frac{4d}{\Delta/2}}\right\}\leq \frac{e^{8d/\Delta}}{1-\frac{\eta}{2}}.
\end{align*}
On the other hand, since $$V_{\Ho,\Ht}^{d>0}=\bigcup_{uv\in E(\Ho)\cap E(\Ht)}\{u,v\},$$ it is easy to see that
\begin{align*}
\prod_{v\in V_{\Ho,\Ht}^{d>0}}\tilde{F}^{d_{E(\Ho)\cap E(\Ht)}(v)}&=\tilde{F}^{2|E(\Ho)\cap E(\Ht)|}\\
&\leq \Paren{ \frac{e^{8d/\Delta}}{1-\frac{\eta}{2}}}^{2|E(\Ho)\cap E(\Ht)|}\leq \Paren{ \frac{e^{16d/\Delta}}{1-\eta}}^{|E(\Ho)\cap E(\Ht)|}.
\end{align*}

By combining this with \cref{eq:eq_lem_E_Y_H_Product_nbsaw_sharing_4} and \cref{eq:eq_lem_E_Y_H_Product_nbsaw_sharing_9}, we get

\begin{equation}
\label{eq:eq_lem_E_Y_H_Product_nbsaw_sharing_10}
\frac{\mathbb{E}\big[\overline{\mathbf{Y}}_H\big]}{\mathbb{E}\big[\overline{\mathbf{Y}}_{\Ho}\big]\cdot \mathbb{E}\big[\overline{\mathbf{Y}}_{\Ht}\big]}\leq \frac{\left(1+\frac{4}{n^{\frac{1}{16}}}\right)}{\left(1-\frac{\eta}{2}\right)^{\big|V_{\Ho,\Ht}^{d=0}\big|}}\cdot \left(\frac{4n\cdot e^{16d/\Delta}}{\epsilon^2 d(1-\eta)}\right)^{\left|E(\Ho)\cap E(\Ht)\right|}.
\end{equation}

Now since $\Ho,\Ht\in \nbsaw{s}{t}$ and $H=\Ho\oplus \Ht$ is $(s,t)$-pleasant, we can have at most one vertex in $V_{\Ho,\Ht}^{d=0}$. Therefore,
\begin{align*}
\frac{\mathbb{E}\big[\overline{\mathbf{Y}}_H\big]}{\mathbb{E}\big[\overline{\mathbf{Y}}_{\Ho}\big]\cdot \mathbb{E}\big[\overline{\mathbf{Y}}_{\Ht}\big]}&\leq \frac{\left(1+\frac{4}{n^{\frac{1}{16}}}\right)}{1-\frac{\eta}{2}}\cdot \left(\frac{4n\cdot e^{16d/\Delta}}{\epsilon^2 d(1-\eta)}\right)^{\left|E(\Ho)\cap E(\Ht)\right|}\\
&\leq \left(1+\frac{\eta}{2}\right)\cdot \left(\frac{4n\cdot e^{16d/\Delta}}{\epsilon^2 d(1-\eta)}\right)^{\left|E(\Ho)\cap E(\Ht)\right|},
\end{align*}
where the last inequality is true for $n$ large enough.
\end{proof}

\begin{lemma}
\label{lem:lem_E_Y_H_Product_nbsaw_sharing_concentration}
Let $H=H'\oplus H''$, where $H'=\Ho'\oplus \Ht'$ and $H''=\Ho''\oplus \Ht''$ for some $\Ho',\Ht',\Ho',\Ht'\in\nbsaw{s}{t}$, where $t=K\log n$. Assume that
\begin{itemize}
\item $V(H')\cap V(H'')=\varnothing$, $V(\Ho')\cap V(\Ht')\neq\varnothing$ and $V(\Ho'')\cap V(\Ht'')\neq\varnothing$.
\item $H'$ and $H''$ are $(s,2t)$-pleasant multigraphs.
\end{itemize}
If $A>\max\left\{1,400K,\frac{100}{K}\right\}$, then
\begin{itemize}
\item If $|E_{\geq 2}'''(H)|\leq\frac{\log n}{12\log\left(\frac{2}{\epsilon}\right)}$ and $d_1^H(v)+d_{\geq 2}^H(v)\leq \Delta$ for every $v\in V(H)$, then for $n$ large enough, we have
$$\mathbb{E}\big[\overline{\mathbf{Y}}_H\big]\leq\left(1+\frac{4}{n^{\frac{1}{16}}}\right)\cdot\mathbb{E}\big[\overline{\mathbf{Y}}_{\Ho}\big]\cdot \mathbb{E}\big[\overline{\mathbf{Y}}_{\Ht}\big].$$
\item If $|E_{\geq 2}'''(H)|>\frac{\log n}{12\log\left(\frac{2}{\epsilon}\right)}$ or there exists $v\in V(H)$ such that $d_1^H(v)+d_{\geq 2}^H(v)> \Delta$, then for $n$ large enough, we have
$$\big|\mathbb{E}\big[\overline{\mathbf{Y}}_H\big]\big|\leq \frac{4}{n^{\frac{1}{12}}}\cdot \left(\frac{\epsilon d}{2n}\right)^{|E_1(H)|}\cdot\left(\frac{d}{n}\right)^{|E_{\geq 2}(H)|}.$$
\end{itemize}
\end{lemma}
\begin{proof}
The lemma follows immediately from \Cref{lem:lem_main_lemma_pleasant_tight_bounds} and \cref{lem:lem_E_Y_H_Product_pleasant}.
\end{proof}

\section{Bounds for the centered matrix}\label{sec:appendix-upper-bound-centered}

In this section, we will study $\left|\mathbb{E}\left[\Tr\left(\big(Q^{(s)}\big(\overline{\mathbf{Y}}\big)-\mathbf{x}\transpose{\mathbf{x}}\big)^t\right)\right]\right|$.

\begin{definition}
For every $H\in\bsaw{s}{t}$, we denote the $s$ self-avoiding-walks that form $H$ as $W_1(H),\ldots,W_t(H)$, and we denote the set $\big\{W_1(H),\ldots,W_t(H)\big\}$ as $\mathcal{W}(H)$.
\end{definition}

\begin{definition}
For every self-avoiding-walk $W$, define $$\mathbf{x}_W=\prod_{uv\in W}\mathbf{x}_u\mathbf{x}_v.$$
It is easy to see that $\mathbf{x}_W=\mathbf{x}_i\mathbf{x}_j$, where $i$ and $j$ are the end-vertices of $W$. In other words, for every $i,j\in[n]$ and every $W\in\saw{ij}{s}$, we have $\mathbf{x}_W=\mathbf{x}_i\mathbf{x}_j$.
\end{definition}

Let us now analyze $\Tr\left(\big(Q^{(s)}\big(\overline{\mathbf{Y}}\big)-\mathbf{x}\transpose{\mathbf{x}}\big)^t\right)$:
\begin{align*}
\Tr\left(\big(Q^{(s)}\big(\overline{\mathbf{Y}}\big)-\mathbf{x}\transpose{\mathbf{x}}\big)^t\right)
&=\sum_{\substack{i_1,\ldots,i_{t+1}\in[n]:\\i_{t+1}=i_1}}\prod_{l\in[t]} \left[\left(\frac{1}{|\saw{i_li_{l+1}}{s}|}\left(\frac{2n}{\epsilon\cdot d}\right)^s\sum_{W\in\saw{i_li_l+1}{s}} \overline{\mathbf{Y}}_W\right)-\mathbf{x}_{i_li_{l+1}}\right]
\\
&=\sum_{\substack{i_1,\ldots,i_{t+1}\in[n]:\\i_{t+1}=i_1}}\prod_{l\in[t]} \left[\frac{1}{|\saw{i_li_{l+1}}{s}|}\left(\frac{2n}{\epsilon\cdot d}\right)^s\cdot \sum_{W\in\saw{i_li_l+1}{s}} \left(\overline{\mathbf{Y}}_W-\left(\frac{\epsilon d}{2n}\right)^s\mathbf{x}_{i_li_{l+1}}\right)\right]
\\
&=\left[\frac{(n-s+1)!}{n!}\left(\frac{2n}{\epsilon\cdot d}\right)^s\right]^t\cdot\sum_{\substack{i_1,\ldots,i_{t+1}\in[n]:\\i_{t+1}=i_1}}\prod_{l\in[t]} \left[\sum_{W\in\saw{i_li_l+1}{s}} \left(\overline{\mathbf{Y}}_W-\left(\frac{\epsilon d}{2n}\right)^s\mathbf{x}_W\right)\right].
\end{align*}
Therefore,
\begin{equation}
\label{eq:eq_trace_centered}
\begin{aligned}
\Tr\left(\big(Q^{(s)}\big(\overline{\mathbf{Y}}\big)-\mathbf{x}\transpose{\mathbf{x}}\big)^t\right)&=(1\pm o(1))\cdot n^t\left(\frac{2}{\epsilon\cdot d}\right)^{st}\cdot\sum_{\substack{i_1,\ldots,i_{t+1}\in[n]:\\i_{t+1}=i_1}}\sum_{\substack{W_1\in\saw{i_1i_2}{s},\\W_2\in\saw{i_2i_3}{s},\\\ldots,\\W_t\in\saw{i_ti_{t+1}}{s}}}\prod_{l\in[t]}\left[\overline{\mathbf{Y}}_W-\left(\frac{\epsilon d}{2n}\right)^s\mathbf{x}_W\right]
\\
&=(1\pm o(1))\cdot n^t\left(\frac{2}{\epsilon\cdot d}\right)^{st}\cdot\sum_{H\in\bsaw{s}{t}}\prod_{l\in[t]}\left[\overline{\mathbf{Y}}_{W_l(H)}-\left(\frac{\epsilon d}{2n}\right)^s\mathbf{x}_{W_l(H)}\right]
\\
&=(1\pm o(1))\cdot n^t\left(\frac{2}{\epsilon\cdot d}\right)^{st}\cdot\sum_{H\in\bsaw{s}{t}}\hat{\mathbf{Y}}_{H},
\\
\end{aligned}
\end{equation}
where
\begin{equation}
\label{eq:eq_def_X_hat}
\hat{\mathbf{Y}}_{H}:=\prod_{W\in\mathcal{W}(H)}\left[\overline{\mathbf{Y}}_{W}-\left(\frac{\epsilon d}{2n}\right)^s\mathbf{x}_{W}\right].
\end{equation}

If we compare \cref{eq:eq_trace_centered} with the non-centered case, we can see that $\Tr\left(\big(Q^{(s)}\big(\overline{\mathbf{Y}}\big)-\mathbf{x}\transpose{\mathbf{x}}\big)^t\right)$ has the same expression as $\Tr\left(Q^{(s)}\big(\overline{\mathbf{Y}}\big)^t\right)$, except that $\overline{\mathbf{Y}}_{H}$ is replaced with $\hat{\mathbf{Y}}_{H}$. The next section is dedicated for the proof of an upper bound on $\big|\mathbb{E}\big[\hat{\mathbf{Y}}_{H}\big|\mathbf{x}\big]\big|$ for every $H\in\bsaw{s}{t}$. This will allow us to prove an upper bound on $\left|\mathbb{E}\left[\Tr\left(\big(Q^{(s)}\big(\overline{\mathbf{Y}}\big)-\mathbf{x}\transpose{\mathbf{x}}\big)^t\right)\right]\right|$.

\subsection{An upper bound for every block self-avoiding-walk}

Before proving the upper bound on $\big|\mathbb{E}\big[\hat{\mathbf{Y}}_H\big|\mathbf{x}\big]\big|$, we will first informally analyze the same quantity but for the non-truncated case. The main reason for doing so is that the non-truncated case is much simpler, and the analysis of the non-truncated case contains many of the elements of the proof of the truncated case. Therefore, understanding the non-truncated case will be helpful later in understanding the much more complicated truncated case.

\subsubsection{Warm-up with the non-truncated case}

We would like to upper bound the following quantity:
$$\left|\mathbb{E}\left[\prod_{W\in\mathcal{W}(H)}\left(\mathbf{Y}_{W}-\left(\frac{\epsilon d}{2n}\right)^s\mathbf{x}_{W}\right)\middle|\mathbf{x}\right]\right|.$$

The following lemma shows that the expectation of every multiplicand in the above expression is exactly zero. 

\begin{lemma}
\label{lem:lem_expectation_walk_multiplicity_1}
For every $i,j\in[n]$ and every $W\in\saw{ij}{s}$, we have
$$\mathbb{E}[\mathbf{Y}_W|\mathbf{x}]=\left(\frac{\epsilon d}{2n}\right)^s \mathbf{x}_W=\left(\frac{\epsilon d}{2n}\right)^s \mathbf{x}_i\mathbf{x}_j.$$
\end{lemma}
\begin{proof}
We have:
\begin{align*}
\mathbb{E}[\mathbf{Y}_W|\mathbf{x}]&=\mathbb{E}\left[\prod_{uv\in W} \mathbf{Y}_{uv}\middle|\mathbf{x}\right]=\prod_{uv\in W} \mathbb{E}\left[\mathbf{Y}_{uv}\middle|\mathbf{x}\right]=\prod_{uv\in W}\left(\frac{\epsilon d \mathbf{x}_u\mathbf{x}_v}{2n}\right) \\
&=\left(\frac{\epsilon d}{2n}\right)^s\prod_{uv\in W}\mathbf{x}_u\mathbf{x}_v=\left(\frac{\epsilon d}{2n}\right)^s \mathbf{x}_W=\left(\frac{\epsilon d}{2n}\right)^s \mathbf{x}_i\mathbf{x}_j.
\end{align*}
\end{proof}

\Cref{lem:lem_expectation_walk_multiplicity_1} implies that for every $W\in \mathcal{W}(H)$, we have:
$$\mathbb{E}\left[\mathbf{Y}_{W}-\left(\frac{\epsilon d}{2n}\right)^s\mathbf{x}_{W}\middle|\mathbf{x}\right]=0.$$

Therefore, if there is one walk $W'\in\mathcal{W}(H)$ such that all the edges in $W'$ have multiplicity 1 in $H$, then $\displaystyle \mathbf{Y}_{W'}-\left(\frac{\epsilon d}{2n}\right)^s\mathbf{x}_{W'}$ is conditionally independent from $\displaystyle \left(\mathbf{Y}_{W''}-\left(\frac{\epsilon d}{2n}\right)^s\mathbf{x}_{W''}\right)_{W''\in\mathcal{W}(H)\setminus\{W'\}}$, hence
\begin{equation}
\label{eq:eq_walk_mult_1_non_truncated}
\begin{aligned}
&\mathbb{E}\left[\prod_{W\in\mathcal{W}(H)}\left(\mathbf{Y}_{W}-\left(\frac{\epsilon d}{2n}\right)^s\mathbf{x}_{W}\right)\middle|\mathbf{x}\right]\\
&\quad\quad\quad\quad\quad\quad=\mathbb{E}\left[\mathbf{Y}_{W'}-\left(\frac{\epsilon d}{2n}\right)^s\mathbf{x}_{W'}\middle|\mathbf{x}\right]\cdot \mathbb{E}\left[\prod_{W''\in\mathcal{W}(H)\setminus\{W'\}}\left(\mathbf{Y}_{W''}-\left(\frac{\epsilon d}{2n}\right)^s\mathbf{x}_{W''}\right)\middle|\mathbf{x}\right]=0.
\end{aligned}
\end{equation}

On the other hand, if every walk $W\in\mathcal{W}(H)$ has at least one edge of multiplicity $\geq\hspace*{-1.2mm}2$, we can do the following:
\begin{align*}
&\mathbb{E}\left[\prod_{W\in\mathcal{W}(H)}\left(\mathbf{Y}_{W}-\left(\frac{\epsilon d}{2n}\right)^s\mathbf{x}_{W}\right)\middle|\mathbf{x}\right]\\
&\quad\quad\quad\quad\quad\quad\quad\quad\quad=\resizebox{0.65\textwidth}{!}{$\displaystyle\sum_{\mathcal{W}\subseteq\mathcal{W}(H)}(-1)^{|\mathcal{W}(H)|-|\mathcal{W}|}\cdot\mathbb{E}\left[\left(\prod_{W\in\mathcal{W}}\mathbf{Y}_{W}\right)\cdot\left( \prod_{W\in \mathcal{W}(H)\setminus \mathcal{W}}\left(\frac{\epsilon d}{2n}\right)^s\mathbf{x}_{W}\right)\middle|\mathbf{x}\right]$}\\
&\quad\quad\quad\quad\quad\quad\quad\quad\quad=\sum_{\mathcal{W}\subseteq\mathcal{W}(H)}(-1)^{t-|\mathcal{W}|}\cdot\left( \prod_{W\in \mathcal{W}(H)\setminus \mathcal{W}}\left(\frac{\epsilon d}{2n}\right)^s\mathbf{x}_{W}\right)\cdot\mathbb{E}\left[\prod_{W\in\mathcal{W}}\mathbf{Y}_{W}\middle|\mathbf{x}\right].
\end{align*}
Therefore,
\begin{equation}
\label{eq:eq_centered_non_truncated}
\begin{aligned}
\left|\mathbb{E}\left[\prod_{W\in\mathcal{W}(H)}\left(\mathbf{Y}_{W}-\left(\frac{\epsilon d}{2n}\right)^s\mathbf{x}_{W}\right)\middle|\mathbf{x}\right]\right|&\leq\sum_{\mathcal{W}\subseteq\mathcal{W}(H)}\left(\frac{\epsilon d}{2n}\right)^{s(|\mathcal{W}(H)|-|\mathcal{W}|)}\cdot\left|\mathbb{E}\left[\prod_{W\in\mathcal{W}}\mathbf{Y}_{W}\middle|\mathbf{x}\right]\right|\\
&=\sum_{\mathcal{W}\subseteq\mathcal{W}(H)}\left(\frac{\epsilon d}{2n}\right)^{s(t-|\mathcal{W}|)}\cdot\big|\mathbb{E}\big[\mathbf{Y}_{H_{\mathcal{W}}}\big|\mathbf{x}\big]\big|,
\end{aligned}
\end{equation}
where
$$H_{\mathcal{W}}=\bigoplus_{W\in\mathcal{W}}W.$$

Recall that in the non-truncated case, for every edge $uv\in E(H)$, we have
$$\mathbb{E}\big[\mathbf{Y}_{uv}^{m_H(uv)}\big|\mathbf{x}\big]=\begin{cases}\frac{\epsilon \mathbf{x}_u\mathbf{x}_v d}{2n}\quad&\text{if }m_H(uv)=1,\\\left(1+\frac{\epsilon \mathbf{x}_u\mathbf{x}_v}{2}\right)\frac{d}{n}+O\left(\frac{1}{n^2}\right)\quad&\text{if }m_H(uv)\geq 2.\end{cases}$$

Let $\mathcal{W}\in\mathcal{W}(H)$ and let $W\in\mathcal{W}$. Define $\mathcal{W}'=\mathcal{W}\setminus\{W\}$. We have the following possibilities:
\begin{itemize}
\item[(a)] $W\cap E(H_{\mathcal{W}'})=\varnothing$, in which case we have
$$\mathbb{E}\big[\mathbf{Y}_{H_{\mathcal{W}}}\big|\mathbf{x}\big]=\mathbb{E}\big[\mathbf{Y}_{W}\big|\mathbf{x}\big]\cdot \mathbb{E}\big[\mathbf{Y}_{H_{\mathcal{W}'}}\big|\mathbf{x}\big]=\left(\frac{\epsilon d}{2n}\right)^s \mathbf{x}_W\cdot \mathbb{E}\big[\mathbf{Y}_{H_{\mathcal{W}'}}\big|\mathbf{x}\big],$$
which implies that
\begin{equation}
\label{eq:eq_non_truncated_centered_mult_1_walk_equality}
\left(\frac{\epsilon d}{2n}\right)^{s(t-|\mathcal{W}|)}\cdot\big|\mathbb{E}\big[\mathbf{Y}_{H_{\mathcal{W}}}\big|\mathbf{x}\big]\big|=\left(\frac{\epsilon d}{2n}\right)^{s(t-|\mathcal{W}'|)}\cdot\big|\mathbb{E}\big[\mathbf{Y}_{H_{\mathcal{W}'}}\big|\mathbf{x}\big]\big|.
\end{equation}
\item[(b)] $W\cap E(H_{\mathcal{W}'})\neq\varnothing$. In this case, let $\mathsf{E}=W\cap E(H_{\mathcal{W}'})$ and partition $\mathsf E$ into 
$$\mathsf E_1=W\cap E_1(H_{\mathcal{W}'})=\big\{e\in W\cap E(H_{\mathcal{W}'}):\; e\text{ has multiplicity 1 in }H_{\mathcal{W}'}\big\},$$
and
$$\mathsf E_{\geq 2}=W\cap E_{\geq 2}(H_{\mathcal{W}'})=\big\{e\in W\cap E(H_{\mathcal{W}'}):\; e\text{ has multiplicity at least 2 in }H_{\mathcal{W}'}\big\}.$$
It is easy to see that we have:
\begin{align*}
\big|\mathbb{E}\big[\mathbf{Y}_{H_{\mathcal{W}}}\big|\mathbf{x}\big]\big|&=(1\pm o(1))\cdot\left(\frac{\epsilon d}{2n}\right)^{|W|-|\mathsf E|}\cdot \left[\prod_{uv\in \mathsf E_1}\left[\left(1+\frac{\epsilon \mathbf{x}_u\mathbf{x}_v}{2}\right)\frac{d}{n}+O\left(\frac{1}{n^2}\right)\right]\right]\cdot\frac{\big|\mathbb{E}\big[\mathbf{Y}_{H_{\mathcal{W}'}}\big|\mathbf{x}\big]\big|}{\left(\frac{\epsilon d}{2n}\right)^{|\mathsf E_1|}}\\
&=(1\pm o(1))\left(\frac{\epsilon d}{2n}\right)^{s-|\mathsf E_1|-|\mathsf E|}\cdot \left[\prod_{uv\in \mathsf E_1}\left(1+\frac{\epsilon \mathbf{x}_u\mathbf{x}_v}{2}\right)\frac{d}{n}\right]\cdot\big|\mathbb{E}\big[\mathbf{Y}_{H_{\mathcal{W}'}}\big|\mathbf{x}\big]\big|.
\end{align*}
Therefore,
\begin{align*}
\left(\frac{\epsilon d}{2n}\right)^{s(t-|\mathcal{W}|)}&\cdot\big|\mathbb{E}\big[\mathbf{Y}_{H_{\mathcal{W}}}\big|\mathbf{x}\big]\big|\\
&=(1\pm o(1))\left(\frac{\epsilon d}{2n}\right)^{s(t-|\mathcal{W}|+1)-|\mathsf E_1|-|\mathsf E|}\cdot \left[\prod_{uv\in \mathsf E_1}\left(1+\frac{\epsilon \mathbf{x}_u\mathbf{x}_v}{2}\right)\frac{d}{n}\right]\cdot\big|\mathbb{E}\big[\mathbf{Y}_{H_{\mathcal{W}'}}\big|\mathbf{x}\big]\big|\\
&=(1\pm o(1))\left(\frac{\epsilon d}{2n}\right)^{s(t-|\mathcal{W}'|)-|\mathsf E_1|-|\mathsf E|}\cdot \left[\prod_{uv\in \mathsf E_1}\left(1+\frac{\epsilon \mathbf{x}_u\mathbf{x}_v}{2}\right)\frac{d}{n}\right]\cdot\big|\mathbb{E}\big[\mathbf{Y}_{H_{\mathcal{W}'}}\big|\mathbf{x}\big]\big|\\
&\geq (1\pm o(1))\left(\frac{\epsilon d}{2n}\right)^{-|\mathsf E_1|-|\mathsf E|}\cdot \left(\left(1-\frac{\epsilon}{2}\right)\frac{d}{n}\right)^{|\mathsf E_1|}\cdot\left(\frac{\epsilon d}{2n}\right)^{s(t-|\mathcal{W}'|)}\cdot \big|\mathbb{E}\big[\mathbf{Y}_{H_{\mathcal{W}'}}\big|\mathbf{x}\big]\big|,
\end{align*}
which implies that
\begin{equation}
\label{eq:eq_non_truncated_O_epsilon_n}
\begin{aligned}
&\left(\frac{\epsilon d}{2n}\right)^{s(t-|\mathcal{W}'|)}\cdot \big|\mathbb{E}\big[\mathbf{Y}_{H_{\mathcal{W}'}}\big|\mathbf{x}\big]\big|\\
&\quad\quad\quad\quad\quad\quad\leq (1\pm o(1))\left(\frac{\frac{\epsilon}{2}}{1-\frac{\epsilon}{2}}\right)^{|\mathsf E_1|}\cdot\left(\frac{\epsilon d}{2n}\right)^{|\mathsf E|}\cdot \left(\frac{\epsilon d}{2n}\right)^{s(t-|\mathcal{W}|)}\cdot\big|\mathbb{E}\big[\mathbf{Y}_{H_{\mathcal{W}}}\big|\mathbf{x}\big]\big|\\
&\quad\quad\quad\quad\quad\quad=O\left(\frac{1}{n^{|E|}}\right)\cdot \left(\frac{\epsilon d}{2n}\right)^{s(t-|\mathcal{W}|)}\cdot\big|\mathbb{E}\big[\mathbf{Y}_{H_{\mathcal{W}}}\big|\mathbf{x}\big]\big|.
\end{aligned}
\end{equation}
\end{itemize}

From \cref{eq:eq_non_truncated_centered_mult_1_walk_equality} and \cref{eq:eq_non_truncated_O_epsilon_n}, we conclude that if $n$ is large enough, then we always have 
\begin{equation}
\label{eq:eq_non_truncated_always_inequality}
\left(\frac{\epsilon d}{2n}\right)^{s(t-|\mathcal{W}'|)}\cdot\big|\mathbb{E}\big[\mathbf{Y}_{H_{\mathcal{W}'}}\big|\mathbf{x}\big]\big|\leq \left(\frac{\epsilon d}{2n}\right)^{s(t-|\mathcal{W}|)}\cdot\big|\mathbb{E}\big[\mathbf{Y}_{H_{\mathcal{W}}}\big|\mathbf{x}\big]\big|.
\end{equation}

\begin{remark}
Note that in \cref{eq:eq_non_truncated_O_epsilon_n} and in the above equation, we implicitly used $\frac{\epsilon}{2-\epsilon}=o(n)$. This means that if $\epsilon$ is close to 2, then $n$ should be large in order for the above equation to be true. We used $\frac{\epsilon}{2-\epsilon}=o(n)$ here because this is an informal discussion, and using $\frac{\epsilon}{2-\epsilon}=o(n)$ will help us in illustrating the proof strategy in a simple way. However, when we formally compute the upper bound for the truncated case, we will not use $\frac{\epsilon}{2-\epsilon}=o(n)$. The main reason why we avoided using $\frac{\epsilon}{2-\epsilon}=o(n)$ in the formal proof is because we do not want to require $n$ to be larger than necessary in order for our results to hold. More precisely, if $d\gg 1$ and $\epsilon$ is close to 2, the weak recovery problem should be easy, and we should not require $n$ to be too large.
\end{remark}

If $W\cap E(H_{\mathcal{W}'})\neq\varnothing$, then \cref{eq:eq_non_truncated_O_epsilon_n} implies that
$$\left(\frac{\epsilon d}{2n}\right)^{s(t-|\mathcal{W}'|)}\cdot\big|\mathbb{E}\big[\mathbf{Y}_{H_{\mathcal{W}'}}\big|\mathbf{x}\big]\big|\leq O\left(\frac{1}{n}\right)\cdot \left(\frac{\epsilon d}{2n}\right)^{s(t-|\mathcal{W}|)}\cdot\big|\mathbb{E}\big[\mathbf{Y}_{H_{\mathcal{W}}}\big|\mathbf{x}\big]\big|.$$

Now if every walk in $\mathcal{W}(H)$ contains at least one edge of multiplicity $\geq\hspace*{-1.2mm}2$ in $H$, then for every $W\in\mathcal{W}(H)$, we have
\begin{align*}
\left(\frac{\epsilon d}{2n}\right)^{s(t-|\mathcal{W}(H)\setminus\{W\}|)}\cdot\big|\mathbb{E}\big[\mathbf{Y}_{H_{\mathcal{W}(H)\setminus\{W\}}}\big|\mathbf{x}\big]\big|&\leq O\left(\frac{1}{n}\right)\cdot \left(\frac{\epsilon d}{2n}\right)^{s(t-|\mathcal{W}(H)|)}\cdot \big|\mathbb{E}\big[\mathbf{Y}_{H_{\mathcal{W}(H)}}\big|\mathbf{x}\big]\big|\\
&= O\left(\frac{1}{n}\right)\cdot \big|\mathbb{E}\big[\mathbf{Y}_H\big|\mathbf{x}\big]\big|.
\end{align*}
On the other hand, for every $\mathcal{W}\subsetneq\mathcal{W}(H)$, there exists $W\in\mathcal{W}(H)$ such that $\mathcal{W}\subseteq\mathcal{W}(H)\setminus\{W\}$. From \cref{eq:eq_non_truncated_always_inequality} we can deduce that
\begin{align*}
\left(\frac{\epsilon d}{2n}\right)^{s(t-|\mathcal{W}|)}\cdot\big|\mathbb{E}\big[\mathbf{Y}_{H_{\mathcal{W}}}\big|\mathbf{x}\big]\big|&\leq
\left(\frac{\epsilon d}{2n}\right)^{s(t-|\mathcal{W}(H)\setminus\{W\}|)}\cdot\big|\mathbb{E}\big[\mathbf{Y}_{H_{\mathcal{W}(H)\setminus\{W\}}}\big|\mathbf{x}\big]\big|\\
&\leq O\left(\frac{1}{n}\right)\cdot \big|\mathbb{E}\big[\mathbf{Y}_H\big|\mathbf{x}\big]\big|.
\end{align*}

If we put this in \cref{eq:eq_centered_non_truncated}, we get
\begin{equation}
\label{eq:eq_walk_mult_2_non_truncated_bound}
\begin{aligned}
\left|\mathbb{E}\left[\prod_{W\in\mathcal{W}(H)}\left(\mathbf{Y}_{W}-\left(\frac{\epsilon d}{2n}\right)^s\mathbf{x}_{W}\right)\middle|\mathbf{x}\right]\right|&\leq \big|\mathbb{E}\big[\mathbf{Y}_H\big|\mathbf{x}\big]\big|+\sum_{\mathcal{W}\subsetneq\mathcal{W}(H)}\left(\frac{\epsilon d}{2n}\right)^{s(t-|\mathcal{W}|)}\cdot\big|\mathbb{E}\big[\mathbf{Y}_{H_{\mathcal{W}}}\big|\mathbf{x}\big]\big|\\
&\leq \big|\mathbb{E}\big[\mathbf{Y}_H\big|\mathbf{x}\big]\big|+\sum_{\mathcal{W}\subsetneq\mathcal{W}(H)}O\left(\frac{1}{n}\right)\cdot \big|\mathbb{E}\big[\mathbf{Y}_H\big|\mathbf{x}\big]\big|\\
&\leq \left(1+O\left(\frac{2^{|\mathcal{W}(H)|}-1}{n}\right)\right)\cdot \big|\mathbb{E}\big[\mathbf{Y}_H\big|\mathbf{x}\big]\big|\\
&= \left(1+O\left(\frac{2^t-1}{n}\right)\right)\cdot \big|\mathbb{E}\big[\mathbf{Y}_H\big|\mathbf{x}\big]\big|\\
&= (1+o(1))\cdot \big|\mathbb{E}\big[\mathbf{Y}_H\big|\mathbf{x}\big]\big|,
\end{aligned}
\end{equation}
where the last equality follows from the fact that $2^t\leq 2^{K\log n}=n^K= o(n)$, assuming that $K<1$.

In summary, if there is at least one walk in $\mathcal{W}(H)$ such that all its edges have multiplicity 1 in $H$, then 
\begin{align*}
\left|\mathbb{E}\left[\prod_{W\in\mathcal{W}(H)}\left(\mathbf{Y}_{W}-\left(\frac{\epsilon d}{2n}\right)^s\mathbf{x}_{W}\right)\middle|\mathbf{x}\right]\right|=0.
\end{align*}
On the other hand, if every walk in $\mathcal{W}(H)$ contains at least one edge of multiplicity $\geq\hspace*{-1.2mm}2$ in $H$, then
\begin{align*}
\left|\mathbb{E}\left[\prod_{W\in\mathcal{W}(H)}\left(\mathbf{Y}_{W}-\left(\frac{\epsilon d}{2n}\right)^s\mathbf{x}_{W}\right)\middle|\mathbf{x}\right]\right|&\leq  (1+o(1))\cdot \big|\mathbb{E}\big[\mathbf{Y}_H\big|\mathbf{x}\big]\big|.
\end{align*}

An approximate version of this phenomenon occurs in the truncated case. We will show that if there is at least one walk in $\mathcal{W}(H)$ such that all its edges have multiplicity 1 in $H$ and all its vertices have low degrees in $H$, then $\big|\mathbb{E}\big[\hat{\mathbf{Y}}_H\big|\mathbf{x}\big]\big|$ will be very small. In all the other cases, we will show that $\big|\mathbb{E}\big[\hat{\mathbf{Y}}_H\big|\mathbf{x}\big]\big|$ is not too large compared to $\big|\mathbb{E}\big[\overline{\mathbf{Y}}_H\big|\mathbf{x}\big]\big|$.

As we will see, most nice block self-avoiding-walks have many walks whose edges are all of multiplicity 1. This means that the multigraphs that contributed significantly in the case of the non-centered matrix $\mathbb{E}\left[\Tr\left(Q^{s}\big(\overline{\mathbf{Y}}\big)^t\right)\right]$, will contribute very little in the case of the centered matrix $\mathbb{E}\left[\Tr\left(\big(Q^{(s)}\big(\overline{\mathbf{Y}}\big)-\mathbf{x}\transpose{\mathbf{x}}\big)^t\right)\right]$.

On the other hand, multigraphs that contributed negligibly in the case of the non-centered matrix $\mathbb{E}\left[\Tr\left(Q^{s}\big(\overline{\mathbf{Y}}\big)^t\right)\right]$, will contribute a comparable amount in the case of the centered matrix $\mathbb{E}\left[\Tr\left(\big(Q^{(s)}\big(\overline{\mathbf{Y}}\big)-\mathbf{x}\transpose{\mathbf{x}}\big)^t\right)\right]$. This essentially means that $\mathbb{E}\left[\Tr\left(\big(Q^{(s)}\big(\overline{\mathbf{Y}}\big)-\mathbf{x}\transpose{\mathbf{x}}\big)^t\right)\right]$ is negligible with respect to $\mathbb{E}\left[\Tr\left(Q^{s}\big(\overline{\mathbf{Y}}\big)^t\right)\right]$.

\subsubsection{Analyzing walks of multiplicity 1}

The discussion in the previous section motivates the following definition:

\begin{definition}
\label{def:def_multiplicity_for_walks}
Let $H\in\bsaw{s}{t}$. For every $W\in\mathcal{W}(H)$, we say that $W$ is of multiplicity 1 in $H$ if \emph{every} edge in $W$ is of multiplicity 1 in $H$. Define
$$\mathcal{W}_1(H)=\big\{W\in\mathcal{W}(H):\;W\text{ is of multiplicity 1 in $H$}\big\},$$
and
$$\mathcal{W}_{\geq 2}(H)=\mathcal{W}(H)\setminus \mathcal{W}_1(H).$$
\end{definition}

In the previous section, we used the fact that if $H$ contains a walk of multiplicity 1, then 
\begin{align*}
\left|\mathbb{E}\left[\prod_{W\in\mathcal{W}(H)}\left(\mathbf{Y}_{W}-\left(\frac{\epsilon d}{2n}\right)^s\mathbf{x}_{W}\right)\middle|\mathbf{x}\right]\right|=0.
\end{align*}

We would like to show something similar for the truncated case. Recall that if an edge has both its end-vertices in $\mathcal{S}_1(H)\cap\mathcal{S}_{\geq 2}(H)$,\footnote{Recall \cref{def:def_deg1_classification} and \cref{def:def_deg2_classification}} then there is a significant probability that it will be a safe edge\footnote{Recall \cref{def:def_safe_vertices}}, in which case it will behave similarly to the non-truncated case. This motivates the following definition:

\begin{definition}
A walk $W\in\mathcal{W}_1(H)$ is said to be \emph{reassuring} if $V(W)\subseteq \mathcal{S}_1(H) \cap\mathcal{S}_{\geq 2}(H)$. We denote the set of reassuring walks in $\mathcal{W}_1(H)$ as $\mathcal{W}_{1r}(H)$.

A walk $W\in\mathcal{W}_1(H)$ that is not reassuring is said to be \emph{disturbing}. We denote the set of disturbing walks in $\mathcal{W}_1(H)$ as $\mathcal{W}_{1d}(H)$.

Clearly, $\big\{\mathcal{W}_{1r}(H),\mathcal{W}_{1d}(H)\big\}$ is a partition of $\mathcal{W}_1(H)$.
\end{definition}

Roughly speaking, if $H$ contains a reassuring walk, then with significant probability this walk will behave similarly to the truncated case. This will cause $\big|\mathbb{E}\big[\hat{\mathbf{Y}}_H\big|\mathbf{x}\big]\big|$ to be small.

\begin{definition}
Recall \cref{def:def_safe_vertices} and let $H\in\bsaw{s}{t}$. For every walk $W\in \mathcal{W}(H)$, if $V(W)$ is completely $(\mathbf{G},H)$-\emph{safe}, we say that $W$ is $(\mathbf{G},H)$-\emph{walk-safe}. We say that $W$ is $(\mathbf{G},H)$-\emph{walk-unsafe} if it is not $(\mathbf{G},H)$-\emph{walk-safe}.

We say that a subset $\mathcal{W}$ of $\mathcal{W}(H)$ is \emph{completely} $(\mathbf{G},H)$-\emph{walk-safe} if all the walks in it are $(\mathbf{G},H)$-\emph{walk-safe}. Similarly, we say that $\mathcal{W}$ is \emph{completely} $(\mathbf{G},H)$-\emph{walk-unsafe} if all the walks in it are $(\mathbf{G},H)$-\emph{walk-unsafe}.

If $\mathbf{G}$ and $H$ are clear from the context, we drop $(\mathbf{G},H)$ and simply write walk-safe, completely walk-safe, walk-unsafe, and completely walk-unsafe.
\end{definition}

We emphasize that in order for a set of walks to be completely walk-unsafe, it is not necessary that the set of vertices forming the walks in it is completely unsafe: It is sufficient that every walk contains at least one unsafe vertex.

\begin{definition}
Let $H\in\bsaw{s}{t}$. For every $\mathcal{W}\subset\mathcal{W}(H)$, we define
$$H_{\mathcal{W}}=\bigoplus_{W\in\mathcal{W}}W,$$
and so
$$\overline{\mathbf{Y}}_{H_{\mathcal{W}}}=\overline{\mathbf{Y}}_{\resizebox{0.05\textwidth}{!}{$\displaystyle\bigoplus_{W\in\mathcal{W}}W$}}=\prod_{W\in\mathcal{W}} \overline{\mathbf{Y}}_W.$$
\end{definition}

\begin{lemma}
\label{lem:lem_X_hat_mult_1}
Let $H\in\bsaw{s}{t}$ be such that $t=K\log n$. We have

\begin{align*}
\big|\mathbb{E}\big[\hat{\mathbf{Y}}_{H}\big|\mathbf{x}\big]\big|\leq \frac{n^{\frac{2K}{A}}}{2^{3As\cdot|\mathcal{W}_{1r}(H)|}} \cdot\left(\frac{6}{\epsilon}\right)^{|E_1^a(H)|}\cdot2^{|\mathcal{W}_1(H)|}\cdot \left(\frac{\epsilon d}{2n}\right)^{s|\mathcal{W}_1(H)|}\cdot\sum_{\mathcal{W}_{\geq 2}\subseteq \mathcal{W}_{\geq 2}(H)}F_{\mathcal{W}_{\geq 2}}(\mathbf{x}),
\end{align*}

where
$$F_{\mathcal{W}_{\geq 2}}(\mathbf{x})=\left(\frac{3d}{n}\right)^{s(|\mathcal{W}_{\geq 2}(H)|-|\mathcal{W}_{\geq 2}|)}\left(\frac{\epsilon d}{2n}\right)^{|E_1(H_{\mathcal{W}_{\geq 2}})|} \cdot\mathbb{E}\big[|\tilde{\mathbf{Y}}_{E_{\geq 2}(H_{\mathcal{W}_{\geq 2}})}|\big|\mathbf{x}\big],$$
and $\tilde{\mathbf{Y}}_{E_{\geq 2}(H_{\mathcal{W}_{\geq 2}})}$ is as in \Cref{def:def_tilde_X_E_2}
\end{lemma}
\begin{proof}
We only provide a high level description of the proof here. The detailed proof can be found in \Cref{subsubsec:subsubsec_centered_walks_mult_1}.

We start by upper bounding the probability that no reassuring walk is walk-safe. Then, we show that in the event that there is at least one reassuring walk that is walk-safe, the conditional expectation will be zero.

Now for the case that no reassuring walk is walk-safe, we write an equation that is similar to \cref{eq:eq_centered_non_truncated}, and upper bound each summand using the same techniques that allowed us to prove \Cref{lem:lem_UH_deg1}.

See \Cref{subsubsec:subsubsec_centered_walks_mult_1} for the details.
\end{proof}

\subsubsection{Analyzing walks of multiplicity at least 2}

\begin{lemma}
\label{lem:lem_F_W_2_total_Upper_Bound}
Let $H\in\bsaw{s}{t}$ be such that $t=K\cdot\log n$ and let $F_{\mathcal{W}_{\geq 2}}(\mathbf{x})$ be as in \cref{lem:lem_X_hat_mult_1}. If $K\leq \frac{1}{100}$ and $n$ is large enough, we have
\begin{equation}
\label{eq:eq_F_W_2_total_Upper_Bound}
\begin{aligned}
\sum_{\mathcal{W}_{\geq 2}\subseteq\mathcal{W}_{\geq 2}(H)} F_{\mathcal{W}_{\geq 2}}(\mathbf{x})&\leq  \frac{2\cdot\left(\frac{\epsilon d}{2n}\right)^{|E_1(H_{\mathcal{W}_{\geq 2}(H)})|}\cdot \left(\frac{d}{n}\right)^{|E_{\geq2}^b(H)|}}{\resizebox{0.06\textwidth}{!}{$\displaystyle\prod_{v\in\mathcal{L}_{\geq2}(H)}$} n^{\frac{1}{4}\left(d^H_{\geq 2}(v)-\Delta\right)}}\cdot\prod_{uv\in E_{\geq2}^a(H)}\left[\left(1+\frac{\epsilon \mathbf{x}_u\mathbf{x}_v}{2}\right)\frac{d}{n}+\frac{3d^2}{n\sqrt{n}}\right],
\end{aligned}
\end{equation}
where $F_{\mathcal{W}_{\geq 2}}(\mathbf{x})$ is as in \Cref{lem:lem_X_hat_mult_1}.
\end{lemma}
\begin{proof}
We use \Cref{lem:lem_UH_deg2} and perform calculations that are similar to those in \cref{eq:eq_walk_mult_2_non_truncated_bound}. The detailed proof can be found in \Cref{subsubsec:subsubsec_centered_walks_mult_2}.
\end{proof}

\subsubsection{Proof of the upper bound for every block self-avoiding walk}

\begin{definition}
\label{def:def_upper_bound_UhatH}
For every $H\in\bsaw{s}{t}$, define the following quantity:
$$\hat{U}_H(\mathbf{x})=2n^{\frac{K}{A}}\cdot\frac{2^{|E_1^a(H)|}}{2^{As\cdot|\mathcal{W}_{1r}(H)|}} \cdot\overline{U}_H(\mathbf{x}),$$
where $\overline{U}_H(\mathbf{x})$ is as in \Cref{def:def_upper_bound_UH}.
\end{definition}

\begin{lemma}
\label{lem:lem_upper_bound_UhatH}
If $A>\max\{100K,1\}$, $K\leq\frac{1}{100}$ and $n$ is large enough, then for every $H\in\bsaw{s}{t}$, we have
$$\big|\mathbb{E}\big[\hat{\mathbf{Y}}_{H}\big|\mathbf{x}\big]\big|\leq \hat{U}_H(\mathbf{x}),$$
where $\hat{U}_H(\mathbf{x})$ is as in \Cref{def:def_upper_bound_UhatH}.
\end{lemma}
\begin{proof}
From \Cref{lem:lem_X_hat_mult_1} and \Cref{lem:lem_F_W_2_total_Upper_Bound}, we have
\begin{align*}
&\big|\mathbb{E}\big[\hat{\mathbf{Y}}_{H}\big|\mathbf{x}\big]\big|\\
&\quad\leq \frac{2^{|\mathcal{W}_1(H)|+1}}{2^{3As\cdot|\mathcal{W}_{1r}(H)|}} \cdot\frac{n^{\frac{2K}{A}}\cdot\left(\frac{6}{\epsilon}\right)^{|E_1^a(H)|}\cdot \left(\frac{\epsilon d}{2n}\right)^{s|\mathcal{W}_1(H)|+|E_1(H_{\mathcal{W}_{\geq 2}(H)})|}\cdot \left(\frac{d}{n}\right)^{|E_{\geq2}^b(H)|}}{\resizebox{0.06\textwidth}{!}{$\displaystyle\prod_{v\in\mathcal{L}_{\geq2}(H)}$} n^{\frac{1}{4}\left(d^H_{\geq 2}(v)-\Delta\right)}}\\
&\quad\quad\quad\quad\quad\quad\quad\quad\quad\quad\quad\quad\quad\quad\quad\quad\quad
\quad\quad\quad\quad\quad\quad\quad\quad\times\prod_{uv\in E_{\geq2}^a(H)}\left[\left(1+\frac{\epsilon \mathbf{x}_u\mathbf{x}_v}{2}\right)\frac{d}{n}+\frac{3d^2}{n\sqrt{n}}\right]\\
&\quad= \frac{2^{|\mathcal{W}_1(H)|+1}}{2^{3As\cdot|\mathcal{W}_{1r}(H)|}} \cdot\frac{n^{\frac{2K}{A}}\cdot\left(\frac{6}{\epsilon}\right)^{|E_1^a(H)|}\cdot \left(\frac{\epsilon d}{2n}\right)^{|E_1(H)|}\cdot \left(\frac{d}{n}\right)^{|E_{\geq2}^b(H)|}}{\resizebox{0.06\textwidth}{!}{$\displaystyle\prod_{v\in\mathcal{L}_{\geq2}(H)}$} n^{\frac{1}{4}\left(d^H_{\geq 2}(v)-\Delta\right)}}\cdot\prod_{uv\in E_{\geq2}^a(H)}\left[\left(1+\frac{\epsilon \mathbf{x}_u\mathbf{x}_v}{2}\right)\frac{d}{n}+\frac{3d^2}{n\sqrt{n}}\right]\\
&\quad= \frac{2^{|\mathcal{W}_1(H)|+1}}{2^{3As\cdot|\mathcal{W}_{1r}(H)|}}\cdot \overline{U}_H(\mathbf{x}).
\end{align*}

Now since $A>\max\{100K,1\}$ and $s\geq 1$, we have $$3As\cdot|\mathcal{W}_{1r}(H)|\geq As\cdot|\mathcal{W}_{1r}(H)|+|\mathcal{W}_{1r}(H)|,$$ hence
\begin{align*}
\big|\mathbb{E}\big[\hat{\mathbf{Y}}_{H}\big|\mathbf{x}\big]\big|&\leq \frac{2^{|\mathcal{W}_{1r}(H)|+|\mathcal{W}_{1d}(H)|+1}}{2^{As\cdot|\mathcal{W}_{1r}(H)|+|\mathcal{W}_{1r}(H)|}} \cdot\overline{U}_H(\mathbf{x})\\
&= 2\cdot \frac{2^{|\mathcal{W}_{1d}(H)|}}{2^{As\cdot|\mathcal{W}_{1r}(H)|}} \cdot\overline{U}_H(\mathbf{x}).
\end{align*}

Observe that for every $W\in \mathcal{W}_{1d}(H)$, there exists at least one edge $e\in W$ such that $e\in E_1^a(H)$ or $e\in E_1^d(H)$, where $E_1^d(H)$ is as in \cref{eq:eq_E_1c_H}. Therefore,
$$|\mathcal{W}_{1d}(H)|\leq |E_1^a(H)|+|E_1^d(H)|,$$
which implies that
\begin{align*}
2^{|\mathcal{W}_{1d}(H)|}&\leq 2^{|E_1^a(H)|}\cdot 2^{|E_1^d(H)|}\leq 2^{|E_1^a(H)|}\cdot \left(\frac{6}{\epsilon}\right)^{|E_1^d(H)|}\stackrel{(\ast)}{\leq} 2^{|E_1^a(H)|}\cdot n^{\frac{K}{A}},
\end{align*}
where $(\ast)$ follows from \Cref{lem:lem_UH_deg1_common_E_1d}. We conclude that
\begin{align*}
\big|\mathbb{E}\big[\hat{\mathbf{Y}}_{H}\big|\mathbf{x}\big]\big|&\leq 2n^{\frac{K}{A}}\cdot\frac{2^{|E_1^a(H)|}}{2^{As\cdot|\mathcal{W}_{1r}(H)|}} \cdot\overline{U}_H(\mathbf{x})= \hat{U}_H(\mathbf{x}).
\end{align*}
\end{proof}

\section{Proofs of technical lemmas for the trace bounds}\label{sec:technical-lemmas-trace-bounds}

\subsection{Proofs of technical lemmas for the non-centered matrix}

\subsubsection{Upper bound on the probability of having unsafe vertices}
\label{subsubsec:subsubsec_lem_bound_prob_large_outside_deg}

\begin{proof}[Proof of \Cref{lem:lem_bound_prob_large_outside_deg}]

By the union bound, we have
\begin{align*}
\mathbb{P}&\left[d_{\mathbf{G}-G(H)}^o(v)>\frac{\Delta}{4}\middle|\mathbf{x}\right]\\
&\leq \sum_{\substack{S\subseteq [n]\setminus V(H):\\|S|=\lceil\Delta/4\rceil}}\mathbb{P}\big[\big\{\forall u\in S, uv\in\mathbf{G}\}\big|\mathbf{x}\big]= \sum_{\substack{S\subseteq [n]\setminus V(H):\\|S|=\lceil\Delta/4\rceil}}\prod_{u\in S} \left[\left(1+\frac{\epsilon}{2}\mathbf{x}_{uv}\right)\frac{d}{n}\right]\\
&\leq \sum_{\substack{S\subseteq [n]\setminus V(H):\\|S|=\lceil\Delta/4\rceil}}\left(\frac{2d}{n}\right)^{|S|}={n-|V(H)|\choose \lceil\Delta/4\rceil}\left(\frac{2d}{n}\right)^{\lceil\Delta/4\rceil}\leq {n\choose \lceil\Delta/4\rceil}\left(\frac{2d}{n}\right)^{\lceil\Delta/4\rceil} \\
&\leq \frac{n^{\lceil\Delta/4\rceil}(2d)^{\lceil\Delta/4\rceil}}{\lceil\Delta/4\rceil!n^{\lceil\Delta/4\rceil}}
\leq \frac{(2d)^{\lceil\Delta/4\rceil}}{(\lceil\Delta/4\rceil/2)^{\lceil\Delta/4\rceil/2}}\leq \frac{(2d)^{2\lceil\Delta/4\rceil}}{(\lceil\Delta/4\rceil/2)^{\lceil\Delta/4\rceil/2}}.
\end{align*}

Now from \cref{eq:eq_Delta_form}, we have $\Delta\geq 128e^4d^4$ and so $\lceil\Delta/4\rceil/2\geq 16e^4d^4$. Therefore,

\begin{align*}
\mathbb{P}\left[d_{\mathbf{G}-G(H)}^o(v)>\frac{\Delta}{2}\middle|\mathbf{x}\right]&\leq \frac{(2d)^{2\lceil\Delta/4\rceil}}{(16e^4d^4)^{\lceil\Delta/4\rceil/2}}=\frac{(2d)^{2\lceil\Delta/4\rceil}}{(2ed)^{2\lceil\Delta/2\rceil}}=e^{-2\lceil\Delta/4\rceil}\leq e^{-\Delta/2}.
\end{align*}

From \cref{eq:eq_Delta_form}, we also have 
\begin{align*}
\frac\Delta2&\geq \log(2As)+ 6A\tau s^2\cdot \log 2 + 4A^2\tau^2s^2 \left(\log\frac{6}{\epsilon}\right)^2\\
&\geq \log (2As)+ 6A\tau s^2\cdot \log 2  + 4\tau^2s^2 \left(\log\frac{6}{\epsilon}\right)\\
&=\log\left[2As\cdot 2^{6A\tau s^2}\left(\frac{6}{\epsilon}\right)^{4\tau^2s^2}\right].
\end{align*}

Therefore,
\begin{align*}
\mathbb{P}\left[d_{\mathbf{G}-G(H)}^o(v)>\frac{\Delta}{2}\middle|\mathbf{x}\right]&\leq \frac{1}{2As\cdot 2^{6A\tau s^2}}\left(\frac{\epsilon}{6} \right)^{4\tau^2s^2}.
\end{align*}
\end{proof}

\begin{proof}[Proof of \Cref{lem:lem_bound_prob_completely_crossing}]
If $S$ is completely crossing, then for every vertex in $S$, we have at least one $H$-cross-edge that is present in $\mathbf{G}$, and which is incident to it. Therefore, we have at least $\left\lceil|S|/2\right\rceil$ $H$-cross-edges that are present in $\mathbf{G}$. Since we have at most $|V(H)|^2\leq s^2t^2$ $H$-cross-edges, the number of collections of $H$-cross-edges of size $\left\lceil|S|/2\right\rceil$ is at most
$${s^2t^2\choose \left\lceil|S|/2\right\rceil}\leq (s^2t^2)^{\left\lceil|S|/2\right\rceil}.$$ 

Given $\mathbf{x}$, the conditional probability that any particular edge is present in $\mathbf{G}$ is at most
$$\left(1+\frac{\epsilon}{2}\right)\frac{d}{n}\leq \frac{2d}{n}.$$

Therefore, given $\mathbf{x}$, the conditional probability that $S$ is completely crossing can be upper bounded by:
\begin{align*}
\mathbb{P}\left[\left\{S\text{ is completely }H\text{-crossing in }\mathbf{G}\right\}\middle|\mathbf{x}\right]&\leq (s^2t^2)^{\left\lceil|S|/2\right\rceil} \left(\frac{2d}{n}\right)^{\left\lceil|S|/2\right\rceil}\\
&= \left(\frac{2d s^2 t^2}{n}\right)^{\left\lceil|S|/2\right\rceil}\leq \left(\frac{2d s^2 t^2}{n}\right)^{|S|/2},
\end{align*}
where the last inequality is true for $n$ large enough.
\end{proof}

\begin{proof}[Proof of \Cref{lem:lem_bound_prob_completely_unsafe}]
First notice that if $v\in \mathcal{S}_1(H)\cap\mathcal{S}_{\geq 2}(H)$, then
$$d^H_{1}(v)\leq \tau \leq \frac{1}{4}\cdot\left[\log(2As)+ 12A\tau s^2\cdot \log 2 + 8A^2\tau^2s^2 \left(\log\frac{6}{\epsilon}\right)^2\right]\leq\frac{\Delta}{4},$$
and
$$d^H_{\geq 2}(v)\leq \tau \leq\frac{\Delta}{4}.$$
Therefore,
$$d_{H}(v)=d^H_{1}(v)+d^H_{\geq 2}(v)\leq \frac{\Delta}{2},$$
which implies that
\begin{align*}
\mathbb{P}\left[\left\{V\text{ is completely unsafe}\right\}\middle|\mathbf{x}\right]&=\mathbb{P}\left[\left\{\forall v\in V, d_{\mathbf{G}-G(H)}(v)> \Delta-d_{G(H)}(v)\right\}\middle|\mathbf{x}\right]\\
&\leq\mathbb{P}\left[\left\{\forall v\in V, d_{\mathbf{G}-G(H)}(v)> \Delta-\frac{\Delta}{2}\right\}\middle|\mathbf{x}\right]\\
&=\mathbb{P}\left[\left\{\forall v\in V, d_{\mathbf{G}-G(H)}(v)> \frac{\Delta}{2}\right\}\middle|\mathbf{x}\right].
\end{align*}

Now for every $S\subseteq V$, define the events
$$\mathcal{E}_{S,H,c\text{-}cf}=\Big\{S\text{ is completely }H\text{-cross-free in }\mathbf{G}\Big\},$$
and
$$\mathcal{E}_{S,H,c\text{-}c}=\Big\{S\text{ is completely }H\text{-crossing in }\mathbf{G}\Big\}.$$
We have:
\begin{align*}
&\mathbb{P}\left[\left\{V\text{ is completely unsafe}\right\}\middle|\mathbf{x}\right]\\
&\quad\leq\mathbb{P}\left[\left\{\forall v\in V, d_{\mathbf{G}-G(H)}(v)> \frac{\Delta}{2}\right\}\middle|\mathbf{x}\right]\\
&\quad=\sum_{S\subseteq V}\mathbb{P}\left[\left\{\forall v\in V, d_{\mathbf{G}-G(H)}(v)> \frac{\Delta}{2}\right\}\middle|\mathbf{x}, \mathcal{E}_{S,H,c\text{-}cf}\cap \mathcal{E}_{V\setminus S,H,c\text{-}c}\right]\mathbb{P}\left[\mathcal{E}_{S,H,c\text{-}cf}\cap \mathcal{E}_{V\setminus S,H,c\text{-}c}\middle|\mathbf{x}\right]\\
&\quad\leq \sum_{S\subseteq V}\mathbb{P}\left[\left\{\forall v\in S, d_{\mathbf{G}-G(H)}(v)> \frac{\Delta}{2}\right\}\middle| \mathbf{x},\mathcal{E}_{S,H,c\text{-}cf}\cap \mathcal{E}_{V\setminus S,H,c\text{-}c}\right]\mathbb{P}\left[\mathcal{E}_{V\setminus S,H,c\text{-}c}\middle|\mathbf{x}\right]\\
&\quad\leq \sum_{S\subseteq V}\mathbb{P}\left[\left\{\forall v\in S, d_{\mathbf{G}-G(H)}^o(v)> \frac{\Delta}{2}\right\}\middle| \mathbf{x},\mathcal{E}_{S,H,c\text{-}cf}\cap \mathcal{E}_{V\setminus S,H,c\text{-}c}\right]\left(\frac{2d s^2 t^2}{n}\right)^{(|V|-|S|)/2},
\end{align*}
where the last inequality follows from \Cref{lem:lem_bound_prob_completely_crossing} and the fact that if $v$ is cross-free, then $d_{\mathbf{G}-G(H)}^i(v)=0$ and so $d_{\mathbf{G}-G(H)}(v)=d_{\mathbf{G}-G(H)}^o(v)$.

Now since $(d_{\mathbf{G}-G(H)}^o(v))_{v\in S}$ are conditionally mutually independent given $\mathbf{x}$, and since they are conditionally independent from $\mathcal{E}_{S,H,c\text{-}cf}\cap \mathcal{E}_{V\setminus S,H,c\text{-}c}$ given $\mathbf{x}$, we have
\begin{align*}
\mathbb{P}\left[\left\{V\text{ is completely unsafe}\right\}\middle|\mathbf{x}\right]\leq \sum_{S\subseteq V}\left(\prod_{v\in S}\mathbb{P}\left[d_{\mathbf{G}-G(H)}^o(v)> \frac{\Delta}{2}\middle|\mathbf{x}\right]\right)\left(\frac{2d s^2 t^2}{n}\right)^{(|V|-|S|)/2}.
\end{align*}
Now from \Cref{lem:lem_bound_prob_large_outside_deg} we get
\begin{align*}
\mathbb{P}\left[\left\{V\text{ is completely unsafe}\right\}\middle|\mathbf{x}\right]&\leq \sum_{S\subseteq V}\left(\frac{\eta}{2}\right)^{|S|}\left(st\sqrt{\frac{2d }{n}}\right)^{|V|-|S|}= \left(\frac{\eta}{2}+st\sqrt{\frac{2d }{n}}\right)^{|V|}\leq \eta^{|V|},
\end{align*}
where the last inequality is true for $n$ large enough.
\end{proof}

\subsubsection{Upper bound on the contribution of edges of multiplicity 1}

\label{subsubsec:subsubsec_edges_mult_1_upper}

In order to prove \Cref{lem:lem_UH_deg1}, we need a few lemmas.

The following lemma proves a result that is similar to \Cref{lem:lem_UH_deg1}, but instead of using the best upper bound for all the edges in $E_1^b(H)$, we use loose upper bounds for the edges that are incident to $\big(\mathcal{S}_1(H)\cap\mathcal{S}_{\geq 2}(H)\big)\setminus U$, where $U$ is some subset of $\mathcal{S}_1(H)\cap\mathcal{S}_{\geq 2}(H)$. This lemma will be useful later in situations where we cannot guarantee that these edges are likely to be safe.

\begin{lemma}
\label{lem:lem_UH_deg1_common}
Let $H$ be a multigraph such that $|V(H)|\leq st=sK\log n$. Recall \Cref{def:def_deg1_classification}, \Cref{def:def_deg2_classification} and \Cref{def:def_E1_E2_annoying}, and let $E_1^b(H)$ and $E_1^d(H)$ be as in \cref{eq:eq_E_1b_H} and \cref{eq:eq_E_1c_H}, respectively.

Let $\mathcal{S}(H)=\mathcal{S}_1(H)\cap \mathcal{S}_{\geq 2}(H)$ and let $U\subseteq \mathcal{S}(H)$. Let $\mathcal{E}$ be an event such that:
\begin{itemize}
\item[(a)] The event $\mathcal{E}$ depends only on $\mathbf{x}$ and $\mathbf{G}-E_1(H)$, i.e., the event $\mathcal{E}$ is $\sigma(\mathbf{x},\mathbf{G}-E_1(H))$-measurable. In other words, if we condition on $\mathbf{x}$, then $\mathcal{E}$ depends only $(\mathbbm{1}_{uv\in\mathbf{G}})_{u,v\in [n]:\; uv\notin E_1(H)}$. This implies that given $\mathbf{x}$, the event $\mathcal{E}$ is conditionally independent from $(\mathbbm{1}_{uv\in\mathbf{G}})_{uv\in E_1(H)}$.
\item[(b)] For every $V\subseteq U$, we have
$$\mathbb{P}\big[\big\{V\text{ is completely unsafe}\big\}\big|\mathbf{x},\mathcal{E}\big]\leq\eta^{|V|}.$$
\item[(c)] If we are given $\mathbf{x}$ and $\mathcal{E}$, then for every $V\subseteq U$, the event
$$\big\{V\text{ is completely safe}\big\}\cap\big\{U\setminus V\text{ is completely unsafe}\big\},$$
is conditionally independent from $(\mathbbm{1}_{\{uv\in \mathbf{G}\}})_{uv\in E(H)}$.
\end{itemize}
Then,
$$\big|\mathbb{E}\big[\overline{\mathbf{Y}}_H\big|\mathbf{x},\mathcal{E}\big]\big|\leq n^{\frac{K}{A}}\left(\frac{6}{\epsilon}\right)^{|E_1^a(H)|+|E_1^d(H)|+\tau(|\mathcal{S}(H)|-|U|)}\cdot\left(\frac{\epsilon d}{2n}\right)^{|E_1(H)|} \cdot\mathbb{E}\big[|\tilde{\mathbf{Y}}_{\geq 2}^H|\big|\mathbf{x},\mathcal{E}\big],$$
where $\tilde{\mathbf{Y}}_{\geq 2}^H$ is as in \Cref{def:def_tilde_X_E_2}.
\end{lemma}
\begin{proof}
It is easy to see that $\{E_1^a(H),E_1^b(H),E_1^d(H)\}$ is a partition of $E_1(H)$.

For every $V\subseteq U$, define the events
$$\mathcal{E}_{V,H,c\text{-}s}=\Big\{V\text{ is completely }H\text{-safe in }\mathbf{G}\Big\},$$
and
$$\mathcal{E}_{V,H,c\text{-}us}=\Big\{V\text{ is completely }H\text{-unsafe in }\mathbf{G}\Big\}.$$

We have:
\begin{align*}
\mathbb{E}\big[\overline{\mathbf{Y}}_H\big|\mathbf{x},\mathcal{E}\big]&=\sum_{V\subseteq  U }\mathbb{E}\big[\overline{\mathbf{Y}}_H\big|\mathbf{x},\mathcal{E}\cap\mathcal{E}_{V,H,c\text{-}s}\cap \mathcal{E}_{ U \setminus V,H,c\text{-}us}\big]\cdot\mathbb{P}\left[\mathcal{E}_{V,H,c\text{-}s}\cap \mathcal{E}_{ U \setminus V,H,c\text{-}us}\middle|\mathbf{x},\mathcal{E}\right].
\end{align*}

Now let $V\subseteq  U $ and suppose that $\mathcal{E}_{V,H,c\text{-}s}\cap \mathcal{E}_{ U \setminus V,H,c\text{-}us}$ occurs, i.e., $V$ is completely safe and $ U \setminus V$ is completely unsafe. Since $\mathcal{S}(H)\setminus V\subseteq \mathcal{S}(H)\subseteq\mathcal{S}_1(H)$, there are at most $\tau\cdot|\mathcal{S}(H)\setminus V|$ edges of multiplicity 1 in $H$ that are incident to vertices in $\mathcal{S}(H)\setminus V$. In particular, there are at most $\tau\cdot|\mathcal{S}(H)\setminus V|$ edges in $E_1^b(H)$ which are incident to vertices in $\mathcal{S}(H)\setminus V$.

Since every edge in $E_1^b(H)$ has both its ends in $\mathcal{S}(H)$, there are at least $|E_1^b(H)|-\tau\cdot|\mathcal{S}(H)\setminus V|$ edges in $E_1^b(H)$ which have both their end-vertices in $V$. This means that there are at least $|E_1^b(H)|-\tau\cdot|\mathcal{S}(H)\setminus V|$ safe edges in $E_1^b(H)$. For every $V\subseteq  U $, let $S_V$ be an arbitrary subset of $E_1^b(H)$ containing exactly $\max\{0,|E_1^b(H)|-\tau\cdot|\mathcal{S}(H)\setminus V|\}$ safe edges. Note that the choice of $S_V$ depends only on the structure of $H$, and does not depend on the random sample $(\mathbf{G},\mathbf{x})$.

Property (a) implies that we can apply \Cref{lem:lem_expectation_safe_edge} to the event $\mathcal{E}\cap \mathcal{E}_{V,H,c\text{-}s}\cap \mathcal{E}_{ U \setminus V,H,c\text{-}us}$ and the edges in $S_V$. We get:
\begin{align*}
\mathbb{E}\big[\overline{\mathbf{Y}}_H\big|\mathbf{x},\mathcal{E}\cap \mathcal{E}_{V,H,c\text{-}s}&\cap \mathcal{E}_{ U \setminus V,H,c\text{-}us}\big]\\
&=\left(\prod_{uv\in S_V}\mathbb{E}[\mathbf{Y}_{uv}|\mathbf{x}]\right)\cdot \mathbb{E}\big[\overline{\mathbf{Y}}_{H-S_V}\big|\mathbf{x},\mathcal{E}\cap\mathcal{E}_{V,H,c\text{-}s}\cap \mathcal{E}_{ U \setminus V,H,c\text{-}us}\big]\\
&=\left(\prod_{uv\in S_V}\frac{\epsilon \mathbf{x}_u\mathbf{x}_v d}{2n}\right)\cdot \mathbb{E}\big[\overline{\mathbf{Y}}_{H-S_V}\big|\mathbf{x},\mathcal{E}\cap\mathcal{E}_{V,H,c\text{-}s}\cap \mathcal{E}_{ U \setminus V,H,c\text{-}us}\big].
\end{align*}

Therefore,

\begin{equation}
\label{eq:eq_lem_UH_deg1_common_eq1}
\begin{aligned}
\big|\mathbb{E}\big[\overline{\mathbf{Y}}_H\big|\mathbf{x},\mathcal{E}\big]\big|&\leq \sum_{V\subseteq  U }\left(\frac{\epsilon  d}{2n}\right)^{|S_V|}\left|\mathbb{E}\big[\overline{\mathbf{Y}}_{H-S_V}\big|\mathbf{x},\mathcal{E}\cap\mathcal{E}_{V,H,c\text{-}s}\cap \mathcal{E}_{ U \setminus V,H,c\text{-}us}\big]\right|\\
&\quad\quad\quad\quad\quad\quad\quad\quad\quad\quad\quad
\quad\quad\quad\quad\quad\times\mathbb{P}\left[\mathcal{E}_{V,H,c\text{-}s}\cap 
\mathcal{E}_{ U \setminus V,H,c\text{-}us}\middle|\mathbf{x},\mathcal{E}\right]\\
&\leq \sum_{V\subseteq  U }\left(\frac{\epsilon d}{2n}\right)^{|S_V|}\mathbb{E}\big[|\overline{\mathbf{Y}}_{H-S_V}|\big|\mathbf{x},\mathcal{E}\cap\mathcal{E}_{V,H,c\text{-}s}\cap \mathcal{E}_{ U \setminus V,H,c\text{-}us}\big]\\
&\quad\quad\quad\quad\quad\quad\quad\quad\quad\quad\quad\quad\quad\quad
\quad\quad\quad\quad\quad\times\mathbb{P}\left[\mathcal{E}_{ U \setminus V,H,c\text{-}us}\middle|\mathbf{x},\mathcal{E}\right]\\
&\leq \sum_{V\subseteq  U }\left(\frac{\epsilon d}{2n}\right)^{|S_V|}\mathbb{E}\big[|\overline{\mathbf{Y}}_{H-S_V}|\big|\mathbf{x},\mathcal{E}\cap\mathcal{E}_{V,H,c\text{-}s}\cap \mathcal{E}_{ U \setminus V,H,c\text{-}us}\big]\cdot \eta^{| U |-|V|},\\
\end{aligned}
\end{equation}
where the last inequality follows from Property (b).

Now we upper bound $|\overline{\mathbf{Y}}_{H-S_V}|$ as follows:
\begin{align*}
|\overline{\mathbf{Y}}_{H-S_V}|&=|\mathbf{Y}_{H-S_V}|\cdot \mathbbm{1}_{\mathcal{E}_{V(H-S_V)}^c}\\
&=|\mathbf{Y}_{H-S_V}|\cdot \prod_{uv\in E(H)\setminus S_V} \left(\mathbbm{1}_{\{d_{\mathbf{G}}(u)\leq \Delta\}}\cdot\mathbbm{1}_{\{d_{\mathbf{G}}(v)\leq \Delta\}}\right)\\
&\leq |\mathbf{Y}_{H-S_V}|\cdot \prod_{uv\in E_{\geq 2}(H)} \left(\mathbbm{1}_{\{d_{\mathbf{G}}(u)\leq \Delta\}}\cdot\mathbbm{1}_{\{d_{\mathbf{G}}(v)\leq \Delta\}}\right)\\
&\leq |\mathbf{Y}_{H-S_V}|\cdot \prod_{uv\in E_{\geq 2}(H)} \left(\mathbbm{1}_{\{d_{\mathbf{G}\cap E_{\geq 2}(H)}(u)\leq \Delta\}}\cdot\mathbbm{1}_{\{d_{\mathbf{G}\cap E_{\geq 2}(H)}(v)\leq \Delta\}}\right)\\
&= |\mathbf{Y}_{(H-S_V)\cap E_1(H)}|\cdot \prod_{uv\in E_{\geq 2}(H)} \left(\mathbf{Y}_{uv}^{m_H(uv)}\mathbbm{1}_{\{d_{\mathbf{G}\cap E_{\geq 2}(H)}(u)\leq \Delta\}}\cdot\mathbbm{1}_{\{d_{\mathbf{G}\cap E_{\geq 2}(H)}(v)\leq \Delta\}}\right)\\
&=  |\mathbf{Y}_{E_1(H)\setminus S_V}|\cdot  |\tilde{\mathbf{Y}}_{\geq 2}^H|.
\end{align*}

Therefore,
\begin{align*}
\mathbb{E}\big[|\overline{\mathbf{Y}}_{H-S_V}|\big|\mathbf{x},\mathcal{E}\cap\mathcal{E}_{V,H,c\text{-}s}&\cap \mathcal{E}_{ U \setminus V,H,c\text{-}us}\big]\\
&\leq\mathbb{E}\Big[|\mathbf{Y}_{E_1(H)\setminus S_V}|\cdot  |\tilde{\mathbf{Y}}_{\geq 2}^H|\Big|\mathbf{x},\mathcal{E}\cap\mathcal{E}_{V,H,c\text{-}s}\cap \mathcal{E}_{ U \setminus V,H,c\text{-}us}\Big]\\
&\stackrel{(\ast)}{=}\mathbb{E}\Big[|\mathbf{Y}_{E_1(H)\setminus S_V}|\cdot  |\tilde{\mathbf{Y}}_{\geq 2}^H|\Big|\mathbf{x},\mathcal{E}\Big]\\
&\stackrel{(\dagger)}{=}\mathbb{E}\big[|\mathbf{Y}_{E_1(H)\setminus S_V}|\big|\mathbf{x},\mathcal{E}\big]\cdot \mathbb{E}\big[|\tilde{\mathbf{Y}}_{\geq 2}^H|\big|\mathbf{x},\mathcal{E}\big]\\
&\stackrel{(\ddagger)}{=}\mathbb{E}\big[|\mathbf{Y}_{E_1(H)\setminus S_V}|\big|\mathbf{x}\big]\cdot \mathbb{E}\big[|\tilde{\mathbf{Y}}_{\geq 2}^H|\big|\mathbf{x},\mathcal{E}\big],
\end{align*}
where $(\ast)$ follows from Property (c), $(\dagger)$ and $(\ddagger)$ follow from the fact that $\mathcal{E}$ is $\sigma(\mathbf{x},\mathbf{G}-E_1(H))$-measurable. Hence,
\begin{equation}
\label{eq:eq_lem_UH_deg1_common_eq2}
\begin{aligned}
\mathbb{E}\big[|\overline{\mathbf{Y}}_{H-S_V}|\big|\mathbf{x},\mathcal{E}\cap\mathcal{E}_{V,H,c\text{-}s}\cap \mathcal{E}_{ U \setminus V,H,c\text{-}us}\big]&\leq\left(\prod_{uv\in E_1(H)\setminus S_V}\mathbb{E}\big[|\mathbf{Y}_{uv}|\big|\mathbf{x}\big]\right)\cdot \mathbb{E}\big[|\tilde{\mathbf{Y}}_{\geq 2}^H|\big|\mathbf{x},\mathcal{E}\big]\\
&\leq \left[\prod_{uv\in E_1(H)\setminus S_V}\left(2+\frac{\epsilon}{2}\right)\frac{d}{n}\right]\cdot \mathbb{E}\big[|\tilde{\mathbf{Y}}_{\geq 2}^H|\big|\mathbf{x},\mathcal{E}\big]\\
&\leq \left(\frac{3d}{n}\right)^{|E_1(H)\setminus S_V|}\cdot \mathbb{E}\big[|\tilde{\mathbf{Y}}_{\geq 2}^H|\big|\mathbf{x},\mathcal{E}\big],
\end{aligned}
\end{equation}

Combining \cref{eq:eq_lem_UH_deg1_common_eq1} and \cref{eq:eq_lem_UH_deg1_common_eq2}, we get:
\begin{align*}
\big|\mathbb{E}\big[\overline{\mathbf{Y}}_H\big|\mathbf{x},\mathcal{E}\big]\big|&\leq \sum_{V\subseteq  U }\left(\frac{\epsilon d}{2n}\right)^{|S_V|}\cdot \eta^{| U |-|V|} \left(\frac{3d}{n}\right)^{|E_1(H)\setminus S_V|}\cdot \mathbb{E}\big[|\tilde{\mathbf{Y}}_{\geq 2}^H|\big|\mathbf{x},\mathcal{E}\big]\\
&\stackrel{(\wr)}{=}\left(\frac{3d}{n}\right)^{|E_1^a(H)|+|E_1^d(H)|}\cdot\mathbb{E}\big[|\tilde{\mathbf{Y}}_{\geq 2}^H|\big|\mathbf{x},\mathcal{E}\big]\cdot \sum_{V\subseteq  U }\left(\frac{\epsilon d}{2n}\right)^{|S_V|}\cdot \eta^{| U |-|V|} \left(\frac{3d}{n}\right)^{|E_1^b(H)\setminus S_V|}\\
&=\left(\frac{3d}{n}\right)^{|E_1^a(H)|+|E_1^d(H)|}\cdot\mathbb{E}\big[|\tilde{\mathbf{Y}}_{\geq 2}^H|\big|\mathbf{x},\mathcal{E}\big]\cdot \left(\frac{\epsilon d}{2n}\right)^{|E_1^b(H)|}\cdot \sum_{V\subseteq  U } \eta^{| U |-|V|} \left(\frac{6}{\epsilon}\right)^{|E_1^b(H)\setminus S_V|},
\end{align*}
where $(\wr)$ follows from the fact that $\big\{E_1^a(H),E_1^b(H)\setminus S_V, E_1^d(H)\big\}$ is a partition of $E_1(H)\setminus S_V$.

Now since $|S_V|\geq |E_1^b(H)|-\tau\cdot|\mathcal{S}(H)\setminus V|$, we have $$|E_1^b(H)\setminus S_V|\leq \tau\cdot|\mathcal{S}(H)\setminus V|.$$

Therefore,
\begin{align*}
\sum_{V\subseteq  U } \eta^{| U |-|V|} &\left(\frac{6}{\epsilon}\right)^{|E_1^b(H)\setminus S_V|}\\
&\leq \sum_{V\subseteq  U } \eta^{|U|-|V|} \left(\frac{6}{\epsilon}\right)^{\tau\cdot(| \mathcal{S}(H) |-|V|)}= \left(\frac{6}{\epsilon}\right)^{\tau(|\mathcal{S}(H)|-|U|)}\sum_{V\subseteq  U } \left[\eta\left(\frac{6}{\epsilon}\right)^{\tau}\right]^{| U |-|V|}\\
&=\left(\frac{6}{\epsilon}\right)^{\tau(|\mathcal{S}(H)|-|U|)}\left[1+\eta\left(\frac{6}{\epsilon}\right)^{\tau}\right]^{| U |}\leq\left(\frac{6}{\epsilon}\right)^{\tau(|\mathcal{S}(H)|-|U|)}\left[1+\frac{1}{As}\left(\frac{\epsilon}{6}\right)^{\tau}\left(\frac{6}{\epsilon}\right)^{\tau}\right]^{st}\\
&=\left(\frac{6}{\epsilon}\right)^{\tau(|\mathcal{S}(H)|-|U|)}\left(1+\frac{1}{As}\right)^{st}\leq e^{\frac{1}{As}st}\left(\frac{6}{\epsilon}\right)^{\tau(|\mathcal{S}(H)|-|U|)}\\
&=e^{\frac{1}{A}\cdot K\cdot \log n}\left(\frac{6}{\epsilon}\right)^{\tau(|\mathcal{S}(H)|-|U|)}=n^{\frac{K}{A}}\left(\frac{6}{\epsilon}\right)^{\tau(|\mathcal{S}(H)|-|U|)}.
\end{align*}

We conclude that

\begin{align*}
&\big|\mathbb{E}\big[\overline{\mathbf{Y}}_H\big|\mathbf{x},\mathcal{E}\big]\big|\\
&\quad\quad\leq n^{\frac{K}{A}}\left(\frac{6}{\epsilon}\right)^{\tau(|\mathcal{S}(H)|-|U|)}\cdot\left(\frac{3d}{n}\right)^{|E_1^a(H)|+|E_1^d(H)|}\cdot\mathbb{E}\big[|\tilde{\mathbf{Y}}_{\geq 2}^H|\big|\mathbf{x},\mathcal{E}\big]\cdot \left(\frac{\epsilon d}{2n}\right)^{|E_1^b(H)|}\\
&\quad\quad=n^{\frac{K}{A}}\left(\frac{6}{\epsilon}\right)^{\tau(|\mathcal{S}(H)|-|U|)}\cdot\left(\frac{\epsilon d}{2n}\right)^{|E_1^a(H)|+|E_1^d(H)|+|E_1^b(H)|}\cdot \left(\frac{6}{\epsilon}\right)^{|E_1^a(H)|+|E_1^d(H)|}\cdot\mathbb{E}\big[|\tilde{\mathbf{Y}}_{\geq 2}^H|\big|\mathbf{x},\mathcal{E}\big]\\
&\quad\quad=n^{\frac{K}{A}}\left(\frac{6}{\epsilon}\right)^{|E_1^a(H)|+|E_1^d(H)|+\tau(|\mathcal{S}(H)|-|U|)}\cdot\left(\frac{\epsilon d}{2n}\right)^{|E_1(H)|}\cdot\mathbb{E}\big[|\tilde{\mathbf{Y}}_{\geq 2}^H|\big|\mathbf{x},\mathcal{E}\big].
\end{align*}
\end{proof}

\begin{lemma}
\label{lem:lem_UH_deg1_common_E_1d}
Let $H$ be a multigraph with at most $st=sK\log n$ vertices and at most $st$ multi-edges, we have
$$\left(\frac{6}{\epsilon}\right)^{|E_1^d(H)|}\leq n^{\frac{K}{A}}.$$
\end{lemma}
\begin{proof}
Recall that
$$E_1^d(H)=\big\{uv\in E_1(H)\setminus E_1^a(H):\; u\notin \mathcal{S}_{\geq 2}(H)\text{ or }v\notin\mathcal{S}_{\geq 2}(H)\big\}.$$

Since every edge $uv\in E_1^d(H)$ satisfies $uv\in E_1(H)\setminus E_1^a(H)$, we must have $u\in\mathcal{S}_1(H)$ and $v\in\mathcal{S}_1(H)$. On the other hand, every edge in $E_1^d(H)$ is incident to at least one vertex in $V(H)\setminus\mathcal{S}_{\geq 2}(H)=\mathcal{I}_{\geq2}(H)\cup\mathcal{L}_{\geq2}(H)$. Therefore, every edge in $E_1^d(H)$ is a multiplicity-1 edge that is incident to at least one vertex in $\mathcal{S}_1(H)\cap (\mathcal{I}_{\geq2}(H)\cup\mathcal{L}_{\geq2}(H))$.

Since we have at most $st$ edges in $G(H)$, we have 
\begin{equation}
\label{eq:eq_Bound_on_I2_and_L2}
\frac{\Delta}{4}\cdot|\mathcal{I}_{\geq2}(H)\cup\mathcal{L}_{\geq2}(H)|\leq \sum_{v\in \mathcal{I}_{\geq2}(H)\cup\mathcal{L}_{\geq2}(H)}d_{G(H)}(v)\leq \sum_{v\in V(H)}d_{G(H)}(v) \leq 2st.
\end{equation}
Therefore,
\begin{align*}
\big|\mathcal{S}_1(H)\cap \big(\mathcal{I}_{\geq2}(H)\cup\mathcal{L}_{\geq2}(H)\big)\big|&\leq |\mathcal{I}_{\geq2}(H)\cup\mathcal{L}_{\geq2}(H)|\\
&\leq \frac{8st}{\Delta}\stackrel{(\ast)}{\leq} \frac{8st}{8\tau^2 A^2s^2 \left(\log\frac{6}{\epsilon}\right)^2}\leq \frac{t}{A\tau \left(\log\frac{6}{\epsilon}\right)},
\end{align*}
where $(\ast)$ follows from \cref{eq:eq_Delta_form}. Now since every edge in $E_1^d(H)$ is is incident to at least one vertex in $\mathcal{S}_1(H)\cap (\mathcal{I}_{\geq2}(H)\cup\mathcal{L}_{\geq2}(H))$, and since every vertex in $\mathcal{S}_1(H)$ is incident to at most $\tau$ multiplicity-1 edges, we conclude that
\begin{align*}
|E_1^d(H)|\leq \frac{t}{A\tau \left(\log\frac{6}{\epsilon}\right)}\cdot \tau\leq \frac{t}{A\log\frac{6}{\epsilon}},
\end{align*}
and
$$\left(\frac{6}{\epsilon}\right)^{|E_1^d(H)|}\leq  e^{\frac{t}{A\log\frac{6}{\epsilon}}\log\frac{6}{\epsilon}}=e^{\frac{1}{A}\cdot K\log n}=n^{\frac{K}{A}}.$$
\end{proof}

\begin{lemma}
\label{lem:lem_UH_deg1_common_2}
Let $H$ be a multigraph with at most $st=sK\log n$ vertices and at most $st$ multi-edges, and let $\mathcal{S}(H)=\mathcal{S}_1(H)\cap\mathcal{S}_{\geq 2}(H)$. Assume that $U\subseteq \mathcal{S}(H)$ and $\mathcal{E}$ satisfy the conditions of \Cref{lem:lem_UH_deg1_common}. Furthermore, assume that $\mathcal{E}$ satisfies the following additional condition:
\begin{itemize}
\item[(d)] Given $\mathbf{x}$, the event $\mathcal{E}$ is conditionally independent from $(\mathbbm{1}_{\{uv\in \mathbf{G}\}})_{uv\in E_{\geq 2}(H)}$.
\end{itemize}
Then,
$$\big|\mathbb{E}\big[\overline{\mathbf{Y}}_H\big|\mathbf{x},\mathcal{E}\big]\big|\leq n^{\frac{2K}{A}}\left(\frac{6}{\epsilon}\right)^{|E_1^a(H)|+\tau(|\mathcal{S}(H)|-|U|)}\left(\frac{\epsilon d}{2n}\right)^{|E_1(H)|} \cdot\mathbb{E}\big[|\tilde{\mathbf{Y}}_{\geq 2}^H|\big|\mathbf{x}\big].$$
\end{lemma}
\begin{proof}
Let $E_1^b(H)$ and $E_1^d(H)$ be as in \cref{eq:eq_E_1b_H} and \cref{eq:eq_E_1c_H}, respectively. From \Cref{lem:lem_UH_deg1_common}, we know that
$$\big|\mathbb{E}\big[\overline{\mathbf{Y}}_H\big|\mathbf{x},\mathcal{E}\big]\big|\leq n^{\frac{K}{A}}\left(\frac{6}{\epsilon}\right)^{|E_1^a(H)|+|E_1^d(H)|+\tau(|\mathcal{S}(H)|-|U|)}\cdot\left(\frac{\epsilon d}{2n}\right)^{|E_1(H)|} \cdot\mathbb{E}\big[|\tilde{\mathbf{Y}}_{\geq 2}^H|\big|\mathbf{x},\mathcal{E}\big].$$
On the other hand, Property (d) implies that $\mathbb{E}\big[|\tilde{\mathbf{Y}}_{\geq 2}^H|\big|\mathbf{x},\mathcal{E}\big]=\mathbb{E}\big[|\tilde{\mathbf{Y}}_{\geq 2}^H||\mathbf{x}\big]$. Therefore,
\begin{equation}
\label{eq:eq_lem_UH_deg1_common_2_eq1}
\big|\mathbb{E}\big[\overline{\mathbf{Y}}_H\big|\mathbf{x},\mathcal{E}\big]\big|\leq n^{\frac{K}{A}}\left(\frac{6}{\epsilon}\right)^{|E_1^a(H)|+|E_1^d(H)|+\tau(|\mathcal{S}(H)|-|U|)}\cdot\left(\frac{\epsilon d}{2n}\right)^{|E_1(H)|} \cdot\mathbb{E}\big[|\tilde{\mathbf{Y}}_{\geq 2}^H|\big|\mathbf{x}\big].
\end{equation}

Combining this with \Cref{lem:lem_UH_deg1_common_E_1d}, we get
\begin{align*}
\left|\mathbb{E}\big[\overline{\mathbf{Y}}_H\big|\mathbf{x},\mathcal{E}\big]\right|&\leq n^{\frac{2K}{A}}\left(\frac{6}{\epsilon}\right)^{|E_1^a(H)|+\tau(|\mathcal{S}(H)|-|U|)}\left(\frac{\epsilon d}{2n}\right)^{|E_1(H)|} \cdot\big|\mathbb{E}\big[|\tilde{\mathbf{Y}}_{\geq 2}^H|\big|\mathbf{x}\big]\big|.
\end{align*}
\end{proof}

Now we are ready to prove \Cref{lem:lem_UH_deg1}.

\begin{proof}[Proof of \Cref{lem:lem_UH_deg1}]
Let $\mathcal{S}(H)=\mathcal{S}_1(H)\cap\mathcal{S}_{\geq 2}(H)$. By using \Cref{lem:lem_bound_prob_completely_unsafe}, it can be easily seen that if we take $U=\mathcal{S}(H)$ and $\mathcal{E}$ to be an "almost sure event", i.e., $\mathbb{P}[\mathcal{E}]=1$, then the conditions of \Cref{lem:lem_UH_deg1_common} and \Cref{lem:lem_UH_deg1_common_2} are satisfied. Therefore,
\begin{align*}
\big|\mathbb{E}\big[\overline{\mathbf{Y}}_H\big|\mathbf{x}\big]\big|&=\big|\mathbb{E}\big[\overline{\mathbf{Y}}_H\big|\mathbf{x},\mathcal{E}\big]\big|\\
&\leq n^{\frac{2K}{A}}\left(\frac{6}{\epsilon}\right)^{|E_1^a(H)|}\left(\frac{\epsilon d}{2n}\right)^{|E_1(H)|} \cdot\big|\mathbb{E}\big[|\tilde{\mathbf{Y}}_{\geq 2}^H|\big|\mathbf{x}\big]\big|.
\end{align*}
\end{proof}

\subsubsection{Upper bound on the contribution of edges of multiplicity at least 2}
\label{subsubsec:subsubsec_edges_mult_2_upper}

\begin{proof}[Proof of \Cref{lem:lem_UH_deg2}]
Recall the definitions of $E_{\geq2}^a(H)$ and $E_{\geq2}^b(H)$ from \Cref{def:def_E1_E2_annoying}.

We have
\begin{align*}
|\tilde{\mathbf{Y}}_{\geq 2}^H|&=\prod_{uv\in E_{\geq 2}(H)}\big|\tilde{\mathbf{Y}}_{uv,E_{\geq 2}(H)}^{m_H(uv)}\big|\\
&=\prod_{uv\in E_{\geq 2}(H)}\big|\mathbf{Y}_{uv}^{m_H(uv)}\big|\cdot\mathbbm{1}_{\{d_{\mathbf{G}\cap E_{\geq 2}(H)}(u)\leq \Delta\}}\cdot \mathbbm{1}_{\{d_{\mathbf{G}\cap E_{\geq 2}(H)}(v)\leq \Delta\}}\\
&\leq \left(\prod_{uv\in E_{\geq2}^a(H)}\big|\mathbf{Y}_{uv}^{m_H(uv)}\big|\right)\cdot\left(\prod_{uv\in E_{\geq2}^b(H)}\big|\mathbf{Y}_{uv}^{m_H(uv)}\big|\right)\cdot\left(\prod_{v\in\mathcal{L}_{\geq2}(H)}\mathbbm{1}_{\{d_{\mathbf{G}\cap E_{\geq2}^b(H)}(v)\leq \Delta\}}\right)\\
&=\big|\mathbf{Y}_{H_{\geq2}^a}\big|\cdot\big|\mathbf{Y}_{H_{\geq2}^b}\big|\cdot \mathbbm{1}_{\mathcal{E}_{\geq 2}^{b,H}},
\end{align*}
where
$$\mathbf{Y}_{H_{\geq2}^a}=\prod_{uv\in E_{\geq2}^a(H)}\mathbf{Y}_{uv}^{m_H(uv)},$$
$$\mathbf{Y}_{H_{\geq2}^b}=\prod_{uv\in E_{\geq2}^b(H)}\mathbf{Y}_{uv}^{m_H(uv)},$$
and
\begin{align*}
\mathcal{E}_{\geq 2}^{b,H}&=\bigcap_{v\in\mathcal{L}_{\geq2}(H)}\big\{d_{\mathbf{G}\cap E_{\geq2}^b(H)}(v)\leq \Delta\big\}.
\end{align*}

Therefore,

\begin{equation}
\label{eq:eq_lem_UH_deg2_eq1}
\begin{aligned}
\mathbb{E}\big[|\tilde{\mathbf{Y}}_{\geq 2}^H|\big|\mathbf{x}\big]&= \mathbb{E}\big[|\mathbf{Y}_{H_{\geq2}^a}|\cdot|\mathbf{Y}_{H_{\geq2}^b}|\cdot \mathbbm{1}_{\mathcal{E}_{\geq 2}^{b,H}}\big|\mathbf{x}\big]\\
&\stackrel{(\ast)}{=}\mathbb{E}\big[|\mathbf{Y}_{H_{\geq2}^a}|\big|\mathbf{x}\big]\cdot \mathbb{E}\big[|\mathbf{Y}_{H_{\geq2}^b}|\cdot \mathbbm{1}_{\mathcal{E}_{\geq 2}^{b,H}}\big|\mathbf{x}\big],
\end{aligned}
\end{equation}
where $(\ast)$ follows from the fact that given $\mathbf{x}$, the random variable $\mathbf{Y}_{H_{\geq2}^a}$ is conditionally independent from $(\mathbf{Y}_{H_{\geq2}^b},\mathbbm{1}_{\mathcal{E}_{\geq 2}^{b,H}})$. This is because the presence or absence of edges in $E_{\geq2}^a(H)$ do not affect the degree $d_{\mathbf{G}\cap E_{\geq2}^b(H)}(v)$ of any vertex $v\in\mathcal{L}_{\geq2}(H)$.

Now let us evaluate $\mathbb{E}\big[|\mathbf{Y}_{H_{\geq2}^a}|\big|\mathbf{x}\big]$ and $\mathbb{E}\big[|\mathbf{Y}_{H_{\geq2}^b}|\cdot \mathbbm{1}_{\mathcal{E}_{\geq 2}^{b,H}}\big|\mathbf{x}\big]$. We have

\begin{equation}
\label{eq:eq_lem_UH_deg2_eq2}
\begin{aligned}
\mathbb{E}\big[|\mathbf{Y}_{H_{\geq2}^a}|\big|\mathbf{x}\big]
&=\prod_{uv\in E_{\geq2}^a(H)}\mathbb{E}\left[|\mathbf{Y}_{uv}|^{m_H(uv)}\middle|\mathbf{x}\right]\\
&=\prod_{uv\in E_{\geq2}^a(H)}\left[\left(1+\frac{\epsilon \mathbf{x}_u\mathbf{x}_v}{2}\right)\frac{d}{n}\left(1-\frac{d}{n}\right)^{m_H(uv)}+\left(1-\left(1+\frac{\epsilon \mathbf{x}_u\mathbf{x}_v}{2}\right)\frac{d}{n}\right)\left(\frac{d}{n}\right)^{m_H(uv)}\right]\\
&\leq\prod_{uv\in E_{\geq2}^a(H)}\left[\left(1+\frac{\epsilon \mathbf{x}_u\mathbf{x}_v}{2}\right)\frac{d}{n}+\frac{d^2}{n^2}\right].
\end{aligned}
\end{equation}

On the other hand,

\begin{equation}
\label{eq:eq_lem_UH_deg2_eq3}
\begin{aligned}
&\mathbb{E}\big[|\mathbf{Y}_{H_{\geq2}^b}|\cdot \mathbbm{1}_{\mathcal{E}_{\geq 2}^{b,H}}\big|\mathbf{x}\big]\\
&=\sum_{\substack{S\subseteq E_{\geq2}^b(H):\\\forall v\in\mathcal{L}_{\geq2}(H), d_{S}(v)\leq \Delta}}\left[\prod_{uv\in S}\left(1+\frac{\epsilon \mathbf{x}_u\mathbf{x}_v}{2}\right)\frac{d}{n}\left(1-\frac{d}{n}\right)^{m_H(uv)}\right]\\
&\quad\quad\quad\quad\quad\quad\quad\quad\quad\quad\quad\quad\quad\quad\quad\quad\quad\quad\times\left[\prod_{uv\in E_{\geq2}^b(H)\setminus S} \left(1-\left(1+\frac{\epsilon \mathbf{x}_u\mathbf{x}_v}{2}\right)\frac{d}{n}\right)\left(\frac{d}{n}\right)^{m_H(uv)}\right]\\
&\leq\sum_{\substack{S\subseteq E_{\geq2}^b(H):\\\forall v\in\mathcal{L}_{\geq2}(H),\; d_{S}(v)\leq \Delta}}\left[\prod_{uv\in S}\frac{2d}{n}\right]\cdot \left[\prod_{uv\in E_{\geq2}^b(H)\setminus S} \left(\frac{2d}{n}\right)^{2}\right]=\sum_{\substack{S\subseteq E_{\geq2}^b(H):\\\forall v\in\mathcal{L}_{\geq2}(H),\; d_{S}(v)\leq \Delta}}\left(\frac{2d}{n}\right)^{|S|+2|E_{\geq2}^b(H)|-2|S|}\\
&=\left(\frac{2d}{n}\right)^{|E_{\geq2}^b(H)|} \sum_{\substack{S\subseteq E_{\geq2}^b(H):\\\forall v\in\mathcal{L}_{\geq2}(H),\; d_{S}(v)\leq \Delta}}\left(\frac{2d}{n}\right)^{|E_{\geq2}^b(H)|-|S|}.
\end{aligned}
\end{equation}

Now define
$$E_{\geq 2}^{b,i}(H)=\big\{uv\in E_{\geq2}^b(H):\; u\in\mathcal{L}_{\geq2}(H)\text{ and }v\in\mathcal{L}_{\geq2}(H)\big\}.$$
The superscript $i$ indicates that the edges in $E_{\geq 2}^{b,i}(H)$ are internal to $\mathcal{L}_{\geq2}(H)$, i.e., both end-vertices are in $\mathcal{L}_{\geq2}(H)$. Furthermore, for every $v\in\mathcal{L}_{\geq2}(H)$, define:
$$E_{\geq2}^b(v,H)=\big\{uv:\;uv\in E_{\geq2}^b(H)\big\},$$
$$E_{\geq2}^{b,i}(v,H)=\big\{uv:\; uv\in E_{\geq2}^b(H)\text{ and }u\in\mathcal{L}_{\geq2}(H)\big\},$$
and
$$E_{\geq2}^{b,o}(v,H)=\big\{uv:\; uv\in E_{\geq2}^b(H)\text{ and }u\notin\mathcal{L}_{\geq2}(H)\big\}.$$

The superscript $o$ in $E_{\geq 2}^{b,o}(v,H)$ indicates that the edges in $E_{\geq 2}^{b,o}(v,H)$ go from $v$ to $V(H)\setminus \mathcal{L}_{\geq2}(H)$, i.e., they go to the "outside" of $\mathcal{L}_{\geq2}(H)$. We have the following:
\begin{itemize}
\item For every $v\in\mathcal{L}_{\geq2}(H)$, $\big\{E_{\geq2}^{b,i}(v,H),E_{\geq2}^{b,o}(v,H)\big\}$ is a partition of $E_{\geq2}^b(v,H)$.
\item For every $v\in\mathcal{L}_{\geq2}(H)$, we have $|E_{\geq2}^b(v,H)|=d^H_{\geq 2}(v)>\Delta$.
\item $\big\{E_{\geq 2}^{b,i}(H)\big\}\cup\big\{E_{\geq2}^{b,o}(v,H):v\in\mathcal{L}_{\geq2}(H)\big\}$ is a partition of $E_{\geq2}^b(H)$.
\end{itemize}

Now for every $S\subseteq E_{\geq2}^b(H)$, define
$$S^i=S\cap E_{\geq 2}^{b,i}(H),$$
and for every $v\in\mathcal{L}_{\geq2}(H)$, define
$$S_v=S\cap E_{\geq2}^b(v,H),\quad S^i_v=S\cap E_{\geq 2}^{b,i}(v,H)\quad\text{and}\quad S^o_v=S\cap E_{\geq 2}^{b,o}(v,H).$$
It is easy to see that $S^i_v=S_v\cap E_{\geq 2}^{b,i}(v,H)=S^i\cap E_{\geq 2}^{b,i}(v,H)$ and $S^o_v=S_v\cap E_{\geq 2}^{b,o}(v,H)$.

Now for two sequences of sets $(A_i)_{i\in I}$ and $(B_i)_{i\in I}$ that are indexed by the same index set $I$, we write $(A_i)_{i\in I}\subseteq(B_i)_{i\in I}$ to indicate that $A_i\subseteq B_i$ for all $i\in I$. We have
\begin{align*}
&\sum_{\substack{S\subseteq E_{\geq2}^b(H):\\\forall v\in\mathcal{L}_{\geq2}(H),\; d_{S}(v)\leq \Delta}}\left(\frac{2d}{n}\right)^{|E_{\geq2}^b(H)|-|S|}\\
&\quad\quad\quad\quad\quad\quad=\sum_{\substack{S^i\subseteq E_{\geq2}^{b,i}(H),\\(S^o_v)_{v\in\mathcal{L}_2(H)}\subseteq (E_{\geq 2}^{b,o}(v,H))_{v\in\mathcal{L}_2(H)}:\\\forall v\in\mathcal{L}_{\geq2}(H),\; |S^i_v|+|S^o_v|\leq \Delta}}\left(\frac{2d}{n}\right)^{|E_{\geq2}^b(H)|-|S^i|-\sum_{v\in\mathcal{L}_{\geq2}(H)}|S^o_v|}\\
&\quad\quad\quad\quad\quad\quad=\sum_{\substack{S^i\subseteq E_{\geq2}^{b,i}(H),\\(S^o_v)_{v\in\mathcal{L}_2(H)}\subseteq (E_{\geq 2}^{b,o}(v,H))_{v\in\mathcal{L}_2(H)}:\\\forall v\in\mathcal{L}_{\geq2}(H),\; |S^i_v|+|S^o_v|\leq \Delta}}\left(\frac{2d}{n}\right)^{|E_{\geq2}^{b,i}(H)|-|S^i|}\prod_{v\in\mathcal{L}_{\geq2}(H)}\left(\frac{2d}{n}\right)^{|E_{\geq2}^{b,o}(v,H)|-|S^o_v|},
\end{align*}
hence,
\begin{align*}
&\sum_{\substack{S\subseteq E_{\geq2}^b(H):\\\forall v\in\mathcal{L}_{\geq2}(H),\; d_{S}(v)\leq \Delta}}\left(\frac{2d}{n}\right)^{|E_{\geq2}^b(H)|-|S|}\\
&\quad\quad=\sum_{\substack{S^i\subseteq E_{\geq2}^{b,i}(H)}}\left(\frac{2d}{n}\right)^{|E_{\geq2}^{b,i}(H)|-|S^i|}\prod_{v\in\mathcal{L}_{\geq2}(H)}\left[\sum_{\substack{S^o_v\subseteq E_{\geq 2}^{b,o}(v,H):\\ |S^i_v|+|S^o_v|\leq \Delta}}\left(\frac{2d}{n}\right)^{|E_{\geq2}^{b,o}(v,H)|-|S^o_v|}\right]\\
&\quad\quad\leq\sum_{\substack{S^i\subseteq E_{\geq2}^{b,i}(H)}}\left(\sqrt{\frac{2d}{n}}\right)^{2|E_{\geq2}^{b,i}(H)|-2|S^i|}\prod_{v\in\mathcal{L}_{\geq2}(H)}\left[\sum_{\substack{S^o_v\subseteq E_{\geq 2}^{b,o}(v,H):\\ |S^i_v|+|S^o_v|\leq \Delta}}\left(\sqrt{\frac{2d}{n}}\right)^{|E_{\geq2}^{b,o}(v,H)|-|S^o_v|}\right]\\
&\quad\quad=\sum_{\substack{S^i\subseteq E_{\geq2}^{b,i}(H)}}\left(\sqrt{\frac{2d}{n}}\right)^{\sum_{v\in\mathcal{L}_{\geq2}(H)}2(|E_{\geq2}^{b,i}(v,H)|-|S^i_v|)}\prod_{v\in\mathcal{L}_{\geq2}(H)}\left[\sum_{\substack{S^o_v\subseteq E_{\geq 2}^{b,o}(v,H):\\ |S^i_v|+|S^o_v|\leq \Delta}}\left(\sqrt{\frac{2d}{n}}\right)^{|E_{\geq2}^{b,o}(v,H)|-|S^o_v|}\right].
\end{align*}

Therefore,
\begin{equation}
\label{eq:eq_lem_UH_deg2_eq4}
\begin{aligned}
\sum_{\substack{S\subseteq E_{\geq2}^b(H):\\\forall v\in\mathcal{L}_{\geq2}(H),\; d_{S}(v)\leq \Delta}}&\left(\frac{2d}{n}\right)^{|E_{\geq2}^b(H)|-|S|}\\
&\leq\sum_{\substack{S^i\subseteq E_{\geq2}^{b,i}(H)}}\prod_{v\in\mathcal{L}_{\geq2}(H)}\left[\sum_{\substack{S^o_v\subseteq E_{\geq 2}^{b,o}(v,H):\\ |S^i_v|+|S^o_v|\leq \Delta}}\left(\sqrt{\frac{2d}{n}}\right)^{|E_{\geq2}^{b,o}(v,H)|-|S^o_v|+|E_{\geq2}^{b,i}(v,H)|-|S^i_v|}\right]\\
&\leq\prod_{v\in\mathcal{L}_{\geq2}(H)}\left[\sum_{\substack{S^i_v\subseteq E_{\geq2}^{b,i}(v,H)}}\sum_{\substack{S^o_v\subseteq E_{\geq 2}^{b,o}(v,H):\\ |S^i_v|+|S^o_v|\leq \Delta}}\left(\sqrt{\frac{2d}{n}}\right)^{|E_{\geq2}^{b,o}(v,H)|-|S^o_v|+|E_{\geq2}^{b,i}(v,H)|-|S^i_v|}\right]\\
&=\prod_{v\in\mathcal{L}_{\geq2}(H)}\left[\sum_{\substack{S_v\subseteq E_{\geq2}^b(v,H):\\ |S_v|\leq \Delta}}\left(\sqrt{\frac{2d}{n}}\right)^{|E_{\geq2}^b(v,H)|-|S_v|}\right].
\end{aligned}
\end{equation}

Now for every $v\in\mathcal{L}_{\geq2}(H)$, we have
\begin{align*}
\sum_{\substack{S_v\subseteq E_{\geq2}^b(v,H):\\ |S_v|\leq \Delta}}&\left(\sqrt{\frac{2d}{n}}\right)^{|E_{\geq2}^b(v,H)|-|S_v|}\leq (\Delta+1){|E_{\geq2}^b(v,H)|\choose \Delta}\left(\sqrt{\frac{2d}{n}}\right)^{|E_{\geq2}^b(v,H)|-\Delta}\\
&\leq (\Delta+1){st\choose \Delta}\left(\sqrt{\frac{2d}{n}}\right)^{|E_{\geq2}^b(v,H)|-\Delta}\leq (\Delta+1)\frac{(st)^{\Delta}}{\Delta!}\left(\sqrt{\frac{2d}{n}}\right)^{d^H_{\geq 2}(v)-\Delta},
\end{align*}

By combining this with \cref{eq:eq_lem_UH_deg2_eq3} and \cref{eq:eq_lem_UH_deg2_eq4}, we get

\begin{align*}
\mathbb{E}\big[|\mathbf{Y}_{H_{\geq2}^b}|\cdot \mathbbm{1}_{\mathcal{E}_{\geq 2}^{b,H}}\big|\mathbf{x}\big]
&\leq\left(\frac{2d}{n}\right)^{|E_{\geq2}^b(H)|} \prod_{v\in\mathcal{L}_{\geq2}(H)}\left[(\Delta+1)\frac{(st)^{\Delta}}{\Delta!}\left(\sqrt{\frac{2d}{n}}\right)^{d^H_{\geq 2}(v)-\Delta}\right]\\
&\stackrel{(\dagger)}{\leq}\left(\frac{d}{n}\right)^{|E_{\geq2}^b(H)|} \prod_{v\in\mathcal{L}_{\geq2}(H)}\left[2^{d^H_{\geq 2}(v)}\cdot(\Delta+1)\frac{(st)^{\Delta}}{\Delta!}\left(\sqrt{\frac{2d}{n}}\right)^{d^H_{\geq 2}(v)-\Delta}\right]\\
&=\left(\frac{d}{n}\right)^{|E_{\geq2}^b(H)|} \prod_{v\in\mathcal{L}_{\geq2}(H)}\left[2^{\Delta}\cdot(\Delta+1)\frac{(st)^{\Delta}}{\Delta!}\left(2\sqrt{\frac{2d}{n}}\right)^{d^H_{\geq 2}(v)-\Delta}\right]\\
&\stackrel{(\ddagger)}{\leq}\left(\frac{d}{n}\right)^{|E_{\geq2}^b(H)|} \frac{1}{\resizebox{0.055\textwidth}{!}{$\displaystyle\prod_{v\in\mathcal{L}_{\geq2}(H)}$} n^{\frac{1}{4}\left(d^H_{\geq 2}(v)-\Delta\right)}},
\end{align*}
where $(\dagger)$ follows from the fact that $\displaystyle |E_{\geq2}^b(H)|\leq \sum_{v\in\mathcal{L}_{\geq2}(H)}d_{\geq 2}^H(v)$, and $(\ddagger)$ is true for $n$ large enough\footnote{Recall that $d_{\geq 2}^H(v)>\Delta$ for all $v\in\mathcal{L}_{\geq 2}(H)$}. By combining this with \cref{eq:eq_lem_UH_deg2_eq1} and \cref{eq:eq_lem_UH_deg2_eq2}, we get

\begin{align*}
\mathbb{E}\big[|\tilde{\mathbf{Y}}_{\geq 2}^H|\big|\mathbf{x}\big]&\leq \frac{1}{\resizebox{0.055\textwidth}{!}{$\displaystyle\prod_{v\in\mathcal{L}_{\geq2}(H)}$} n^{\frac{1}{4}\left(d^H_{\geq 2}(v)-\Delta\right)}}\left(\frac{d}{n}\right)^{|E_{\geq2}^b(H)|}\prod_{uv\in E_{\geq2}^a(H)}\left[\left(1+\frac{\epsilon \mathbf{x}_u\mathbf{x}_v}{2}\right)\frac{d}{n}+\frac{d^2}{n^2}\right]\\
&= \frac{1}{\resizebox{0.055\textwidth}{!}{$\displaystyle\prod_{v\in\mathcal{L}_{\geq2}(H)}$} n^{\frac{1}{4}\left(d^H_{\geq 2}(v)-\Delta\right)}}\left(\frac{d}{n}\right)^{|E_{\geq2}(H)|}\prod_{uv\in E_{\geq2}^a(H)}\left[1+\frac{\epsilon \mathbf{x}_u\mathbf{x}_v}{2}+\frac{d}{n}\right].
\end{align*}
\end{proof}

\subsubsection{Upper bounds for the contribution of the not-well-behaved event}
\label{subsubsec:subsubsec_upper_bounding_not_well_behaved}

We will prove \Cref{lem:lem_upper_bound_nice_Wc} in three steps. We start by proving an upper bound on $\left|\mathbb{E}\big[\overline{\mathbf{Y}}_H\big|\mathcal{E}_{H,b}^c\big]\cdot\mathbb{P}[\mathcal{E}_{H,b}^c]\right|$, where $\mathcal{E}_{H,b}$ is as in \cref{eq:eq_def_E_wb_H_x}.

\begin{lemma}
\label{lem:lem_Prob_Balanced_Y}
Let $H$ be a multigraph with at most $st=sK\log n$ vertices and let $\mathcal{E}_{H,b}$ be the event that $\mathbf{x}$ is approximately balanced on $[n]\setminus V(H)$. If $n$ is large enough, then
$$\mathbb{P}[\mathcal{E}_{H,b}^c]\leq 2e^{-\frac{9}{8}\sqrt{n}}.$$
\end{lemma}
\begin{proof}
Define
$$S_{[n]\setminus V(H)}=\sum_{v\in [n]\setminus V(H)}\mathbf{x}_v,$$
where the sum is performed according to the arithmetic of integers in $\mathbb{Z}$. It is easy to see that if $\mathcal{E}_{H,b}^c$ occurs, then $|S_{[n]\setminus V(H)}|>  n-|V(H)|-2\left\lceil\frac{n}{2}- n^{\frac{3}{4}}\right\rceil$. Therefore,
\begin{align*}
\mathbb{P}[\mathcal{E}_{H,b}^c]
&\leq\resizebox{0.85\textwidth}{!}{$\mathbb{P}\left[|S_{[n]\setminus V(H)}|\geq n-|V(H)|-2\left\lceil\frac{n}{2}- n^{\frac{3}{4}}\right\rceil\right]\leq\mathbb{P}\left[|S_{[n]\setminus V(H)}|\geq  n-st -2\left(\frac{n}{2}- n^{\frac{3}{4}} +1\right)\right]$}\\
&=\mathbb{P}\left[|S_{[n]\setminus V(H)}|\geq  2\left(n^{\frac{3}{4}} -\frac{st}{2} -1\right)\right]\leq\mathbb{P}\left[|S_{[n]\setminus V(H)}|\geq  \frac{3}{2} \cdot n^{\frac{3}{4}}\right],
\end{align*}
where the last inequality is true for $n$ large enough. By applying Hoeffding's inequality, we get:
\begin{align*}
\mathbb{P}[\mathcal{E}_{H,b}^c]
&\leq 2\cdot e^{-2\frac{\left(\frac{3}{2}\cdot  n^{\frac{3}{4}}\right)^2}{(n-|V(H)|)\cdot (1-(-1))^2}}= 2e^{-\frac{18 n^{\frac{3}{2}}}{16(n-|V(H)|)}}\leq 2 e^{-\frac{9 n^{\frac{3}{2}}}{8n}}=  2e^{-\frac{9}{8}\sqrt{n}}.
\end{align*}
\end{proof}

\begin{lemma}
\label{lem:lem_upper_bound_nice_y}
Let $H$ be a multigraph with at most $st=sK\log n$ vertices and at most $st$ multi-edges, and assume that $E_1^a(H)=\varnothing$. If $n$ is large enough, then
$$\left|\mathbb{E}\big[\overline{\mathbf{Y}}_H\big|\mathcal{E}_{H,b}^c\big]\cdot\mathbb{P}[\mathcal{E}_{H,b}^c]\right|\leq e^{-\sqrt{n}}\cdot\left(\frac{\epsilon d}{2n}\right)^{|E_1(H)|}\cdot\left(\frac{d}{n}\right)^{|E_{\geq 2}(H)|}.$$
\end{lemma}
\begin{proof}
Since $E_1^a(H)=\varnothing$, it follows from \Cref{def:def_upper_bound_UH} and \Cref{lem:lem_UH} that for $n$ large enough, we have
\begin{align*}
\big|\mathbb{E}\big[\overline{\mathbf{Y}}_H\big| \mathbf{x}\big]\big|&\leq \overline{U}_H(\mathbf{x})\\
&=\resizebox{0.85\textwidth}{!}{$\displaystyle n^{\frac{2K}{A}}\left(\frac{6}{\epsilon}\right)^{|E_1^a(H)|}\frac{1}{\resizebox{0.07\textwidth}{!}{$\displaystyle\prod_{v\in\mathcal{L}_{\geq2}(H)}$} n^{\frac{1}{4}\left(d^H_{\geq 2}(v)-\Delta\right)}}\left(\frac{\epsilon d}{2n}\right)^{|E_1(H)|}\left(\frac{d}{n}\right)^{|\E_{\geq2}(H)|}\prod_{uv\in E_{\geq2}^a(H)}\left[1+\frac{\epsilon \mathbf{x}_u\mathbf{x}_v}{2}+\frac{3d}{\sqrt{n}}\right].$}\\
&\stackrel{(\ast)}{\leq} n^{\frac{2K}{A}}\left(\frac{\epsilon d}{2n}\right)^{|E_1(H)|}\left(\frac{d}{n}\right)^{|E_{\geq2}(H)|}3^{|E_{\geq 2}^a(H)|}\leq n^{\frac{2K}{A}}\left(\frac{\epsilon d}{2n}\right)^{|E_1(H)|}\left(\frac{d}{n}\right)^{|E_{\geq 2}(H)|}\cdot 3^{st},
\end{align*}
where $(\ast)$ is true for $n$ large enough. Since $t=K\cdot \log n$, we have $3^{st}=e^{(\log 3)sK\log n}=n^{sK\log 3}$. Therefore,

\begin{align*}
\big|\mathbb{E}\big[\overline{\mathbf{Y}}_H\big| \mathbf{x}\big]\big|\leq n^{\frac{2K}{A}+sK\log 3}\left(\frac{\epsilon d}{2n}\right)^{|E_1(H)|}\left(\frac{d}{n}\right)^{|E_{\geq 2}(H)|}.
\end{align*}

Now since the event $\mathcal{E}_{H,b}^c$ depends only on $\mathbf{x}$, i.e., it is $\sigma(\mathbf{x})$-measurable, we deduce that for $n$ large enough, we have

\begin{align*}
\left|\mathbb{E}\big[\overline{\mathbf{Y}}_H\big|\mathcal{E}_{H,b}^c\big]\cdot\mathbb{P}[\mathcal{E}_{H,b}^c]\right|&\leq n^{\frac{2K}{A}+sK\log 3}\left(\frac{\epsilon d}{2n}\right)^{|E_1(H)|}\left(\frac{d}{n}\right)^{|E_{\geq 2}(H)|} \cdot\mathbb{P}[\mathcal{E}_{H,b}^c]\\
&\stackrel{(\dagger)}{\leq} n^{\frac{2K}{A}+sK\log 3}\left(\frac{\epsilon d}{2n}\right)^{|E_1(H)|}\left(\frac{d}{n}\right)^{|E_{\geq 2}(H)|}\cdot 2e^{-\frac{9}{8}\sqrt{n}}\\
&\stackrel{(\ddagger)}{\leq} \left(\frac{\epsilon d}{2n}\right)^{|E_1(H)|}\left(\frac{d}{n}\right)^{|E_{\geq 2}(H)|}\cdot e^{-\sqrt{n}},
\end{align*}
where $(\dagger)$ follows from \Cref{lem:lem_Prob_Balanced_Y} and $(\ddagger)$ is true for $n$ large enough.
\end{proof}

Next, we prove an upper bound on $\left|\mathbb{E}\big[\overline{\mathbf{Y}}_H\big|\mathcal{E}_{H,b\overline{g}}\big]\cdot\mathbb{P}[\mathcal{E}_{H,b\overline{g}}]\right|$, where $\mathcal{E}_{H,b\overline{g}}$ is as in \cref{eq:eq_def_E_wb_H_xog}. We will need the following definition:

\begin{definition}
\label{def:def_unlucky}
Assume that $\mathbf{x}$ is approximately balanced on $[n]\setminus V(H)$. We say that a vertex $v\in V(H)$ is $(\mathbf{G},\mathbf{x},H)$-\emph{lucky} if it satisfies the following two conditions:
\begin{itemize}
\item[(1)] There is no $H$-cross-edge that is present in $\mathbf{G}$, and which is incident to it. In other words\footnote{Recall \Cref{def:def_in_out_degree}}, $d_{\mathbf{G}-G(H)}^i(v)=0$.
\item[(2)] In the sampled graph $\mathbf{G}$, the vertex $v$ is not adjacent to any vertex in $\big([n]\setminus V(H)\big)\setminus V_b(H,\mathbf{x})$.
\end{itemize}

We say that $v\in V(H)$ is $(\mathbf{G},\mathbf{x},H)$-\emph{unlucky} if it is not $(\mathbf{G},\mathbf{x},H)$-\emph{lucky}.

A subset of $V(H)$ is said to be \emph{completely} $(\mathbf{G},\mathbf{x},H)$-\emph{lucky} if all the vertices in it are $(\mathbf{G},\mathbf{x},H)$-\emph{lucky}. We say that it is \emph{completely} $(\mathbf{G},\mathbf{x},H)$-\emph{unlucky} if all the vertices in it are $(\mathbf{G},\mathbf{x},H)$-\emph{unlucky}.

If $\mathbf{G},\mathbf{x}$ and $H$ are clear from the context, we drop $(\mathbf{G},\mathbf{x},H)$ and simply write lucky, unlucky, completely lucky, and completely unlucky.
\end{definition}

\begin{lemma}
\label{lem:lem_bound_prob_completely_unlucky}
Let $H$ be a multigraph with at most $st=sK\log n$ vertices. If $n$ is large enough, then for every $S\subseteq V(H)$, the conditional probability that $S$ is completely $(\mathbf{G},\mathbf{x},H)$-unlucky given $\mathbf{x}$ that is approximately balanced on $[n]\setminus V(H)$ can be upper bounded by:
$$\mathbb{P}\left[\big\{S\text{ is completely }(\mathbf{G},\mathbf{x},H)\text{-unlucky}\big\}\middle| \mathbf{x},\mathcal{E}_{H,b}\right]\leq \frac{1}{n^{|S|/5}}.$$
\end{lemma}
\begin{proof}
Given $\mathbf{x}$, the conditional probability that any particular edge is present in $\mathbf{G}$ is at most
$$\left(1+\frac{\epsilon}{2}\right)\frac{d}{n}\leq \frac{2d}{n}.$$

A vertex can be unlucky either because it is incident to an $H$-cross-edge, or because it is adjacent to a vertex in $\big([n]\setminus V(H)\big)\setminus V_b(H,\mathbf{x})$.

Since $|\big([n]\setminus V(H)\big)\setminus V_b(H,\mathbf{x})|\leq 2\left\lceil n^{\frac{3}{4}}\right\rceil$, it follows from the union bound that, given $\mathbf{x}$ that is approximately balanced on $[n]\setminus V(H)$, the conditional probability that any particular vertex $v\in S$ is adjacent to some vertex in $\big([n]\setminus V(H)\big)\setminus V_b(H,\mathbf{x})$ is at most
$$2\left\lceil n^{\frac{3}{4}}\right\rceil\frac{2d}{n}\leq \frac{8 n^{\frac{3}{4}}d}{n}=\frac{8d}{n^{\frac{1}{4}}}.$$

Let $S'\subseteq S$ be the set of vertices in $S$ that are not adjacent to any vertex in $\big([n]\setminus V(H)\big)\setminus V_b(H,\mathbf{x})$. If $S$ is completely $(\mathbf{G},\mathbf{x},H)$-unlucky, then for every $v\in S'$, we have at least one $H$-cross-edge that is present in $\mathbf{G}$, and which is incident to it. Therefore, we have at least $\left\lceil\frac{|S'|}2\right\rceil$ $H$-cross-edges that are present in $\mathbf{G}$. Since we have at most $|V(H)|^2\leq s^2t^2$ $H$-cross-edges, the number of collections of $H$-cross-edges of size $\left\lceil\frac{|S'|}2\right\rceil$ is at most
$${s^2t^2\choose \left\lceil\frac{|S'|}2\right\rceil}\leq (s^2t^2)^{\left\lceil\frac{|S'|}2\right\rceil}.$$ 

Therefore, for $n$ large enough, given $\mathbf{x}$ that is approximately balanced on $[n]\setminus V(H)$, the conditional probability that $S$ is completely unlucky can be upper bounded by:
\begin{align*}
\mathbb{P}&\left[\big\{S\text{ is completely }(\mathbf{G},\mathbf{x},H)\text{-unlucky}\big\}\middle| \mathcal{E}_{H,b}\right]\\
&\quad\quad\quad\quad\leq \sum_{S'\subseteq S}\left(\frac{8d}{n^{\frac{1}{4}}}\right)^{|S|-|S'|}\cdot  (s^2t^2)^{\left\lceil\frac{|S'|}{2}\right\rceil}\left(\frac{2d}{n}\right)^{\left\lceil\frac{|S'|}{2}\right\rceil}\stackrel{(\ast)}{\leq} \sum_{S'\subseteq S}\left(\frac{8d}{n^{\frac{1}{4}}}\right)^{|S|-|S'|}\cdot  \left(\frac{2ds^2t^2}{n}\right)^{\frac{|S'|}{2}}\\
&\quad\quad\quad\quad= \sum_{S'\subseteq S}\left(\frac{8d}{n^{\frac{1}{4}}}\right)^{|S|-|S'|}\cdot  \left(\frac{\sqrt{2d}\cdot st}{n^{\frac{1}{2}}}\right)^{|S'|}\leq \sum_{S'\subseteq S}\left(\frac{8dst}{n^{\frac{1}{4}}}\right)^{|S|-|S'|}\cdot  \left(\frac{8d st}{n^{\frac{1}{4}}}\right)^{|S'|}= \sum_{S'\subseteq S}\left(\frac{8dst}{n^{\frac{1}{4}}}\right)^{|S|}\\
&\quad\quad\quad\quad= 2^{|S|}\left(\frac{8dst}{n^{\frac{1}{4}}}\right)^{|S|}= \left(\frac{16dst}{n^{\frac{1}{4}}}\right)^{|S|}\stackrel{(\dagger)}{\leq}\frac{1}{n^{|S|/5}},
\end{align*}
where $(\ast)$ and $(\dagger)$ are true for $n$ large enough.
\end{proof}


\begin{lemma}
\label{lem:lem_prob_completely_unsafe_nice}
Let $H$ be a multigraph with at most $st=sK\log n$ vertices and at most $st$ multi-edges. Assume that $\mathcal{L}_1(H)=\varnothing$. Let $V\subseteq \mathcal{S}_{\geq 2}(H)$, and for every $v\in V$, define
$$d_{\mathbf{G}-G(H),b}^o(v)=\big|\big\{u\in V_b(H,\mathbf{x}):\;uv\in\mathbf{G}\big\}\big|,$$
and
$$d_{\mathbf{G}-G(H),\overline{b}}^o(v)=\big|\big\{u\in\big([n]\setminus V(H)\big)\setminus V_b(H,\mathbf{x}):\;uv\in\mathbf{G}\big\}\big|.$$
Clearly\footnote{Recall \Cref{def:def_in_out_degree}}, $d_{\mathbf{G}-G(H)}^o(v)=d_{\mathbf{G}-G(H),b}^o(v)+d_{\mathbf{G}-G(H),\overline{b}}^o(v)$. Let $\mathcal{E}$ be an event such that:
\begin{itemize}
\item $\mathcal{E}$ implies that $\mathbf{x}$ is approximately balanced on $[n]\setminus V(H)$, i.e., $\mathcal{E}_{H,b}\subseteq \mathcal{E}$.
\item $\mathcal{E}$ implies that $V$ is completely $(\mathbf{G},H,\mathbf{x})$-lucky, i.e., $d_{\mathbf{G}-G(H)}^i(v)=d_{\mathbf{G}-G(H),\overline{b}}^o(v)=0$ for every $v\in V$.
\item Given $\mathbf{x}$, the event $\mathcal{E}$ is conditionally independent from $\big(d_{\mathbf{G}-G(H),b}^o(v)\big)_{v\in V}$.
\end{itemize}
Then
$$\mathbb{P}\big[\big\{V\text{ is completely }H\text{-unsafe in }\mathbf{G}\big\}\big|\mathbf{x},\mathcal{E}\big]\leq\left(\frac{\eta}{2}\right)^{|V|},$$
where $\eta$ is as in \Cref{lem:lem_bound_prob_large_outside_deg}.
\end{lemma}
\begin{proof}
Since $d_{\mathbf{G}-G(H)}^i(v)=d_{\mathbf{G}-G(H),\overline{b}}^o(v)=0$ for every $v\in V$, then $v$ is unsafe if and only if $d_{\mathbf{G}-G(H),b}^o(v)> \Delta - d_{G(H)}(v)$. Therefore, for every $V\subseteq V(H)$, we have
\begin{align*}
\mathbb{P}\big[\big\{V\text{ is completely unsafe}\big\}\big|\mathbf{x},\mathcal{E}\big]&= \mathbb{P}\big[\big\{d_{\mathbf{G}-G(H),b}^o(v)> \Delta - d_{G(H)}(v),\;\forall v\in V\big\}\big|\mathbf{x},\mathcal{E}\big]\\
&\stackrel{(\ast)}{=} \mathbb{P}\big[\big\{d_{\mathbf{G}-G(H),b}^o(v)> \Delta - d_{G(H)}(v),\;\forall v\in V\big\}\big|\mathbf{x}\big]\\
&= \prod_{v\in V} \mathbb{P}\big[d_{\mathbf{G}-G(H),b}^o(v)> \Delta - d_{G(H)}(v)\big|\mathbf{x}\big]\\
&\leq \prod_{v\in V} \mathbb{P}\big[d_{\mathbf{G}-G(H)}^o(v)> \Delta - d_{G(H)}(v)\big|\mathbf{x}\big],
\end{align*}
where $(\ast)$ follows from the fact that given $\mathbf{x}$, the event $\mathcal{E}$ is conditionally independent from $\big(d_{\mathbf{G}-G(H),b}^o(v)\big)_{v\in V(H)}$.

Now since $V\subset \mathcal{S}_{\geq 2}(H)$ and since $\mathcal{L}_1(H)=\varnothing$, we have $d_{G(H)}(v)=d_1^H(v)+d_{\geq 2}^H(v)\leq \frac{\Delta}{4}+\tau\leq\frac{\Delta}{2}$. Therefore,
\begin{equation}
\label{eq:eq_lem_upper_bound_nice_g_eq_4}
\begin{aligned}
\mathbb{P}\big[\big\{V\text{ is completely unsafe}\big\}\big|\mathbf{x},\mathcal{E}\big]&\leq  \prod_{v\in V} \mathbb{P}\big[d_{\mathbf{G}-G(H)}^o(v)> \Delta - d_{G(H)}(v)\big|\mathbf{x}\big]\\
&\leq  \prod_{v\in V} \mathbb{P}\big[d_{\mathbf{G}-G(H)}^o(v)> \frac{\Delta}{2}\big|\mathbf{x}\big]\stackrel{(\dagger)}{\leq}  \left(\frac{\eta}{2}\right)^{|V|},\\
\end{aligned}
\end{equation}
where $(\dagger)$ follows from \Cref{lem:lem_bound_prob_large_outside_deg}.
\end{proof}

\begin{lemma}
\label{lem:lem_upper_bound_nice_g}
Let $H$ be a multigraph with at most $st=sK\log n$ vertices and at most $st$ multi-edges. Assume that $\mathcal{L}_1(H)=\mathcal{L}_{\geq 2}(H)=\varnothing$, and that $E_{\geq 2}(H)$ forms a forest. If $A>\max\{100K,1\}$ and $n$ is large enough, then we have
$$\left|\mathbb{E}\big[\overline{\mathbf{Y}}_H\big|\mathcal{E}_{H,b\overline{g}}\big]\cdot\mathbb{P}[\mathcal{E}_{H,b\overline{g}}]\right|\leq \frac{1}{n^{\frac{1}{6}}}\cdot \left(\frac{\epsilon d}{2n}\right)^{|E_1(H)|}\cdot\left(\frac{d}{n}\right)^{|E_{\geq 2}(H)|},$$
where $\mathcal{E}_{H,b\overline{g}}$ is an in \cref{eq:eq_def_E_wb_H_xog}.
\end{lemma}
\begin{proof}
Recall \Cref{def:def_deg1_classification}, \Cref{def:def_deg2_classification} and \Cref{def:def_E1_E2_annoying}. Since $\mathcal{L}_1(H)=\mathcal{L}_{\geq 2}(H)=\varnothing$, we have $E_1^a(H)=\varnothing$, $E_{\geq2}^b(H)=\varnothing$, and $E_{\geq 2}(H)=E_{\geq2}^a(H)$. Furthermore, $\mathcal{S}_1(H)=V(H)$ and so $\mathcal{S}(H):=\mathcal{S}_1(H)\cap\mathcal{S}_{\geq 2}(H)=\mathcal{S}_{\geq 2}(H)$.

For every $V\subseteq V(H)$, define the events
$$\mathcal{E}_{V,H,c\text{-}l}=\Big\{V\text{ is completely }(\mathbf{G},\mathbf{x},H)\text{-lucky}\Big\},$$
$$\mathcal{E}_{V,H,c\text{-}ul}=\Big\{V\text{ is completely }(\mathbf{G},\mathbf{x},H)\text{-unlucky}\Big\}.$$

If $\mathcal{E}_{H,g}^c$ occurs then there exists at least one vertex $v\in V(H)$ that is unlucky. Hence, $\big\{\mathcal{E}_{U,H,c\text{-}l}\cap \mathcal{E}_{V(H)\setminus U,H,c\text{-}ul}:\; U\subsetneq V(H)\big\}$
forms a partition of the event $\mathcal{E}_{H,g}^c$. Therefore,
\begin{equation}
\label{eq:eq_lem_upper_bound_nice_g_eq_1}
\begin{aligned}
&\left|\mathbb{E}\big[\overline{\mathbf{Y}}_H\big|\mathbf{x},\mathcal{E}_{H,b\overline{g}}\big]\cdot\mathbb{P}[\mathcal{E}_{H,b\overline{g}}|\mathbf{x}]\right|\\
&=\resizebox{0.95\textwidth}{!}{$\displaystyle\left|\sum_{U\subsetneq V(H)}\mathbb{E}\big[\overline{\mathbf{Y}}_H\big|\mathbf{x},\mathcal{E}_{H,b}\cap\mathcal{E}_{U,H,c\text{-}l}\cap \mathcal{E}_{V(H)\setminus U,H,c\text{-}ul}\big]\cdot\mathbb{P}\left[\mathcal{E}_{H,b}\cap\mathcal{E}_{U,H,c\text{-}l}\cap \mathcal{E}_{V(H)\setminus U,H,c\text{-}ul}\middle|\mathbf{x}\right]\right|$}\\
&\leq\resizebox{0.95\textwidth}{!}{$\displaystyle\sum_{U\subsetneq V(H)}\left|\mathbb{E}\big[\overline{\mathbf{Y}}_H\big|\mathbf{x},\mathcal{E}_{H,b}\cap\mathcal{E}_{U,H,c\text{-}l}\cap \mathcal{E}_{V(H)\setminus U,H,c\text{-}ul}\big]\right|\cdot\mathbb{P}\left[\mathcal{E}_{H,b}\cap\mathcal{E}_{U,H,c\text{-}l}\cap \mathcal{E}_{V(H)\setminus U,H,c\text{-}ul}\middle|\mathbf{x}\right]$}\\
&\leq\sum_{U\subsetneq V(H)}\left|\mathbb{E}\big[\overline{\mathbf{Y}}_H\big|\mathbf{x},\mathcal{E}_{H,b}\cap\mathcal{E}_{U,H,c\text{-}l}\cap \mathcal{E}_{V(H)\setminus U,H,c\text{-}ul}\big]\right|\cdot\mathbb{P}\left[\mathcal{E}_{V(H)\setminus U,H,c\text{-}ul}\middle|\mathbf{x},\mathcal{E}_{H,b}\right]\\
&\leq\sum_{U\subsetneq V(H)}\left|\mathbb{E}\big[\overline{\mathbf{Y}}_H\big|\mathbf{x},\mathcal{E}_{H,b}\cap\mathcal{E}_{U,H,c\text{-}l}\cap \mathcal{E}_{V(H)\setminus U,H,c\text{-}ul}\big]\right|\cdot\frac{1}{n^{\frac{1}{5}|V(H)\setminus U|}},
\end{aligned}
\end{equation}
where the last inequality follows from \Cref{lem:lem_bound_prob_completely_unlucky}.

Now fix $U\subsetneq V(H)$. \Cref{lem:lem_prob_completely_unsafe_nice} implies that for every $V\subseteq U\cap\mathcal{S}(H)$, we have
$$\mathbb{P}\big[\big\{V\text{ is completely }H\text{-unsafe in }\mathbf{G}\big\}\big|\mathbf{x},\mathcal{E}_{H,b}\cap\mathcal{E}_{U,H,c\text{-}l}\cap \mathcal{E}_{V(H)\setminus U,H,c\text{-}ul}\big]\leq\left(\frac{\eta}{2}\right)^{|V|}\leq\eta^{|V|}.$$
It is now easy to see that the conditions of \Cref{lem:lem_UH_deg1_common} are satisfied for $U\cap\mathcal{S}(H)\subset \mathcal{S}(H)$ and $\mathcal{E}=\mathcal{E}_{H,b}\cap\mathcal{E}_{U,H,c\text{-}l}\cap \mathcal{E}_{V(H)\setminus U,H,c\text{-}ul}$. Therefore,
\begin{equation}
\label{eq:eq_lem_upper_bound_nice_g_eq_2}
\begin{aligned}
&\left|\mathbb{E}\big[\overline{\mathbf{Y}}_H\big|\mathbf{x},\mathcal{E}_{H,b}\cap\mathcal{E}_{U,H,c\text{-}l}\cap \mathcal{E}_{V(H)\setminus U,H,c\text{-}ul}\big]\right|= \big|\mathbb{E}\big[\overline{\mathbf{Y}}_H\big|\mathbf{x},\mathcal{E}\big]\big|\\
&\quad\quad\leq n^{\frac{K}{A}}\left(\frac{6}{\epsilon}\right)^{|E_1^a(H)|+|E_1^d(H)|+\tau(|\mathcal{S}(H)|-|U\cap \mathcal{S}(H)|)}\cdot\left(\frac{\epsilon d}{2n}\right)^{|E_1(H)|} \cdot\mathbb{E}\big[|\tilde{\mathbf{Y}}_{\geq 2}^H|\big|\mathbf{x},\mathcal{E}\big]\\
&\quad\quad\leq n^{\frac{2K}{A}}\left(\frac{6}{\epsilon}\right)^{\tau(|\mathcal{S}(H)|-|U\cap \mathcal{S}(H)|)}\cdot\left(\frac{\epsilon d}{2n}\right)^{|E_1(H)|} \cdot\mathbb{E}\big[|\tilde{\mathbf{Y}}_{\geq 2}^H|\big|\mathbf{x}\big],
\end{aligned}
\end{equation}
where the last equality follows from \Cref{lem:lem_UH_deg1_common_E_1d}, the fact that $E_1^a(H)=\varnothing$, and the fact that given $\mathbf{x}$, the event $\mathcal{E}$ is conditionally independent from $\tilde{\mathbf{Y}}_{\geq 2}^H$.

Notice that
\begin{equation}
\label{eq:eq_lem_upper_bound_nice_g_eq_3}
|\mathcal{S}(H)|-|U\cap \mathcal{S}(H)| = |\mathcal{S}(H)\setminus U|\leq  |V(H)\setminus U|= |V(H)|-|U|.
\end{equation}

Now by combining \cref{eq:eq_lem_upper_bound_nice_g_eq_1} and \cref{eq:eq_lem_upper_bound_nice_g_eq_2} and \cref{eq:eq_lem_upper_bound_nice_g_eq_3}, we get

\begin{align*}
&\left|\mathbb{E}\big[\overline{\mathbf{Y}}_H\big|\mathbf{x},\mathcal{E}_{H,b\overline{g}}\big]\cdot\mathbb{P}[\mathcal{E}_{H,b\overline{g}}|\mathbf{x}]\right|\\
&\quad\quad\quad\quad\leq\sum_{U\subsetneq V(H)} n^{\frac{2K}{A}}\left(\frac{6}{\epsilon}\right)^{\tau(|V(H)|-|U|)}\cdot\left(\frac{\epsilon d}{2n}\right)^{|E_1(H)|} \cdot\mathbb{E}\big[|\tilde{\mathbf{Y}}_{\geq 2}^H|\big|\mathbf{x}\big]\cdot\frac{1}{n^{\frac{1}{5}(|V(H)|-|U|)}}\\
&\quad\quad\quad\quad= n^{\frac{2K}{A}}\cdot\left(\frac{\epsilon d}{2n}\right)^{|E_1(H)|} \cdot\mathbb{E}\big[|\tilde{\mathbf{Y}}_{\geq 2}^H|\big|\mathbf{x}\big]\sum_{U\subsetneq V(H)}\left(\frac{6^\tau}{\epsilon^\tau \cdot n^{\frac{1}{5}}}\right)^{|V(H)|-|U|}\\
&\quad\quad\quad\quad= n^{\frac{2K}{A}}\cdot\left(\frac{\epsilon d}{2n}\right)^{|E_1(H)|} \cdot\mathbb{E}\big[|\tilde{\mathbf{Y}}_{\geq 2}^H|\big|\mathbf{x}\big]\cdot\left[\left(1+\frac{6^\tau}{\epsilon^\tau \cdot n^{\frac{1}{5}}}\right)^{|V(H)|}-1\right].
\end{align*}

Now since $|V(H)|\leq st=sK\log n$, we have
\begin{align*}
\left(1+\frac{6^\tau}{\epsilon^\tau \cdot n^{\frac{1}{5}}}\right)^{|V(H)|}&\leq \left(1+\frac{6^\tau}{\epsilon^\tau \cdot n^{\frac{1}{5}}}\right)^{st}\leq e^{\frac{6^\tau}{\epsilon^\tau \cdot n^{\frac{1}{5}}}st}=1+O\left(\frac{6^\tau \cdot st}{\epsilon^\tau \cdot n^{\frac{1}{5}}}\right).
\end{align*}

Thus,
\begin{align*}
&\left|\mathbb{E}\big[\overline{\mathbf{Y}}_H\big|\mathbf{x},\mathcal{E}_{H,b\overline{g}}\big]\cdot\mathbb{P}[\mathcal{E}_{H,b\overline{g}}|\mathbf{x}]\right|\leq O\left(\frac{6^\tau \cdot st\cdot n^{\frac{2K}{A}}}{\epsilon^\tau \cdot n^{\frac{1}{5}}}\right)\cdot \left(\frac{\epsilon d}{2n}\right)^{|E_1(H)|}\cdot\mathbb{E}\big[|\tilde{\mathbf{Y}}_{\geq 2}^H|\big|\mathbf{x}\big].
\end{align*}

By using \Cref{lem:lem_UH_deg2} and the fact that $E_{\geq2}^b(H)=\varnothing$, $E_{\geq2}^a(H)=E_{\geq2}(H)$, and $\mathcal{L}_{\geq2}(H)=\varnothing$, we get

$$\mathbb{E}\big[|\tilde{\mathbf{Y}}_{\geq 2}^H|\big|\mathbf{x}\big]\leq \left(\frac{d}{n}\right)^{|E_{\geq2}(H)|}\prod_{uv\in E_{\geq2}(H)}\left[1+\frac{\epsilon \mathbf{x}_u\mathbf{x}_v}{2}+\frac{d}{n}\right].$$

Therefore,
\begin{align*}
&\left|\mathbb{E}\big[\overline{\mathbf{Y}}_H\big|\mathbf{x},\mathcal{E}_{H,b\overline{g}}\big]\cdot\mathbb{P}[\mathcal{E}_{H,b\overline{g}}|\mathbf{x}]\right|\\
&\quad\quad\quad\quad\quad\quad\quad\leq O\left(\frac{6^\tau \cdot st\cdot n^{\frac{2K}{A}}}{\epsilon^\tau \cdot n^{\frac{1}{5}}}\right)\cdot \left(\frac{\epsilon d}{2n}\right)^{|E_1(H)|}\cdot\left(\frac{d}{n}\right)^{|E_{\geq2}(H)|}\prod_{uv\in E_{\geq2}(H)}\left[1+\frac{\epsilon \mathbf{x}_u\mathbf{x}_v}{2}+\frac{d}{n}\right].
\end{align*}

If $A>\max\{100K,1\}$, and $n$ is large enough, we get
\begin{align*}
\left|\mathbb{E}\big[\overline{\mathbf{Y}}_H\big|\mathbf{x},\mathcal{E}_{H,b\overline{g}}\big]\cdot\mathbb{P}[\mathcal{E}_{H,b\overline{g}}|\mathbf{x}]\right|\leq \frac{1}{2n^{\frac{1}{6}}}\cdot \left(\frac{\epsilon d}{2n}\right)^{|E_1(H)|}\left(\frac{d}{n}\right)^{|E_{\geq2}(H)|}\prod_{uv\in E_{\geq2}(H)}\left[1+\frac{\epsilon \mathbf{x}_u\mathbf{x}_v}{2}+\frac{d}{n}\right].
\end{align*}

We conclude that
\begin{align*}
\left|\mathbb{E}\big[\overline{\mathbf{Y}}_H\big|\mathcal{E}_{H,b\overline{g}}\big]\cdot\mathbb{P}[\mathcal{E}_{H,b\overline{g}}]\right|&\leq \frac{1}{2n^{\frac{1}{6}}}\cdot \left(\frac{\epsilon d}{2n}\right)^{|E_1(H)|}\left(\frac{d}{n}\right)^{|E_{\geq2}(H)|}\cdot\mathbb{E}\left[\prod_{uv\in E_{\geq2}(H)}\left[1+\frac{\epsilon \mathbf{x}_u\mathbf{x}_v}{2}+\frac{d}{n}\right]\right]\\
&\stackrel{(\ddagger)}{=} \frac{1}{2n^{\frac{1}{6}}}\cdot \left(\frac{\epsilon d}{2n}\right)^{|E_1(H)|}\left(\frac{d}{n}\right)^{|E_{\geq2}(H)|}\cdot \left[\prod_{uv\in E_{\geq 2}(H)}\left(1+\frac{d}{n}\right)\right]\\
&= \frac{1}{2n^{\frac{1}{6}}} \cdot\left(\frac{\epsilon d}{2n}\right)^{|E_1(H)|}\cdot\left(\frac{d}{n}\right)^{|E_{\geq 2}(H)|}\cdot\left(1+O\left(\frac{dst}{n}\right)\right)\\
&\stackrel{(\wr)}{\leq} \frac{1}{n^{\frac{1}{6}}} \cdot \left(\frac{\epsilon d}{2n}\right)^{|E_1(H)|}\cdot\left(\frac{d}{n}\right)^{|E_{\geq 2}(H)|},
\end{align*}
where $(\ddagger)$ follows from \Cref{lem:lem_Expectation_Y_Nice} and the fact that $E_{\geq 2}(H)$ forms a forest, and $(\wr)$ is true for $n$ large enough.
\end{proof}

Now we will prove an upper bound on $\left|\mathbb{E}\big[\overline{\mathbf{Y}}_H\big|\mathcal{E}_{H,bg\overline{d}}\big]\cdot\mathbb{P}[\mathcal{E}_{H,bg\overline{d}}]\right|$, where $\mathcal{E}_{H,bg\overline{d}}$ is as in \cref{eq:eq_def_E_wb_H_xgod}.

\begin{lemma}
\label{lem:lem_upper_bound_nice_d}
Let $H$ be a multigraph with at most $st=sK\log n$ vertices and at most $st$ multi-edges. Assume that $\mathcal{L}_1(H)=\mathcal{L}_{\geq 2}(H)=\varnothing$, and that $E_{\geq 2}(H)$ forms a forest. If $A>\max\{100K,1\}$ and $n$ is large enough, then we have
$$\left|\mathbb{E}\big[\overline{\mathbf{Y}}_H\big|\mathcal{E}_{H,bg\overline{d}}\big]\cdot\mathbb{P}[\mathcal{E}_{H,bg\overline{d}}]\right|\leq \frac{1}{n^{\frac{1}{2}}}\cdot \left(\frac{\epsilon d}{2n}\right)^{|E_1(H)|}\cdot\left(\frac{d}{n}\right)^{|E_{\geq 2}(H)|},$$
where $\mathcal{E}_{H,bg\overline{d}}$ is an in \cref{eq:eq_def_E_wb_H_xgod}.
\end{lemma}
\begin{proof}
Recall \Cref{def:def_deg1_classification}, \Cref{def:def_deg2_classification} and \Cref{def:def_E1_E2_annoying}. It is easy to see that if we take $U=\mathcal{S}(H):=\mathcal{S}_1(H)\cap\mathcal{S}_{\geq 2}(H)=V(H)$ and $\mathcal{E}=\mathcal{E}_{H,bg\overline{d}}$, then the conditions of \Cref{lem:lem_UH_deg1_common} are satisfied. Therefore,
\begin{align*}
&\big|\mathbb{E}\big[\overline{\mathbf{Y}}_H\big|\mathbf{x},\mathcal{E}_{H,bg\overline{d}}\big]\big|\\
&\quad\quad\quad\quad\quad\leq n^{\frac{K}{A}}\left(\frac{6}{\epsilon}\right)^{|E_1^a(H)|+|E_1^d(H)|+\tau(|\mathcal{S}(H)|-|U|)}\cdot\left(\frac{\epsilon d}{2n}\right)^{|E_1(H)|} \cdot\mathbb{E}\big[|\tilde{\mathbf{Y}}_{\geq 2}^H|\big|\mathbf{x},\mathcal{E}_{H,bg\overline{d}}\big]\\
&\quad\quad\quad\quad\quad\leq n^{\frac{2K}{A}}\left(\frac{\epsilon d}{2n}\right)^{|E_1(H)|} \cdot\mathbb{E}\big[|\tilde{\mathbf{Y}}_{\geq 2}^H|\big|\mathbf{x},\mathcal{E}_{H,bg\overline{d}}\big],
\end{align*}
where the last inequality follows from \Cref{lem:lem_UH_deg1_common_E_1d} and the fact that $E_1^a(H)=\varnothing$ and $U=\mathcal{S}(H)$.

Therefore,
\begin{equation}
\label{eq:eq_lem_upper_bound_nice_d_eq_1}
\begin{aligned}
&\left|\mathbb{E}\big[\overline{\mathbf{Y}}_H\big|\mathbf{x},\mathcal{E}_{H,bg\overline{d}}\big]\cdot\mathbb{P}[\mathcal{E}_{H,bg\overline{d}}|\mathbf{x}]\right|\\
&\quad\quad\quad\quad\quad\quad\leq  n^{\frac{2K}{A}}\left(\frac{\epsilon d}{2n}\right)^{|E_1(H)|} \cdot\mathbb{E}\big[|\tilde{\mathbf{Y}}_{\geq 2}^H|\big|\mathbf{x},\mathcal{E}_{H,bg\overline{d}}\big]\cdot\mathbb{P}[\mathcal{E}_{H,bg\overline{d}}|\mathbf{x}]\\
&\quad\quad\quad\quad\quad\quad=  n^{\frac{2K}{A}}\left(\frac{\epsilon d}{2n}\right)^{|E_1(H)|} \cdot\mathbb{E}\big[|\mathbf{Y}_{\geq 2}^H|\big|\mathbf{x},\mathcal{E}_{H,bg\overline{d}}\big]\cdot\mathbb{P}[\mathcal{E}_{H,bg\overline{d}}|\mathbf{x}],
\end{aligned}
\end{equation}
where the last equality follows\footnote{We could have used $|\tilde{\mathbf{Y}}_{\geq 2}^H|\leq |\mathbf{Y}_{\geq 2}^H|$, which is true in general.} from the fact that $\mathcal{L}_{\geq 2}(H)=\varnothing$, which means that $\tilde{\mathbf{Y}}_{\geq 2}^H=\mathbf{Y}_{\geq 2}^H$, where
$$\mathbf{Y}_{\geq 2}^H=\prod_{uv\in E_{\geq 2}(H)}\mathbf{Y}_{uv}^{m_H(uv)}.$$

Now for every $E\subseteq E_{\geq 2}(H)$, define the events:
$$\mathcal{E}_{E,c\text{-}d}=\big\{e\in\mathbf{G},\;\forall e\in E\big\}\quad\text{and}\quad\mathcal{E}_{E,c\text{-}\overline{d}}=\big\{e\notin\mathbf{G},\;\forall e\in E\big\}.$$

Let $\mathcal{E}_{H,bg}$ be as in \cref{eq:eq_def_E_wb_H_xg}. Since $\big\{\mathcal{E}_{E,c\text{-}d}\cap \mathcal{E}_{E_{\geq 2}(H)\setminus E,c\text{-}\overline{d}}:\; E\subsetneq E_{\geq 2}(H)\big\}$ is a partition of $\mathcal{E}_{H,d}^c$ and since $\mathcal{E}_{H,bg\overline{d}}=\mathcal{E}_{H,bg}\cap \mathcal{E}_{H,d}^c$, we can write:

\begin{align*}
&\mathbb{E}\big[|\mathbf{Y}_{\geq 2}^H|\big|\mathbf{x},\mathcal{E}_{H,bg\overline{d}}\big]\cdot\mathbb{P}[\mathcal{E}_{H,bg\overline{d}}|\mathbf{x}]\\
&=\mathbb{E}\big[|\mathbf{Y}_{\geq 2}^H|\big|\mathbf{x},\mathcal{E}_{H,bg}\cap \mathcal{E}_{H,d}^c\big]\cdot\mathbb{P}[\mathcal{E}_{H,bg}\cap \mathcal{E}_{H,d}^c|\mathbf{x}]\\
&=\resizebox{0.95\textwidth}{!}{$\displaystyle\sum_{E\subsetneq E_{\geq 2}(H)} \mathbb{E}\big[|\mathbf{Y}_{\geq 2}^H|\big|\mathbf{x},\mathcal{E}_{H,bg}\cap \mathcal{E}_{E,c\text{-}d}\cap \mathcal{E}_{E_{\geq 2}(H)\setminus E,c\text{-}\overline{d}}\big]\cdot\mathbb{P}\big[\mathcal{E}_{H,bg}\cap \mathcal{E}_{E,c\text{-}d}\cap \mathcal{E}_{E_{\geq 2}(H)\setminus E,c\text{-}\overline{d}}\big|\mathbf{x}\big]$}\\
&\leq\sum_{E\subsetneq E_{\geq 2}(H)} \mathbb{E}\big[|\mathbf{Y}_{\geq 2}^H|\big|\mathbf{x},\mathcal{E}_{H,bg}\cap \mathcal{E}_{E,c\text{-}d}\cap \mathcal{E}_{E_{\geq 2}(H)\setminus E,c\text{-}\overline{d}}\big]\cdot\mathbb{P}[\mathcal{E}_{E,c\text{-}d}\cap \mathcal{E}_{E_{\geq 2}(H)\setminus E,c\text{-}\overline{d}}|\mathbf{x}]\\
&=\resizebox{0.95\textwidth}{!}{$\displaystyle\sum_{E\subsetneq E_{\geq 2}(H)}\prod_{uv\in E}\left[ \left(1-\frac{d}{n}\right)^{m_H(uv)}\cdot\left(1+\frac{\epsilon \mathbf{x}_u\mathbf{x}_v}{2}\right)\frac{d}{n}\right] \cdot\prod_{uv\in E_{\geq 2}(H)\setminus E}\left[ \left(1-\left(1+\frac{\epsilon \mathbf{x}_u\mathbf{x}_v}{2}\right)\frac{d}{n}\right)\cdot \left(\frac{d}{n}\right)^{m_H(uv)}\right]$} \\
&\leq\sum_{E\subsetneq E_{\geq 2}(H)}\prod_{uv\in E}\left[\left(1+\frac{\epsilon \mathbf{x}_u\mathbf{x}_v}{2}\right)\frac{d}{n}\right]\cdot\prod_{uv\in E_{\geq 2}(H)\setminus E} \left(\frac{d}{n}\right)^2 \\
&=\prod_{uv\in E_{\geq 2}(H)}\left[\left(1+\frac{\epsilon \mathbf{x}_u\mathbf{x}_v}{2}\right)\frac{d}{n}+\frac{d}{n^2} \right] - \prod_{uv\in E_{\geq 2}(H)}\left[\left(1+\frac{\epsilon \mathbf{x}_u\mathbf{x}_v}{2}\right)\frac{d}{n}\right].
\end{align*}

By combining this with \cref{eq:eq_lem_upper_bound_nice_d_eq_1}, we get

\begin{align*}
&\left|\mathbb{E}\big[\overline{\mathbf{Y}}_H\big|\mathbf{x},\mathcal{E}_{H,bg\overline{d}}\big]\cdot\mathbb{P}[\mathcal{E}_{H,bg\overline{d}}|\mathbf{x}]\right|\\
&\quad\quad\quad\quad\leq n^{\frac{2K}{A}}\left(\frac{\epsilon d}{2n}\right)^{|E_1(H)|} \cdot\left[\prod_{uv\in E_{\geq 2}(H)}\left[\left(1+\frac{\epsilon \mathbf{x}_u\mathbf{x}_v}{2}\right)\frac{d}{n}+\frac{d}{n^2} \right] - \prod_{uv\in E_{\geq 2}(H)}\left[\left(1+\frac{\epsilon \mathbf{x}_u\mathbf{x}_v}{2}\right)\frac{d}{n}\right]\right].
\end{align*}

Therefore,
\begin{align*}
&\left|\mathbb{E}\big[\overline{\mathbf{Y}}_H\big|\mathcal{E}_{H,bg\overline{d}}\big]\cdot\mathbb{P}[\mathcal{E}_{H,bg\overline{d}}]\right|\\
&\quad\quad\quad\quad\leq n^{\frac{2K}{A}}\left(\frac{\epsilon d}{2n}\right)^{|E_1(H)|} \cdot \mathbb{E}\left[\prod_{uv\in E_{\geq 2}(H)}\left[\left(1+\frac{\epsilon \mathbf{x}_u\mathbf{x}_v}{2}\right)\frac{d}{n}+\frac{d}{n^2} \right] - \prod_{uv\in E_{\geq 2}(H)}\left[\left(1+\frac{\epsilon \mathbf{x}_u\mathbf{x}_v}{2}\right)\frac{d}{n}\right]\right]\\
&\quad\quad\quad\quad\stackrel{(\ast)}{=} n^{\frac{2K}{A}}\left(\frac{\epsilon d}{2n}\right)^{|E_1(H)|} \cdot \left[\prod_{uv\in E_{\geq 2}(H)}\left(\frac{d}{n}+\frac{d}{n^2} \right) - \prod_{uv\in E_{\geq 2}(H)}\left(\frac{d}{n}\right)\right]\\
&\quad\quad\quad\quad\leq n^{\frac{2K}{A}}\left(\frac{\epsilon d}{2n}\right)^{|E_1(H)|} \cdot \left[\left(\frac{d}{n}\right)^{|E_{\geq 2}(H)|}\left(1+O\left(\frac{dst}{n} \right)\right) - \left(\frac{d}{n}\right)^{|E_{\geq 2}(H)|}\right]\\
&\quad\quad\quad\quad = O\left(\frac{dst\cdot n^{\frac{2K}{A}}}{n} \right)\left(\frac{\epsilon d}{2n}\right)^{|E_1(H)|}  \left(\frac{d}{n}\right)^{|E_{\geq 2}(H)|}\stackrel{(\dagger)}{\leq} \frac{1}{\sqrt{n}}\left(\frac{\epsilon d}{2n}\right)^{|E_1(H)|}  \left(\frac{d}{n}\right)^{|E_{\geq 2}(H)|},
\end{align*}
where $(\ast)$ follows from \Cref{lem:lem_Expectation_Y_Nice} and the fact thatr $E_{\geq 2}(H)$ is a forest, and $(\dagger)$ is true when $A>\max\{100K,1\}$ and $n$ is large enough.
\end{proof}

Now we are ready to prove \Cref{lem:lem_upper_bound_nice_Wc}

\begin{proof}[Proof of \Cref{lem:lem_upper_bound_nice_Wc}]
Since $\big\{\mathcal{E}_{H,b}^c,\mathcal{E}_{H,b\overline{g}},\mathcal{E}_{H,bg\overline{d}}\big\}$ is a partition of $\mathcal{E}_{wb,H}^c$, we get:
\begin{align*}
\mathbb{E}\big[\overline{\mathbf{Y}}_H\big|\mathcal{E}_{wb,H}^c\big]\cdot\mathbb{P}[\mathcal{E}_{wb,H}^c]=\mathbb{E}\big[\overline{\mathbf{Y}}_H\big|\mathcal{E}_{H,b}^c\big]\cdot\mathbb{P}[\mathcal{E}_{H,b}^c]
&+ \mathbb{E}\big[\overline{\mathbf{Y}}_H\big|\mathcal{E}_{H,b\overline{g}}\big]\cdot\mathbb{P}[\mathcal{E}_{H,b\overline{g}}]\\
&+ \mathbb{E}\big[\overline{\mathbf{Y}}_H\big|\mathcal{E}_{H,bg\overline{d}}\big]\cdot\mathbb{P}[\mathcal{E}_{H,bg\overline{d}}].
\end{align*}

It follows from \Cref{lem:lem_upper_bound_nice_y}, \Cref{lem:lem_upper_bound_nice_g} and \Cref{lem:lem_upper_bound_nice_d} that:
\begin{align*}
\left|\mathbb{E}\big[\overline{\mathbf{Y}}_H\big|\mathcal{E}_{wb,H}^c\big]\cdot\mathbb{P}[\mathcal{E}_{wb,H}^c]\right|&\leq\left|\mathbb{E}\big[\overline{\mathbf{Y}}_H\big|\mathcal{E}_{H,b}^c\big]\cdot\mathbb{P}[\mathcal{E}_{H,b}^c] \right|+ \left|\mathbb{E}\big[\overline{\mathbf{Y}}_H\big|\mathcal{E}_{H,b\overline{g}}\big]\cdot\mathbb{P}[\mathcal{E}_{H,b\overline{g}}]\right|\\
&\quad\quad\quad\quad\quad\quad\quad\quad\quad\quad\quad\quad\quad+\left| \mathbb{E}\big[\overline{\mathbf{Y}}_H\big|\mathcal{E}_{H,bg\overline{d}}\big]\cdot\mathbb{P}[\mathcal{E}_{H,bg\overline{d}}]\right|\\
&\leq \left(e^{-\sqrt{n}}+\frac{1}{n^{\frac{1}{6}}}+\frac{1}{\sqrt{n}}\right)\cdot \left(\frac{\epsilon d}{2n}\right)^{|E_1(H)|}\cdot\left(\frac{d}{n}\right)^{|E_{\geq 2}(H)|}\\
&\leq \frac{2}{n^{\frac{1}{6}}}\cdot\left(\frac{\epsilon d}{2n}\right)^{|E_1(H)|}\cdot\left(\frac{d}{n}\right)^{|E_{\geq 2}(H)|},
\end{align*}
where the last inequality is true for $n$ large enough.
\end{proof}

\subsubsection{Upper bound on the negligible part of the contribution of the well-behaved event}
\label{subsubsec:subsubsec_upper_bounding_well_behaved_negligible}

In order to prove \Cref{lem:lem_upper_bound_nice_W_2}, we need the following lemma:

\begin{lemma}
\label{lem:lem_Prob_Well_Behaved_Y_2}
If $P_{2,H}^{wb}(\mathbf{x})$ is as in \cref{eq:eq_Prob_Well_Behaved_Y_2}, then for $n$ large enough, we have
$$|P_{2,H}^{wb}(\mathbf{x})|\leq\frac{1}{n^{\frac{1}{5}}}\cdot\prod_{uv\in E_{\geq 2}(H)}\left[\left(1+\frac{\epsilon \mathbf{x}_u\mathbf{x}_v}{2}\right)\frac{d}{n}\right].$$
\end{lemma}
\begin{proof}
If $\mathbf{x}$ is not approximately balanced on $[n]\setminus V(H)$, then 
\begin{align*}
|P_{2,H}^{wb}(\mathbf{x})|&=0\leq\frac{1}{n^{\frac{1}{5}}}\cdot\prod_{uv\in E_{\geq 2}(H)}\left[\left(1+\frac{\epsilon \mathbf{x}_u\mathbf{x}_v}{2}\right)\frac{d}{n}\right].
\end{align*}

Now assume that $\mathbf{x}$ is approximately balanced on $[n]\setminus V(H)$, and let $E_{\overline{b}c}(H,\mathbf{x})$ be as in \cref{eq:eq_def_E_obc_H_H}. We have:
\begin{align*}
|P_{2,H}^{wb}(\mathbf{x})|&=\left|\sum_{\substack{S\subseteq E_{\overline{b}c}(H,\mathbf{x}):\\S\neq\varnothing}}\prod_{uv\in S}\left[-\left(1+\frac{\epsilon \mathbf{x}_u\mathbf{x}_v}{2}\right)\frac{d}{n}\right]\right|\cdot\left[\prod_{uv\in E_{\geq 2}(H)}\left[\left(1+\frac{\epsilon \mathbf{x}_u\mathbf{x}_v}{2}\right)\frac{d}{n}\right]\right]\\
&\leq\sum_{\substack{S\subseteq E_{\overline{b}c}(H,\mathbf{x}):\\S\neq\varnothing}}\left(\frac{2d}{n}\right)^{|S|}\cdot\prod_{uv\in E_{\geq 2}(H)}\left[\left(1+\frac{\epsilon \mathbf{x}_u\mathbf{x}_v}{2}\right)\frac{d}{n}\right].
\end{align*}

On the other hand, for $n$ large enough, we have:
$$|E_{\overline{b}c}(H,\mathbf{x})|=|E_c(H)|+|V(H)|\cdot\big|\big([n]\setminus V(H)\big)\setminus V_b(H,\mathbf{x})\big|\leq s^2t^2 + 2st\cdot  n^{3/4}\leq 3s^2t^2 \cdot n^{3/4}.$$

Hence,
\begin{align*}
\sum_{\substack{S\subseteq E_{\overline{b}c}(H,\mathbf{x}):\\S\neq\varnothing}} \left(\frac{2d}{n}\right)^{|S|}
&=\left(1+\frac{2d}{n}\right)^{|E_{\overline{b}c}(H,\mathbf{x})|}-1=O\left(\frac{2d\cdot |E_{\overline{b}c}(H,\mathbf{x})|}{n}\right)\leq O\left(\frac{6ds^2t^2\cdot n^{\frac{3}{4}}}{n}\right)\leq \frac{1}{n^{\frac{1}{5}}},
\end{align*}
where the last inequality is true for $n$ large enough. Therefore, we have
\begin{align*}
|P_{2,H}^{wb}(\mathbf{x})|&\leq \frac{1}{n^{\frac{1}{5}}}\cdot\prod_{uv\in E_{\geq 2}(H)}\left[\left(1+\frac{\epsilon \mathbf{x}_u\mathbf{x}_v}{2}\right)\frac{d}{n}\right].
\end{align*}
\end{proof}

Now we are ready to prove \Cref{lem:lem_upper_bound_nice_W_2}.

\begin{proof}[Proof of \Cref{lem:lem_upper_bound_nice_W_2}]
 \Cref{lem:lem_prob_completely_unsafe_nice} implies that for every $V\subseteq \mathcal{S}(H)$, we have
$$\mathbb{P}\big[\big\{V\text{ is completely }H\text{-unsafe in }\mathbf{G}\big\}\big|\mathbf{x},\mathcal{E}_{wb,H}\big]\leq\left(\frac{\eta}{2}\right)^{|V|}\leq\eta^{|V|}.$$

It is now easy to see that if we take $\mathcal{E}=\mathcal{E}_{wb,H}$ and $U=\mathcal{S}(H)$, then the conditions of \Cref{lem:lem_UH_deg1_common} are satisfied. Therefore,
\begin{align*}
\left|\mathbb{E}\big[\overline{\mathbf{Y}}_H\big|\mathbf{x},\mathcal{E}_{wb,H}\big]\right|&\leq n^{\frac{K}{A}}\left(\frac{6}{\epsilon}\right)^{|E_1^a(H)|+|E_1^d(H)|+\tau(|\mathcal{S}(H)|-|U|)}\cdot\left(\frac{\epsilon d}{2n}\right)^{|E_1(H)|} \cdot\mathbb{E}\big[|\tilde{\mathbf{Y}}_{\geq 2}^H|\big|\mathbf{x},\mathcal{E}_{wb,H}\big]\\
&\stackrel{(\ast)}{\leq} n^{\frac{2K}{A}}\left(\frac{\epsilon d}{2n}\right)^{|E_1(H)|} \cdot\mathbb{E}\big[|\tilde{\mathbf{Y}}_{\geq 2}^H|\big|\mathbf{x},\mathcal{E}_{wb,H}\big]\stackrel{(\dagger)}{\leq} n^{\frac{K}{A}}\left(\frac{\epsilon d}{2n}\right)^{|E_1(H)|},
\end{align*}
where $(\ast)$ follows from \Cref{lem:lem_UH_deg1_common_E_1d} and the fact that $E_1^a(H)=\varnothing$ and $U=\mathcal{S}(H)$, and $(\dagger)$ follows from the fact that $|\tilde{\mathbf{Y}}_{\geq 2}^H|\leq 1$. Combining this with \Cref{lem:lem_Prob_Well_Behaved_Y_2}, we get:
\begin{align*}
\left|\mathbb{E}\big[\overline{\mathbf{Y}}_H\big|\mathbf{x},\mathcal{E}_{wb,H}\big]\cdot P_{2,H}^{wb}(\mathbf{x})\right|&\leq \frac{n^{\frac{2K}{A}}}{n^{\frac{1}{5}}} \left(\frac{\epsilon d}{2n}\right)^{|E_1(H)|}\cdot \prod_{uv\in E_{\geq 2}(H)}\left[\left(1+\frac{\epsilon \mathbf{x}_u\mathbf{x}_v}{2}\right)\frac{d}{n}\right]\\
&\leq \frac{1}{n^{\frac{1}{6}}} \left(\frac{\epsilon d}{2n}\right)^{|E_1(H)|}\cdot \prod_{uv\in E_{\geq 2}(H)}\left[\left(1+\frac{\epsilon \mathbf{x}_u\mathbf{x}_v}{2}\right)\frac{d}{n}\right],
\end{align*}
where the last inequality is true if $A>\max\{100K,1\}$ and $n$ is large enough. Therefore,

\begin{align*}
\mathbb{E}\left[\left|\mathbb{E}\big[\overline{\mathbf{Y}}_H\big|\mathbf{x},\mathcal{E}_{wb,H}\big]\cdot P_{2,H}^{wb}(\mathbf{x})\right|\right]&\leq \frac{1}{n^{\frac{1}{6}}} \left(\frac{\epsilon d}{2n}\right)^{|E_1(H)|}\cdot \mathbb{E}\left[\prod_{uv\in E_{\geq 2}(H)}\left[\left(1+\frac{\epsilon \mathbf{x}_u\mathbf{x}_v}{2}\right)\frac{d}{n}\right]\right]\\
&\stackrel{(\ddagger)}{=}\resizebox{0.6\textwidth}{!}{$\displaystyle \frac{1}{n^{\frac{1}{6}}} \left(\frac{\epsilon d}{2n}\right)^{|E_1(H)|}\prod_{uv\in E_{\geq 2}(H)}\left(\frac{d}{n}\right)=\frac{1}{n^{\frac{1}{6}}}  \left(\frac{\epsilon d}{2n}\right)^{|E_1(H)|} \left(\frac{d}{n}\right)^{|E_{\geq 2}(H)|}$},
\end{align*}
where $(\ddagger)$ follows from \Cref{lem:lem_Expectation_Y_Nice} and the fact that $E_{\geq 2}(H)$ is a forest.
\end{proof}

\subsubsection{Tight bounds on the significant part of the contribution of the well-behaved event in the case of pleasant multigraphs}
\label{subsubsec:subsubsec_bounding_well_behaved_significant}

In order to prove \Cref{lem:lem_main_lemma_pleasant_wb}, we need a few definitions and lemmas.

\begin{definition}
For every $v\in V(H)$, let $\mathbf{D}^H_v$ be the number of edges in $\mathbf{G}$ between $v$ and $V_b(H,\mathbf{x})$. We denote $(\mathbf{D}^H_v)_{v\in V(H)}$ as $\mathbf{D}^H$.

Note that if $\mathcal{E}_{wb,H}$ occurs, then\footnote{Recall \Cref{def:def_in_out_degree}}
$$d_{\mathbf{G}-G(H)}(v)=d_{\mathbf{G}-G(H)}^o(v)=\mathbf{D}^H_v.$$
Furthermore, $|V_b^+(H,\mathbf{x})|=\left\lceil\frac{n}{2}-n^{\frac{3}{4}}\right\rceil$ and $|V_b^-(H,\mathbf{x})|=\left\lceil\frac{n}{2}-n^{\frac{3}{4}}\right\rceil$, where
$$V_b^+(H,\mathbf{x})=\big\{u\in V_b(H,\mathbf{x}):\; \mathbf{x}_u=+1\big\},$$
and
$$V_b^-(H,\mathbf{x})=\big\{u\in V_b(H,\mathbf{x}):\; \mathbf{x}_u=-1\big\}.$$
\end{definition}

It is easy to see that given $\mathbf{x}$ that is approximately balanced on $[n]\setminus V(H)$, the random variables $\{\mathbf{D}^H_v\}_{v\in V(H)}$ are conditionally mutually independent, and they are conditionally independent of $(\mathcal{E}_{H,g}, \mathcal{E}_{H,d})$.

\begin{lemma}
\label{lem:lem_prob_nice_d}
Let $H$ be a multigraph with at most $st=sK\log n$ vertices. For every $v\in V(H)$, and every $\mathsf{d}=(\mathsf{d}_v)_{v\in V(H)}\in \mathbb{N}^{V(H)}$, we have:
\begin{align*}
\resizebox{0.95\textwidth}{!}{$\displaystyle\mathbb{P}\big[\mathbf{D}^H=\mathsf{d}\big|\mathbf{x},\mathcal{E}_{wb,H}\big]=\mathbb{P}\big[\mathbf{D}^H=\mathsf{d}\big|\mathcal{E}_{wb,H}\big]=\mathbb{P}\big[\mathbf{D}^H=\mathsf{d}\big|\mathbf{x},\mathcal{E}_{H,b}\big]=\mathbb{P}\big[\mathbf{D}^H=\mathsf{d}\big|\mathcal{E}_{H,b}\big]=P_{wb}(\mathsf{d}),$}
\end{align*}
where
\begin{align*}
P_{wb}(\mathsf{d})&=\prod_{v\in V(H)} P_{wb}(\mathsf{d}_v)
\end{align*}
and
\begin{align*}
P_{wb}(\mathsf{d}_v)&=\sum_{a=0}^{\mathsf{d}_v}\left[{\big\lceil\frac{n}{2}-n^{\frac{3}{4}}\big\rceil\choose a}\cdot \left[\left(1+\frac{\epsilon}{2}\right)\frac{d}{n}\right]^a\cdot \left[1-\left(1+\frac{\epsilon}{2}\right)\frac{d}{n}\right]^{\big\lceil\frac{n}{2}-n^{\frac{3}{4}}\big\rceil-a}\right]\\
&\quad\quad\quad\quad\quad\quad\quad\times \left[{\big\lceil\frac{n}{2}-n^{\frac{3}{4}}\big\rceil\choose \mathsf{d}_v - a}\cdot \left[\left(1-\frac{\epsilon}{2}\right)\frac{d}{n}\right]^{\mathsf{d}_v-a}\cdot \left[1-\left(1-\frac{\epsilon}{2}\right)\frac{d}{n}\right]^{\big\lceil\frac{n}{2}-n^{\frac{3}{4}}\big\rceil-\mathsf{d}_v+a}\right].
\end{align*}
Furthermore, for every $v\in V(H)$, we have
$$\mathbb{P}\big[\mathbf{D}^H_v=\mathsf{d}_v\big|\mathbf{x},\mathcal{E}_{H,b}\big]=\mathbb{P}\big[\mathbf{D}^H_v=\mathsf{d}_v\big|\mathcal{E}_{H,b}\big]=P_{wb}(\mathsf{d}_v).$$

In other words, $\mathbb{P}\big[\mathbf{D}^H=\mathsf{d}\big|\mathbf{x},\mathcal{E}_{wb,H}\big]$ depends only on $\mathsf{d}=(\mathsf{d}_v)_{v\in V(H)}$, i.e., it does not depend on $\mathbf{x}$, and it depends on $\mathbf{G}$ only through $\mathbf{D}^H=(\mathbf{D}^H_v)_{v\in V(H)}$. Furthermore, given $\mathcal{E}_{wb,H}$, the random variables $(\mathbf{D}^H_v)_{v\in V(H)}$ are conditionally mutually independent. The same is true if we condition on $\mathcal{E}_{H,b}$ instead of $\mathcal{E}_{wb,H}$.

Note that if $\mathsf{d}_v>2\left\lceil\frac{n}{2}-n^{\frac{3}{4}}\right\rceil$, then $P_{wb}(\mathsf{d}_v)=0$.
\end{lemma}
\begin{proof}
We have
\begin{align*}
\mathbb{P}\big[\mathbf{D}^H=\mathsf{d}\big|\mathbf{x},\mathcal{E}_{wb,H}\big]&=\mathbb{P}\big[\mathbf{D}^H=\mathsf{d}\big|\mathbf{x},\mathcal{E}_{H,b}\cap \mathcal{E}_{H,g}\cap \mathcal{E}_{H,d}\big]\\
&\stackrel{(*)}{=}\mathbb{P}\big[\mathbf{D}^H=\mathsf{d}\big|\mathbf{x},\mathcal{E}_{H,b}\big]
=\prod_{v\in V(H)}\mathbb{P}\big[\mathbf{D}^H_v=\mathsf{d}_v\big|\mathbf{x},\mathcal{E}_{H,b}\big],
\end{align*}
where $(\ast)$ follows from the fact that given $\mathbf{x}$ that is approximately balanced on $[n]\setminus V(H)$, the random variable $\mathbf{D}^H$ is conditionally independent from $( \mathcal{E}_{H,g}, \mathcal{E}_{H,d})$.

Now for every $v\in V(H)$, we have
\begin{align*}
&\mathbb{P}\big[\mathbf{D}^H_v=\mathsf{d}_v\big|\mathbf{x},\mathcal{E}_{H,b}\big]\\
&=\sum_{\substack{U\subseteq V_b(H,\mathbf{x}):\\|U|=\mathsf{d}_v}}\left[\prod_{u\in U}\mathbb{P}[uv\in \mathbf{G}|\mathbf{x},\mathcal{E}_{H,b}]\right]\cdot \left[\prod_{u\in V_b(H,\mathbf{x})\setminus U}\mathbb{P}[uv\notin \mathbf{G}|\mathbf{x},\mathcal{E}_{H,b}]\right]\\
&\quad\quad=\sum_{a=0}^{\mathsf{d}_v}\left(\sum_{\substack{U^+\subseteq V_b^+(H,\mathbf{x}):\\|U^+|=a}}\left[\prod_{u\in U^+}\mathbb{P}[uv\in \mathbf{G}|\mathbf{x},\mathcal{E}_{H,b}]\right]\cdot \left[\prod_{u\in V_b^+(H,\mathbf{x})\setminus U^+}\mathbb{P}[uv\notin \mathbf{G}|\mathbf{x},\mathcal{E}_{H,b}]\right]\right)\\
&\quad\quad\quad\quad\times \left(\sum_{\substack{U^-\subseteq V_b^-(H,\mathbf{x}):\\|U^-|=\mathsf{d}_v-a}}\left[\prod_{u\in U^-}\mathbb{P}[uv\in \mathbf{G}|\mathbf{x},\mathcal{E}_{H,b}]\right]\cdot \left[\prod_{u\in V_b^-(H,\mathbf{x})\setminus U^-}\mathbb{P}[uv\notin \mathbf{G}|\mathbf{x},\mathcal{E}_{H,b}]\right]\right),
\end{align*}
hence,
\begin{align*}
\mathbb{P}\big[\mathbf{D}^H_v=\mathsf{d}_v\big|\mathbf{x},\mathcal{E}_{H,b}\big]&=\sum_{a=0}^{\mathsf{d}_v}\left[{\big\lceil\frac{n}{2}-n^{\frac{3}{4}}\big\rceil\choose a}\left[\left(1+\frac{\epsilon \mathbf{x}_v}{2}\right)\frac{d}{n}\right]^a\cdot \left[1-\left(1+\frac{\epsilon \mathbf{x}_v}{2}\right)\frac{d}{n}\right]^{\big\lceil\frac{n}{2}-n^{\frac{3}{4}}\big\rceil-a}\right]\\
&\quad\quad\times \left[{\big\lceil\frac{n}{2}-n^{\frac{3}{4}}\big\rceil\choose \mathsf{d}_v - a}\left[\left(1-\frac{\epsilon \mathbf{x}_v}{2}\right)\frac{d}{n}\right]^{\mathsf{d}_v-a}\cdot \left[1-\left(1-\frac{\epsilon \mathbf{x}_v}{2}\right)\frac{d}{n}\right]^{\big\lceil\frac{n}{2}-n^{\frac{3}{4}}\big\rceil-\mathsf{d}_v+a}\right]\\
&=\sum_{a=0}^{\mathsf{d}_v}\left[{\big\lceil\frac{n}{2}-n^{\frac{3}{4}}\big\rceil\choose a}\left[\left(1+\frac{\epsilon}{2}\right)\frac{d}{n}\right]^a\cdot \left[1-\left(1+\frac{\epsilon}{2}\right)\frac{d}{n}\right]^{\big\lceil\frac{n}{2}-n^{\frac{3}{4}}\big\rceil-a}\right]\\
&\quad\quad\quad\quad\times \left[{\big\lceil\frac{n}{2}-n^{\frac{3}{4}}\big\rceil\choose \mathsf{d}_v - a}\left[\left(1-\frac{\epsilon}{2}\right)\frac{d}{n}\right]^{\mathsf{d}_v-a}\cdot \left[1-\left(1-\frac{\epsilon}{2}\right)\frac{d}{n}\right]^{\big\lceil\frac{n}{2}-n^{\frac{3}{4}}\big\rceil-\mathsf{d}_v+a}\right]\\
&=P_{wb}(\mathsf{d}_v)=\mathbb{P}\big[\mathbf{D}^H_v=\mathsf{d}_v\big|\mathcal{E}_{H,b}\big],
\end{align*}
where the last equality follows from the fact that $\mathbb{P}\big[\mathbf{D}^H_v=\mathsf{d}_v\big|\mathbf{x},\mathcal{E}_{H,b}\big]$ does not depend on $\mathbf{x}$. We conclude that

$$\mathbb{P}\big[\mathbf{D}^H=\mathsf{d}\big|\mathbf{x},\mathcal{E}_{wb,H}\big]=\prod_{v\in V(H)} P_{wb}(\mathsf{d}_v)=P_{wb}(\mathsf{d}).$$

Furthermore, since $\mathbb{P}\big[\mathbf{D}^H=\mathsf{d}\big|\mathbf{x},\mathcal{E}_{wb,H}\big]=\mathbb{P}\big[\mathbf{D}^H=\mathsf{d}\big|\mathbf{x},\mathcal{E}_{H,b}\big]$ does not depend on $\mathbf{x}$, we have
$$\mathbb{P}\big[\mathbf{D}^H=\mathsf{d}\big|\mathbf{x},\mathcal{E}_{wb,H}\big]=\mathbb{P}\big[\mathbf{D}^H=\mathsf{d}\big|\mathcal{E}_{wb,H}\big]=\mathbb{P}\big[\mathbf{D}^H=\mathsf{d}\big|\mathbf{x},\mathcal{E}_{H,b}\big]=\mathbb{P}\big[\mathbf{D}^H=\mathsf{d}\big|\mathcal{E}_{H,b}\big].$$
\end{proof}

\begin{lemma}
\label{lem:lem_nice_L_definition}
Let $H$ be a multigraph with at most $st=sK\log n$ vertices. We have 

\begin{equation*}
\begin{aligned}
\mathbb{E}\big[\overline{\mathbf{Y}}_H\big|\mathbf{x},\mathcal{E}_{wb,H}\big]=\sum_{\mathsf{d}\in\mathbb{N}^{V(H)}}\mathbb{E}\big[\overline{\mathbf{Y}}_H\big|\mathbf{x},\mathcal{E}_{wb,H},\mathbf{D}^H=\mathsf{d}\big]\cdot P_{wb}(\mathsf{d}),
\end{aligned}
\end{equation*}
where $P_{wb}(\mathsf{d})$ is as in \Cref{lem:lem_prob_nice_d}.
\end{lemma}
\begin{proof}
We have:

\begin{align*}
\mathbb{E}\big[\overline{\mathbf{Y}}_H\big|\mathbf{x},\mathcal{E}_{wb,H}\big]&=\sum_{\mathsf{d}\in\mathbb{N}^{V(H)}}\mathbb{E}\big[\overline{\mathbf{Y}}_H\big|\mathbf{x},\mathcal{E}_{wb,H},\mathbf{D}^H=\mathsf{d}\big]\cdot\mathbb{P}[\mathbf{D}^H=\mathsf{d}|\mathbf{x},\mathcal{E}_{wb,H}]\\
&=\sum_{\mathsf{d}\in\mathbb{N}^{V(H)}}\mathbb{E}\big[\overline{\mathbf{Y}}_H\big|\mathbf{x},\mathcal{E}_{wb,H},\mathbf{D}^H=\mathsf{d}\big]\cdot P_{wb}(\mathsf{d}).
\end{align*}
\end{proof}

In the following, we will fix $\mathsf{d}\in\mathbb{N}^{V(H)}$ and focus on studying $\mathbb{E}\big[\overline{\mathbf{Y}}_H\big|\mathbf{x},\mathcal{E}_{wb,H},\mathbf{D}^H=\mathsf{d}\big]$.

\begin{definition}
Let $H$ be an arbitrary multigraph and let $\mathsf{d}\in\mathbb{N}^{V(H)}$. For every $v\in V(E_1(H))$, we define the $\mathsf{d}$-\emph{criticality} of $v$ in $H$ as $c_{\mathsf{d}}^H(v)=\Delta-d^H_{\geq 2}(v)-\mathsf{d}_v$. If $d^H_1(v)\leq c_{\mathsf{d}}^H(v)$, we say that $v$ is $\mathsf{d}$-\emph{safe in $H$}. If $d^H_1(v)> c_{\mathsf{d}}^H(v)$, we say that $v$ is $\mathsf{d}$-\emph{unsafe in $H$}.

An edge $uv\in E_1(H)$ is said to be $\mathsf{d}$-\emph{safe} if both $u$ and $v$ are $\mathsf{d}$-safe. We denote the set of $\mathsf{d}$-safe edges in $E_1(H)$ as $S_{\mathsf{d}}^H$.
\end{definition}

\begin{remark}
\label{rem:rem_safe_critical_vertices}
If $\mathbf{D}=\mathsf{d}$ and $\mathcal{E}_{wb,H}$ occurs, then since all the edges of multiplicity at least 2 are present in $\mathbf{G}$, we can see that a vertex $v\in V(E_1(H))$ causes truncation if and only if
$$\big|\big\{e\in E_1(H)\cap\mathbf{G}:\;e\text{ is incident to }v\big\}\big|> c_{\mathsf{d}}^H(v).$$

As can be easily seen, the criticality of a vertex $v$ is the maximum number of edges from $\big\{e\in E_1(H):\; e\text{ is incident to }v\big\}$ which can be present in $\mathbf{G}$ without causing truncation.

If $v$ is $\mathsf{d}$-safe, then we are sure that $v$ does not cause truncation. If $uv$ is a $\mathsf{d}$-safe vertex, then we are sure that its presence in $\mathbf{G}$ will not cause truncation.
\end{remark}

\begin{definition}
Let $H$ be an arbitrary multigraph and let $\mathsf{d}\in\mathbb{N}^{V(H)}$. For every $\mathsf{E}\subseteq E_1(H)$ and every $\mathsf{d}\in\mathbb{N}^{V(H)}$, we say that $\mathsf{E}$ is \emph{$\mathsf{d}$-structurally safe} if for every $v\in V(E_1(H))$ we have $d_{\mathsf{E}}(v)\leq c_{\mathsf{d}}^H(v)$. For every $\mathsf{E}\subseteq E_1(H)$, define:
$$\mathcal{S}_{\mathsf{d}}^H(\mathsf{E})=\big\{\mathsf{S}\subseteq \mathsf{E}:\; \mathsf{S}\text{ is }\mathsf{d}\text{-structurally safe}\big\}.$$
\end{definition}

\begin{remark}
\label{rem:rem_S_C_divide_on_Agreeables}
Let $H$ be an $(s,t)$-pleasant multigraph and let $H^{(1)},\ldots,H^{(r_H)}$ be the agreeable components of $H$. Since $H^{(1)},\ldots,H^{(r_H)}$ are vertex-disjoint. We can deduce from this that
$$\mathcal{S}_{\mathsf{d}}^H(E_1(H))=\left\{\mathsf{S}_1\cup\ldots\cup\mathsf{S}_{r_H}:\; \forall i\in[r_H], \;\mathsf{S}_{i}\in \mathcal{S}_{\mathsf{d}}^H\big(E_1\big(H^{(i)}\big)\big)\right\},$$
and
$$\left\{\mathsf{E}_{\geq 2}:\mathsf{E}_{\geq 2}\subseteq E_{\geq 2}\big(H^{(\ast)}\big)\right\}=\left\{\mathsf{E}^{(1)}_{\geq 2}\cup\ldots\cup \mathsf{E}^{(r_H)}_{\geq 2}:\;\forall i\in[r_H], \;\mathsf{E}^{(i)}_{\geq 2}\subseteq E_{\geq 2}\big(H^{(i)}\big)\right\},$$
where $H^{(\ast)}=\displaystyle\bigcup_{i\in[r_H]} H^{(i)}$.
\end{remark}

\begin{lemma}
\label{lem:lem_lemma_pleasant_decompose_on_agreeables}
Let $H$ be an $(s,t)$-pleasant multigraph and let $H^{(1)},\ldots,H^{(r_H)}$ be its agreeable components\footnote{See \Cref{def:def_pleasant_multigraphs}}. For every $\mathsf{d}\in\mathbb{N}^{V(H)}$, we have
\begin{align*}
&\mathbb{E}\left[\mathbb{E}\big[\overline{\mathbf{Y}}_H\big|\mathbf{x},\mathcal{E}_{wb,H},\mathbf{D}^H=\mathsf{d}\big]\cdot P_{1,H}^{wb}(\mathbf{x})\right]\\
&\quad\quad\quad\quad\quad\quad\quad\quad=\left(\frac{\tilde{d}}{n}\right)^{|E_1(H)|}\left(1-\frac{d}{n}\right)^{m_H(E_{\geq 2}(H))}\left(\frac{d}{n}\right)^{|E_{\geq 2}(H)|}\cdot\mathbb{P}\left[\mathcal{E}_{H,b}\right]\cdot\prod_{i\in[r_H]}J_{H,\mathsf{d},i},
\end{align*}
where
\begin{equation}
\label{eq:eq_def_J_H_d_i}
\begin{aligned}
J_{H,\mathsf{d},i}&=\sum_{\mathsf{S}_i\in\mathcal{S}_{\mathsf{d}}^H(E_1(H^{(i)}))} \sum_{\mathsf{S}_i'\subseteq \mathsf{S}_i}\sum_{\mathsf{S}_i''\subseteq E_1(H^{(i)})\setminus \mathsf{S}_i}\sum_{\mathsf{E}_{\geq 2}^{(i)}\subseteq E_{\geq 2}(H^{(i)})}\left(\frac{\epsilon}{2}\right)^{|\mathsf{S}_i'|+|\mathsf{E}_{\geq 2}^{(i)}|}\cdot \left(\frac{\tilde{\epsilon}}{n}\right)^{|\mathsf{S}_i''|}\\
&\quad\quad\quad\quad\quad\quad\quad\quad\quad\quad\quad\quad\quad\quad\times
(-1)^{|E_1(H^{(i)})|-|\mathsf{S}_i|-|\mathsf{S}_i''|}\cdot\mathbb{E}\left[\prod_{uv\in \mathsf{S}'_i\cup \mathsf{S}''_i\cup \mathsf{E}^{(i)}_{\geq 2}} \mathbf{x}_u\mathbf{x}_v\right],
\end{aligned}
\end{equation}
$$\tilde{d}=d\left(1-\frac{d}{n}\right),\quad\quad\text{and}\quad\quad\tilde{\epsilon}=\frac{\epsilon d^2}{2\tilde{d}}.$$
\end{lemma}
\begin{proof}
Let $\mathbf{x}$ be such that $\mathcal{E}_{H,b}$ occurs\footnote{Note that $\mathcal{E}_{H,b}$ depends only on $(\mathbf{x}_u)_{u\in [n]\setminus V(H)}$, i.e., it is $\sigma\big(\big\{\mathbf{x}_u:\;u\in [n]\setminus V(H)\big\}\big)$-measurable.}. From \Cref{rem:rem_safe_critical_vertices} we can see that if $\mathcal{E}_{wb,H}$ occurs and $\mathbf{D}^H=\mathsf{d}$, then $H$ is not truncated if and only if $E(\mathbf{G})\cap E_1(H)\in\mathcal{S}_{\mathsf{d}}^H(E_1(H))$. Now using the fact that given $\mathbf{x}$, the random variables $(\mathbbm{1}_{uv\in\mathbf{G}})_{uv\in E_1(H)}$ are conditionally independent from the events $\mathcal{E}_{wb,H}$ and $\{\mathbf{D}^H=\mathsf{d}\}$, and the fact that given $\mathcal{E}_{wb,H}$, we have $\mathbf{Y}_{uv}=1-\frac{d}{n}$ for every $uv\in E_{\geq 2}(H)$, we deduce that:
\begin{equation}
\label{eq:eq_D_E_wb_DH_D_1}
\begin{aligned}
&\mathbb{E}\big[\overline{\mathbf{Y}}_H\big|\mathbf{x},\mathcal{E}_{wb,H},\mathbf{D}^H=\mathsf{d}\big]\\
&=\resizebox{0.9\textwidth}{!}{$\displaystyle\sum_{\mathsf{S}\in\mathcal{S}_{\mathsf{d}}^H(E_1(H))}\prod_{uv\in \mathsf{S}}\left[\left(1-\frac{d}{n}\right)\cdot\mathbb{P}[uv\in\mathbf{G}|\mathbf{x}]\right]\cdot\prod_{uv\in E_1(H)\setminus \mathsf{S}} \left[\left(-\frac{d}{n}\right)\cdot\mathbb{P}[uv\notin\mathbf{G}|\mathbf{x}]\right]\cdot\prod_{uv\in E_{\geq 2}(H)} \left(1-\frac{d}{n}\right)^{m_H(uv)}$}.
\end{aligned}
\end{equation}

Now for every $uv\in E_1(H)$, we have
\begin{equation}
\label{eq:eq_D_E_wb_DH_D_uv_in}
\begin{aligned}
\left(1-\frac{d}{n}\right)\cdot\mathbb{P}[uv\in\mathbf{G}|\mathbf{x}]&=\left(1-\frac{d}{n}\right)\cdot\left(1+\frac{\epsilon \mathbf{x}_u\mathbf{x}_v}{2}\right)\frac{d}{n}=\left(1+\frac{\epsilon \mathbf{x}_u\mathbf{x}_v}{2}\right)\frac{\tilde{d}}{n}.
\end{aligned}
\end{equation}
On the other hand,
\begin{equation}
\label{eq:eq_D_E_wb_DH_D_uv_out}
\begin{aligned}
\left(-\frac{d}{n}\right)\cdot\mathbb{P}[e\notin\mathbf{G}|\mathbf{x}]&=\left(-\frac{d}{n}\right)\cdot\left[1-\left(1+\frac{\epsilon \mathbf{x}_u\mathbf{x}_v}{2}\right)\frac{d}{n}\right]=-\frac{d}{n}\left(1-\frac{d}{n} -\frac{\epsilon \mathbf{x}_u\mathbf{x}_v d}{2n}\right)\\
&=-\frac{\tilde{d}}{n}+\frac{\epsilon \mathbf{x}_u\mathbf{x}_v d^2}{2n^2}=-\frac{\tilde{d}}{n}\left(1-\frac{\epsilon \mathbf{x}_u\mathbf{x}_v d^2}{2\tilde{d} n}\right)=-\frac{\tilde{d}}{n}\left(1-\frac{\tilde{\epsilon} \mathbf{x}_u\mathbf{x}_v }{n}\right).
\end{aligned}
\end{equation}

By combining \cref{eq:eq_Prob_Well_Behaved_Y_1}, \cref{eq:eq_D_E_wb_DH_D_1}, \cref{eq:eq_D_E_wb_DH_D_uv_in} and \cref{eq:eq_D_E_wb_DH_D_uv_in}, we get
\begin{align*}
\mathbb{E}&\big[\overline{\mathbf{Y}}_H\big|\mathbf{x},\mathcal{E}_{wb,H},\mathbf{D}^H=\mathsf{d}\big]\cdot P_{1,H}^{wb}(\mathbf{x})\\
&=\sum_{\mathsf{S}\in\mathcal{S}_{\mathsf{d}}^H(E_1(H))}\prod_{uv\in \mathsf{S}}\left[\left(1+\frac{\epsilon \mathbf{x}_u\mathbf{x}_v}{2}\right)\frac{\tilde{d}}{n}\right]\cdot\prod_{uv\in E_1(H)\setminus \mathsf{S}} \left[-\frac{\tilde{d}}{n}\left(1-\frac{\tilde{\epsilon} \mathbf{x}_u\mathbf{x}_v }{n}\right)\right]\\
&\quad\quad\quad\quad\quad\quad\quad\quad\quad\quad\quad\quad\times\prod_{uv\in E_{\geq 2}(H)} \left[\left(1-\frac{d}{n}\right)^{m_H(uv)}\cdot\left(1+\frac{\epsilon \mathbf{x}_u\mathbf{x}_v}{2}\right)\frac{d}{n}\right] \cdot \mathbbm{1}_{\mathcal{E}_{H,b}}(\mathbf{x}).
\end{align*}

Now for every $\mathsf{E}\subseteq E(H)$, define $\displaystyle m_H(\mathsf{E})=\sum_{uv\in \mathsf{E}}m_H(uv)$. We have:

\begin{align*}
\mathbb{E}&\big[\overline{\mathbf{Y}}_H\big|\mathbf{x},\mathcal{E}_{wb,H},\mathbf{D}^H=\mathsf{d}\big]\cdot P_{1,H}^{wb}(\mathbf{x})\\
&=\left(\frac{\tilde{d}}{n}\right)^{|E_1(H)|}\left(1-\frac{d}{n}\right)^{m_H(E_{\geq 2}(H))}\left(\frac{d}{n}\right)^{|E_{\geq 2}(H)|}\\
&\quad\quad\times\sum_{\mathsf{S}\in\mathcal{S}_{\mathsf{d}}^H(E_1(H))}\prod_{uv\in \mathsf{S}}\left(1+\frac{\epsilon \mathbf{x}_u\mathbf{x}_v}{2}\right)\cdot\prod_{uv\in E_1(H)\setminus \mathsf{S}} \left(-1+\frac{\tilde{\epsilon} \mathbf{x}_u\mathbf{x}_v }{n}\right)\cdot \prod_{uv\in E_{\geq 2}(H)} \left(1+\frac{\epsilon \mathbf{x}_u\mathbf{x}_v}{2}\right) \cdot \mathbbm{1}_{\mathcal{E}_{H,b}}(\mathbf{x})\\
&=\left(\frac{\tilde{d}}{n}\right)^{|E_1(H)|}\left(1-\frac{d}{n}\right)^{m_H(E_{\geq 2}(H))}\left(\frac{d}{n}\right)^{|E_{\geq 2}(H)|}\\
&\quad\quad\quad\quad\quad\times\sum_{\mathsf{S}\in\mathcal{S}_{\mathsf{d}}^H(E_1(H))}
\left[\sum_{\mathsf{S}'\subseteq \mathsf{S}}\prod_{uv\in \mathsf{S}'}\left(\frac{\epsilon \mathbf{x}_u\mathbf{x}_v}{2}\right)\right]
\cdot
\left[\sum_{\mathsf{S}''\subseteq E_1(H)\setminus \mathsf{S}} (-1)^{|E_1(H)|-|\mathsf{S}|-|\mathsf{S}''|}\prod_{uv\in \mathsf{S}''} \left(\frac{\tilde{\epsilon} \mathbf{x}_u\mathbf{x}_v }{n}\right)\right]\\
&\quad\quad\quad\quad\quad\quad\quad\quad\quad\quad\quad\quad\quad\quad\quad\quad\quad\quad\quad\quad\times
\left[\sum_{\mathsf{E}_{\geq 2}\subseteq E_{\geq 2}(H)} \prod_{uv\in \mathsf{E}_{\geq 2}} \left(\frac{\epsilon \mathbf{x}_u\mathbf{x}_v}{2}\right)\right]\cdot \mathbbm{1}_{\mathcal{E}_{H,b}}(\mathbf{x})\\
&=\left(\frac{\tilde{d}}{n}\right)^{|E_1(H)|}\left(1-\frac{d}{n}\right)^{m_H(E_{\geq 2}(H))}\left(\frac{d}{n}\right)^{|E_{\geq 2}(H)|}\\
&\quad\quad\quad\quad\times
\sum_{\mathsf{S}\in\mathcal{S}_{\mathsf{d}}^H(E_1(H))}\sum_{\mathsf{S}'\subseteq \mathsf{S}}\sum_{\mathsf{S}''\subseteq E_1(H)\setminus \mathsf{S}}\sum_{\mathsf{E}_{\geq 2}\subseteq E_{\geq 2}(H)} \left(\frac{\epsilon}{2}\right)^{|\mathsf{S}'|+|\mathsf{E}_{\geq 2}|}\cdot \left(\frac{\tilde{\epsilon}}{n}\right)^{|\mathsf{S}''|}\cdot (-1)^{|E_1(H)|-|\mathsf{S}|-|\mathsf{S}''|}\\
&\quad\quad\quad\quad\quad\quad\quad\quad\quad\quad\quad\quad\quad\quad\quad\quad \quad\quad\quad\quad\quad\quad\quad\quad
\times  \mathbbm{1}_{\mathcal{E}_{H,b}}(\mathbf{x})\cdot \prod_{uv\in \mathsf{S}'\cup \mathsf{S}''\cup \mathsf{E}_{\geq 2}} \mathbf{x}_u\mathbf{x}_v.
\end{align*}

Now since $\mathbbm{1}_{\mathcal{E}_{H,b}}(\mathbf{x})$ depends only on $(\mathbf{x}_u)_{u\in[n]\setminus V(H)}$, it is independent from $(\mathbf{x}_v)_{v\in V(H)}$. Therefore, for every $\mathsf{S}\in\mathcal{S}_{\mathsf{d}}^H(E_1(H))$, $\mathsf{S}'\subseteq \mathsf{S}$, $\mathsf{S}''\subseteq E_1(H)\setminus \mathsf{S}$ and $\mathsf{E}_{\geq 2}\subseteq E_{\geq 2}(H)$, we have:
\begin{align*}
\mathbb{E}\left[\mathbbm{1}_{\mathcal{E}_{H,b}}(\mathbf{x})\cdot \prod_{uv\in \mathsf{S}'\cup \mathsf{S}''\cup \mathsf{E}_{\geq 2}} \mathbf{x}_u\mathbf{x}_v\right]&=\mathbb{E}\left[\mathbbm{1}_{\mathcal{E}_{H,b}}(\mathbf{x})\right]\cdot \mathbb{E}\left[\prod_{uv\in \mathsf{S}'\cup \mathsf{S}''\cup \mathsf{E}_{\geq 2}} \mathbf{x}_u\mathbf{x}_v\right]\\
&=\mathbb{P}\left[\mathcal{E}_{H,b}\right]\cdot \mathbb{E}\left[\prod_{uv\in \mathsf{S}'\cup \mathsf{S}''\cup \mathsf{E}_{\geq 2}} \mathbf{x}_u\mathbf{x}_v\right].
\end{align*}

Notice that if $\mathsf{E}_{\geq 2}$ contains an edge outside $H^{(\ast)}$, then $\mathsf{S}'\cup \mathsf{S}''\cup \mathsf{E}_{\geq 2}$ cannot be a union of edge-disjoint cycles\footnote{Recall that $H^{(\ast)}=\displaystyle\bigcup_{i\in[r_H]} H^{(i)}$.} (see \Cref{def:def_pleasant_multigraphs}). \Cref{lem:lem_Expectation_Y_Nice} now implies that if $\mathsf{E}_{\geq 2}\not{\subseteq}E_{\geq2}(H^{(\ast)})$, then
$$\mathbb{E}\left[\prod_{uv\in \mathsf{S}'\cup \mathsf{S}''\cup \mathsf{E}_{\geq 2}} \mathbf{x}_u\mathbf{x}_v\right] =0.$$

Therefore,

\begin{equation}
\label{eq:eq_nice_expectation_d_P_1H}
\begin{aligned}
&\mathbb{E}\left[\mathbb{E}\big[\overline{\mathbf{Y}}_H\big|\mathbf{x},\mathcal{E}_{wb,H},\mathbf{D}^H=\mathsf{d}\big]\cdot P_{1,H}^{wb}(\mathbf{x})\right]\\
&=\left(\frac{\tilde{d}}{n}\right)^{|E_1(H)|}\left(1-\frac{d}{n}\right)^{m_H(E_{\geq 2}(H))}\left(\frac{d}{n}\right)^{|E_{\geq 2}(H)|}\cdot\mathbb{P}\left[\mathcal{E}_{H,b}\right]\\
&\quad\quad\quad\quad\times
\sum_{\mathsf{S}\in\mathcal{S}_{\mathsf{d}}^H(E_1(H))}\sum_{\mathsf{S}'\subseteq \mathsf{S}}\sum_{\mathsf{S}''\subseteq E_1(H)\setminus \mathsf{S}}\sum_{\mathsf{E}_{\geq 2}\subseteq E_{\geq 2}(H^{(\ast)})} \left(\frac{\epsilon}{2}\right)^{|\mathsf{S}'|+|\mathsf{E}_{\geq 2}|}\cdot \left(\frac{\tilde{\epsilon}}{n}\right)^{|\mathsf{S}''|}\cdot (-1)^{|E_1(H)|-|\mathsf{S}|-|\mathsf{S}''|}\\
&\quad\quad\quad\quad\quad\quad\quad\quad\quad\quad\quad\quad\quad\quad\quad\quad \quad\quad\quad\quad\quad\quad\quad\quad
\times  \mathbb{E}\left[\prod_{uv\in \mathsf{S}'\cup \mathsf{S}''\cup \mathsf{E}_{\geq 2}} \mathbf{x}_u\mathbf{x}_v\right].
\end{aligned}
\end{equation}

Now since $H$ is $(s,t)$-pleasant, the agreeable components $H^{(1)},\ldots,H^{(r_H)}$ are vertex disjoint. Combining this with \Cref{rem:rem_S_C_divide_on_Agreeables}, it is easy to see that we can rewrite \cref{eq:eq_nice_expectation_d_P_1H} as follows:

\begin{equation*}
\begin{aligned}
&\mathbb{E}\left[\mathbb{E}\big[\overline{\mathbf{Y}}_H\big|\mathbf{x},\mathcal{E}_{wb,H},\mathbf{D}^H=\mathsf{d}\big]\cdot P_{1,H}^{wb}(\mathbf{x})\right]\\
&=\left(\frac{\tilde{d}}{n}\right)^{|E_1(H)|}\left(1-\frac{d}{n}\right)^{m_H(E_{\geq 2}(H))}\left(\frac{d}{n}\right)^{|E_{\geq 2}(H)|}\cdot\mathbb{P}\left[\mathcal{E}_{H,b}\right]\\
&\quad\quad\quad\quad\times
\prod_{i\in[r_H]}\left(\sum_{\mathsf{S}_i\in\mathcal{S}_{\mathsf{d}}^H(E_1(H^{(i)}))} \sum_{\mathsf{S}_i'\subseteq \mathsf{S}_i}\sum_{\mathsf{S}_i''\subseteq E_1(H^{(i)})\setminus \mathsf{S}_i}\sum_{\mathsf{E}_{\geq 2}^{(i)}\subseteq E_{\geq 2}(H^{(i)})}\left(\frac{\epsilon}{2}\right)^{|\mathsf{S}_i'|+|\mathsf{E}_{\geq 2}^{(i)}|}\cdot \left(\frac{\tilde{\epsilon}}{n}\right)^{|\mathsf{S}_i''|}\right.\\
&\quad\quad\quad\quad\quad\quad\quad\quad\quad\quad\quad\quad\quad\quad\quad\quad\times
(-1)^{|E_1(H^{(i)})|-|\mathsf{S}_i|-|\mathsf{S}_i''|}\cdot\left.\mathbb{E}\left[\prod_{uv\in \mathsf{S}'_i\cup \mathsf{S}''_i\cup \mathsf{E}^{(i)}_{\geq 2}} \mathbf{x}_u\mathbf{x}_v\right]\right).
\end{aligned}
\end{equation*}
Therefore,
\begin{equation*}
\begin{aligned}
&\mathbb{E}\left[\mathbb{E}\big[\overline{\mathbf{Y}}_H\big|\mathbf{x},\mathcal{E}_{wb,H},\mathbf{D}^H=\mathsf{d}\big]\cdot P_{1,H}^{wb}(\mathbf{x})\right]\\
&\quad\quad\quad\quad\quad\quad\quad=\left(\frac{\tilde{d}}{n}\right)^{|E_1(H)|}\left(1-\frac{d}{n}\right)^{m_H(E_{\geq 2}(H))}\left(\frac{d}{n}\right)^{|E_{\geq 2}(H)|}\cdot\mathbb{P}\left[\mathcal{E}_{H,b}\right]\cdot\prod_{i\in[r_H]}J_{H,\mathsf{d},i}.
\end{aligned}
\end{equation*}
\end{proof}

We now turn to study the value of $J_{H,\mathsf{d},i}$. The following lemma will be useful.

\begin{lemma}
\label{lem:lem_sum_structurelly_safe_edge_safe_exists}
Let $\mathsf{E}\subseteq E_1(H)$ and let $\{\mathsf{E}',\mathsf{E}''\}$ be a partition of $\mathsf{E}$. Let $$f:\{\mathsf{S}'':\mathsf{S}''\subseteq\mathsf{E}''\}\rightarrow\mathbb{R}$$ be an arbitrary function on the subsets of $\mathsf{E}''$. If $\mathsf{E}'$ contains a $\mathsf{d}$-safe edge $uv$ then
\begin{align*}
\sum_{\mathsf{S}\in\mathcal{S}_{\mathsf{d}}^H(\mathsf{E})} (-1)^{|\mathsf{E}'|-|\mathsf{S}\cap \mathsf{E}'|}\cdot f(\mathsf{S}\cap \mathsf{E}'') =0.
\end{align*}
\end{lemma}
\begin{proof}
Since $uv\in\mathsf{E}$ is $\mathsf{d}$-safe, we can write
$$\mathcal{S}_{\mathsf{d}}^H(\mathsf{E})=\mathcal{S}_{\mathsf{d}}^H(\mathsf{E}\setminus\{uv\})\cup\big\{\mathsf{S}\cup\{uv\}:\;\mathsf{S}\in \mathcal{S}_{\mathsf{d}}^H(\mathsf{E}\setminus\{uv\})\big\}.$$
Therefore,
\begin{align*}
\sum_{\mathsf{S}\in\mathcal{S}_{\mathsf{d}}^H(\mathsf{E})}& (-1)^{|\mathsf{E}'|-|\mathsf{S}\cap \mathsf{E}'|}\cdot f(\mathsf{S}\cap \mathsf{E}'')\\
 &= \sum_{\mathsf{S}\in\mathcal{S}_{\mathsf{d}}^H(\mathsf{E}\setminus\{uv\})} \left[(-1)^{|\mathsf{E}'|-|\mathsf{S}\cap \mathsf{E}'|}\cdot f(\mathsf{S}\cap \mathsf{E}'')+ (-1)^{|\mathsf{E}'|-|(\mathsf{S}\cup\{uv\})\cap \mathsf{E}'|}\cdot f\big((\mathsf{S}\cup\{uv\})\cap \mathsf{E}''\big)\right]\\
 &= \sum_{\mathsf{S}\in\mathcal{S}_{\mathsf{d}}^H(\mathsf{E}\setminus\{uv\})} \left[(-1)^{|\mathsf{E}'|-|\mathsf{S}\cap \mathsf{E}'|}\cdot f(\mathsf{S}\cap \mathsf{E}'')+ (-1)^{|\mathsf{E}'|-|(\mathsf{S}\cap \mathsf{E}')\cup\{uv\}|}\cdot f\big(\mathsf{S}\cap \mathsf{E}''\big)\right]\\
&= \sum_{\mathsf{S}\in\mathcal{S}_{\mathsf{d}}^H(\mathsf{E}\setminus\{uv\})} (-1)^{|\mathsf{E}'|-|\mathsf{S}\cap \mathsf{E}'|}\left[1+ (-1)\right]\cdot f\big(\mathsf{S}\cap \mathsf{E}''\big)=0.
\end{align*}
\end{proof}

\begin{lemma}
\label{lem:lem_lemma_pleasant_each_agreeable_type_1}
Let $H$ be an $(s,t)$-pleasant multigraph. Let $H^{(i)}$ be a type-1 agreeable component of $H$ and let $J_{H,\mathsf{d},i}$ be as in \Cref{lem:lem_lemma_pleasant_decompose_on_agreeables}. We have:
\begin{itemize}
\item If $E_1(H^{(i)})$ contains a $\mathsf{d}$-safe edge, then\footnote{Recall that for a type-1 agreeable component $H^{(i)}$, we have $E_{\geq 2}'''(H^{(i)})=E_{\geq 2}(H^{(i)})$. See \Cref{def:def_E_geq2_triple_prime}.}
\begin{align*}
J_{H,\mathsf{d},i}=
\sum_{\mathsf{S}_i\in\mathcal{S}_{\mathsf{d}}^H(E_1(H^{(i)}))}
\left(\frac{\epsilon}{2}\right)^{|\mathsf{S}_i|+|E_{\geq 2}'''(H^{(i)})|}\cdot \left(\frac{\tilde{\epsilon}}{n}\right)^{|E_1(H^{(i)})| - |\mathsf{S}_i|}.
\end{align*}
\item If $E_1(H^{(i)})$ does not contain a $\mathsf{d}$-safe edge, then
\begin{align*}
|J_{H,\mathsf{d},i}|\leq 4\cdot 2^{|E_1(H^{(i)})|}.
\end{align*}
\end{itemize}
\end{lemma}
\begin{proof}
Since $H^{(i)}$ is a cycle, \Cref{lem:lem_Expectation_Y_Nice} implies that

$$\mathbb{E}\left[\prod_{uv\in \mathsf{S}'_i\cup \mathsf{S}''_i\cup \mathsf{E}^{(i)}_{\geq 2}} \mathbf{x}_u\mathbf{x}_v\right]\neq 0$$
if and only if $\mathsf{S}'_i\cup \mathsf{S}''_i\cup \mathsf{E}^{(i)}_{\geq 2}\in\big\{\varnothing, E\big(H^{(i)}\big)\big\}$, in which case we have
$$\mathbb{E}\left[\prod_{uv\in \mathsf{S}'_i\cup \mathsf{S}''_i\cup \mathsf{E}^{(i)}_{\geq 2}} \mathbf{x}_u\mathbf{x}_v\right]=1.$$

Notice the following:
\begin{itemize}
\item If $\mathsf{S}'_i\cup \mathsf{S}''_i\cup \mathsf{E}^{(i)}_{\geq 2}=\varnothing$, then $\mathsf{S}'_i= \mathsf{S}''_i= \mathsf{E}^{(i)}_{\geq 2}=\varnothing$, in which case we have
\begin{align*}
\big|E_1\big(H^{(i)}\big)\big|-|\mathsf{S}_i|-|\mathsf{S}_i''|=\big|E_1\big(H^{(i)}\big)\big|-|\mathsf{S}_i|,
\end{align*}
and
$$|\mathsf{S}_i'|+\big|\mathsf{E}^{(i)}_{\geq 2}\big|=|\mathsf{S}_i''|=0.$$
\item If $\mathsf{S}'_i\cup \mathsf{S}''_i\cup \mathsf{E}^{(i)}_{\geq 2}=E(H^{(i)})$, then $\mathsf{S}_i'=\mathsf{S}_i$, $\mathsf{S}_i''=E_1\big(H^{(i)}\big)\setminus \mathsf{S}_i$, and $\mathsf{E}^{(i)}_{\geq 2}=E_{\geq 2}(H^{(i)})$, in which case we have
\begin{align*}
\big|E_1\big(H^{(i)}\big)\big|-|\mathsf{S}_i|-|\mathsf{S}_i''|=0,
\end{align*}
$$|\mathsf{S}_i'|+\big|\mathsf{E}^{(i)}_{\geq 2}\big|=|\mathsf{S}_i|+ \big|E_{\geq 2}\big(H^{(i)}\big)\big|,$$
and
$$|\mathsf{S}_i''|=\big|E_1\big(H^{(i)}\big)\big|-|\mathsf{S}_i|.$$
\end{itemize}

Applying this to \cref{eq:eq_def_J_H_d_i}, we get

\begin{align*}
J_{H,\mathsf{d},i}=
\sum_{\mathsf{S}_i\in\mathcal{S}_{\mathsf{d}}^H(E_1(H^{(i)}))}
\left[(-1)^{|E_1(H^{(i)})|-|\mathsf{S}_i|}+
\left(\frac{\epsilon}{2}\right)^{|\mathsf{S}_i|+|E_{\geq 2}(H^{(i)})|}\cdot \left(\frac{\tilde{\epsilon}}{n}\right)^{|E_1(H^{(i)})| - |\mathsf{S}_i|}\right].
\end{align*}

It follows from \Cref{lem:lem_sum_structurelly_safe_edge_safe_exists} that if there exists a $\mathsf{d}$-safe edge in $E_1\big(H^{(i)}\big)$, then

\begin{align*}
J_{H,\mathsf{d},i}=
\sum_{\mathsf{S}_i\in\mathcal{S}_{\mathsf{d}}^H(E_1(H^{(i)}))}
\left(\frac{\epsilon}{2}\right)^{|\mathsf{S}_i|+|E_{\geq 2}(H^{(i)})|}\cdot \left(\frac{\tilde{\epsilon}}{n}\right)^{|E_1(H^{(i)})| - |\mathsf{S}_i|}.
\end{align*}

On the other hand, if $E_1\big(H^{(i)}\big)$ does not contain any $\mathsf{d}$-safe edge, then
\begin{align*}
|J_{H,\mathsf{d},i}|\leq 
\sum_{\mathsf{S}_i\in\mathcal{S}_{\mathsf{d}}^H(E_1(H^{(i)}))}
\left(1+1\right)= 2\cdot|\mathcal{S}_{\mathsf{d}}^H(E_1(H^{(i)}))|\leq 2\cdot 2^{|E_1(H^{(i)})|}\leq 4\cdot 2^{|E_1(H^{(i)})|}.
\end{align*}
\end{proof}

\begin{lemma}
\label{lem:lem_lemma_pleasant_each_agreeable_type_2}
Let $H$ be an $(s,t)$-pleasant multigraph. Let $H^{(i)}$ be a type-2 agreeable component of $H$ and let $J_{H,\mathsf{d},i}$ be as in \Cref{lem:lem_lemma_pleasant_decompose_on_agreeables}. We have:
\begin{itemize}
\item If each cycle\footnote{There are two cycles in $E(H^{(i)})$.} of $E(H^{(i)})$ contains a $\mathsf{d}$-safe edge of multiplicity 1, then\footnote{Recall that for a type-2 agreeable component $H^{(i)}$, we have $E_{\geq 2}'''(H^{(i)})=E_{\geq 2}(H^{(i)})$. See \Cref{def:def_E_geq2_triple_prime}.}
\begin{align*}
J_{H,\mathsf{d},i}=
\sum_{\mathsf{S}_i\in\mathcal{S}_{\mathsf{d}}^H(E_1(H^{(i)}))}
\left(\frac{\epsilon}{2}\right)^{|\mathsf{S}_i|+|E_{\geq 2}'''(H^{(i)})|}\cdot \left(\frac{\tilde{\epsilon}}{n}\right)^{|E_1(H^{(i)})| - |\mathsf{S}_i|}.
\end{align*}
\item If there exists at least one cycle of $E(H^{(i)})$ that does not contain any $\mathsf{d}$-safe edge of multiplicity 1, then
\begin{align*}
|J_{H,\mathsf{d},i}|\leq 4\cdot 2^{|E_1(H^{(i)})|}.
\end{align*}
\end{itemize}
\end{lemma}
\begin{proof}
Let $E'\big(H^{(i)}\big)$ and $E''\big(H^{(i)}\big)$ be the two cycles of $H^{(i)}$ as in definition \Cref{def:def_agreeable_multigraphs}. Define
$$E'_1\big(H^{(i)}\big)=E_1\big(H^{(i)}\big)\cap E'\big(H^{(i)}\big),$$
$$E''_1\big(H^{(i)}\big)=E_1\big(H^{(i)}\big)\cap E''\big(H^{(i)}\big),$$
$$E'_{\geq2}\big(H^{(i)}\big)=E_{\geq2}\big(H^{(i)}\big)\cap E'\big(H^{(i)}\big) ,$$
and
$$E''_{\geq2}\big(H^{(i)}\big)=E_{\geq2}\big(H^{(i)}\big)\cap E''\big(H^{(i)}\big).$$

Since $E'(H^{(i)})$ and $E''(H^{(i)})$ are two edge-disjoint cycles that intersect only in one vertex, it follows from \Cref{lem:lem_Expectation_Y_Nice} that

$$\mathbb{E}\left[\prod_{uv\in \mathsf{S}'_i\cup \mathsf{S}''_i\cup \mathsf{E}^{(i)}_{\geq 2}} \mathbf{x}_u\mathbf{x}_v\right]\neq 0$$
if and only if $\mathsf{S}'_i\cup \mathsf{S}''_i\cup \mathsf{E}^{(i)}_{\geq 2}\in\big\{\varnothing, E'(H^{(i)}), E''(H^{(i)}), E(H^{(i)})\big\}$, in which case we have
$$\mathbb{E}\left[\prod_{uv\in \mathsf{S}'_i\cup \mathsf{S}''_i\cup \mathsf{E}^{(i)}_{\geq 2}} \mathbf{x}_u\mathbf{x}_v\right]=1.$$

Notice the following:
\begin{itemize}
\item If $\mathsf{S}'_i\cup \mathsf{S}''_i\cup \mathsf{E}^{(i)}_{\geq 2}=\varnothing$, then $\mathsf{S}'_i= \mathsf{S}''_i= \mathsf{E}^{(i)}_{\geq 2}=\varnothing$, in which case we have
\begin{align*}
\big|E_1\big(H^{(i)}\big)\big|-|\mathsf{S}_i|-|\mathsf{S}_i''|&=\big|E_1\big(H^{(i)}\big)\big|-|\mathsf{S}_i|\\
&=\big|E_1'\big(H^{(i)}\big)\big|-\big|\mathsf{S}_i\cap E_1'\big(H^{(i)}\big)\big|+\big|E_1''\big(H^{(i)}\big)\big|-\big|\mathsf{S}_i\cap E_1''\big(H^{(i)}\big)\big|,
\end{align*}
and
$$|\mathsf{S}_i'|+\big|\mathsf{E}^{(i)}_{\geq 2}\big|=|\mathsf{S}_i''|=0.$$
\item If $\mathsf{S}'_i\cup \mathsf{S}''_i\cup \mathsf{E}^{(i)}_{\geq 2}=E'\big(H^{(i)}\big)$, then $\mathsf{S}_i'=\mathsf{S}_i\cap E'_1\big(H^{(i)}\big)$, $\mathsf{S}_i''=E'_1\big(H^{(i)}\big)\setminus \mathsf{S}_i$ and $\mathsf{E}^{(i)}_{\geq 2}=E_{\geq 2}'\big(H^{(i)}\big)$, in which case we have
\begin{align*}
\big|E_1\big(H^{(i)}\big)&\big|-|\mathsf{S}_i|-|\mathsf{S}_i''|\\
&=\big|E_1'\big(H^{(i)}\big)\big|+\big|E_1''\big(H^{(i)}\big)\big|-\big|\mathsf{S}_i\cap E_1'\big(H^{(i)}\big)\big|-\big|\mathsf{S}_i\cap E_1''\big(H^{(i)}\big)\big|-\big|E'_1\big(H^{(i)}\big)\setminus \mathsf{S}_i\big|\\
&=\big|E_1''\big(H^{(i)}\big)\big|-\big|\mathsf{S}_i\cap E_1''\big(H^{(i)}\big)\big|,
\end{align*}
$$|\mathsf{S}_i'|+\big|\mathsf{E}^{(i)}_{\geq 2}\big|=\big|\mathsf{S}_i\cap E'_1\big(H^{(i)}\big)\big|+ \big|E_{\geq 2}'\big(H^{(i)}\big)\big|,$$
and
$$|\mathsf{S}_i''|=\big|E'_1\big(H^{(i)}\big)\setminus \mathsf{S}_i\big|=\big|E'_1\big(H^{(i)}\big)\big|-\big|\mathsf{S}_i\cap E'_1\big(H^{(i)}\big)\big|.$$
\item If $\mathsf{S}'_i\cup \mathsf{S}''_i\cup \mathsf{E}^{(i)}_{\geq 2}=E''\big(H^{(i)}\big)$, then $\mathsf{S}_i'=\mathsf{S}_i\cap E''_1\big(H^{(i)}\big)$, $\mathsf{S}_i''=E''_1\big(H^{(i)}\big)\setminus \mathsf{S}_i$ and $\mathsf{E}^{(i)}_{\geq 2}=E_{\geq 2}''\big(H^{(i)}\big)$, in which case we have
\begin{align*}
\big|E_1\big(H^{(i)}\big)&\big|-|\mathsf{S}_i|-|\mathsf{S}_i''|\\
&=\big|E_1'\big(H^{(i)}\big)\big|+\big|E_1''\big(H^{(i)}\big)\big|-\big|\mathsf{S}_i\cap E_1'\big(H^{(i)}\big)\big|-\big|\mathsf{S}_i\cap E_1''\big(H^{(i)}\big)\big|-\big|E''_1\big(H^{(i)}\big)\setminus \mathsf{S}_i\big|\\
&=\big|E_1'\big(H^{(i)}\big)\big|-\big|\mathsf{S}_i\cap E_1'\big(H^{(i)}\big)\big|,
\end{align*}
$$|\mathsf{S}_i'|+\big|\mathsf{E}^{(i)}_{\geq 2}\big|=\big|\mathsf{S}_i\cap E''_1\big(H^{(i)}\big)\big|+ \big|E_{\geq 2}''\big(H^{(i)}\big)\big|,$$
and
$$|\mathsf{S}_i''|=\big|E''_1\big(H^{(i)}\big)\setminus \mathsf{S}_i\big|=\big|E''_1\big(H^{(i)}\big)\big|-\big|\mathsf{S}_i\cap E''_1\big(H^{(i)}\big) \big|.$$
\item If $\mathsf{S}'_i\cup \mathsf{S}''_i\cup \mathsf{E}^{(i)}_{\geq 2}=E\big(H^{(i)}\big)$, then $\mathsf{S}_i'=\mathsf{S}_i$, $\mathsf{S}_i''=E_1\big(H^{(i)}\big)\setminus \mathsf{S}_i$ and $\mathsf{E}^{(i)}_{\geq 2}=E_{\geq 2}(H^{(i)})$, in which case we have
\begin{align*}
\big|E_1\big(H^{(i)}\big)\big|-|\mathsf{S}_i|-|\mathsf{S}_i''|=0,
\end{align*}
\begin{align*}
|\mathsf{S}_i'|+\big|\mathsf{E}^{(i)}_{\geq 2}\big|&=|\mathsf{S}_i|+ \big|E_{\geq 2}\big(H^{(i)}\big)\big|,
\end{align*}
and
\begin{align*}
|\mathsf{S}_i''|&=\big|E_1\big(H^{(i)}\big)\big|-|\mathsf{S}_i|.
\end{align*}
\end{itemize}

Applying this to \cref{eq:eq_def_J_H_d_i}, we get

\begin{align*}
J_{H,\mathsf{d},i}&=
\sum_{\mathsf{S}_i\in\mathcal{S}_{\mathsf{d}}^H(E_1(H^{(i)}))}
\Bigg[(-1)^{|E_1'(H^{(i)})|-|\mathsf{S}_i\cap E_1'(H^{(i)})|}\cdot(-1)^{|E_1''(H^{(i)})|-|\mathsf{S}_i\cap E_1''(H^{(i)})|}\\
&\quad\quad\quad\quad+(-1)^{|E_1''(H^{(i)})|-|\mathsf{S}_i\cap E_1''(H^{(i)})|}\cdot \left(\frac{\epsilon}{2}\right)^{|\mathsf{S}_i\cap E_1'(H^{(i)})|+|E_{\geq 2}'(H^{(i)})|}\cdot \left(\frac{\tilde{\epsilon}}{n}\right)^{|E_1'(H^{(i)})| - |\mathsf{S}_i\cap E_1'(H^{(i)})|}\\
&\quad\quad\quad\quad+ (-1)^{|E_1'(H^{(i)})|-|\mathsf{S}_i\cap E_1'(H^{(i)})|}\cdot \left(\frac{\epsilon}{2}\right)^{|\mathsf{S}_i\cap E_1''(H^{(i)})|+|E_{\geq 2}''(H^{(i)})|}\cdot \left(\frac{\tilde{\epsilon}}{n}\right)^{|E_1''(H^{(i)})| - |\mathsf{S}_i\cap E_1''(H^{(i)})|}\\
&\quad\quad\quad\quad\quad\quad\quad\quad\quad\quad\quad\quad\quad\quad\quad\quad\quad\quad \quad\quad+\left(\frac{\epsilon}{2}\right)^{|\mathsf{S}_i|+|E_{\geq 2}(H^{(i)})|}\cdot \left(\frac{\tilde{\epsilon}}{n}\right)^{|E_1(H^{(i)})| - |\mathsf{S}_i|}\Bigg].
\end{align*}

It follows from \Cref{lem:lem_sum_structurelly_safe_edge_safe_exists} that if there exists a $\mathsf{d}$-safe edge in $E_1'\big(H^{(i)}\big)$ and a $\mathsf{d}$-safe edge in $E_1''\big(H^{(i)}\big)$, then

\begin{align*}
J_{H,\mathsf{d},i}=
\sum_{\mathsf{S}_i\in\mathcal{S}_{\mathsf{d}}^H(E_1(H^{(i)}))}
\left(\frac{\epsilon}{2}\right)^{|\mathsf{S}_i|+|E_{\geq 2}(H^{(i)})|}\cdot \left(\frac{\tilde{\epsilon}}{n}\right)^{|E_1(H^{(i)})| - |\mathsf{S}_i|}.
\end{align*}

On the other hand, if $E_1'\big(H^{(i)}\big)$ does not contain any $\mathsf{d}$-safe edge or $E_1''\big(H^{(i)}\big)$ does not contain any $\mathsf{d}$-safe edge, then
\begin{align*}
|J_{H,\mathsf{d},i}|\leq 
\sum_{\mathsf{S}_i\in\mathcal{S}_{\mathsf{d}}^H(E_1(H^{(i)}))}
\left(1+1+1+1\right)= 4\cdot|\mathcal{S}_{\mathsf{d}}^H(E_1(H^{(i)}))|\leq 4\cdot 2^{|E_1(H^{(i)})|}.
\end{align*}
\end{proof}

\begin{lemma}
\label{lem:lem_lemma_pleasant_each_agreeable_type_3}
Let $H$ be an $(s,t)$-pleasant multigraph. Let $H^{(i)}$ be a type-3 agreeable component of $H$ and let $J_{H,\mathsf{d},i}$ be as in \Cref{lem:lem_lemma_pleasant_decompose_on_agreeables}. Let $E'\big(H^{(i)}\big)$ and $E''\big(H^{(i)}\big)$ be two cycles of $H^{(i)}$ as in definition \Cref{def:def_agreeable_multigraphs}, i.e., $E'\big(H^{(i)}\big)\cap E''\big(H^{(i)}\big)$ is a simple path of edges of multiplicity at least 2. Define
$$E'''\big(H^{(i)}\big)=\Big(E'\big(H^{(i)}\big)\cup E''\big(H^{(i)}\big)\Big)\setminus \Big(E'\big(H^{(i)}\big)\cap E''\big(H^{(i)}\big)\Big),$$
$$E'''_{\geq2}\big(H^{(i)}\big)=E_{\geq2}\big(H^{(i)}\big)\cap E'''\big(H^{(i)}\big)=E_{\geq2}\big(H^{(i)}\big)\setminus \Big(E'\big(H^{(i)}\big)\cap E''\big(H^{(i)}\big)\Big),$$
$$E'_1\big(H^{(i)}\big)=E_1\big(H^{(i)}\big)\cap E'\big(H^{(i)}\big),$$
and
$$E''_1\big(H^{(i)}\big)=E_1\big(H^{(i)}\big)\cap E''\big(H^{(i)}\big).$$
We have:
\begin{itemize}
\item If there exists a $\mathsf{d}$-safe edge in $E_1'\big(H^{(i)}\big)$ and a $\mathsf{d}$-safe edge in $E_1''\big(H^{(i)}\big)$, then
\begin{align*}
J_{H,\mathsf{d},i}=
\sum_{\mathsf{S}_i\in\mathcal{S}_{\mathsf{d}}^H(E_1(H^{(i)}))}
\left(\frac{\epsilon}{2}\right)^{|\mathsf{S}_i|+|E_{\geq 2}'''(H^{(i)})|}\cdot \left(\frac{\tilde{\epsilon}}{n}\right)^{|E_1(H^{(i)})| - |\mathsf{S}_i|}.
\end{align*}
\item If $E_1'\big(H^{(i)}\big)$ does not contain any $\mathsf{d}$-safe edge or $E_1''\big(H^{(i)}\big)$ does not contain any $\mathsf{d}$-safe edge, then
\begin{align*}
|J_{H,\mathsf{d},i}|\leq 4\cdot 2^{|E_1(H^{(i)})|}.
\end{align*}
\end{itemize}
\end{lemma}
\begin{proof}
Define
$$E'_{\geq2}\big(H^{(i)}\big)=E_{\geq2}\big(H^{(i)}\big)\cap E'\big(H^{(i)}\big),$$
and
$$E''_{\geq2}\big(H^{(i)}\big)=E_{\geq2}\big(H^{(i)}\big)\cap E''\big(H^{(i)}\big).$$

The only subsets of $E\big(H^{(i)}\big)$ consisting of edge-disjoint unions of cycles are $\varnothing$, $E'\big(H^{(i)}\big)$, $E''\big(H^{(i)}\big)$ and $E'''\big(H^{(i)}\big)$. It follows from \Cref{lem:lem_Expectation_Y_Nice} that

$$\mathbb{E}\left[\prod_{uv\in \mathsf{S}'_i\cup \mathsf{S}''_i\cup \mathsf{E}^{(i)}_{\geq 2}} \mathbf{x}_u\mathbf{x}_v\right]\neq 0$$
if and only if $\mathsf{S}'_i\cup \mathsf{S}''_i\cup \mathsf{E}^{(i)}_{\geq 2}\in\big\{\varnothing, E'(H^{(i)}), E''(H^{(i)}), E'''(H^{(i)})\big\}$, in which case we have
$$\mathbb{E}\left[\prod_{uv\in \mathsf{S}'_i\cup \mathsf{S}''_i\cup \mathsf{E}^{(i)}_{\geq 2}} \mathbf{x}_u\mathbf{x}_v\right]=1.$$

Notice the following:
\begin{itemize}
\item If $\mathsf{S}'_i\cup \mathsf{S}''_i\cup \mathsf{E}^{(i)}_{\geq 2}=\varnothing$, then $\mathsf{S}'_i= \mathsf{S}''_i= \mathsf{E}^{(i)}_{\geq 2}=\varnothing$, in which case we have
\begin{align*}
\big|E_1\big(H^{(i)}\big)\big|-|\mathsf{S}_i|-|\mathsf{S}_i''|&=\big|E_1\big(H^{(i)}\big)\big|-|\mathsf{S}_i|\\
&=\big|E_1'\big(H^{(i)}\big)\big|-\big|\mathsf{S}_i\cap E_1'\big(H^{(i)}\big)\big|+\big|E_1''\big(H^{(i)}\big)\big|-\big|\mathsf{S}_i\cap E_1''\big(H^{(i)}\big)\big|,
\end{align*}
and
$$|\mathsf{S}_i'|+\big|\mathsf{E}^{(i)}_{\geq 2}\big|=|\mathsf{S}_i''|=0.$$
\item If $\mathsf{S}'_i\cup \mathsf{S}''_i\cup \mathsf{E}^{(i)}_{\geq 2}=E'\big(H^{(i)}\big)$, then $\mathsf{S}_i'=\mathsf{S}_i\cap E'_1\big(H^{(i)}\big)$, $\mathsf{S}_i''=E'_1\big(H^{(i)}\big)\setminus \mathsf{S}_i$ and $\mathsf{E}^{(i)}_{\geq 2}=E_{\geq 2}'\big(H^{(i)}\big)$, in which case we have
\begin{align*}
\big|E_1\big(H^{(i)}\big)&\big|-|\mathsf{S}_i|-|\mathsf{S}_i''|\\
&=\big|E_1'\big(H^{(i)}\big)\big|+\big|E_1''\big(H^{(i)}\big)\big|-\big|\mathsf{S}_i\cap E_1'\big(H^{(i)}\big)\big|-\big|\mathsf{S}_i\cap E_1''\big(H^{(i)}\big)\big|-\big|E'_1\big(H^{(i)}\big)\setminus \mathsf{S}_i\big|\\
&=\big|E_1''\big(H^{(i)}\big)\big|-\big|\mathsf{S}_i\cap E_1''\big(H^{(i)}\big)\big|,
\end{align*}
$$|\mathsf{S}_i'|+\big|\mathsf{E}^{(i)}_{\geq 2}\big|=\big|\mathsf{S}_i\cap E'_1\big(H^{(i)}\big)\big|+ \big|E_{\geq 2}'\big(H^{(i)}\big)\big|,$$
and
$$|\mathsf{S}_i''|=\big|E'_1\big(H^{(i)}\big)\setminus \mathsf{S}_i\big|=\big|E'_1\big(H^{(i)}\big)\big|-\big|\mathsf{S}_i\cap E'_1\big(H^{(i)}\big)\big|.$$
\item If $\mathsf{S}'_i\cup \mathsf{S}''_i\cup \mathsf{E}^{(i)}_{\geq 2}=E''\big(H^{(i)}\big)$, then $\mathsf{S}_i'=\mathsf{S}_i\cap E''_1\big(H^{(i)}\big)$, $\mathsf{S}_i''=E''_1\big(H^{(i)}\big)\setminus \mathsf{S}_i$ and $\mathsf{E}^{(i)}_{\geq 2}=E_{\geq 2}''\big(H^{(i)}\big)$, in which case we have
\begin{align*}
\big|E_1\big(H^{(i)}\big)&\big|-|\mathsf{S}_i|-|\mathsf{S}_i''|\\
&=\big|E_1'\big(H^{(i)}\big)\big|+\big|E_1''\big(H^{(i)}\big)\big|-\big|\mathsf{S}_i\cap E_1'\big(H^{(i)}\big)\big|-\big|\mathsf{S}_i\cap E_1''\big(H^{(i)}\big)\big|-\big|E''_1\big(H^{(i)}\big)\setminus \mathsf{S}_i\big|\\
&=\big|E_1'\big(H^{(i)}\big)\big|-\big|\mathsf{S}_i\cap E_1'\big(H^{(i)}\big)\big|,
\end{align*}
$$|\mathsf{S}_i'|+\big|\mathsf{E}^{(i)}_{\geq 2}\big|=\big|\mathsf{S}_i\cap E''_1\big(H^{(i)}\big)\big|+ \big|E_{\geq 2}''\big(H^{(i)}\big)\big|,$$
and
$$|\mathsf{S}_i''|=\big|E''_1\big(H^{(i)}\big)\setminus \mathsf{S}_i\big|=\big|E''_1\big(H^{(i)}\big)\big|-\big|\mathsf{S}_i\cap E''_1\big(H^{(i)}\big) \big|.$$
\item If $\mathsf{S}'_i\cup \mathsf{S}''_i\cup \mathsf{E}^{(i)}_{\geq 2}=E'''\big(H^{(i)}\big)$, then $\mathsf{S}_i'=\mathsf{S}_i$, $\mathsf{S}_i''=E_1\big(H^{(i)}\big)\setminus \mathsf{S}_i$ and $\mathsf{E}^{(i)}_{\geq 2}=E_{\geq 2}'''(H^{(i)})$, in which case we have
\begin{align*}
\big|E_1\big(H^{(i)}\big)\big|-|\mathsf{S}_i|-|\mathsf{S}_i''|=0,
\end{align*}
\begin{align*}
|\mathsf{S}_i'|+\big|\mathsf{E}^{(i)}_{\geq 2}\big|&=|\mathsf{S}_i|+ \big|E_{\geq 2}'''\big(H^{(i)}\big)\big|,
\end{align*}
and
\begin{align*}
|\mathsf{S}_i''|&=\big|E_1\big(H^{(i)}\big)\big|-|\mathsf{S}_i|.
\end{align*}
\end{itemize}

Applying this to \cref{eq:eq_def_J_H_d_i}, we get

\begin{align*}
J_{H,\mathsf{d},i}&=
\sum_{\mathsf{S}_i\in\mathcal{S}_{\mathsf{d}}^H(E_1(H^{(i)}))}
\Bigg[(-1)^{|E_1'(H^{(i)})|-|\mathsf{S}_i\cap E_1'(H^{(i)})|}\cdot(-1)^{|E_1''(H^{(i)})|-|\mathsf{S}_i\cap E_1''(H^{(i)})|}\\
&\quad\quad\quad\quad+(-1)^{|E_1''(H^{(i)})|-|\mathsf{S}_i\cap E_1''(H^{(i)})|}\cdot \left(\frac{\epsilon}{2}\right)^{|\mathsf{S}_i\cap E_1'(H^{(i)})|+|E_{\geq 2}'(H^{(i)})|}\cdot \left(\frac{\tilde{\epsilon}}{n}\right)^{|E_1'(H^{(i)})| - |\mathsf{S}_i\cap E_1'(H^{(i)})|}\\
&\quad\quad\quad\quad+ (-1)^{|E_1'(H^{(i)})|-|\mathsf{S}_i\cap E_1'(H^{(i)})|}\cdot \left(\frac{\epsilon}{2}\right)^{|\mathsf{S}_i\cap E_1''(H^{(i)})|+|E_{\geq 2}''(H^{(i)})|}\cdot \left(\frac{\tilde{\epsilon}}{n}\right)^{|E_1''(H^{(i)})| - |\mathsf{S}_i\cap E_1''(H^{(i)})|}\\
&\quad\quad\quad\quad\quad\quad\quad\quad\quad\quad\quad\quad\quad\quad\quad\quad\quad\quad \quad\quad+\left(\frac{\epsilon}{2}\right)^{|\mathsf{S}_i|+|E_{\geq 2}'''(H^{(i)})|}\cdot \left(\frac{\tilde{\epsilon}}{n}\right)^{|E_1(H^{(i)})| - |\mathsf{S}_i|}\Bigg].
\end{align*}

It follows from \Cref{lem:lem_sum_structurelly_safe_edge_safe_exists} that if there exists a $\mathsf{d}$-safe edge in $E_1'\big(H^{(i)}\big)$ and a $\mathsf{d}$-safe edge in $E_1''\big(H^{(i)}\big)$, then

\begin{align*}
J_{H,\mathsf{d},i}=
\sum_{\mathsf{S}_i\in\mathcal{S}_{\mathsf{d}}^H(E_1(H^{(i)}))}
\left(\frac{\epsilon}{2}\right)^{|\mathsf{S}_i|+|E_{\geq 2}'''(H^{(i)})|}\cdot \left(\frac{\tilde{\epsilon}}{n}\right)^{|E_1(H^{(i)})| - |\mathsf{S}_i|}.
\end{align*}

On the other hand, if $E_1'\big(H^{(i)}\big)$ does not contain any $\mathsf{d}$-safe edge or $E_1''\big(H^{(i)}\big)$ does not contain any $\mathsf{d}$-safe edge, then
\begin{align*}
|J_{H,\mathsf{d},i}|\leq 
\sum_{\mathsf{S}_i\in\mathcal{S}_{\mathsf{d}}^H(E_1(H^{(i)}))}
\left(1+1+1+1\right)= 4\cdot|\mathcal{S}_{\mathsf{d}}^H(E_1(H^{(i)}))|\leq 4\cdot 2^{|E_1(H^{(i)})|}.
\end{align*}
\end{proof}

\begin{lemma}
\label{lem:lem_lemma_pleasant_each_agreeable_all_types}
Let $H$ be an $(s,t)$-pleasant multigraph. Let $H^{(i)}$ be an agreeable component of $H$ and let $J_{H,\mathsf{d},i}$ be as in \Cref{lem:lem_lemma_pleasant_decompose_on_agreeables}. We have:
\begin{itemize}
\item If every cycle of $E(H^{(i)})$ contains a $\mathsf{d}$-safe edge of multiplicity 1, then
\begin{align*}
J_{H,\mathsf{d},i}=
\sum_{\mathsf{S}_i\in\mathcal{S}_{\mathsf{d}}^H(E_1(H^{(i)}))}
\left(\frac{\epsilon}{2}\right)^{|\mathsf{S}_i|+|E_{\geq 2}'''(H^{(i)})|}\cdot \left(\frac{\tilde{\epsilon}}{n}\right)^{|E_1(H^{(i)})| - |\mathsf{S}_i|}.
\end{align*}
\item If there exists at least one cycle of $E(H^{(i)})$ that does not contain any $\mathsf{d}$-safe edge of multiplicity 1, then
\begin{align*}
|J_{H,\mathsf{d},i}|\leq 4\cdot 2^{|E_1(H^{(i)})|}.
\end{align*}
\end{itemize}
\end{lemma}
\begin{proof}
This is a direct corollary of \Cref{lem:lem_lemma_pleasant_each_agreeable_type_1}, \Cref{lem:lem_lemma_pleasant_each_agreeable_type_2} and \Cref{lem:lem_lemma_pleasant_each_agreeable_type_3}.
\end{proof}

\begin{lemma}
\label{lem:lem_sum_over_structurally_safe_bound}
Let $H$ be an $(s,t)$-pleasant multigraph and let $H^{(i)}$ be an agreeable component of $H$. We have:
\begin{itemize}
\item If all the edges of $E_1\big(H^{(i)}\big)$ are $\mathsf{d}$-safe, then
\begin{align*}
\sum_{\mathsf{S}_i\in\mathcal{S}_{\mathsf{d}}^H(E_1(H^{(i)}))}
\left(\frac{\epsilon}{2}\right)^{|\mathsf{S}_i|}\cdot \left(\frac{\tilde{\epsilon}}{n}\right)^{|E_1(H^{(i)})| - |\mathsf{S}_i|}=\left(\frac{\epsilon d}{2\tilde{d}}\right)^{|E_1(H^{(i)})|}.
\end{align*}
\item If there exists an edge in $E_1\big(H^{(i)}\big)$ that is not $\mathsf{d}$-safe, then
\begin{align*}
\sum_{\mathsf{S}_i\in\mathcal{S}_{\mathsf{d}}^H(E_1(H^{(i)}))}
\left(\frac{\epsilon}{2}\right)^{|\mathsf{S}_i|}\cdot \left(\frac{\tilde{\epsilon}}{n}\right)^{|E_1(H^{(i)})| - |\mathsf{S}_i|}\leq\frac{1}{\sqrt{n}}\cdot\left(\frac{\epsilon d}{2\tilde{d}}\right)^{|E_1(H^{(i)})|}.
\end{align*}
\end{itemize}
\end{lemma}
\begin{proof}
Notice that
\begin{equation}
\label{eq:eq_epsilon_epsilon_tilde}
\frac{\epsilon}{2}+\frac{\tilde{\epsilon}}{n}=\frac{\epsilon}{2} + \frac{\epsilon d^2}{2\tilde{d}n}=\frac{\epsilon}{2}\left(1 + \frac{d^2}{d\left(1-\frac{d}{n}\right)n}\right)=\frac{\epsilon}{2}\left(1 + \frac{d}{n-d}\right)=\frac{\epsilon}{2}\cdot\frac{n}{n-d}=\frac{\epsilon d}{2\tilde{d}}.
\end{equation}
\begin{itemize}
\item If all the edges of $E_1\big(H^{(i)}\big)$ are $\mathsf{d}$-safe, then all subsets of $E_1\big(H^{(i)}\big)$ are $\mathsf{d}$-structurally safe, hence
\begin{align*}
\sum_{\mathsf{S}_i\in\mathcal{S}_{\mathsf{d}}^H(E_1(H^{(i)}))}
\left(\frac{\epsilon}{2}\right)^{|\mathsf{S}_i|}\cdot \left(\frac{\tilde{\epsilon}}{n}\right)^{|E_1(H^{(i)})| - |\mathsf{S}_i|}&=\sum_{\mathsf{S}_i\subseteq E_1(H^{(i)})}
\left(\frac{\epsilon}{2}\right)^{|\mathsf{S}_i|}\cdot \left(\frac{\tilde{\epsilon}}{n}\right)^{|E_1(H^{(i)})| - |\mathsf{S}_i|}\\
&=\left(\frac{\epsilon}{2}+\frac{\tilde{\epsilon}}{n}\right)^{|E_1(H^{(i)})|}=\left(\frac{\epsilon d}{2\tilde{d}}\right)^{|E_1(H^{(i)})|}.
\end{align*}
\item If there exists an edge in $E_1\big(H^{(i)}\big)$ that is not $\mathsf{d}$-safe, then $E_1\big(H^{(i)}\big)$ is not $\mathsf{d}$-structurally safe, hence
\begin{align*}
\sum_{\mathsf{S}_i\in\mathcal{S}_{\mathsf{d}}^H(E_1(H^{(i)}))}&
\left(\frac{\epsilon}{2}\right)^{|\mathsf{S}_i|}\cdot \left(\frac{\tilde{\epsilon}}{n}\right)^{|E_1(H^{(i)})| - |\mathsf{S}_i|}\\
&\leq\sum_{\mathsf{S}_i\subsetneq E_1(H^{(i)})}
\left(\frac{\epsilon}{2}\right)^{|\mathsf{S}_i|}\cdot \left(\frac{\tilde{\epsilon}}{n}\right)^{|E_1(H^{(i)})| - |\mathsf{S}_i|}=\left(\frac{\epsilon}{2}+\frac{\tilde{\epsilon}}{n}\right)^{|E_1(H^{(i)})|}-\left(\frac{\epsilon}{2}\right)^{|E_1(H^{(i)})|}\\
&\stackrel{(\ast)}{=}\left(\frac{\epsilon d}{2\tilde{d}}\right)^{|E_1(H^{(i)})|}-\left(\frac{\epsilon}{2}\right)^{|E_1(H^{(i)})|}=\left(\frac{\epsilon d}{2\tilde{d}}\right)^{|E_1(H^{(i)})|}\left(1-\left(\frac{\tilde{d}}{d}\right)^{|E_1(H^{(i)})|}\right)\\
&=\left(\frac{\epsilon d}{2\tilde{d}}\right)^{|E_1(H^{(i)})|}\left(1-\left(1-\frac{d}{n}\right)^{|E_1(H^{(i)})|}\right)\leq\left(\frac{\epsilon d}{2\tilde{d}}\right)^{|E_1(H^{(i)})|}\left(1-\left(1-\frac{d}{n}\right)^{st}\right)\\
&=\left(\frac{\epsilon d}{2\tilde{d}}\right)^{|E_1(H^{(i)})|}\left(1-\left(1-O\left(\frac{std}{n}\right)\right)\right)\stackrel{(\dagger)}{\leq} \frac{1}{\sqrt{n}}\cdot\left(\frac{\epsilon d}{2\tilde{d}}\right)^{|E_1(H^{(i)})|},
\end{align*}
where $(\ast)$ follows from \cref{eq:eq_epsilon_epsilon_tilde} and $(\dagger)$ is true for $n$ large enough.
\end{itemize}
\end{proof}

\begin{lemma} 
\label{lem:lem_technical_lemma_pleasant_wb_cond_d}
Let $H$ be an $(s,t)$-pleasant multigraph, where $t=K\log n$. Let $E_{\geq 2}'''(H)$ be as in \Cref{def:def_E_geq2_triple_prime}. We have the following:
\begin{itemize}
\item If $S_{\mathsf{d}}^H= E_1(H)$, i.e., if all the edges of $E_1(H)$ are $\mathsf{d}$-safe, then
\begin{align*}
&\left(1-\frac{1}{\sqrt{n}}\right)\cdot\left(\frac{\epsilon d}{2n}\right)^{|E_1(H)|+|E_{\geq 2}'''(H)|}\left(\frac{d}{n}\right)^{|E_{\geq 2}(H)|-|E_{\geq 2}'''(H)|}\\
&\;\;\leq \mathbb{E}\Big[\mathbb{E}\big[\overline{\mathbf{Y}}_H\big|\mathbf{x},\mathcal{E}_{wb,H},\mathbf{D}^H=\mathsf{d}\big]\cdot P_{1,H}^{wb}(\mathbf{x})\Big]\leq \left(\frac{\epsilon d}{2n}\right)^{|E_1(H)|+|E_{\geq 2}'''(H)|}\left(\frac{d}{n}\right)^{|E_{\geq 2}(H)|-|E_{\geq 2}'''(H)|}.
\end{align*}
\item If $S_{\mathsf{d}}^H\neq E_1(H)$ and every cycle in $E(H)$ contains a $\mathsf{d}$-edge of multiplicity 1, then for $n$ large enough, we have
\begin{align*}
\resizebox{0.9\textwidth}{!}{$\displaystyle 0\leq\mathbb{E}\Big[\mathbb{E}\big[\overline{\mathbf{Y}}_H\big|\mathbf{x},\mathcal{E}_{wb,H},\mathbf{D}^H=\mathsf{d}\big]\cdot P_{1,H}^{wb}(\mathbf{x})\Big]\leq \frac{1}{\sqrt{n}}\cdot\left(\frac{\epsilon d}{2n}\right)^{|E_1(H)|+|E_{\geq 2}'''(H)|}\left(\frac{d}{n}\right)^{|E_{\geq 2}(H)|-|E_{\geq 2}'''(H)|}.$}
\end{align*}
\item If there is a cycle in $E(H)$ that does not contain any $\mathsf{d}$-safe edge of multiplicity 1, then for $n$ large enough, we have
\begin{align*}
\resizebox{0.9\textwidth}{!}{$\displaystyle\Big|\mathbb{E}\Big[\mathbb{E}\big[\overline{\mathbf{Y}}_H\big|\mathbf{x},\mathcal{E}_{wb,H},\mathbf{D}^H=\mathsf{d}\big]\cdot P_{1,H}^{wb}(\mathbf{x})\Big]\Big|\leq \left(\frac{16}{\epsilon}\right)^{st}\cdot \left(\frac{\epsilon d}{2n}\right)^{|E_1(H)|+|E_{\geq 2}'''(H)|}\left(\frac{d}{n}\right)^{|E_{\geq 2}(H)|-|E_{\geq 2}'''(H)|}.$}
\end{align*}
\end{itemize}
\end{lemma}
\begin{proof}
From \Cref{lem:lem_lemma_pleasant_decompose_on_agreeables} we have
\begin{align*}
&\mathbb{E}\left[\mathbb{E}\big[\overline{\mathbf{Y}}_H\big|\mathbf{x},\mathcal{E}_{wb,H},\mathbf{D}^H=\mathsf{d}\big]\cdot P_{1,H}^{wb}(\mathbf{x})\right]\\
&\quad\quad\quad\quad\quad\quad\quad\quad=\left(\frac{\tilde{d}}{n}\right)^{|E_1(H)|}\left(1-\frac{d}{n}\right)^{m_H(E_{\geq 2}(H))}\left(\frac{d}{n}\right)^{|E_{\geq 2}(H)|}\cdot\mathbb{P}\left[\mathcal{E}_{H,b}\right]\cdot\prod_{i\in[r_H]}J_{H,\mathsf{d},i},
\end{align*}

We will provide tight bounds on $\mathbb{E}\Big[\mathbb{E}\big[\overline{\mathbf{Y}}_H\big|\mathbf{x},\mathcal{E}_{wb,H},\mathbf{D}^H=\mathsf{d}\big]\cdot P_{1,H}^{wb}(\mathbf{x})\Big]$. In order to do this, we will distinguish between three cases:
\begin{itemize}
\item If $S_{\mathsf{d}}^H=E_1(H)$, then for every $i\in[r_H]$, all the edges in $E_1\big(H^{(i)}\big)$ are $\mathsf{d}$-safe. Now from \Cref{lem:lem_lemma_pleasant_decompose_on_agreeables}, \Cref{lem:lem_lemma_pleasant_each_agreeable_all_types} and \Cref{lem:lem_sum_over_structurally_safe_bound}, and from the fact that $\displaystyle |E_1(H)|=\sum_{i\in [r_H]}\big|E_1\big(H^{(i)}\big)\big|$ and $\displaystyle |E_{\geq 2}'''(H)|=\sum_{i\in [r_H]}\big|E_{\geq 2}'''\big(H^{(i)}\big)\big|$, we can deduce that

\begin{align*}
\mathbb{E}\Big[\mathbb{E}\big[&\overline{\mathbf{Y}}_H\big|\mathbf{x},\mathcal{E}_{wb,H},\mathbf{D}^H=\mathsf{d}\big]\cdot P_{1,H}^{wb}(\mathbf{x})\Big]\\
&=\left(\frac{\epsilon d}{2n}\right)^{|E_1(H)|+|E_{\geq 2}'''(H)|}\left(1-\frac{d}{n}\right)^{m_H(E_{\geq 2}(H))}\left(\frac{d}{n}\right)^{|E_{\geq 2}(H)|-|E_{\geq 2}'''(H)|}\cdot\mathbb{P}\left[\mathcal{E}_{H,b}\right].
\end{align*}

Now from \Cref{lem:lem_Prob_Balanced_Y} we have

\begin{equation*}
\begin{aligned}
\left(1-\frac{d}{n}\right)^{m_H(E_{\geq 2}(H))}\cdot \mathbb{P}\left[\mathcal{E}_{H,b}\right]&\geq \left(1-\frac{d}{n}\right)^{st}\cdot \big(1-2e^{\frac{9}{8}\sqrt{n}}\big)\\
&\geq \left(1-O\left(\frac{dst}{n}\right)\right)\cdot \big(1-2e^{\frac{9}{8}\sqrt{n}}\big)\geq 1-\frac{1}{\sqrt{n}},
\end{aligned}
\end{equation*}

where the last inequality is true for $n$ large enough. Therefore, if $S_{\mathsf{d}}^H=E_1(H)$, we have

\begin{align*}
&\left(1-\frac{1}{\sqrt{n}}\right)\cdot\left(\frac{\epsilon d}{2n}\right)^{|E_1(H)|+|E_{\geq 2}'''(H)|}\left(\frac{d}{n}\right)^{|E_{\geq 2}(H)|-|E_{\geq 2}'''(H)|}\\
&\;\;\leq \mathbb{E}\Big[\mathbb{E}\big[\overline{\mathbf{Y}}_H\big|\mathbf{x},\mathcal{E}_{wb,H},\mathbf{D}^H=\mathsf{d}\big]\cdot P_{1,H}^{wb}(\mathbf{x})\Big]\leq \left(\frac{\epsilon d}{2n}\right)^{|E_1(H)|+|E_{\geq 2}'''(H)|}\left(\frac{d}{n}\right)^{|E_{\geq 2}(H)|-|E_{\geq 2}'''(H)|}.
\end{align*}
\item If $S_{\mathsf{d}}^H\neq E_1(H)$ and every cycle in $E(H)$ has at least one $\mathsf{d}$-safe edge of multiplicity 1, define $$\mathcal{I}^{H,\mathsf{d}}_{E\not{subseteq} S}=\big\{i\in[r_H]:\; E_1\big(H^{(i)}\big)\not{subseteq} S_{\mathsf{d}}^H\big\}.$$ 

Since $S_{\mathsf{d}}^H\neq E_1(H)$, then we must have $\mathcal{I}^{H,\mathsf{d}}_{E\not{subseteq} S}\neq\varnothing$. Now from \Cref{lem:lem_lemma_pleasant_decompose_on_agreeables}, \Cref{lem:lem_lemma_pleasant_each_agreeable_all_types}, \Cref{lem:lem_sum_over_structurally_safe_bound}, and from the fact that $\displaystyle |E_1(H)|=\sum_{i\in [r_H]}\big|E_1\big(H^{(i)}\big)\big|$ and $\displaystyle |E_{\geq 2}'''(H)|=\sum_{i\in [r_H]}\big|E_{\geq 2}'''\big(H^{(i)}\big)\big|$, we can deduce that

\begin{align*}
0&\leq \mathbb{E}\Big[\mathbb{E}\big[\overline{\mathbf{Y}}_H\big|\mathbf{x},\mathcal{E}_{wb,H},\mathbf{D}^H=\mathsf{d}\big]\cdot P_{1,H}^{wb}(\mathbf{x})\Big]\\
&\leq \left(\frac{1}{\sqrt{n}}\right)^{|\mathcal{I}^{H,\mathsf{d}}_{E\not{subseteq} S}|}\left(\frac{\epsilon d}{2n}\right)^{|E_1(H)|+|E_{\geq 2}'''(H)|}\left(1-\frac{d}{n}\right)^{m_H(E_{\geq 2}(H))}\left(\frac{d}{n}\right)^{|E_{\geq 2}(H)|-|E_{\geq 2}'''(H)|}\cdot\mathbb{P}\left[\mathcal{E}_{H,b}\right]\\
&\leq \frac{1}{\sqrt{n}}\cdot\left(\frac{\epsilon d}{2n}\right)^{|E_1(H)|+|E_{\geq 2}'''(H)|}\left(\frac{d}{n}\right)^{|E_{\geq 2}(H)|-|E_{\geq 2}'''(H)|}.
\end{align*}

\item If there is a cycle in $E(H)$ that does not contain any $\mathsf{d}$-safe edge of multiplicity 1, define
\begin{align*}
\mathcal{I}_{\mathsf{d}}^H=\Big\{i\in[r_H]:\; H^{(i)}&\text{ contains a cycle that does not contain}\\
&\quad\quad\quad\quad\quad\quad\quad\quad\text{any }\mathsf{d}\text{-safe edge of multiplicity 1}\Big\}.
\end{align*}
From \Cref{lem:lem_lemma_pleasant_decompose_on_agreeables}, \Cref{lem:lem_lemma_pleasant_each_agreeable_all_types}, \Cref{lem:lem_sum_over_structurally_safe_bound}, from the fact that $\displaystyle |E_1(H)|=\sum_{i\in [r_H]}\big|E_1\big(H^{(i)}\big)\big|$, and from the definition of $\mathcal{I}_{\mathsf{d}}^H$, we can deduce that

\begin{align*}
&\left|\mathbb{E}\Big[\mathbb{E}\big[\overline{\mathbf{Y}}_H\big|\mathbf{x},\mathcal{E}_{wb,H},\mathbf{D}^H=\mathsf{d}\big]\cdot P_{1,H}^{wb}(\mathbf{x})\Big]\right|\\
&\leq \left(\frac{\epsilon d}{2n}\right)^{|E_1(H)|}\left(1-\frac{d}{n}\right)^{m_H(E_{\geq 2}(H))}\left(\frac{d}{n}\right)^{|E_{\geq 2}(H)|}\cdot\mathbb{P}\left[\mathcal{E}_{H,b}\right] \cdot\prod_{i\in \mathcal{I}_{\mathsf{d}}^H}\left[4\cdot\left(\frac{4}{\epsilon}\right)^{|E_1(H^{(i)})|}\right]\\
&\leq \left(\frac{\epsilon d}{2n}\right)^{|E_1(H)|}\left(\frac{d}{n}\right)^{|E_{\geq 2}(H)|}\cdot\prod_{i\in \mathcal{I}_{\mathsf{d}}^H}\left[\left(\frac{16}{\epsilon}\right)^{|E_1(H^{(i)})|}\right]\leq \left(\frac{\epsilon d}{2n}\right)^{|E_1(H)|}\left(\frac{d}{n}\right)^{|E_{\geq 2}(H)|}\cdot\left(\frac{16}{\epsilon}\right)^{|E_1(H)|}\\
&\leq \left(\frac{\epsilon d}{2n}\right)^{|E_1(H)|+|E_{\geq 2}'''(H)|}\left(\frac{d}{n}\right)^{|E_{\geq 2}(H)|-|E_{\geq 2}'''(H)|}\cdot\left(\frac{16}{\epsilon}\right)^{|E_1(H)|+|E_{\geq 2}'''(H)|}\\
&\leq \left(\frac{16}{\epsilon}\right)^{st}\cdot \left(\frac{\epsilon d}{2n}\right)^{|E_1(H)|+|E_{\geq 2}'''(H)|}\left(\frac{d}{n}\right)^{|E_{\geq 2}(H)|-|E_{\geq 2}'''(H)|}.
\end{align*}
\end{itemize}
\end{proof}

\begin{lemma}
\label{lem:lem_bound_prob_large_D_wb}
Let $H$ be an arbitrary multigraph with at most $st=sK\log n$ vertices. If $A\geq 1$, then for every $v\in \mathcal{S}_1(H)\cap\mathcal{S}_{\geq 2}(H)$, we have
$$\mathbb{P}\left[\mathbf{D}^H_v + d_1^H(v) + d_{\geq 2}^H(v)>\Delta\middle|\mathbf{x},\mathcal{E}_{wb,H}\right]\leq \frac{1}{2}\cdot\left(\frac{2\epsilon}{33}\right)^{4 A^2 s}.$$
\end{lemma}
\begin{proof}
We have
\begin{align*}
\mathbb{P}&\left[\mathbf{D}^H_v + d_1^H(v) + d_{\geq 2}^H(v)>\Delta\middle|\mathbf{x},\mathcal{E}_{wb,H}\right]\\
&\quad\quad\quad\leq \mathbb{P}\left[\mathbf{D}^H_v + \tau + \frac{\Delta}{4}>\Delta\middle|\mathbf{x},\mathcal{E}_{wb,H}\right]\leq \mathbb{P}\left[\mathbf{D}^H_v>\frac{\Delta}{2}\middle|\mathbf{x},\mathcal{E}_{wb,H}\right]\\
&\quad\quad\quad= \mathbb{P}\left[\mathbf{D}^H_v>\frac{\Delta}{2}\middle|\mathbf{x},\mathcal{E}_{H,b}\cap \mathcal{E}_{H,g}\cap \mathcal{E}_{H,d}\right]\stackrel{(\ast)}{=} \mathbb{P}\left[\mathbf{D}^H_v>\frac{\Delta}{2}\middle|\mathbf{x},\mathcal{E}_{H,b}\right]\\
&\quad\quad\quad\stackrel{(\dagger)}{\leq} \mathbb{P}\left[d_{\mathbf{G}-G(H)}^o(v) >\frac{\Delta}{2}\middle|\mathbf{x},\mathcal{E}_{H,b}\right]\stackrel{(\ddagger)}{\leq} \frac{1}{2As\cdot 2^{6A\tau s^2}}\left(\frac{\epsilon}{6}\right)^{4s^2\tau^2}\stackrel{(\wr)}{\leq} \frac{1}{2\cdot 2^{6A^2 s}}\left(\frac{\epsilon}{6}\right)^{4 A^2 s}\\
&\quad\quad\quad=\frac{1}{2}\cdot\left(\frac{\epsilon}{6\cdot 2^{3/2}}\right)^{4 A^2 s}\leq \frac{1}{2}\cdot\left(\frac{2\epsilon}{33}\right)^{4 A^2 s},
\end{align*}
where $(\ast)$ follows from the fact that given $\mathbf{x}$ and $\mathcal{E}_{H,b}$, the random variable $\mathbf{D}^H_v$ is conditionally independent from $(\mathcal{E}_{H,g},\mathcal{E}_{H,d})$, $(\dagger)$ follows from the fact that $\mathbf{D}_v^H\leq d_{\mathbf{G}-G(H)}^o(v)$, and $(\ddagger)$ follows from  \Cref{lem:lem_bound_prob_large_outside_deg} and the fact that $\mathcal{E}_{H,b}$ is $\sigma(\mathbf{x})$-measurable. $(\wr)$ follows from the fact that $\tau=As \log\frac{6}{\epsilon}\geq A\geq 1$ and the fact that $s\geq 1$.
\end{proof}

\begin{lemma}
\label{lem:lem_pleasant_prob_S_is_empty}
Let $H$ be an $(s,t)$-pleasant multigraph where $t=K\log n$ vertices. Define the set
\begin{align*}
\mathsf{D}_{\text{cycles}}^{H,\text{safe}}=\Big\{\mathsf{d}\in\mathbb{N}^{V(H)}:\; \text{Every cycle of } H&\text{ has at least one }\mathsf{d}\text{-safe edge of multiplicity 1}\Big\}.
\end{align*}
We have:
\begin{align*}
\mathbb{P}\left[\mathbf{D}^H\notin \mathsf{D}_{\text{cycles}}^{H,\text{safe}}\middle|\mathbf{x},\mathcal{E}_{wb,H}\right]\leq 3st\cdot \left(\frac{2\epsilon}{33}\right)^{A st}.
\end{align*}
\end{lemma}
\begin{proof}
Let $C$ be an arbitrary cycle of $H$. Since $C$ is a cycle, it is easy to see that there is a subset $C'\subset E_1(C)$ of size at least $$|C'|\geq\left\lfloor\frac{|E_1(C)|}{2}\right\rfloor\geq\frac{|E_1(C)|}{2}-1,$$ in such a way that no two edges in $C'$ are incident to each other. Let 
\begin{align*}
C''&=\left\{uv\in C':\; d_{\geq 2}^H(u)\leq \frac{\Delta}{4}\text{ and }d_{\geq 2}^H(v)\leq \frac{\Delta}{4}\right\}\\
&=\big\{e\in  C':\; e\text{ is not incident to any vertex in }\mathcal{I}_{\geq 2}(H)\big\}.
\end{align*}

Since there are no edges in $C'$ that are incident to each other, it is easy to see that every vertex in $\mathcal{I}_{\geq 2}(H)$ is incident to at most one vertex in $E_1(C)$. Therefore, $|C''|\geq |C'|-|\mathcal{I}_{\geq 2}(H)|$. Now from \cref{eq:eq_Bound_on_I2_and_L2} and \cref{eq:eq_Delta_form} we have $$|\mathcal{I}_{\geq2}(H)|\leq \frac{8st}{\Delta}\leq\frac{8st}{40Asd}=\frac{t}{5Ad}\leq\frac{t}{5A}.$$ Hence,
$$|C''|\geq |C'|-|\mathcal{I}_{\geq 2}(H)|\geq  \frac{|E_1(C)|}{2}-1-\frac{t}{5A}.$$
Recall that every cycle of an $(s,t)$ pleasant multigraph contains at least $\frac{t}{A}$ edges of multiplicity 1, hence $|E_1(C)|\geq\frac{t}{A}$. Therefore,
\begin{equation}
\label{eq:eq_C_i_double_prime}
|C''|\geq \frac{t}{2A}-1-\frac{t}{5A}\geq \frac{t}{4A}.
\end{equation}

If $C$ does not contain any $\mathbf{D}^H$-safe of multiplicity 1, then for every $uv\in C''\subset E_1(C)$, either $u$ is $\mathbf{D}^H$-unsafe or $v$ is unsafe, i.e., we must have $\mathbf{D}^H_u + d_{\geq 2}^H(u)>\Delta-2$ or $\mathbf{D}^H_v + d_{\geq 2}^H(v)>\Delta-2$. Therefore,

\begin{align*}
&\mathbb{P}\left[\text{$C$ does not contain any $\mathbf{D}^H$-safe of multiplicity 1}\middle|\mathbf{x},\mathcal{E}_{wb,H}\right]\\
&\leq \mathbb{P}\left[\left\{\forall uv\in C'', \mathbf{D}^H_u + d_{\geq 2}^H(u)>\Delta-2\text{ or }\mathbf{D}^H_v + d_{\geq 2}^H(v)>\Delta-2\right\}\middle|\mathbf{x},\mathcal{E}_{wb,H}\right]\\
&\stackrel{(\ast)}{=}\prod_{uv\in C''} \mathbb{P}\left[\left\{\mathbf{D}^H_u + d_{\geq 2}^H(u)>\Delta-2\text{ or }\mathbf{D}^H_v + d_{\geq 2}^H(v)>\Delta-2\right\}\middle|\mathbf{x},\mathcal{E}_{wb,H}\right]\\
&\leq\resizebox{0.9\textwidth}{!}{$\displaystyle\prod_{uv\in C''}\left( \mathbb{P}\left[\mathbf{D}^H_u + d_1^H(u)+d_{\geq 2}^H(u)>\Delta\middle|\mathbf{x},\mathcal{E}_{wb,H}\right] + \mathbb{P}\left[\mathbf{D}^H_v +d_1^H(v)+ d_{\geq 2}^H(v)>\Delta\middle|\mathbf{x},\mathcal{E}_{wb,H}\right]\right)$}\\
&\stackrel{(\dagger)}{=}\prod_{uv\in C''} \left(\frac{1}{2}\cdot\left(\frac{2\epsilon}{33}\right)^{4 A^2 s}+\frac{1}{2}\cdot\left(\frac{2\epsilon}{33}\right)^{4 A^2 s}\right)= \left(\frac{2\epsilon}{33}\right)^{4 A^2 s\cdot|C''|}\stackrel{(\star)}{\leq} \left(\frac{2\epsilon}{33}\right)^{4 A^2 s\cdot\frac{t}{4A}}=\left(\frac{2\epsilon}{33}\right)^{A st},
\end{align*}
where $(\ast)$ follows from the fact that given $\mathbf{x}$ and $\mathcal{E}_{wb,H}$, the random variables $\left(\mathbf{D}_v^H\right)_{v\in V(H)}$ are conditionally mutually independent and from the fact that for every two different edges $u_1v_1,u_2v_2\in C''$, we have\footnote{This follows from the fact that there are no edges in $C''$ that are incident to each other.} $\{u_1,v_1\}\cap\{u_2,v_2\}=\varnothing$. $(\dagger)$ follows from \Cref{lem:lem_bound_prob_large_D_wb}. $(\star)$ follows from \cref{eq:eq_C_i_double_prime}.

Now since $H$ is $(s,t)$-pleasant, there are $r_H$ agreeable components of $H$. Furthermore, each agreeable component contains 1, 2 or 3 cycles depending on whether the agreeable component is of type 1, 2 or 3, respectively. Since the cycles of $H$ are exactly those of its agreeable components, we conclude that there are at most $3r_H\leq 3st$ cycles in $H$. We conclude that
\begin{align*}
\mathbb{P}\left[\mathbf{D}^H\notin \mathsf{D}_{\text{cycles}}^{H,\text{safe}}\middle|\mathbf{x},\mathcal{E}_{wb,H}\right]\leq 3st\cdot \left(\frac{2\epsilon}{33}\right)^{A st}.
\end{align*}
\end{proof}

\begin{lemma}
\label{lem:lem_pleasant_expectation_S_is_empty}
Let $H$ be an $(s,t)$-pleasant multigraph where $t=K\log n$. If $A>\max\left\{1,\frac{100}{K}\right\}$ and $n$ is large enough, then
\begin{align*}
\sum_{\substack{\mathsf{d}\in\mathbb{N}^{V(H)}:\\\mathsf{d}\notin\mathsf{D}_{\text{cycles}}^{H,\text{safe}}}} \left|\mathbb{E}\Big[\mathbb{E}\big[\overline{\mathbf{Y}}_H\big|\mathbf{x},\mathcal{E}_{wb,H},\mathbf{D}^H=\mathsf{d}\big]\cdot P_{1,H}^{wb}(\mathbf{x})\Big]\right|\cdot P_{wb}(\mathsf{d})\leq\frac{1}{\sqrt{n}}\cdot \left(\frac{\epsilon d}{2n}\right)^{|E_1(H)|}\left(\frac{d}{n}\right)^{|E_{\geq 2}(H)|},
\end{align*}
where $P_{wb}(\mathsf{d})=\mathbb{P}\left[D^H=\mathsf{d}\middle|\mathbf{x},\mathcal{E}_{wb,H}\right]$ is as in \Cref{lem:lem_prob_nice_d}.
\end{lemma}
\begin{proof}
From \Cref{lem:lem_technical_lemma_pleasant_wb_cond_d}, we have
\begin{equation}
\label{eq:eq_lem_pleasant_expectation_S_is_empty_1}
\begin{aligned}
\sum_{\substack{\mathsf{d}\in\mathbb{N}^{V(H)}:\\\mathsf{d}\notin\mathsf{D}_{\text{cycles}}^{H,\text{safe}}}}& \left|\mathbb{E}\Big[\mathbb{E}\big[\overline{\mathbf{Y}}_H\big|\mathbf{x},\mathcal{E}_{wb,H},\mathbf{D}^H=\mathsf{d}\big]\cdot P_{1,H}^{wb}(\mathbf{x})\Big]\right|\cdot P_{wb}(\mathsf{d})\\
&\leq\sum_{\substack{\mathsf{d}\in\mathbb{N}^{V(H)}:\\\mathsf{d}\notin\mathsf{D}_{\text{cycles}}^{H,\text{safe}}}}\left(\frac{16}{\epsilon}\right)^{st}\cdot \left(\frac{\epsilon d}{2n}\right)^{|E_1(H)|+|E_{\geq 2}'''(H)|}\left(\frac{d}{n}\right)^{|E_{\geq 2}(H)|-|E_{\geq 2}'''(H)|}\cdot P_{wb}(\mathsf{d})\\
&=\sum_{\substack{\mathsf{d}\in\mathbb{N}^{V(H)}:\\\mathsf{d}\notin\mathsf{D}_{\text{cycles}}^{H,\text{safe}}}}\left(\frac{16}{\epsilon}\right)^{st}\cdot \left(\frac{\epsilon d}{2n}\right)^{|E_1(H)|+|E_{\geq 2}'''(H)|}\left(\frac{d}{n}\right)^{|E_{\geq 2}(H)|-|E_{\geq 2}'''(H)|}\mathbb{P}\left[D^H=\mathsf{d}\middle|\mathbf{x},\mathcal{E}_{wb,H}\right]\\
&=\left(\frac{16}{\epsilon}\right)^{st}\cdot \left(\frac{\epsilon d}{2n}\right)^{|E_1(H)|+|E_{\geq 2}'''(H)|}\left(\frac{d}{n}\right)^{|E_{\geq 2}(H)|-|E_{\geq 2}'''(H)|}\cdot \mathbb{P}\left[\mathbf{D}^H\notin \mathsf{D}_{\text{cycles}}^{H,\text{safe}}\middle|\mathbf{x},\mathcal{E}_{wb,H}\right]\\
&\leq \left(\frac{16}{\epsilon}\right)^{st}\cdot \left(\frac{\epsilon d}{2n}\right)^{|E_1(H)|+|E_{\geq 2}'''(H)|}\left(\frac{d}{n}\right)^{|E_{\geq 2}(H)|-|E_{\geq 2}'''(H)|}\cdot 3st\cdot \left(\frac{2\epsilon}{33}\right)^{A st},
\end{aligned}
\end{equation}
where the last inequality follows from \Cref{lem:lem_pleasant_prob_S_is_empty}. Now notice that

\begin{align*}
\left(\frac{16}{\epsilon}\right)^{st}\cdot 3st\cdot \left(\frac{2\epsilon}{33}\right)^{A st} &\stackrel{(\ast)}\leq \left(\frac{16}{\epsilon}\right)^{Ast}\cdot 3st\cdot \left(\frac{2\epsilon}{33}\right)^{A st}= 3st\cdot\left(\frac{32}{33}\right)^{Ast}\leq 3st\cdot\left(\frac{32}{33}\right)^{At}\\
&=3st\cdot\left(\frac{32}{33}\right)^{A K\log n}= \frac{3st}{n ^{A K\log \frac{33}{32}}}\leq \frac{3st}{n ^{A K\cdot\frac{1}{100}}}\stackrel{(\dagger)}\leq \frac{1}{\sqrt{n}},
\end{align*}

where $(\ast)$ and $(\dagger)$ are true if $A>\max\left\{1,\frac{100}{K}\right\}$ and $n$ is large enough. By combining this with \cref{eq:eq_lem_pleasant_expectation_S_is_empty_1}, we get the lemma.
\end{proof}

Now we are ready to prove \Cref{lem:lem_main_lemma_pleasant_wb}.

\begin{proof}[Proof of \Cref{lem:lem_main_lemma_pleasant_wb}]
From \Cref{lem:lem_nice_L_definition}, we have
\begin{align*}
\mathbb{E}\Big[\mathbb{E}\big[\overline{\mathbf{Y}}_H\big|\mathbf{x},\mathcal{E}_{wb,H}\big]\cdot P_{1,H}^{wb}(\mathbf{x})\Big]&=\sum_{\substack{\mathsf{d}\in\mathbb{N}^{V(H)}:\\L^H(\mathsf{d})=1}}\mathbb{E}\Big[\mathbb{E}\big[\overline{\mathbf{Y}}_H\big|\mathbf{x},\mathcal{E}_{wb,H},\mathbf{D}^H=\mathsf{d}\big]\cdot P_{1,H}^{wb}(\mathbf{x})\Big]\cdot P_{wb}(\mathsf{d}).
\end{align*}

Now for every $\mathsf{d}\in\mathbb{N}^{V(H)}$, define
\begin{align*}
L_s^H(\mathsf{d})&= \mathbbm{1}_{\{\forall v\in V(H),\;\mathsf{d}_v+ d_1^H(v)+d_{\geq 2}^H(v)\leq \Delta\}}.
\end{align*}

Note that $L_s^H(\mathsf{d})=1$ if and only if $S_{\mathsf{d}}^H=E_1(H)$. Note also that if $L_s^H(\mathsf{d})=1$, then we must have $\mathsf{d}\in\mathsf{D}_{\text{cycles}}^{H,\text{safe}}$. Therefore,

\begin{equation}
\label{eq:eq_lem_main_lemma_pleasant_wb_1}
\begin{aligned}
\mathbb{E}\Big[\mathbb{E}\big[\overline{\mathbf{Y}}_H\big|\mathbf{x},\mathcal{E}_{wb,H}\big]\cdot P_{1,H}^{wb}(\mathbf{x})\Big]&=\sum_{\substack{\mathsf{d}\in\mathbb{N}^{V(H)}:\\L_s^H(\mathsf{d})=1}}\mathbb{E}\Big[\mathbb{E}\big[\overline{\mathbf{Y}}_H\big|\mathbf{x},\mathcal{E}_{wb,H},\mathbf{D}^H=\mathsf{d}\big]\cdot P_{1,H}^{wb}(\mathbf{x})\Big]\cdot P_{wb}(\mathsf{d})\\
&\quad+\sum_{\substack{\mathsf{d}\in\mathsf{D}_{\text{cycles}}^{H,\text{safe}}:\\L_s^H(\mathsf{d})=0}}\mathbb{E}\Big[\mathbb{E}\big[\overline{\mathbf{Y}}_H\big|\mathbf{x},\mathcal{E}_{wb,H},\mathbf{D}^H=\mathsf{d}\big]\cdot P_{1,H}^{wb}(\mathbf{x})\Big]\cdot P_{wb}(\mathsf{d})\\
&\quad+\sum_{\substack{\mathsf{d}\in\mathbb{N}^{V(H)}:\\\mathsf{d}\notin\mathsf{D}_{\text{cycles}}^{H,\text{safe}}}}\mathbb{E}\Big[\mathbb{E}\big[\overline{\mathbf{Y}}_H\big|\mathbf{x},\mathcal{E}_{wb,H},\mathbf{D}^H=\mathsf{d}\big]\cdot P_{1,H}^{wb}(\mathbf{x})\Big]\cdot P_{wb}(\mathsf{d}).
\end{aligned}
\end{equation}

If $L_s^H(\mathsf{d})=1$, then $S_{\mathsf{d}}=E_1(H)$ and \Cref{lem:lem_technical_lemma_pleasant_wb_cond_d} implies that
\begin{equation}
\label{eq:eq_lem_main_lemma_pleasant_wb_2}
\begin{aligned}
&\left(1-\frac{1}{\sqrt{n}}\right)\cdot\left(\frac{\epsilon d}{2n}\right)^{|E_1(H)|+|E_{\geq 2}'''(H)|}\left(\frac{d}{n}\right)^{|E_{\geq 2}(H)|-|E_{\geq 2}'''(H)|}\\
&\;\;\leq \mathbb{E}\Big[\mathbb{E}\big[\overline{\mathbf{Y}}_H\big|\mathbf{x},\mathcal{E}_{wb,H},\mathbf{D}^H=\mathsf{d}\big]\cdot P_{1,H}^{wb}(\mathbf{x})\Big]\leq \left(\frac{\epsilon d}{2n}\right)^{|E_1(H)|+|E_{\geq 2}'''(H)|}\left(\frac{d}{n}\right)^{|E_{\geq 2}(H)|-|E_{\geq 2}'''(H)|}.
\end{aligned}
\end{equation}

On the other hand, if $\mathsf{d}\in\mathsf{D}_{\text{cycles}}^{H,\text{safe}}$ and $L_s^H(\mathsf{d})=0$, then \Cref{lem:lem_technical_lemma_pleasant_wb_cond_d} implies that for $n$ large enough, we have
\begin{equation}
\label{eq:eq_lem_main_lemma_pleasant_wb_3}
\begin{aligned}
\resizebox{0.9\textwidth}{!}{$\displaystyle 0\leq\mathbb{E}\Big[\mathbb{E}\big[\overline{\mathbf{Y}}_H\big|\mathbf{x},\mathcal{E}_{wb,H},\mathbf{D}^H=\mathsf{d}\big]\cdot P_{1,H}^{wb}(\mathbf{x})\Big]\leq \frac{1}{\sqrt{n}}\cdot\left(\frac{\epsilon d}{2n}\right)^{|E_1(H)|+|E_{\geq 2}'''(H)|}\left(\frac{d}{n}\right)^{|E_{\geq 2}(H)|-|E_{\geq 2}'''(H)|}.$}
\end{aligned}
\end{equation}

By combining \cref{eq:eq_lem_main_lemma_pleasant_wb_1}, \cref{eq:eq_lem_main_lemma_pleasant_wb_2},\cref{eq:eq_lem_main_lemma_pleasant_wb_3} and \Cref{lem:lem_pleasant_expectation_S_is_empty}, we get

\begin{align*}
&\mathbb{E}\Big[\mathbb{E}\big[\overline{\mathbf{Y}}_H\big|\mathbf{x},\mathcal{E}_{wb,H}\big]\cdot P_{1,H}^{wb}(\mathbf{x})\Big]\\
&\leq \sum_{\substack{\mathsf{d}\in\mathbb{N}^{V(H)}:\\L_s^H(\mathsf{d})=1}}\left(\frac{\epsilon d}{2n}\right)^{|E_1(H)|+|E_{\geq 2}'''(H)|}\left(\frac{d}{n}\right)^{|E_{\geq 2}(H)|-|E_{\geq 2}'''(H)|}\cdot P_{wb}(\mathsf{d})\\
&\quad\quad+\sum_{\substack{\mathsf{d}\in\mathsf{D}_{\text{cycles}}^{H,\text{safe}}:\\L_s^H(\mathsf{d})=0}}\frac{1}{\sqrt{n}}\cdot\left(\frac{\epsilon d}{2n}\right)^{|E_1(H)|+|E_{\geq 2}'''(H)|}\left(\frac{d}{n}\right)^{|E_{\geq 2}(H)|-|E_{\geq 2}'''(H)|}\cdot P_{wb}(\mathsf{d})\\
&\quad\quad+\frac{1}{\sqrt{n}}\cdot\left(\frac{\epsilon d}{2n}\right)^{|E_1(H)|+|E_{\geq 2}'''(H)|}\left(\frac{d}{n}\right)^{|E_{\geq 2}(H)|-|E_{\geq 2}'''(H)|}.
\end{align*}

Now from \Cref{lem:lem_prob_nice_d} we have $P_{wb}(\mathsf{d})=\mathbb{P}\big[\mathbf{D}^H=\mathsf{d}\big|\mathcal{E}_{H,b}\big]$. Therefore,
\begin{align*}
&\mathbb{E}\Big[\mathbb{E}\big[\overline{\mathbf{Y}}_H\big|\mathbf{x},\mathcal{E}_{H,b}\big]\cdot P_{1,H}^{wb}(\mathbf{x})\Big]\\
&\leq \left(\frac{\epsilon d}{2n}\right)^{|E_1(H)|+|E_{\geq 2}'''(H)|}\left(\frac{d}{n}\right)^{|E_{\geq 2}(H)|-|E_{\geq 2}'''(H)|}\cdot \mathbb{P}\big[L_s^H(\mathbf{D}^H)=1\big|\mathcal{E}_{H,b}\big]\\
&\quad\quad+\resizebox{0.85\textwidth}{!}{$\displaystyle\frac{1}{\sqrt{n}}\cdot\left(\frac{\epsilon d}{2n}\right)^{|E_1(H)|+|E_{\geq 2}'''(H)|}\left(\frac{d}{n}\right)^{|E_{\geq 2}(H)|-|E_{\geq 2}'''(H)|}\cdot \mathbb{P}\big[\mathbf{D}^H\in \mathsf{D}_{\text{cycles}}^{H,\text{safe}}\text{ and }L_s^H(\mathbf{D}^H)=0\big|\mathcal{E}_{H,b}\big]$}\\
&\quad\quad+\frac{1}{\sqrt{n}}\cdot\left(\frac{\epsilon d}{2n}\right)^{|E_1(H)|+|E_{\geq 2}'''(H)|}\left(\frac{d}{n}\right)^{|E_{\geq 2}(H)|-|E_{\geq 2}'''(H)|}\\
&\leq \left(\mathbb{P}\big[L_s^H(\mathbf{D}^H)=1\big|\mathcal{E}_{H,b}\big]+\frac{2}{\sqrt{n}}\right) \left(\frac{\epsilon d}{2n}\right)^{|E_1(H)|+|E_{\geq 2}'''(H)|}\left(\frac{d}{n}\right)^{|E_{\geq 2}(H)|-|E_{\geq 2}'''(H)|}.
\end{align*}

Now recall from \Cref{lem:lem_prob_nice_d} that $\mathbb{P}\big[\mathbf{D}^H=\mathsf{d}\big|\mathcal{E}_{H,b}\big]=\mathbb{P}\big[\mathbf{D}^H=\mathsf{d}\big|\mathcal{E}_{H,b}\big]$. Therefore,
\begin{equation}
\label{eq:eq_lem_main_lemma_pleasant_wb_4}
\begin{aligned}
\mathbb{P}\big[L_s^H(\mathbf{D}^H)=1\big|\mathcal{E}_{H,b}\big]&=\mathbb{P}\big[\big\{\forall v\in V(H),\;\mathbf{D}^H_v+ d_1^H(v)+d_{\geq 2}^H(v)\leq \Delta\big\}\big|\mathcal{E}_{H,b}\big]\\
&=P_s^H.
\end{aligned}
\end{equation}

Similarly, from \cref{eq:eq_lem_main_lemma_pleasant_wb_1}, \cref{eq:eq_lem_main_lemma_pleasant_wb_2},\cref{eq:eq_lem_main_lemma_pleasant_wb_3}, \cref{eq:eq_lem_main_lemma_pleasant_wb_4} and \Cref{lem:lem_pleasant_expectation_S_is_empty}, we have

\begin{align*}
&\mathbb{E}\Big[\mathbb{E}\big[\overline{\mathbf{Y}}_H\big|\mathbf{x},\mathcal{E}_{wb,H}\big]\cdot P_{1,H}^{wb}(\mathbf{x})\Big]\\
&\geq \left(\sum_{\substack{\mathsf{d}\in\mathbb{N}^{V(H)}:\\L_s^H(\mathsf{d})=1}}\left(1-\frac{1}{\sqrt{n}}\right)\left(\frac{\epsilon d}{2n}\right)^{|E_1(H)|+|E_{\geq 2}'''(H)|}\left(\frac{d}{n}\right)^{|E_{\geq 2}(H)|-|E_{\geq 2}'''(H)|}\cdot P_{wb}(\mathsf{d})\right)\\
&\quad\quad\quad\quad\quad\quad\quad-\frac{1}{\sqrt{n}}\cdot\left(\frac{\epsilon d}{2n}\right)^{|E_1(H)|+|E_{\geq 2}'''(H)|}\left(\frac{d}{n}\right)^{|E_{\geq 2}(H)|-|E_{\geq 2}'''(H)|}\\
&=\left[\left(1-\frac{1}{\sqrt{n}}\right)\cdot\mathbb{P}\big[L_s^H(\mathbf{D}^H)=1\big|\mathcal{E}_{H,b}\big]-\frac{1}{\sqrt{n}}\right]\cdot\left(\frac{\epsilon d}{2n}\right)^{|E_1(H)|+|E_{\geq 2}'''(H)|}\left(\frac{d}{n}\right)^{|E_{\geq 2}(H)|-|E_{\geq 2}'''(H)|}\\
&\geq\left(P_s^H-\frac{2}{\sqrt{n}}\right)\cdot\left(\frac{\epsilon d}{2n}\right)^{|E_1(H)|+|E_{\geq 2}'''(H)|}\left(\frac{d}{n}\right)^{|E_{\geq 2}(H)|-|E_{\geq 2}'''(H)|}.
\end{align*}
\end{proof}

\subsubsection{Probability of safety in pleasant multigraphs}
\label{subsec:subsec_prob_safety_pleasant}

\begin{proof}[Proof of \Cref{lem:lem_Prob_safety_pleasant_wb}]
Let $H$ be an $(s,t)$-pleasant multigraph with $t=K\cdot\log n$. From \Cref{lem:lem_prob_nice_d}, we know that given $\mathcal{E}_{H,b}$, the random variables $(\mathbf{D}_v^H)_{v\in V(H)}$ are conditionally mutually independent, hence we can rewrite $P_s^H$ as follows:
\begin{align*}
P_s^H
&= \mathbb{P}\big[\big\{\forall v\in V(H),\;\mathbf{D}^H_v+ d_1^H(v)+d_{\geq 2}^H(v)\leq \Delta\big\}\big|\mathcal{E}_{H,b}\big]\\
&= \prod_{v\in V(H)}\mathbb{P}\big[\mathbf{D}^H_v+ d_1^H(v)+d_{\geq 2}^H(v)\leq \Delta\big|\mathcal{E}_{H,b}\big].
\end{align*}

Now for every $v\in V(H)$, we have:
\begin{align*}
\mathbb{P}\big[\mathbf{D}^H_v+ d_1^H(v)+d_{\geq 2}^H(v)\leq \Delta\big|\mathcal{E}_{H,b}\big]
&=\mathbb{P}\big[\mathbf{D}^H_v\leq \Delta -d_1^H(v)-d_{\geq 2}^H(v)\big|\mathcal{E}_{H,b}\big]\\
&=\sum_{\substack{\mathsf{d}_v\in\mathbb{N}:\\ \mathsf{d}_v\leq \Delta- d_1^H(v)-d_{\geq 2}^H(v)}}\mathbb{P}\big[\mathbf{D}^H_v =\mathsf{d}_v\big|\mathcal{E}_{H,b}\big]\\
&=P_s\big(\Delta- d_1^H(v)-d_{\geq 2}^H(v)\big),
\end{align*}
where $P_s(\ell)$ is defined for every $\ell\in\mathbb{Z}$ as
$$P_s(\ell)=\sum_{\substack{\mathsf{d}_v\in\mathbb{N}:\\ \mathsf{d}_v\leq \mathsf{d}}}P_{wb}(\mathsf{d}_v),$$
and $P_{wb}(\mathsf{d}_v)$ is as in \Cref{lem:lem_prob_nice_d}. Clearly, $P_s$ is a non-decreasing function.

If $\ell<0$, then by picking an arbitrary $v\in V(H)$, we get
$$P_s(\ell)=\mathbb{P}\big[\mathbf{D}^H_v\leq \ell\big|\mathcal{E}_{H,b}\big]=0.$$
On the other hand, if $\ell\geq 0$, we have
\begin{equation}
\label{eq:eq_lem_Prob_safety_pleasant_wb_1}
\begin{aligned}
P_s(\ell)&=\mathbb{P}\big[\mathbf{D}^H_v\leq \ell\big|\mathcal{E}_{H,b}\big]\stackrel{(\ast)}= \mathbb{P}\big[\mathbf{D}^H_v\leq \ell\big|\mathbf{x},\mathcal{E}_{H,b}\big]\\
&\stackrel{(\dagger)}{\geq} \mathbb{P}\big[d^o_{\mathbf{G}-G(H)}(v)\leq \ell\big|\mathbf{x},\mathcal{E}_{H,b}\big]\stackrel{(\ddagger)}= \mathbb{P}\big[d^o_{\mathbf{G}-G(H)}(v)\leq \ell\big|\mathbf{x}\big],
\end{aligned}
\end{equation}
where $(\ast)$ follows from the fact that given $\mathcal{E}_{wb,H}$, the random variable $\mathbf{D}^H$ is conditionally independent from $\mathbf{x}$ (see \Cref{lem:lem_prob_nice_d}). $(\dagger)$ follows from the fact that $\mathbf{D}_v^H\leq d_{\mathbf{G}-G(H)}^o(v)$ for every $v\in V(H)$. $(\ddagger)$ follows from the fact that the event $\mathcal{E}_{H,b}$ is $\sigma(\mathbf{x})$-measurable. For every $\ell\geq 0$, we can further lower bound $P_s(\ell)$ as follows:
\begin{align*}
P_s(\ell)&\geq \mathbb{P}\big[d^o_{\mathbf{G}-G(H)}(v)\leq \ell\big|\mathbf{x}\big]\geq \mathbb{P}\left[d^o_{\mathbf{G}-G(H)}(v)=0\middle|\mathbf{x}\right]=\prod_{u\in [n]\setminus V(H)}\mathbb{P}\big[uv\notin \mathbf{G}\big|\mathbf{x}\big]\\
&=\prod_{u\in [n]\setminus V(H)}\left[1-\left(1+\frac{\epsilon\mathbf{x}_u\mathbf{x}_v}{2}\right)\frac{d}{n}\right]\geq \left(1-\frac{2d}{n}\right)^{n-|V(H)|}\geq \left(1-\frac{2d}{n}\right)^n\\
&= \left(\frac{1}{1+\frac{2d}{n-2d}}\right)^n\stackrel{(\star)}{\geq} \left(\frac{1}{1+\frac{4d}{n}}\right)^n \geq e^{-4d},
\end{align*}
where $(\star)$ is true for $n$ large enough.

Now if $\ell\geq\frac{\Delta}{4}$, we get from \cref{eq:eq_lem_Prob_safety_pleasant_wb_1} and \Cref{lem:lem_bound_prob_large_outside_deg} that
$$P_s(\ell)\geq \mathbb{P}\big[d^o_{\mathbf{G}-G(H)}\leq \ell\big|\mathbf{x}\big]\geq \mathbb{P}\left[d^o_{\mathbf{G}-G(H)}\leq \frac{\Delta}{4}\middle|\mathbf{x}\right]\geq 1-\frac{\eta}{2}.$$
\end{proof}

\begin{proof}[Proof of \Cref{lem:lem_Prob_nice_safe_lower_bound}]
From \Cref{lem:lem_Prob_safety_pleasant_wb}, we have
\begin{align*}
P_s^H&=\prod_{v\in V(H)}P_s\big(\Delta- d_1^H(v)-d_{\geq 2}^H(v)\big).
\end{align*}

If there exists $v\in V(H)$ such that $d_{1}^H(v)+d_{\geq 2}^H(v)> \Delta$, then $P_s\big(\Delta- d_1^H(v)-d_{\geq 2}^H(v)\big)=0$ and $P_s^H=0$.

Now assume that $d_{1}^H(v)+d_{\geq 2}^H(v)\leq\Delta$ for every $v\in V(H)$. Since $H$ is $(s,t)$-pleasant, we have $\mathcal{L}_{\geq 2}(H)=\varnothing$, so $V(H)=\mathcal{S}_{\geq2}(H)\cup\mathcal{I}_{\geq2}(H)$.  Notice the following:
\begin{itemize}
\item If $v\in\mathcal{S}_{\geq 2}(H)$, then $d_{\geq 2}^H(v)\leq \frac{\Delta}{4}$. Now since $H$ is $(s,t)$-pleasant, we have $d_1^H(v)\leq 4$, and so $$\Delta- d_1^H(v)-d_{\geq 2}^H(v)\geq \Delta-4-\frac{\Delta}{4}\geq\frac{\Delta}{2}.$$
\Cref{lem:lem_Prob_safety_pleasant_wb} and \Cref{lem:lem_bound_prob_large_outside_deg} now imply that
\begin{equation}
\label{eq:eq_lem_Prob_nice_safe_lower_bound_1}
\begin{aligned}
P_s\big(\Delta- d_1^H(v)-d_{\geq 2}^H(v)\big)&\geq 1-\frac{\eta}{2}\geq 1-\frac{1}{2As}\left(\frac{\epsilon}{6}\right)^{\tau}\\
&\geq 1-\frac{1}{2As}=\frac{1}{1+\frac{1}{2As-1}}\stackrel{(\ast)}{\geq} \frac{1}{1+\frac{1}{As}},
\end{aligned}
\end{equation}
where $(\ast)$ is true if $A>1$.
\item If $v\in\mathcal{I}_{\geq 2}(H)$, we will use the assumption $\Delta- d_1^H(v)-d_{\geq 2}^H(v)\geq 0$. \Cref{lem:lem_Prob_safety_pleasant_wb} now implies that
\begin{equation}
\label{eq:eq_lem_Prob_nice_safe_lower_bound_2}
\begin{aligned}
P_s\big(\Delta- d_1^H(v)-d_{\geq 2}^H(v)\big)\geq e^{-4d}.
\end{aligned}
\end{equation}
\end{itemize}
Now From \Cref{lem:lem_Prob_safety_pleasant_wb}, we have
\begin{align*}
P_s^H&=\prod_{v\in V(H)}P_s\big(\Delta- d_1^H(v)-d_{\geq 2}^H(v)\big)\\
&=\left(\prod_{v\in \mathcal{S}_{\geq2}(H)}P_s\big(\Delta- d_1^H(v)-d_{\geq 2}^H(v)\big)\right)\cdot \left(\prod_{v\in \mathcal{I}_{\geq2}(H)}P_s\big(\Delta- d_1^H(v)-d_{\geq 2}^H(v)\big)\right)\\
&\stackrel{(\dagger)}{\geq} \frac{1}{\left(1+\frac{1}{As}\right)^{|\mathcal{S}_{\geq2}(H)|}}\cdot e^{-4d\cdot|\mathcal{I}_{\geq2}(H)|}\stackrel{(\ddagger)}{\geq} \frac{1}{\left(1+\frac{1}{As}\right)^{st}} \cdot e^{-4d\cdot\frac{t}{4Ad}}\geq \frac{1}{e^{\frac{t}{A}}}\cdot e^{-\frac{t}{A}}\\
&=\frac{1}{e^{\frac{2t}{A}}}=\frac{1}{e^{\frac{2K\log n}{A}}}=\frac{1}{n^{\frac{2K}{A}}},
\end{align*}
where $(\dagger)$ follows from \cref{eq:eq_lem_Prob_nice_safe_lower_bound_1} and \cref{eq:eq_lem_Prob_nice_safe_lower_bound_2}. $(\ddagger)$ follows from \cref{eq:eq_Bound_on_I2_and_L2} and \cref{eq:eq_Delta_form}, which imply that $|\mathcal{I}_{\geq2}(H)|\leq \frac{8st}{\Delta}\leq\frac{8st}{40Asd}\leq \frac{t}{4Ad}$.
\end{proof}

\subsection{Proofs of technical lemmas for the centered matrix}

\subsubsection{Analyzing walks of multiplicity 1}
\label{subsubsec:subsubsec_centered_walks_mult_1}

In order to prove \Cref{lem:lem_X_hat_mult_1}, we need two lemmas. The following lemma shows that it is unlikely that the set of reassuring walks is completely walk-unsafe.

\begin{lemma}
\label{lem:lem_bound_prob_walks_completely_unsafe}
If $A>\max\{100K,1\}$ and $n$ is large enough, then given $\mathbf{x}$, the conditional probability that $\mathcal{W}_{1r}(H)$ is completely walk-unsafe can be upper bounded by:
$$\mathbb{P}\left[\left\{\mathcal{W}_{1r}(H)\text{ is completely walk-unsafe}\right\}\middle|\mathbf{x}\right]\leq (s\eta)^{\frac{1}{2s\tau}|\mathcal{W}_{1r}(H)|},$$
where $\eta$ is as in \cref{eq:eq_def_eta}.
\end{lemma}
\begin{proof}
Since every $W\in \mathcal{W}_{1r}(H)$ contains $s+1$ vertices, and since $d^H_{1}(v)\leq\tau$ for every vertex $v\in V(W)$, we can see that $W\in\mathcal{W}_{1r}(H)$ intersects at most $(s+1)(\tau-1)$ others walks in $\mathcal{W}_{1r}(H)$. Therefore, we can find a subset $\mathcal{W}\subseteq\mathcal{W}_{1r}(H)$ such that:
\begin{itemize}
\item For every $W_1,W_2\in\mathcal{W}$, we have $V(W_1)\cap V(W_2)=\varnothing$.
\item $|\mathcal{W}|\geq \frac{1}{(s+1)(\tau-1)+1}|\mathcal{W}_{1r}(H)|\geq \frac{1}{2s\tau}|\mathcal{W}_{1r}(H)|$.
\end{itemize}

Now if $\mathcal{W}_{1r}(H)$ is completely walk-unsafe, then $\mathcal{W}$ is completely walk-unsafe, and so every walk $W\in\mathcal{W}$ is walk-unsafe. Therefore, for every $W\in\mathcal{W}$, the set $V(W)$ is not completely safe, which means that there exists at least one vertex $v_W\in V(W)$ which is unsafe. Now since $V(W_1)\cap V(W_2)=\varnothing$ for every $W_1,W_2\in\mathcal{W}$, we have $v_{W_1}\neq v_{W_2}$ for every $W_1,W_2\in\mathcal{W}$. Therefore, the set $V_{\mathcal{W}}=\big\{v_W:\; W\in\mathcal{W}\big\}\subseteq \bigcup_{W\in\mathcal{W}_{1r}}V(W)$ satisfies the following:
\begin{itemize}
\item $|V_{\mathcal{W}}|=|\mathcal{W}|\geq \frac{1}{2s\tau}|\mathcal{W}_{1r}(H)|$.
\item $|V_{\mathcal{W}}\cap V(W)|=1$ for every $W\in\mathcal{W}$.
\item $V_{\mathcal{W}}$ is completely unsafe.
\end{itemize}

Now since $V(W)\subseteq \mathcal{S}_1(H)\cap\mathcal{S}_{\geq 2}(H)$ for every $W\in\mathcal{W}_{1r}(H)$, it follows from \Cref{lem:lem_bound_prob_completely_unsafe} that for every  $\displaystyle V\subseteq \bigcup_{W\in\mathcal{W}_{1r}}V(W)$ satisfying $|V|=|\mathcal{W}|$, we have
$$\mathbb{P}\left[\left\{V\text{ is completely unsafe}\right\}\middle|\mathbf{x}\right]\leq \eta^{|V|}=\eta^{|\mathcal{W}|}.$$

On the other hand, since there are $s^{|\mathcal{W}|}$ subsets $V\subseteq \bigcup_{W\in\mathcal{W}_{1r}}V(W)$ which satisfy $|V|=|\mathcal{W}|$ and $|V_{\mathcal{W}}\cap V(W)|=1$ for every $W\in\mathcal{W}$, we conclude that
\begin{align*}
\mathbb{P}\left[\left\{\mathcal{W}_{1r}(H)\text{ is completely walk-unsafe}\right\}\middle|\mathbf{x}\right]&\leq s^{|\mathcal{W}|}\cdot\eta^{|\mathcal{W}|}= (s\eta)^{|\mathcal{W}|}\leq  (s\eta)^{\frac{1}{2s\tau}|\mathcal{W}_{1r}(H)|},
\end{align*}
where the last inequality follows from the fact that if $A>\max\{100K,1\}$ then $s\eta<1$ (see the definition of $\eta$ in \cref{eq:eq_def_eta}).
\end{proof}

The following lemma shows that in the event that there is one walk of multiplicity 1 that is walk-safe, the conditional expectation of $\hat{\mathbf{Y}}_H$ given this event and given $\mathbf{x}$ will be zero. This can be seen as the truncated version of \cref{eq:eq_walk_mult_1_non_truncated}.

\begin{lemma}
\label{lem:lem_expectation_safe_walk}
Let $W\in \mathcal{W}_1(H)$, and let $\mathcal{E}$ be an event satisfying:
\begin{itemize}
\item $\mathcal{E}$ implies that $W$ is $(\mathbf{G},H)$-walk-safe, i.e., $\forall G\in\mathcal{E}$, the walk $W$ is $(G,H)$-walk-safe.
\item Given $\mathbf{x}$ and $\mathcal{E}$, the random variable $(\mathbbm{1}_{\{uv\in\mathbf{G}\}})_{uv\in E(W)}$ is conditionally independent from $(\overline{\mathbf{Y}}_{u'v'})_{u'v'\in E(H)\setminus E(W)}$.
\item Given $\mathbf{x}$, the event $\mathcal{E}$ is conditionally independent of $(\mathbbm{1}_{\{uv\in\mathbf{G}\}})_{uv\in E(W)}$.
\end{itemize}
We have
 \begin{equation*}
\mathbb{E}\big[\hat{\mathbf{Y}}_H\big|\mathbf{x},\mathcal{E}\big]=0.
\end{equation*}
\end{lemma}
\begin{proof}
We have
\begin{align*}
\mathbb{E}\big[\hat{\mathbf{Y}}_H\big|\mathbf{x},\mathcal{E}\big]&=\mathbb{E}\left[\left(\overline{\mathbf{Y}}_W-\left(\frac{\epsilon d}{2n}\right)^s \mathbf{x}_W\right)\cdot \prod_{W'\in\mathcal{W}(H)\setminus\{W\}} \left(\overline{\mathbf{Y}}_{W'}-\left(\frac{\epsilon d}{2n}\right)^s \mathbf{x}_{W'}\right)\middle|\mathbf{x},\mathcal{E}\right]\\
&\stackrel{(\ast)}{=}\mathbb{E}\left[\left(\mathbf{Y}_W-\left(\frac{\epsilon d}{2n}\right)^s \mathbf{x}_W\right)\cdot \prod_{W'\in\mathcal{W}(H)\setminus\{W\}} \left(\overline{\mathbf{Y}}_{W'}-\left(\frac{\epsilon d}{2n}\right)^s \mathbf{x}_{W'}\right)\middle|\mathbf{x},\mathcal{E}\right]\\
&\stackrel{(\dagger)}{=}\mathbb{E}\left[\left(\mathbf{Y}_W-\left(\frac{\epsilon d}{2n}\right)^s \mathbf{x}_W\right)\middle|\mathbf{x},\mathcal{E}\right]\cdot \mathbb{E}\left[\prod_{W'\in\mathcal{W}(H)\setminus\{W\}} \left(\overline{\mathbf{Y}}_{W'}-\left(\frac{\epsilon d}{2n}\right)^s \mathbf{x}_{W'}\right)\middle|\mathbf{x},\mathcal{E}\right]\\
&\stackrel{(\ddagger)}{=}\mathbb{E}\left[\left(\mathbf{Y}_W-\left(\frac{\epsilon d}{2n}\right)^s \mathbf{x}_W\right)\middle|\mathbf{x}\right]\cdot \mathbb{E}\left[\prod_{W'\in\mathcal{W}(H)\setminus\{W\}} \left(\overline{\mathbf{Y}}_{W'}-\left(\frac{\epsilon d}{2n}\right)^s \mathbf{x}_{W'}\right)\middle|\mathbf{x},\mathcal{E}\right]\\
&\stackrel{(\wr)}{=}0,\\
\end{align*}
where $(\ast)$ follows from the fact that if $W$ is walk-safe, then $\overline{\mathbf{Y}}_W=\mathbf{Y}_W$, $(\dagger)$ follows from the fact that given $\mathbf{x}$ and $\mathcal{E}$, the random variable $(\mathbbm{1}_{\{uv\in\mathbf{G}\}})_{uv\in E(W)}$ is conditionally independent from $(\overline{\mathbf{Y}}_{u'v'})_{u'v'\in E(H)\setminus E(W)}$, $(\ddagger)$ follows from the fact that given $\mathbf{x}$, the event $\mathcal{E}$ is conditionally independent of $(\mathbbm{1}_{\{uv\in\mathbf{G}\}})_{uv\in E(W)}$, and $(\wr)$ follows from \Cref{lem:lem_expectation_walk_multiplicity_1}.
\end{proof}

Now we are ready to prove \Cref{lem:lem_X_hat_mult_1}.

\begin{proof}[Proof of \Cref{lem:lem_X_hat_mult_1}]
For every $\mathcal{W}\subseteq\mathcal{W}_{1r}(H)$, define the following events:
$$\mathcal{E}_{\mathcal{W},H,c\text{-}ws}=\big\{\mathcal{W}\text{ is completely }(\mathbf{G},H)\text{-walk-safe}\big\},$$
and
$$\mathcal{E}_{\mathcal{W},H,c\text{-}wus}=\big\{\mathcal{W}\text{ is completely }(\mathbf{G},H)\text{-walk-unsafe}\big\}.$$

We have:
\begin{align*}
&\mathbb{E}\big[\hat{\mathbf{Y}}_{H}\big|\mathbf{x}\big]\\
&=\sum_{\mathcal{W}\subseteq\mathcal{W}_{1r}(H)}\mathbb{E}\big[\hat{\mathbf{Y}}_{H}\big|\mathbf{x},\mathcal{E}_{\mathcal{W},H,c\text{-}ws}\cap \mathcal{E}_{\mathcal{W}_{1r}(H)\setminus\mathcal{W},H,c\text{-}wus}\big]\cdot \mathbb{P}\left[\mathcal{E}_{\mathcal{W},H,c\text{-}ws}\cap \mathcal{E}_{\mathcal{W}_{1r}(H)\setminus\mathcal{W},H,c\text{-}wus}\right].
\end{align*}

Now from \Cref{lem:lem_expectation_safe_walk}, we know that for every $\mathcal{W}\subseteq\mathcal{W}_{1r}(H)$ satisfying $\mathcal{W}\neq\varnothing$, we have
$$\mathbb{E}\big[\hat{\mathbf{Y}}_{H}\big|\mathbf{x},\mathcal{E}_{\mathcal{W},H,c\text{-}ws}\cap \mathcal{E}_{\mathcal{W}_{1r}(H)\setminus\mathcal{W},H,c\text{-}wus}\big]=0.$$
Therefore,
\begin{equation}
\label{eq:eq_lem_X_hat_mult_1_1}
\mathbb{E}\big[\hat{\mathbf{Y}}_{H}\big|\mathbf{x}\big]=\mathbb{E}\big[\hat{\mathbf{Y}}_{H}\big|\mathbf{x},\mathcal{E}_{\mathcal{W}_{1r}(H),H,c\text{-}wus}\big]\cdot \mathbb{P}\left[\mathcal{E}_{\mathcal{W}_{1r}(H),H,c\text{-}wus}\right].
\end{equation}

Now from \Cref{lem:lem_bound_prob_walks_completely_unsafe}, we have
\begin{equation}
\label{eq:eq_lem_X_hat_mult_1_2}
\mathbb{P}\left[\mathcal{E}_{\mathcal{W}_{1r}(H),H,c\text{-}wus}\right]\leq (s\eta)^{\frac{1}{2s\tau}|\mathcal{W}_{1r}(H)|}=\left[(s\eta)^{\frac{1}{2s\tau}}\right]^{|\mathcal{W}_{1r}(H)|}.
\end{equation}

On the other hand, we have
\begin{align*}
\mathbb{E}\big[\hat{\mathbf{Y}}_{H}\big|&\mathbf{x},\mathcal{E}_{\mathcal{W}_{1r}(H),H,c\text{-}wus}\big]\\
&=\mathbb{E}\left[\prod_{W\in\mathcal{W}(H)} \left(\overline{\mathbf{Y}}_W-\left(\frac{\epsilon d}{2n}\right)^s \mathbf{x}_W\right)\middle|\mathbf{x},\mathcal{E}_{\mathcal{W}_{1r}(H),H,c\text{-}wus}\right]\\
&=\sum_{\mathcal{W}\subseteq \mathcal{W}(H)}\left[\prod_{W\in\mathcal{W}(H)\setminus \mathcal{W}}\left(-\left(\frac{\epsilon d}{2n}\right)^s \mathbf{x}_W\right)\right]\cdot \mathbb{E}\left[\prod_{W\in\mathcal{W}} \overline{\mathbf{Y}}_W\middle|\mathbf{x},\mathcal{E}_{\mathcal{W}_{1r}(H),H,c\text{-}wus}\right]\\
&=\sum_{\mathcal{W}\subseteq \mathcal{W}(H)}\left[\prod_{W\in\mathcal{W}(H)\setminus \mathcal{W}}\left(-\left(\frac{\epsilon d}{2n}\right)^s \mathbf{x}_W\right)\right]\cdot \mathbb{E}\left[\overline{\mathbf{Y}}_{H_{\mathcal{W}}}\middle|\mathbf{x},\mathcal{E}_{\mathcal{W}_{1r}(H),H,c\text{-}wus}\right].
\end{align*}

Therefore,
\begin{equation}
\label{eq:eq_lem_X_hat_mult_1_3}
\left|\mathbb{E}\big[\hat{\mathbf{Y}}_{H}\big|\mathbf{x},\mathcal{E}_{\mathcal{W}_{1r}(H),H,c\text{-}wus}\big]\right|\leq\sum_{\mathcal{W}\subseteq \mathcal{W}(H)}\left(\frac{\epsilon d}{2n}\right)^{s|\mathcal{W}(H)\setminus \mathcal{W}|} \cdot \left|\mathbb{E}\left[\overline{\mathbf{Y}}_{H_{\mathcal{W}}}\middle|\mathbf{x},\mathcal{E}_{\mathcal{W}_{1r}(H),H,c\text{-}wus}\right]\right|.
\end{equation}

Now fix $\mathcal{W}\subseteq\mathcal{W}(H)$ and let $\mathcal{S}(H_{\mathcal{W}})=\mathcal{S}_1(H_{\mathcal{W}})\cap \mathcal{S}_{\geq 2}(H_{\mathcal{W}})$. If we take $\mathcal{E}=\mathcal{E}_{\mathcal{W}_{1r}(H),H,c\text{-}wus}$ and $\displaystyle U=\mathcal{S}(H_{\mathcal{W}})\setminus \left(\bigcup_{W\in\mathcal{W}_{1r}(H)}V(W)\right)$, it is easy to see that the conditions of \Cref{lem:lem_UH_deg1_common} and \Cref{lem:lem_UH_deg1_common_2} are satisfied. Therefore,

$$\big|\mathbb{E}\big[\overline{\mathbf{Y}}_H\big|\mathbf{x},\mathcal{E}_{\mathcal{W}_{1r}(H),H,c\text{-}wus}\big]\big|\leq n^{\frac{2K}{A}}\left(\frac{6}{\epsilon}\right)^{|E_{1a}(H_{\mathcal{W}})|+\tau(|\mathcal{S}(H_{\mathcal{W}})|-|U|)}\left(\frac{\epsilon d}{2n}\right)^{|E_1(H_{\mathcal{W}})|} \cdot\mathbb{E}\big[|\tilde{\mathbf{Y}}_{E_{\geq 2}(H_{\mathcal{W}})}|\big|\mathbf{x}\big].$$

Now observe that $E_{1a}(H_{\mathcal{W}})\subseteq E_{1a}(H)\cup \big(E_1(H_{\mathcal{W}})\cap E_{\geq 2}(H)\big)$ and
$$E_1(H_{\mathcal{W}})\cap E_{\geq 2}(H)\subseteq \bigcup_{W\in\mathcal{W}_{\geq2}(H)\setminus\mathcal{W}}W,$$
which implies that
$$|E_{1a}(H_{\mathcal{W}})|\leq |E_{1a}(H)|+s|\mathcal{W}_{\geq2}(H)\setminus\mathcal{W}|.$$

On the other hand, we have
$$|S(H_{\mathcal{W}})|-|U|\leq \left|\bigcup_{W\in\mathcal{W}_{1r}(H)}V(W)\right|\leq \sum_{W\in\mathcal{W}_{1r}(H)}|V(W)|=|\mathcal{W}_{1r}(H)|\cdot(s+1)\leq 2s\cdot |\mathcal{W}_{1r}(H)|.$$

Therefore,
\begin{align*}
\big|\mathbb{E}\big[\overline{\mathbf{Y}}_H\big|\mathbf{x},&\mathcal{E}_{\mathcal{W}_{1r}(H),H,c\text{-}wus}\big]\big|\\
&\leq n^{\frac{2K}{A}}\left(\frac{6}{\epsilon}\right)^{|E_1^a(H)|+s|\mathcal{W}_{\geq2}(H)\setminus\mathcal{W}|+2s\tau\cdot |\mathcal{W}_{1r}(H)|}\left(\frac{\epsilon d}{2n}\right)^{|E_1(H_{\mathcal{W}})|} \cdot\mathbb{E}\big[|\tilde{\mathbf{Y}}_{E_{\geq 2}(H_{\mathcal{W}})}|\big|\mathbf{x}\big]\\
&\leq n^{\frac{2K}{A}}\left(\frac{6}{\epsilon}\right)^{|E_1^a(H)|+s|\mathcal{W}_{\geq2}(H)\setminus\mathcal{W}|}\cdot\left[\left(\frac{6}{\epsilon}\right)^{2s\tau}\right]^{|\mathcal{W}_{1r}(H)|}\cdot\left(\frac{\epsilon d}{2n}\right)^{|E_1(H_{\mathcal{W}})|} \cdot\mathbb{E}\big[|\tilde{\mathbf{Y}}_{E_{\geq 2}(H_{\mathcal{W}})}|\big|\mathbf{x}\big]\\
\end{align*}

By combining this with \cref{eq:eq_lem_X_hat_mult_1_1} and \cref{eq:eq_lem_X_hat_mult_1_2} and \cref{eq:eq_lem_X_hat_mult_1_3}, we get
\begin{equation*}
\begin{aligned}
\big|\mathbb{E}\big[\hat{\mathbf{Y}}_{H}\big|\mathbf{x}\big]\big|\leq \sum_{\mathcal{W}\subseteq \mathcal{W}(H)}&\left(\frac{\epsilon d}{2n}\right)^{s|\mathcal{W}(H)\setminus \mathcal{W}|} \cdot n^{\frac{2K}{A}}\cdot\left(\frac{6}{\epsilon}\right)^{|E_1^a(H)|+s|\mathcal{W}_{\geq2}(H)\setminus\mathcal{W}|}\\
&\times \left[(s\eta)^{\frac{1}{2s\tau}}\cdot\left(\frac{6}{\epsilon}\right)^{2s\tau}\right]^{|\mathcal{W}_{1r}(H)|}\cdot\left(\frac{\epsilon d}{2n}\right)^{|E_1(H_{\mathcal{W}})|} \cdot\mathbb{E}\big[|\tilde{\mathbf{Y}}_{E_{\geq 2}(H_{\mathcal{W}})}|\big|\mathbf{x}\big].
\end{aligned}
\end{equation*}

Now if $A>\max\{100K,1\}$, we get from \cref{eq:eq_def_eta} that $$\eta = \frac{1}{As\cdot 2^{6A\tau s^2}}\left(\frac{\epsilon}{6}\right)^{4s^2\tau^2}\leq \frac{1}{s\cdot 2^{6A\tau s^2}}\left(\frac{\epsilon}{6}\right)^{4s^2\tau^2},$$ hence
\begin{align*}
(s\eta)^{\frac{1}{2s\tau}}\cdot\left(\frac{6}{\epsilon}\right)^{2s\tau}&\leq\left(\frac{1}{2^{6A\tau s^2}}\left(\frac{\epsilon}{6}\right)^{4s^2\tau^2}\right)^{\frac{1}{2s\tau}}\cdot\left(\frac{6}{\epsilon}\right)^{2s\tau}=\frac{1}{2^{3As}}\cdot\left(\frac{\epsilon}{6}\right)^{2s\tau}\cdot\left(\frac{6}{\epsilon}\right)^{2s\tau}=\frac{1}{2^{3As}}.
\end{align*}

Therefore,
\begin{equation}
\label{eq:eq_lem_X_hat_mult_1_4}
\begin{aligned}
\resizebox{0.95\textwidth}{!}{$\displaystyle\big|\mathbb{E}\big[\hat{\mathbf{Y}}_{H}\big|\mathbf{x}\big]\big|\leq \frac{n^{\frac{2K}{A}}}{2^{3As\cdot|\mathcal{W}_{1r}(H)|}} \cdot\left(\frac{6}{\epsilon}\right)^{|E_1^a(H)|}\cdot\sum_{\mathcal{W}\subseteq \mathcal{W}(H)}\left(\frac{6}{\epsilon}\right)^{s|\mathcal{W}_{\geq2}(H)\setminus\mathcal{W}|}\left(\frac{\epsilon d}{2n}\right)^{s|\mathcal{W}(H)\setminus \mathcal{W}|+|E_1(H_{\mathcal{W}})|} \cdot\mathbb{E}\big[|\tilde{\mathbf{Y}}_{E_{\geq 2}(H_{\mathcal{W}})}|\big|\mathbf{x}\big]$}.
\end{aligned}
\end{equation}

Now we have
\begin{equation}
\label{eq:eq_lem_X_hat_mult_1_5}
\begin{aligned}
&\sum_{\mathcal{W}\subseteq \mathcal{W}(H)}\left(\frac{6}{\epsilon}\right)^{s|\mathcal{W}_{\geq2}(H)\setminus\mathcal{W}|}\left(\frac{\epsilon d}{2n}\right)^{s|\mathcal{W}(H)\setminus \mathcal{W}|+|E_1(H_{\mathcal{W}})|} \cdot\mathbb{E}\big[|\tilde{\mathbf{Y}}_{E_{\geq 2}(H_{\mathcal{W}})}|\big|\mathbf{x}\big]\\
&=\resizebox{0.95\textwidth}{!}{$\displaystyle\sum_{\mathcal{W}_1\subseteq \mathcal{W}_1(H)}\sum_{\mathcal{W}_{\geq 2}\subseteq \mathcal{W}_{\geq 2}(H)}\left(\frac{6}{\epsilon}\right)^{s|\mathcal{W}_{\geq2}(H)\setminus(\mathcal{W}_1\cup\mathcal{W}_{\geq 2})|}\left(\frac{\epsilon d}{2n}\right)^{s|\mathcal{W}(H)\setminus (\mathcal{W}_1\cup\mathcal{W}_{\geq 2})|+|E_1(H_{\mathcal{W}_1\cup\mathcal{W}_{\geq 2}})|} \cdot\mathbb{E}\big[|\tilde{\mathbf{Y}}_{E_{\geq 2}(H_{\mathcal{W}_1\cup\mathcal{W}_{\geq 2}})}|\big|\mathbf{x}\big]$},\\
\end{aligned}
\end{equation}

Now for every $\mathcal{W}_1\subseteq\mathcal{W}_1(H)$ and every $\mathcal{W}_{\geq2}\subset\mathcal{W}_{\geq 2}(H)$, we have
\begin{equation}
\label{eq:eq_lem_X_hat_mult_1_625}
\big|\mathcal{W}_{\geq2}(H)\setminus\big(\mathcal{W}_1\cup\mathcal{W}_{\geq 2}\big)\big|=|\mathcal{W}_{\geq2}(H)\setminus \mathcal{W}_{\geq 2}|=|\mathcal{W}_{\geq2}(H)|-|\mathcal{W}_{\geq 2}|,
\end{equation}
$$E_1(H_{\mathcal{W}_1\cup\mathcal{W}_{\geq 2}})=E_1(H_{\mathcal{W}_1})\cup E_1(H_{\mathcal{W}_{\geq 2}}),$$
$$E_1(H_{\mathcal{W}_1})\cap E_1(H_{\mathcal{W}_{\geq 2}})=\varnothing,$$
and
$$|E_1(H_{\mathcal{W}_1})|=s|\mathcal{W}_1|,$$
hence
$$|E_1(H_{\mathcal{W}_1\cup\mathcal{W}_{\geq 2}})|=s|\mathcal{W}_1|+|E_1(H_{\mathcal{W}_{\geq 2}})|.$$

On the other hand, 
\begin{align*}
|\mathcal{W}(H)\setminus (\mathcal{W}_1\cup\mathcal{W}_{\geq 2})|&=|\mathcal{W}(H)|-|\mathcal{W}_1|-|\mathcal{W}_{\geq 2}|\\
&=|\mathcal{W}_1(H)|+|\mathcal{W}_{\geq 2}(H)|-|\mathcal{W}_1|-|\mathcal{W}_{\geq 2}|.
\end{align*}

Therefore,
\begin{equation}
\label{eq:eq_lem_X_hat_mult_1_6}
\begin{aligned}
s|\mathcal{W}(H)\setminus& (\mathcal{W}_1\cup\mathcal{W}_{\geq 2})|+|E_1(H_{\mathcal{W}_1\cup\mathcal{W}_{\geq 2}})|\\
&=s|\mathcal{W}_1(H)|+s|\mathcal{W}_{\geq 2}(H)|-s|\mathcal{W}_1|-s|\mathcal{W}_{\geq 2}|+s|\mathcal{W}_1|+|E_1(H_{\mathcal{W}_{\geq 2}})|\\
&=s|\mathcal{W}_1(H)|+s(|\mathcal{W}_{\geq 2}(H)|-|\mathcal{W}_{\geq 2}|)+|E_1(H_{\mathcal{W}_{\geq 2}})|.
\end{aligned}
\end{equation}

Furthermore, we have
\begin{equation}
\label{eq:eq_lem_X_hat_mult_1_7}
E_{\geq 2}(H_{\mathcal{W}_1\cup\mathcal{W}_{\geq 2}})=E_{\geq 2}(H_{\mathcal{W}_{\geq 2}}).
\end{equation}

By combining \cref{eq:eq_lem_X_hat_mult_1_625}, \cref{eq:eq_lem_X_hat_mult_1_5}, \cref{eq:eq_lem_X_hat_mult_1_6}, and \cref{eq:eq_lem_X_hat_mult_1_7}, we get
\begin{align*}
&\sum_{\mathcal{W}\subseteq \mathcal{W}(H)}\left(\frac{6}{\epsilon}\right)^{s|\mathcal{W}_{\geq2}(H)\setminus\mathcal{W}|}\left(\frac{\epsilon d}{2n}\right)^{s|\mathcal{W}(H)\setminus \mathcal{W}|+|E_1(H_{\mathcal{W}})|} \cdot\mathbb{E}\big[|\tilde{\mathbf{Y}}_{E_{\geq 2}(H_{\mathcal{W}})}|\big|\mathbf{x}\big]\\
&=\resizebox{0.95\textwidth}{!}{$\displaystyle\sum_{\mathcal{W}_1\subseteq \mathcal{W}_1(H)}\sum_{\mathcal{W}_{\geq 2}\subseteq \mathcal{W}_{\geq 2}(H)}\left(\frac{6}{\epsilon}\right)^{s(|\mathcal{W}_{\geq2}(H)|-|\mathcal{W}_{\geq 2}|)}\left(\frac{\epsilon d}{2n}\right)^{s|\mathcal{W}_1(H)|+s(|\mathcal{W}_{\geq 2}(H)|-|\mathcal{W}_{\geq 2}|)+|E_1(H_{\mathcal{W}_{\geq 2}})|} \cdot\mathbb{E}\big[|\tilde{\mathbf{Y}}_{E_{\geq 2}(H_{\mathcal{W}_{\geq 2}})}|\big|\mathbf{x}\big]$}\\
&=\resizebox{0.95\textwidth}{!}{$\displaystyle\left(\frac{\epsilon d}{2n}\right)^{s|\mathcal{W}_1(H)|}\cdot\left[\sum_{\mathcal{W}_1\subseteq \mathcal{W}_1(H)} 1\right]\cdot\sum_{\mathcal{W}_{\geq 2}\subseteq \mathcal{W}_{\geq 2}(H)}\left(\frac{3d}{n}\right)^{s(|\mathcal{W}_{\geq 2}(H)|-|\mathcal{W}_{\geq 2}|)}\left(\frac{\epsilon d}{2n}\right)^{|E_1(H_{\mathcal{W}_{\geq 2}})|} \cdot\mathbb{E}\big[|\tilde{\mathbf{Y}}_{E_{\geq 2}(H_{\mathcal{W}_{\geq 2}})}|\big|\mathbf{x}\big]$}\\
&=2^{|\mathcal{W}_1(H)|}\cdot \left(\frac{\epsilon d}{2n}\right)^{s|\mathcal{W}_1(H)|}\cdot\sum_{\mathcal{W}_{\geq 2}\subseteq \mathcal{W}_{\geq 2}(H)}\left(\frac{3d}{n}\right)^{s(|\mathcal{W}_{\geq 2}(H)|-|\mathcal{W}_{\geq 2}|)}\left(\frac{\epsilon d}{2n}\right)^{|E_1(H_{\mathcal{W}_{\geq 2}})|} \cdot\mathbb{E}\big[|\tilde{\mathbf{Y}}_{E_{\geq 2}(H_{\mathcal{W}_{\geq 2}})}|\big|\mathbf{x}\big]\\
&=2^{|\mathcal{W}_1(H)|}\cdot \left(\frac{\epsilon d}{2n}\right)^{s|\mathcal{W}_1(H)|}\cdot\sum_{\mathcal{W}_{\geq 2}\subseteq \mathcal{W}_{\geq 2}(H)}F_{\mathcal{W}_{\geq 2}}(\mathbf{x}).
\end{align*}

By combining this with \cref{eq:eq_lem_X_hat_mult_1_4} , we get the lemma.
\end{proof}

\subsubsection{Analyzing walks of multiplicity at least 2}
\label{subsubsec:subsubsec_centered_walks_mult_2}

In order to prove \Cref{lem:lem_F_W_2_total_Upper_Bound}, we need a few definitions and lemmas.

\begin{definition}
\label{def:def_K_W_2_Y}
Let $H\in\bsaw{s}{t}$. For every $\mathcal{W}_{\geq 2}\subseteq\mathcal{W}_{\geq 2}(H)$, define
$$K_{\mathcal{W}_{\geq 2}}(\mathbf{x})=\frac{\left(\frac{3d}{n}\right)^{s(|\mathcal{W}_{\geq2}(H)|-|\mathcal{W}_{\geq 2}|)}\left(\frac{\epsilon d}{2n}\right)^{|E_1(H_{\mathcal{W}_{\geq 2}})|}}{\resizebox{0.1\textwidth}{!}{$\displaystyle\prod_{v\in\mathcal{L}_{\geq2}(H_{\mathcal{W}_{\geq 2}})}$} n^{\frac{1}{4}\left(d_{\geq 2}^{H_{\mathcal{W}_{\geq 2}}}(v)-\Delta\right)}}\cdot\left(\frac{d}{n}\right)^{|E_{\geq2}^b(H_{\mathcal{W}_{\geq 2}})|}\cdot\prod_{uv\in E_{\geq2}^a(H_{\mathcal{W}_{\geq 2}})}\left[\left(1+\frac{\epsilon \mathbf{x}_u\mathbf{x}_v}{2}\right)\frac{d}{n}+\frac{d^2}{n^2}\right].$$
\end{definition}

\begin{lemma}
\label{lem:lem_F_W_2_K_W_2}
Let $H\in\bsaw{s}{t}$. For every $\mathcal{W}_{\geq 2}\subseteq\mathcal{W}_{\geq 2}(H)$, we have
$$F_{\mathcal{W}_{\geq 2}}(\mathbf{x})\leq K_{\mathcal{W}_{\geq 2}}(\mathbf{x}).$$
\end{lemma}
\begin{proof}
This is a direct corollary from \Cref{lem:lem_UH_deg2} and the fact that $E_{\geq2}(H_{\mathcal{W}_{\geq 2}})=E_{\geq2}^a(H_{\mathcal{W}_{\geq 2}})\cup E_{\geq2}^b(H_{\mathcal{W}_{\geq 2}})$.
\end{proof}

\Cref{lem:lem_F_W_2_K_W_2} implies that $\displaystyle\sum_{\mathcal{W}_{\geq 2}\subseteq\mathcal{W}_{\geq 2}(H)}F_{\mathcal{W}_{\geq 2}}(\mathbf{x})\leq \sum_{\mathcal{W}_{\geq 2}\subseteq\mathcal{W}_{\geq 2}(H)}K_{\mathcal{W}_{\geq 2}}(\mathbf{x})$. In the following few lemmas, we will prove an upper bound on $\displaystyle\sum_{\mathcal{W}_{\geq 2}\subseteq\mathcal{W}_{\geq 2}(H)}K_{\mathcal{W}_{\geq 2}}(\mathbf{x})$. This will yield an upper bound on $\displaystyle\sum_{\mathcal{W}_{\geq 2}\subseteq\mathcal{W}_{\geq 2}(H)}F_{\mathcal{W}_{\geq 2}}(\mathbf{x})$.

The following lemma compares $K_{\mathcal{W}_{\geq 2}}(\mathbf{x})$ and $K_{\mathcal{W}_{\geq 2}'}(\mathbf{x})$ in the case where $\mathcal{W}_{\geq 2}$ and $\mathcal{W}_{\geq 2}'$ differ by exactly one walk.

\begin{lemma} 
\label{lem:lem_K_W_2_Comparison}
If $H\in\bsaw{s}{t}$ and $n$ is large enough, then for every $\mathcal{W}_{\geq 2}\subseteq\mathcal{W}_{\geq 2}(H)$ and every $W\in \mathcal{W}_{\geq 2}$, if we define $\mathcal{W}_{\geq 2}'=\mathcal{W}_{\geq 2}\setminus \{W\}$, then
\begin{align*}
&\frac{K_{\mathcal{W}_{\geq 2}'}(\mathbf{x})}{K_{\mathcal{W}_{\geq 2}}(\mathbf{x})}\\
&\leq \resizebox{0.95\textwidth}{!}{$\displaystyle  \frac{\left(1+\frac{1}{\sqrt{n}}\right)\cdot \left(\frac{3 d}{n}\right)^{|E_{\geq 2}(H_{\mathcal{W}_{\geq 2}'})\cap W|}}{\left(\frac{2\sqrt{n}}{9 d}\right)^{|E_{\geq 2}(H_{\mathcal{W}_{\geq 2}})\setminus E_{\geq 2}(H_{\mathcal{W}_{\geq 2}'})|}\cdot \left(\frac{n\sqrt{n}}{2d}\right)^{|E_{\geq2}^a(H_{\mathcal{W}_{\geq 2}})\setminus E_{\geq 2}(H_{\mathcal{W}_{\geq 2}'})|}\cdot\resizebox{0.17\textwidth}{!}{$\displaystyle\prod_{uv\in E_{\geq2}^a(H_{\mathcal{W}_{\geq 2}})\setminus E_{\geq 2}(H_{\mathcal{W}_{\geq 2}'})}$}\left[\left(1+\frac{\epsilon \mathbf{x}_u\mathbf{x}_v}{2}\right)\frac{d}{n}+\frac{d^2}{n^2}\right]}$}.
\end{align*}
\end{lemma}
\begin{proof}
We have
\begin{align*}
\prod_{v\in\mathcal{L}_{\geq2}(H_{\mathcal{W}_{\geq 2}})} n^{\frac{1}{4}\left(d^{H_{\mathcal{W}_{\geq 2}}}_{\geq 2}(v)-\Delta\right)}&=n^{-\frac{\Delta}{4}|\mathcal{L}_{\geq2}(H_{\mathcal{W}_{\geq 2}})|}\cdot n^{\resizebox{0.07\textwidth}{!}{$\displaystyle\sum_{v\in\mathcal{L}_{\geq2}(H_{\mathcal{W}_{\geq 2}})}$} \frac{1}{4}d^{H_{\mathcal{W}_{\geq 2}}}_{\geq 2}(v)}.
\end{align*}

Furthermore,
$$\sum_{v\in\mathcal{L}_{\geq2}(H_{\mathcal{W}_{\geq 2}})}d_{H_{\mathcal{W}_{\geq 2}}}(v)=2\cdot|E_{\geq2}^{b,i}(H_{\mathcal{W}_{\geq 2}})|+|E_{\geq2}^{b,o}(H_{\mathcal{W}_{\geq 2}})|,$$
where
$$E_{\geq 2}^{b,i}(H_{\mathcal{W}_{\geq 2}})=\big\{uv\in E_{\geq 2}(H_{\mathcal{W}_{\geq 2}}):\; u\in\mathcal{L}_{\geq2}(H_{\mathcal{W}_{\geq 2}})\text{ and }v\in\mathcal{L}_{\geq2}(H_{\mathcal{W}_{\geq 2}})\big\},$$
and
$$E_{\geq 2}^{b,o}(H_{\mathcal{W}_{\geq 2}})=\big\{uv\in E_{\geq 2}(H_{\mathcal{W}_{\geq 2}}):\; u\in\mathcal{L}_{\geq2}(H_{\mathcal{W}_{\geq 2}})\text{ and }v\notin\mathcal{L}_{\geq2}(H_{\mathcal{W}_{\geq 2}})\big\}.$$

Therefore,
\begin{align*}
\prod_{v\in\mathcal{L}_{\geq2}(H_{\mathcal{W}_{\geq 2}})} n^{\frac{1}{4}\left(d^{H_{\mathcal{W}_{\geq 2}}}_{\geq 2}(v)-\Delta\right)}&=n^{-\frac{\Delta}{4}|\mathcal{L}_{\geq2}(H_{\mathcal{W}_{\geq 2}})|+\frac{1}{2}|E_{\geq2}^{b,i}(H_{\mathcal{W}_{\geq 2}})|+\frac{1}{4}|E_{\geq2}^{b,o}(H_{\mathcal{W}_{\geq 2}})|},
\end{align*}
hence
\begin{align*}
&K_{\mathcal{W}_{\geq 2}}(\mathbf{x})\\
&=\frac{\left(\frac{3d}{n}\right)^{s(|\mathcal{W}_{\geq2}(H)|-|\mathcal{W}_{\geq 2}|)}\left(\frac{\epsilon d}{2n}\right)^{|E_1(H_{\mathcal{W}_{\geq 2}})|}}{n^{-\frac{\Delta}{4}|\mathcal{L}_{\geq2}(H_{\mathcal{W}_{\geq 2}})|+\frac{1}{2}|E_{\geq2}^{b,i}(H_{\mathcal{W}_{\geq 2}})|+\frac{1}{4}|E_{\geq2}^{b,o}(H_{\mathcal{W}_{\geq 2}})|}}\cdot\left(\frac{d}{n}\right)^{|E_{\geq2}^b(H_{\mathcal{W}_{\geq 2}})|}\cdot\prod_{uv\in E_{\geq2}^a(H_{\mathcal{W}_{\geq 2}})}\left[\left(1+\frac{\epsilon \mathbf{x}_u\mathbf{x}_v}{2}\right)\frac{d}{n}+\frac{d^2}{n^2}\right]\\
&=n^{\frac{\Delta}{4}|\mathcal{L}_{\geq2}(H_{\mathcal{W}_{\geq 2}})|}\cdot\left(\frac{3d}{n}\right)^{s(|\mathcal{W}_{\geq2}(H)|-|\mathcal{W}_{\geq 2}|)}\left(\frac{\epsilon d}{2n}\right)^{|E_1(H_{\mathcal{W}_{\geq 2}})|}\cdot\left(\frac{d}{n^{\frac{3}{2}}}\right)^{|E_{\geq2}^{b,i}(H_{\mathcal{W}_{\geq 2}})|}\cdot\left(\frac{d}{n^{\frac{5}{4}}}\right)^{|E_{\geq2}^{b,o}(H_{\mathcal{W}_{\geq 2}})|}\\
&\quad\quad\quad\quad\quad\quad\quad\quad\quad\quad\quad\quad\quad\quad\quad\quad\quad\quad\quad\quad\quad\quad\quad\quad\times\prod_{uv\in E_{\geq2}^a(H_{\mathcal{W}_{\geq 2}})}\left[\left(1+\frac{\epsilon \mathbf{x}_u\mathbf{x}_v}{2}\right)\frac{d}{n}+\frac{d^2}{n^2}\right].
\end{align*}

If we apply the above equation to $\mathcal{W}_2'$ and use the fact that $|\mathcal{W}_{\geq 2}'|=|\mathcal{W}_{\geq 2}|-1$, we get
\begin{align*}
&K_{\mathcal{W}_{\geq 2}'}(\mathbf{x})\\
&\quad=n^{\frac{\Delta}{4}|\mathcal{L}_{\geq2}(H_{\mathcal{W}_{\geq 2}'})|}\cdot\left(\frac{3d}{n}\right)^{s+s(|\mathcal{W}_{\geq2}(H)|-|\mathcal{W}_{\geq 2}|)}\left(\frac{\epsilon d}{2n}\right)^{|E_1(H_{\mathcal{W}_{\geq 2}'})|}\cdot\left(\frac{d}{n^{\frac{3}{2}}}\right)^{|E_{\geq2}^{b,i}(H_{\mathcal{W}_{\geq 2}'})|}\cdot\left(\frac{d}{n^{\frac{5}{4}}}\right)^{|E_{\geq2}^{b,o}(H_{\mathcal{W}_{\geq 2}'})|}\\
&\quad\quad\quad\quad\quad\quad\quad\quad\quad\quad\quad\quad\quad\quad\quad\quad\quad\quad\quad\quad\quad\quad\quad\quad\times\prod_{uv\in E_{\geq2}^a(H_{\mathcal{W}_{\geq 2}'})}\left[\left(1+\frac{\epsilon \mathbf{x}_u\mathbf{x}_v}{2}\right)\frac{d}{n}+\frac{d^2}{n^2}\right].
\end{align*}

Therefore,
\begin{equation}
\label{eq:eq_lem_K_W_2_Comparison_1}
\begin{aligned}
&\frac{K_{\mathcal{W}_{\geq 2}'}(\mathbf{x})}{K_{\mathcal{W}_{\geq 2}}(\mathbf{x})}\\
&=n^{\frac{\Delta}{4}\left(|\mathcal{L}_{\geq2}(H_{\mathcal{W}_{\geq 2}'})|-|\mathcal{L}_{\geq2}(H_{\mathcal{W}_{\geq 2}})|\right)} \cdot \left(\frac{3 d}{n}\right)^s\cdot\left(\frac{\epsilon d}{2n}\right)^{|E_1(H_{\mathcal{W}_{\geq 2}'})|-|E_1(H_{\mathcal{W}_{\geq 2}})|}\cdot \left(\frac{d}{n^{\frac{3}{2}}}\right)^{|E_{\geq2}^{b,i}(H_{\mathcal{W}_{\geq 2}'})|-|E_{\geq2}^{b,i}(H_{\mathcal{W}_{\geq 2}})|}\\
&\quad\quad\quad\quad\quad\quad\quad\quad\quad\quad\quad\quad
\times
\left(\frac{d}{n^{\frac{5}{4}}}\right)^{|E_{\geq2}^{b,o}(H_{\mathcal{W}_{\geq 2}'})|-|E_{\geq2}^{b,o}(H_{\mathcal{W}_{\geq 2}})|}\cdot\frac{
\resizebox{0.085\textwidth}{!}{$\displaystyle\prod_{uv\in E_{\geq2}^a(H_{\mathcal{W}_{\geq 2}'})}$}
\left[\left(1+\frac{\epsilon \mathbf{x}_u\mathbf{x}_v}{2}\right)\frac{d}{n}+\frac{d^2}{n^2}\right]}{\resizebox{0.085\textwidth}{!}{$\displaystyle\prod_{uv\in E_{\geq2}^a(H_{\mathcal{W}_{\geq 2}})}$}\left[\left(1+\frac{\epsilon \mathbf{x}_u\mathbf{x}_v}{2}\right)\frac{d}{n}+\frac{d^2}{n^2}\right]}.
\end{aligned}
\end{equation}

Now since $\mathcal{W}_{\geq 2}'\subseteq\mathcal{W}_{\geq 2}$, the multigraph $H_{\mathcal{W}_{\geq 2}'}$ is a submultigraph of $H_{\mathcal{W}_{\geq 2}}$, hence $d^{H_{\mathcal{W}_{\geq 2}'}}_{\geq 2}(v)\leq d^{H_{\mathcal{W}_{\geq 2}'}}_{\geq 2}(v)$ for every $v\in V(H_{\mathcal{W}_{\geq 2}'})$.  This means that $\mathcal{L}_{\geq2}(H_{\mathcal{W}_{\geq 2}'})\subseteq \mathcal{L}_{\geq2}(H_{\mathcal{W}_{\geq 2}})$. Therefore,
\begin{align*}
&\Delta\cdot \left(|\mathcal{L}_{\geq2}(H_{\mathcal{W}_{\geq 2}})|-|\mathcal{L}_{\geq2}(H_{\mathcal{W}_{\geq 2}'})|\right)\\
&=\sum_{v\in \mathcal{L}_{\geq2}(H_{\mathcal{W}_{\geq 2}})\setminus\mathcal{L}_{\geq2}(H_{\mathcal{W}_{\geq 2}'})}\Delta\geq \sum_{v\in \mathcal{L}_{\geq2}(H_{\mathcal{W}_{\geq 2}})\setminus\mathcal{L}_{\geq2}(H_{\mathcal{W}_{\geq 2}'})}d^{H_{\mathcal{W}_{\geq 2}'}}_{\geq 2}(v)\\
&= \resizebox{0.95\textwidth}{!}{$\displaystyle 2\cdot\big|\big\{uv\in E_{\geq 2}(H_{\mathcal{W}_{\geq 2}'}):\; u\in\mathcal{L}_{\geq2}(H_{\mathcal{W}_{\geq 2}})\setminus\mathcal{L}_{\geq2}(H_{\mathcal{W}_{\geq 2}'})\text{ and }v\in\mathcal{L}_{\geq2}(H_{\mathcal{W}_{\geq 2}})\setminus\mathcal{L}_{\geq2}(H_{\mathcal{W}_{\geq 2}'})\big\}\big|$}\\
&\quad+\resizebox{0.92\textwidth}{!}{$\big|\big\{uv\in E_{\geq 2}(H_{\mathcal{W}_{\geq 2}'}):\; u\in\mathcal{L}_{\geq2}(H_{\mathcal{W}_{\geq 2}})\setminus\mathcal{L}_{\geq2}(H_{\mathcal{W}_{\geq 2}'})\text{ and }v\notin\mathcal{L}_{\geq2}(H_{\mathcal{W}_{\geq 2}})\setminus\mathcal{L}_{\geq2}(H_{\mathcal{W}_{\geq 2}'})\big\}\big|$}\\
&= \resizebox{0.95\textwidth}{!}{$2\cdot\big|\big\{uv\in E_{\geq 2}(H_{\mathcal{W}_{\geq 2}'}):\; u\in\mathcal{L}_{\geq2}(H_{\mathcal{W}_{\geq 2}})\setminus\mathcal{L}_{\geq2}(H_{\mathcal{W}_{\geq 2}'})\text{ and }v\in\mathcal{L}_{\geq2}(H_{\mathcal{W}_{\geq 2}})\setminus\mathcal{L}_{\geq2}(H_{\mathcal{W}_{\geq 2}'})\big\}\big|$}\\
&\quad+\resizebox{0.92\textwidth}{!}{$\big|\big\{uv\in E_{\geq 2}(H_{\mathcal{W}_{\geq 2}'}):\; u\in\mathcal{L}_{\geq2}(H_{\mathcal{W}_{\geq 2}})\setminus\mathcal{L}_{\geq2}(H_{\mathcal{W}_{\geq 2}'})\text{ and }v\notin\mathcal{L}_{\geq2}(H_{\mathcal{W}_{\geq 2}})\text{ and }v\notin\mathcal{L}_{\geq2}(H_{\mathcal{W}_{\geq 2}'})\big\}\big|$}\\
&\quad+\big|\big\{uv\in E_{\geq 2}(H_{\mathcal{W}_{\geq 2}'}):\; u\in\mathcal{L}_{\geq2}(H_{\mathcal{W}_{\geq 2}})\setminus\mathcal{L}_{\geq2}(H_{\mathcal{W}_{\geq 2}'})\text{ and }v\in\mathcal{L}_{\geq2}(H_{\mathcal{W}_{\geq 2}'})\big\}\big|\\
&\stackrel{(\ast)}{=} \resizebox{0.95\textwidth}{!}{$2\cdot\big|\big\{uv\in E_{\geq 2}(H_{\mathcal{W}_{\geq 2}'}):\; u\in\mathcal{L}_{\geq2}(H_{\mathcal{W}_{\geq 2}})\setminus\mathcal{L}_{\geq2}(H_{\mathcal{W}_{\geq 2}'})\text{ and }v\in\mathcal{L}_{\geq2}(H_{\mathcal{W}_{\geq 2}})\setminus\mathcal{L}_{\geq2}(H_{\mathcal{W}_{\geq 2}'})\big\}\big|$}\\
&\quad+\resizebox{0.94\textwidth}{!}{$\big|\big\{uv\in E_{\geq 2}(H_{\mathcal{W}_{\geq 2}'}):\; u\in\mathcal{L}_{\geq2}(H_{\mathcal{W}_{\geq 2}})\setminus\mathcal{L}_{\geq2}(H_{\mathcal{W}_{\geq 2}'})\text{ and }v\notin\mathcal{L}_{\geq2}(H_{\mathcal{W}_{\geq 2}})\text{ and }v\notin\mathcal{L}_{\geq2}(H_{\mathcal{W}_{\geq 2}'})\big\}\big|$}\\
&\quad+\resizebox{0.94\textwidth}{!}{$\big|\big\{uv\in E_{\geq 2}(H_{\mathcal{W}_{\geq 2}'}):\; u\in\mathcal{L}_{\geq2}(H_{\mathcal{W}_{\geq 2}})\setminus\mathcal{L}_{\geq2}(H_{\mathcal{W}_{\geq 2}'})\text{ and }v\in\mathcal{L}_{\geq2}(H_{\mathcal{W}_{\geq 2}})\text{ and }v\in\mathcal{L}_{\geq2}(H_{\mathcal{W}_{\geq 2}'})\big\}\big|$}\\
&= 2\cdot\big|E_{\geq2}^a(H_{\mathcal{W}_{\geq 2}'})\cap E_{\geq2}^{b,i}(H_{\mathcal{W}_{\geq 2}})\big|+\big|E_{\geq2}^a(H_{\mathcal{W}_{\geq 2}'})\cap E_{\geq2}^{b,o}(H_{\mathcal{W}_{\geq 2}})\big|+\big|E_{\geq2}^{b,o}(H_{\mathcal{W}_{\geq 2}'})\cap E_{\geq2}^{b,i}(H_{\mathcal{W}_{\geq 2}})\big|,
\end{align*}
where $(\ast)$ follows from the fact that $\mathcal{L}_{\geq2}(H_{\mathcal{W}_{\geq 2}'})\subseteq \mathcal{L}_{\geq2}(H_{\mathcal{W}_{\geq 2}})$. By combining this with \cref{eq:eq_lem_K_W_2_Comparison_1}, we get

\begin{equation}
\label{eq:eq_lem_K_W_2_Comparison_2}
\begin{aligned}
\frac{K_{\mathcal{W}_{\geq 2}'}(\mathbf{x})}{K_{\mathcal{W}_{\geq 2}}(\mathbf{x})}&\leq n^{-\frac{1}{2}|E_{\geq2}^a(H_{\mathcal{W}_{\geq 2}'})\cap E_{\geq2}^{b,i}(H_{\mathcal{W}_{\geq 2}})|-\frac{1}{4}|E_{\geq2}^a(H_{\mathcal{W}_{\geq 2}'})\cap E_{\geq2}^{b,o}(H_{\mathcal{W}_{\geq 2}})|-\frac{1}{4}|E_{\geq2}^{b,o}(H_{\mathcal{W}_{\geq 2}'})\cap E_{\geq2}^{b,i}(H_{\mathcal{W}_{\geq 2}})|}\\
& \quad\quad\quad\quad\quad\quad\quad\times \left(\frac{3 d}{n}\right)^s\cdot\left(\frac{\epsilon d}{2n}\right)^{|E_1(H_{\mathcal{W}_{\geq 2}'})|-|E_1(H_{\mathcal{W}_{\geq 2}})|}\cdot
\left(\frac{2d}{n^{\frac{3}{2}}}\right)^{|E_{\geq2}^{b,i}(H_{\mathcal{W}_{\geq 2}'})|-|E_{\geq2}^{b,i}(H_{\mathcal{W}_{\geq 2}})|}\\
&\quad\quad\quad\quad\quad\quad\quad\times\left(\frac{2d}{n^{\frac{5}{4}}}\right)^{|E_{\geq2}^{b,o}(H_{\mathcal{W}_{\geq 2}'})|-|E_{\geq2}^{b,o}(H_{\mathcal{W}_{\geq 2}})|}\cdot\frac{
\resizebox{0.085\textwidth}{!}{$\displaystyle\prod_{uv\in E_{\geq2}^a(H_{\mathcal{W}_{\geq 2}'})}$}
\left[\left(1+\frac{\epsilon \mathbf{x}_u\mathbf{x}_v}{2}\right)\frac{d}{n}+\frac{d^2}{n^2}\right]}{\resizebox{0.085\textwidth}{!}{$\displaystyle\prod_{uv\in E_{\geq2}^a(H_{\mathcal{W}_{\geq 2}})}$}\left[\left(1+\frac{\epsilon \mathbf{x}_u\mathbf{x}_v}{2}\right)\frac{d}{n}+\frac{d^2}{n^2}\right]}.
\end{aligned}
\end{equation}

Now we have:
\begin{align*}
&\prod_{uv\in E_{\geq2}^a(H_{\mathcal{W}_{\geq 2}'})}\left[\left(1+\frac{\epsilon \mathbf{x}_u\mathbf{x}_v}{2}\right)\frac{d}{n}+\frac{d^2}{n^2}\right]\\
&=\prod_{uv\in E_{\geq2}^a(H_{\mathcal{W}_{\geq 2}'})\cap E_{\geq2}^b(H_{\mathcal{W}_{\geq 2}})}\left[\left(1+\frac{\epsilon \mathbf{x}_u\mathbf{x}_v}{2}\right)\frac{d}{n}+\frac{d^2}{n^2}\right]\cdot\prod_{uv\in E_{\geq2}^a(H_{\mathcal{W}_{\geq 2}'})\setminus E_{\geq2}^b(H_{\mathcal{W}_{\geq 2}})}\left[\left(1+\frac{\epsilon \mathbf{x}_u\mathbf{x}_v}{2}\right)\frac{d}{n}+\frac{d^2}{n^2}\right]\\
&\leq\prod_{uv\in E_{\geq2}^a(H_{\mathcal{W}_{\geq 2}'})\cap E_{\geq2}^b(H_{\mathcal{W}_{\geq 2}})}\left[\frac{2d}{n}+\frac{d^2}{n^2}\right]\cdot\prod_{uv\in E_{\geq2}^a(H_{\mathcal{W}_{\geq 2}'})\setminus E_{\geq2}^b(H_{\mathcal{W}_{\geq 2}})}\left[\left(1+\frac{\epsilon \mathbf{x}_u\mathbf{x}_v}{2}\right)\frac{d}{n}+\frac{d^2}{n^2}\right]\\
&=\prod_{uv\in E_{\geq2}^a(H_{\mathcal{W}_{\geq 2}'})\cap E_{\geq2}^b(H_{\mathcal{W}_{\geq 2}})}\left[\frac{2d}{n}\left(1+\frac{d}{2n}\right)\right]\cdot\prod_{uv\in E_{\geq2}^a(H_{\mathcal{W}_{\geq 2}'})\setminus E_{\geq2}^b(H_{\mathcal{W}_{\geq 2}})}\left[\left(1+\frac{\epsilon \mathbf{x}_u\mathbf{x}_v}{2}\right)\frac{d}{n}+\frac{d^2}{n^2}\right]\\
&\leq \left(1+O\left(\frac{d st}{2n}\right)\right)\cdot\left(\frac{2d}{n}\right)^{|E_{\geq2}^a(H_{\mathcal{W}_{\geq 2}'})\cap E_{\geq2}^b(H_{\mathcal{W}_{\geq 2}})|} \prod_{uv\in E_{\geq2}^a(H_{\mathcal{W}_{\geq 2}'})\setminus E_{\geq2}^b(H_{\mathcal{W}_{\geq 2}})}\left[\left(1+\frac{\epsilon \mathbf{x}_u\mathbf{x}_v}{2}\right)\frac{d}{n}+\frac{d^2}{n^2}\right]\\
&\leq \left(1+\frac{1}{\sqrt{n}}\right)\cdot\left(\frac{2d}{n}\right)^{|E_{\geq2}^a(H_{\mathcal{W}_{\geq 2}'})\cap E_{\geq2}^{b,i}(H_{\mathcal{W}_{\geq 2}})|}\cdot \left(\frac{2d}{n}\right)^{|E_{\geq2}^a(H_{\mathcal{W}_{\geq 2}'})\cap E_{\geq2}^{b,o}(H_{\mathcal{W}_{\geq 2}})|}\\
&\quad\quad\quad\quad\quad\quad\quad\quad\quad\quad\quad\quad\quad\quad\quad\quad\quad
\times\prod_{uv\in E_{\geq2}^a(H_{\mathcal{W}_{\geq 2}'})\setminus E_{\geq2}^b(H_{\mathcal{W}_{\geq 2}})}\left[\left(1+\frac{\epsilon \mathbf{x}_u\mathbf{x}_v}{2}\right)\frac{d}{n}+\frac{d^2}{n^2}\right],
\end{align*}
where the last inequality is true for $n$ large enough. By combining this with \cref{eq:eq_lem_K_W_2_Comparison_2}, we get

\begin{align*}
&\frac{K_{\mathcal{W}_{\geq 2}'}(\mathbf{x})}{K_{\mathcal{W}_{\geq 2}}(\mathbf{x})}\\
&\leq  \left(1+\frac{1}{\sqrt{n}}\right)\cdot n^{-\frac{1}{4}|E_{\geq2}^{b,o}(H_{\mathcal{W}_{\geq 2}'})\cap E_{\geq2}^{b,i}(H_{\mathcal{W}_{\geq 2}})|}\cdot\left(\frac{2d}{n^{\frac{3}{2}}}\right)^{|E_{\geq2}^a(H_{\mathcal{W}_{\geq 2}'})\cap E_{\geq2}^{b,i}(H_{\mathcal{W}_{\geq 2}})|}\cdot \left(\frac{2d}{n^{\frac{5}{4}}}\right)^{|E_{\geq2}^a(H_{\mathcal{W}_{\geq 2}'})\cap E_{\geq2}^{b,o}(H_{\mathcal{W}_{\geq 2}})|}\\
&\quad
\times
\left(\frac{3 d}{n}\right)^s\cdot\left(\frac{\epsilon d}{2n}\right)^{|E_1(H_{\mathcal{W}_{\geq 2}'})|-|E_1(H_{\mathcal{W}_{\geq 2}})|}\cdot\left(\frac{2d}{n^{\frac{3}{2}}}\right)^{|E_{\geq2}^{b,i}(H_{\mathcal{W}_{\geq 2}'})|-|E_{\geq2}^{b,i}(H_{\mathcal{W}_{\geq 2}})|}\cdot\left(\frac{2d}{n^{\frac{5}{4}}}\right)^{|E_{\geq2}^{b,o}(H_{\mathcal{W}_{\geq 2}'})|-|E_{\geq2}^{b,o}(H_{\mathcal{W}_{\geq 2}})|}\\
&\quad\quad\quad\quad\quad\quad\quad\quad\quad\quad\quad\quad\quad\quad\quad\quad\quad\quad
\times\frac{
\resizebox{0.17\textwidth}{!}{$\displaystyle\prod_{uv\in E_{\geq2}^a(H_{\mathcal{W}_{\geq 2}'})\setminus E_{\geq2}^b(H_{\mathcal{W}_{\geq 2}})}$}
\left[\left(1+\frac{\epsilon \mathbf{x}_u\mathbf{x}_v}{2}\right)\frac{d}{n}+\frac{d^2}{n^2}\right]}{\resizebox{0.085\textwidth}{!}{$\displaystyle\prod_{uv\in E_{\geq2}^a(H_{\mathcal{W}_{\geq 2}})}$}\left[\left(1+\frac{\epsilon \mathbf{x}_u\mathbf{x}_v}{2}\right)\frac{d}{n}+\frac{d^2}{n^2}\right]}.
\end{align*}

Hence,
\begin{align*}
\frac{K_{\mathcal{W}_{\geq 2}'}(\mathbf{x})}{K_{\mathcal{W}_{\geq 2}}(\mathbf{x})}&\leq  \left(1+\frac{1}{\sqrt{n}}\right)\cdot n^{-\frac{1}{4}|E_{\geq2}^{b,o}(H_{\mathcal{W}_{\geq 2}'})\cap E_{\geq2}^{b,i}(H_{\mathcal{W}_{\geq 2}})|}\cdot \left(\frac{3 d}{n}\right)^s\cdot\left(\frac{\epsilon d}{2n}\right)^{|E_1(H_{\mathcal{W}_{\geq 2}'})|-|E_1(H_{\mathcal{W}_{\geq 2}})|}\\
&\quad
\times\left(\frac{2d}{n^{\frac{3}{2}}}\right)^{|E_{\geq2}^{b,i}(H_{\mathcal{W}_{\geq 2}'})|-|E_{\geq2}^{b,i}(H_{\mathcal{W}_{\geq 2}})\setminus E_{\geq2}^a(H_{\mathcal{W}_{\geq 2}'})|}\cdot\left(\frac{2d}{n^{\frac{5}{4}}}\right)^{|E_{\geq2}^{b,o}(H_{\mathcal{W}_{\geq 2}'})|-|E_{\geq2}^{b,o}(H_{\mathcal{W}_{\geq 2}})\setminus E_{\geq2}^a(H_{\mathcal{W}_{\geq 2}'})|}\\
&\quad\quad\quad\quad\quad\quad\quad\quad\quad\quad\quad\quad\quad\quad\quad\quad\quad\quad
\times\frac{
\resizebox{0.17\textwidth}{!}{$\displaystyle\prod_{uv\in E_{\geq2}^a(H_{\mathcal{W}_{\geq 2}'})\setminus E_{\geq2}^b(H_{\mathcal{W}_{\geq 2}})}$}
\left[\left(1+\frac{\epsilon \mathbf{x}_u\mathbf{x}_v}{2}\right)\frac{d}{n}+\frac{d^2}{n^2}\right]}{\resizebox{0.085\textwidth}{!}{$\displaystyle\prod_{uv\in E_{\geq2}^a(H_{\mathcal{W}_{\geq 2}})}$}\left[\left(1+\frac{\epsilon \mathbf{x}_u\mathbf{x}_v}{2}\right)\frac{d}{n}+\frac{d^2}{n^2}\right]}.
\end{align*}

Now since $H_{\mathcal{W}_{\geq 2}'}$ is a submultigraph of $H_{\mathcal{W}_{\geq 2}}$, we have $E_{\geq 2}(H_{\mathcal{W}_{\geq 2}'})\subseteq E_{\geq 2}(H_{\mathcal{W}_{\geq 2}})=E_{\geq2}^a(H_{\mathcal{W}_{\geq 2}})\cup E_{\geq2}^b(H_{\mathcal{W}_{\geq 2}})$, which implies that $$E_{\geq2}^a(H_{\mathcal{W}_{\geq 2}'})\subseteq E_{\geq2}^a(H_{\mathcal{W}_{\geq 2}})\cup E_{\geq2}^b(H_{\mathcal{W}_{\geq 2}}).$$
This means that $$E_{\geq2}^a(H_{\mathcal{W}_{\geq 2}'})\setminus E_{\geq2}^b(H_{\mathcal{W}_{\geq 2}})\subseteq E_{\geq2}^a(H_{\mathcal{W}_{\geq 2}}).$$
On the other hand, since $E_{\geq2}^a(H_{\mathcal{W}_{\geq 2}})\cap E_{\geq2}^b(H_{\mathcal{W}_{\geq 2}})=\varnothing$, we have
$$E_{\geq2}^a(H_{\mathcal{W}_{\geq 2}})\setminus\big(E_{\geq2}^a(H_{\mathcal{W}_{\geq 2}'})\setminus E_{\geq2}^b(H_{\mathcal{W}_{\geq 2}})\big)=E_{\geq2}^a(H_{\mathcal{W}_{\geq 2}})\setminus E_{\geq2}^a(H_{\mathcal{W}_{\geq 2}'}).$$

Therefore,
\begin{align*}
\frac{K_{\mathcal{W}_{\geq 2}'}(\mathbf{x})}{K_{\mathcal{W}_{\geq 2}}(\mathbf{x})}&\leq  \left(1+\frac{1}{\sqrt{n}}\right)\cdot n^{-\frac{1}{4}|E_{\geq2}^{b,o}(H_{\mathcal{W}_{\geq 2}'})\cap E_{\geq2}^{b,i}(H_{\mathcal{W}_{\geq 2}})|}\cdot\left(\frac{3 d}{n}\right)^s\cdot\left(\frac{\epsilon d}{2n}\right)^{|E_1(H_{\mathcal{W}_{\geq 2}'})|-|E_1(H_{\mathcal{W}_{\geq 2}})|}\\
&\quad\quad
\times\left(\frac{2d}{n^{\frac{3}{2}}}\right)^{|E_{\geq2}^{b,i}(H_{\mathcal{W}_{\geq 2}'})|-|E_{\geq2}^{b,i}(H_{\mathcal{W}_{\geq 2}})\setminus E_{\geq2}^a(H_{\mathcal{W}_{\geq 2}'})|}\cdot\frac{\left(\frac{2d}{n^{\frac{5}{4}}}\right)^{|E_{\geq2}^{b,o}(H_{\mathcal{W}_{\geq 2}'})|-|E_{\geq2}^{b,o}(H_{\mathcal{W}_{\geq 2}})\setminus E_{\geq2}^a(H_{\mathcal{W}_{\geq 2}'})|}}{\resizebox{0.17\textwidth}{!}{$\displaystyle\prod_{uv\in E_{\geq2}^a(H_{\mathcal{W}_{\geq 2}})\setminus E_{\geq2}^a(H_{\mathcal{W}_{\geq 2}'})}$}\left[\left(1+\frac{\epsilon \mathbf{x}_u\mathbf{x}_v}{2}\right)\frac{d}{n}+\frac{d^2}{n^2}\right]}.
\end{align*}

Now since $\mathcal{L}_{\geq2}(H_{\mathcal{W}_{\geq 2}'})\subseteq \mathcal{L}_{\geq2}(H_{\mathcal{W}_{\geq 2}})$ and $E_{\geq 2}(H_{\mathcal{W}_{\geq 2}'})\subseteq E_{\geq 2}(H_{\mathcal{W}_{\geq 2}})$, we have
\begin{equation}
\label{eq:eq_lem_K_W_2_Comparison_3}
\begin{aligned}
E_{\geq2}^b(H_{\mathcal{W}_{\geq 2}'})&=\big\{uv\in E_{\geq 2}(H_{\mathcal{W}_{\geq 2}'}):\; u\in \mathcal{L}_{\geq2}(H_{\mathcal{W}_{\geq 2}'})\text{ or }v\in\mathcal{L}_{\geq2}(H_{\mathcal{W}_{\geq 2}'})\big\}\\
&\subseteq \big\{uv\in E_{\geq 2}(H_{\mathcal{W}_{\geq 2}}):\; u\in \mathcal{L}_{\geq2}(H_{\mathcal{W}_{\geq 2}})\text{ or }v\in\mathcal{L}_{\geq2}(H_{\mathcal{W}_{\geq 2}})\big\}=E_{\geq2}^b(H_{\mathcal{W}_{\geq 2}}).
\end{aligned}
\end{equation}
On the other hand, since $E_{\geq2}^{b,o}(H_{\mathcal{W}_{\geq 2}'})\subseteq E_{\geq2}^b(H_{\mathcal{W}_{\geq 2}'}) \subseteq E_{\geq2}^b(H_{\mathcal{W}_{\geq 2}})$ and since $\big\{E_{\geq2}^{b,i}(H_{\mathcal{W}_{\geq 2}}),E_{\geq2}^{b,o}(H_{\mathcal{W}_{\geq 2}})\big\}$ is a partition of $E_{\geq2}^b(H_{\mathcal{W}_{\geq 2}})$, we have
\begin{align*}
E_{\geq2}^{b,o}(H_{\mathcal{W}_{\geq 2}'})&=E_{\geq2}^{b,o}(H_{\mathcal{W}_{\geq 2}'})\cap E_{\geq2}^{b,o}(H_{\mathcal{W}_{\geq 2}})\\
&=\left(E_{\geq2}^{b,o}(H_{\mathcal{W}_{\geq 2}'})\cap E_{\geq2}^{b,i}(H_{\mathcal{W}_{\geq 2}})\right)\cup \left(E_{\geq2}^{b,o}(H_{\mathcal{W}_{\geq 2}'})\cap E_{\geq2}^{b,o}(H_{\mathcal{W}_{\geq 2}})\right).
\end{align*}

Therefore,
\begin{align*}
\frac{K_{\mathcal{W}_{\geq 2}'}(\mathbf{x})}{K_{\mathcal{W}_{\geq 2}}(\mathbf{x})}&\leq  \left(1+\frac{1}{\sqrt{n}}\right)\cdot n^{-\frac{1}{4}|E_{\geq2}^{b,o}(H_{\mathcal{W}_{\geq 2}'})\cap E_{\geq2}^{b,i}(H_{\mathcal{W}_{\geq 2}})|}\cdot\left(\frac{3 d}{n}\right)^s\cdot\left(\frac{\epsilon d}{2n}\right)^{|E_1(H_{\mathcal{W}_{\geq 2}'})|-|E_1(H_{\mathcal{W}_{\geq 2}})|}\\
&\quad\quad\quad\quad\quad
\times\left(\frac{2d}{n^{\frac{3}{2}}}\right)^{|E_{\geq2}^{b,i}(H_{\mathcal{W}_{\geq 2}'})|-|E_{\geq2}^{b,i}(H_{\mathcal{W}_{\geq 2}})\setminus E_{\geq2}^a(H_{\mathcal{W}_{\geq 2}'})|}\cdot\left(\frac{2d}{n^{\frac{5}{4}}}\right)^{|E_{\geq2}^{b,o}(H_{\mathcal{W}_{\geq 2}'})\cap E_{\geq2}^{b,i}(H_{\mathcal{W}_{\geq 2}})|}\\
&\quad\quad\quad\quad\quad
\times
\frac{\left(\frac{2d}{n^{\frac{5}{4}}}\right)^{|E_{\geq2}^{b,o}(H_{\mathcal{W}_{\geq 2}'})\cap E_{\geq2}^{b,o}(H_{\mathcal{W}_{\geq 2}})|-|E_{\geq2}^{b,o}(H_{\mathcal{W}_{\geq 2}})\setminus E_{\geq2}^a(H_{\mathcal{W}_{\geq 2}'})|}}{\resizebox{0.17\textwidth}{!}{$\displaystyle\prod_{uv\in E_{\geq2}^a(H_{\mathcal{W}_{\geq 2}})\setminus E_{\geq2}^a(H_{\mathcal{W}_{\geq 2}'})}$}\left[\left(1+\frac{\epsilon \mathbf{x}_u\mathbf{x}_v}{2}\right)\frac{d}{n}+\frac{d^2}{n^2}\right]}\\
&=  \left(1+\frac{1}{\sqrt{n}}\right)\cdot\left(\frac{3 d}{n}\right)^s\cdot\left(\frac{\epsilon d}{2n}\right)^{|E_1(H_{\mathcal{W}_{\geq 2}'})|-|E_1(H_{\mathcal{W}_{\geq 2}})|}\cdot\left(\frac{2d}{n^{\frac{3}{2}}}\right)^{|E_{\geq2}^{b,i}(H_{\mathcal{W}_{\geq 2}'})|-|E_{\geq2}^{b,i}(H_{\mathcal{W}_{\geq 2}})\setminus E_{\geq2}^a(H_{\mathcal{W}_{\geq 2}'})|}\\
&\quad\quad
\times
\left(\frac{2d}{n^{\frac{3}{2}}}\right)^{|E_{\geq2}^{b,o}(H_{\mathcal{W}_{\geq 2}'})\cap E_{\geq2}^{b,i}(H_{\mathcal{W}_{\geq 2}})|}\cdot\frac{\left(\frac{2d}{n^{\frac{5}{4}}}\right)^{|E_{\geq2}^{b,o}(H_{\mathcal{W}_{\geq 2}'})\cap E_{\geq2}^{b,o}(H_{\mathcal{W}_{\geq 2}})|-|E_{\geq2}^{b,o}(H_{\mathcal{W}_{\geq 2}})\setminus E_{\geq2}^a(H_{\mathcal{W}_{\geq 2}'})|}}{\resizebox{0.17\textwidth}{!}{$\displaystyle\prod_{uv\in E_{\geq2}^a(H_{\mathcal{W}_{\geq 2}})\setminus E_{\geq2}^a(H_{\mathcal{W}_{\geq 2}'})}$}\left[\left(1+\frac{\epsilon \mathbf{x}_u\mathbf{x}_v}{2}\right)\frac{d}{n}+\frac{d^2}{n^2}\right]}.
\end{align*}
Now since $E_{\geq2}^{b,o}(H_{\mathcal{W}_{\geq 2}'})\cap E_{\geq2}^{b,o}(H_{\mathcal{W}_{\geq 2}})\subseteq E_{\geq2}^{b,o}(H_{\mathcal{W}_{\geq 2}})$ and $E_{\geq2}^{b,o}(H_{\mathcal{W}_{\geq 2}'})\cap E_{\geq2}^{b,o}(H_{\mathcal{W}_{\geq 2}})\cap E_{\geq2}^a(H_{\mathcal{W}_{\geq 2}'})=\varnothing$, we have
$$E_{\geq2}^{b,o}(H_{\mathcal{W}_{\geq 2}'})\cap E_{\geq2}^{b,o}(H_{\mathcal{W}_{\geq 2}})\subseteq E_{\geq2}^{b,o}(H_{\mathcal{W}_{\geq 2}})\setminus E_{\geq2}^a(H_{\mathcal{W}_{\geq 2}'}),$$
and so
$$|E_{\geq2}^{b,o}(H_{\mathcal{W}_{\geq 2}'})\cap E_{\geq2}^{b,o}(H_{\mathcal{W}_{\geq 2}})|- |E_{\geq2}^{b,o}(H_{\mathcal{W}_{\geq 2}})\setminus E_{\geq2}^a(H_{\mathcal{W}_{\geq 2}'})|\leq 0.$$
Hence, for $n$ large enough, we have
\begin{align*}
\frac{K_{\mathcal{W}_{\geq 2}'}(\mathbf{x})}{K_{\mathcal{W}_{\geq 2}}(\mathbf{x})}&\leq  \left(1+\frac{1}{\sqrt{n}}\right)\cdot\left(\frac{3 d}{n}\right)^s\cdot\left(\frac{\epsilon d}{2n}\right)^{|E_1(H_{\mathcal{W}_{\geq 2}'})|-|E_1(H_{\mathcal{W}_{\geq 2}})|}\cdot\left(\frac{2d}{n^{\frac{3}{2}}}\right)^{|E_{\geq2}^{b,i}(H_{\mathcal{W}_{\geq 2}'})|-|E_{\geq2}^{b,i}(H_{\mathcal{W}_{\geq 2}})\setminus E_{\geq2}^a(H_{\mathcal{W}_{\geq 2}'})|}\\
&\quad\quad
\times
\left(\frac{2d}{n^{\frac{3}{2}}}\right)^{|E_{\geq2}^{b,o}(H_{\mathcal{W}_{\geq 2}'})\cap E_{\geq2}^{b,i}(H_{\mathcal{W}_{\geq 2}})|}\cdot\frac{\left(\frac{2d}{n^{\frac{3}{2}}}\right)^{|E_{\geq2}^{b,o}(H_{\mathcal{W}_{\geq 2}'})\cap E_{\geq2}^{b,o}(H_{\mathcal{W}_{\geq 2}})|-|E_{\geq2}^{b,o}(H_{\mathcal{W}_{\geq 2}})\setminus E_{\geq2}^a(H_{\mathcal{W}_{\geq 2}'})|}}{\resizebox{0.17\textwidth}{!}{$\displaystyle\prod_{uv\in E_{\geq2}^a(H_{\mathcal{W}_{\geq 2}})\setminus E_{\geq2}^a(H_{\mathcal{W}_{\geq 2}'})}$}\left[\left(1+\frac{\epsilon \mathbf{x}_u\mathbf{x}_v}{2}\right)\frac{d}{n}+\frac{d^2}{n^2}\right]}.
\end{align*}

Now observe that
\begin{align*}
|&E_{\geq2}^b(H_{\mathcal{W}_{\geq 2}})\setminus E_{\geq 2}(H_{\mathcal{W}_{\geq 2}'})|\\
&=\big|E_{\geq2}^b(H_{\mathcal{W}_{\geq 2}})\setminus \big(E_{\geq2}^a(H_{\mathcal{W}_{\geq 2}'})\cup E_{\geq2}^b(H_{\mathcal{W}_{\geq 2}'})\big)\big|\\
&=\big|\big(E_{\geq2}^b(H_{\mathcal{W}_{\geq 2}})\setminus E_{\geq2}^a(H_{\mathcal{W}_{\geq 2}'})\big)\setminus E_{\geq2}^b(H_{\mathcal{W}_{\geq 2}'})\big|\\
&\stackrel{(\dagger)}{=}|E_{\geq2}^b(H_{\mathcal{W}_{\geq 2}})\setminus E_{\geq2}^a(H_{\mathcal{W}_{\geq 2}'})|- |E_{\geq2}^b(H_{\mathcal{W}_{\geq 2}'})|\\
&=|E_{\geq2}^{b,i}(H_{\mathcal{W}_{\geq 2}})\setminus E_{\geq2}^a(H_{\mathcal{W}_{\geq 2}'})|+|E_{\geq2}^{b,o}(H_{\mathcal{W}_{\geq 2}})\setminus E_{\geq2}^a(H_{\mathcal{W}_{\geq 2}'})| -|E_{\geq2}^{b,i}(H_{\mathcal{W}_{\geq 2}'})|-|E_{\geq2}^{b,o}(H_{\mathcal{W}_{\geq 2}'})|\\
&=|E_{\geq2}^{b,i}(H_{\mathcal{W}_{\geq 2}})\setminus E_{\geq2}^a(H_{\mathcal{W}_{\geq 2}'})|+|E_{\geq2}^{b,o}(H_{\mathcal{W}_{\geq 2}})\setminus E_{\geq2}^a(H_{\mathcal{W}_{\geq 2}'})|\\
&\quad\quad\quad\quad\quad\quad -|E_{\geq2}^{b,i}(H_{\mathcal{W}_{\geq 2}'})|-|E_{\geq2}^{b,o}(H_{\mathcal{W}_{\geq 2}'})\cap E_{\geq2}^{b,i}(H_{\mathcal{W}_{\geq 2}})|-|E_{\geq2}^{b,o}(H_{\mathcal{W}_{\geq 2}'})\cap E_{\geq2}^{b,o}(H_{\mathcal{W}_{\geq 2}})|,
\end{align*}
where $(\dagger)$ follows from the fact that $E_{\geq2}^b(H_{\mathcal{W}_{\geq 2}'})\subseteq E_{\geq2}^b(H_{\mathcal{W}_{\geq 2}})$ and $E_{\geq2}^b(H_{\mathcal{W}_{\geq 2}'})\cap E_{\geq2}^a(H_{\mathcal{W}_{\geq 2}'})=\varnothing$, which imply that $E_{\geq2}^b(H_{\mathcal{W}_{\geq 2}'})\subseteq E_{\geq2}^b(H_{\mathcal{W}_{\geq 2}})\setminus E_{\geq2}^a(H_{\mathcal{W}_{\geq 2}'})$. Therefore,

\begin{align*}
\frac{K_{\mathcal{W}_{\geq 2}'}(\mathbf{x})}{K_{\mathcal{W}_{\geq 2}}(\mathbf{x})}&\leq  \left(1+\frac{1}{\sqrt{n}}\right)\cdot\left(\frac{3 d}{n}\right)^s\cdot\left(\frac{\epsilon d}{2n}\right)^{|E_1(H_{\mathcal{W}_{\geq 2}'})|-|E_1(H_{\mathcal{W}_{\geq 2}})|}\cdot\frac{\left(\frac{2d}{n^{\frac{3}{2}}}\right)^{-|E_{\geq2}^b(H_{\mathcal{W}_{\geq 2}})\setminus E_{\geq 2}(H_{\mathcal{W}_{\geq 2}'})|}}{\resizebox{0.17\textwidth}{!}{$\displaystyle\prod_{uv\in E_{\geq2}^a(H_{\mathcal{W}_{\geq 2}})\setminus E_{\geq2}^a(H_{\mathcal{W}_{\geq 2}'})}$}\left[\left(1+\frac{\epsilon \mathbf{x}_u\mathbf{x}_v}{2}\right)\frac{d}{n}+\frac{d^2}{n^2}\right]}.
\end{align*}

Now since $E_{\geq2}^b(H_{\mathcal{W}_{\geq 2}'})\subseteq E_{\geq2}^b(H_{\mathcal{W}_{\geq 2}})$, we have
\begin{align*}
\big(E_{\geq2}^a(H_{\mathcal{W}_{\geq 2}})\setminus E_{\geq2}^a(H_{\mathcal{W}_{\geq 2}'})\big)\cap E_{\geq2}^b(H_{\mathcal{W}_{\geq 2}'})&\subseteq\big(E_{\geq2}^a(H_{\mathcal{W}_{\geq 2}})\setminus E_{\geq2}^a(H_{\mathcal{W}_{\geq 2}'})\big)\cap E_{\geq2}^b(H_{\mathcal{W}_{\geq 2}})\\
&\subseteq E_{\geq2}^a(H_{\mathcal{W}_{\geq 2}})\cap E_{\geq2}^b(H_{\mathcal{W}_{\geq 2}})=\varnothing,
\end{align*}
hence,
\begin{align*}
E_{\geq2}^a(H_{\mathcal{W}_{\geq 2}})\setminus E_{\geq2}^a(H_{\mathcal{W}_{\geq 2}'})&= \big(E_{\geq2}^a(H_{\mathcal{W}_{\geq 2}})\setminus E_{\geq2}^a(H_{\mathcal{W}_{\geq 2}'})\big)\setminus E_{\geq2}^b(H_{\mathcal{W}_{\geq 2}'})\\
&= E_{\geq2}^a(H_{\mathcal{W}_{\geq 2}})\setminus \big(E_{\geq2}^a(H_{\mathcal{W}_{\geq 2}'})\cup E_{\geq2}^b(H_{\mathcal{W}_{\geq 2}'})\big)\\
&= E_{\geq2}^a(H_{\mathcal{W}_{\geq 2}})\setminus E_{\geq 2}(H_{\mathcal{W}_{\geq 2}'}).
\end{align*}
Therefore, 
\begin{equation}
\label{eq:eq_lem_K_W_2_Comparison_4}
\begin{aligned}
\frac{K_{\mathcal{W}_{\geq 2}'}(\mathbf{x})}{K_{\mathcal{W}_{\geq 2}}(\mathbf{x})}&\leq  \frac{\left(1+\frac{1}{\sqrt{n}}\right)\cdot\left(\frac{3 d}{n}\right)^s\cdot\left(\frac{\epsilon d}{2n}\right)^{|E_1(H_{\mathcal{W}_{\geq 2}'})|-|E_1(H_{\mathcal{W}_{\geq 2}})|}}{\left(\frac{2d}{n^{\frac{3}{2}}}\right)^{|E_{\geq2}^b(H_{\mathcal{W}_{\geq 2}})\setminus E_{\geq 2}(H_{\mathcal{W}_{\geq 2}'})|}\cdot \resizebox{0.17\textwidth}{!}{$\displaystyle\prod_{uv\in E_{\geq2}^a(H_{\mathcal{W}_{\geq 2}})\setminus E_{\geq 2}(H_{\mathcal{W}_{\geq 2}'})}$}\left[\left(1+\frac{\epsilon \mathbf{x}_u\mathbf{x}_v}{2}\right)\frac{d}{n}+\frac{d^2}{n^2}\right]}\\
&\leq  \frac{\left(1+\frac{1}{\sqrt{n}}\right)\cdot\left(\frac{3 d}{n}\right)^{s+|E_1(H_{\mathcal{W}_{\geq 2}'})|-|E_1(H_{\mathcal{W}_{\geq 2}})|}}{\left(\frac{2d}{n^{\frac{3}{2}}}\right)^{|E_{\geq2}^b(H_{\mathcal{W}_{\geq 2}})\setminus E_{\geq 2}(H_{\mathcal{W}_{\geq 2}'})|}\cdot \resizebox{0.17\textwidth}{!}{$\displaystyle\prod_{uv\in E_{\geq2}^a(H_{\mathcal{W}_{\geq 2}})\setminus E_{\geq 2}(H_{\mathcal{W}_{\geq 2}'})}$}\left[\left(1+\frac{\epsilon \mathbf{x}_u\mathbf{x}_v}{2}\right)\frac{d}{n}+\frac{d^2}{n^2}\right]}.
\end{aligned}
\end{equation}

Now notice that $H_{\mathcal{W}_{\geq 2}}$ is obtained from $H_{\mathcal{W}_{\geq 2}'}$ by adding the self-avoiding-walk $W$ of length $s$. This means that
\begin{align*}
E_1(H_{\mathcal{W}_{\geq 2}})&=\big(E_1(H_{\mathcal{W}_{\geq 2}'})\setminus W\big)\cup\big(W\setminus E(H_{\mathcal{W}_{\geq 2}'})\big).
\end{align*}

Therefore, 
\begin{align*}
s+|E_1(H_{\mathcal{W}_{\geq 2}'})|-|E_1(H_{\mathcal{W}_{\geq 2}})|&=|W|+|E_1(H_{\mathcal{W}_{\geq 2}'})|-|E_1(H_{\mathcal{W}_{\geq 2}'})\setminus W|-|W\setminus E(H_{\mathcal{W}_{\geq 2}'})|\\
&= \big|E_1(H_{\mathcal{W}_{\geq 2}'})\cap W\big|+\big|E(H_{\mathcal{W}_{\geq 2}'})\cap W\big|\\
&= \big|E_1(H_{\mathcal{W}_{\geq 2}'})\cap W\big|+\big|E_1(H_{\mathcal{W}_{\geq 2}'})\cap W\big|+\big|E_{\geq 2}(H_{\mathcal{W}_{\geq 2}'})\cap W\big|\\
&= 2\cdot\big|E_1(H_{\mathcal{W}_{\geq 2}'})\cap W\big|+\big|E_{\geq 2}(H_{\mathcal{W}_{\geq 2}'})\cap W\big|.
\end{align*}
Now it is easy to see that
\begin{align*}
E_1(H_{\mathcal{W}_{\geq 2}'})\cap W&=E_{\geq 2}(H_{\mathcal{W}_{\geq 2}})\setminus E_{\geq 2}(H_{\mathcal{W}_{\geq 2}'})\\
&=\big(E_{\geq2}^a(H_{\mathcal{W}_{\geq 2}})\setminus E_{\geq 2}(H_{\mathcal{W}_{\geq 2}'})\big)\cup \big(E_{\geq2}^b(H_{\mathcal{W}_{\geq 2}})\setminus E_{\geq 2}(H_{\mathcal{W}_{\geq 2}'})\big).
\end{align*}
Thus,
\begin{align*}
s+|E_1(H_{\mathcal{W}_{\geq 2}'})|-|E_1(H_{\mathcal{W}_{\geq 2}})|= 2\cdot\big|E_{\geq2}^a(H_{\mathcal{W}_{\geq 2}})\setminus E_{\geq 2}(H_{\mathcal{W}_{\geq 2}'})\big|+2\cdot\big|E_{\geq2}^b&(H_{\mathcal{W}_{\geq 2}})\setminus E_{\geq 2}(H_{\mathcal{W}_{\geq 2}'})\big|\\
&\;\;+\big|E_{\geq 2}(H_{\mathcal{W}_{\geq 2}'})\cap W\big|.
\end{align*}

By combining this with \cref{eq:eq_lem_K_W_2_Comparison_4}, it follows that for $n$ large enough, we have

\begin{align*}
&\frac{K_{\mathcal{W}_{\geq 2}'}(\mathbf{x})}{K_{\mathcal{W}_{\geq 2}}(\mathbf{x})}\\
&\leq  \frac{\left(1+\frac{1}{\sqrt{n}}\right)\cdot\left(\frac{3 d}{n}\right)^{2|E_{\geq2}^a(H_{\mathcal{W}_{\geq 2}})\setminus E_{\geq 2}(H_{\mathcal{W}_{\geq 2}'})|+2|E_{\geq2}^b(H_{\mathcal{W}_{\geq 2}})\setminus E_{\geq 2}(H_{\mathcal{W}_{\geq 2}'})|+|E_{\geq 2}(H_{\mathcal{W}_{\geq 2}'})\cap W|}}{\left(\frac{2d}{n^{\frac{3}{2}}}\right)^{|E_{\geq2}^b(H_{\mathcal{W}_{\geq 2}})\setminus E_{\geq 2}(H_{\mathcal{W}_{\geq 2}'})|}\cdot \resizebox{0.17\textwidth}{!}{$\displaystyle\prod_{uv\in E_{\geq2}^a(H_{\mathcal{W}_{\geq 2}})\setminus E_{\geq 2}(H_{\mathcal{W}_{\geq 2}'})}$}\left[\left(1+\frac{\epsilon \mathbf{x}_u\mathbf{x}_v}{2}\right)\frac{d}{n}+\frac{d^2}{n^2}\right]}\\
&=  \frac{\left(1+\frac{1}{\sqrt{n}}\right)\cdot \left(\frac{3 d}{n}\right)^{|E_{\geq 2}(H_{\mathcal{W}_{\geq 2}'})\cap W|}}{\left(\frac{2\sqrt{n}}{9d}\right)^{|E_{\geq2}^b(H_{\mathcal{W}_{\geq 2}})\setminus E_{\geq 2}(H_{\mathcal{W}_{\geq 2}'})|}\cdot \left(\frac{n^2}{9d^2}\right)^{|E_{\geq2}^a(H_{\mathcal{W}_{\geq 2}})\setminus E_{\geq 2}(H_{\mathcal{W}_{\geq 2}'})|}\cdot\resizebox{0.17\textwidth}{!}{$\displaystyle\prod_{uv\in E_{\geq2}^a(H_{\mathcal{W}_{\geq 2}})\setminus E_{\geq 2}(H_{\mathcal{W}_{\geq 2}'})}$}\left[\left(1+\frac{\epsilon \mathbf{x}_u\mathbf{x}_v}{2}\right)\frac{d}{n}+\frac{d^2}{n^2}\right]}.
\end{align*}

Now since $$|E_{\geq 2}(H_{\mathcal{W}_{\geq 2}})\setminus E_{\geq 2}(H_{\mathcal{W}_{\geq 2}'})|=|E_{\geq2}^a(H_{\mathcal{W}_{\geq 2}})\setminus E_{\geq 2}(H_{\mathcal{W}_{\geq 2}'})|+|E_{\geq2}^b(H_{\mathcal{W}_{\geq 2}})\setminus E_{\geq 2}(H_{\mathcal{W}_{\geq 2}'})|,$$ we get
\begin{align*}
&\frac{K_{\mathcal{W}_{\geq 2}'}(\mathbf{x})}{K_{\mathcal{W}_{\geq 2}}(\mathbf{x})}\\
&\leq \resizebox{0.95\textwidth}{!}{$\displaystyle  \frac{\left(1+\frac{1}{\sqrt{n}}\right)\cdot \left(\frac{3 d}{n}\right)^{|E_{\geq 2}(H_{\mathcal{W}_{\geq 2}'})\cap W|}}{\left(\frac{2\sqrt{n}}{9 d}\right)^{|E_{\geq 2}(H_{\mathcal{W}_{\geq 2}})\setminus E_{\geq 2}(H_{\mathcal{W}_{\geq 2}'})|}\cdot \left(\frac{n\sqrt{n}}{2d}\right)^{|E_{\geq2}^a(H_{\mathcal{W}_{\geq 2}})\setminus E_{\geq 2}(H_{\mathcal{W}_{\geq 2}'})|}\cdot\resizebox{0.17\textwidth}{!}{$\displaystyle\prod_{uv\in E_{\geq2}^a(H_{\mathcal{W}_{\geq 2}})\setminus E_{\geq 2}(H_{\mathcal{W}_{\geq 2}'})}$}\left[\left(1+\frac{\epsilon \mathbf{x}_u\mathbf{x}_v}{2}\right)\frac{d}{n}+\frac{d^2}{n^2}\right]}$}.
\end{align*}

\end{proof}

The following lemma proves an upper bound on $K_{\mathcal{W}_{\geq 2}}(\mathbf{x})$ for every $\mathcal{W}_{\geq 2}\subsetneq\mathcal{W}_{\geq 2}(H)$.

\begin{lemma}
\label{lem:lem_K_W_2_Upper_Bound}
If $H\in\bsaw{s}{t}$ and $n$ is large enough, then for every $\mathcal{W}_{\geq 2}\subsetneq\mathcal{W}_{\geq 2}(H)$, we have

\begin{align*}
K_{\mathcal{W}_{\geq 2}}(\mathbf{x})&\leq  \frac{9 d}{\sqrt{n}}\cdot \hat{K}_{\mathcal{W}_{\geq 2}(H)}(\mathbf{x}),
\end{align*}
where
\begin{equation}
\label{eq:eq_K_tilde_W_2}
\resizebox{0.97\textwidth}{!}{$\displaystyle
\hat{K}_{\mathcal{W}_{\geq 2}(H)}(\mathbf{x})=  \frac{\left(\frac{\epsilon d}{2n}\right)^{|E_1(H_{\mathcal{W}_{\geq 2}(H)})|}}{\resizebox{0.14\textwidth}{!}{$\displaystyle\prod_{v\in\mathcal{L}_{\geq2}(H_{\mathcal{W}_{\geq 2}(H)})}$} n^{\frac{1}{4}\left(d^{H_{\mathcal{W}_{\geq 2}(H)}}_{\geq 2}(v)-\Delta\right)}}\cdot \left(\frac{2d}{n}\right)^{|E_{\geq2}^b(H_{\mathcal{W}_{\geq 2}(H)})|}\cdot\prod_{uv\in E_{\geq2}^a(H_{\mathcal{W}_{\geq 2}(H)})}\left[\left(1+\frac{\epsilon \mathbf{x}_u\mathbf{x}_v}{2}\right)\frac{d}{n}+\frac{3d^2}{n\sqrt{n}}\right].$}
\end{equation}
\end{lemma}
\begin{proof}
Let $k=|\mathcal{W}_{\geq 2}(H)\setminus \mathcal{W}_2|$, and let $W_1,\ldots,W_k\in \mathcal{W}_{\geq 2}(H)$ be such that
$$\mathcal{W}_{\geq 2}(H)\setminus \mathcal{W}_2=\big\{W_1,\ldots,W_k\big\}.$$
Now define $\mathcal{W}_{\geq 2}^{(0)}=\mathcal{W}_{\geq 2}$ and for every $1\leq m \leq k$, define
$$\mathcal{W}_{\geq 2}^{(m)}=\mathcal{W}_{\geq 2}\cup \big\{W_1,\ldots,W_m\big\}.$$
Clearly, $\mathcal{W}_{\geq 2}^{(k)}=\mathcal{W}_{\geq 2}(H)$, and for every $1\leq m\leq k$, we have $\mathcal{W}_{\geq 2}^{(m)}=\mathcal{W}_{\geq 2}^{(m-1)}\cup\{W_{m}\}$.

From \Cref{lem:lem_K_W_2_Comparison}, we know that for every $1\leq m\leq k$, we have

\begin{align*}
\frac{K_{\mathcal{W}_{\geq 2}^{(m-1)}}(\mathbf{x})}{K_{\mathcal{W}_{\geq 2}^{(m)}}(\mathbf{x})}&\leq \frac{\left(1+\frac{1}{\sqrt{n}}\right)\cdot \left(\frac{9 d}{2\sqrt{n}}\right)^{|E_{\geq 2}(H_{\mathcal{W}_{\geq 2}^{(m)}})\setminus E_{\geq 2}(H_{\mathcal{W}_{\geq 2}^{(m-1)}})|}\cdot \left(\frac{2 d}{n\sqrt{n}}\right)^{|E_{\geq2}^a(H_{\mathcal{W}_{\geq 2}^{(m)}})\setminus E_{\geq 2}(H_{\mathcal{W}_{\geq 2}^{(m-1)}})|}}{\left(\frac{n}{3 d}\right)^{|E_{\geq 2}(H_{\mathcal{W}_{\geq 2}^{(m-1)}})\cap W_m|}\cdot \resizebox{0.17\textwidth}{!}{$\displaystyle\prod_{uv\in E_{\geq2}^a(H_{\mathcal{W}_{\geq 2}^{(m)}})\setminus E_{\geq 2}(H_{\mathcal{W}_{\geq 2}^{(m-1)}})}$}\left[\left(1+\frac{\epsilon \mathbf{x}_u\mathbf{x}_v}{2}\right)\frac{d}{n}+\frac{d^2}{n^2}\right]}.
\end{align*}

Therefore,
\begin{equation}
\label{eq:eq_lem_K_W_2_Upper_Bound_eq_1}
\begin{aligned}
\frac{K_{\mathcal{W}_{\geq 2}}(\mathbf{x})}{K_{\mathcal{W}_{\geq 2}(H)}(\mathbf{x})}&=\frac{K_{\mathcal{W}_{\geq 2}^{(0)}}(\mathbf{x})}{K_{\mathcal{W}_{\geq 2}^{(k)}}(\mathbf{x})}=\prod_{m=1}^{k}\frac{K_{\mathcal{W}_{\geq 2}^{(m-1)}}(\mathbf{x})}{K_{\mathcal{W}_{\geq 2}^{(m)}}(\mathbf{x})}\\
&\leq \frac{\left(1+\frac{1}{\sqrt{n}}\right)^k\cdot \resizebox{0.026\textwidth}{!}{$\displaystyle\prod_{m=1}^k$}\left[\left(\frac{9 d}{2\sqrt{n}}\right)^{|E_{\geq 2}(H_{\mathcal{W}_{\geq 2}^{(m)}})\setminus E_{\geq 2}(H_{\mathcal{W}_{\geq 2}^{(m-1)}})|}\cdot \left(\frac{2d}{n\sqrt{n}}\right)^{|E_{\geq2}^a(H_{\mathcal{W}_{\geq 2}^{(m)}})\setminus E_{\geq 2}(H_{\mathcal{W}_{\geq 2}^{(m-1)}})|}\right]}{\resizebox{0.026\textwidth}{!}{$\displaystyle\prod_{m=1}^k$}\left[\left(\frac{n}{3 d}\right)^{|E_{\geq 2}(H_{\mathcal{W}_{\geq 2}^{(m-1)}})\cap W_m|}\cdot\;\;\resizebox{0.17\textwidth}{!}{$\displaystyle\prod_{uv\in E_{\geq2}^a(H_{\mathcal{W}_{\geq 2}^{(m)}})\setminus E_{\geq 2}(H_{\mathcal{W}_{\geq 2}^{(m-1)}})}$}\left[\left(1+\frac{\epsilon \mathbf{x}_u\mathbf{x}_v}{2}\right)\frac{d}{n}+\frac{d^2}{n^2}\right]\right]}.
\end{aligned}
\end{equation}

Now since $$\mathcal{W}_{\geq 2}^{(0)}\subseteq\mathcal{W}_{\geq 2}^{(1)}\subseteq\ldots\subseteq \mathcal{W}_{\geq 2}^{(m)},$$ we have
$$E_{\geq 2}(\mathcal{W}_{\geq 2}^{(0)})\subseteq E_{\geq 2}(\mathcal{W}_{\geq 2}^{(1)})\subseteq\ldots\subseteq E_{\geq 2}(\mathcal{W}_{\geq 2}^{(m)}),$$
and
\begin{equation}
\label{eq:eq_lem_K_W_2_Upper_Bound_eq_2}
\begin{aligned}
\prod_{m=1}^{k} \left(\frac{9d}{2\sqrt{n}}\right)^{|E_{\geq 2}(H_{\mathcal{W}_{\geq 2}^{(m)}})\setminus E_{\geq 2}(H_{\mathcal{W}_{\geq 2}^{(m-1)}})|}&= \prod_{m=1}^{k} \left(\frac{9d}{2\sqrt{n}}\right)^{|E_{\geq 2}(H_{\mathcal{W}_{\geq 2}^{(m)}})|- | E_{\geq 2}(H_{\mathcal{W}_{\geq 2}^{(m-1)}})|}\\
&= \left(\frac{9d}{2\sqrt{n}}\right)^{|E_{\geq 2}(H_{\mathcal{W}_{\geq 2}^{(k)}})|- | E_{\geq 2}(H_{\mathcal{W}_{\geq 2}^{(0)}})|}\\
&= \left(\frac{9d}{2\sqrt{n}}\right)^{|E_{\geq 2}(H_{\mathcal{W}_{\geq 2}(H)})|- | E_{\geq 2}(H_{\mathcal{W}_{\geq 2}})|}.
\end{aligned}
\end{equation}

On the other hand, we have
\begin{equation}
\label{eq:eq_lem_K_W_2_Upper_Bound_eq_3}
\left(1+\frac{1}{\sqrt{n}}\right)^k=1+O\left(\frac{k}{\sqrt{n}}\right)\leq 1+O\left(\frac{|\mathcal{W}_{\geq 2}(H)|}{\sqrt{n}}\right)\leq 1+O\left(\frac{|\mathcal{W}(H)|}{\sqrt{n}}\right) = 1+O\left(\frac{t}{\sqrt{n}}\right)\leq 2,
\end{equation}
where the last inequality is true for $n$ large enough.

Furthermore, from \Cref{def:def_K_W_2_Y}, we have
\begin{equation}
\label{eq:eq_K_W_2_H_Y_expression}
\begin{aligned}
K_{\mathcal{W}_{\geq 2}(H)}(\mathbf{x})&=\resizebox{0.82\textwidth}{!}{$\displaystyle\frac{\left(\frac{\epsilon d}{2n}\right)^{|E_1(H_{\mathcal{W}_{\geq 2}(H)})|}}{\resizebox{0.14\textwidth}{!}{$\displaystyle\prod_{v\in\mathcal{L}_{\geq2}(H_{\mathcal{W}_{\geq 2}(H)})}$} n^{\frac{1}{4}\left(d^{H_{\mathcal{W}_{\geq 2}(H)}}_{\geq 2}(v)-\Delta\right)}}\cdot\left(\frac{2d}{n}\right)^{|E_{\geq2}^b(H_{\mathcal{W}_{\geq 2}(H)})|}\cdot\prod_{uv\in E_{\geq2}^a(H_{\mathcal{W}_{\geq 2}(H)})}\left[\left(1+\frac{\epsilon \mathbf{x}_u\mathbf{x}_v}{2}\right)\frac{d}{n}+\frac{d^2}{n^2}\right]$}\\
&=R_{\mathcal{W}_{\geq 2}(H)}(\mathbf{x})\cdot\prod_{uv\in E_{\geq2}^a(H_{\mathcal{W}_{\geq 2}(H)})}\left[\left(1+\frac{\epsilon \mathbf{x}_u\mathbf{x}_v}{2}\right)\frac{d}{n}+\frac{d^2}{n^2}\right],
\end{aligned}
\end{equation}
where
$$R_{\mathcal{W}_{\geq 2}(H)}(\mathbf{x})=\frac{\left(\frac{\epsilon d}{2n}\right)^{|E_1(H_{\mathcal{W}_{\geq 2}(H)})|}}{\resizebox{0.09\textwidth}{!}{$\displaystyle\prod_{v\in\mathcal{L}_{\geq2}(H_{\mathcal{W}_{\geq 2}(H)})}$} n^{\frac{1}{4}\left(d^{H_{\mathcal{W}_{\geq 2}(H)}}_{\geq 2}(v)-\Delta\right)}}\cdot \left(\frac{2d}{n}\right)^{|E_{\geq2}^b(H_{\mathcal{W}_{\geq 2}(H)})|}.$$

Therefore,
\begin{align*}
&\frac{K_{\mathcal{W}_{\geq 2}(H)}(\mathbf{x})}{\resizebox{0.026\textwidth}{!}{$\displaystyle\prod_{m=1}^k$}\;\;\resizebox{0.17\textwidth}{!}{$\displaystyle\prod_{uv\in E_{\geq2}^a(H_{\mathcal{W}_{\geq 2}^{(m)}})\setminus E_{\geq 2}(H_{\mathcal{W}_{\geq 2}^{(m-1)}})}$}\left[\left(1+\frac{\epsilon \mathbf{x}_u\mathbf{x}_v}{2}\right)\frac{d}{n}+\frac{d^2}{n^2}\right]}\\
&\quad\quad\quad\quad\quad=R_{\mathcal{W}_{\geq 2}(H)}(\mathbf{x})\cdot\frac{\resizebox{0.085\textwidth}{!}{$\displaystyle\prod_{uv\in E_{\geq2}^a(H_{\mathcal{W}_{\geq 2}(H)})}$}\left[\left(1+\frac{\epsilon \mathbf{x}_u\mathbf{x}_v}{2}\right)\frac{d}{n}+\frac{d^2}{n^2}\right]}{\resizebox{0.026\textwidth}{!}{$\displaystyle\prod_{m=1}^k$}\;\;\resizebox{0.17\textwidth}{!}{$\displaystyle\prod_{uv\in E_{\geq2}^a(H_{\mathcal{W}_{\geq 2}^{(m)}})\setminus E_{\geq 2}(H_{\mathcal{W}_{\geq 2}^{(m-1)}})}$}\left[\left(1+\frac{\epsilon \mathbf{x}_u\mathbf{x}_v}{2}\right)\frac{d}{n}+\frac{d^2}{n^2}\right]}\\
&\quad\quad\quad\quad\quad=R_{\mathcal{W}_{\geq 2}(H)}(\mathbf{x})\cdot \prod_{uv\in E_{\geq2}^a(H_{\mathcal{W}_{\geq 2}(H)})\setminus\left[\resizebox{0.026\textwidth}{!}{$\displaystyle\bigcup_{m=1}^k$} \left(E_{\geq2}^a(H_{\mathcal{W}_{\geq 2}^{(m)}})\setminus E_{\geq 2}(H_{\mathcal{W}_{\geq 2}^{(m-1)}})\right)\right]}\left[\left(1+\frac{\epsilon \mathbf{x}_u\mathbf{x}_v}{2}\right)\frac{d}{n}+\frac{d^2}{n^2}\right].
\end{align*}
By combining this with \cref{eq:eq_lem_K_W_2_Upper_Bound_eq_1}, \cref{eq:eq_lem_K_W_2_Upper_Bound_eq_2}, \cref{eq:eq_lem_K_W_2_Upper_Bound_eq_3}, we get

\begin{align*}
K_{\mathcal{W}_{\geq 2}}(\mathbf{x})&\leq  2\cdot \left(\frac{9d}{2\sqrt{n}}\right)^{|E_{\geq 2}(H_{\mathcal{W}_{\geq 2}(H)})\setminus E_{\geq 2}(H_{\mathcal{W}_{\geq 2}})|}\cdot  R_{\mathcal{W}_{\geq 2}(H)}(\mathbf{x})\\
&\quad\quad\quad\quad\quad\quad\quad\quad\times\prod_{m=1}^k\left[ \left(\frac{3d}{n}\right)^{|E_{\geq 2}(H_{\mathcal{W}_{\geq 2}^{(m-1)}})\cap W_m|}\cdot\left(\frac{2d}{n\sqrt{n}}\right)^{|E_{\geq2}^a(H_{\mathcal{W}_{\geq 2}^{(m)}})\setminus E_{\geq 2}(H_{\mathcal{W}_{\geq 2}^{(m-1)}})|}\right]\\
&\quad\quad\quad\quad\quad\quad\quad\quad
\times\prod_{uv\in E_{\geq2}^a(H_{\mathcal{W}_{\geq 2}(H)})\setminus\left[\resizebox{0.026\textwidth}{!}{$\displaystyle\bigcup_{m=1}^k$} \left(E_{\geq2}^a(H_{\mathcal{W}_{\geq 2}^{(m)}})\setminus E_{\geq 2}(H_{\mathcal{W}_{\geq 2}^{(m-1)}})\right)\right]}\left[\left(1+\frac{\epsilon \mathbf{x}_u\mathbf{x}_v}{2}\right)\frac{d}{n}+\frac{d^2}{n^2}\right]\\
&\leq  2\cdot \left(\frac{9d}{2\sqrt{n}}\right)^{|E_{\geq 2}(H_{\mathcal{W}_{\geq 2}(H)})\setminus E_{\geq 2}(H_{\mathcal{W}_{\geq 2}})|}\cdot  R_{\mathcal{W}_{\geq 2}(H)}(\mathbf{x})\cdot\prod_{m=1}^k \left(\frac{3d}{n}\right)^{|E_{\geq 2}(H_{\mathcal{W}_{\geq 2}^{(m-1)}})\cap W_m|}\\
&\quad\quad\quad\quad\quad\quad\quad\quad\quad\quad\quad\quad\quad\quad\quad\quad\times\prod_{uv\in E_{\geq2}^a(H_{\mathcal{W}_{\geq 2}(H)})}\left[\left(1+\frac{\epsilon \mathbf{x}_u\mathbf{x}_v}{2}\right)\frac{d}{n}+\frac{d^2}{n^2}+\frac{2d}{n\sqrt{n}}\right]\\
&\leq  2\cdot \left(\frac{9d}{2\sqrt{n}}\right)^{|E_{\geq 2}(H_{\mathcal{W}_{\geq 2}(H)})\setminus E_{\geq 2}(H_{\mathcal{W}_{\geq 2}})|}\cdot  R_{\mathcal{W}_{\geq 2}(H)}(\mathbf{x})\cdot\prod_{m=1}^k \left(\frac{3d}{n}\right)^{|E_{\geq 2}(H_{\mathcal{W}_{\geq 2}^{(m-1)}})\cap W_m|}\\
&\quad\quad\quad\quad\quad\quad\quad\quad\quad\quad\quad\quad\quad\quad\quad\quad\times\prod_{uv\in E_{\geq2}^a(H_{\mathcal{W}_{\geq 2}(H)})}\left[\left(1+\frac{\epsilon \mathbf{x}_u\mathbf{x}_v}{2}\right)\frac{d}{n}+\frac{3d^2}{n\sqrt{n}}\right]\\
&\leq  2\cdot \left(\frac{9 d}{2\sqrt{n}}\right)^{|E_{\geq 2}(H_{\mathcal{W}_{\geq 2}(H)})\setminus E_{\geq 2}(H_{\mathcal{W}_{\geq 2}})|+\resizebox{0.025\textwidth}{!}{$\displaystyle\sum_{m=1}^k$} |E_{\geq 2}(H_{\mathcal{W}_{\geq 2}^{(m-1)}})\cap W_m|}\cdot  R_{\mathcal{W}_{\geq 2}(H)}(\mathbf{x})\\
&\quad\quad\quad\quad\quad\quad\quad\quad\quad\quad\quad\quad\quad\quad\quad\quad\times\prod_{uv\in E_{\geq2}^a(H_{\mathcal{W}_{\geq 2}(H)})}\left[\left(1+\frac{\epsilon \mathbf{x}_u\mathbf{x}_v}{2}\right)\frac{d}{n}+\frac{3d^2}{n\sqrt{n}}\right].
\end{align*}

Now observe that
\begin{align*}
|E_{\geq 2}(H_{\mathcal{W}_{\geq 2}(H)})\setminus & E_{\geq 2}(H_{\mathcal{W}_{\geq 2}})|+\displaystyle\sum_{m=1}^k |E_{\geq 2}(H_{\mathcal{W}_{\geq 2}^{(m-1)}})\cap W_m|\\
&\geq |E_{\geq 2}(H_{\mathcal{W}_{\geq 2}(H)})\setminus E_{\geq 2}(H_{\mathcal{W}_{\geq 2}^{(k-1)}})|+|E_{\geq 2}(H_{\mathcal{W}_{\geq 2}^{(k-1)}})\cap W_k|\\
&= |E_1(H_{\mathcal{W}_{\geq 2}^{(k-1)}})\cap W_k|+|E_{\geq 2}(H_{\mathcal{W}_{\geq 2}^{(k-1)}})\cap W_k|= |E(H_{\mathcal{W}_{\geq 2}^{(k-1)}})\cap W_k|\geq 1,
\end{align*}
where the last inequality follows from the fact that $W_k\in\mathcal{W}_{\geq 2}(H)$, which means that there is at least one edge in $W_k$ that already appears in $E(H_{\mathcal{W}_{\geq 2}^{(k-1)}})$.

Therefore, for $n$ large enough, we have
\begin{align*}
\resizebox{0.97\textwidth}{!}{$\displaystyle
K_{\mathcal{W}_{\geq 2}}(\mathbf{x})\leq  2\cdot \left(\frac{9 d}{2\sqrt{n}}\right)\cdot  R_{\mathcal{W}_{\geq 2}(H)}(\mathbf{x})\cdot \prod_{uv\in E_{\geq2}^a(H_{\mathcal{W}_{\geq 2}(H)})}\left[\left(1+\frac{\epsilon \mathbf{x}_u\mathbf{x}_v}{2}\right)\frac{d}{n}+\frac{3d^2}{n\sqrt{n}}\right]=  \frac{9 d}{\sqrt{n}}\cdot  \hat{K}_{\mathcal{W}_{\geq 2}(H)}(\mathbf{x}).$}
\end{align*}
\end{proof}

Now we are ready to prove \Cref{lem:lem_F_W_2_total_Upper_Bound}.

\begin{proof}[Proof of \Cref{lem:lem_F_W_2_total_Upper_Bound}]
We know from \Cref{lem:lem_K_W_2_Upper_Bound} that for every $\mathcal{W}_{\geq 2}\subsetneq\mathcal{W}_{\geq 2}(H)$, we have
\begin{align*}
K_{\mathcal{W}_{\geq 2}}(\mathbf{x})&\leq \frac{9d}{\sqrt{n}}\cdot  \hat{K}_{\mathcal{W}_{\geq 2}(H)}(\mathbf{x}).
\end{align*}

On the other hand, it is easy to see that
\begin{align*}
K_{\mathcal{W}_{\geq 2}(H)}(\mathbf{x})&\leq \hat{K}_{\mathcal{W}_{\geq 2}(H)}(\mathbf{x}).
\end{align*}

Now from \Cref{lem:lem_F_W_2_K_W_2} we get
\begin{equation}
\label{eq:eq_lem_F_W_2_total_Upper_Bound_eq_1}
\begin{aligned}
\sum_{\mathcal{W}_{\geq 2}\subseteq\mathcal{W}_{\geq 2}(H)}& F_{\mathcal{W}_{\geq 2}}(\mathbf{x})\\
&\leq \sum_{\mathcal{W}_{\geq 2}\subseteq\mathcal{W}_{\geq 2}(H)} K_{\mathcal{W}_{\geq 2}}(\mathbf{x})= K_{\mathcal{W}_{\geq 2}(H)}(\mathbf{x})+\sum_{\mathcal{W}_{\geq 2}\subsetneq\mathcal{W}_{\geq 2}(H)} K_{\mathcal{W}_{\geq 2}}(\mathbf{x})\\
&\leq \hat{K}_{\mathcal{W}_{\geq 2}(H)}(\mathbf{x})\cdot\left(1+\sum_{\mathcal{W}_{\geq 2}\subsetneq \mathcal{W}_{\geq 2}(H)} \frac{9d}{\sqrt{n}}\right)= \hat{K}_{\mathcal{W}_{\geq 2}(H)}(\mathbf{x})\cdot\left[1+ \frac{9d}{\sqrt{n}}\left(2^{|\mathcal{W}_{\geq 2}(H)|}-1\right)\right]\\
&\leq \hat{K}_{\mathcal{W}_{\geq 2}(H)}(\mathbf{x})\cdot\left(1+ \frac{9d}{\sqrt{n}}\cdot e^t\right)= \hat{K}_{\mathcal{W}_{\geq 2}(H)}(\mathbf{x})\cdot\left(1+ \frac{9d}{\sqrt{n}}\cdot n^K\right)\leq 2\hat{K}_{\mathcal{W}_{\geq 2}(H)}(\mathbf{x}),
\end{aligned}
\end{equation}
where the last inequality is true if $K\leq\frac{1}{100}$ and $n$ is large enough.

By noticing that $E_{\geq2}^a(H_{\mathcal{W}_{\geq 2}(H)})=E_{\geq2}^a(H)$, $E_{\geq2}^b(H_{\mathcal{W}_{\geq 2}(H)})=E_{\geq2}^b(H)$, $\mathcal{L}_{\geq2}(H_{\mathcal{W}_{\geq 2}(H)})=\mathcal{L}_{\geq2}(H)$ and that $d^{H_{\mathcal{W}_{\geq 2}(H)}}_{\geq 2}(v)=d^H_{\geq 2}(v)$ for every $v\in \mathcal{L}_{\geq2}(H_{\mathcal{W}_{\geq 2}(H)})=\mathcal{L}_{\geq2}(H)$, if we combine \cref{eq:eq_lem_F_W_2_total_Upper_Bound_eq_1} with \cref{eq:eq_K_tilde_W_2}, we get \cref{eq:eq_F_W_2_total_Upper_Bound}.
\end{proof}

\section{Tools for block self-avoiding walks}\label{sec:tools-bsaws}
\renewcommand{\bsaw}[2]{\text{BSAW}_{#1,#2}}
\renewcommand{\nbsaw}[2]{\text{NBSAW}_{#1,#2}}

\subsection{Splitting the expectation of block self-avoiding walks}\label{sec:splitting-expectations}

In this section we prove \Cref{fact:splitting-upper-bound}, \cref{fact:splitting-upper-bound-second-moment} and \Cref{fact:sbm-upperbound-removing-cycles}.

\begin{fact}[Restatement of \Cref{fact:splitting-upper-bound}]\label{fact:splitting-upper-bound-appendix}
	Let $\cB=\Set{B_1,\ldots,B_z}$ be a collection of disjoint connected graphs on at least $2$ vertices. 
	Then for any $x$, $H \in \cM_{\Set{\cF_i,\Psi_i}_{i=1}^{z}}(\cB)$ and $j \in [z]$
	\begin{align*}
		\E \overline{U}_H(\mathbf{x}) &\leq  \frac{1}{4}n^{-1/25A}\Paren{\frac{6}{\eps}}^{2\ell_j+2q_j}\E \overline{U}_{H\Paren{V,V\setminus B_j}}(\mathbf{x})\cdot \E \overline{U}_{H\Paren{V\setminus B_j}}(\mathbf{x})\cdot \E\overline{U}_{H(B_j)}(\mathbf{x})\,,\\
		\E \overline{U}_H(\mathbf{x}) &\geq  \frac{1}{4}n^{-1/25A}\E \overline{U}_{H\Paren{V,V\setminus B_i}}(\mathbf{x})\cdot \E \overline{U}_{H\Paren{V\setminus B_i}}(\mathbf{x})\cdot \E\overline{U}_{H(B_i)}(\mathbf{x})\,.
	\end{align*}
\begin{proof}
	The first inequality follows observing that there are at most $\ell_i+q_i$ edges in $E_{1}^a(H)$ incident to $H(B_i)$ and observing that the edges in the multiway cut separating $H(B_1),\ldots, H(B_z)$ have multiplicity $1$.
	The second inequality is a consequence of \cref{def:upper-bound-UH} and the fact that for any $B_i\in \cB$,
	the edges in the cut $H(V\setminus B_i, B_i)$ have multiplicity $1$.
\end{proof}
\end{fact}

\begin{fact}[Restatement of \cref{fact:splitting-upper-bound-second-moment}]\label{fact:splitting-upper-bound-second-moment-appendix}
	Consider the settings of \cref{thm:sbm-bound-second-moment-large-t}.
	Let $u,v\in [n]$ , let $\cB=\Set{B_1,\ldots,B_z}$ be a collection of disjoint connected graphs on at least $2$ vertices and let $B^{uv}$ be a (possibly empty) graph disjoint from any graph in $\cB$.
	Then for any $x$, $H \in \cM_{\Set{\cF_i,\Psi_i}_{i=1}^{z+1}}(\cB, B^{uv})$ and $i \in [z+1]$
	\begin{align*}
		\E \overline{U}_H(\mathbf{x}) &\leq  \frac{1}{4}n^{-1/25A} \Paren{\frac{6}{\eps}}^{2\ell_i+2q_i}\E \overline{U}_{H\Paren{V,V\setminus B_i}}(\mathbf{x})\cdot \E \overline{U}_{H\Paren{V\setminus B_i}}(\mathbf{x})\cdot \E\overline{U}_{H(B_i)}(\mathbf{x})\,,\\
		\E \overline{U}_H(\mathbf{x}) &\geq \frac{1}{4}n^{-1/25A} \E \overline{U}_{H\Paren{V,V\setminus B_i}}(\mathbf{x})\cdot \E \overline{U}_{H\Paren{V\setminus B_i}}(\mathbf{x})\cdot \E\overline{U}_{H(B_i)}(\mathbf{x})\,.
	\end{align*}
	\begin{proof}
		The proof is similar to the one of \Cref{fact:splitting-upper-bound}. There are at most $\ell_i+q_i$ edges in $E_{1}^a(H)$ incident to $H(B_i)$ and observing that the edges in the multiway cut separating $H(B_1),\ldots, H(B_z)$ have multiplicity $1$.
		The second inequality follows by \cref{def:upper-bound-UH} and definition of $\cM_{\Set{\cF_i,\Psi_i}_{i=1}^{z+1}}(\cB, B^{uv})$.
	\end{proof}
\end{fact}

\begin{fact}[Restatement of \Cref{fact:sbm-upperbound-removing-cycles}]\label{fact:sbm-upperbound-removing-cycles-appendix}
	Consider the settings of \Cref{thm:sbm-schatten-norm} and let $\Psi\geq 0$.
	Let $H\in \bsaw{s}{t}$ be a multigraph on at most $O(t)$ vertices and let $H^*$ be an induced sub-multigraph of $H$ satisfying:
	\begin{enumerate}
		\item the maximum $(\geq 2)$-degree in $H^*$ is $\Psi\geq 0\,,$
		\item all the edges in the cut $H(V(H), V(H)\setminus V(H^*))$ have multiplicity one in $H\,.$
	\end{enumerate}
	We denote $\ell,q\geq 0$ as the number of multiplicity-$1$ edges in $H^*$ and $H(V(H), V(H)\setminus V(H^*))$
	respectively. Let $Z$ be a set of vertices in $V(H^*)$ such that 
	$H(V(H^*)\setminus Z)$ has no multiplicity-$2$ cycles. Then
	\begin{align*}
		\E \overline{U}_H(\mathbf{x}) \leq & \frac{1}{4}n^{-1/25A}\Paren{\frac{6}{\epsilon}}^{2\ell+2q}\E \overline{U}_{H\Paren{V,V\setminus V(H^*)}}(\mathbf{x})\cdot \E \overline{U}_{H\Paren{V\setminus V(H^*)}}(\mathbf{x})
		 \\
		&\cdot \Paren{1+\frac{\eps}{2}}^{\Card{Z}\cdot \Psi}
		\cdot \Paren{\frac{2d}{n}}^{\Card{E^b_{\geq 2}(H^*)}}\cdot \underset{v\in V(H^*)}{\prod}\left(\frac{6}{\epsilon}\right)^{\max\Set{2d_1^H(v)-\tau,0}}\\
		&\cdot\frac{1}{\resizebox{0.055\textwidth}{!}{$\displaystyle\prod_{v\in\mathcal{L}_{\geq 2}(H^*)}$} n^{\frac{1}{4}\left(d^{H^*}_{\geq 2}(v)-\Delta\right)}}
		\left(\frac{\epsilon d}{2n}\right)^{|E_1(H^*)|}\prod_{e\in E_{\geq 2}^a(H^*)}\left[\frac{d}{n}+\frac{3d^2}{n\sqrt{n}} \right]\,,
	\end{align*}
	and
	\begin{align*}
		\E \overline{U}_H(\mathbf{x}) \geq & \frac{1}{4}n^{-1/25A}\E \overline{U}_{H\Paren{V,V\setminus V(H^*)}}(\mathbf{x})\cdot \E \overline{U}_{H\Paren{V\setminus V(H^*)}}(\mathbf{x}) \\
		&\cdot \Paren{\frac{2d}{n}}^{\Card{E^b_{\geq 2}(H^*)}}\cdot \underset{v\in V(H^*)}{\prod}\left(\frac{6}{\epsilon}\right)^{\max\Set{2d_1^H(v)-\tau,0}}\\
		&\cdot\frac{1}{\resizebox{0.055\textwidth}{!}{$\displaystyle\prod_{v\in\mathcal{L}_{\geq 2}(H^*)}$} 
		n^{\frac{1}{4}\left(d^{H^*}_{\geq 2}(v)-\Delta\right)}}\left(\frac{\epsilon d}{2n}\right)^{|E_1(H^*)|}\prod_{e\in E_{\geq 2}^a(H^*)}\left[\frac{d}{n}+\frac{3d^2}{n\sqrt{n}} \right]\,,
	\end{align*}
	\begin{proof}
		Consider the first inequality. We denote multigraph $H'$ as the multigraph obtained
		by removing all edges incident to vertice in $Z$. Then there is no multiplicity-$2$ cycles in $H$. 
		By the \cref{fact:splitting-upper-bound-second-moment-appendix}(or \cref{fact:splitting-upper-bound-appendix}), we have 
		\begin{align*}
			\E \overline{U}_H(\mathbf{x}) \leq \frac{1}{4}n^{-1/25A} \Paren{\frac{6}{\eps}}^{2\ell+2q}
			\E \overline{U}_{H\Paren{V,V\setminus V(H^*)}}(\mathbf{x})\cdot \E \overline{U}_{H\Paren{V\setminus V(H^*)}}(\mathbf{x})\cdot 
			\E\overline{U}_{H*}(\mathbf{x})
		\end{align*}
		It remains to bound $\E\overline{U}_{H*}(\mathbf{x})$. We note that
		\begin{align*}
			\E\overline{U}_{H*}(\mathbf{x})=&
			\Paren{\frac{2d}{n}}^{\Card{E^b_{\geq 2}(H^*)}}\cdot \underset{v\in V(H^*)}{\prod}\left(\frac{6}{\epsilon}\right)^{\max\Set{2d_1^H(v)-\tau,0}}\\
			&\cdot\frac{1}{\resizebox{0.055\textwidth}{!}{$\displaystyle\prod_{v\in\mathcal{L}_{\geq 2}(H^*)}$} n^{\frac{1}{4}\left(d^{H^*}_{\geq 2}(v)-\Delta\right)}}
			\left(\frac{\epsilon d}{2n}\right)^{|E_1(H^*)|}\\
			&\E\prod_{e\in E_{\geq 2}^a(H^*)}\left[\frac{d}{n}(1+\frac{\epsilon}{2} x_ix_j)+\frac{3d^2}{n\sqrt{n}} \right]\\
			\leq & \Paren{\frac{2d}{n}}^{\Card{E^b_{\geq 2}(H^*)}}\cdot \underset{v\in V(H^*)}{\prod}\left(\frac{6}{\epsilon}\right)^{\max\Set{2d_1^H(v)-\tau,0}}\\
			&\cdot\frac{1}{\resizebox{0.055\textwidth}{!}{$\displaystyle\prod_{v\in\mathcal{L}_{\geq 2}(H^*)}$} n^{\frac{1}{4}\left(d^{H^*}_{\geq 2}(v)-\Delta\right)}}
			\left(\frac{\epsilon d}{2n}\right)^{|E_1(H^*)|}\\
			& (1+\frac{\epsilon}{2})^{\Card{Z}\cdot \Psi}\E\prod_{e\in E_{\geq 2}^a(H^*)\cap E(H')}\left[\frac{d}{n}(1+\frac{\epsilon}{2} x_ix_j)+\frac{3d^2}{n\sqrt{n}} \right]\\
			= & \Paren{\frac{2d}{n}}^{\Card{E^b_{\geq 2}(H^*)}}\cdot \underset{v\in V(H^*)}{\prod}\left(\frac{6}{\epsilon}\right)^{\max\Set{2d_1^H(v)-\tau,0}}\\
			&\cdot\frac{1}{\resizebox{0.055\textwidth}{!}{$\displaystyle\prod_{v\in\mathcal{L}_{\geq 2}(H^*)}$} n^{\frac{1}{4}\left(d^{H^*}_{\geq 2}(v)-\Delta\right)}}
			\left(\frac{\epsilon d}{2n}\right)^{|E_1(H^*)|}\\
			&\cdot (1+\frac{\epsilon}{2})^{\Card{Z}\cdot \Psi} \prod_{e\in E_{\geq 2}^a(H^*)\cap E(H')}
			\Paren{\frac{d}{n}+\frac{3d^2}{n\sqrt{n}}}\\
			\leq & \Paren{\frac{2d}{n}}^{\Card{E^b_{\geq 2}(H^*)}}\cdot \underset{v\in V(H^*)}{\prod}\left(\frac{6}{\epsilon}\right)^{\max\Set{2d_1^H(v)-\tau,0}}\\
			&\cdot\frac{1}{\resizebox{0.055\textwidth}{!}{$\displaystyle\prod_{v\in\mathcal{L}_{\geq 2}(H^*)}$} n^{\frac{1}{4}\left(d^{H^*}_{\geq 2}(v)-\Delta\right)}}
			\left(\frac{\epsilon d}{2n}\right)^{|E_1(H^*)|}\\
			&\cdot (1+\frac{\epsilon}{2})^{\Card{Z}\cdot \Psi} \prod_{e\in E_{\geq 2}^a(H^*)}
			\Paren{\frac{d}{n}+\frac{3d^2}{n\sqrt{n}}}
		\end{align*}
		The first inequality follows from by observing that observing that $\Card{E(H^*)\setminus E(H')}\leq \Card{Z}\cdot \Psi$.
		The second inequality follows since $H'$ does not contain any multiplicity-$2$ 
		cycle. The first bound thus follows. 

		For the second inequality, we still use \cref{fact:splitting-upper-bound-second-moment-appendix}. 
		\begin{align*}
				\E \overline{U}_H(\mathbf{x}) \leq  \frac{1}{4}n^{-1/25A}
				\E \overline{U}_{H\Paren{V,V\setminus V(H^*)}}(\mathbf{x})\cdot \E \overline{U}_{H\Paren{V\setminus V(H^*)}}(\mathbf{x})\cdot 
				\E\overline{U}_{H*}(\mathbf{x})
		\end{align*}
		We note that for any $H^*$
		\begin{align*}
			\E\prod_{e\in E_{\geq 2}^a(H^*)}
			\Paren{\frac{d}{n}(1+\frac{\epsilon}{2}x_ix_j)+\frac{3d^2}{n\sqrt{n}}}
			\geq \prod_{e\in E_{\geq 2}^a(H^*)}
			\Paren{\frac{d}{n}+\frac{3d^2}{n\sqrt{n}}}
		\end{align*}
		Therefore we have
		\begin{align*}
			\E\overline{U}_{H*}(\mathbf{x})=&
			\Paren{\frac{2d}{n}}^{\Card{E^b_{\geq 2}(H^*)}}\cdot \underset{v\in V(H^*)}{\prod}\left(\frac{6}{\epsilon}\right)^{\max\Set{2d_1^H(v)-\tau,0}}\\
			&\cdot\frac{1}{\resizebox{0.055\textwidth}{!}{$\displaystyle\prod_{v\in\mathcal{L}_{\geq 2}(H^*)}$} n^{\frac{1}{4}\left(d^{H^*}_{\geq 2}(v)-\Delta\right)}}
			\left(\frac{\epsilon d}{2n}\right)^{|E_1(H^*)|}\\
			& \E\prod_{e\in E_{\geq 2}^a(H^*)}\left[\frac{d}{n}(1+\frac{\epsilon}{2} x_ix_j)+\frac{3d^2}{n\sqrt{n}} \right]\\
			\geq & \Paren{\frac{2d}{n}}^{\Card{E^b_{\geq 2}(H^*)}}\cdot \underset{v\in V(H^*)}{\prod}\left(\frac{6}{\epsilon}\right)^{\max\Set{2d_1^H(v)-\tau,0}}\\
			&\cdot\frac{1}{\resizebox{0.055\textwidth}{!}{$\displaystyle\prod_{v\in\mathcal{L}_{\geq 2}(H^*)}$} n^{\frac{1}{4}\left(d^{H^*}_{\geq 2}(v)-\Delta\right)}}
			\left(\frac{\epsilon d}{2n}\right)^{|E_1(H^*)|}	\\
			& \prod_{e\in E_{\geq 2}^a(H^*)}	\left[\frac{d}{n}+\frac{3d^2}{n\sqrt{n}} \right]
		\end{align*}
		The claim thus follows.
	\end{proof}
\end{fact}

\subsection{Counting block self-avoiding walks}\label{sec:counting-bsaw}

\subsubsection{Bounding number of BSAW given a subgraph}
In this section we provide the counting arguments needed in \Cref{sec:bounds-moments-Q}. We reuse the notation of such section and we assume the premises of \Cref{thm:sbm-schatten-norm} and \Cref{thm:sbm-bound-second-moment-large-t} to hold.

\begin{lemma}\label{lem:tree_high_degree_sum}
	Let $B_i$ be a connected subgraph of the underlying graph of a block self-avoiding walk $H\in M_{s,t}$ with $m_i$ vertices. Let $H(B_i)$ be the multigraph on $V(B_i)$ that is induced by $H(B_i)$. Suppose that:
\begin{itemize}
	\item The cut $H(V(B_i),V(H)\setminus V(B_i))$ consists of $\ell_i$ edges of multiplicity 1.
	\item All edges in $H(B_i)$ of multiplicity $\geq 2$ are in $B_i$.
	\item The number of edges of multiplicity 1 in $H(B_i)$ is $q_i$.
	\item The number of edges of multiplicity 2 in $H(B_i)$ is $h_i$.
	\item The edges of multiplicity larger than 2 satisfy
	$$\sum_{\substack{e\in H(B_i)\\ m_H(e)\geq 3}} m_H(e)= p_i.$$
\end{itemize}
	 Then we have
	 \begin{equation*}
		 \sum_{\substack{v\in V(B_i)\\ d^H(v)\geq 3 }} \left(\frac{m_H(v)}{2}-1\right)\leq (8h_i+4p_i)/s+11q_i+8\ell_i+4p_i+6(h_i - m_i +1)\,,
	 \end{equation*}
	 where $d^H(v)$ is the degree of $v$ in the underlying graph $G(H)$, i.e., without counting multiplicities, and
	 $$m_H(v)=\sum_{\substack{e\in E(H),\\e\text{ is incident to }v}}m_H(e).$$
	 \begin{proof}
		
Let $B^{\ast}$ be a spanning tree of $B_i$. We start by deriving an upper bound on the number of leaves of $B^{\ast}$. Notice that if a vertex $v\in V(B^{\ast})$ is a leaf of $B^{\ast}$, then it must satisfy at least one of the following three conditions:
\begin{itemize}
\item[(a)] $v$ is incident to an edge in the cut $H(V(B_i),V(H)\setminus V(B_i))$.
\item[(b)] $v$ is incident to an edge in $G(H(B_i))\setminus B^{\ast}$.
\item[(c)] $v$ is a pivot of $H$, i.e., $v$ is an end-vertex of one of the $s$ self-avoiding walks forming $H$.
\end{itemize}

We will now upper-bound the number of leaves of each kind:
\begin{itemize}
\item Since each edge in $H(V(B_i),V(H)\setminus V(B_i))$ is incident to exactly one vertex in $V(B_i)$, there are at most $\ell_i$ leaves of $B^{\ast}$ satisfying Condition (a).
\item Since there are at most $\left(q_i+h_i+\frac{p_i}{3}\right)-(m_i-1)$ edges in $G(H(B_i))\setminus B^{\ast}$, and since each edge is incident to exactly two vertices, there are at most $2q_i+2h_i+\frac{2}{3}p_i-2(m_i-1)$ leaves of $B^{\ast}$ satisfying Condition (b).
\item Each pivot vertex must be an end-vertex of a self-avoiding walk of $H$. There are at most $(q_i+2h_i+p_i)/s$ self-avoiding walks of $H$ that lie entirely in $H(B_i)$, and there are at most $\ell_i$ self-avoiding walks of $H$ that intersect $H(B_i)$ without being entirely in $H(B_i)$. Hence, $H(B_i)$ intersects at most $(q_i+2h_i+p_i)/s+\ell_i$ self-avoiding walks of $H$. Now since each self-avoiding walk contains exactly 2 end-vertices, we deduce that the number of pivots of $H$ in $V(B_i)$ is at most
\begin{equation}
\label{eq:tree_high_degree_sum_upperboud_num_pivots}
2(q_i+2h_i+p_i)/s+2\ell_i.
\end{equation}
We now conclude that there are at most $2(q_i+2h_i+p_i)/s+2\ell_i$ leaves of $B^{\ast}$ satisfying Condition (c).
\end{itemize}

Therefore, the number of leaves of $B^{\ast}$ is at most
\begin{align*}
\ell_i+2q_i+2h_i+\frac{2}{3}p_i-2(m_i-1) + &2(q_i+2h_i+p_i)/s+2\ell_i\\
&\leq (4h_i+2p_i)/s +3\ell_i + 4q_i + p_i + 2(h_i - m_i +1)
\end{align*}

Let $d^{B^{\ast}}(v)$ be the degree of a vertex $v\in B^{\ast}$. It is easy to see that the number of leaves in $B^\ast$ is equal to
$$2+\sum_{\substack{v\in V(B_i)\\ d^{B^{\ast}}(v)\geq 3}} (d^{B^{\ast}}(v)-2).$$
Therefore,
		 \begin{equation}
		 \label{eq:tree_high_degree_sum_upperboud_num_leaves}
			 \sum_{\substack{v\in V(B_i)\\ d^{B^{\ast}}(v)\geq 3}} (d^{B^{\ast}}(v)-2)\leq  (4h_i+2p_i)/s +3\ell_i + 4q_i + p_i + 2(h_i - m_i +1).
		\end{equation}
	 
	Now define
	$$m_{H,2}(v)=\sum_{\substack{e\in E(H),\\e\text{ is incident to }v,\\m_H(e)=2}}m_H(e)\,,$$	 
	and
	$$m_{H,\neq 2}(v)=\sum_{\substack{e\in E(H),\\e\text{ is incident to }v,\\m_H(e)\neq 2}}m_H(e)\,.$$
	
	Clearly, $m_H(v)=m_{H,2}(v)+m_{H,\neq 2}(v)$. Therefore,
	 \begin{align*}
		 \sum_{\substack{v\in V(B_i)\\ d^H(v)\geq 3 }} \left(\frac{m_H(v)}{2}-1\right)  & = \sum_{\substack{v\in V(B_i)\\ d^H(v)\geq 3 }} \left(\frac{m_{H,2}(v)}{2}-1\right) + \sum_{\substack{v\in V(B_i)\\ d^H(v)\geq 3 }} \frac{m_{H,\neq 2}(v)}{2} \\
		 &\stackrel{(a)}{\leq} \sum_{\substack{v\in V(B_i)\\ d^H(v)\geq 3 }} \left(d^{B_i}(v)-1\right) + \sum_{v\in V(B_i)} \frac{m_{H,\neq 2}(v)}{2}\\
		 &= \sum_{\substack{v\in V(B_i)\\ d^H(v)\geq 3 }} \left(d^{B^{\ast}}(v)-1\right) + \sum_{\substack{v\in V(B_i)\\ d^H(v)\geq 3 }} \left(d^{B_i}(v)-d^{B^{\ast}}(v)\right) + \frac{\ell_i}{2} + \sum_{\substack{e\in H(B_i)\\ m_H(e)\neq 2}} m_H(e)\\
	     &= \sum_{\substack{v\in V(B_i)\\ d^H(v)\geq 3 }} \left(d^{B^{\ast}}(v)-1\right) + 2\big(|E(B_i)|-|E(B^{\ast})|\big) + \frac{\ell_i}{2} +q_i+p_i\\
		 &\leq \sum_{\substack{v\in V(B_i)\\ d^H(v)\geq 3}}(d^{B^{\ast}}(v)-1)+ 2\left(h_i+\frac{p_i}{3}-(m_i-1)\right) + \frac{\ell_i}{2}+q_i+p_i
		 \\ & \leq \sum_{\substack{v\in V(B_i)\\ d^{B^{\ast}}(v)\geq 3}}(d^{B^{\ast}}(v)-1)+\Bigg(\sum_{\substack{v\in V(B_i)\\ d^H(v)\geq 3\\ d^{B^{\ast}}(v)=2}}1\Bigg)+ 2(h_i-m_i+1) + \ell_i +q_i+2p_i
		 \\ & \stackrel{(b)}{\leq} 2\cdot\sum_{\substack{v\in V(B_i)\\ d^{B^{\ast}}(v)\geq 3}}(d^{B^{\ast}}(v)-2)+(2q_i+\ell_i)+ 2(h_i-m_i+1) + \ell_i +q_i+2p_i
		 \\ & \stackrel{(c)}{\leq} (8h_i+4p_i)/s +6\ell_i + 8q_i + 2p_i + 4(h_i - m_i +1)\\
		 &\quad\quad\quad\quad\quad\quad\quad\quad\quad\quad\quad\quad+ 2(h_i-m_i+1) + 2\ell_i +3q_i+2p_i
		 \\ &=(8h_i+4p_i)/s+11q_i+8\ell_i+4p_i+6(h_i - m_i +1)\,,
	 \end{align*}
	 where (a) follows from the fact that every edge of multiplicity 2 and incident to $v\in V(B_i)$ must be in $B_i$, hence $m_{H,2}(v)\leq 2 d^{B_i}(v)$. (b) follows from the fact that if $v\in V(B_i)$ satisfies $d^H(v)\geq 3$ and $d^{B_i}(v)=2$, then $v$ must be incident to an edge of multiplicity 1, i.e., $v$ must be incident to an edge in $E(H(B_i))\setminus E(B_i)$, or an edge in the cut $H(V(B_i),V(H)\setminus V(B_i))$. There are $q_i$ edges in $E(H(B_i))\setminus E(B_i)$, each of which is incident to two vertices in $V(B_i)$, and there are $\ell_i$ edges  in the cut $H(V(B_i),V(H)\setminus V(B_i))$, each of which is incident to one vertex in $V(B_i)$. (c) follows from \eqref{eq:tree_high_degree_sum_upperboud_num_leaves}.
	 \end{proof}
 \end{lemma}
 
 \begin{lemma}\label{lem:sbm-encoding-of-H-B-equivalent}
 Let $\cM_{s,t,\Set{\mathcal{F}_i,\Psi_i}_{i=1}^{z}}(\cB)$ be as defined in \Cref{def:generic-bsaw}. We denote $\ell=\sum_{i=1}^z \ell_i$, and define $q,p,h$ in the same way.  
	 Then, the number of block self-avoiding walks in the set $\cM_{s,t,\Set{\mathcal{F}_i,\Psi_i}_{i=1}^{z}}(\cB)$ is at most $$2n^{st-2h-p-q-\frac{\ell}{2}+m}\cdot m^{2p+2q+\ell}\cdot  (st)^{3\frac{\ell}{2}+q+2}\cdot\prod_{i=1}^z \Psi_i^{(8h_i+4p_i)/s+11q_i+8\ell_i+4p_i+6(h_i - m_i +1)}\,.$$
	 
	 \begin{proof}
We would like to derive an upper bound on the number of block self-avoiding walks in $\cM_{s,t,\Set{\mathcal{F}_i,\Psi_i}_{i=1}^{z}}(\cB)$. We will do this by analyzing how we can construct a block self-avoiding walk $H\in \cM_{s,t,\Set{\mathcal{F}_i,\Psi_i}_{i=1}^{z}}(\cB)$. In order to construct $H$, we will need to make some choices, and by counting the number of possibilities of each choice, we can derive an upper-bound on the size of $\cM_{s,t,\Set{\mathcal{F}_i,\Psi_i}_{i=1}^{z}}(\cB)$.

We first choose the vertices of the respective copies $B_1^{\prime},B_2^{\prime},\ldots,B_z^{\prime}$ of $B_1,B_2,\ldots,B_z$ inside $H$. There are at most $n^m$ possibilities for choosing these vertices.
	 
		We denote the union of the $z$ disjoint graphs $B_1^{\prime},B_2^{\prime},\ldots,B_z^{\prime}$ as
		$$B^{\prime}=\bigcup_{i=1}^z B_i^{\prime}.$$ 
		The size of multiway cut $H\Paren{V(B_1^{\prime}),\ldots,V(B_z^{\prime}), V(H)\setminus \Paren{\underset{j \in [z]}{\bigcup}V(B_z^{\prime})}}$ is bounded by $\ell$, and all the
		edges in the cuts have multiplicity $1$. Let $H(B')$ be the multi-graph that is induced by $H$ on $V(B')$. There are $q$ multiplicity-$1$ edges and $h$ multiplicity-$2$ edges in $H(B')$. The edges of multiplicity at least 3 in $H(B')$ satisfy the following equation:
			\begin{equation*}
				\sum_{\substack{e\in H(B^{\prime})\\ m_H(e)\geq 3}} m_H(e)= p.
		\end{equation*}

		 Next, we choose the multi-graph $H(B')$ that is induced by $H$ on $V(B')$, i.e., we have to choose multiplicities for the edges of $B'$ (which must be at least 2), and we have to choose $q$ edges of multiplicity 1 whose both end-vertices must lie in $V(B')$. Note that there are $p$ multi-edges in $H(B')$ that correspond to edges of multiplicity at least 3, so specifying these multi-edges is equivalent to specifying the multiplicities of the edges of $B'$, because once we specify the edges of $B'$ of multiplicity at least 3, any remaining edge in $B'$ must have multiplicity-$2$. Since both end-vertices of these multi-edges lie in $V(B')$ and since $|V(B')|=m$, there are at most $m^{2p}$ ways to specify the $p$ multi-edges. For choosing the $q$ multiplicity-$1$ edges, there are at most $m^{2q}$ ways of doing so.
		
		The remaining of the multi-graph $H$ consists of edges of multiplicity-1. More precisely, since $H$ is a block self-avoiding walk, the remaining edges of $H$ can be partitioned into a number of disjoint walks of multiplicity 1: Each time we exit $V(B')$, we exit through one of the cut edges, we go through a walk of multiplicity 1, and then re-enter $V(B')$ through another cut edge. We call these walks as \emph{outside walks}. We call the first edge of each outside walk as an \emph{outgoing cut-edge}, and we call the last edge of each outside walk as a \emph{returning cut-edge}. The remaining edges of the outside walks are called \emph{outside edges}.
		
		Note that there might be some cut-edges that are outgoing and returning at the same time. This can happen if a cut-edge is incident to two disjoint connected components $B_i$ and $B_j$ with $i\neq j$. We call such cut-edges as \emph{bridge cut-edges}. Let $b$ be the number of bridge cut-edges. It is easy to see the following:
\begin{itemize}
\item The total number of cut-edges is equal to $\ell-b$.
\item The total number of non-bridge cut-edges is $\ell-2b$.
\item The total number of non-bridge outgoing cut-edges is $\frac{\ell-2b}{2}=\frac{\ell}{2}-b$. Similarly, the total number of non-bridge returning cut-edges is $\frac{\ell}{2}-b$.
\item There are $b$ outside walks of length 1. These correspond to bridge cut-edges.
\item There are $\frac{\ell-2b}{2}=\frac{\ell}{2}-b$ outside walks of length at least 2. We call these outside walks as \emph{non-bridge outside walks}.
\end{itemize}		
As we can see, the total number of outside walks is $\frac{\ell}{2}$.
		
		Since both end-vertices of each bridge cut-edge must lie inside $V(B')$, there are at most $m^{2b}$ ways of choosing them.
		
		Now since one end-vertex of each non-bridge outgoing cut-edge must lie in $V(B')$ while the other must like outside $V(B')$, there are at most $nm$ ways to choose each non-bridge outgoing cut-edge. Since there are $\frac{\ell}{2}-b$ non-bridge outgoing cut-edges, we have at most $(nm)^{\frac{\ell}{2}-b}$ ways to choose them. Now for each non-bridge outgoing cut-edge, we specify the length of the corresponding non-bridge outside walk. Since the length of each outside walk is at most $st$, there are at most $(st)^{\frac{\ell}{2}-b}$ ways for choosing the lengths of all non-bridge outside walks.
		
		Next, we choose the outside edges. Since we have already specified the non-bridge outgoing cut-edges, this already specifies one end-vertex of the first outside edge in each of the $\frac{\ell}{2}-b$ non-bridge outside walks, so we can specify the first outside edge by only specifying the second end-vertex. Similarly, we can iteratively specify all the outside edges by successively specifying the second end-vertex for each outside edge. Since there are $st-2h-p-q-\ell+b$ outside edges\footnote{Recall that the total number of cut edges is $\ell-b$.}, there are at most $n^{st-2h-p-q-\ell+b}$ ways for specifying them.
		
		Now we turn to specifying the non-bridge returning cut-edges. Since we have already specified the lengths of the $\frac{\ell}{2}-b$ non-bridge outside walks, we know when we have reached the last outside edge, and so the next edge in the walk must be a returning non-bridge cut-edge. Since the last outside edge of each non-bridge outside walk is already specified, this determines one end-vertex of the returning cut-edge and we only need to specify the other end-vertex that lies inside $V(B')$. Now since there are $\frac{\ell}{2}-b$ non-bridge returning cut-edges, there are at most $m^{\frac{\ell}{2}-b}$ ways to specify them.
		
		So far, we have completely specified the multi-graph structure of $H$. However, the block self-avoiding walk structure is more than that. We have to specify an Eulerian walk of $H$ which can be divided into $s$ subwalks that are self-avoiding.
		
		Now before starting to specify the vertices $v_1,\ldots,v_{st}$ in the sequence of the block self-avoiding walk $H$, we will further specify where exactly will the $q$ multiplicity-1 edges of $H(B')$ and the $\ell-b$ cut-edges appear in the sequence $(v_1v_2),\;(v_2,v_3),\ldots,(v_{st-1},v_{st}),\;(v_{st},v_1)$, i.e., for each edge $e$ among these edges, we will specify an index $k_e\in\{1,\ldots,st\}$ such that $e=(v_{k_e},v_{k_e+1})$ if $k_e<st$, and $e=(v_{st},v_1)$ if $k_e=st$. Since there are $st$ choices for each index, there are at most $(st)^{q+\ell-b}$ possibilities to choose all these indices. Let $\mathcal{K}$ be the set of indices that we obtain.
		
		We start by specifying the first and second vertices\footnote{We specified the first two vertices instead of only specifying the first vertex in order to know where to go in case $v_1\notin V(B')$. This will be made clear when we discuss how we choose the next vertex $v_{k+1}$ assuming that we already specified $v_1.\ldots,v_k$.}, which we denote as $v_1$ and $v_2$, respectively, and from there we iteratively specify each next vertex in the walk. We have at most $(st)^2$ choices for choosing the pair $(v_1,v_2)$.
		
		Assume that we have already specified the first $k$ vertices $v_1,\ldots,v_k$ of the walk, where $k\geq 2$. The next vertex $v_{k+1}$ must be adjacent to $v_k$ in $H$. Notice the following:
		\begin{itemize}
		\item If $v_k\notin V(B')$, then since we know both $v_{k-1}$ and $v_k$, we can deduce which non-bridge outside walk contains the edge $(v_{k-1},v_k)$. Now given that all the non-bridge outside walks have been specified, there is a unique choice for the vertex $v_{k+1}$: It is the next vertex in the outside walk.\footnote{It is important to realize here that when we specified the non-bridge outside walks, we did not only specify the structure of the graph that is formed by the cut-edges and the outside edges: We also specified how the block self-avoiding walk $H$ will pass through these edges and in what order.}
		\item If $k\in\mathcal{K}$, then the edge $(v_k,v_{k+1})$ has already been fixed to be one specific edge: It is either a cut-edge, or an edge of multiplicity 1 in $H(B')$. Therefore, there is a unique choice for $v_{k+1}$.
	\item If $v_k \in V(B_i)$ for some $i\in[z]$ and $k\notin \mathcal{K}$, the next edge must be an edge of multiplicity at least 2. Therefore, there are at most $d^H_{\geq 2}(v_k)\leq \Psi_i$ choices for $v_{k+1}$. We note that each vertex $v$ is visited in $H$ for $\frac{m_H(v)}{2}$ times, where
	 \begin{equation*}
		 m_H(v)=\sum_{\substack{e\in E(H),\\e\text{ is incident to }v}}m_H(e)\,.
	 \end{equation*}
	 Furthermore, when a vertex $v\in V(B')$ is visited for the last time, there is only one choice for the next vertex.
		\end{itemize}
	Therefore, once we have fixed the multigraph, the number $b$ of bridge cut-edges, the $\frac{\ell}{2}-b$ non-bridge outside walks, the indices of the cut-edges and the multiplicity-1 edges of $H(B')$ in the walk, and the first two vertices, the number of remaining choices to completely specify the block self-avoiding walk $H$ is at most:
	\begin{align*}
\prod_{i=1}^z\prod_{v\in V(B_i')} \Psi_i^{\frac{m_H(v)}{2}-1}&= \prod_{i=1}^z \Psi_i^{\sum_{v\in V(B_i')}\frac{m_H(v)}{2}-1}\\
	&\leq  \prod_{i=1}^z \Psi_i^{(8h_i+4p_i)/s+11q_i+8\ell_i+4p_i+6(h_i - m_i +1)}\,,
	\end{align*}
	where the last inequality follows from \Cref{lem:tree_high_degree_sum}.
	
	Now since $0\leq b\leq\frac{\ell}{2}$, we conclude that the size of $\cM_{s,t,\Set{\mathcal{F}_i,\Psi_i}_{i=1}^{z}}(\cB)$ can be upper-bounded as follows:
	\begin{align*}
	&\left|\cM_{s,t,\Set{\mathcal{F}_i,\Psi_i}_{i=1}^{z}}(\cB)\right|\\
	&\leq \sum_{b=0}^{\lfloor\frac{\ell}{2}\rfloor} n^m\cdot m^{2p}\cdot m^{2q}\cdot m^{2b}\cdot (nm)^{\frac{\ell}{2}-b}\cdot(st)^{\frac{\ell}{2}-b}\cdot n^{st-2h-p-q-\ell+b}\cdot m^{\frac{\ell}{2}-b}\cdot (st)^{q+\ell-b}\cdot (st)^2\\
	&\quad\quad\quad\quad\quad\quad\quad\quad \quad\quad\quad\quad\quad\quad\quad\quad\quad\quad\quad\quad \times \prod_{i=1}^z \Psi_i^{(8h_i+4p_i)/s+11q_i+8\ell_i+4p_i+6(h_i - m_i +1)}\\
	&= \left(\sum_{b=0}^{\lfloor\frac{\ell}{2}\rfloor}(st)^{-2b}\right) n^{st-2h-p-q-\frac{\ell}{2}+m}\cdot m^{2p+2q+\ell}\cdot  (st)^{3\frac{\ell}{2}+q+2}\cdot\prod_{i=1}^z \Psi_i^{(8h_i+4p_i)/s+11q_i+8\ell_i+4p_i+6(h_i - m_i +1)}\\
	&\leq \frac{1}{1-(st)^{-2}}\cdot n^{st-2h-p-q-\frac{\ell}{2}+m}\cdot m^{2p+2q+\ell}\cdot  (st)^{3\frac{\ell}{2}+q+2}\cdot\prod_{i=1}^z \Psi_i^{(8h_i+4p_i)/s+11q_i+8\ell_i+4p_i+6(h_i - m_i +1)}\\
	&\leq 2n^{st-2h-p-q-\frac{\ell}{2}+m}\cdot m^{2p+2q+\ell}\cdot  (st)^{3\frac{\ell}{2}+q+2}\cdot\prod_{i=1}^z \Psi_i^{(8h_i+4p_i)/s+11q_i+8\ell_i+4p_i+6(h_i - m_i +1)}\,,
	\end{align*}
	where the last inequality is true when $st\geq 2$.

	 \end{proof}
 \end{lemma}
 

\subsubsection{Counting nice block self-avoiding walks}
	We split the proof of \Cref{lem:nice-bsaw-upper-bound-m-z} into two steps.
	\begin{lemma}\label{lem:nice-counting-multigraph-lemma}
	We denote $\text{NBSAW}_{s,t,m,z,t_F}$ as the subset of $\text{NBSAW}_{s,t,m,z}$, in which all nice block self-avoiding walks
	contain $t_F$ pivoting vertices incident to any multiplicity-$2$ edge. Then
	there are at most $n^{\ell+m-z}(20s)^{2t_F} {\ell\choose z}{m+z-1\choose z}{t+z-1\choose z}$ different multi-graphs associated with the block self-avoiding walks in $\text{NBSAW}_{s,t,m,z,t_F}$
	\begin{proof}
	For any $H\in \text{NBSAW}_{s,t,m,z}$, we denote the forest formed by all the multiplicity-2 edges in $H$ as $F_{H}$. Then we note the leaves in $F_{H}$ must be pivots in $H$,
	and there are $t$ pivots in $H$. Thus we have $z\leq t$. We denote the length of the cycle as $\ell$. Since there are $m-z$ multiplicity-2 edges, we have $\ell=st-2(m-z)$.	
	We pick the vertices in the cycle and choose the roots of trees in $F_{H}$. There are at most $n^{\ell}{\ell\choose z}$ such choices.
	 
	Next we generate each of the $z$ trees. We first fix the number of vertices and leaves in each tree. Since there are at most $t$ leaves and $m$ vertices in $F_{H}$, 
	we have at most ${m+z-1\choose z}{t+z-1\choose z}$ choices, which is also bounded by $(2es)^t$. 
	 
	Then we bound the number of choices for each tree $B_i$ in the $F_{H}$ given the number of vertices $m_i$ and the number of leaves $p_i$. Using lemma \Cref{fact:bound-unlabeled-trees}, there are at most ${2m_i\choose 2t_i}2^{2t_i}$ shapes for a tree with $m_i$ vertices and $p_i$ leaves.  Multiplying together we have the number of choices bounded by
	 \begin{equation*}
		 \prod_{i=1}^z {2m_i\choose 2p_i}2^{2p_i}\leq 2^{2p} \prod_{i=1}^z
		 {2m_i\choose 2p_i}
	 \end{equation*}
	 where $p=\sum_{i=1}^z p_i$ is the total number of leaves in the forest.
	The inequality follows from the fact that each leaf must be a pivot. Further we observe that $m_i\leq s(t_i+2)$, where $t_i\geq z_i$ is the number of pivots contained
	 in the tree $B_i$. This follows since the tree is at most incident to $t_i+2$ blocks of self-avoiding walks in $H$. Thus we have
	\begin{equation*}
		\prod_{i=1}^z
		 {2m_i\choose 2p_i}\leq \prod_{i=1}^z{2m_i\choose 2t_i}\leq \prod_{i=1}^z 
		 \left(\frac{em_i}{t_i}\right)^{2t_i}\leq  \prod_{i=1}^z (10s)^{2t_i}
	\end{equation*}
	We denote $t_F=\sum_{i=1}^z t_i$ as the number of pivots contained in the forest. Then we have
	\begin{equation*}
		\prod_{i=1}^z {2m_i\choose 2p_i}\leq (10s)^{2t_F}
	\end{equation*}
	
	Finally there are $n^{m-z}$ ways of choosing vertices in the trees which are not contained by the multiplicity-1 cycle.
	Therefore in all, we have at most $n^{\ell+m-z}(20s)^{2t_F} {\ell\choose z}{m+z-1\choose z}{t+z-1\choose z}$ choices for 
	the multigraph associated with such block self-avoiding walks. This leads to the claim. 
	\end{proof}
	\end{lemma}

	The following lemma is then a simple corollary:
	\begin{lemma}\label{lem:nice-counting-multigraph}
		There are at most $n^{st-m+z}(20s)^{4t}$ different multi-graphs 
		associated with the 
		block self-avoiding walks in $\text{NBSAW}_{s,t,m,z}$
		\begin{proof}
			From the above lemma \Cref{lem:nice-counting-multigraph-lemma}, we already have bound
			$n^{\ell+m-z}(20s)^{2t_F} {\ell\choose z}{m+z-1\choose z}{t+z-1\choose z}$. Next we note
			 that $z\leq t$, thus we have ${\ell\choose z}\leq (es)^t$ and ${t+z-1\choose z}\leq 2^{2t}$.
			 Further we have $m\leq st$, thus ${m+z-1\choose z}\leq (2es)^t$. 
			 Therefore for $z\leq t_F\leq t$, we have 
			 \begin{align*}
				n^{\ell+m-z}(20s)^{2t_F} {\ell\choose z}{m+z-1\choose z}{t+z-1\choose z}
				\leq n^{\ell+m-z}(20s)^{2t_F}(2es)^{2t}
			 \end{align*}
			 Finally by summing $t_F$ from $z$ to $t$, we have the claim. 
		\end{proof}
	\end{lemma}

	Finally we bound the number of nice block self-avoiding walks. For proving this,
	we observe a simple fact which will be useful for counting nice block self-avoiding walks
	 here and future subsections.
	We say that the sequence of $H$ enters a tree $T$ at step $k$ if the $k$-th vertex
	is not in $T$ and the $k+1$-th vertex is the root of $T$. Similarly we say that the sequence 
	of $H$ leaves a tree $T$ at step $k$ if the $k$-th vertex
	is the root of $T$ and the $k+1$-th vertex is not in $T$. 
	\begin{fact}
		For each tree formed by a set of multiplicity-$2$ edges in the nice block self-avoiding walk
		$H$, once the sequence of $H$ enters the tree through a cut, then 
		all of the edges in the tree must be visited twice before the sequence of $H$
		leaves the tree through any cut edge. 
	\end{fact}
	\begin{proof}
		We note that the cutting edges of any multiplicity-$2$ tree are either
		one multiplicity-$2$ edge or two multiplicity-$1$ edges. Thus if the sequence
		of $H$ enters the tree through a cut and then leaves through a cut, there is no
		cutting edge to take such that the sequence can enter the tree again. 
		Therefore all of the edges in this tree must be visited between the entering cut edge
		 and leaving cut edge. 
	\end{proof}

Now we conclude the bound on the number of nice block self-avoiding walks. 
\begin{lemma}
	 The size of $\text{NBSAW}_{s,t,m,z}$ is bounded by
	 $$\Delta^{3t} (40s)^{4t}n^{st-m+z}$$, where $\Delta\leq st$ is the upper bound on the 
	 maximum degree $\geq 2$ in the nice block self-avoiding walks.
\end{lemma}
	\begin{proof}
	 By lemma \Cref{lem:nice-counting-multigraph}, there are at most $n^{st-m+z}(20s)^{4t}$ multigraphs associated with block self-avoiding walks in the set $\text{NBSAW}_{s,t,\ell,\Delta}$.
	Then fixing the associated multi-graph, we choose the block self-avoiding walk. 
	
	By truncation, the maximum degree-$\geq 2$ of vertices is bounded by $\Delta$. There are at most $t$ leaves in the forest of multiplicity-2 edges, and other vertices have degree at least $2$. 
	Therefore as in the proof of \Cref{lem:tree_high_degree_sum}, for one of the $z$
	trees $B$, we have
	\begin{equation*}
		\sum_{\substack{v\in V(B)\\ \text{deg}_{B}(v)\geq 3}} (\text{deg}_{B}(v)-1)\leq 2t+z\leq 3t
	\end{equation*}
	By the above fact, for each root of the forest, all of the edges in  associated tree must be visited twice in the sequence of $H$
	before the next adjacent vertex in the cycle is taken unless it is root of the tree containing the 
	first vertex in the block self-avoiding walk. 
	Further we denote $r$ as the root of $B$, then $\text{deg}_B(r)$ is smaller than the number of leaves
	 in $B$. 

	 For choosing the block self-avoiding walk based on the multigraph, we first decide
	 the first edge in the walk, which takes at most $\Delta st$ choices. Next if the
	 first edge is contained in a multiplicity-$2$ tree $B$, then we choose the 
	the cut edge leaving $B$ and its index in the sequence of $H$, which takes $2st$ choices.
	After these two steps, the multiplicities of edges in each index of the sequence of $H$
	is fixed. Finally we generate the sequence of $H$ respecting these previous two steps.
	Then the number of choices for the block self-avoiding walk given the associated multigraph is bounded by
	\begin{align*}
		\Delta st\cdot 2st\prod_B\left[\prod_{\substack{v\in V(B)\\ \text{ded}_{B}(v)\geq 3}}(\text{deg}_{B}(v)-1)
		\right]\leq 2^{t} \Delta{3t}
	\end{align*}	
	In all, we have the total number of such block self-avoiding walks bounded by $\Delta^{3t} (40s)^{4t}n^{\ell-z+m}$
\end{proof}
The lemma \label{lem:nice-bsaw-upper-bound} is a simple corollary
\begin{lemma}[Restatement of \cref{lem:nice-bsaw-upper-bound-m-z}]\label{lem:nice-bsaw-upper-bound}
	Consider the settings of \cref{thm:sbm-schatten-norm}. Let $r\geq 0$ be an integer. For $n$ large enough
	\begin{align*}
		\underset{\substack{m,z\geq 0\\
		m-z=r}}{\sum}\Card{\nbsaw{s}{t,m,z}}\leq \Delta^{5t+\Delta}n^{st-r}\,.
	\end{align*}
\end{lemma}
\begin{proof}
	Since $\Delta\geq s^2$, we have $\Delta^{3t} (40s)^{4t}n^{\ell-z+m}\leq (40)^{4t}\Delta^{5t}n^{\ell-r}$
	No we sum $m,z$ from $0$ to $t$,
	\begin{align*}
		\underset{\substack{m,z\geq 0\\
		m-z=r}}{\sum}\Card{\nbsaw{s}{t,m,z}}\leq t^2(40)^{4t}\Delta^{5t}n^{st-r}\leq \Delta^{6t}
		\,.
	\end{align*}
\end{proof}

\subsubsection{Counting refined set of nice block self-avoiding walk}

Here we prove \cref{lem:counting-special-nbsaw-centered}

\begin{lemma}[Restatement of  \cref{lem:counting-special-nbsaw-centered}]\label{lem:counting-special-nbsaw-centered-appendix}
	Consider the settings of \cref{thm:sbm-schatten-norm}. Let $m,z\geq0$ be integers such that $0\leq m-z< \frac{t\sqrt{s}}{2}$. Then for $n$ large enough
	\begin{align*}
		\Card{\text{N}_{w,q, m,z}}\leq 2^{ws}\cdot  \Paren{1+\delta}^{\frac{t}{10}}\cdot n^{-q}\cdot \Card{\nbsaw{s}{t, d^H\leq \Delta, m',z'}}
	\end{align*}
	for some $m',z'$ such that $m'-z'=m-z-q$.
	Moreover, it holds that $q>t-2t/\sqrt{s}$.
	\begin{proof}
		Let $H\in \text{N}_{w,q,m,z}$. Denote by $q^*$ the number of pivots in $H$ that are incident to an edge of multiplicity $2$. 
		We have 
		\[
		\Card{E_{\geq 2}(H)} = m -z\geq q + (q^*-q)\cdot s/2
		\]
		since if a pivot has edges of multiplicity $2$ and it is not a leaf, then one of the adjacent walks must have all its edges of multiplicity $2$.
		By definition of $\text{N}_{w,q,m,z}$, we must have $q^*\geq t-2t/s$ since otherwise $w\geq t/s$ and the set is empty.
		Suppose now that $q\leq q^*-t/\sqrt{s}$. By the above reasoning we have $m-z\geq \frac{t\sqrt{s}}{2}$ contradicting our assumption.
		Thus it must be that $q>q^*-t/\sqrt{s}\geq t-2t/\sqrt{s}$.
		
		Now we need to bound the number of different ways in which we can choose $w$ walks to be in $W_1(H)$, $q^*$ pivots among $t$ and $q$ leaves among $q^*$ pivots. This can be bounded by
		\begin{align}\label{eq:bound-ways-picking-pivots}
			\binom{q^*}{t-q^*-w}\binom{w}{t-q^*-w}\binom{q^*}{q}\leq 2^{t/\sqrt{s}}\cdot \binom{q^*}{t-q^*-w}^2\leq 2^{t/\sqrt{s}}\cdot \binom{q^*}{w}^2\leq s^{10t/s}\leq (1+\delta)^{t/10}\,.
		\end{align}
		
		We are almost ready to carry out the counting argument. \Tnote{To update from here}
		First we introduce additional notation.
		For a nice block self-avoiding walk $H_1$ with edges $u_1u_2, u_2u_3, u_3,u_4$ of multiplicity $1$. We say that $H_2$ is obtained from $H_1$ folding $u_2$ if $H_2$ obtained from $H$ replacing the edges $u_2u_3, u_3u_4$ with the edges $u_2u_1, u_2u_4$. Notice that if $d^{H_2}(u)=d^{H_1}(u)+1$, thus if $d^{H_1}(u)<\Delta$ then $H_2$ is a nice block self-avoiding walk. Moreover if $\Card{E_{\geq 2}(H_1)}=m_1$ then $\Card{E_{\geq 2}(H_2)}=m_1+1$.
		
		Now notice that each $H\in \text{N}_{w,q,m}$ can be obtained from some $H^*$ in $\nbsaw{s}{t}\cap \cM(\emptyset)$ through a sequence of $m$ foldings. Moreover, we may assume that the first $q$ folding are those fixing the leaves of $H$.
		Let $H^{**}$ be the nice block self-avoiding walk obtained from $H^*$ by folding the $q$ leaves of $H$ and let $S^**$ be the sequence of folding that results in $H$ from $H^{**}$.
		
		Applying $S^**$ to $H^*$ we obtain a nice block self-avoiding walk $H'$ in $\nnbsaw{s}{t}^c$ with $\Card{E_{\geq 2}(H')}= \Card{E_{\geq 2}(H)}-q$. Moreover  \cref{eq:bound-ways-picking-pivots}  bounds the number of ways that a walk $H'$ may arise with this process from non-isomorphic $H\in \text{N}_{w,q,m}$.
		The result follows.

	\end{proof}
\end{lemma}

\subsubsection{Counting nice walks with large degree}
Here we prove \cref{lem:count-nice-walks-large-degree}.
\begin{lemma}\label{lem:count-nice-walks-large-degree-appendix}
	Consider the settings of \cref{thm:sbm-schatten-norm}.  
	Define the set 
	\begin{equation*}
		\begin{split}
			\nbsaw{s}{t,m,z, \ell_1,\ell_2} := \left\{H \in \nbsaw{s}{t,m,z}\suchthat  \Card{\Set{v\in v(H)\suchthat d^H(v)=\Delta+1}}= \ell_1  \right.   \,,\\
			\left. \Card{\Set{v\in v(H)\suchthat d^H(v)=\Delta+2}}= \ell_2 \right\}
		\end{split}
	\end{equation*}
	Then for $n$ large enough, there exists $t\geq \ell\geq \ell_1+2\ell_2$ such that
	\begin{align*}
		\Card{\nbsaw{s}{t,m,z,\ell_1,\ell_2}}\leq (1+o(1))\cdot 2^t\cdot n^{-\ell}\Card{\nbsaw{s}{t,d^H\leq \Delta, m,z+\ell}}\,.
	\end{align*}
\end{lemma} 
\begin{proof}
Our plan is to show that for each $H\in \nbsaw{s}{t, m,z,\ell_1,\ell_2}$, we can transform it 
into at least $n^\ell$ different $H'\in \nbsaw{s}{t,d^H\leq \Delta, m,z+\ell}$.
Further we show that for any $H'\in \nbsaw{s}{t,d^H\leq \Delta, m,z+\ell}$, it can be
generated by at most $2^t$ different $H\in \nbsaw{s}{t, m,z,\ell_1,\ell_2}$.

For each vertex $v$ in $H$ with degree larger 
than $\Delta$, it must be contained in the multiplicity-$1$ cycle of $H$. 
For each such vertex $u$, we denote the last multiplicity-2 edge incident to it
in the sequence as $(u,v)$. Further we denote the second multiplicity-1 edge 
incident to it in the sequence as $(u,u')$. Then our transformation strategy
is to replace the last $(u,v)$ edge 
in $H$ with $(v,w_u)$ and $(u,u')$ edge with $(w_u,u)$ where 
$w_u\in [n]\setminus V(H)$. We further require that for different $u$, $w_u$ are 
different. 

We note that the nice block self-avoiding walk after transformation 
is also a nice block self-avoiding walk. Further for each vertex with degree larger than $\Delta$, 
its degree is reduced by $1$(but we need to be cautious that some vertices of 
degree $\Delta$ can have degree $\Delta+1$ after the transformation). 
Finally after performing transformation
for one vertex $v$, the number of connected components of multiplicity-$2$
edges(i.e the value of $z$) is increased by $1$ and the number of different 
vertices in the nice block self-avoiding walk is also increased by $1$.  

We keep iterating on such transformation until every vertex has degree $\leq\Delta$. 
Such process terminates after at most $t$ iterations because for nice block self-avoiding walks
$z<t$. Suppose we do such transformation $\ell$ times. Then the number of connected 
components is increased by $\ell$ the number of vertices in the cycle is increased by $\ell$.
Further the number of vertices in the path remains unchanged. Since there are $(1\pm o(1))n^{\ell}$
 ways of choosing the new vertices added in the transformations, each 
$H\in \nbsaw{s}{t, m,z,\ell_1,\ell_2}$ can be transformed into $(1\pm o(1))n^{\ell}$
$H'\in \nbsaw{s}{t,m,z+\ell}$. 

Next we show that for each $H'\in \nbsaw{s}{t,m,z+\ell}$, it can be generated by at most
$2^t$ different $H\in \nbsaw{s}{t,m,z,\ell_1,\ell_2}$.  Given $H'$, 
if for the $z+\ell$ roots of multiplicity-$2$ trees which are contained in the multiplicity-$1$
cycle, we know whether they are also contained in the multiplicity-$1$ cycle of $H$, 
 then we can identify the unique $H$ from $H'$. Since there are at most $2^t$ choices
  for determining which of the $z+\ell$ roots are contained in the multiplicity-$1$
   cycle of $H$, we can conclude that each $H'\in \nbsaw{s}{t,m,z+\ell}$ can be generated by at most
   $2^t$ different $H\in \nbsaw{s}{t,m,z,\ell_1,\ell_2}$
 
Combining the two facts, we have
\begin{itemize}
	\item each $H'\in \nbsaw{s}{t,m,z+\ell}$ can be generated by at most
	$2^t$ different $H\in \nbsaw{s}{t,m,z,\ell_1,\ell_2}$
	\item each $H\in \nbsaw{s}{t, m,z,\ell_1,\ell_2}$ can be transformed into $(1\pm o(1))n^{\ell}$
	 $H'\in \nbsaw{s}{t,m,z+\ell}$. 
\end{itemize}
It follows that
\begin{align*}
	\Card{\nbsaw{s}{t,m,z,\ell_1,\ell_2}}\leq 2^tn^{-\ell} \Card{\nbsaw{s}{t,m,z+\ell}}
\end{align*}
which fulfills the proof. 
\end{proof}

\subsubsection{Counting nice block self-avoiding walks with short cycles}
We prove the \cref{lem:count-nice-walks-few-edges-multiplicity-1}
\begin{lemma}\label{lem:count-nice-walks-few-edges-multiplicity-1-appendix}
	Consider the settings of \cref{thm:sbm-schatten-norm}. For $A=(1000s)^2$,
	Let $m,z,q>1$ be integers such that $st-2m+2z+q=\frac{t}{10\sqrt{A}}$, then
	\begin{align*}
		\Card{\nbsaw{s}{t,m,z}}\leq 2^{10t}s^t\cdot n^{-m+z}\cdot \Card{\nbsaw{s}{t,0,0}}\,.
	\end{align*}
\end{lemma}
\begin{proof}
We denote $\ell=st-2m+2z$ as the number of multiplicity-1 edges in the cycle. Then
 we have $z\leq \ell\leq \frac{t}{10\sqrt{A}}$.
We first pick the number of vertices and leaves in each of the $z$ trees. 
Since there are $m$ vertices and $t$ leaves to assign, there are at most 
${m\choose z}{t\choose z}$ choices. Since $z\leq \frac{t}{10\sqrt{A}}$, this 
is bounded by 
$(100As)^{t/10\sqrt{A}}\leq 2^t$. Further we fix the staring point of these trees in the sequence.
This take ${\ell\choose z}\leq 2^{t/10\sqrt{A}}$ choices.

We denote the number of vertices and leaves in the tree $T_i$ as $m_i$ and $t_i$, 
for $i\in [z]$. 
Then for each tree, we note that if the order of vertices and degree of each vertex 
is determined, then the subsequence of this tree in $H$ is totally determined
(For order of vertices, we mean order of the first appearence of vertices in 
the sequence). For determining the vertices and the order, there are at most $n^{m_i}$
 choices. 
 
Next we determine the degrees of the vertices, there are at most 
${m_i\choose t_i}$ ways to choose the leaves. We note that the leaves must be
pivot vertices, and there are at most $m_i/s+2$ pivots of self-avoiding walks
in tree $i$. Therefore we have $t_i\leq m_i/s+2$, and it follows that
${2m_i/s+2\choose t_i}\leq 2^{2m_i/s+2}$. Let $d_{T_i}(v)$ be the degree of vertice
$v$ in the tree $T_i$. 
Then by the degree constraint, we have
\begin{equation*}
	  \sum_{v\in T_i} (d_{T_i}(v)-2) \ind{d_{T_i}(v)\geq 2}= t_i
\end{equation*}
Since there are $m_i-t_i$ vertices with degree at least $2$ in $T_i$, the total number
 of choices for the degrees of these vertices is upper bounded by
\begin{equation*}
	{m_i-t_i+t_i\choose t_i}\leq (es)^{t_i}
\end{equation*}
In all, we have the number of choices for the degrees of vertices upper bounded by
$(es)^{t_i}2^{2m_i/s+2}$. 

Now we multiplying the number of choices for each tree $T_i$, and we have
\begin{equation*}
	\prod_{i=1}^z (es)^{t_i}2^{2m_i/s+2}\leq 2^{10t} s^{t} 
\end{equation*}

Therefore there are at most $2^{10t}s^t n^{m_i}$ choices for the multiplicity-$2$
tree. Finally we choose the $st-2m+z$ vertices which are only incident to multiplicity-$1$ 
edges, which takes at most $n^{st-2m+z}$ choices. Multiplying together, we get the
claim.
\end{proof}

\subsubsection{Counting pairs of block self-avoiding walk}
We prove lemma \Cref{lem:sbm-counting-product-bsaws}	 here. 
The proof is very similar to lemma \Cref{lem:sbm-encoding-of-H-B-equivalent}.

\begin{lemma}\label{lem:tree_high_degree_sum_pair}
	Let $B_i$ be a connected subgraph of the underlying graph of a block self-avoiding walk pair $H\in \bsaw{s}{t,u}\times \bsaw{s}{t,v}$ with $m_i$ vertices. 
	Let $H(B_i)$ be the multigraph on $V(B_i)$ that is induced by $H(B_i)$. Suppose that:
\begin{itemize}
	\item The cut $H(V(B_i),V(H)\setminus V(B_i))$ consists of $\ell_i$ edges of multiplicity 1.
	\item All edges in $H(B_i)$ of multiplicity $\geq 2$ are in $B_i$.
	\item The number of edges of multiplicity 1 in $H(B_i)$ is $q_i$.
	\item The number of edges of multiplicity 2 in $H(B_i)$ is $h_i$.
	\item The edges of multiplicity larger than 2 satisfy
	$$\sum_{\substack{e\in H(B_i)\\ m_H(e)\geq 3}} m_H(e)= p_i.$$
\end{itemize}
	 Then we have
	 \begin{equation*}
		 \sum_{\substack{v\in V(B_i)\\ d^H(v)\geq 3 }} \left(\frac{m_H(v)}{2}-1\right)\leq (8h_i+4p_i)/s+11q_i+8\ell_i+4p_i+6(h_i - m_i +1)\,,
	 \end{equation*}
	 where $d^H(v)$ is the degree of $v$ in the underlying graph $G(H)$, i.e., without counting multiplicities, and
	 $$m_H(v)=\sum_{\substack{e\in E(H),\\e\text{ is incident to }v}}m_H(e).$$

\begin{proof}
For the case where $u=v$, each block self-avoiding walk pair in $H\in \bsaw{s}{t,u}\times \bsaw{s}{t,v}$ is a block self-avoiding walk contained in 
 $M_{s,2t}$. Further the $B_i,q_i,h_i$ and $p_i$ are defined in the same way as the \Cref{lem:tree_high_degree_sum}. Therefore, the bound follows by
 as a corollary of \Cref{lem:tree_high_degree_sum}. 
 
It remains to bound for the case $u\neq v$, which also follows from the same proof.  Let $H=H_1\otimes H_2$ be the decomposition of $H$ into two block self-avoiding walks
$H_1\in \bsaw{s}{t,u}$ and $H_2\in \bsaw{s}{t,v}$. 
Let $B^{\ast}$ be a spanning tree of $B_i$. We start by deriving an upper bound on the number of leaves of $B^{\ast}$. 
Notice that if a vertex $v\in V(B^{\ast})$ is a leaf of $B^{\ast}$, then it must satisfy at least one of the following three conditions:
\begin{itemize}
\item[(a)] $v$ is incident to an edge in the cut $H(V(B_i),V(H)\setminus V(B_i))$.
\item[(b)] $v$ is incident to an edge in $G(H(B_i))\setminus B^{\ast}$.
\item[(c)] $v$ is a pivot of $H_1$ or $H_2$, i.e., $v$ is an end-vertex of one of the $2t$ blocks of self-avoiding walks forming $H$.
\end{itemize}


By the same argument in \Cref{lem:tree_high_degree_sum}, there are at most $\ell_i$ leaves satisfying (a), $2q_i+2h_i+\frac{2}{3}p_i-2(m_i-1)$
 satisfying (b), and $2(q_i+2h_i+p_i)/s+2\ell_i$ leaves satisfying condition (c). 
Therefore the number of leaves of $B^{\ast}$ is at most
\begin{align*}
\ell_i+2q_i+2h_i+\frac{2}{3}p_i-2(m_i-1) + &2(q_i+2h_i+p_i)/s+2\ell_i\\
&\leq (4h_i+2p_i)/s +3\ell_i + 4q_i + p_i + 2(h_i - m_i +1)
\end{align*}

Let $d^{B^{\ast}}(v)$ be the degree of a vertex $v\in B^{\ast}$. By degree constraint the number of leaves in $B^\ast$ is equal to
$$2+\sum_{\substack{v\in V(B_i)\\ d^{B^{\ast}}(v)\geq 3}} (d^{B^{\ast}}(v)-2).$$
Therefore, same as \eqref{eq:tree_high_degree_sum_upperboud_num_leaves}, we have
		 \begin{equation}
			\label{eq:tree_high_degree_sum_upperboud_num_leaves_copy}
			\sum_{\substack{v\in V(B_i)\\ d^{B^{\ast}}(v)\geq 3}} (d^{B^{\ast}}(v)-2)\leq  (4h_i+2p_i)/s +3\ell_i + 4q_i + p_i + 2(h_i - m_i +1).
		\end{equation}
	 
Now define
$$m_{H,2}(v)=\sum_{\substack{e\in E(H),\\e\text{ is incident to }v,\\m_H(e)=2}}m_H(e)\,,$$	 
and
$$m_{H,\neq 2}(v)=\sum_{\substack{e\in E(H),\\e\text{ is incident to }v,\\m_H(e)\neq 2}}m_H(e)\,.$$
	
	Clearly, $m_H(v)=m_{H,2}(v)+m_{H,\neq 2}(v)$. Therefore, by the same argument in \Cref{lem:tree_high_degree_sum}
	 \begin{align*}
		 \sum_{\substack{v\in V(B_i)\\ d^H(v)\geq 3 }} \left(\frac{m_H(v)}{2}-1\right)  & = \sum_{\substack{v\in V(B_i)\\ d^H(v)\geq 3 }} \left(\frac{m_{H,2}(v)}{2}-1\right) + \sum_{\substack{v\in V(B_i)\\ d^H(v)\geq 3 }} \frac{m_{H,\neq 2}(v)}{2} \\
		 &\stackrel{(a)}{\leq} \sum_{\substack{v\in V(B_i)\\ d^H(v)\geq 3 }} \left(d^{B_i}(v)-1\right) + \sum_{v\in V(B_i)} \frac{m_{H,\neq 2}(v)}{2}\\
	     &= \sum_{\substack{v\in V(B_i)\\ d^H(v)\geq 3 }} \left(d^{B^{\ast}}(v)-1\right) + 2\big(|E(B_i)|-|E(B^{\ast})|\big) + \frac{\ell_i}{2} +q_i+p_i\\
		 &\leq \sum_{\substack{v\in V(B_i)\\ d^H(v)\geq 3}}(d^{B^{\ast}}(v)-1)+ 2\left(h_i+\frac{p_i}{3}-(m_i-1)\right) + \frac{\ell_i}{2}+q_i+p_i
		 \\ & \leq \sum_{\substack{v\in V(B_i)\\ d^{B^{\ast}}(v)\geq 3}}(d^{B^{\ast}}(v)-1)+\Bigg(\sum_{\substack{v\in V(B_i)\\ d^H(v)\geq 3\\ d^{B^{\ast}}(v)=2}}1\Bigg)+ 2(h_i-m_i+1) + \ell_i +q_i+2p_i
		 \\ & \stackrel{(b)}{\leq} 2\cdot\sum_{\substack{v\in V(B_i)\\ d^{B^{\ast}}(v)\geq 3}}(d^{B^{\ast}}(v)-2)+(2q_i+\ell_i)+ 2(h_i-m_i+1) + \ell_i +q_i+2p_i
		 \\ & \stackrel{(c)}{\leq} (8h_i+4p_i)/s +6\ell_i + 8q_i + 2p_i + 4(h_i - m_i +1)\\
		 &\quad\quad\quad\quad\quad\quad\quad\quad\quad\quad\quad\quad+ 2(h_i-m_i+1) + 2\ell_i +3q_i+2p_i
		 \\ &=(8h_i+4p_i)/s+11q_i+8\ell_i+4p_i+6(h_i - m_i +1)\,,
	 \end{align*}
	 where (a) follows from the fact that every edge of multiplicity 2 and incident to $v\in V(B_i)$ must be in $B_i$, hence $m_{H,2}(v)\leq 2 d^{B_i}(v)$. 
	 (b) follows from the fact that if $v\in V(B_i)$ satisfies $d^H(v)\geq 3$ and $d^{B_i}(v)=2$, then $v$ must be incident to an edge of multiplicity 1, 
	 i.e., $v$ must be incident to an edge in $E(H(B_i))\setminus E(B_i)$, or an edge in the cut $H(V(B_i),V(H)\setminus V(B_i))$. There are $q_i$ edges in $E(H(B_i))\setminus E(B_i)$, each of which is incident to two vertices in $V(B_i)$, 
	 and there are $\ell_i$ edges  in the cut $H(V(B_i),V(H)\setminus V(B_i))$, each of which is incident to one vertex in $V(B_i)$.
	  (c) follows from \eqref{eq:tree_high_degree_sum_upperboud_num_leaves_copy}.
	 \end{proof}
 \end{lemma}

Now we prove lemma \cref{lem:sbm-counting-product-bsaws}.
\begin{lemma}\label{lem:sbm-counting-product-bsaws-appendix}
	Consider the settings of \cref{thm:sbm-bound-second-moment-large-t}. Let $u,v \in [n]$.
	Let $\cB=\Set{B_1,\ldots,B_z}$ be collections of disjoint connected graphs each with respectively $m_1,\ldots,m_{z}\geq 2$ vertices. Let $B^{uv}$ be a connected graph disjoint from any graph in $\cB$ and with $m_{z+1}\geq 2$ vertices.
	Let $\Set{\cF_k}_{i=1}^{z+1}$ be a sequence of tuples of integers  as in \cref{def:product-bsaws-shaped}.
	Let $f^*_{s,t},g^*_{s,t}$ be the functions 
	\begin{align*}
		f^*_{s,t}(m, m', \mathcal{F},\Psi) =& \Paren{\Psi}^{2h/s+10(q+\ell+p+1)+2h-2(m'-1)}\cdot (st)^{5\ell+5q+8p+4h+4-4(m'-1)}\,,\\
		g^*_{s,t}(m', \mathcal{F}) =& n^{-p-\ell/2-q-2h+m'}\,.
	\end{align*}
	Let $m=\underset{j \in [z+1]}{\sum}m_j$. 
	Then there are at most 
	\begin{align*}
		2n^{2st-2+\ind{u=v}\ind{B_{uv}\neq\emptyset}}\cdot &\underset{1\leq k\leq z'}{\prod}f^*_{s,t}(m,m_k,\mathcal{F}_i,\Psi_i)\cdot g^*_{s,t}(m_i,\mathcal{F}_i,h_i)
	\end{align*} 
	block self-avoiding walk pairs in the set $\cM_{s,t,u,v\Set{\cF_i,\Psi_i}_{i=1}^{z+z'}}(\cB)$.
\end{lemma}
\begin{proof}
	We will prove by analyzing how we can construct a block self-avoiding walk pair $H\in \cM_{s,t,u,v\Set{\cF_i,\Psi_i}_{i=1}^{z+z'}}(\cB)$
	as a sequence of edges $\Set{v_0=u,v_1,v_2,\ldots,v_{2st-1},v_{2st}=v}$.We denote the 
	respective copies of $\Set{B_1,\ldots,B_z,B_{uv}}$ in $H$ as $\Set{B_1^{\prime},\ldots,B_z^{\prime},B_{z+1}^{\prime}}$
	The construction is divided into two steps
	\begin{itemize}
		 \item We first choose each $B_i^{\prime}$ and the underlying graph of $H(V(B_i^{\prime}))$ for $i\in [z+1]$.
		 \item Second we choose the indice of the multiplicity-$1$ edges in the sequence of $H$. Furthermore
		 we decide which of the vertices in the cuts are contained in $V(B_i^\prime)$ for $i\in [z+1]$.
		 \item Finally we construct a sequence of $2st$ vertices $\Set{v_0=u,v_1,v_2,\ldots,v_{2st-1},v}$ as the 
		 block self-avoiding walk pair $H$, 
		respecting the underlying graph of $H$ and indices of cut edges fixed in the first two steps.	 
	\end{itemize}
	We denote the union of the $z$ disjoint graphs $B_1^{\prime},B_2^{\prime},\ldots,B_z^{\prime},B_{z+1}^\prime$ as
	$$B^{\prime}=\bigcup_{i=1}^{z+1} B_i^{\prime}.$$ 
		
	We recap the setting of the theorem. The size of multiway cut $H\Paren{V(B_1^{\prime}),\ldots,V(B_z^{\prime}), V(H)\setminus \Paren{\underset{j \in [z]}{\bigcup}V(B_z^{\prime})}}$ is bounded by $\ell$, and all the
	edges in the cuts have multiplicity $1$. Let $H(B')$ be the multi-graph that is induced by $H$ on $V(B')$. There are $q$ multiplicity-$1$ edges and $h$ multiplicity-$2$ edges in $H(B')$. 
	The edges of multiplicity at least 3 in $H(B')$ satisfy the following equation:
   \begin{equation*}
			\sum_{\substack{e\in H(B^{\prime})\\ m_H(e)\geq 3}} m_H(e)= p.
	\end{equation*}
	
	For convenience of notation, we denote the number of vertices $u,v$ contained in 
	$V(B^{\prime}$ as $c_{u,v}$
	$$c_{u,v}=\ind{u\in V(B^{\prime})}+\ind{u\neq v}\ind{v\in V(B^{\prime})}$$

	For the first construction step, we first choose the vertices of the respective
	copies of $\Set{B_1,\ldots,B_z,B_{uv}}$ in $H$. There are at most $n^{m-c_{uv}}$ choices.

	  Next, we choose the multi-graph $H(B')$ that is induced by $H$ on $V(B')$, i.e., 
	  we have to choose multiplicities for the edges of $B'$ (which must be at least 2),
	   and we have to choose $q$ edges of multiplicity 1 whose both end-vertices must
	lie in $V(B')$. Note that there are $p$ multi-edges in $H(B')$ that correspond 
	to edges of multiplicity at least 3, so specifying these multi-edges is 
	equivalent to specifying the multiplicities of the edges of $B'$, because once 
	we specify the edges of $B'$ of multiplicity at least 3, any remaining edge in 
	$B'$ must have multiplicity-$2$. Since both end-vertices of these multi-edges lie in $V(B')$
	 and since $|V(B')|=m$, there are at most $m^{2p}$ ways to specify the $p$ 
	 multi-edges. For choosing the $q$ multiplicity-$1$ edges, there are at most
	  $m^{2q}$ ways of doing so. 

	Thus there are at most $m^{2q+2p}n^{m-c_{u,v}}$ choices for the first step.
	 
	In the second step of construction, we consider $H\setminus H(V(B^{\prime}))$, 
	which consists of edges of multiplicity-1. 
	More precisely, since $H$ is a block self-avoiding walk, the remaining edges 
	of $H$ can be partitioned into a number of disjoint walks of multiplicity 1.  
	Each time the walk of multiplicity 1 starts  
	\begin{itemize}
		\item either by exiting $V(B_i^{\prime})$ for any $i\in [z+1]$ through a cut edge
		\item or from $u$ as the start of $H_1$ and $v_{st+1}=v$ as the start of $H_2$
	\end{itemize}
    and then ends by
	\begin{itemize}
		\item either entering $V(B_j^{\prime})$ through a cut edge for any
		$j\in [z+1]$
		\item or reaching $v_{st}=u$ as the end of $H_1$ or $v_{2st}$ as the end of $H_2$.
	\end{itemize}
	. We call these walks as \emph{outside walks}. We call the first edge of each 
	outside walk as an \emph{outgoing edge}, and we call the last edge of each
	outside walk as a \emph{returning edge}. The remaining edges of the 
	outside walks are called \emph{outside edges}. Note that for $i\neq j$, then an outside walk
	 between $B_i'$ and $B_j'$ can be of length $1$, and in such case the outgoing edge and returning edge will be the same edge.
	
	We decide the index of cutting edges in the block self-avoiding walk pair sequence $H$,
	 and for each vertex incident to an cut edge. Since there are at most $\ell$ cut edges, the number of choices
	 for the indices of cut edges in the sequence of $H$ is bounded by ${2st\choose \ell}$. Further for deciding
	  which of the vertices in the cuts are contained in each of $B_i$ for $i\in [z+1]$, there are
	  at most $(z+1)^{\ell}$ choices. 
	  
	  Finally there are at most $(st)^q$ choices for the indices of the $q$ multiplicity-1 edges
	  contained in $H(V(B_i))$ for $i\in [z+1]$. Thus the total number of choice for the second step is upper bounded by
	  ${2st\choose \ell}(z+1)^{\ell}(st)^q$

	In the third step of construction, for $k\in [0,2st-1]$, given $v_1,v_2,\ldots,v_k$, we choose 
	$v_{k+1}$. For $k=st$, we have $v_{k+1}$ fixed as $v$. Since the second step of construction, 
	we can decide whether $v_k,v_{k+1}\in V(B^{\prime})$ for each $k$. 
	Thus for $k\neq st$, we can divide them into 3 conditions:
	\begin{itemize}
		\item If $v_{k+1}\in V(B_i^{\prime})$ and $v_{k}\not \in V(B_i^{\prime})$, 
		then there are at most $m$ choices for $v_{k+1}$ given $v_k$.
		\item If $v_{k+1}\in V(B_i^{\prime})$ and $v_{k}\in V(B_i^{\prime})$, then there are at most
		$\Psi_i+q_i$ choices for $v_{k+1}$ given $v_k$.	
		\item If $v_{k+1}\not\in V(B^{\prime})$, then there are at most $n$ choices for $v_{k+1}$ given $v_k$.
	\end{itemize}
	The total number of possibilities is upper bounded by the product of choices at each step. 
	
	Since $(v_k,v_{k+1})$ must be a cut for satisfying the first condition,
	there are at most $\ell$ values of $k\in [0,2st-1]$ satisfying the first 
	condition. For each such $k$ there are at most $st$ choices for $v_{k+1}$.

	For the third condition, $v_{k+1}$ is not contained in $V(B^{\prime})$ and $k\neq st$. The number of vertices not contained in $V(B^{\prime})$
	is upper bounded by the number of edges in the outside walks which are not returning edges. 
	There are $2st-2h-p-q$ edges in the outside walks, and there 
	are at least $\ell/2$ returning edges, thus there are at most
	 $2st-2h-p-q-\ell/2$ values of $k$ such that $v_{k+1}$ is not contained in $V(B^{\prime})$. 
	 We consider three cases
	 \begin{itemize}
		 \item When $B_{uv}$ exists, $u,v$ are contained
		 in $V(B^{\prime})$ then there are at least $\ell/2$ returning edges. 
		 \item When $B_{uv}$ doesn't exist and $u=v$, by definition the vertex $u\not\in V(B^{\prime})$
		 (otherwise $B_{uv}$ will exist as one of the $B_{1},B_2,\ldots,B_z$, leading to 
		 contradiction).
		 In this case, there are at least $\ell/2+2$ returning edges. Also for such case $c_{u,v}=0$.
		 \item When $B_{uv}$ doesn't exist and $u\neq v$. In this case there are at least
		 	$\ell/2+2-c_{u,v}$ returning edges, 
	 \end{itemize}
	 
	 Thus when $B_{uv}$ exists, there are at most 
	 $2st-2h-p-q-\ell/2$ values of $k\neq st$ satisfying 
	 the third condition.
	 When $B_{uv}$ doesn't exist, there are
	 at most $2st-2+c_{u,v}-2h-p-q-\ell/2$ values of $k\neq st$ satisfying the third condition. 
	 
	For the second condition, we already decide whether $(v_k,v_{k+1})$ is a multiplicity-1 edge 
	in the second step of construction. If $(v_k,v_{k+1})$ is of multiplicity $1$, then 
	$v_{k+1}$ is already fixed in the second step of construction. Thus we only need to consider
	 $k$ such that $(v_k,v_{k+1})$ has multiplicity $\geq 2$. By product rule for counting, the
	  number of possibilities is bounded by
	\begin{align*}
		\prod_{i=1}^z\prod_{v\in V(B_i')} \Psi_i^{\frac{m_H(v)}{2}-1}&= \prod_{i=1}^z \Psi_i^{\sum_{v\in V(B_i')}\frac{m_H(v)}{2}-1}\\
			&\leq  \prod_{i=1}^z \Psi_i^{(8h_i+4p_i)/s+11q_i+8\ell_i+4p_i+6(h_i - m_i +1)}\,,
	\end{align*}
	where the last inequality follows from \Cref{lem:tree_high_degree_sum_pair}.

	Therefore multiplying together, when $B_{uv}$ is empty, the number of choices for the third step of construction
	 is bounded by
	 \begin{equation*}
		\prod_{i=1}^z \Psi_i^{(8h_i+4p_i)/s+11q_i+8\ell_i+4p_i+6(h_i - m_i +1)}
		n^{2st-2h-p-q-\ell/2-2+c(u,v)} (st)^{\ell}		
	 \end{equation*}
	when $B_{uv}$ exists, the number of choices is bounded by
	 \begin{equation*}
		\prod_{i=1}^z \Psi_i^{(8h_i+4p_i)/s+11q_i+8\ell_i+4p_i+6(h_i - m_i +1)}
		n^{2st-2h-p-q-\ell/2} (st)^{\ell}		
	 \end{equation*}

	 Combining all three construction steps, when $B_{uv}$ exists,
	 the number of choices is bounded by
	 \begin{equation*}
		\prod_{i=1}^z \Psi_i^{(8h_i+4p_i)/s+11q_i+8\ell_i+4p_i+6(h_i - m_i +1)}
		n^{2st-2h-p-q-\ell/2-c(u,v)} (st)^{4\ell+3q+2p}
	 \end{equation*}
	 which is equivalent to 
	 \begin{equation*}
		\prod_{i=1}^z \Psi_i^{(8h_i+4p_i)/s+11q_i+8\ell_i+4p_i+6(h_i - m_i +1)}
		n^{2st-2h-p-q-\ell/2-2+\ind{u=v}} (st)^{4\ell+3q+2p}
	 \end{equation*}
	 When $B_{uv}$ doesn't exist, the number of choices is bounded by
	 \begin{equation*}
		\prod_{i=1}^z \Psi_i^{(8h_i+4p_i)/s+11q_i+8\ell_i+4p_i+6(h_i - m_i +1)}
		n^{2st-2h-p-q-\ell/2-2} (st)^{4\ell+3q+2p}
	 \end{equation*} 
	which proves the claim.
\end{proof}

\subsubsection{Counting nice block self-avoiding walk pairs}
Here we prove lemma \cref{lem:second-moment-counting-nice-multigraphs}
\begin{lemma}\label{lem:second-moment-counting-nice-multigraphs-appendix}
		Consider the settings of \cref{thm:sbm-bound-second-moment-large-t}. Let $u\in [n]$ and 
		Then
		\begin{align*}
			\Card{\nmultig{s,t}{u,u, m}}\leq \frac{m^{10}}{n^m}	\Card{\nmultig{s,t}{u,u, 0}}
		\end{align*}
\end{lemma}
	
	For proving this, our plan is first to show that for every $m$, the size of the set 
	$\nmultig{s,t}{u,u, m}$ is dominated by two multiplity-$1$ cycles sharing $m$ edges. 
	Then we compare the number of such simple block self-avoiding pairs for different
	$m$.  
	\begin{lemma}
		\label{lem:nice-bsaw-simple}
		For every $m\geq 0$, we denote $\text{CyclePair}_{u,u,m}$ as the set of nice block self-avoiding walk pairs which are two
		simple cycles starting from $u$ and only sharing a length-$m$ path containing $u$. Then we have
		\begin{align*}
			\text{CyclePair}_{u,u,m}\leq \Card{\nmultig{s,t}{u,u, m}}\leq 2\text{CyclePair}_{u,u,m}
		\end{align*}
	\end{lemma}
	\begin{proof}
	The nice block self-avoiding walks in the set $\nmultig{s,t}{u,u, m}$ can be
    represented as $\{v_0,v_1,\ldots,v_{2st-1},v_{0}\}$, where 
	$v_0,v_1,\ldots,v_{2st-1}\in [n]$ are vertices and $v_0=v_{st}=u$. 

	First we note that for each $H\in \text{CyclePair}_{u,u,m}$, the underlying graph contains
	$2st-m-2$ vertices excluding vertex $u$. Therefore the size of $\text{CyclePair}_{u,u,m}$
	is at least $n^{2st-m-2}$. Further when $m=st$, the size of $\text{CyclePair}_{u,u,m}$
	is at least $n^{2st-m-1}$.

	Now for $H\in \nmultig{s,t}{u,u,m}$, we denote
	\begin{align*}
			p &=\sum_{e\in H} m_e(H)\ind{m_e(H)\geq 2}\\
			\ell &= \sum_{e\in H} \ind{m_e(H)\geq 2} \\
	\end{align*}

	As an ordered sequence of edges, the block self-avoiding walk pairs $H$
	can be decomposed into alternating segments of multiplicity-$\geq 2$ 
	edges(type I) and multiplicity-$1$ edges(type II). For convenience, 
	we view the first segment and last segment of $H$ as the same segment 
	if they are of the same type.
	Then we denote the number of type I segments in $H$ as $z$.
	and the number of connected components formed by multiplicity-$\geq 2$ edges 
	in $H$ as $k$. Further we denote the number of type I segments 
	containing any of the $m$ shared edges in $H$ as $z^{\prime}$. 
	Finally we denote the indicator variable of the event that $v_0$ 
	and $v_{st}$ are both contained in type I segments as 
	$\ind{E}$.

	We note the relation between $k,z$ and $z^{\prime}$. For each connected component containing
	any of the $m$ shared edges in $H$, it must satisfy one of the following conditions
	\begin{itemize}
			 \item it is visited in one segment of $H_1$ 
			 and one segment of $H_2$
			 \item it is visited in a segment containing $v_{st}$ or $v_0$.
	\end{itemize}
	Thus we must have $z-k\geq (z^{\prime}-2)/2$

	We further define $\nmultig{s,t}{u,u,m,k,z,\ind{E}}$ as the subset of 
	$\nmultig{s,t}{u,u,m}$ respecting $k,z,\ind{E}$ as defined above.

	By the above definitions, the number of multiplicity $1$ edges is 
	equal to $2st-p$, and there are at most $2st-p-z$ vertices 
	in the interior of type I segments. Further there are at most 
	$\ell+k$ different vertices contained in the type 
	II segments. Thus there are at most $2st-p+\ell+k-z$ different 
	vertices in $H$. 

	If $\ind{E}=0$, then at least one of $v_0,v_{st}$ are contained in
	type II segments. Thus in this case there are at most 
	$2st-p+\ell+k-z-2$ different vertices 
	other than $u$ contained in $\nmultig{s,t}{u,u,m,k,z,\ind{E}}$. 
	The only case that $p-\ell-k+z=m$ is that $p=2\ell=m$ and $k=z$. In such case
	 $H$ must be two multiplicity-$1$ cycles overlapping a length $m$ path. 
	 The number of such $H\in \nmultig{s,t}{u,u,m,k,z,0}$ is bounded by 
	 $n^{2st-2-m}$. Otherwise we have $2st-p+\ell+k-z-2\leq 2st-2-m-1$

	If $\ind{E}=1$, then both of $v_0,v_{st}$ are contained in type II segments,
    then there are at most $2st-p+\ell+k-z-1$ different vertices 
	other than $u$ contained. Now because $2st-p+\ell+k-z-1=2st-2-m-(p-\ell+k-z-m-1)$,
	there are $3$ cases
	\begin{itemize}
		\item $p-\ell+k-z-m=0$, in this case $H^{(1)}$ and $H^{(2)}$ are two identical
		cycles. For such case, we must have $m=st$, and it corresponds exactly
		to the set $\text{CyclePair}_{u,u,m}$
		\item $p-\ell+k-z-m=1$, in this case $H^{(1)}$ and $H^{(2)}$ are two multiplicity-$1$ 
		cycles sharing a path containing $u$. Thus such case also corresponds to the
		 set $\text{CyclePair}_{u,u,m}$
		\item Finally we have $p-\ell+k-z-m>1$.  
	\end{itemize}

	For proving $\Card{\nmultig{s}{t,u,u,m}}\leq 2\Card{\text{CyclePair}_{u,u,m}}$, we only need to prove that
	for $m\neq st$
	\begin{align*}
		\Card{\nmultig{s}{t,u,u,m}\setminus \text{CyclePair}_{u,u,m}}
		\leq \Card{\text{CyclePair}_{u,u,m}} 
	\end{align*}
	By the above analysis, for any $H\in \nmultig{s}{t,u,u,m}\setminus \text{CyclePair}_{u,u,m}$
	, there are at most $2st-m-2-(p-\ell+k-z-m)/2$ different vertices other than $u$ 
	in $H$.

	For fixed $p,\ell,k,z$, we count the number of such $H\in \nmultig{s,t}{u,u,m,k,z,\ind{E}}$.
	We divide the construction of $H$ as a sequence of vertices into 3 steps
	\begin{itemize}
			 \item In the first step, we put the multiplicity-$\geq 2$ edges which are not
			 shared by $H^{(1)}$ and $H^{(2)}$ in the sequence. 
			 \item In the second step, we put the $m$ shared edges which are shared by
			  $H^{(1)}$ and $H^{(2)}$
			\item In the final step, we label the vertices in $H$ 
	\end{itemize}
	Then we denote
	\begin{align*}
			p_{1} &=\sum_{e\in H^{(1)}\cap H^{(2)}} m_e(H)\ind{m_e(H)\geq 2}\\
			p_2 &=\sum_{e\in H^{(1)}\Delta H^{(2)}} m_e(H)\ind{m_e(H)\geq 2}
	\end{align*}
	It's easy to see $p=p_1+p_2$. 

	For the first step, there are at most $(st)^{2p_2}$ choices. 
	
		
	For the second step, we first assign the multiplicities to the shared edges in $H$. 
	Since each shared edge has multiplicity $2,3$ or $4$, 
	there are at most $(2st)^{p_1-2m}$ choices. For each of the shared edges with
	 multiplicity $3$ or $4$, there are at most $(2st)^{4}$ choices for
	fixing all its locations in the sequence. Therefore there are at most
	$(2st)^{4(p_1-2m)}$ choices for the locations of all shared edges with multiplicity
	at least $3$.   
		
	Next we decide the locations of multiplicity-$2$ shared edges in 
	the sequence. We call the segments of multiplicity-$2$ shared edges
	 in $H$ as type III segments. Since the shared edges are contained 
	in at most $z^{\prime}$ type I segments, and there are at most $p-2m$ 
	shared edges with multiplicity other than $2$,  
	the multiplicity-$2$ shared edges are split into
	at most $z^{\prime}+p-2m$ type III segments. For fixing the locations of these
	 segments in the sequence, we only need to specify the starting indices and
	  ending indices of these segments.  
	Thus there are only
	$(st)^{z^{\prime}+p-2m}$ choices for the locations of these type III segments.

	For each of the multiplicity-$2$ shared edge, they appear exactly once both in 
	$H^{(1)}$ and $H^{(2)}$(where $H=H^{(1)}\otimes H^{(2)}$ is the decomposition of
	nice block self-avoiding walk $H$). For fixing two locations of each multiplicity-$2$ 
	shared edge with multiplicity-$2$, we first choose the set of locations
	taken by multiplicity-$2$ shared edges in $H$, and then match the locations 
	in $H^{(1)}$ with locations in $H^{(2)}$. For choosing the set of locations taken by multiplicity-$2$ shared edges,
	there are at most $(2st)^{p_2}$ ways. 

	Now for the number of ways of matching locations in $H^{(1)}$ and $H^{(2)}$,
	we make the following observation: once for the first and last edges in every type III segment,
	we fix all their locations in the sequence of $H$, then the locations of each
	multiplicity-$2$ shared edge are fixed. For proving this observation, we note that
	for $H\in \nmultig{s}{t,u,u}$, for each vertex incident to a shared multiplicity-$2$
	edge, it must be incident to $2$ multiplicity-$1$ edges in the subgraph of $H^{(1)}$
	and subgraph of $H^{(2)}$ respectively. For any such vertex $v$, there are two possibilities
	for the two edges $e_i,e_{i+1}$ incident to it in $H^{(1)}$
	\begin{itemize}
		\item $e_i,e_{i+1}$ are two consecutive edges in the sequence
		\item each of $e_i,e_{i+1}$ is the first or last edge of some type III segment 
	\end{itemize}
	We denote the correspondance of $e_i,e_{i+1}$ in $H^{(2)}$ as $e_i',e_{i+1}'$.
	For the second case, $e_i',e_{i+1}'$ has been fixed under the setting. For the 
	first case, the correspondance of $e_i,e_{i+1}$ are either consecutive in the sequence 
	of $H^{(2)}$, or has been fixed as the first or last edge of some type III segment.
	Therefore suppose $e_i'$ is fixed, then $e_{i+1}'$ is also fixed. 
		 
	Since for fixing the first and last edge in each type III segment of $H^{(1)}$($H^{(2)}$),
	, there are at most
	$(st)^{4(z^{\prime}+p-2m)}$ choices. We conclude that there are at most 
	$(2st)^{p_2}(st)^{4(z^{\prime}+p-2m)}(st)^{z^{\prime}+p-2m}$ choices for the  
	shared multiplicity-$2$ edges in the sequence. 

	Finally we only need to label at most $2st-m-2-(p-\ell-m+z-k)/2$ vertices other than
	$u$. Thus there are at most $n^{2st-m-2-(p-\ell-m+z-k)/2}$ choices for 
	labelling the vertices.

	Putting together for fixed $p,\ell,k,z,z^{\prime}$, the size of $\nmultig{s}{t,u,u,m}\setminus \text{CyclePair}_{u,u,m}$ is bounded by
	\begin{align*}
		n^{2st-m-2-(p-\ell-m+z-k)/2}(st)^{5(z^{\prime}+p-2m)}(st)^{z^{\prime}+2p_2+4(p_1-2m)+z^{\prime}}
	\end{align*}
	Arranging the terms, this is upper bounded by
	\begin{align*}
		n^{2st-m-2}n^{-(p-\ell-m+z-k)/2}(st)^{7z'+9(p-2m)}
	\end{align*}
	Now we note that $p-2m\geq 2(\ell-m)$, thus $p-\ell-m\geq (p-2m)/2$.
	 Therefore this is upper bounded by
	\begin{align*}
		n^{2st-m-2}n^{-(p-2m)/4}n^{-(z-k)/2}(st)^{7z'+9(p-2m)}
	\end{align*}
	Since we have proved above that $z-k\geq (z'-2)/2$, we can sum this geometric series
	satisfying constraints that $z'\leq 2(z-k)+2,z-k\geq 0, p\geq 2m$ and $p-\ell+z-k>0$.
	\begin{align*}
		& \sum_{p\geq 2m} \sum_{\substack{z-k\geq 0\\z-k+p-2m>0}}\sum_{z'\leq 2(z-k)+2}
		n^{2st-m-2}n^{-(p-2m)/4}n^{-(z-k)/2}(st)^{7z'+9(p-2m)}\\
		\leq & \sum_{p\geq 2m} \sum_{\substack{z-k\geq 0\\z-k\geq 2m+1-p}} 
		n^{2st-m-2}n^{-(p-2m)/4}n^{-(z-k)/2}(st)^{14(z-k)+15+9(p-2m)}\\
		\leq & \sum_{p\geq 2m} n^{2st-m-2} n^{-\text{max}(1,p-2m)/4} (st)^{30+9(p-2m)}\\
		\leq & n^{-0.24}n^{2st-m-2}
	\end{align*}
	Therefore we obtain that 
	$$\nmultig{s}{t,u,u,m}\setminus \text{CyclePair}_{u,u,m}\leq o(\Card{\text{CyclePair}_{u,u,m}})$$
	The claim thus follows. 
\end{proof}

Now the \cref{lem:second-moment-counting-nice-multigraphs} is easy to prove.
\begin{proof}
	By the last lemma, we have
	\begin{align*}
		\Card{\nmultig{s,t}{u,u, m}}\leq 2\Card{\text{CyclePair}_{u,u,m}}
	\end{align*}
	Thus for $m\neq st$,
	\begin{align*}
		 \Card{\nmultig{s,t}{u,u, m}}\leq 4stn^{2st-m-2}
	\end{align*}
	For $m=st$, we have
	\begin{align*}
		\Card{\nmultig{s,t}{u,u,st}}\leq 4n^{st-1}
   \end{align*}	
	Further
	\begin{align*}
		\Card{\nmultig{s,t}{u,u,0}}\geq (1-o(1))n^{2st-m-2}
	\end{align*}
	Therefore it follows that for $m\neq st$
	\begin{align*}
		\Card{\nmultig{s,t}{u,u,m}}\leq \frac{2st}{n^m} \Card{\nmultig{s,t}{u,u, 0}}
	\end{align*}
	and for $m=st$
	\begin{align*}
		\Card{\nmultig{s,t}{u,u,st}}\leq \frac{4}{n^{st-1}} \Card{\nmultig{s,t}{u,u, 0}}
	\end{align*}
\end{proof}

\section{Additional tools}\label{sec:appendix-additional-tools}
We present here some generic tools that are used throughout the paper.
We start by stating  the classical Paley-Zygmund inequality:
\begin{lemma}[Paley-Zygmund inequality]\label{lem:Paley-Zyg}
	If $Z\geq 0$ is a random variable with finite variance, then for $0\leq \theta\leq 1$, we have 
	\[\mathrm{P}(Z>\theta \mathrm{E}[Z]) \geq(1-\theta)^{2} \frac{\mathrm{E}[Z]^{2}}{\mathrm{E}\left[Z^{2}\right]}\]
\end{lemma}

The next lemma relates different norms in $\R^n$.

\begin{lemma}\label{lem:relationship-between-norms}
	Let $v\in \R^n$ be a vector and let $q< p \geq 1$.
	Then 
	\begin{align*}
		\Norm{v}_p \leq n^{1/p-1/q}\Norm{v}_q\,.
	\end{align*}
	\begin{proof}
		Let $r=q/p$. Applying \Holder's inequality
		\begin{align*}
			\Norm{v}^p_p \leq \Paren{\underset{i\in [n]}{\sum} v_i^{q}}^{1/r}\cdot \Paren{\underset{i \in [n]}{\sum} 1}^{1-1/r}=n^{1-p/q}\Paren{\underset{i\in [n]}{\sum} v_i^{q}}^{p/q}\,.
		\end{align*}
		Taking the $p$-th root the lemma follows.
	\end{proof}
\end{lemma}

\begin{lemma}[Courant-Fischer min-max theorem]\label{thm:courant-min-max-theorem}
    Let $Z$ be an $n\times n$ Hermitian matrix with eigenvalues $\lambda_1\leq \lambda_2\ldots\leq\lambda_k\leq\ldots\leq \lambda_n$, then we have
    \begin{equation*}
        \lambda_{k+1}=\max _{U}\left\{\min _{x}\left\{R_{Z}(x) \mid x \in U \text { and } x \neq 0\right\} \mid \operatorname{dim}(U)=n-k\right\}
    \end{equation*}
    where $R_{Z}(x)=\frac{\iprod{Z,xx^\top}}{\Norm{x}^2}$
\end{lemma}

The fact below counts the number of non-isomorphic trees with bounded number of leaves.

\begin{fact}\label{fact:bound-unlabeled-trees}
	Let $2\leq s \leq t$.
	The number of non-isomorphic trees on $st$ vertices with at most $t$ leaves, denoted by $T(st,t)$ is at most
	\begin{align*}
		T(st,t)\leq 2t\cdot \Paren{8e\cdot s  }^{2t}\,.
	\end{align*}
	\begin{proof}
		Each non-isomorphic tree can be encoded with a length-$2st$ planar code.
		There is a one-to-one mapping between leaves in the graph and flips from $1$ to $0$ in the code. Hence there are at most $2t$ flips in the code.
		The number of length-$2st$ binary codes with at most $2t$ flips is bounded by
			\begin{align*}
				\binom{2st}{2t} \cdot 2^{2t}\leq  \Paren{8e\cdot  s}^{2t}\,.
			\end{align*}
	\end{proof}
\end{fact}

\end{document}